%% file: arxiv.tex
\documentclass[10pt,twocolumn,letterpaper]{article}
\usepackage[pagenumbers]{cvpr}

\usepackage{algorithm}
\usepackage{algpseudocode}
\usepackage{amsmath}
\usepackage{amssymb}
\usepackage{amsthm}
\usepackage{booktabs}
\usepackage{enumitem}
\usepackage{gensymb} %
\usepackage{multirow, makecell}
\usepackage{soul}
\usepackage{bbm}
\usepackage{subcaption}
\usepackage{xcolor, color, colortbl}

\usepackage{tikz}
\usetikzlibrary{spy,backgrounds}

\usepackage[pagebackref,breaklinks,colorlinks]{hyperref}
\input{mycommand}

\usepackage[capitalize]{cleveref}

\crefname{section}{Sec.}{Secs.}
\Crefname{section}{Section}{Sections}
\Crefname{table}{Table}{Tables}
\crefname{table}{Tab.}{Tabs.}

\newcommand{\methodname}{GES~}

\def\cvprPaperID{***} %
\def\confName{CVPR}
\def\confYear{2023}

\begin{document}

\title{\methodname: Generalized Exponential Splatting for Efficient Radiance Field Rendering}

\author{
\textbf{Abdullah Hamdi$^{1}$}\quad \quad
\textbf{Luke Melas-Kyriazi$^{1}$}\quad \quad
\textbf{Jinjie Mai$^{2}$} \quad \quad
\textbf{Guocheng Qian$^{2,4}$} \\
\textbf{Ruoshi Liu$^{3}$}\quad \quad
\textbf{Carl Vondrick$^{3}$}\quad \quad
\textbf{Bernard Ghanem$^{2}$}\quad \quad
\textbf{Andrea Vedaldi$^{1}$}
\\
$^1$Visual Geometry Group, University of Oxford\\
$^2$King Abdullah University of Science and Technology (KAUST) \\ $^3$Columbia University \quad \quad \quad $^4$Snap Inc.\\ \small
\texttt{abdullah.hamdi@eng.ox.ac.uk}
}

\maketitle

\input{sections/abstract}
\input{sections/introduction}
\input{sections/relatedwork}
\input{sections/preliminary}

\input{sections/methodology}
\input{figures/sota}
\input{sections/experiments}

\input{sections/results}

\input{figures/fair}
\input{sections/ablation}
\input{figures/laplacian}

\input{tables/ablation_full_mip}

\input{figures/gen3d_vis}
\input{sections/conclusion}

{
\small
\bibliographystyle{ieee_fullname}
\bibliography{main}
}
\clearpage
\appendix
\input{sections/supplement}

\end{document}

%% file: mycommand.tex
\usepackage{bm}

\newcommand{\figlabel}{Fig.\xspace}
\newcommand{\eqlabel}{Eq.\xspace}
\newcommand{\seclabel}{Sec.\xspace}
\newcommand{\secLabel}{Sec.\xspace}

\newcommand{\inlinesection}[1]{\noindent\textbf{#1.}}
\newcommand{\minorsection}[1]{\noindent\textbf{#1.}}
\newcommand{\supp}{\textit{the Appendix}\xspace}

\newcommand{\REMOVAL}[1]{}

\definecolor{Highlight}{HTML}{39b54a}  %

\usepackage{amssymb}%
\usepackage{pifont}%

\usepackage{algorithm}
\usepackage{listings}
\usepackage{etoolbox}
\makeatletter
\AfterEndEnvironment{algorithm}{\let\@algcomment\relax}
\AtEndEnvironment{algorithm}{\kern2pt\hrule\relax\vskip3pt\@algcomment}
\let\@algcomment\relax
\newcommand\algcomment[1]{\def\@algcomment{\footnotesize#1}}
\renewcommand\fs@ruled{\def\@fs@cfont{\bfseries}\let\@fs@capt\floatc@ruled
  \def\@fs@pre{\hrule height.8pt depth0pt \kern2pt}%
  \def\@fs@post{}%
  \def\@fs@mid{\kern2pt\hrule\kern2pt}%
  \let\@fs@iftopcapt\iftrue}
\makeatother
\newcommand{\cmmnt}[1]{}

\usepackage[normalem]{ulem}

\newcommand{\zoomin}[9]{ %
\begin{tikzpicture}[spy using outlines={rectangle,#9,magnification=#8,size=#6}]   
	\node[anchor=south west,inner sep=0]  {\includegraphics[width=#7]{#1}};
	\spy on (#2, #3) in node at (#4,#5);
\end{tikzpicture}
}

\newcommand{\zoomincrop}[9]{ %
\begin{tikzpicture}[spy using outlines={rectangle,#9,magnification=#8,size=#6}]   
	\node[anchor=south west,inner sep=0]  {\includegraphics[width=#7]{#1}};
	\spy on (#2, #3) in node at (#4,#5);
\end{tikzpicture}
}

%% file: sections/abstract.tex
\vspace{-2em}
\begin{abstract}
Advancements in 3D Gaussian Splatting have significantly accelerated 3D reconstruction and generation.
However, it may require a large number of Gaussians, which creates a substantial memory footprint.
This paper introduces \textit{\methodname} (Generalized Exponential Splatting), a novel representation that employs Generalized Exponential Function (GEF) to model 3D scenes, requiring far fewer particles to represent a scene and thus significantly outperforming Gaussian Splatting methods in efficiency with a plug-and-play replacement ability for Gaussian-based utilities.
\methodname is validated theoretically and empirically in both principled 1D setup and realistic 3D scenes. It is shown to represent signals with sharp edges more accurately, which are typically challenging for Gaussians due to their inherent low-pass characteristics.
Our empirical analysis demonstrates that GEF outperforms Gaussians in fitting natural-occurring signals (\eg squares, triangles, parabolic signals), thereby reducing the need for extensive splitting operations that increase the memory footprint of Gaussian Splatting. With the aid of a frequency-modulated loss, 
\methodname achieves competitive performance in novel-view synthesis benchmarks while requiring less than half the memory storage of Gaussian Splatting and increasing the rendering speed by up to 39\%.
The code is available on the project website \url{https://abdullahamdi.com/ges}.
\end{abstract}

%% file: sections/introduction.tex
\section{Introduction}%
\label{sec:intro}
\vspace{-4pt}
\input{figures/pullfig}

\input{figures/part_foureir}
The pursuit of more engaging and immersive virtual experiences across gaming, cinema, and the metaverse demands advancements in 3D technologies that balance visual richness with computational efficiency. In this regard, 3D Gaussian Splatting (GS)~\cite{gaussiansplatter} is a recent alternative to neural radiance fields \cite{NeRF,NerfInTheWild,dnerf,PlenOctrees,plenoxels,InstantNGP} for learning and rendering 3D objects and scenes.
GS represents a scene as a large mixture of small, coloured Gaussians.
Its key advantage is the existence of a very fast differentiable renderer, which makes this representation ideally suited for real-time applications and significantly reduces the learning cost.
Specifically, fast rendering of learnable 3D representations is of key importance for applications like gaming, where high-quality, fluid, and responsive graphics are essential.

However, GS is not without shortcomings.
We notice in particular that GS implicitly makes an assumption on the nature of the modeled signals, which is suboptimal.
Specifically, Gaussians correspond to \emph{low-pass filters}, but most 3D scenes are far from low-pass as they contain abrupt discontinuities in shape and appearance. \figlabel{\ref{fig:partfoureir}} demosntrates this inherent low-pass limitation of Gaussian-based methods. 
As a result, GS needs to use a huge number of very small Gaussians to represent such 3D scenes, far more than if a more appropriate basis was selected, which negatively impacts memory utilization.

To address this shortcoming, in this work, we introduce \textit{\methodname} (Generalized Exponential Splatting), a new approach that utilizes the Generalized Exponential Function (GEF) for modeling 3D scenes (\figlabel{\ref{fig:pullfigure}}).
Our method is designed to effectively represent signals, especially those with sharp features, which previous Gaussian splatting techniques often smooth out or require extensive splitting to model~\cite{gaussiansplatter}.
Demonstrated in \figlabel{\ref{fig:genealized}}, we show that while $N=5$ randomly initialized Gaussians are required to fit a square, only $2$ GEFs are needed for the same signal. This stems from the fact that Gaussian mixtures have a low-pass frequency domain, while many common signals, like the square, are not band-limited.
This high-band modeling constitutes a fundamental challenge to Gaussian-based methods.
To help \methodname to train gradually from low-frequency to high-frequency details, we propose a specialized frequency-modulated image loss. This allows \methodname to achieve more than 50\% reduction in the memory requirement of Gaussian splatting and up to 39\% increase in rendering speed while maintaining a competitive performance on standard novel view synthesis benchmarks.

We summarize our contributions as follows:
\begin{itemize}[leftmargin=1em,topsep=0pt]
\item We present principled numerical simulations motivating the use of the Generalized Exponential Functions (GEF) instead of Gaussians for scene modeling.
\item We propose Generalized Exponential Splatting (\methodname\!\!), a novel 3D representation that leverages GEF to develop a splatting-based method for realistic, real-time, and memory-efficient novel view synthesis. 
\item Equipped with a specialized frequency-modulated image loss and through extensive experiments on standard benchmarks on novel view synthesis, \methodname shows a 50\% reduction in memory requirement and up to 39\% increase in rendering speed for real-time radiance field rendering based on Gaussian Splatting. \methodname can act as a plug-and-play replacement for \textit{any} Gaussian-based utilities. 
\end{itemize}

%% file: figures/pullfig.tex
\begin{figure}[t] 
\centering
\includegraphics[width=\linewidth]{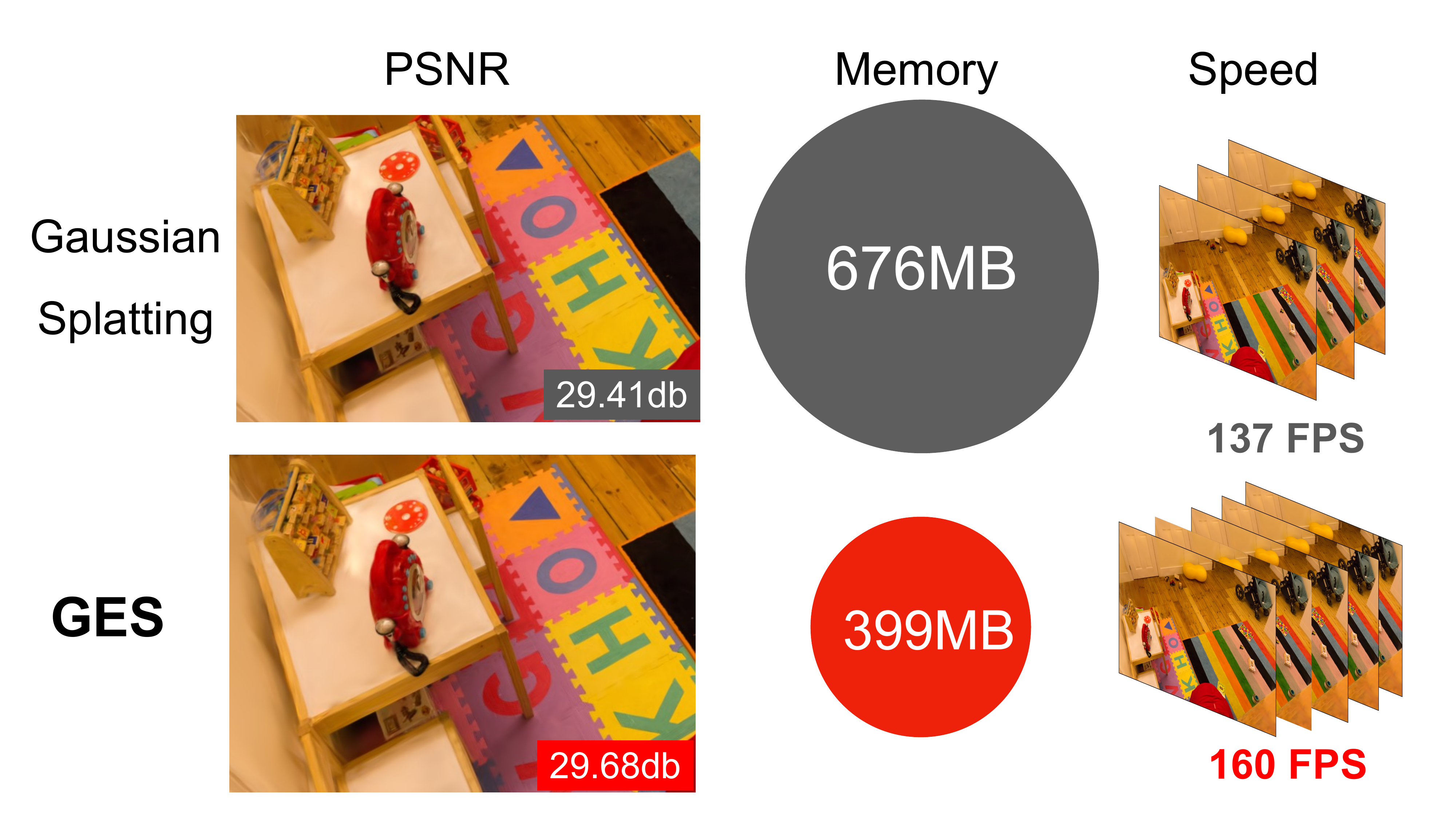}
\caption{
\textbf{GES: Generalized Exponential Splatting} We propose a faster and more memory-efficient alternative to Gaussian Splatting \cite{gaussiansplatter}  that relies on Generalized exponential Functions (with additional learnable shape parameters) instead of  Gaussians.}

\label{fig:pullfigure}
\end{figure}

%% file: figures/part_foureir.tex
\begin{figure*}[t]
\centering
\includegraphics[width=0.99\linewidth]{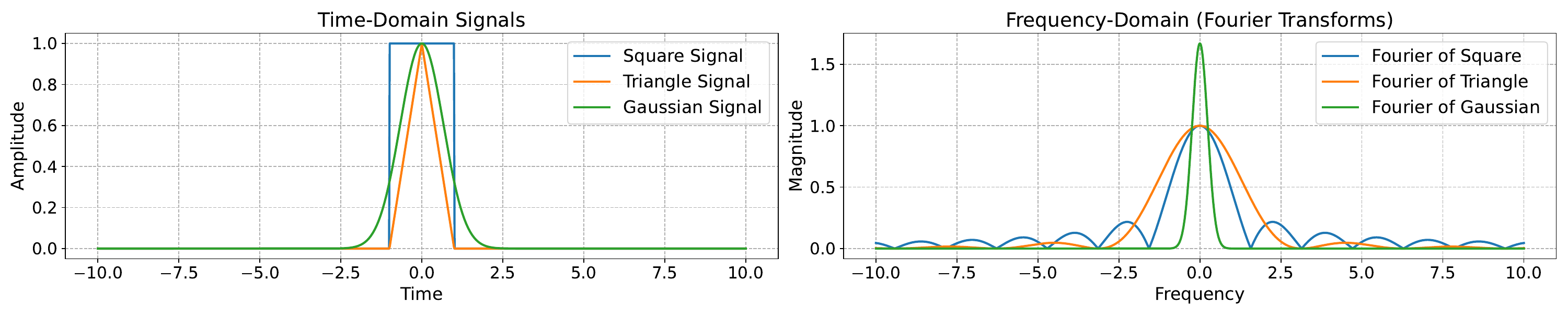}
\caption{\textbf{The Inherent Low-Pass Limitation of Gaussians}. We illustrate the bandwidth constraint of Gaussian functions compared to square and triangle signals. The Gaussian functions' low-pass property restricts their ability to fit signals with sharp edges that have infinite bandwidth. This limitation constitutes a challenge for 3D Gaussian Splatting \cite{gaussiansplatter} in accurately fitting high-bandwidth 3D spatial data.}
\label{fig:partfoureir}
\end{figure*} 

%% file: sections/relatedwork.tex
\input{figures/generalized}
\section{Related work}\label{sec:related}
\vspace{-4pt}
\minorsection{Multi-view 3D reconstruction}
Multi-view 3D reconstruction aims to recover the 3D structure of a scene from its 2D RGB images captured from different camera positions~\cite{UncalibratedStereoRig, RomeInADay}.
Classical approaches usually recover a scene's geometry as a point cloud using SIFT-based~\cite{SIFT} point matching~\cite{colmap_sfm, colmap_mvs}.
More recent methods enhance them by relying on neural networks for feature extraction (\eg \cite{mvsnet, DeepMVS, rmvsnet, fastmvsnet}).
The development of Neural Radiance Fields (NeRF)~\cite{NeRF, NeuralVolumes} has prompted a shift towards reconstructing 3D as volume radiance~\cite{VolumeRenderingDigest}, enabling the synthesis of photo-realistic novel views~\cite{RefNeRF, MipNeRF, MipNeRF-360}.
Subsequent works have also explored the optimization of NeRF in few-shot (\eg \cite{DietNeRF, InfoNeRF, WideBaseline}) and one-shot (\eg \cite{PixelNeRF, EG3D}) settings.
NeRF does not store any 3D geometry explicitly (only the density field), and several works propose to use a signed distance function to recover a scene's surface~\cite{IDR, NeuS, VolSDF, HF-NeuS, PatchNeuS, liu2023you,liu2023humans}, including in the few-shot setting as well (\eg \cite{MonoSDF, NeRS}).

\minorsection{Differentiable rendering}
Gaussian Splatting is a point-based rendering~\cite{gross2011point,aliev2020neural} algorithm that parameterizes 3D points as Gaussian functions (mean, variance, opacity) with spherical harmonic coefficients for the angular radiance component~\cite{yu2021plenoxels}. Prior works have extensively studied differentiable rasterization, with a series of works\cite{loper2014opendr,kato2018neural,liu2019soft} proposing techniques to define a differentiable function between triangles in a triangle mesh and pixels, which allows for adjusting parameters of triangle mesh from observation. These works range from proposing a differentiable renderer for mesh processing with image filters ~\cite{liu2018paparazzi}, and proposing to blend schemes of nearby triangles ~\cite{petersen2019pix2vex}, to extending differentiable rasterization to large-scale indoor scenes ~\cite{yifan2019differentiable}. On the point-based rendering~\cite{gross2011point} side, neural point-based rendering~\cite{kato2018neural} allows features to be learned and stored in 3D points for geometrical and textural information. 
Wiles \etal combine neural point-based rendering with an adversarial loss for better photorealism ~\cite{wiles2020synsin}, whereas later works  use points to represent a radiance field, combining NeRF and point-based rendering ~\cite{xu2022point,zhang2022differentiable}. Our \methodname is a point-based rasterizer in which every point represents a generalized exponential with scale, opacity, and shape, affecting the rasterization accordingly.

\minorsection{Prior-based 3D reconstruction}
Modern zero-shot text-to-image generators~\cite{DALL-E, LDM, DALLE-2, Imagen, ediffi, Make-a-Scene} have improved the results by providing stronger synthesis priors~\cite{DreamFusion, SJC, Latent-NeRF, text2tex, sdfusion}.
DreamFusion~\cite{DreamFusion} is a seminal work that proposed to distill an off-the-shelf diffusion model~\cite{Imagen} into a NeRF~\cite{NeRF, MipNeRF-360} for a given text query.
It sparked numerous follow-up approaches for text-to-3D synthesis (\eg \cite{Magic3D, Fantasia3D}) and image-to-3D reconstruction (\eg \cite{DITTO-NeRF, RealFusion, Zero-1-to-3, deitke2023objaverse}).
The latter is achieved via additional reconstruction losses on the frontal camera position~\cite{Zero-1-to-3} and/or subject-driven diffusion guidance~\cite{DreamBooth3D, Magic3D}.
The developed methods improved the underlying 3D representation~\cite{Magic3D, Fantasia3D, stable-dreamfusion} and 3D consistency of the supervision~\cite{Zero-1-to-3, 3DFuse}; explored task-specific priors~\cite{Text2Room, Farm3D, TEXTure} and additional controls~\cite{SKED}. Lately, Gaussian-based methods \cite{dreemgaussian} improved the speed of optimization of 3D generation, utilizing the fast rasterization of Gaussian Splatting. We showcase how our \methodname can act as a plug-and-play replacement for Gaussian Splatting in this application and other utilities.

%% file: figures/generalized.tex
\begin{figure*}[t] 
\centering
\resizebox{1.02\linewidth}{!}{
\begin{tabular}{c|cc}
\includegraphics[width=0.34\linewidth]{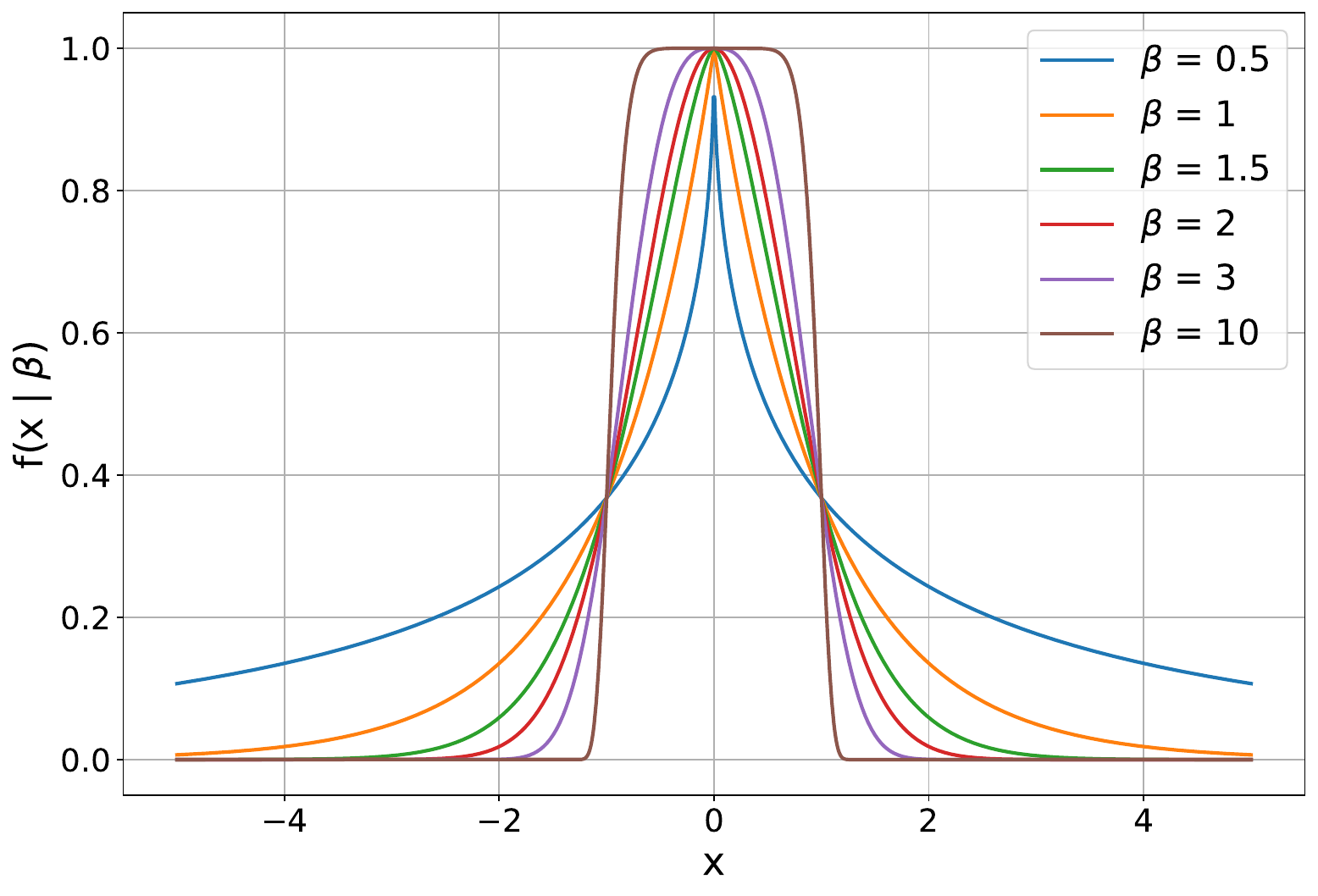}
& \includegraphics[width=0.32\linewidth]{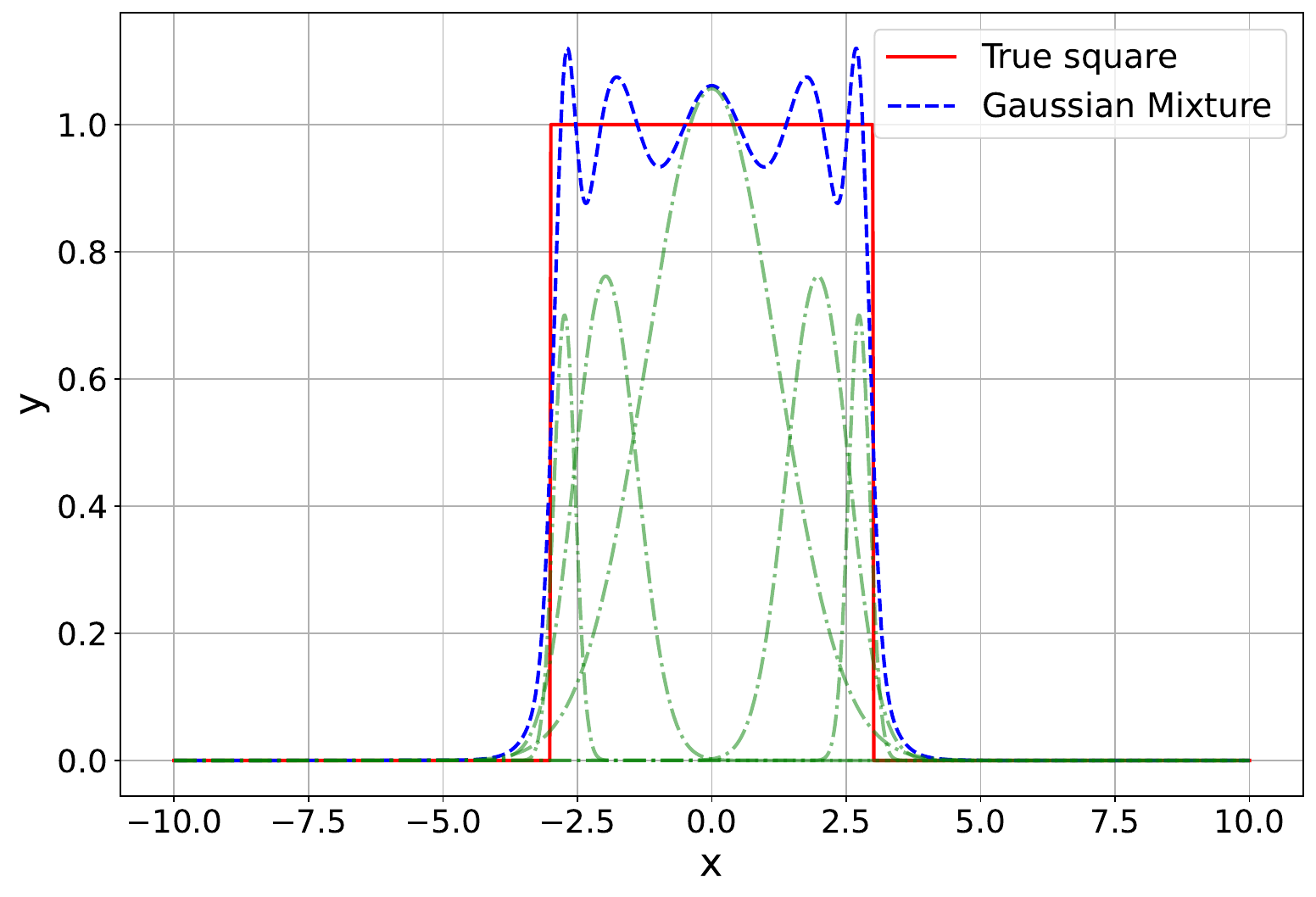} &  \includegraphics[width=0.32\linewidth]{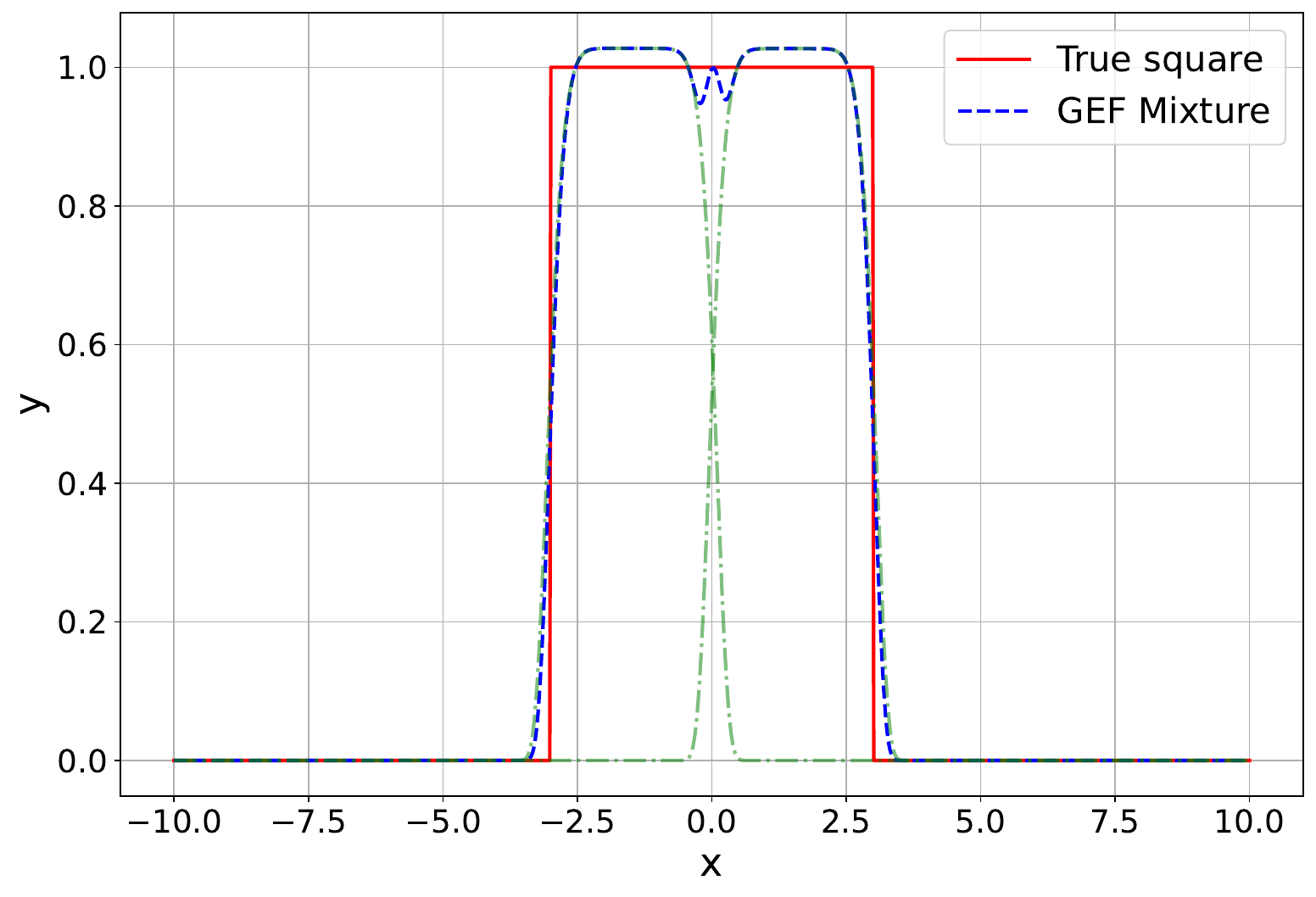} \\
(a) A family of GEFs $ f_{\beta}(x)$ &  (b) Five Gaussians fitting a square & (c) Two GEFs fitting a square
\end{tabular}
}
\caption{\textbf{Generalized Exponential Function (GEF)}. (\textit{a}): We show a family of GEFs $ f_{\beta}(x) = Ae^{-\left(\frac{|x - \mu|}{\alpha}\right)^\beta}$ with different $\beta$ values for $\alpha=1, \mu=0$. When $\beta=2$, the function reduces to the Gaussian function followed in 3D gaussian splatting \cite{gaussiansplatter}. In our \methodname, we learn $\beta$ as another parameter of each splatting component. (\textit{b,c}): The proposed GEF mixture, with learnable $\beta$, fits the same signal (square) with fewer components compared to Gaussian functions using gradient-based optimizations. (\textit{b}): We show an example of the fitted mixture with $N=5$ components when Gaussians are used \vs (\textit{c}) when GEF is used with $N=2$ components. GEF achieves less error loss (0.44) and approximates sharp edges better than the Gaussian counterpart (0.48 error) with less number of components. The optimized individual components (initialized with random parameters) are shown in green after convergence.
}
\label{fig:genealized}
\end{figure*}

%% file: sections/preliminary.tex
\input{figures/theory}
\section{Properties of Generalized Exponentials}\label{sec:prelinminary}
\vspace{-4pt}
\subsection{Generalized Exponential Function}
\vspace{-2pt}
\minorsection{Preliminaries}
The Generalized Exponential Function (GEF) is similar to the probability density function (PDF) of the Generalized Normal Distribution (GND) \cite{generlizedgaussian}. This function allows for a more flexible adaptation to various data shapes by adjusting the shape parameter $ \beta \in (0,\infty) $. The GEF is given by:
\begin{equation}\label{eq:gef}
    f(x | \mu, \alpha, \beta, A) = A \exp\left(-\left(\frac{|x - \mu|}{\alpha}\right)^\beta\right)
\end{equation}
where $ \mu \in \mathbb{R} $ is the location parameter, $ \alpha \in \mathbb{R} $ is the scale parameter, $ A \in \mathbb{R}^{+} $ defines a positive amplitude. The behavior of this function is illustrated in \figlabel{\ref{fig:genealized}}.
For $ \beta = 2 $, the GEF becomes a scaled Gaussian $f(x | \mu, \alpha, \beta=2, A) = A e^{-\frac{1}{2}\left(\frac{x - \mu}{\alpha/\sqrt{2}}\right)^2}$.
 The GEF, therefore, provides a versatile framework for modeling a wide range of data by varying $ \beta $, unlike the Gaussian mixtures, which have a low-pass frequency domain. Many common signals, like the square or triangle, are band-unlimited, constituting a fundamental challenge to Gaussian-based methods. In this paper, we try to \textit{learn} a positive $\beta$ for every component of the Gaussian splatting to allow for a generalized 3D representation. 

\minorsection{Theoretical Results}
Despite its generalizable capabilities, the behavior of the GEF cannot be easily studied analytically, as it involves complex integrals of exponentials without closed form that depend on the shape parameter $\beta$. 
We demonstrate in Theorem \ref{thrm:1} in \supp that for specific cases, such as for a square signal, the GEF can achieve a strictly smaller approximation error than the corresponding Gaussian function by properly choosing $ \beta $. The proof exploits the symmetry of the square wave signal to simplify the error calculations. Theorem \ref{thrm:1} provides a theoretical foundation for preferring the GEF over standard Gaussian functions in our \methodname representation instead of 3D Gaussian Splatting \cite{gaussiansplatter}.

\subsection{Assessing 1D GEF Mixtures in Simulation}
\vspace{-2pt}
We evaluate the effectiveness of a mixture of GEFs in representing various one-dimensional (1D) signal types.
This evaluation is conducted by fitting the model to synthetic signals that replicate characteristics properties of common real-world signals.
More details and additional simulation results are provided in \supp.  

\minorsection{Simulation Setup}
The experimental framework was based on a series of parametric models implemented in PyTorch \cite{Paszke2019Pytorch}, designed to approximate 1D signals using mixtures of different functions such as Gaussian (low-pass), Difference of Gaussians (DoG), Laplacian of Gaussian (LoG), and a GEF mixture model. Each model comprised parameters for means, variances (or scales), and weights, with the generalized model incorporating an additional parameter, $\beta$, to control the exponentiation of the GEF function.

\minorsection{Models}
In this section, we briefly overview the mixture models employed to approximate true signals. Detailed formulations are provided in \supp.

    \textbf{Gaussian Mixture:} This model uses a combination of multiple Gaussian functions. Each Gaussian is characterized by its own mean, variance, and weight. The overall model is a weighted sum of these Gaussian functions, which is a low-pass filter.
    
    \textbf{Difference of Gaussians (DoG) Mixture:} The DoG model is a variation of the Gaussian mixture. It is formed by taking the difference between pairs of Gaussian functions with a predefined variance ratio. This model is particularly effective in highlighting contrasts in the signal and is considered a band-pass filter.
    
    \textbf{Laplacian of Gaussian (LoG) Mixture:} This model combines the characteristics of a Laplacian of Gaussian function. Each component in the mixture has specific parameters that control its shape and scale. Just like the DoG, the LoG model is adept at capturing fine details in the signal and is a band-pass filter.
    
    \textbf{Generalized Exponential (GEF) Mixture:} A more flexible version of the Gaussian mixture, this model introduces an additional shape parameter $\beta$. By adjusting this parameter, we can fine-tune the model to better fit the characteristics of the signal. The GEF Mixture frequency response depends on the shape parameter $\beta$.

\input{figures/viscomparison}

\minorsection{Model Configuration}
The models were configured with a varying number of components $N$, with tests conducted using $N = \{2, 5, 8, 10, 15, 20\}$. The weights of the components are chosen to be positive. All the parameters of all the $N$ components were learned. Each model was trained using the Adam optimizer with a mean squared error loss function. The input $x$ was a linearly spaced tensor representing the domain of the synthetic signal, and the target $y$ was the value of the signal at each point in $x$. Training proceeded for a predetermined number of epochs, and the loss was recorded at the end of training.

\minorsection{Data Generation}
Synthetic 1D signals were generated for various signal types over a specified range, with a given data size and signal width. The signals were used as the ground truth for training the mixture models. The ground truth signals used in the experiment are one-dimensional (1D) functions that serve as benchmarks for evaluating signal processing algorithms. The signal types under study are: \textit{square}, \textit{triangle}, \textit{parabolic}, \textit{half sinusoidal}, \textit{Gaussian}, and \textit{exponential} functions. We show \figlabel{\ref{fig:genealized}} an example of fitting a Gaussian when $N=5$ and a Generalized mixture on the square signal when $N=2$. Note how sharp edges constitute a challenge for Gaussians that have low pass bandwidth while a square signal has an infinite bandwidth known by the sinc function \cite{shannon}.

\minorsection{Simulation Results}
The models' performance was evaluated based on the loss value after training. Additionally, the model's ability to represent the input signal was visually inspected through generated plots. Multiple runs per configuration were executed to account for variance in the results. For a comprehensive evaluation, each configuration was run multiple times (20 runs per configuration) to account for variability in the training process. During these runs, the number of instances where the training resulted in a 'nan' loss was removed from the loss plots, and hence some plots in \figlabel{\ref{fig:loss-stability}} do not have loss values at some $N$. As depicted in \figlabel{\ref{fig:loss-stability}}, the GEF Mixture consistently yielded the lowest loss across the number of components, indicating its effective approximation of many common signals, especially band-unlimited signals like the square and triangle. The only exception is the Gaussian signal, which is (obviously) fitted better with a Gaussian Mixture.

%% file: figures/theory.tex
\begin{figure*}[htbp]
  \centering
  \resizebox{1.02\linewidth}{!}{
  \begin{tabular}{ccc}
    \multicolumn{1}{c}{(a) Square signal} & \multicolumn{1}{c}{(b) Parabolic signal} & \multicolumn{1}{c}{(c) Exponential signal} \\ 

      \includegraphics[width=0.325\textwidth]{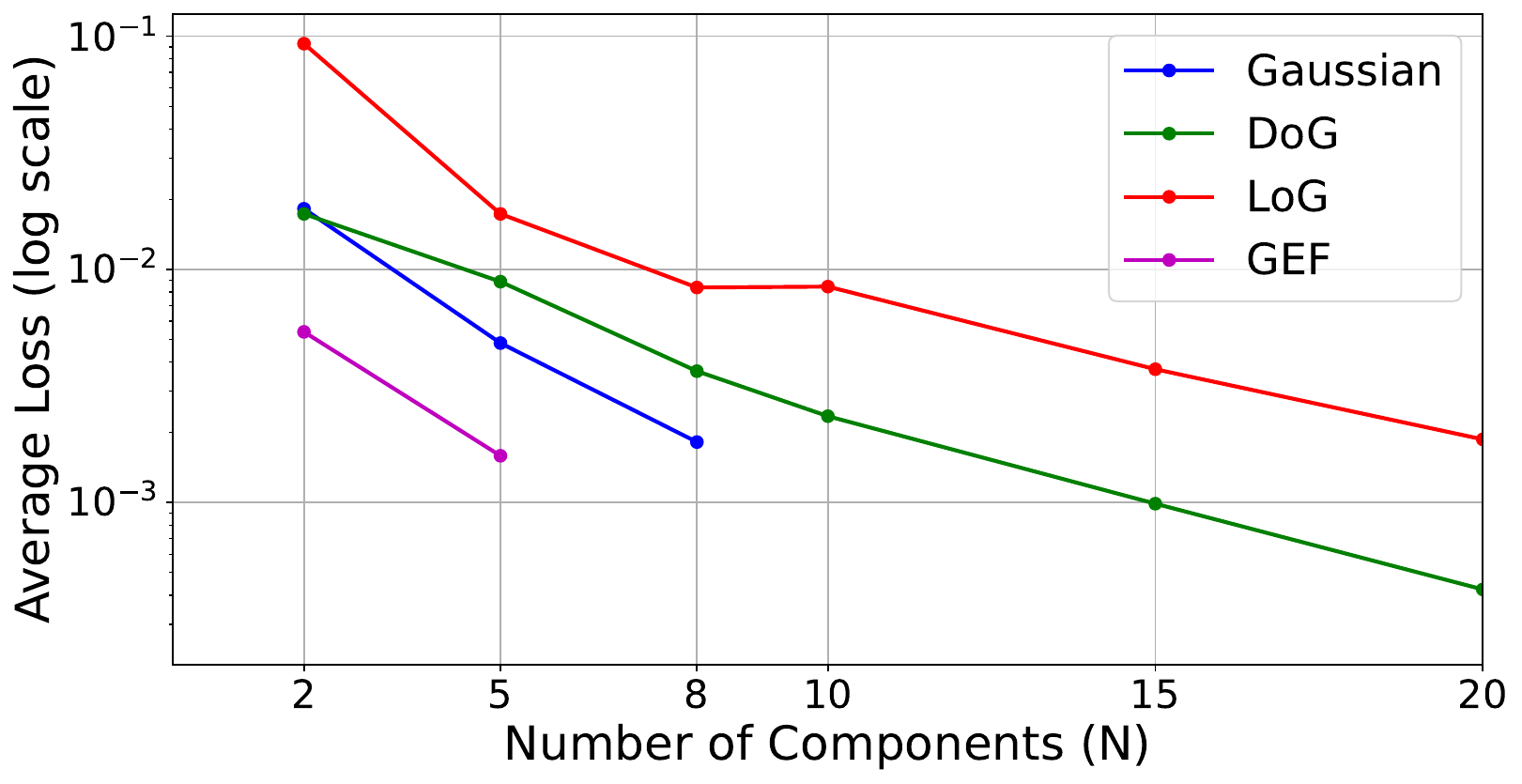} &
      \includegraphics[width=0.325\textwidth]{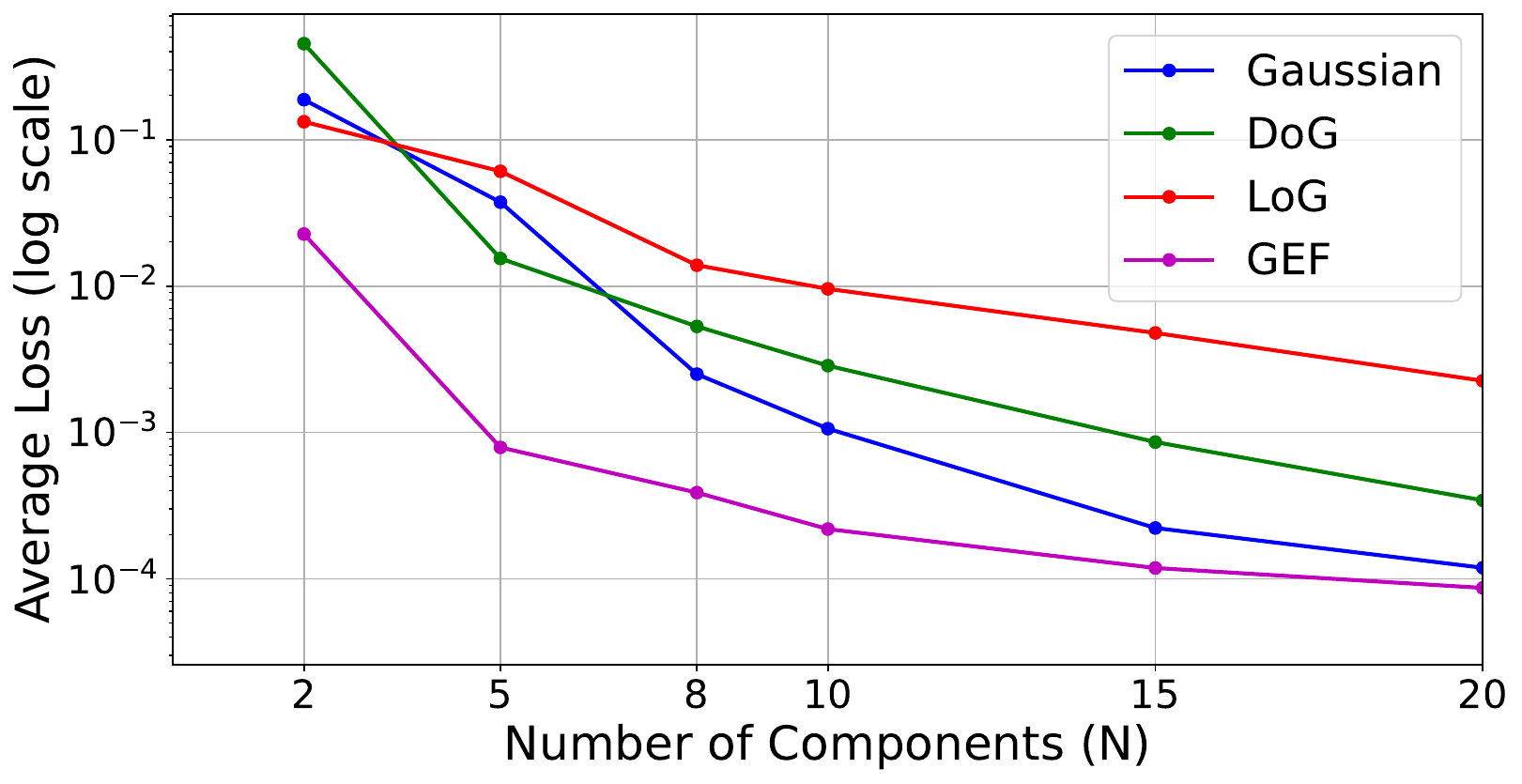} & 
      \includegraphics[width=0.325\textwidth]{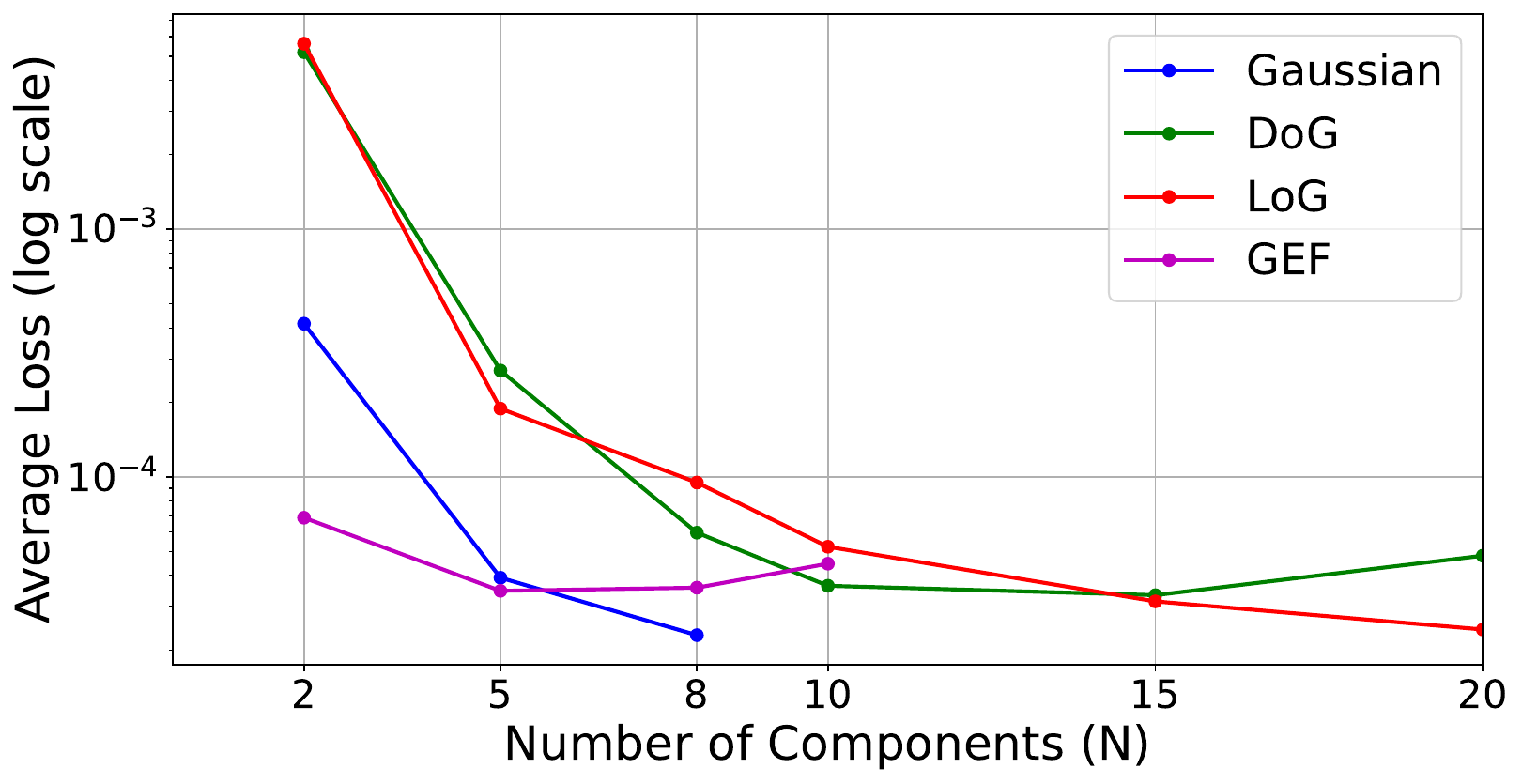} 

      \\ 
      \multicolumn{1}{c}{(d) Triangle signal} & \multicolumn{1}{c}{(e) Gaussian signal} & \multicolumn{1}{c}{(f) Half sinusoid signal} \\ 
      \includegraphics[width=0.325\textwidth]{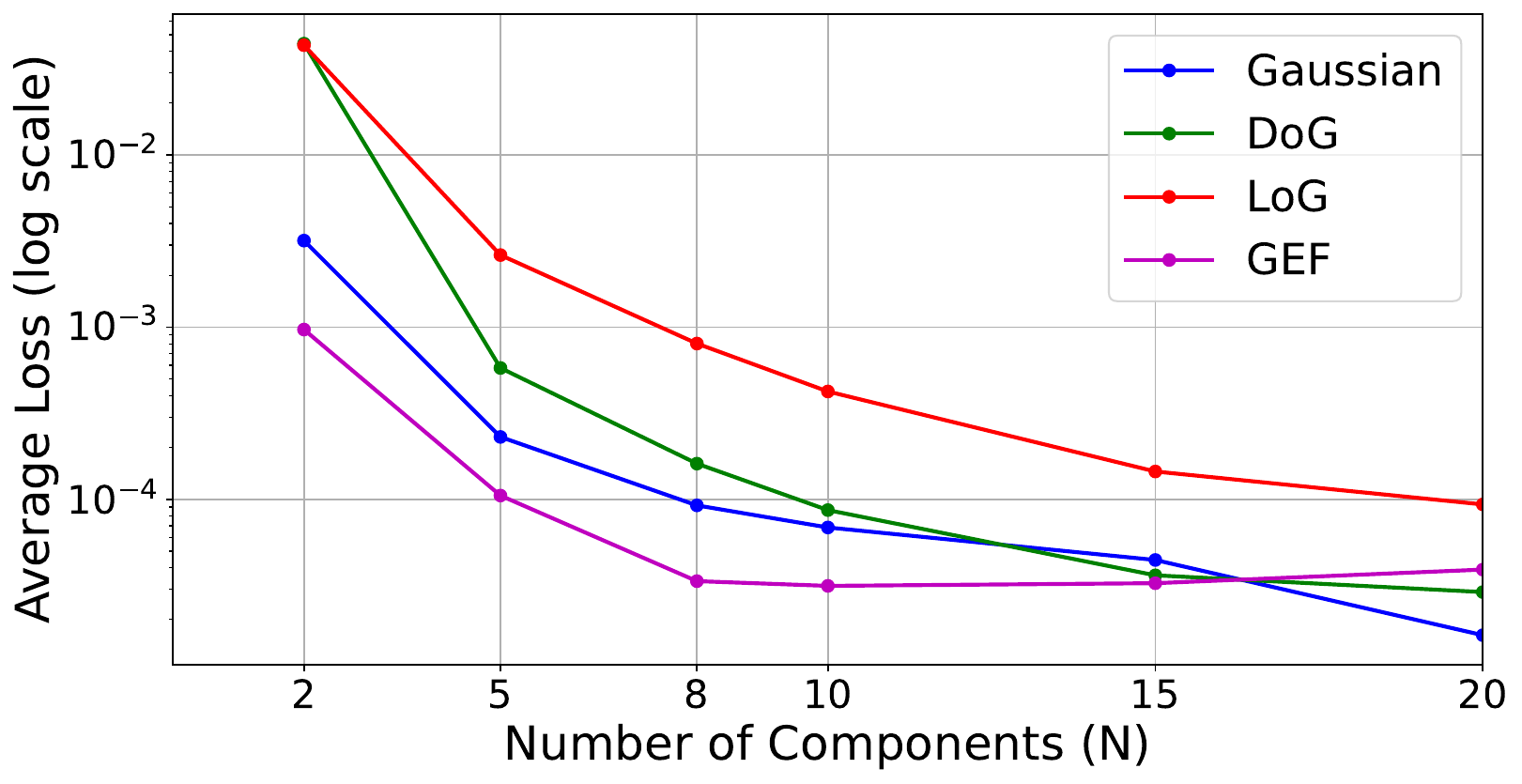} &
      \includegraphics[width=0.325\textwidth]{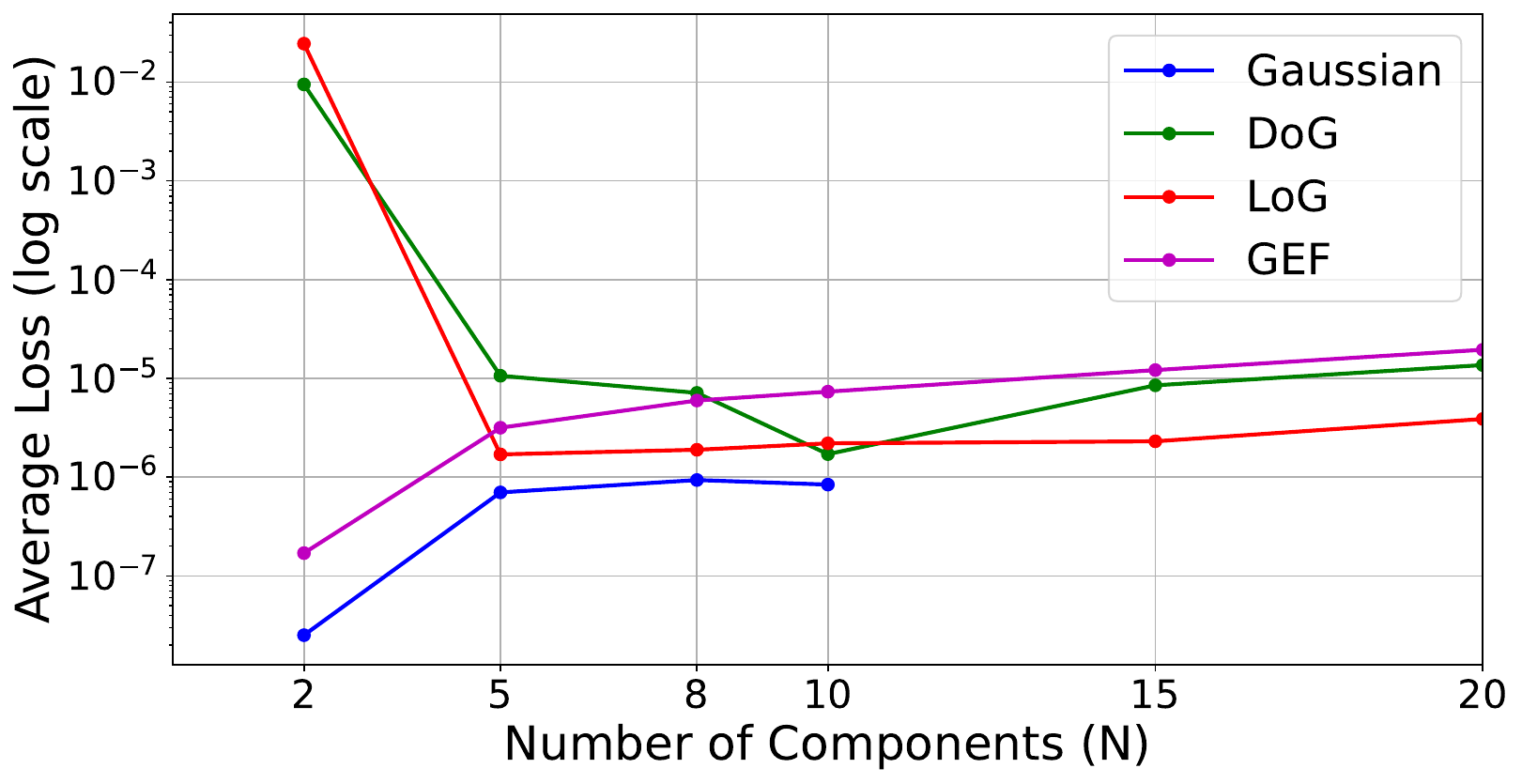} &
      \includegraphics[width=0.325\textwidth]{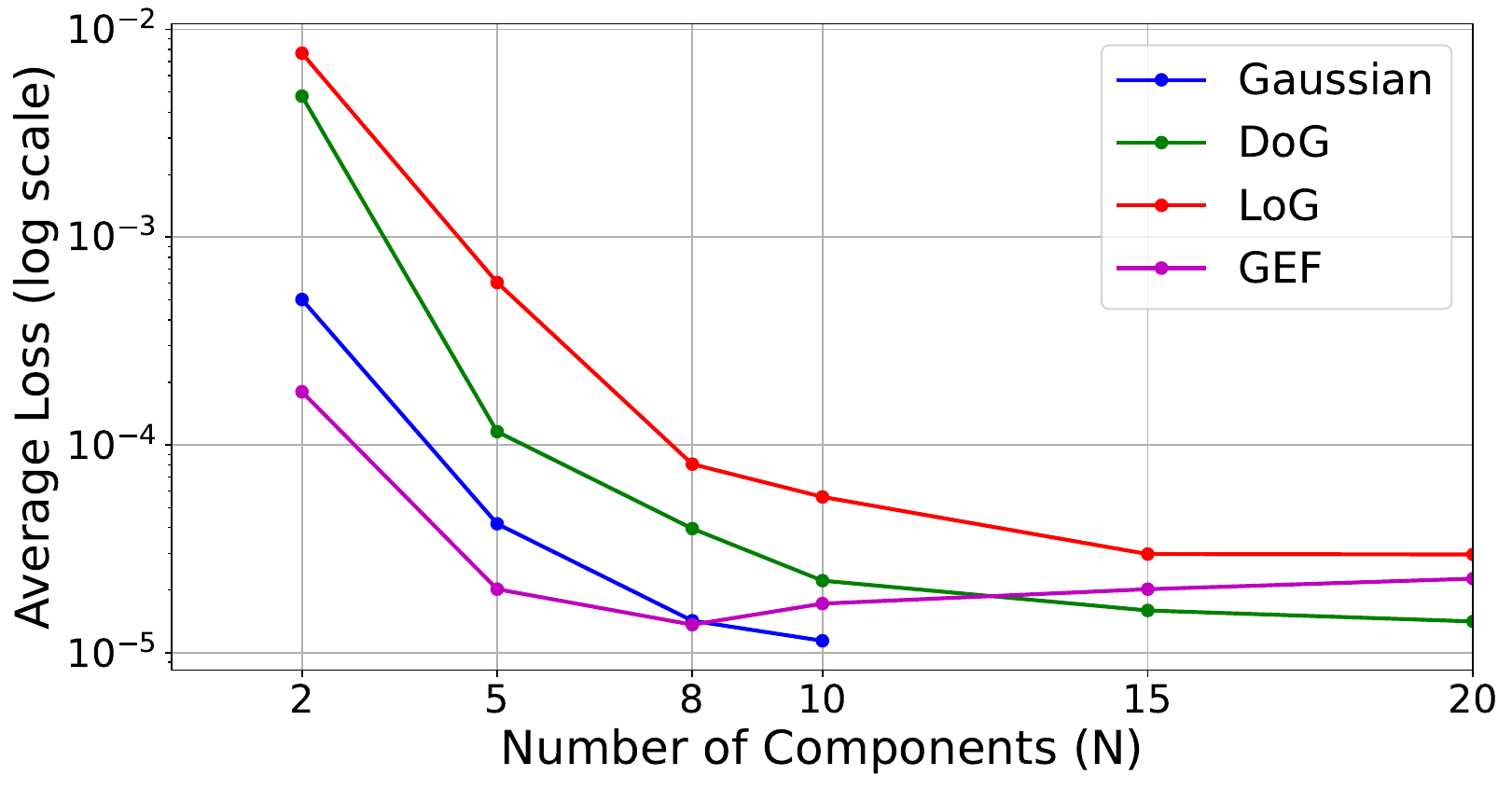} 
  \end{tabular}
  }
  \caption{\textbf{Numerical Simulation Results of Different Mixtures.} We show a comparison of average loss for different mixture models optimized with gradient-based optimizers across varying numbers of components on various signal types (a-f). In the case of `NaN` loss ( gradient explosion), the results are not shown on the plots. Full simulation results are provided in \supp }
  \label{fig:loss-stability}
  \end{figure*}

%% file: figures/viscomparison.tex
\begin{figure*}[t!]
	\newlength\mytmplen
	\setlength\mytmplen{.193\linewidth}
	\setlength{\tabcolsep}{1pt}
	\renewcommand{\arraystretch}{0.5}
	\centering
	\begin{tabular}{cccccc}
		Ground Truth & \methodname (Ours) & Gaussians & Mip-NeRF360 & InstantNGP \\
\zoomin{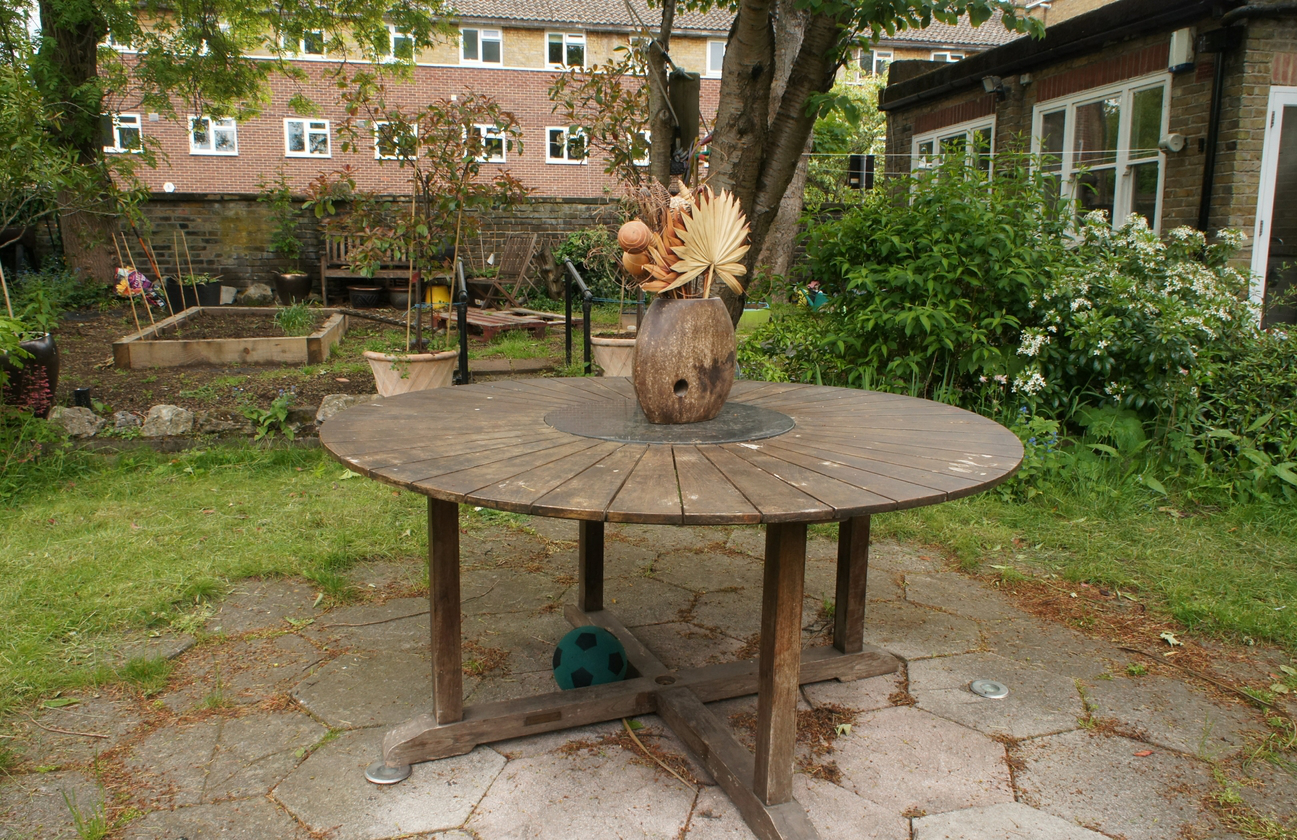}{0.8}{1.65}{0.55cm}{0.55cm}{1cm}{\mytmplen}{3}{red}&
		\zoomin{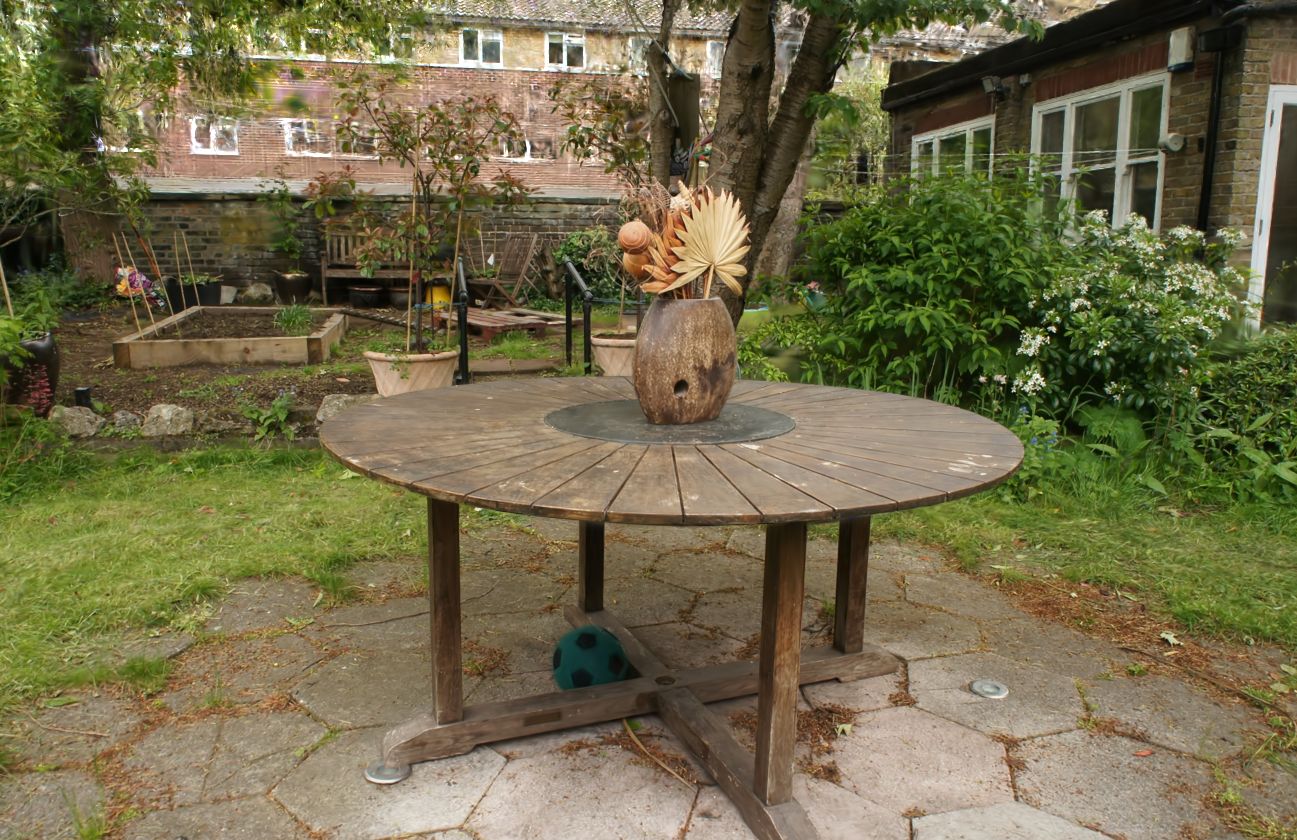}{0.8}{1.65}{0.55cm}{0.55cm}{1cm}{\mytmplen}{3}{red} &
		\zoomin{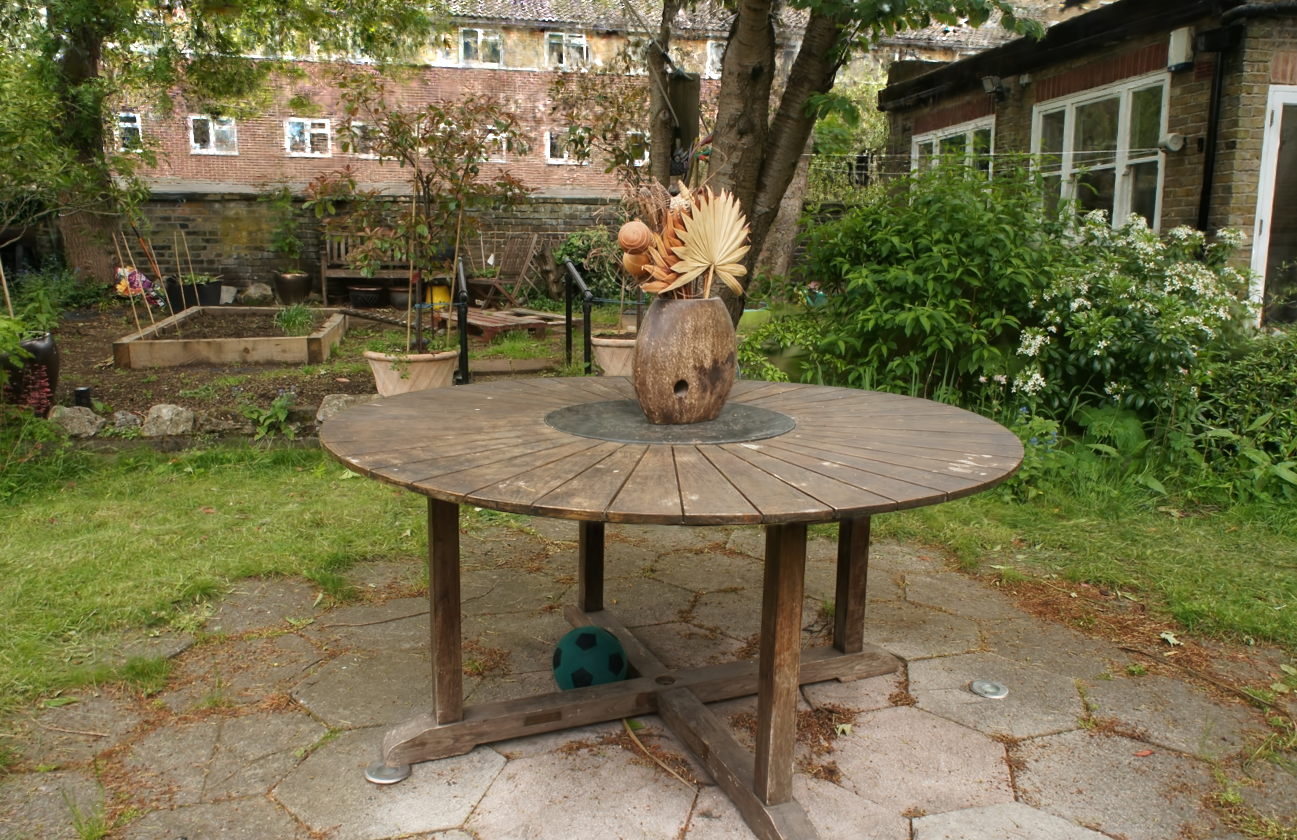}{0.8}{1.65}{0.55cm}{0.55cm}{1cm}{\mytmplen}{3}{red}&
		\zoomin{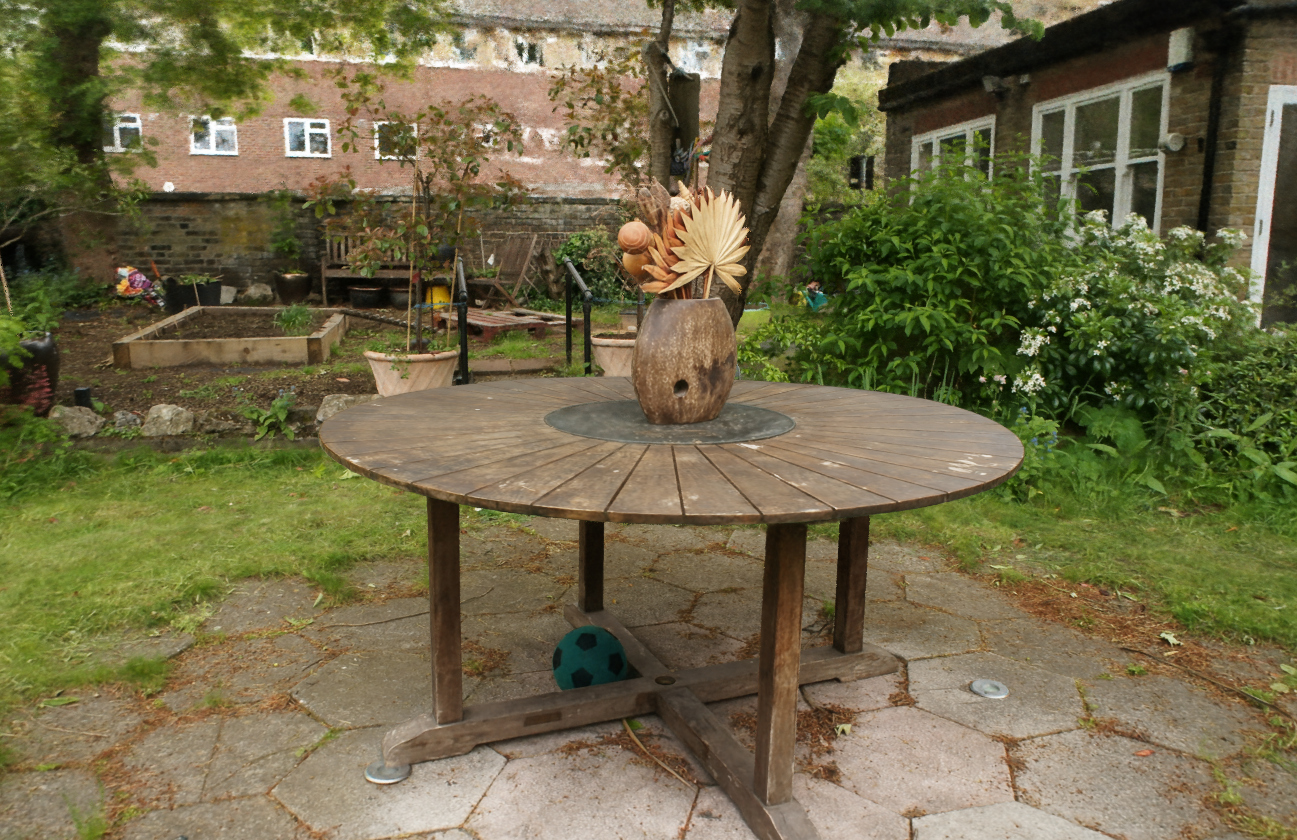}{0.8}{1.65}{0.55cm}{0.55cm}{1cm}{\mytmplen}{3}{red}&
		\zoomin{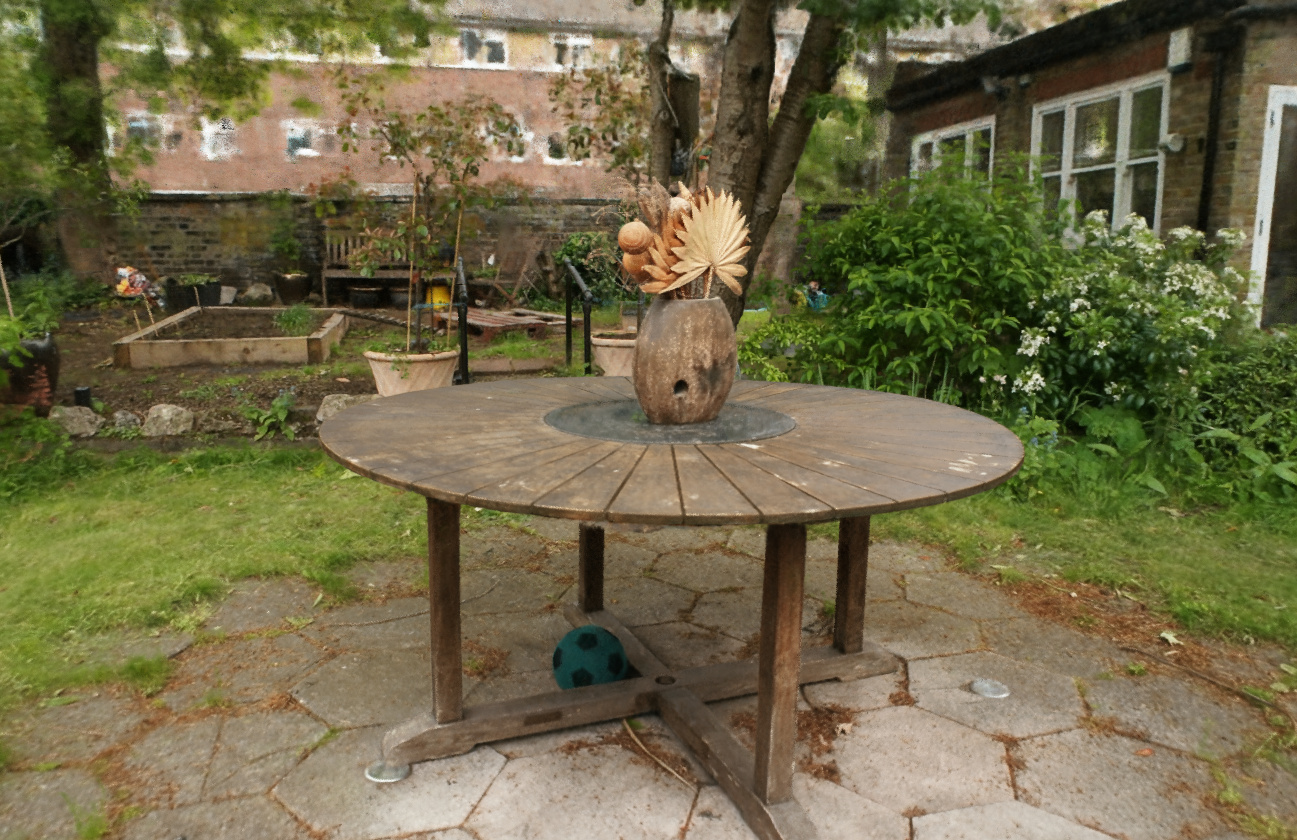}{0.8}{1.65}{0.55cm}{0.55cm}{1cm}{\mytmplen}{3}{red}
\\
		\zoomin{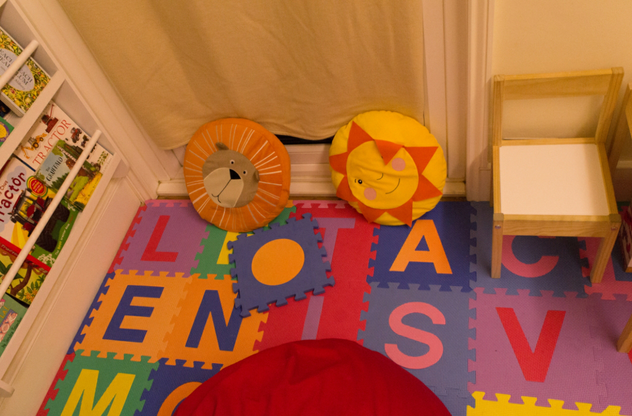}{1.2}{1.4}{0.55cm}{0.55cm}{1cm}{\mytmplen}{3}{green}&
		\zoomin{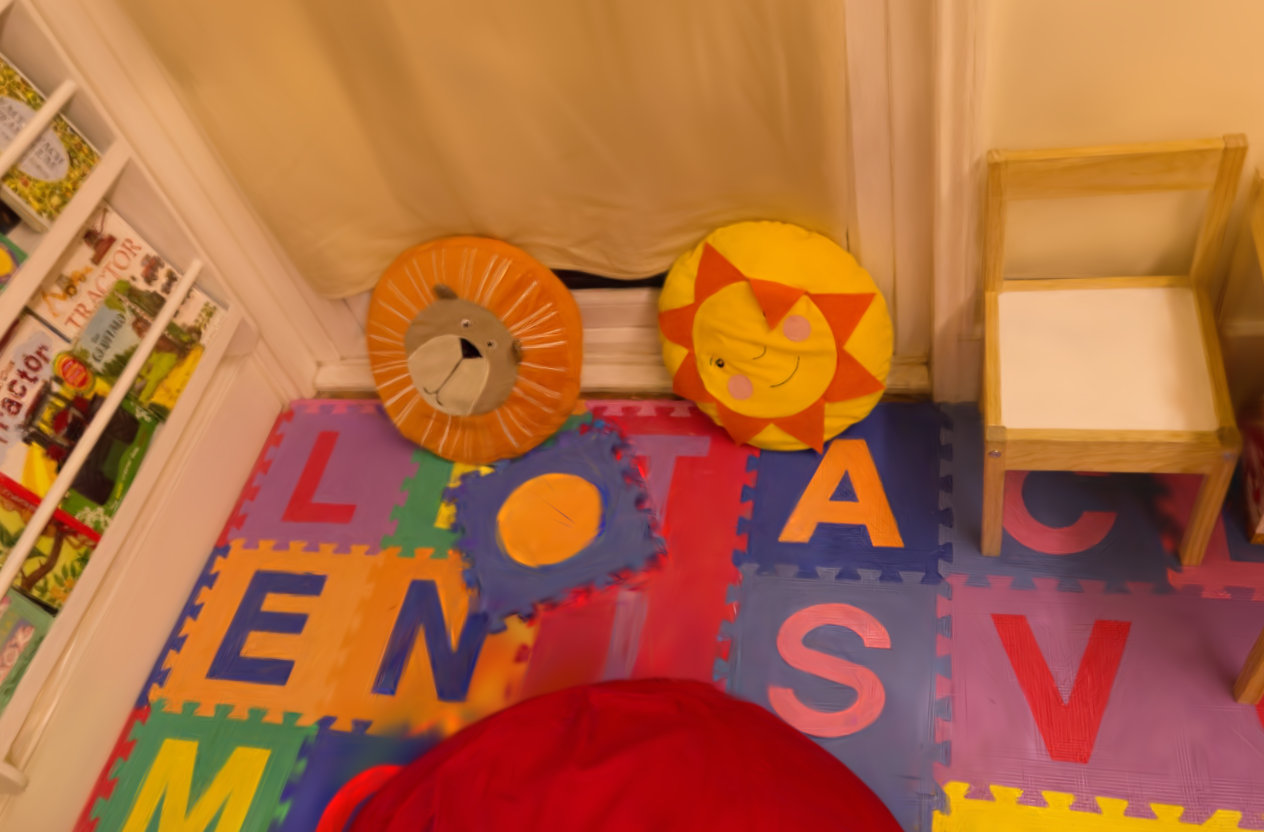}{1.2}{1.4}{0.55cm}{0.55cm}{1cm}{\mytmplen}{3}{green} &
		\zoomin{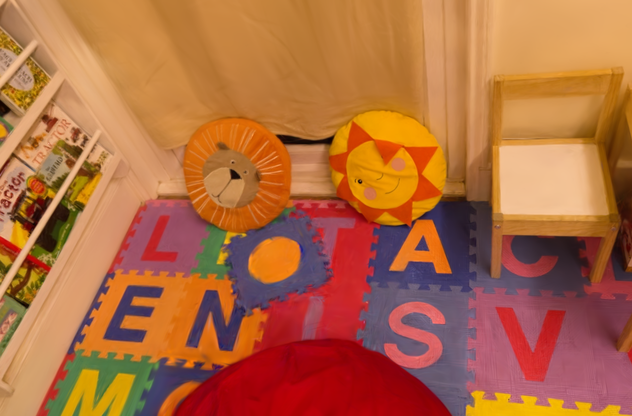}{1.2}{1.4}{0.55cm}{0.55cm}{1cm}{\mytmplen}{3}{green}&
		\zoomin{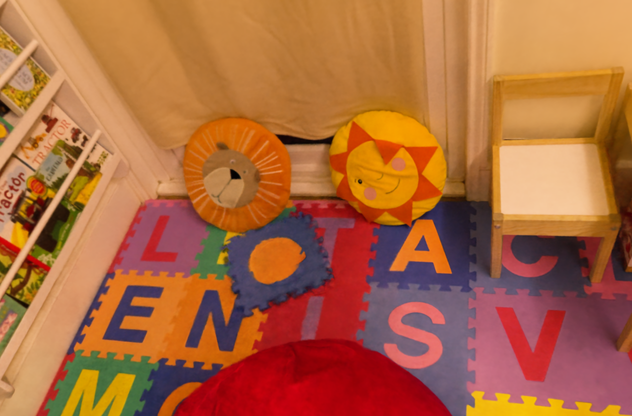}{1.2}{1.4}{0.55cm}{0.55cm}{1cm}{\mytmplen}{3}{green}&
		\zoomin{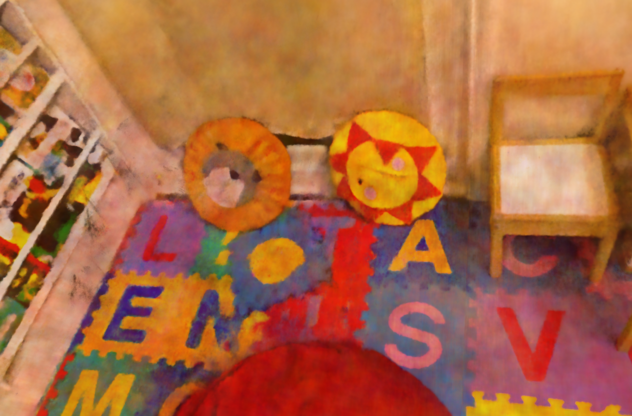}{1.2}{1.4}{0.55cm}{0.55cm}{1cm}{\mytmplen}{3}{green}
\\
		\includegraphics[width=\mytmplen]{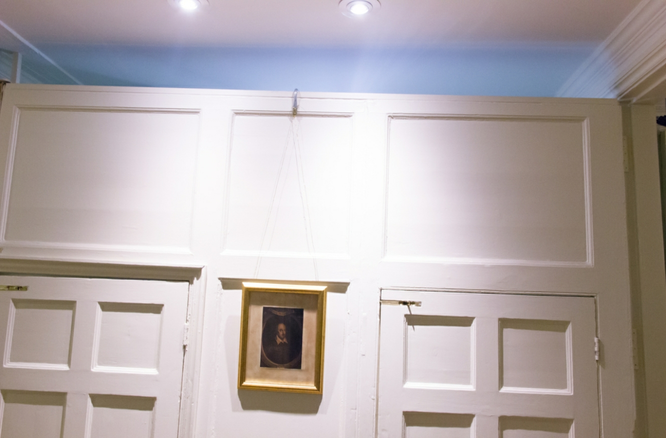} &
		\includegraphics[width=\mytmplen]{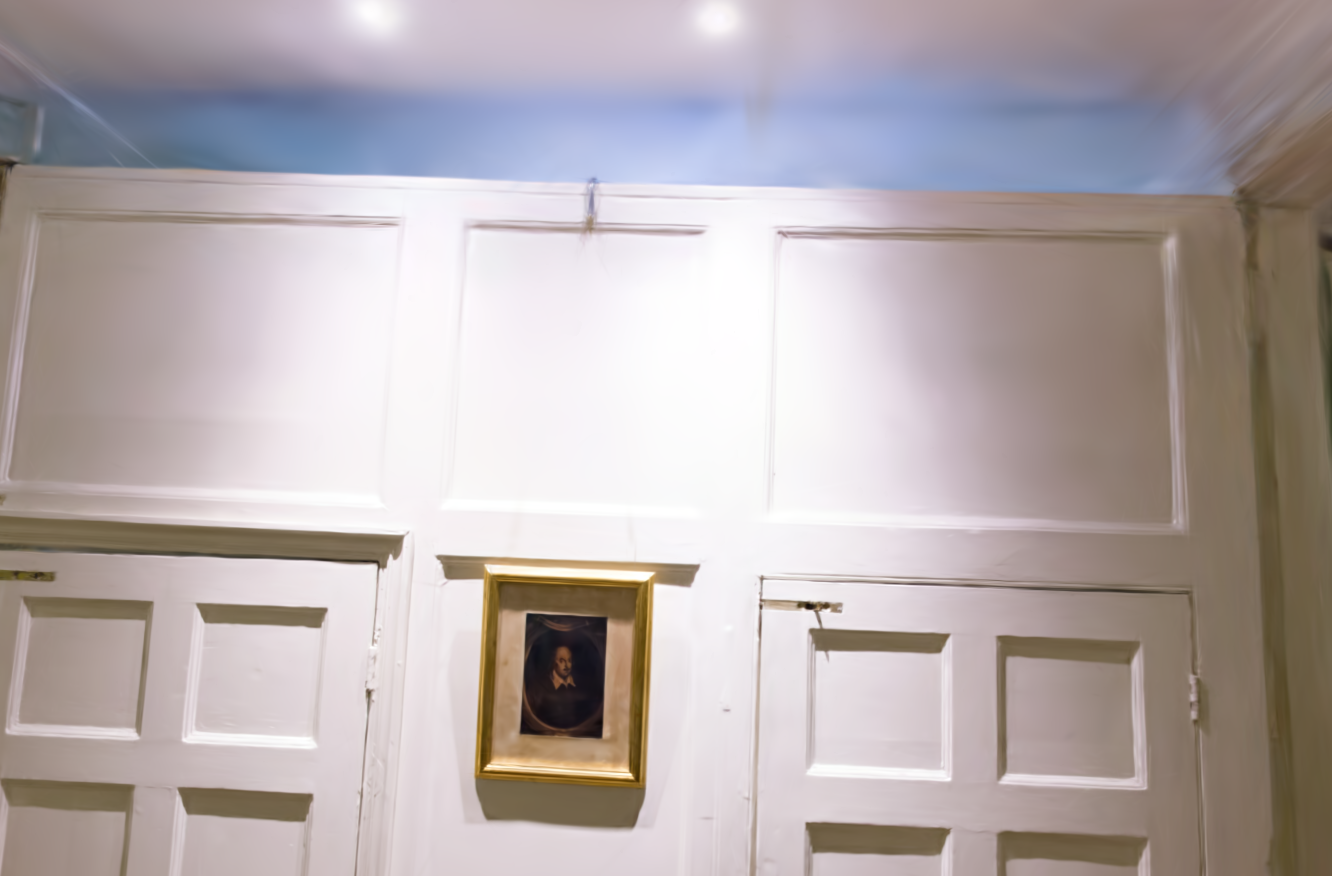} &
		\includegraphics[width=\mytmplen]{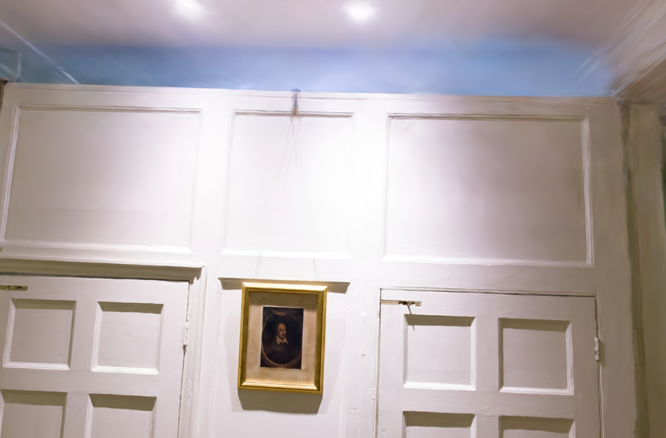} &
		\includegraphics[width=\mytmplen]{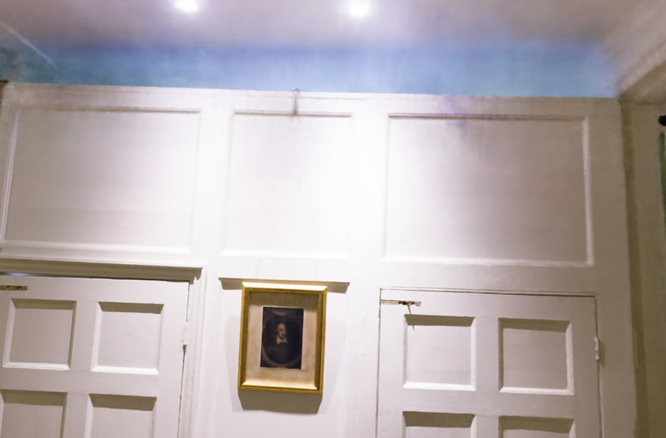} &
		\includegraphics[width=\mytmplen]{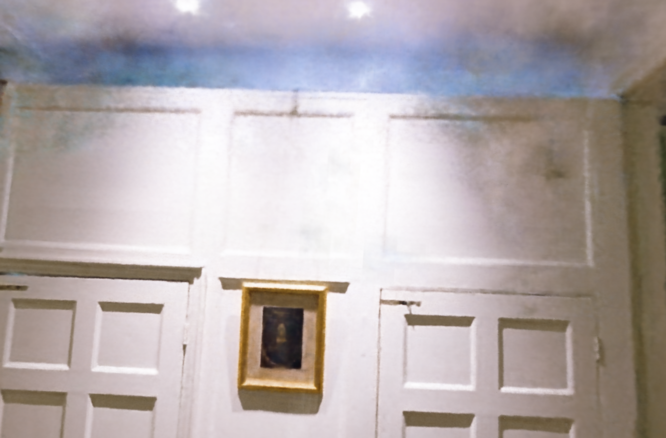} 
\\
		\includegraphics[width=\mytmplen]{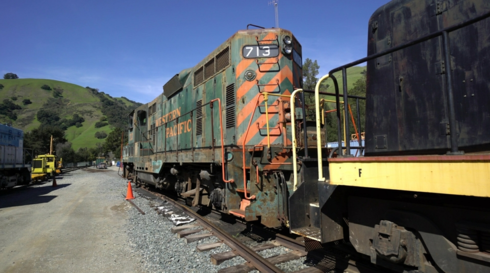} &
		\includegraphics[width=\mytmplen]{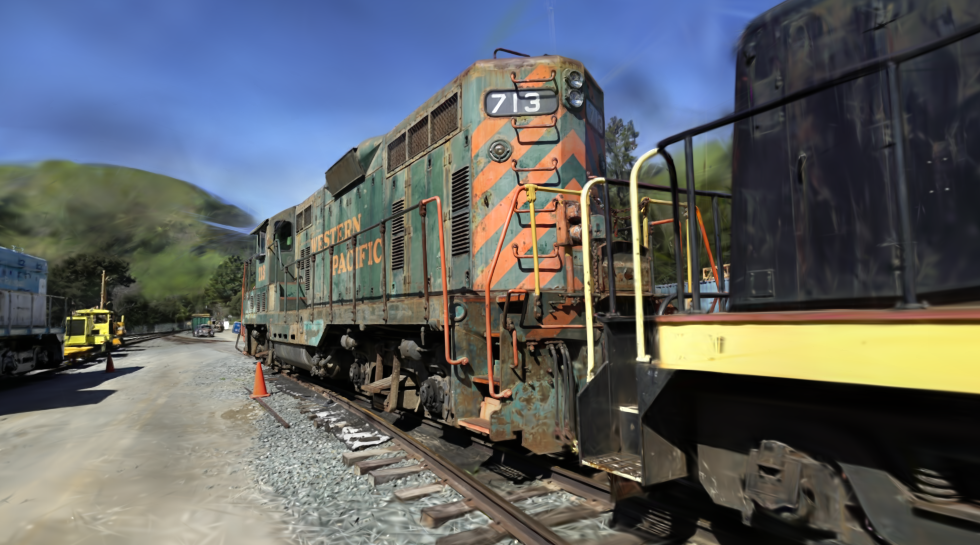} &
		\includegraphics[width=\mytmplen]{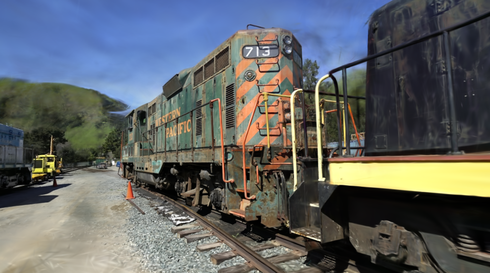} &
		\includegraphics[width=\mytmplen]{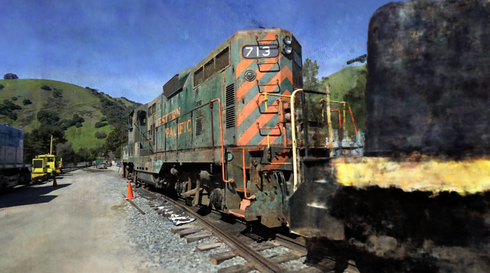} &
		\includegraphics[width=\mytmplen]{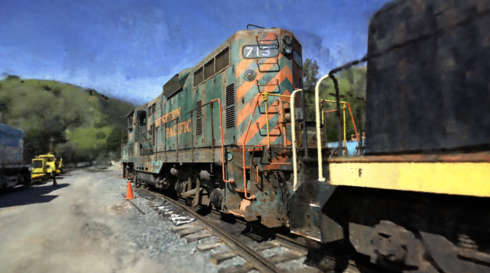}
 \\
	\end{tabular}
    \vspace{-.3cm} %
\caption{
    \label{fig:comparisons}
    \textbf{Visual Comparison on Novel View Synthesis.} We display comparisons between our proposed method and established baselines alongside their respective ground truth images. The depicted scenes are ordered as follows: \textsc{Garden}  and \textsc{Room} from the Mip-NeRF360 dataset; \textsc{DrJohnson} from the Deep Blending dataset; and \textsc{Train} from Tanks\&Temples. Subtle differences in rendering quality are accentuated through zoomed-in details. These specific scenes were picked similarly to Gaussin Splatting \cite{gaussiansplatter} for a fair comparison. It might be difficult in general to see differences between \methodname and Gaussians because they have almost the same PSNR (despite \methodname requiring 50\% less memory).}
	
\end{figure*}

%% file: sections/methodology.tex
\section{Generalized Exponential Splatting (GES)}\label{sec:method} 
\vspace{-4pt}
Having established the benefits of GEF of \eqlabel{\eqref{eq:gef}} over Gaussian functions, we will now demonstrate how to extend GEF into the Generalized Exponential Splatting (GES) framework, offering a plug-and-play replacement for Gaussian Splatting. We also start with a collection of static images of a scene and their corresponding camera calibrations obtained through Structure from Motion (SfM) \cite{schoenberger2016sfm}, which additionally provides a sparse point cloud. Moving beyond Gaussian models \cite{gaussiansplatter}, \methodname adopts an exponent $\beta$ to tailor the focus of the splats, thus sharpening the delineation of scene edges. 
This technique is not only more efficient in memory usage but also can surpass Gaussian splatting in established benchmarks for novel view synthesis.

\subsection{Differentiable GES Formulation}
\label{sec:laplacian-splats}
\vspace{-2pt}
Our objective is to enhance novel view synthesis with a refined scene representation. We leverage a generalized exponential form, here termed Generalized Exponential Splatting, which for location $ \mathbf{x} $ in 3D space and a positive definite matrix $ \boldsymbol{\Sigma} $, is defined by:
\begin{equation} \label{eq:ges}
	L(\mathbf{x}; \boldsymbol{\mu}, \boldsymbol{\Sigma}, \beta) = \exp \left\{ -\frac{1}{2} \big((\mathbf{x} - \boldsymbol{\mu})^{\intercal}\boldsymbol{\Sigma}^{-1}(\mathbf{x} - \boldsymbol{\mu})\big)^\frac{\beta}{2} \right\},
\end{equation}
where $ \boldsymbol{\mu} $ is the location parameter and $ \boldsymbol{\Sigma} $ is the covariance matrix equivalance in Gaussian Splatting\cite{gaussiansplatter}. $ \beta $ is a shape parameter that controls the sharpness of the splat. When $ \beta = 2 $, this formulation is equivalent to Gaussian splatting \cite{gaussiansplatter}.
Our approach maintains an opacity measure $\kappa$ for blending and utilizes spherical harmonics for coloring, similar to Gaussian splatting \cite{gaussiansplatter}.

For 2D image projection, we adapt the technique by Zwicker \etal \cite{zwicker2001ewa}, but keep track of our variable exponent $ \beta $. The camera-space covariance matrix $ \boldsymbol{\Sigma}' $ is transformed as follows:
$
	\boldsymbol{\Sigma}' = \mathbf{J} \mathbf{W} \boldsymbol{\Sigma} \mathbf{W}^{\intercal} \mathbf{J}^{\intercal},
$ 
where $ \mathbf{J} $ is the Jacobian of the transformation from world to camera space, and $ \mathbf{W} $ is a diagonal matrix containing the inverse square root of the eigenvalues of $ \boldsymbol{\Sigma} $.
We ensure $\boldsymbol{\Sigma}$ remains positively semi-definite throughout the optimization by formulating it as a product of a scaling matrix $\mathbf{S}$ (modified by some positive modification function $\phi(\beta) > 0$ as we show later) and a rotation matrix $\mathbf{R}$, with optimization of these components facilitated through separate 3D scale vectors $\mathbf{s}$ and quaternion rotations $\mathbf{q}$.

\subsection{Fast Differentiable Rasterizer for Generalized Exponential Splats}
\label{sec:fast-raster-laplacians}
\vspace{-2pt}
\minorsection{Intuition from Volume Rendering}
The concept of volume rendering in the context of neural radiance fields \cite{NeRF} involves the integration of emitted radiance along a ray passing through a scene. The integral equation for the expected color $C(\mathbf{r})$ of a camera ray $\mathbf{r}(t) = \mathbf{o} + t\mathbf{d}$, with near and far bounds $t_n$ and $t_f$, respectively, is given by:
 \begin{align}
 \begin{aligned} \label{eq:transimttance}
C(\mathbf{r}) = \int_{t_n}^{t_f} T(t) \kappa(\mathbf{r}(t)) c(\mathbf{r}(t), \mathbf{d}) \,dt, \\ \text{where} \quad T(t) = \exp\left(-\int_{t_n}^{t} \kappa(\mathbf{r}(s)) \,ds\right).
 \end{aligned}
 \end{align}
Here, $T(t)$ represents the transmittance along the ray from $t_n$ to $t$, $\kappa(\mathbf{r}(t))$ is the volume density, and $c(\mathbf{r}(t), \mathbf{d})$ is the emitted radiance at point $\mathbf{r}(t)$ in the direction $\mathbf{d}$. The total distance $[{t_n},{t_f}]$ crossed by the ray across non-empty space dictates the amount of lost energy and hence the reduction of the intensity of the rendered colors. In the Gaussian Splatting world \cite{gaussiansplatter}, this distance $[{t_n},{t_f}]$ is composed of the projected variances $\alpha$ of each component along the ray direction $\mathbf{o} + t\mathbf{d}$. In our \methodname of \eqlabel{\eqref{eq:ges}}, if the shape parameter $\beta$ of some individual component changes, the effective impact on \eqlabel{\eqref{eq:transimttance}} will be determined by the effective variance projection $\widehat{\alpha}$ of the same component modified by the modifcation function $\phi(\beta)$ as follows: 
 \begin{align}
 \begin{aligned} \label{eq:effective}
\widehat{\alpha}(\beta) = \phi(\beta)\alpha \quad.
 \end{aligned}
 \end{align}
   Note that the modification function $\phi$ we chose does not depend on the ray direction since the shape parameter $\beta$ is a global property of the splatting component, and we assume the scene to comprise many components. We tackle next the choice of the modification function $\phi$ and how it fits into the rasterization framework of Gaussian Splatting \cite{gaussiansplatter}.

\input{figures/explaination}

\minorsection{Approximate Rasterization}
The main question is how to represent the GES in the rasterization framework. In effect, the rasterization in Gaussian Splatting \cite{gaussiansplatter} only relies on the variance splats of each component. So, we only need to simulate the effect of the shape parameter $\beta$ on the covariance of each component to get the rasterization of \methodname. To do that, we modify the scales matrix of the covariance in each component by the scaler function $\phi(\beta)$ of that component.  From probability theory, the exact conversion between the variance of the generalized exponential distribution and the variance of the Gaussian distribution is given by \cite{generlizedgaussian} as 
\begin{equation} \label{eq:probablity}
\phi(\beta) = \frac{\Gamma(3/\beta)}{\Gamma(1/\beta)} 
\end{equation}
, where $\Gamma$ is the Gamma function. This conversion in \eqlabel{\eqref{eq:probablity}} ensures the PDF integrates to 1. In a similar manner, the integrals in \eqlabel{\eqref{eq:transimttance}} under \eqlabel{\eqref{eq:effective}} can be shown to be equivalent for  Gaussians and \methodname using the same modification of \eqlabel{\eqref{eq:probablity}}. The modification will affect the rasterization \textit{as if} we did perform the exponent change. It is a trick that allows using generalized exponential rasterization without taking the $\beta$ exponent. Similarly, the Gaussian splatting \cite{gaussiansplatter} is \textit{not learning rigid Gaussians}, it learns properties of point clouds that \textit{act as if} there are Gaussians placed there when they splat on the image plane. Both our \methodname and Gaussians are in the same spirit of splatting, and representing 3D with splat properties. \figlabel{\ref{fig:explaination}} demonstrates this concept for an individual splatting component intersecting a ray $\mathbf{r}$ from the camera and the idea of effective variance projection $\widehat{\alpha}$. However, as can be in \figlabel{\ref{fig:explaination}}, this scaler modification $\phi(\beta)$ introduces some view-dependent boundary effect error (\eg if the ray $\mathbf{r}$ passed on the diagonal). We provide an upper bound estimate on this error in \supp . 

Due to the instability of the $\Gamma$ function in \eqlabel{\eqref{eq:probablity}}, we can approximate $\phi(\beta)$ with the following smooth function.
\begin{equation} \label{eq:approx}
\bar{\phi}_{\rho}(\beta)= \frac{2}{1 + e^{-(\rho  \beta - 2 \rho)}}~~.
\end{equation}
The difference between the exact modification $\phi(\beta)$ and the approximate $\bar{\phi}_{\rho}(\beta)$ ( controlled by the hyperparameter shape strength $\rho$ ) is shown in \figlabel{\ref{fig:approximation}}. At $\beta=2$  (Gaussian shape), the modifications $\phi$ and $ \bar{\phi}$ are exactly 1. This parameterization $\bar{\phi}_{\rho}(\beta)$ ensures that the variance of each component remains positive. 
\input{figures/approxmation}

\subsection{Frequency-Modulated Image Loss }
\label{sec:image-laplacian-guidance}
\vspace{-2pt}
To effectively utilize the broad-spectrum capabilities of \methodname, it has been enhanced with a frequency-modulated image loss, denoted as $\mathcal{L}_{\omega}$. This loss is grounded in the rationale that \methodname, initially configured with Gaussian low-pass band splats, should primarily concentrate on low-frequency details during the initial stages of training. As training advances, with the splat formations adapting to encapsulate higher frequencies, the optimization's emphasis should gradually shift towards these higher frequency bands within the image. This concept bears a technical resemblance to the frequency modulation approach used in BARF \cite{barf}, albeit applied within the image domain rather than the 3D coordinate space. The loss is guided by a frequency-conditioned mask implemented via a Difference of Gaussians (DoG) filter to enhance edge-aware optimization in image reconstruction tasks modulated by the normalized frequency $\omega$. The DoG filter acts as a band-pass filter, emphasizing the edges by subtracting a blurred version of the image from another less blurred version, thus approximating the second spatial derivative of the image. This operation is mathematically represented as:
\begin{align*} 
\text{DoG}(I) = G(I, \sigma_1) - G(I, \sigma_2),  ~~~ 0 < \sigma_2 < \sigma_1 
\end{align*}
where $ G(I, \sigma) $ denotes the Gaussian blur operation on image $ I $ with standard deviation $ \sigma $. The choice of $ \sigma $ values dictates the scale of edges to be highlighted, effectively determining the frequency band of the filter. We chose $\sigma_1 = 2 \sigma_2$ to ensure the validity of the band-pass filter, where the choice of $\sigma_2$ will determine the target frequency band of the filter. In our formulation, we use predetermined target normalized frequencies $\omega$ ( $\omega= 0\%$ for low frequencies to $\omega= 100\%$ for high frequencies). We chose $\sigma_2= 0.1+10\omega$ to ensure the stability of the filter and reasonable resulting masks.
The filtered image is then used to generate an edge-aware mask $M_\omega $ through a pixel-wise comparison to a threshold value (after normalization) as follows. 
 \begin{align}
 \begin{aligned} \label{eq:mask}
M_\omega  = \mathbbm{1}\big( & \text{DoG}_{\omega}(I_{\text{gt}})_{\text{normalized}} > \epsilon_{\omega}\big)~~, \\ \text{DoG}_{\omega}(I)& = G(I, 0.2+20\omega) - G(I,0.1+10\omega) 
\end{aligned}
\end{align}
, where $0\leq\epsilon_{\omega}\leq 1$ is the threshold ( we pick 0.5) for a normalized response of the filter $\text{DoG}_{\omega}$, $ I_{\text{gt}} $ is the ground truth image, and $\mathbbm{1}$ is the indicator function. See \figlabel{\ref{fig:mask}} for examples of the masks.
The edge-aware frequency-modulated loss $ \mathcal{L}_{\omega} $ is defined as:
\begin{equation} \label{eq:laplace}
\mathcal{L}_{\omega} = \lVert (I - I_{\text{gt}}) \cdot M_\omega \rVert_1,
\end{equation}
where $ I $ is the reconstructed image, and $ \lVert \cdot \rVert_1 $ denotes the L1 norm. This term is integrated into the overall loss, as shown later. The mask is targeted for the specified frequencies $\omega$. We use a linear schedule to determine these target $\omega$ values in \eqlabel{\eqref{eq:laplace}} and \eqlabel{\eqref{eq:mask}} during the optimization of \methodname, $\omega= \frac{\text{current iteration}}{\text{total iterations}}$. 
The loss $\mathcal{L}_{\omega}$ aims to help in tuning the shape $\beta$ based on the nature of the scene. It does so by focusing the \methodname components on low pass signals first during the training before focusing on high frequency with tuning $\beta$ from their initial values. This helps the \textit{efficiency} of \methodname as can be seen later in Table \ref{tab:ablation-general-tanks} (almost free 9\% reduction in memory).

Due to DoG filter sensitivity for high-frequencies, the mask for $ 0\%< \omega \leq 50 \%$ is defined as $1 - M_\omega $ of $ 50\%<\omega \leq 100\%$. This ensures that all parts of the image will be covered by one of the masks $M_\omega$, while focusing on the details more as the optimization progresses.

\input{figures/masks}
\subsection{Optimization of the Generalized Exponential Splats}
\label{sec:opt-laplaciansplats}

We detail a novel approach for controlling shape density, which selectively prunes \methodname according to their shape attributes, thus eliminating the need for a variable density mechanism. This optimization strategy encompasses the $\beta$ parameter as well as the splat's position $\mathbf{x}$, opacity $\kappa$, covariance matrix $\boldsymbol{\Sigma}$, and color representation through spherical harmonics coefficients \cite{gaussiansplatter}. Optimization of these elements is conducted using stochastic gradient descent, with the process accelerated by GPU-powered computation and specialized CUDA kernels.

Starting estimates for $\boldsymbol{\Sigma}$ and $\mathbf{x}$ are deduced from the SfM points, while all $\beta$ values are initialized with $\beta=2$ (pure Gaussian spalts). The loss function integrates an $\mathcal{L}_1$ metric combined with a structural similarity loss (SSIM), and the frequency-modulated loss$\mathcal{L}_{\omega}$:
\begin{equation} \label{eq:final}
	\mathcal{L} = \lambda_{\text{L1}} \mathcal{L}_1 + \lambda_{\text{ssim}} \mathcal{L}_{\text{ssim}} + 
 \lambda_{\omega} \mathcal{L}_{\omega},
\end{equation}
where $\lambda_{\text{ssim}} = 0.2$ is applied uniformly in all evaluations, and $\lambda_{\text{L1}} = 1 - \lambda_{\text{ssim}} - \lambda_{\omega}  $. Expanded details on the learning algorithm and other specific procedural elements are available in \supp.

%% file: figures/explaination.tex
\begin{figure}[t]
\centering
\includegraphics[trim={5.3cm 0 7.6cm 2.3cm},clip,width=0.75\linewidth]{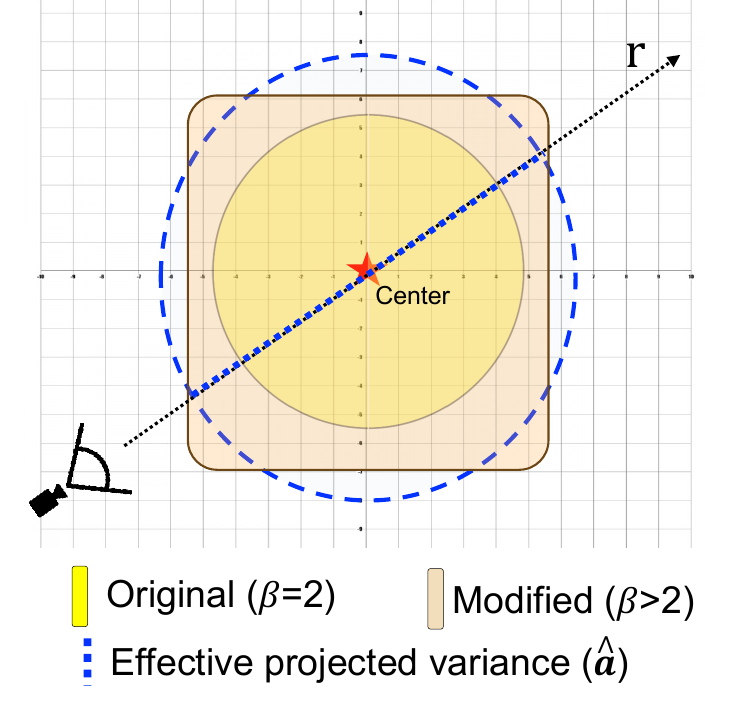}
\vspace{-4pt}
\caption{\textbf{Effective Variance of GES components}. We demonstrate the concept of effective variance projection $\widehat{\alpha}(\beta)$ for an individual splatting component intersecting a camera ray $\mathbf{r}$ under shape modification $(\beta>2)$. Note that $\widehat{\alpha}(\beta)$ is a scaled version of the original splat projected variance $\alpha$.}
\label{fig:explaination}
\end{figure}

%% file: figures/approxmation.tex
\begin{figure}[t]
\centering
\includegraphics[width=0.98\linewidth]{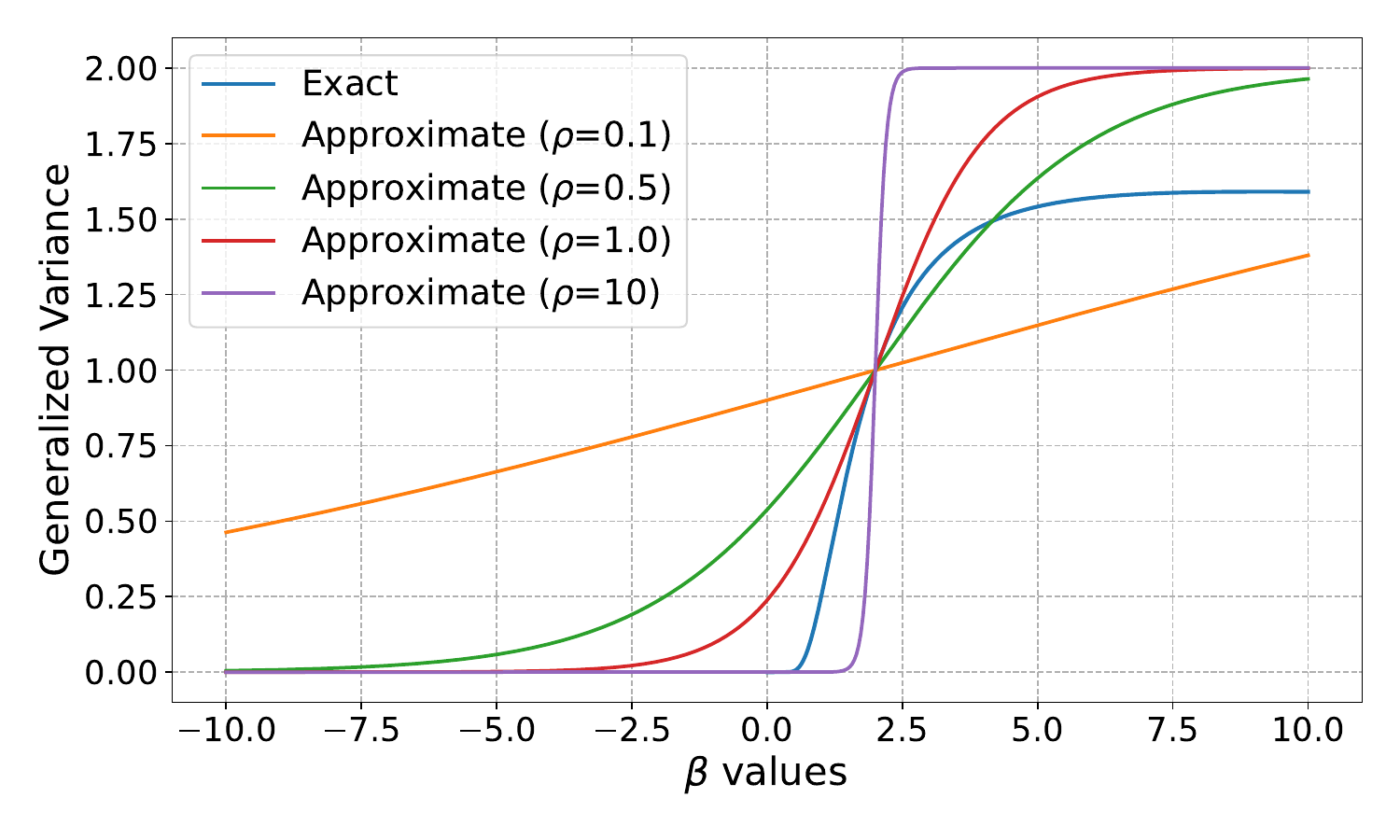}
\vspace{-4pt}
\caption{\textbf{The Modification Function $\phi(\beta)$}. We show different $\rho$ shape strength values of the approximate functions $\bar{\phi}_{\rho}(\beta)$ in \eqlabel{\eqref{eq:approx}} and the exact modification function $\phi(\beta)$ in \eqlabel{\eqref{eq:probablity}}. At $\beta=2$ ( gaussian splats), \textit{all} functions have a variance modification of 1, and \methodname reduces to Gaussian Splatting. In the extreme case of $\rho=0$, \methodname reduces to Gaussian Splatting for \textit{any} $\beta$.}
\label{fig:approximation}
\end{figure}

%% file: figures/masks.tex
\begin{figure*}[t] 
\centering
\includegraphics[width=0.24\linewidth]{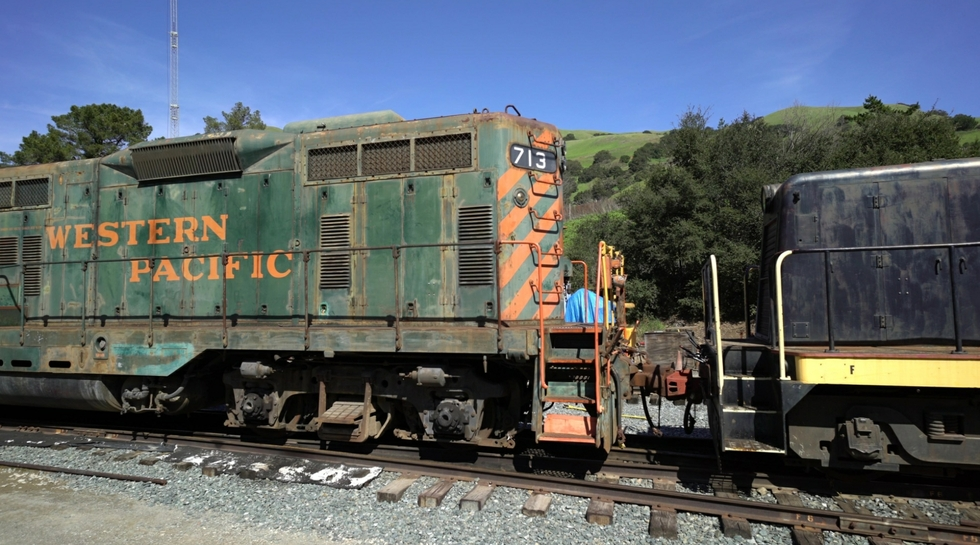}
\includegraphics[trim={1.9cm 2.3cm 1cm 0},clip,width=0.24\linewidth]{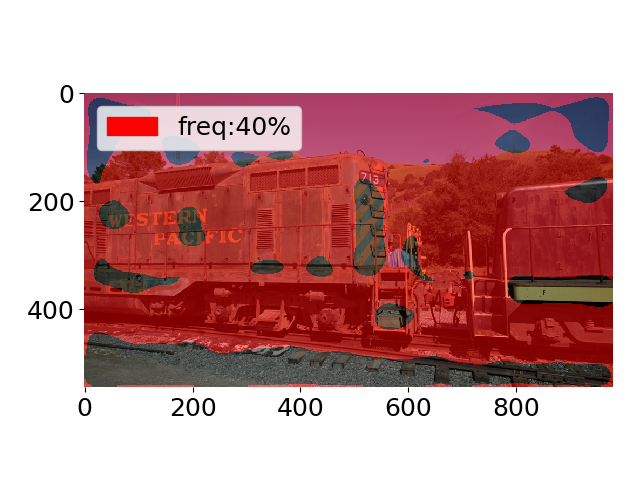}
\includegraphics[trim={1.9cm 2.3cm 1cm 0},clip,width=0.24\linewidth]{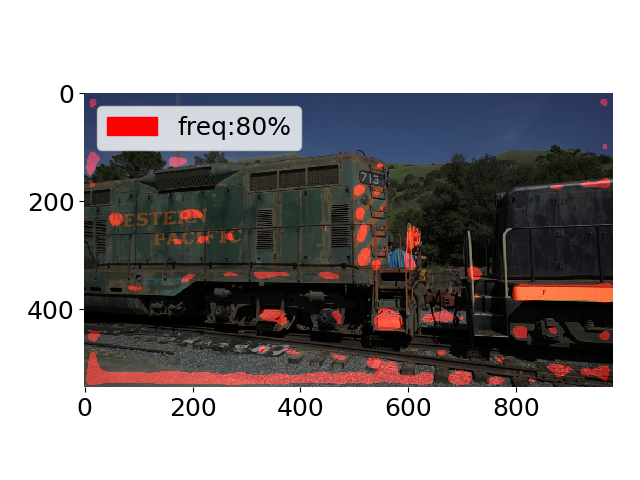}
\includegraphics[trim={1.9cm 2.3cm 1cm 0},clip,width=0.24\linewidth]{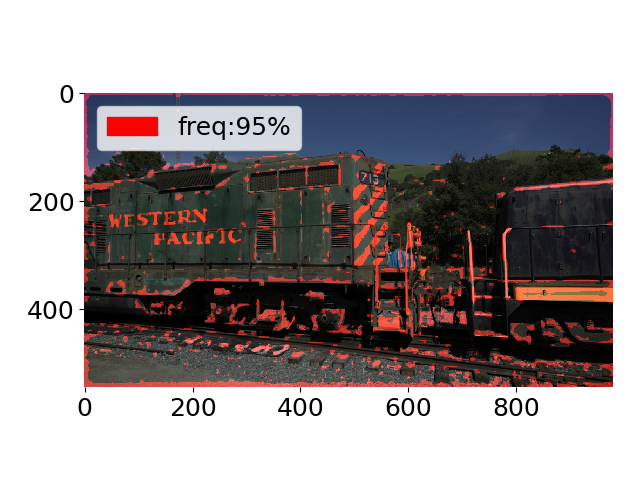}
\vspace{-4pt}
\caption{
\textbf{Frequency-Modulated Image Masks.} For the input example image on the left, We show examples of the frequency loss masks $M_\omega$ used in \seclabel{\ref{sec:image-laplacian-guidance}} for different numbers of target normalized frequencies $\omega$ ( $\omega= 0\%$ for low frequencies to $\omega= 100\%$ for high frequencies). This masked loss helps our \methodname learn specific bands of frequencies. We use a linear schedule to determine these target $\omega$ values during the optimization of \methodname, $\omega= \frac{\text{current iteration}}{\text{total iterations}}$. Note that due to DoG filter sensitivity for high-frequencies, the mask for $ 0< \omega \leq 50 \%$ is defined as $1 - M_\omega $ of $ 50<\omega \leq 100 \%$. This ensures that all parts of the image will be covered by one of the masks $M_\omega$, while focusing on the details more as the optimization progresses.
}
\label{fig:mask}
\end{figure*}

%% file: figures/sota.tex
\begin{table*}[t!]
	\small
\resizebox{1.0\linewidth}{!}{
  \tabcolsep=0.07cm
\begin{tabular}{l|cccccc|cccccc|cccccc}
			
	Dataset & \multicolumn{6}{c|}{Mip-NeRF360 Dataset}  & \multicolumn{6}{c|}{Tanks\&Temples} & \multicolumn{6}{c}{Deep Blending}\\
	Method|Metric
	& $SSIM^\uparrow$   & $PSNR^\uparrow$    & $LPIPS^\downarrow$  & Train$^\downarrow$  & FPS$^\uparrow$ & Mem$^\downarrow$  
	& $SSIM^\uparrow$   & $PSNR^\uparrow$    & $LPIPS^\downarrow$  & Train$^\downarrow$ & FPS$^\uparrow$  & Mem$^\downarrow$  
	& $SSIM^\uparrow$   & $PSNR^\uparrow$    & $LPIPS^\downarrow$  & Train$^\downarrow$  & FPS$^\uparrow$ & Mem$^\downarrow$  \\
	\hline 
	Plenoxels& 0.626 & 23.08 & 0.463 & 26m & 6.79 & 2.1GB & 0.719 & 21.08 & 0.379 & 25m & 13.0 & 2.3GB & 0.795 & 23.06 & 0.510 & 28m & 11.2 & 2.7GB \\
	INGP & 0.699 & 25.59 & 0.331 & 7.5m & 9.43 & 
 \cellcolor{orange!40} 48MB & 0.745 & 21.92 & 0.305 & 7m & 14.4 & \cellcolor{orange!40} 48MB & 0.817 & 24.96 & 0.390 & 8m & 2.79 & \cellcolor{orange!40} 48MB \\ 
	Mip-NeRF360& 0.792 \cellcolor{yellow!40} & \cellcolor{red!40}  27.69 & \cellcolor{orange!40}  0.237 & 48h & 0.06 & \cellcolor{red!40} 8.6MB  & 0.759 \cellcolor{yellow!40} & \cellcolor{yellow!40}  22.22 & \cellcolor{yellow!40} 0.257 & 48h & 0.14 & \cellcolor{red!40} 8.6MB & \cellcolor{orange!40}  0.901 & \cellcolor{yellow!40}  29.40 &  \cellcolor{orange!40} 0.245 & 48h & 0.09 & \cellcolor{red!40} 8.6MB\\
	3D Gaussians-7K& 0.770 & 25.60 & 0.279 & 6.5m & \cellcolor{orange!40}  160 & 523MB &  0.767 & 21.20 & 0.280 & 7m & \cellcolor{orange!40}  197 & 270MB & 0.875 &27.78 & 0.317 & 4.5m & \cellcolor{red!40}  172 & 386MB \\
	3D Gaussians-30K& \cellcolor{red!40} 0.815 & \cellcolor{orange!40} 27.21 & \cellcolor{red!40} 0.214 & 42m & \cellcolor{yellow!40} 134 & 734MB & \cellcolor{red!40} 0.841 & \cellcolor{orange!40} 23.14 & \cellcolor{red!40} 0.183 & 26m & \cellcolor{yellow!40} 154 & 411MB & \cellcolor{red!40} 0.903 & 29.41 \cellcolor{orange!40}  & \cellcolor{red!40} 0.243 & 36m & \cellcolor{yellow!40} 137 & 676MB\\ \midrule
	\methodname (ours) & \cellcolor{orange!40}   0.794  & \cellcolor{yellow!40}  26.91  & \cellcolor{yellow!40}   0.250 &   32m  & \cellcolor{red!40} 186  &  \cellcolor{yellow!40} 377MB  & \cellcolor{orange!40}  0.836 \cellcolor{orange!40}   & 23.35   \cellcolor{red!40}  & \cellcolor{orange!40} 0.198   &  21m  & \cellcolor{red!40} 210  & \cellcolor{yellow!40} 222MB  & \cellcolor{orange!40}  0.901 \cellcolor{orange!40} & 29.68 \cellcolor{red!40} & 0.252 \cellcolor{yellow!40}  &  30m & \cellcolor{orange!40} 160 & \cellcolor{yellow!40} 399MB \\
\end{tabular}
	}
\caption{
    \textbf{Comparative Analysis of Novel View Synthesis Techniques.} This table presents a comprehensive comparison of our approach with established methods across various datasets. The metrics, inclusive of SSIM, PSNR, and LPIPS, alongside training duration, frames per second, and memory usage, provide a multidimensional perspective of performance efficacy. Note that our training time numbers of the different methods may be computed on different GPUs; they are not necessarily perfectly comparable but are still valid. Note that non-explicit representations (INGP, Mip-NeRF360) have low memory because they rely on additional slow neural networks for decoding. Red-colored results are the best.}
 \label{tab:comparisons}
 \vspace{-3mm}
\end{table*}

%% file: sections/experiments.tex
\section{Experiments}\label{sec:exp}
\vspace{-4pt}
\subsection{Datasets and Metrics}\label{sec:dataset}
\vspace{-2pt}
In our experiments, we utilized a diverse range of datasets to test the effectiveness of our algorithm in rendering real-world scenes. This evaluation encompassed 13 real scenes from various sources. We particularly focused on scenes from the Mip-Nerf360 dataset \cite{MipNeRF-360}, renowned for its superior NeRF rendering quality, alongside select scenes from the Tanks \& Temples dataset \cite{Knapitsch2017}, and instances provided by Hedman et al. \cite{hedman2018deep} for their work in Deep Blending. These scenes presented a wide array of capture styles, ranging from bounded indoor settings to expansive unbounded outdoor environments.

The quality benchmark in our study was set by the Mip-Nerf360 \cite{MipNeRF}, which we compared against other contemporary fast NeRF methods, such as InstantNGP \cite{InstantNGP} and Plenoxels. Our train/test split followed the methodology recommended by Mip-NeRF360, using every 8th photo for testing. This approach facilitated consistent and meaningful error metric comparisons, including standard measures such as PSNR, L-PIPS, and SSIM, as frequently employed in existing literature (see Table~\ref{tab:comparisons}). Our results encompassed various configurations and iterations, highlighting differences in training time, rendering speeds, and memory requirements for optimized parameters.

\subsection{Implementation Details of \methodname} \label{sec:details}
\vspace{-2pt}
Our methodology maintained consistent hyperparameter settings across all scenes, ensuring uniformity in our evaluations. We deployed an A6000 GPU for most of our tests. Our Generalized Exponential Splatting (\methodname) was implemented over 40,000 iterations, and the density gradient threshold is set to 0.0003. The learning rate for the shape parameter was set at 0.0015, with a shape reset interval of 1000 iterations and a shape pruning interval of 100 iterations. The threshold for pruning based on shape was set at 0.5, while the shape strength parameter was determined to be 0.1, offering a balance between accuracy and computational load. Additionally, the Image Laplacian scale factor was set at 0.2, with the corresponding $\lambda_{\omega}$ frequency loss coefficient marked at 0.5, ensuring edge-enhanced optimization in our image reconstruction tasks. The other hyperparameters and design choices (like opacity splitting and pruning) shared with Gaussian splitting \cite{gaussiansplatter}  were kept the same.  More details are provided in \supp.

%% file: sections/results.tex
\section{Results}\label{sec:results}
\vspace{-4pt}
\subsection{Novel View Synthesis Results}
\vspace{-2pt}
We evaluated \textit{\methodname} against several state-of-the-art techniques in both novel view synthesis tasks. Table \ref{tab:comparisons} encapsulate the comparative results in addition to \figlabel{\ref{fig:comparisons}}.
Table \ref{tab:comparisons} demonstrates that \textit{\methodname} achieves a balance between high fidelity and efficiency in novel view synthesis. Although it does not always surpass other methods in SSIM or PSNR, it significantly excels in memory usage and speed. With only 377MB of memory and a processing speed of 2 minutes, \textit{\methodname} stands out as a highly efficient method, particularly when compared to the 3D Gaussians-30K and Instant NGP, which require substantially more memory or longer processing times.  Overall, the results underscore \textit{\methodname}'s capability to deliver balanced performance with remarkable efficiency, making it a viable option for real-time applications that demand both high-quality output and operational speed and memory efficiency.

Note that it is difficult to see the differences in \textit{visual effects} between \methodname and Gaussians in \figlabel{\ref{fig:comparisons}} since they have almost the same PSNR but a different file size (Table \ref{tab:comparisons}). For a fair visual comparison, we restrict the number of components to be roughly the same (by controlling the splitting of Gaussians) and show the results in \figlabel{\ref{fig:fair}}. It clearly shows that \methodname can model tiny and sharp edges for that scene better than Gaussians.

%% file: figures/fair.tex
    \begin{figure}[t]
\setlength\mytmplen{.32\linewidth}
      \centering
  \tabcolsep=0.03cm
  \resizebox{1.0\linewidth}{!}{
  \begin{tabular}{ccc}
 Ground Truth & \textbf{GES(ours)} & Gaussians \\
\zoomincrop{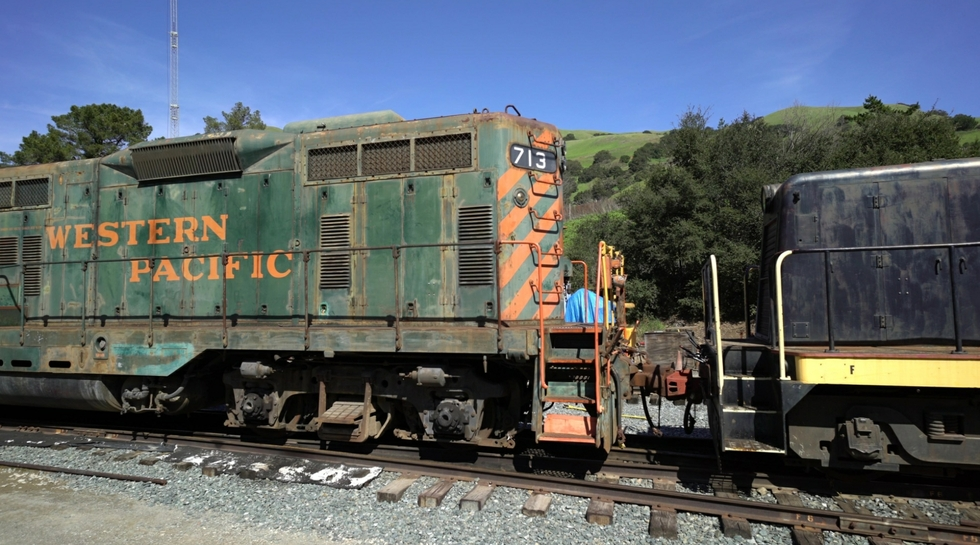}{0.5}{1.25}{1.2cm}{0.5cm}{1cm}{\mytmplen}{2.5}{red}  &
\zoomin{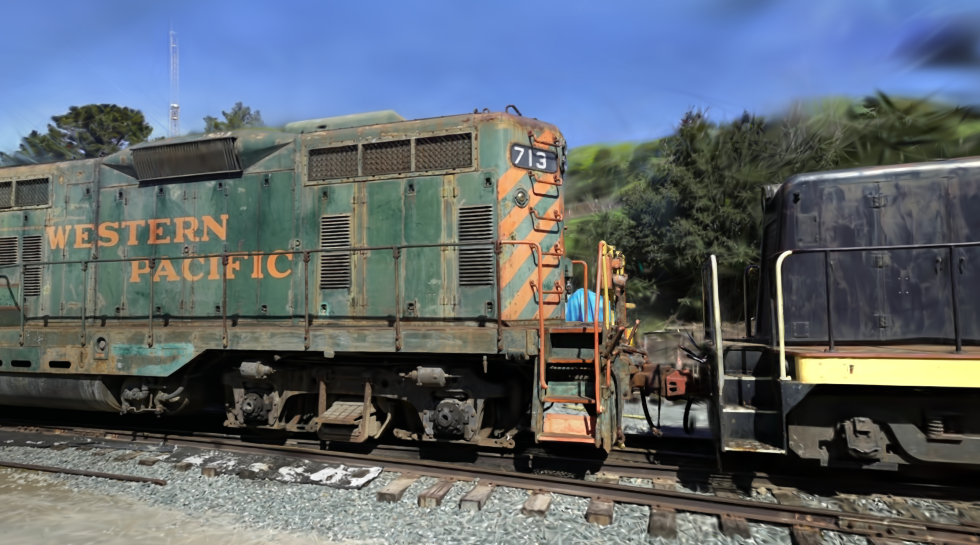}{0.5}{1.25}{1.2cm}{0.5cm}{1cm}{\mytmplen}{2.5}{red}  &
\zoomin{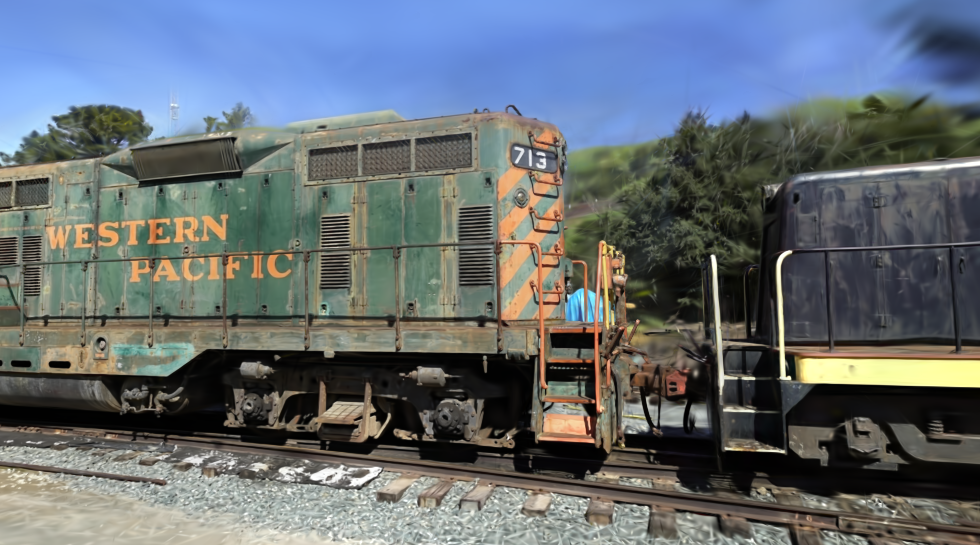}{0.5}{1.25}{1.2cm}{0.5cm}{1cm}{\mytmplen}{2.5}{red}  \\          
  \end{tabular}
  }
  \vspace{-4pt}
      \caption{
\textbf{Fair Visual Comparison.} We show an example of Gaussians \cite{gaussiansplatter} and \methodname when constrained to \textit{the same number} of splatting components for a fair visual comparison. It clearly shows that \methodname can model tiny and sharp edges for that scene better than Gaussians.
      }
      \label{fig:fair}
  \end{figure}

%% file: sections/ablation.tex
\subsection{Ablation and analysis}\label{sec:ablation}
\vspace{-2pt}

\minorsection{Shape parameters} In Table \ref{tab:ablation-general-mip}, we explore the effect of important hyperparameters associated with the new shape parameter on novel view synthesis performance. We see that proper approximation $\bar{\phi}_{\rho}$ in \eqlabel{\eqref{eq:approx}} is necessary, because  if we set $\rho=10$ for $\bar{\phi}_{\rho}$ to be as close to the exact $\phi(\beta)$ (\figlabel{\ref{fig:approximation}}), the PSNR would drop to 11.6. Additional detailed analysis is provided in \supp .

\minorsection{Effect of frequency-modulated image loss}
We study the effect of the frequency loss $\mathcal{L}_{\omega}$ introduced in \secLabel{\ref{sec:image-laplacian-guidance}} on the performance by varying $\lambda_\omega$. In table 
\ref{tab:ablation-general-mip} and in \figlabel{\ref{figsup:laplace}} we demonstrate how adding this $\mathcal{L}_{\omega}$ improves the optimization in areas where large contrast exists or where the smooth background is rendered and also improves the efficiency of GES. We notice that increasing $\lambda_\omega$ in \methodname indeed reduces the size of the file, but can affect the performance. We chose $\lambda_\omega =0.5$ as a middle ground between improved performance and reduced file size.  

\minorsection{Analyzing memory reduction}
We find that the reduction in memory after learning $\beta$ is indeed attributed to the reduction of the number of components needed. For example, in the ``Train'' sequence, the number of components is 1,087,264 and 548,064 for Gaussian splatting and \methodname respectively. This translates into the reduction of file size from 275 MB to 129.5 MB when utilizing \methodname.

\minorsection{Applying \methodname in fast 3D generation}
Recent works have proposed to use Gaussian Splatting for 3D generation pipelines such as DreamGaussian~\cite{dreemgaussian} and Text-to-3D using Gaussian Splatting~\cite{text2gaussian}. Integrating \methodname into these Gaussian-based 3D generation pipelines has yielded fast and compelling results with a plug-and-play ability of GES in place of Gaussian Splatting (see \figlabel{\ref{fig:gen3d_vis}}).

%% file: figures/laplacian.tex
\begin{figure*}[t] 
    \setlength\mytmplen{.243\linewidth}
    \setlength{\tabcolsep}{1pt}
    \renewcommand{\arraystretch}{0.5}
    \centering
    \resizebox{1.0\linewidth}{!}{
    \begin{tabular}{cccc}
    \tabcolsep=0.00cm
    Ground Truth& \textbf{\methodname (full)} & \methodname (w/o $\mathcal{L}_{\omega}$ )  & Gaussian Splatting \cite{gaussiansplatter} \\ 
\zoomin{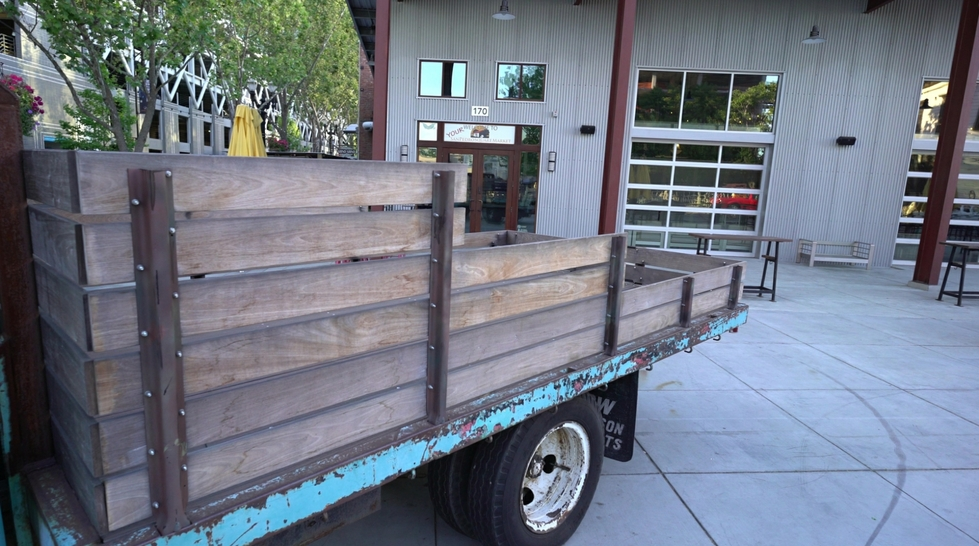}{3.0}{2.1}{3.6cm}{0.55cm}{1cm}{\mytmplen}{3}{red} &
\zoomin{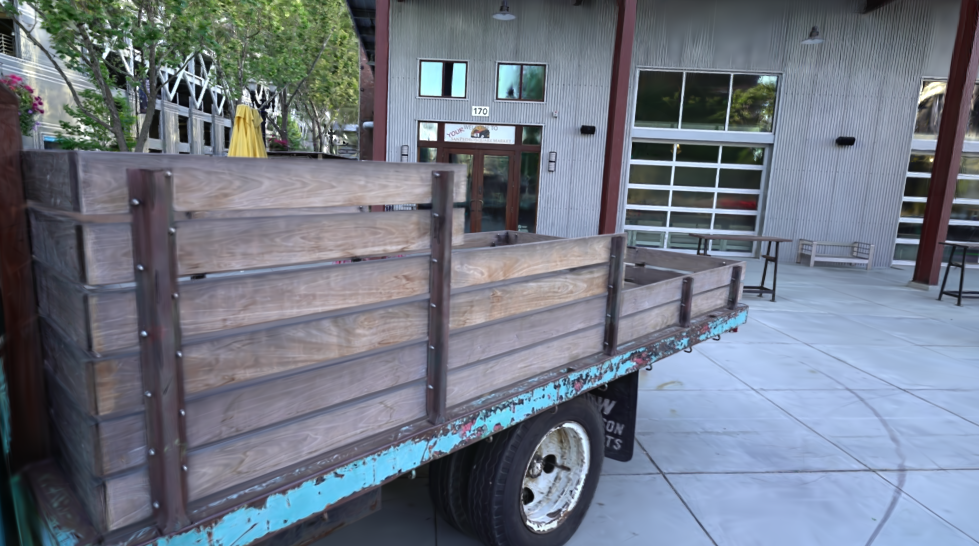}{3.0}{2.1}{3.6cm}{0.55cm}{1cm}{\mytmplen}{3}{red} &
\zoomin{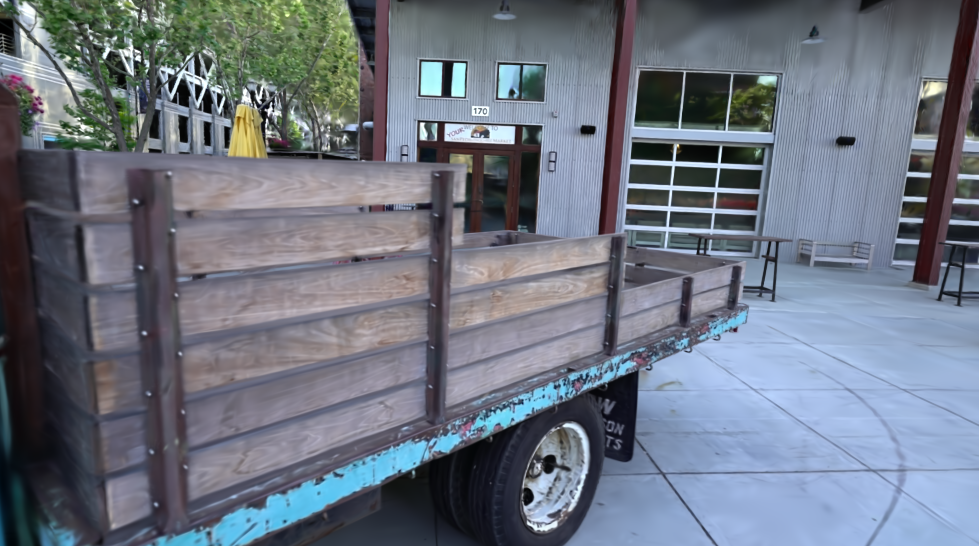}{3.0}{2.1}{3.6cm}{0.55cm}{1cm}{\mytmplen}{3}{red} &
\zoomin{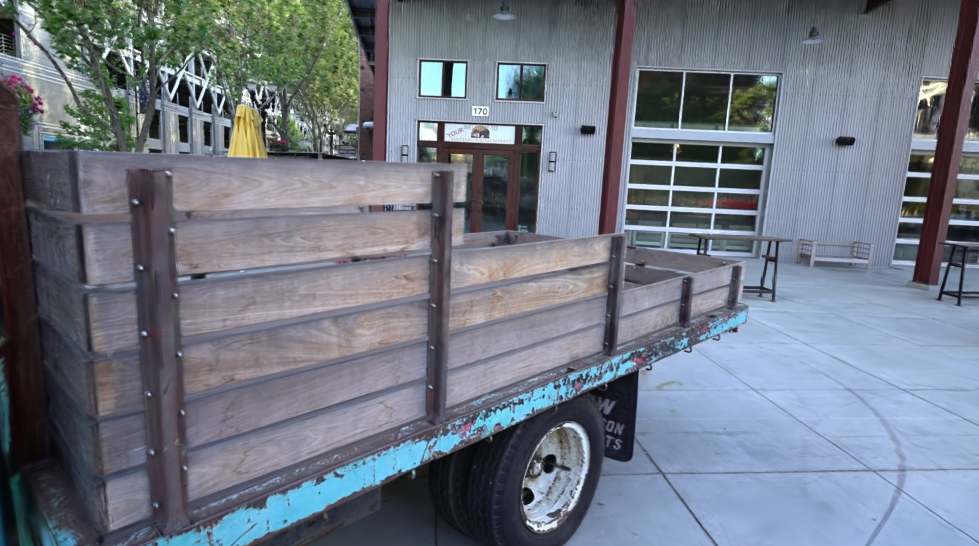}{3.0}{2.1}{3.6cm}{0.55cm}{1cm}{\mytmplen}{3}{red} \\
    \end{tabular} }
    \caption{
    \textbf{Frequency-Modulated Loss Effect.} We show the effect of the frequency-modulated image loss $\mathcal{L}_{\omega}$ on the performance on novel views synthesis. Note how adding this $\mathcal{L}_{\omega}$ improves the optimization in areas where a large contrast exists or a smooth background is rendered.  
    }
    \label{figsup:laplace}
    \end{figure*}

%% file: tables/ablation_full_mip.tex
\begin{table}[t]
    \centering
      \resizebox{0.99\linewidth}{!}{%
          \tabcolsep=0.1cm
    \begin{tabular}{lcccc}
    \toprule
     \multirowcell{1}{Ablation Setup}  & \multirowcell{1}{$PSNR^\uparrow$} & \multirowcell{1}{$SSIM^\uparrow$} & \multirowcell{1}{$LPIPS^\downarrow$} & \multirowcell{1}{\textbf{Size (MB)}$^\downarrow$} \\
    \midrule
    Gaussians   & 27.21 & 0.815 & 0.214 & 734 \\ \hline
    GES w/o approx. $\bar{\phi}_{\rho}$ & 11.60 & 0.345 & 0.684 & 364 \\
    GES w/o shape reset & 26.57 & 0.788 & 0.257 & 374 \\
    GES w/o $\mathcal{L}_{\omega}$ loss & 27.07 & 0.800 & 0.250 & 411 \\
     Full GES  & 26.91 & 0.794 & 0.250 & 377 \\
    \bottomrule
    \end{tabular}
    \vspace{-6pt} 
    }    
    \footnotesize
    \caption{\textbf{Ablation Study on Novel View Synthesis.} We study the impact of several components in \methodname on the reconstruction quality and file size in the Mip-NeRF360 dataset.}
    \label{tab:ablation-general-mip}
    \end{table}

%% file: figures/gen3d_vis.tex
\begin{figure}[t!] 
\centering

\includegraphics[trim={2.6cm 1.9cm 9.0cm 1.0cm},clip,width=0.7\linewidth]{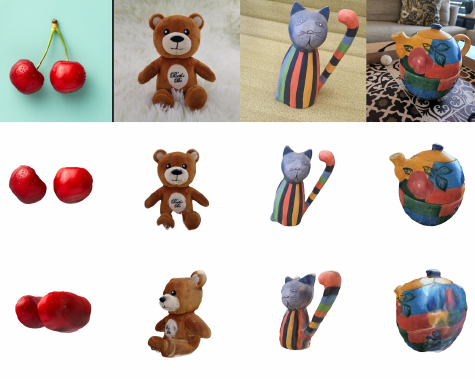}
\caption{
\textbf{\methodname Application: Fast Image-to-3D Generation}. We show selected 3D generated examples from Co3D images \cite{co3d} by combining GES with the Gaussian-based 3D generation pipeline \cite{dreemgaussian}, highlighting the plug-and-play benefits of \methodname to replace Gaussian Splatting \cite{gaussiansplatter}.  
}
\label{fig:gen3d_vis}
\vspace{-1mm}
\end{figure}

%% file: sections/conclusion.tex
\section{Conclusion and discussion}\label{sec:con}
\vspace{-4pt}
This paper introduced \textit{\methodname} (Generalized Exponential Splatting), a new technique for 3D scene modeling that improves upon Gaussian Splatting in memory efficiency and signal representation, particularly for high-frequency signals. Our empirical results demonstrate its efficacy in novel view synthesis and 3D generation tasks.

\input{sections/limitation}

%% file: sections/limitation.tex
\inlinesection{Limitation}\label{sec:limit}
One obvious limitation in our approach is that performance typically drops trying to make the representation as memor-efficient and as compact as possible. This is more noticeable for more complex scenes due to the pruning operations that depend on $\beta$-tuning. Removing many of the components can eventually drop the PSNR performance (Table 1 last 2 rows). Future research could focus on enhancing \methodname's performance in more complex and dynamic environments and exploring its integration with other technologies in 3D modeling.

%% file: sections/supplement.tex
\renewcommand{\thesection}{\Alph{section}}

\setcounter{section}{0}

\section{Theory Behind Generalized Exponentials}%
\label{secsup:theory}

\subsection{Generalized Exponential Function}

The Generalized Exponential Function (GEF) is similar to the probability density function (PDF) of the Generalized Normal Distribution (GND) \cite{generlizedgaussian} with an additional amplitude parameter $ A \in \mathbb{R}$. This function allows for a more flexible adaptation to various data shapes by adjusting the shape parameter $ \beta \in (0,\infty) $. The GEF is given by the following.
\begin{equation}\label{eq:gef-supp}
    f(x | \mu, \alpha, \beta, A) = A \exp\left(-\left(\frac{|x - \mu|}{\alpha}\right)^\beta\right)
\end{equation}
where $ \mu \in \mathbb{R} $ is the location parameter, $ \alpha \in \mathbb{R} $ is the scale parameter, $ A $ defines the amplitude, and $ \beta > 0 $ is the shape parameter. 
For $ \beta = 2 $, the GEF becomes a scaled Gaussian distribution:
\begin{equation}\label{eq:gaussian}
f(x | \mu, \alpha, \beta=2, A) = \frac{A}{\alpha\sqrt{2\pi}} \exp\left(-\frac{1}{2}\left(\frac{x - \mu}{\alpha/\sqrt{2}}\right)^2\right)
\end{equation}
And for $ \beta = 1 $, \eqlabel\ref{eq:gef-supp} reduces to a scaled Laplace distribution:
\begin{equation}\label{eq:lap}
f(x | \mu, \alpha, \beta=1, A) = \frac{A}{2\alpha} \exp\left(-\frac{|x - \mu|}{\alpha}\right)
\end{equation}
 The GEF, therefore, provides a versatile framework for modeling a wide range of data by varying $ \beta $, unlike the Gaussian mixtures, which have a low-pass frequency domain. Many common signals, like the square or triangle, are band-unlimited, constituting a fundamental challenge to Gaussian-based methods (see \figlabel{\ref{fig:foureir}}). In this paper, we try to \textit{learn} a positive $\beta$ for every component of the Gaussian splatting to allow for a generalized 3D representation.

\subsection{Theoretical Results}

Despite its generalizable capabilities, the GEF has no fixed behavior in terms of frequency domain. The error functions of the GEF and its Fourier domain cannot be studied analytically, as they involve complex integrals of exponentials without closed form that depend on the shape parameter $\beta$. For example, the Fourier of GEF is given by $$\mathcal{F}(f)(\xi) = \int_{-\infty}^{\infty} A \exp\left(-\left(\frac{|x-\mu|}{\alpha}\right)^\beta\right) e^{-2\pi i x \xi} \, dx
$$ which does not have a closed-form solution for a general $\beta$. We demonstrate that for specific cases, such as for a square signal, the GEF can achieve a smaller approximation error than the corresponding Gaussian function by properly choosing $ \beta $. Theorem \ref{thrm:1} provides a theoretical foundation for preferring the GEF over standard Gaussian functions in our \methodname representation instead of 3D Gaussian Splatting \cite{gaussiansplatter}.

\newtheorem{theorem}{Theorem}
\begin{theorem}[Superiority of GEF Approximation Over Gaussian for Square Wave Signals] \label{thrm:1}
Let $ S(t) $ represent a square wave signal with amplitude $ A > 0 $ and width $ L >0 $ centered at $ t = 0 $. Define two functions: a scaled Gaussian $ G(t; \alpha, A) = A e^{-\frac{t^2}{\alpha^2}} $, and a Generalized Exponential Function $ GEF(t; \alpha, \beta, A) = A e^{-(|t|/\alpha)^\beta} $. For any given scale parameter $ \alpha $, there exists a shape parameter $ \beta $ such that the approximation error $
E_f =   \int_{-\infty}^{\infty} |S(t) - f(t)| dt.$ of the square signal $S(t)$ using GEF is strictly smaller than that using the Gaussian $G$.
\end{theorem}

\begin{proof}
The error metric $E_f$ for the square signal $S(t)$ approximation using $f$ function as $
E_f =   \int_{-\infty}^{\infty} |S(t) - f(t)| dt.$
Utilizing symmetry and definition of $S(t)$, and the fact that $S(t)> G(t; \alpha, A) $, the error for the Gaussian approximation simplifies to:
$$
E_G = 2 \int_{0}^{L/2} A (1 - e^{-\frac{t^2}{\alpha^2}}) dt + 2 \int_{L/2}^{\infty} A e^{-\frac{t^2}{\alpha^2}} dt.
$$

For the GEF approximation, the error is:
$$
E_{GEF} = 2 \int_{0}^{L/2} A (1 - e^{-(t/\alpha)^\beta}) dt + 2 \int_{L/2}^{\infty} A e^{-(t/\alpha)^\beta} dt.
$$
The goal is to show the difference in errors $\Delta E = E_G  - E_{GEF} $ to be strictly positive, by picking $\beta$ appropriately. The error difference can be described as follows.
$$
\Delta E = \Delta E_{middle} + \Delta E_{tail}
$$
$$
\Delta E_{middle} = 2 \int_{0}^{L/2} A (1 - e^{-\frac{t^2}{\alpha^2}}) dt ~~ -~~ 2 \int_{0}^{L/2} A (1 - e^{-(t/\alpha)^\beta}) dt
$$
$$
\Delta E_{tail} = 2 \int_{L/2}^{\infty} A e^{-\frac{t^2}{\alpha^2}} dt - 2 \int_{L/2}^{\infty} A e^{-(t/\alpha)^\beta} dt
$$
Let us
Define $ \text{err}(t) $ as the difference between the exponential terms:
$$
\text{err}(t) = e^{-\frac{t^2}{\alpha^2}} - e^{-(t/\alpha)^\beta}.
$$
The difference in the middle error terms for the Gaussian and GEF approximations, $ \Delta E_{middle} $, can be expressed using $ \text{err}(t) $ as:
$$
\Delta E_{middle} = 2 A \int_{0}^{L/2} \text{err}(t) \, dt.
$$

Using the trapezoidal approximation of the integral, this simplifies to:
$$
\Delta E_{middle} \approx L A ~\text{err}(L/2) =  L A\left( e^{-\frac{L^2}{4\alpha^2}} - e^{-(L/2\alpha)^\beta} \right).
$$

Based on the fact that the negative exponential is monotonically decreasing and to ensure $ \Delta E_{middle} $ is always positive, we choose $ \beta $ based on the relationship between $ L/2 $ and $ \alpha $ :
\begin{itemize}
  \item If $ \frac{L}{2} > \alpha $ (i.e., $ \frac{L}{2\alpha} > 1 $), choosing $ \beta > 2 $ ensures $ e^{-(L/2\alpha)^\beta} < e^{-\frac{L^2}{4\alpha^2}} $.
  \item If $ \frac{L}{2} < \alpha $ (i.e., $ \frac{L}{2\alpha} < 1 $), choosing $ 0 < \beta < 2 $ results in $ e^{-(L/2\alpha)^\beta} < e^{-\frac{L^2}{4\alpha^2}} $.
\end{itemize}
Thus, $ \Delta E_{middle} $ can always be made positive by choosing $ \beta $ appropriately, implying that the error in the GEF approximation in the interval $ [-L/2, L/2] $ is always less than that of the Gaussian approximation. Similarly, the difference of tail errors $\Delta E_{tail}$ can be made positive by an appropriate choice of $ \beta $, concluding that the total error $ E_{GEF} $ is strictly less than $ E_{G} $. This concludes the proof.
\end{proof}

\subsection{Numerical Simulation of Gradient-Based 1D Mixtures}

\minorsection{Objective}
The primary objective of this numerical simulation is to evaluate the effectiveness of the generalized exponential model in representing various one-dimensional (1D) signal types. This evaluation was conducted by fitting the model to synthetic signals generated to embody characteristics of square, triangle, parabolic, half sinusoidal, Gaussian, and exponential functions, which can constitute a non-exclusive list of basic topologies available in the real world.

\minorsection{Simulation Setup}
The experimental framework was based on a series of parametric models implemented in PyTorch, designed to approximate 1D signals using mixtures of different functions such as Gaussian, Difference of Gaussians (DoG), Laplacian of Gaussian (LoG), and a Generalized mixture model. Each model comprised parameters for means, variances (or scales), and weights, with the generalized model incorporating an additional parameter, $\beta$, to control the exponentiation of the Gaussian function.

\input{figures/foureir}

\minorsection{Models}
Here, we describe the mixture models used to approximate the true signal forms.

\begin{itemize}
    \item \textbf{Gaussian Mixture Model (GMM):}
    The GMM combines several Gaussian functions, each defined by its mean ($\mu_i$), variance ($\sigma_i^2$), and weight ($w_i$). For a set of $N$ Gaussian functions, the mixture model $g(x)$ can be expressed as:
    \begin{equation} \label{eq:gaussianmixture}
    g(x) = \sum_{i=1}^{N} w_i \exp\left(-\frac{(x - \mu_i)^2}{2\sigma_i^2 + \epsilon}\right),
    \end{equation}
    where $\epsilon$ is a small constant to avoid division by zero, with $ \epsilon = 1e-8$.

    \item \textbf{Difference of Gaussians (DoG) Mixture Model:}
    The DoG mixture model is comprised of elements that represent the difference between two Gaussian functions with a fixed variance ratio $\nu$. The model $d(x)$ for $N$ components is given by:
    \begin{equation}
    \begin{aligned}
    &d(x) = \sum_{i=1}^{N} w_i D_i \\
    D_i = &\left(\exp\left(-\frac{(x - \mu_i)^2}{2\sigma_i^2 + \epsilon}\right) - \exp\left(-\frac{(x - \mu_i)^2}{2(\sigma_i^2 / \nu) + \epsilon}\right)\right),
    \end{aligned}
    \end{equation}
    where $\sigma_i$ is a scale parameter, and the variance ratio $\nu$ is fixed to be 4.

    \item \textbf{Laplacian of Gaussian (LoG) Mixture Model:}
    The LoG mixture model is formed by a series of Laplacian of Gaussian functions, each defined by a mean ($\mu_i$), scale ($\gamma_i$), and weight ($w_i$). The mixture model $l(x)$ is:
    \begin{equation}
    l(x) = \sum_{i=1}^{N} w_i \left(-\frac{(x - \mu_i)^2}{\gamma_i^2} + 1\right) \exp\left(-\frac{(x - \mu_i)^2}{2\gamma_i^2 + \epsilon}\right),
    \end{equation}

    \item \textbf{Generalized Mixture Model:}
    This model generalizes the Gaussian mixture by introducing a shape parameter $\beta$. Each component of the model $h(x)$ is expressed as:
    \begin{equation} \label{eq:genmixture}
    h(x) = \sum_{i=1}^{N} w_i \exp\left(-\frac{|x - \mu_i|^\beta}{2\sigma_i^2 + \epsilon}\right),
    \end{equation}
    where $\beta$ is a learnable parameter that is optimized alongside other parameters. When $\beta = 2$ is fixed, the equation in \eqlabel{\eqref{eq:genmixture}} reduces to the one in \eqlabel{\eqref{eq:gaussianmixture}}.
\end{itemize}

\minorsection{Model Configuration}
The models were configured with a varying number of components $N$, with tests conducted using $N = \{2, 5, 8, 10, 15, 20, 50, 100\}$. The weights of the components could be either positive or unrestricted. For the generalized model, the $\beta$ parameter was learnable.

\minorsection{Training Procedure}
Each model was trained using the Adam optimizer with a mean squared error loss function. The input $x$ was a linearly spaced tensor representing the domain of the synthetic signal, and the target $y$ was the value of the signal at each point in $x$. Training proceeded for a predetermined number of epochs, and the loss was recorded at the end of training.

\minorsection{Data Generation}
Synthetic 1D signals were generated for various signal types over a specified range, with a given data size and signal width. The signals were used as the ground truth for training the mixture models. The ground truth signals used in the experiment are one-dimensional (1D) functions that serve as benchmarks for evaluating signal processing algorithms. Each signal type is defined within a specified width around the origin, and the value outside this interval is zero ( see \figlabel{\ref{fig:foureir}}). The parameter \texttt{width} $\sigma$ dictates the effective span of the non-zero portion of the signal. We define six distinct signal types as follows:

\begin{enumerate}
    \item \textbf{Square Signal:} The square signal is a binary function where the value is 1 within the interval $(- \frac{\sigma}{2}, \frac{\sigma}{2})$ and 0 elsewhere. Mathematically, it is represented as
    \begin{equation}
    f_{\text{square}}(x) =
    \begin{cases}
    1 & \text{if } -\frac{\sigma}{2} < x < \frac{\sigma}{2}, \\
    0 & \text{otherwise}.
    \end{cases}
    \end{equation}
    Its Fourier Transform is given by
    \begin{equation}
    \text{FT}\{\text{Square Wave}\}(f) = \text{sinc}\left(\frac{f \cdot \sigma}{\pi}\right)
    \end{equation}

    \item \textbf{Triangle Signal:} This signal increases linearly from the left edge of the interval to the center and decreases symmetrically to the right edge, forming a triangular shape. It is defined as
    \begin{equation}
    f_{\text{triangle}}(x) =
    \begin{cases}
    \frac{\sigma}{2} - |x| & \text{if } -\frac{\sigma}{2} < x < \frac{\sigma}{2}, \\
    0 & \text{otherwise}.
    \end{cases}
    \end{equation}
    Its Fourier Transform is
    \begin{equation}
    \text{FT}\{\text{Triangle Wave}\}(f) = \left(\text{sinc}\left(\frac{f \cdot \sigma}{2\pi}\right)\right)^2
    \end{equation}

    \item \textbf{Parabolic Signal:} This signal forms a downward-facing parabola within the interval, and its expression is
    \begin{equation}
    f_{\text{parabolic}}(x) =
    \begin{cases}
    (\frac{\sigma}{2})^2 - x^2 & \text{if } -\frac{\sigma}{2} < x < \frac{\sigma}{2}, \\
    0 & \text{otherwise}.
    \end{cases}
    \end{equation}
    The Fourier Transform of the parabolic signal is
    \begin{equation}
    \text{FT}\{\text{Parabolic Wave}\}(f) = \frac{3 \cdot \left(\text{sinc}\left(\frac{f \cdot \sigma}{2\pi}\right)\right)^2}{\pi^2 \cdot f^2}
    \end{equation}

    \item \textbf{Half Sinusoid Signal:} A half-cycle of a sine wave is contained within the interval, starting and ending with zero amplitude. Its formula is
    \begin{equation}
    f_{\text{half\_sinusoid}}(x) =
    \begin{cases}
    \sin\left((x + \frac{\sigma}{2}) \frac{\pi}{\sigma}\right) & \text{if } -\frac{\sigma}{2} < x < \frac{\sigma}{2}, \\
    0 & \text{otherwise}.
    \end{cases}
    \end{equation}
    Its Fourier Transform is described by
    \begin{equation}
    \text{FT}\{\text{Half Sinusoid}\}(f) = 
    \begin{cases} 
        \frac{\sigma}{2} & \text{if } f = 0 \\
        \frac{\sigma \cdot \sin(\pi \cdot f \cdot \sigma)}{\pi^2 \cdot f^2} & \text{otherwise}
    \end{cases}
    \end{equation}

    \item \textbf{Exponential Signal:} Exhibiting an exponential decay centered at the origin, this signal is represented by
    \begin{equation}
    f_{\text{exponential}}(x) =
    \begin{cases}
    \exp(-|x|) & \text{if } -\frac{\sigma}{2} < x < \frac{\sigma}{2}, \\
    0 & \text{otherwise}.
    \end{cases}
    \end{equation}
    The Fourier Transform for the exponential signal is
    \begin{equation}
    \text{FT}\{\text{Exponential}\}(f) = \frac{\sigma}{f^2 + \left(\frac{\sigma}{2}\right)^2}
    \end{equation}

    \item \textbf{Gaussian Signal:} Unlike the others, the Gaussian signal is not bounded within a specific interval but instead extends over the entire range of $ x $, with its amplitude governed by a Gaussian distribution. It is given by
    \begin{equation}
    f_{\text{Gaussian}}(x) = \exp\left(-\frac{x^2}{2\sigma^2}\right).
    \end{equation}
    The Fourier Transform of the Gaussian signal is also a Gaussian, which in the context of standard deviation $ \sigma $ is represented as
    \begin{equation}
    \text{FT}\{\text{Gaussian}\}(f) = \sqrt{2\pi} \cdot \sigma \cdot \exp\left(-2\pi^2 \sigma^2 f^2\right)
    \end{equation}
\end{enumerate}

As shown in \figlabel{\ref{fig:foureir}}, the Gaussian function has a low-pass band, while signals like the square and triangle with sharp edges have infinite bandwidth, making them challenging for mixtures that have low-pass frequency bandwidth (\eg Gaussian mixtures, represented by Gaussian Splatting \cite{gaussiansplatter}). 

Each signal is sampled at discrete points using a PyTorch tensor to facilitate computational manipulation and analysis within the experiment's framework. We show in \figlabel{\ref{supfig:fitting_square_p},\ref{supfig:fitting_square_N},\ref{supfig:fitting_parabolic_p},\ref{supfig:fitting_parabolic_N},\ref{supfig:fitting_exponential_p},\ref{supfig:fitting_exponential_N},\ref{supfig:fitting_triangle_p},\ref{supfig:fitting_triangle_N},\ref{supfig:fitting_gaussian_p},\ref{supfig:fitting_gaussian_N},\ref{supfig:fitting_half_sinusoid_p}, and \ref{supfig:fitting_half_sinusoid_N}} examples of fitting all the mixture on all different signal types of interest when positive weighting is used in the mixture \vs when allowing real weighting in the combinations in the above equations. Note how sharp edges constitute a challenge for Gaussians that have low pass bandwidth while a square signal has an infinite bandwidth known by the sinc function \cite{shannon}.

\input{figures/theory_supp}
\minorsection{Loss Evaluation}
The models' performance was evaluated based on the loss value after training. Additionally, the model's ability to represent the input signal was visually inspected through generated plots. Multiple runs per configuration were executed to account for variance in the results.

\minorsection{Stability Evaluation}
Model stability and performance were assessed using a series of experiments involving various signal types and mixture models. Each model was trained on a 1D signal generated according to predefined signal types (square, triangle, parabolic, half sinusoid, Gaussian, and exponential), with the goal of minimizing the mean squared error (MSE) loss between the model output and the ground truth signal. The number of components in the mixture models ($N$) varied among a set of values, and models were also differentiated based on whether they were constrained to positive weights.
For a comprehensive evaluation, each configuration was run multiple times (20 runs per configuration) to account for variability in the training process. During these runs, the number of instances where the training resulted in a \texttt{NaN} loss was recorded as an indicator of stability issues.
The stability of each model was quantified by the percentage of successful training runs ($ \frac{\text{Total Runs} - \text{NaN Loss Counts}}{\text{Total Runs}} \times 100 \% $). The experiments that failed failed because the loss has diverged to \texttt{NaN}. This typical numerical instability in optimization is the result of learning the variance which can go close to zero, resulting in the exponential formula (in \eqlabel{\eqref{eq:gef-supp}}) to divide by an extremely small number.

The average MSE loss from successful runs was calculated to provide a measure of model performance. The results of these experiments were plotted, showing the relationship between the number of components and the stability and loss of the models for each signal type.

\minorsection{Simulation Results}
In the conducted analysis, both the loss and stability of various mixture models with positive and non-positive weights were evaluated on signals with different shapes. As depicted in Figure~\ref{figsup:loss-stability}, the Gaussian Mixture Model with positive weights consistently yielded the lowest loss across the number of components, indicating its effective approximation of the square signal. Conversely, non-positive weights in the Gaussian and General models showed a higher loss, emphasizing the importance of weight sign-on model performance.  These findings highlight the intricate balance between model complexity and weight constraints in achieving both low loss and high stability. Note that GEF is very efficient in fitting the square with few components, while LoG and DoG are more stable for a larger number of components. Also, note that positive weight mixtures tend to achieve lower loss with a smaller number of components but are less stable for a larger number of components. 
\input{figures/fitting/fitting_square_p}
\input{figures/fitting/fitting_square_n}
\input{figures/fitting/fitting_parabola_p}
\input{figures/fitting/fitting_parabola_n}
\input{figures/fitting/fitting_exponential_p}
\input{figures/fitting/fitting_exponential_n}
\input{figures/fitting/fitting_traingle_p}
\input{figures/fitting/fitting_traingle_n}
\input{figures/fitting/fitting_gaussian_p}
\input{figures/fitting/fitting_gaussian_n}
\input{figures/fitting/fitting_sinusoid_p}
\input{figures/fitting/fitting_sinusoid_n}

\clearpage \clearpage
\section{Genealized Exponential Splatting Details}\label{secsup:splatting} 
\subsection{Upper Bound on the Boundary View-Dependant Error in the Approximate GES Rasterization}
Given the Generalized Exponential Splatting (GES) function defined in \eqlabel{\eqref{eq:ges}} and our approximate rasterization given by \eqlabel{\eqref{eq:transimttance},\ref{eq:effective}, and \ref{eq:probablity}}, we seek to establish an upper bound on the error of our approximation in \methodname rendering. Since it is very difficult to estimate the error accumulated in each individual pixel from \eqlabel{\eqref{eq:transimttance}}, we seek to estimate the error directly on each splatting component affecting the energy of all passing rays. 

Let us consider a simple 2D case with symmetrical components as in \figlabel{\ref{fig:explaination}}. The error between the scaled Gaussian component and the original \methodname component is related to the energy loss of rays and can be represented by simply estimating the \textit{ratio} $\eta$ between the area difference and the area of the scaled Gaussian. Here we will show we can estimate an upper bound on $\eta$ relative to the area of each component. 

For the worst-case scenario when $\beta \rightarrow \infty$, we consider two non-overlapping conditions for the approximation: one where the square is the outer shape and one where the circle covers the square. The side length of the square is $2r$ for the former case and $2r/\sqrt{2}$ for the latter case. The radius $r$ of the circle is determined by the effective projected variance $\alpha$ from \eqlabel{\eqref{eq:effective}}.
For a square with side length $2r$ and a circle with radius $r$, we have: $
A_{\text{square}} = 4r^2, A_{\text{circle}} = \pi r^2.$
For a square with side length $2r/\sqrt{2}$, the area is:$
A_{\text{square, covered}} = 2r^2.$

The area difference $\Delta A$ is:
\begin{align}
\Delta A_{\text{square larger}} &= A_{\text{square}} - A_{\text{circle}} = 4r^2 - \pi r^2, \\
\Delta A_{\text{circle larger}} &= A_{\text{circle}} - A_{\text{square, covered}} = \pi r^2 - 2r^2.
\end{align}

The ratio of the difference in areas to the area of the inner shape, denoted as $\eta$, is bounded by:
\begin{align}
\eta_{\text{square larger}} &= \frac{\Delta A_{\text{square larger}}}{A_{\text{circle}}} = \frac{4r^2 - \pi r^2}{\pi r^2} \approx 0.2732, \\
\eta_{\text{circle larger}} &= \frac{\Delta A_{\text{circle larger}}}{A_{\text{circle}}} = \frac{\pi r^2 - 2r^2}{\pi r^2} \approx 0.3634.
\end{align}

Due to the PDF normalization constraint in GND \cite{generlizedgaussian}, the approximation followed in \eqlabel{\eqref{eq:effective}, and \ref{eq:probablity}} will always ensure $\eta_{\text{square larger}} \leq \eta \leq \eta_{\text{circle larger}}$. Thus, our target ratio $\eta$ when using our approximate scaling of variance based on $\beta$ should be within the range $0.2732 \leq \eta \leq 0.3634$. This implies in the worst case, our \methodname approximation will result in 36.34\% energy error in the lost energy of all rays passing through \textit{all} the splatting components. In practice, the error will be much smaller due to the large number of components and the small scale of all the splatting components.

\subsection{Implementation Details} 
\label{secsup:implementation-details}
Note that the $\text{DoG}$ in \eqlabel{\eqref{eq:mask}} will be very large when $\sigma_2$ is large, so we downsample the ground truth image by a factor `$\text{scale}_{\text{im,freq}}$` and upsample the mask $M_\omega$ similarly before calculating the loss in \eqlabel{\eqref{eq:laplace}}.
In the implementation of our Generalized Exponential Splatting (\methodname) approach, we fine-tuned several hyperparameters to optimize the performance. The following list details the specific values and purposes of each parameter in our implementation:

\begin{itemize}
    \item Iterations: The algorithm ran for a total of 40,000 iterations.
    \item Learning Rates:
    \begin{itemize}
        \item Initial position learning rate ($ \text{lr}_{\text{pos, init}} $) was set to 0.00016.
        \item Final position learning rate ($ \text{lr}_{\text{pos, final}} $) was reduced to 0.0000016.
        \item Learning rate delay multiplier ($ \text{lr}_{\text{delay mult}} $) was set to 0.01.
        \item Maximum steps for position learning rate ($ \text{lr}_{\text{pos, max steps}} $) were set to 30,000.
    \end{itemize}
    \item Other Learning Rates:
    \begin{itemize}
        \item Feature learning rate ($ \text{lr}_{\text{feature}} $) was 0.0025.
        \item Opacity learning rate ($ \text{lr}_{\text{opacity}} $) was 0.05.
        \item Shape and rotation learning rates ($ \text{lr}_{\text{shape}} $ and $ \text{lr}_{\text{rotation}} $) were both set to 0.001.
        \item Scaling learning rate ($ \text{lr}_{\text{scaling}} $) was 0.005.
    \end{itemize}
    \item Density and Pruning Parameters:
    \begin{itemize}
        \item Percentage of dense points ($ \text{percent}_{\text{dense}} $) was 0.01.
        \item Opacity and shape pruning thresholds were set to 0.005.
    \end{itemize}
    \item Loss Weights and Intervals:
    \begin{itemize}
        \item SSIM loss weight ($ \lambda_{\text{ssim}} $) was 0.2.
        \item Densification, opacity reset, shape reset, and shape pruning intervals were set to 100, 3000, 1000, and 100 iterations, respectively.
    \end{itemize}
    \item Densification Details:
    \begin{itemize}
        \item Densification started from iteration 500 and continued until iteration 15,000.
        \item Gradient threshold for densification was set to 0.0003.
    \end{itemize}
    \item Image Laplacian Parameters:
    \begin{itemize}
        \item Image Laplacian scale factor ($ \text{scale}_{\text{im,freq}} $) was 0.2.
        \item Weight for image Laplacian loss ($ \lambda_{\omega} $) was 0.5.
    \end{itemize}
    \item Miscellaneous:
    \begin{itemize}
        \item Strength of shape $ \rho$ was set to 0.1.
    \end{itemize}
\end{itemize}

These parameters were carefully chosen to balance the trade-off between computational efficiency and the fidelity of the synthesized views. The above hyperparameter configuration played a crucial role in the effective implementation of our \methodname approach. For implementation purposes, the modification functions have been shifted by -2 and the $\beta$ initialization is set to 0 instead of 2 ( which should not have any effect on the optimization).

\section{Additional Results and Analysis} \label{secsup:results}
\subsection{Additional Results} 
We show in \figlabel{\ref{fig:comparisons-supp}} additional \methodname results (test views) and comparisons to the ground truth and baselines.  In \figlabel{\ref{supfig:detailed}}, show PSNR, LPIPS, SSIM, and file size results for every single scene in MIPNeRF 360 dataset \cite{MipNeRF-360,ZipNerf} of our \methodname and re-running the Gaussian Splatting \cite{gaussiansplatter} baseline with the \textit{exact same} hyperparameters of our  \methodname and on different number of iterations. 

\subsection{Applying \methodname in Fast 3D Generation}\label{sec:3dgen}
\vspace{-2pt}
\methodname is adapted to modern 3D generation pipelines using score distillation sampling from a 2D text-to-image model \cite{DreamFusion}, replacing Gaussian Splatting for improved efficiency. We employ the same setup as DreamGaussian \cite{dreemgaussian}, altering only the 3D representation to \methodname. This change demonstrates \methodname's capability for real-time representation applications and memory efficiency.

For evaluation, we use datasets NeRF4 and RealFusion15 with metrics PSNR, LPIPS \cite{LPIPS}, and CLIP-similarity \cite{CLIP} following the benchmarks in Realfusion \cite{RealFusion} and Magic123 \cite{magic123}. Our \methodname exhibits swift optimization with an average runtime of 2 minutes, maintaining quality, as shown in Table \ref{tab:img3d} and \figlabel{\ref{figsup:gen3d_vis}}.

\subsection{Shape Parameters} In Table \ref{tab:ablateshape}, we explore the effect of all hyperparameters associated with the new shape parameter on novel view synthesis performance. 
We find that the optimization process is relatively robust to these changes, as it retains relatively strong performance and yields results with similar sizes. 

\minorsection{Density Gradient Threshold}
In \figlabel{\ref{fig:ablatedenity}}, we visualize the impact of modifying the density gradient threshold for splitting, using both \methodname and the standard Gaussian Splatting ( after modifying the setup for a fair comparison to \methodname). 
We see that the threshold has a significant impact on the tradeoff between performance and size, with a higher threshold decreasing size at the expense of performance. 
Notably, we see that \methodname outperforms GS across the range of density gradient thresholds, yielding similar performance while using less memory. 

\subsection{Analysing the Frequency-Modulated Image Loss}
We study the effect of the frequency-modulated loss $\mathcal{L}_{\omega}$ on the performance by varying $\lambda_{\omega}$ and show the results in Table \ref{figsup:lambda_laplace} and Table \ref{tab:ablation-general-mip}. Note that increasing $\lambda_{\omega}$ in \methodname indeed reduces the size of the file, but can affect the performance. We chose $\lambda_{\omega} =0.5$ as a middle ground between improved performance and reduced file size.  

\subsection{Visualizing the Distribution of Parameters}
We visualize the distribution of shape parameters $\beta$ in \figlabel{\ref{fig:distribution}} and the sizes of the splatting components in \figlabel{\ref{fig:sizes}}. They clearly show a smooth distribution of the components in the scene, which indicates the importance of initialization. This hints a possible future direction in this line of research.

\subsection{Typical Convergence Plots}
We show in \figlabel{\ref{fig:convergence}} examples of the convergence plots of both \methodname and Gaussians if the training continues up to 50K iterations to inspect the diminishing returns of more training. Despite requiring more iterations to converge, \methodname trains faster than Gaussians due to its smaller number of splatting components. 

\input{tables/image_to_3d}
\input{figures/gen3d_vis_supp}
\input{tables/ablate_density_grad_threshold}
\input{figures/supplementary_ablations}
\input{figures/distribution}
\input{figures/sizes}
\input{figures/convergence}

\input{figures/masks_supp}
\input{tables/ablate_shape}
\input{figures/viscomparison_supp}
\input{figures/detailed}
\input{figures/laplacian_supp}
\input{tables/ablation_full_tanks}

%% file: figures/foureir.tex
\begin{figure*}[t]
\centering
\includegraphics[width=0.99\linewidth]{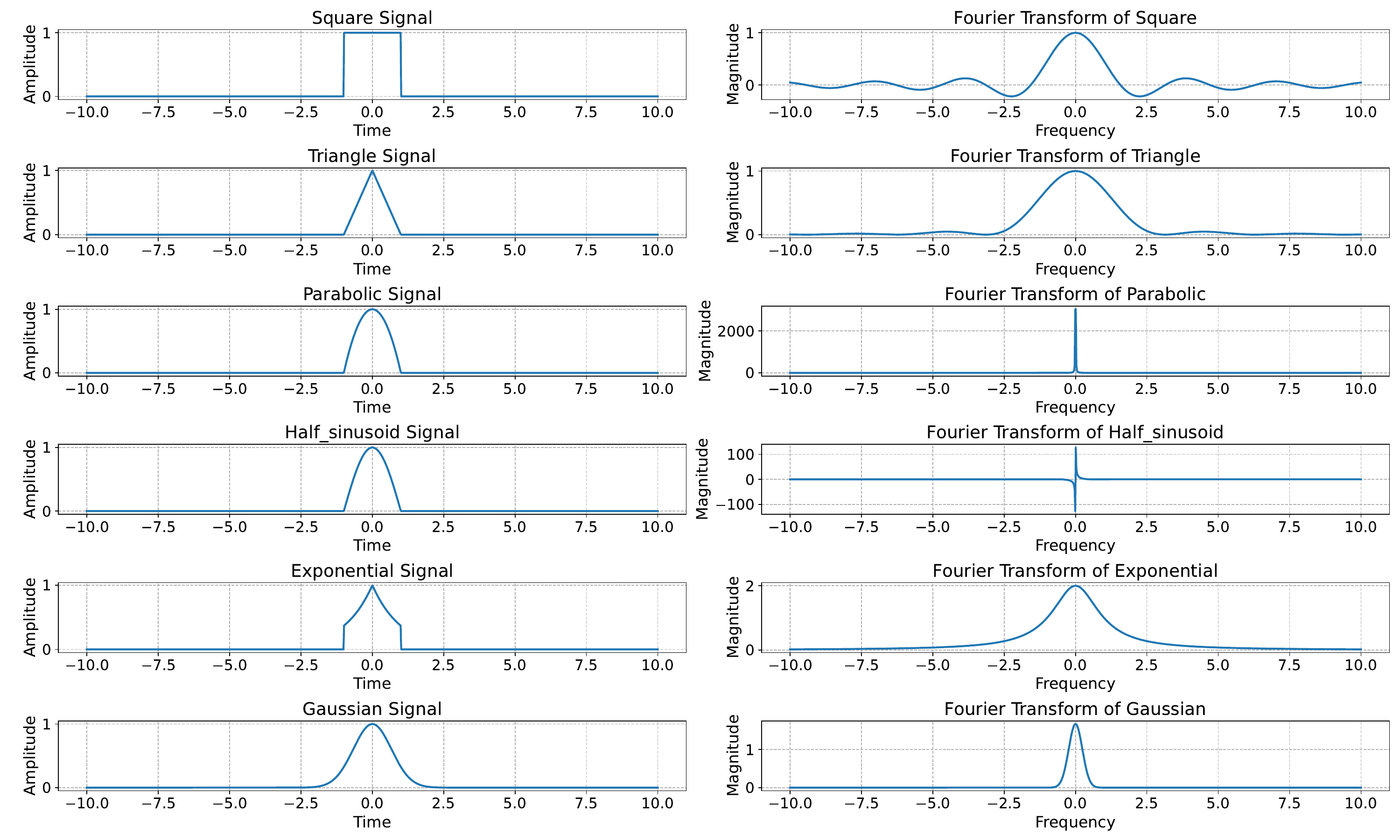}
\caption{\textbf{Commong Signals Used and Their Fourier Transforms}. Note that the Gaussian function is low-pass bandwidth, while common signals like the square and triangle with sharp edges have infinite bandwidth, making them challenging to be fitted with mixtures that have low-pass frequency bandwidth (\eg Gaussian mixtures, represented by Gaussian Splatting \cite{gaussiansplatter}).}
\label{fig:foureir}
\end{figure*}

%% file: figures/theory_supp.tex
\begin{figure*}[t]
\centering
\begin{tabular}{cc|cc}
    \includegraphics[width=0.24\textwidth]{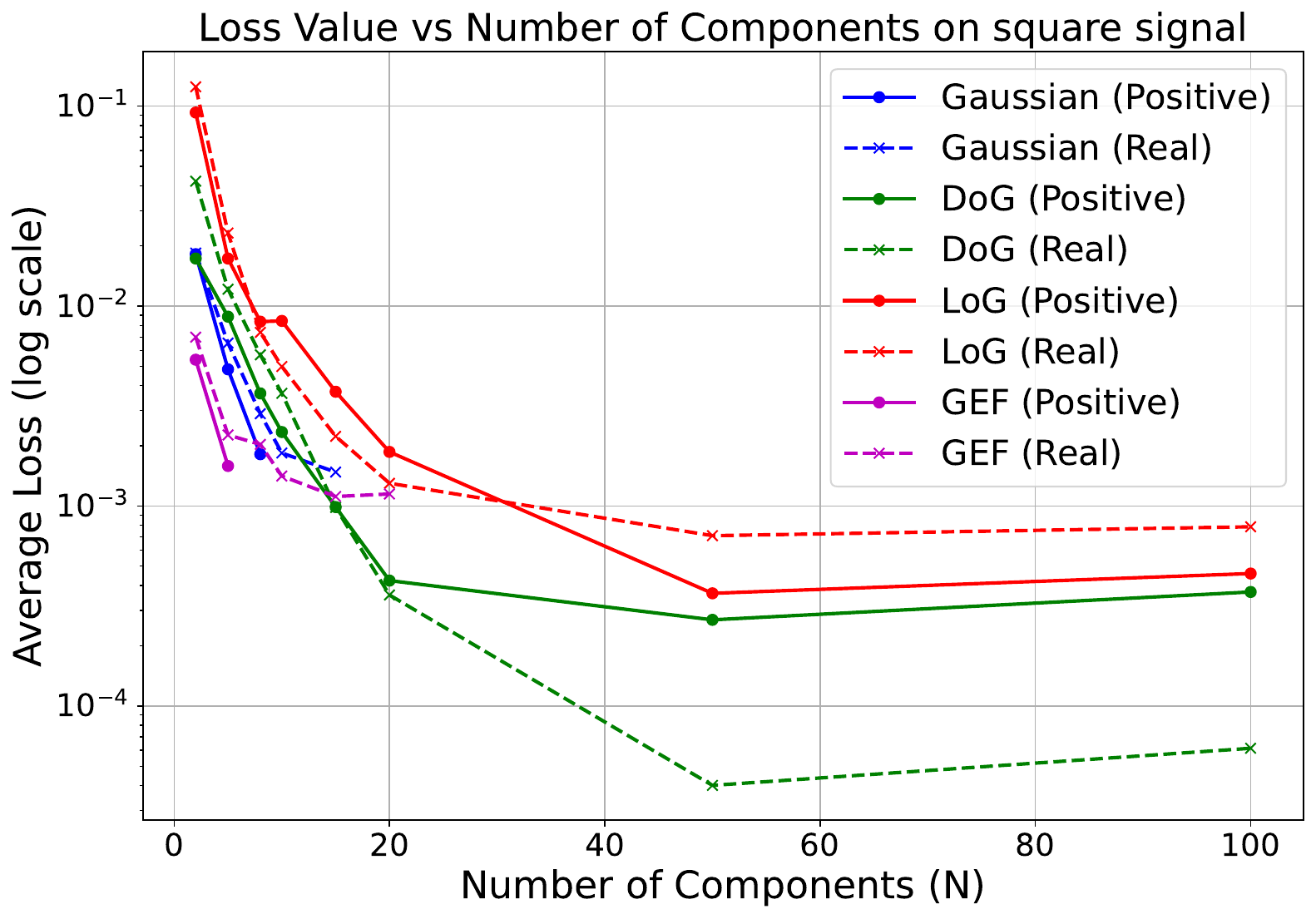} &
    \includegraphics[width=0.24\textwidth]{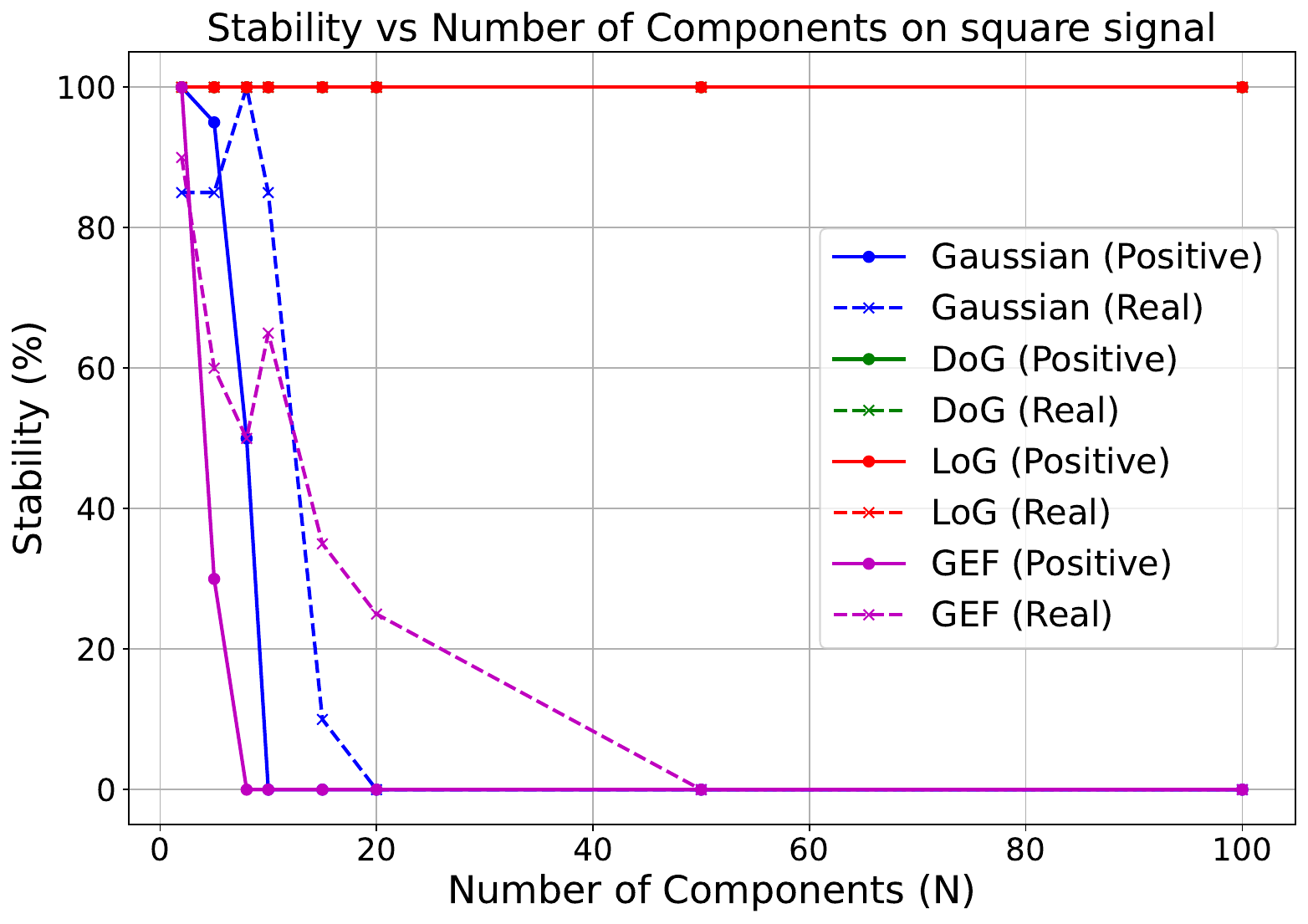} &
    \includegraphics[width=0.24\textwidth]{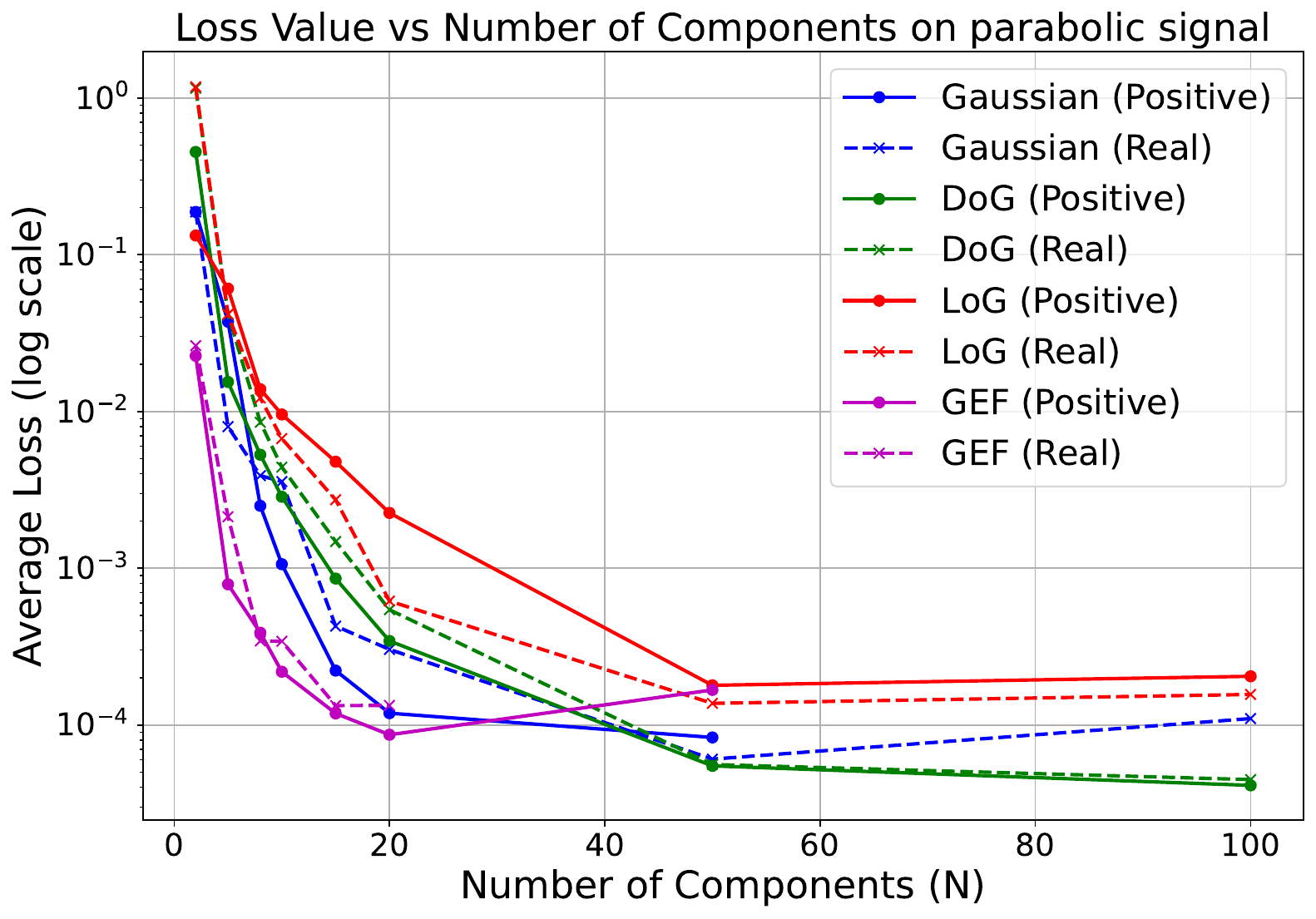} &
    \includegraphics[width=0.24\textwidth]{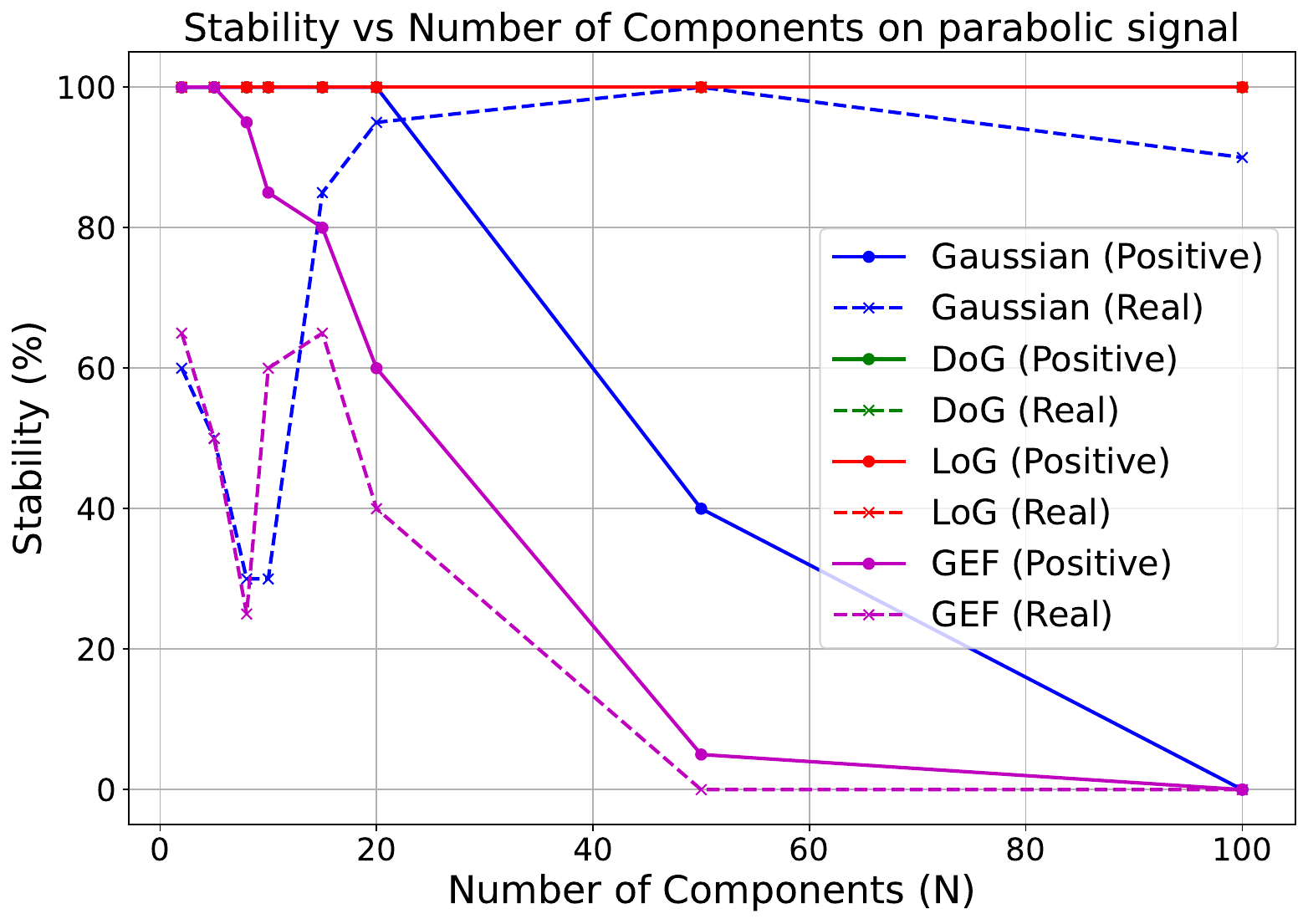} \\ 
    \multicolumn{2}{c|}{(a) Square signal} & \multicolumn{2}{c}{(b) Parabolic signal} \\ \hline
    \includegraphics[width=0.24\textwidth]{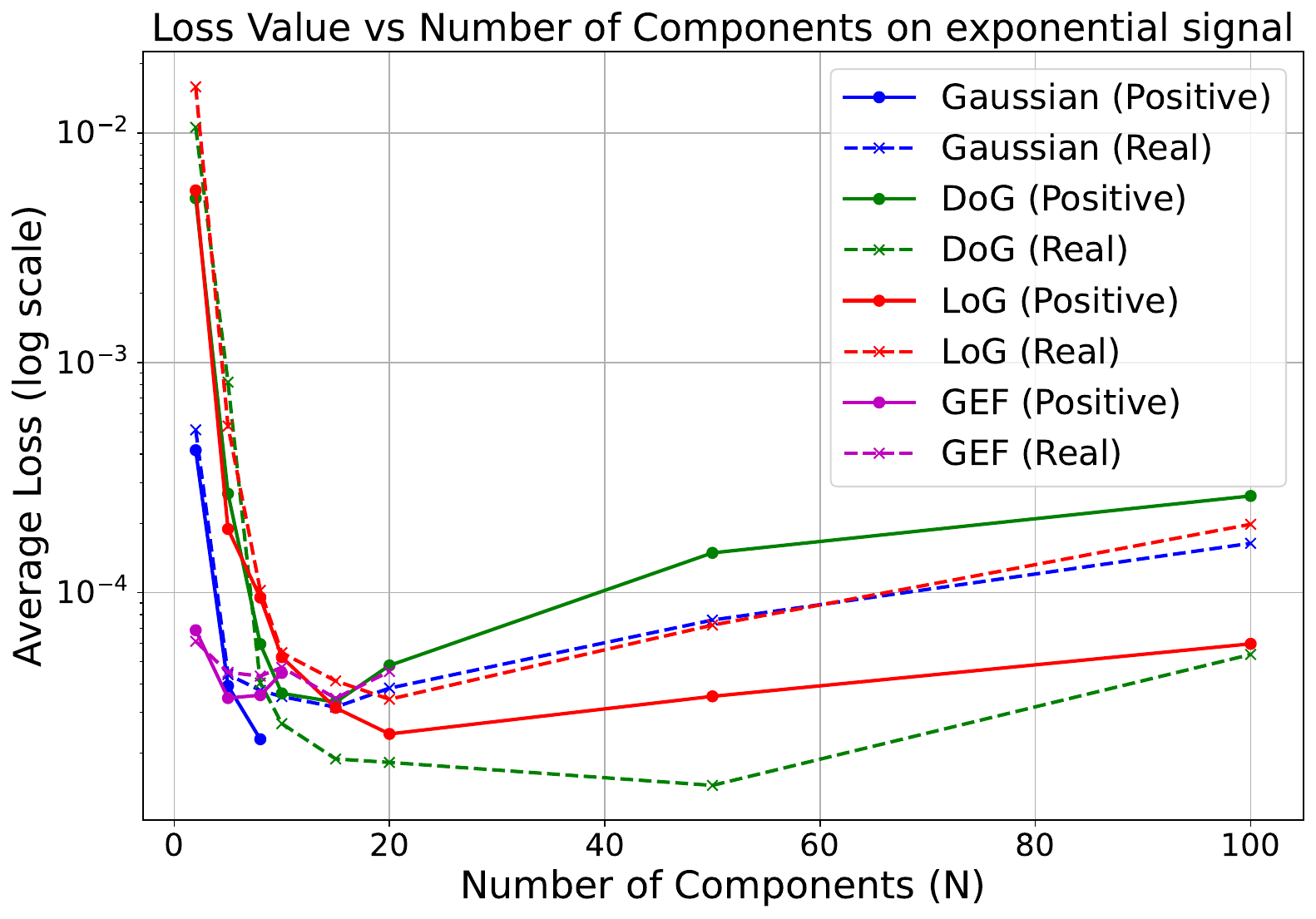} &
    \includegraphics[width=0.24\textwidth]{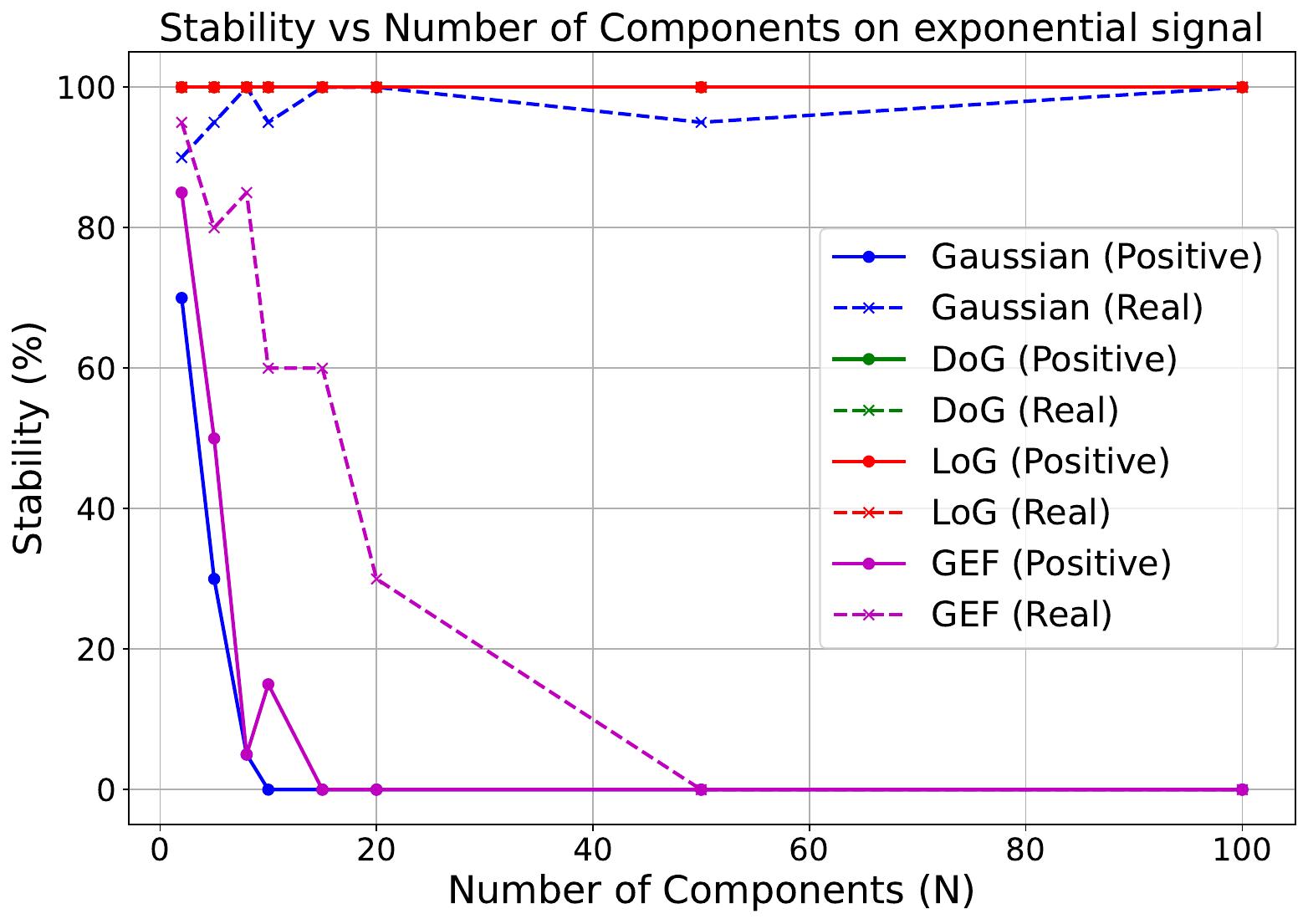} &
    \includegraphics[width=0.24\textwidth]{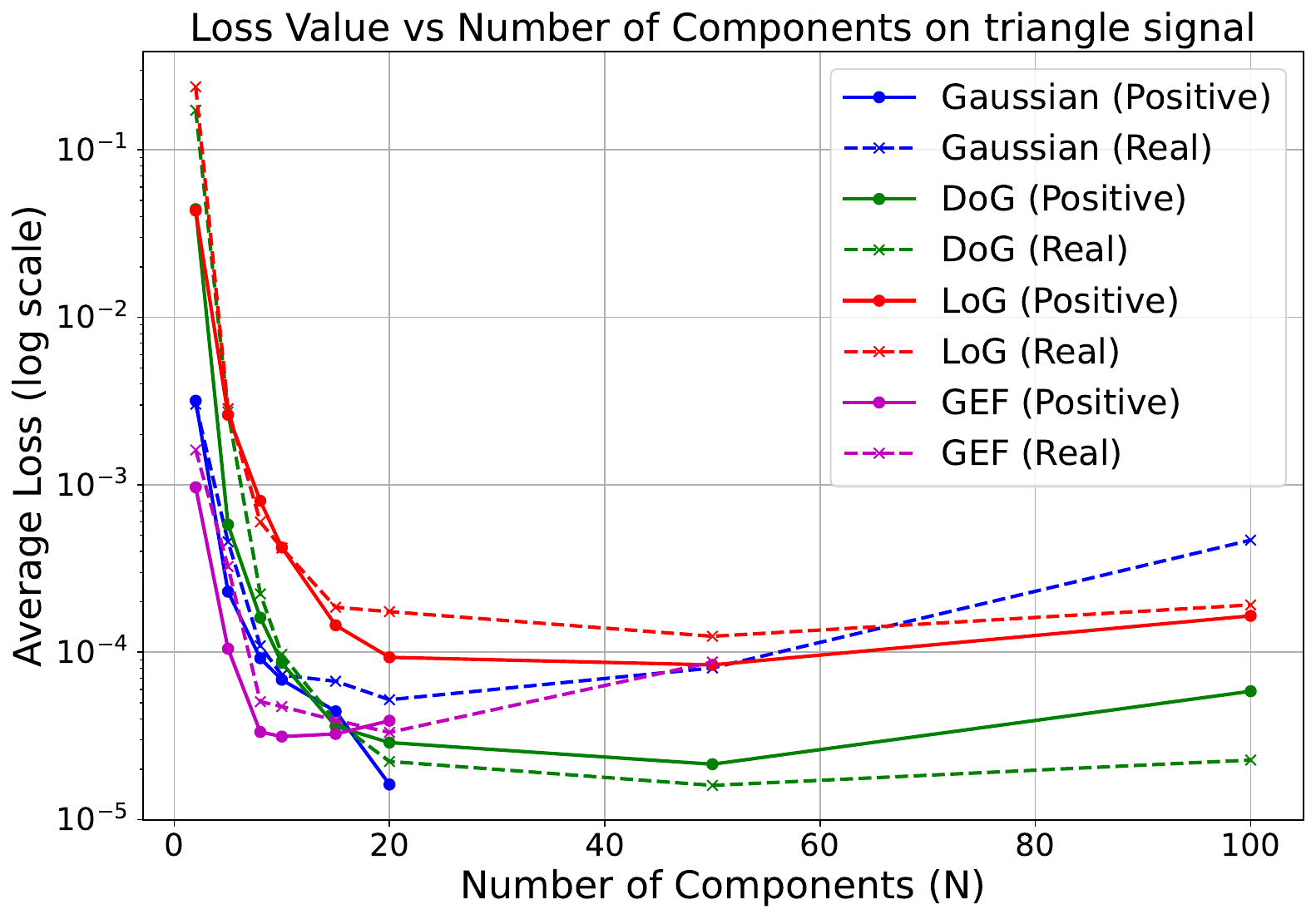} &
    \includegraphics[width=0.24\textwidth]{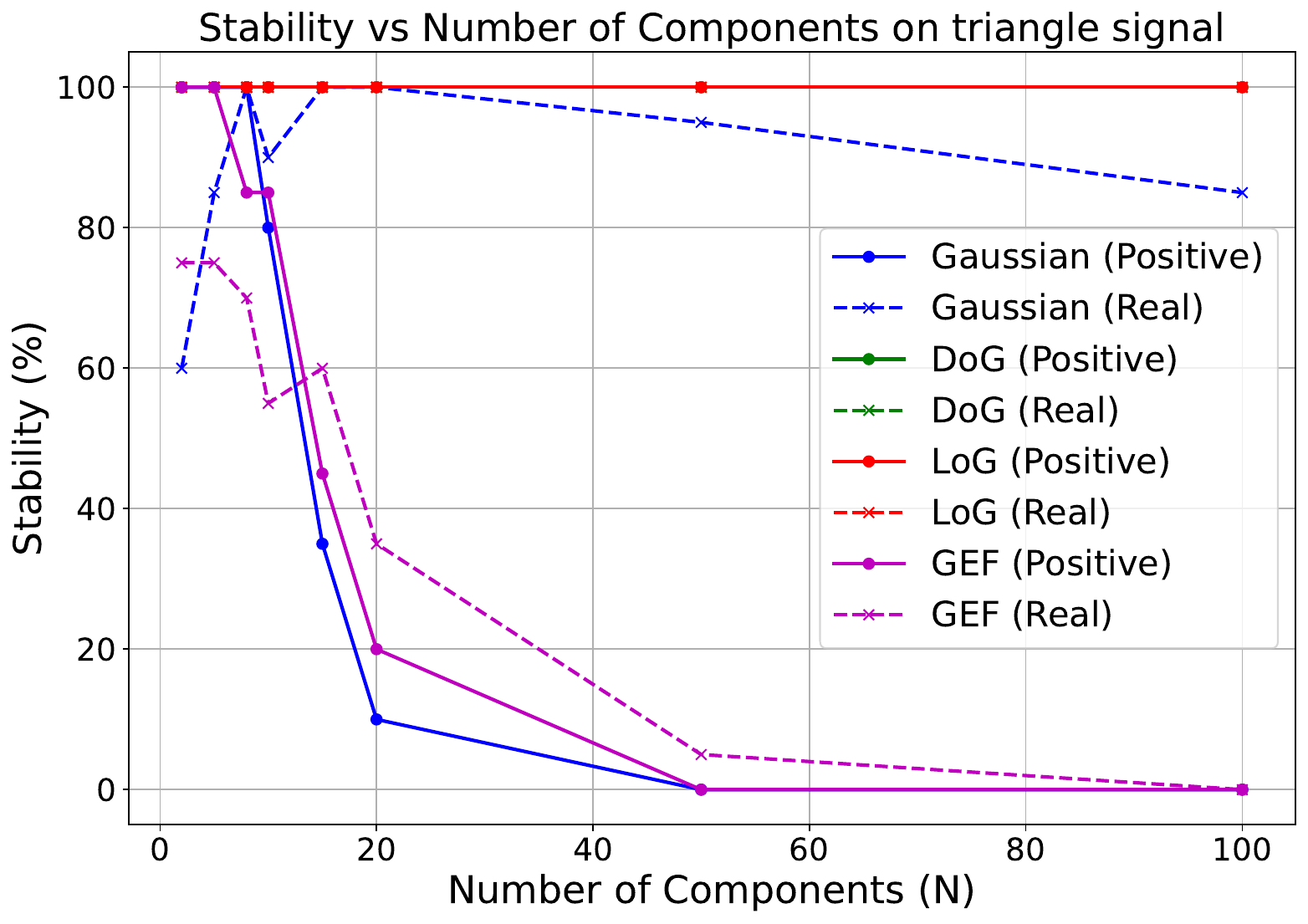} \\
    \multicolumn{2}{c|}{(c) Exponential signal} & \multicolumn{2}{c}{(d) Triangle signal} \\ \hline
    \includegraphics[width=0.24\textwidth]{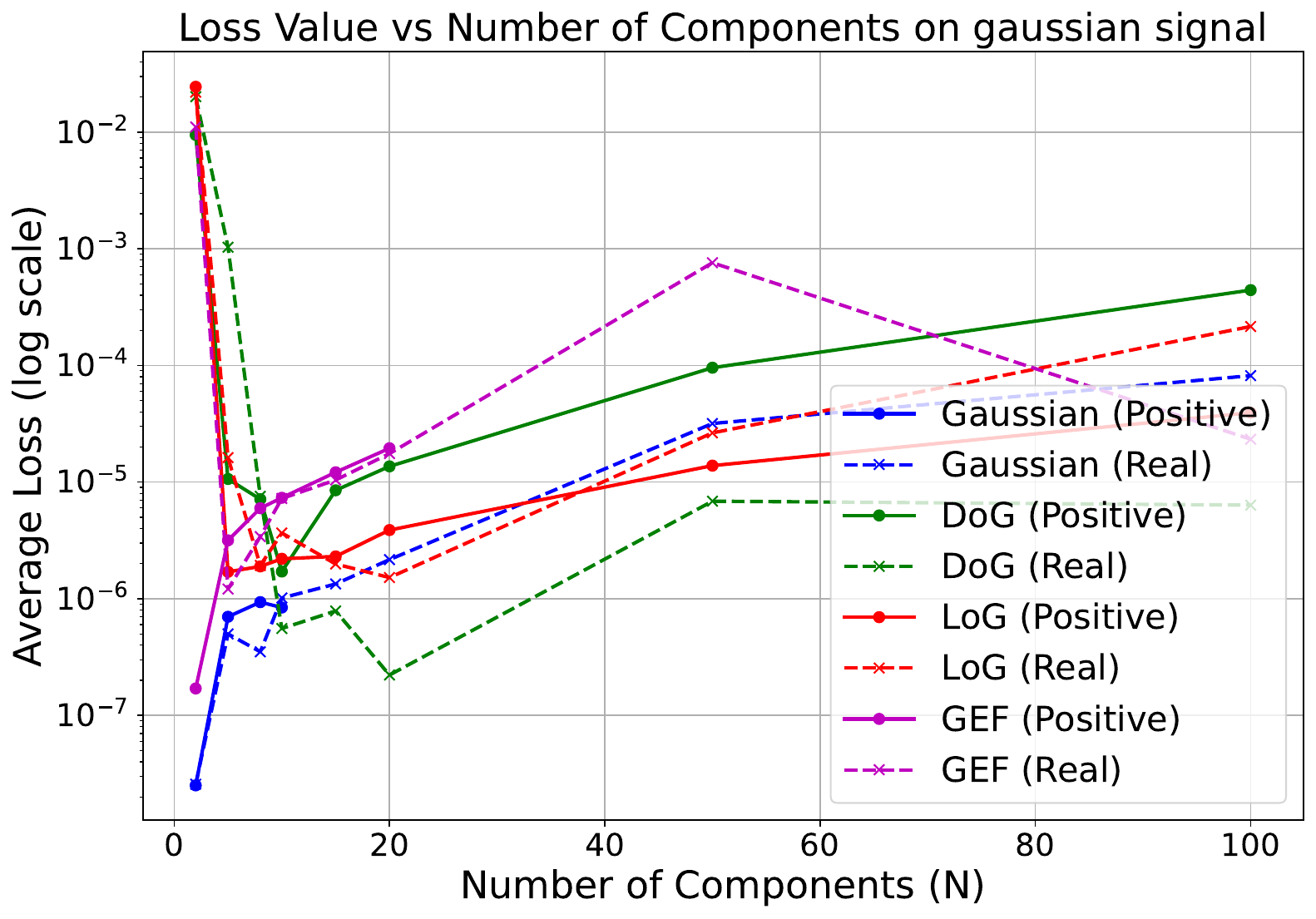} &
    \includegraphics[width=0.24\textwidth]{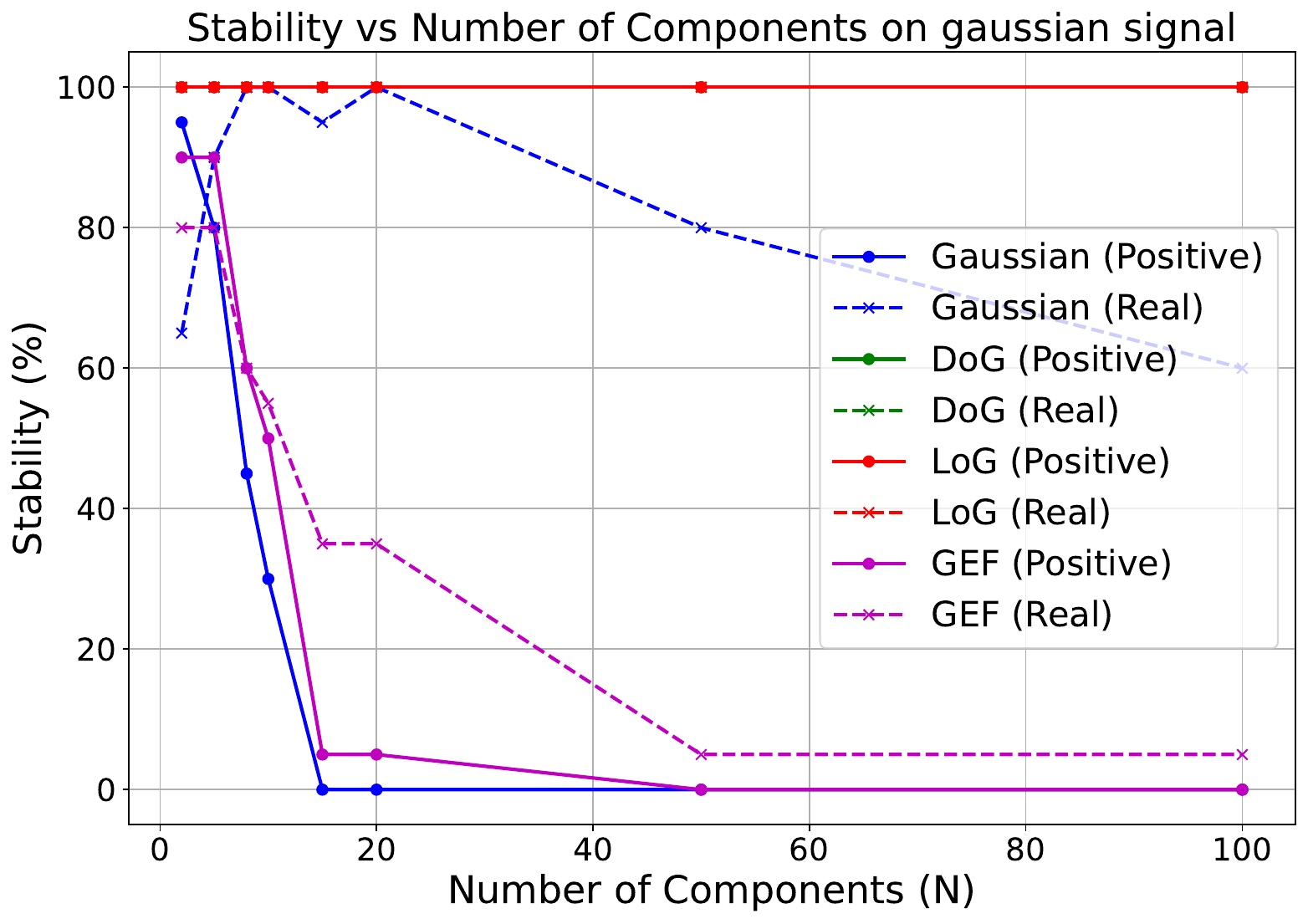} &
    \includegraphics[width=0.24\textwidth]{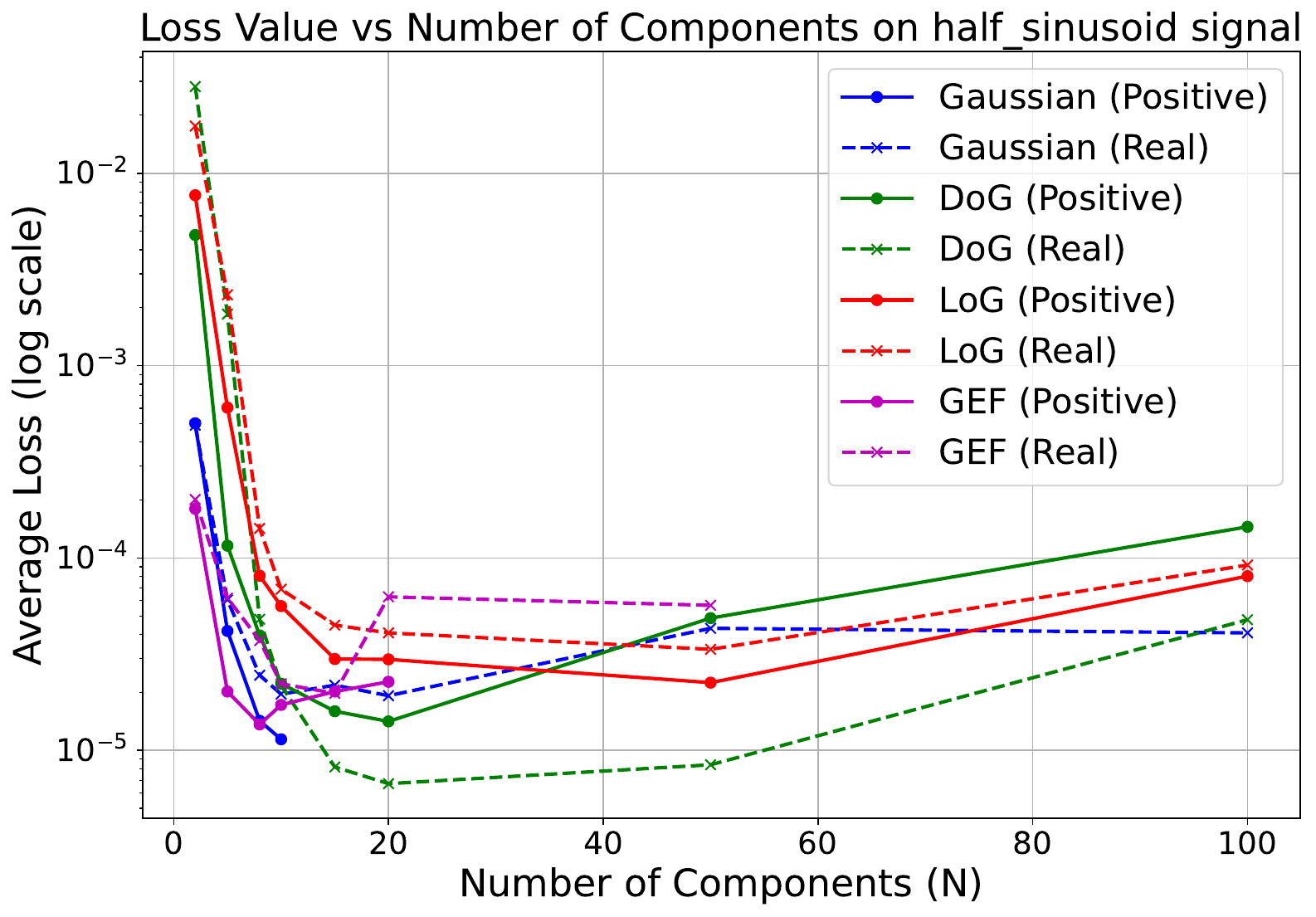} &
    \includegraphics[width=0.24\textwidth]{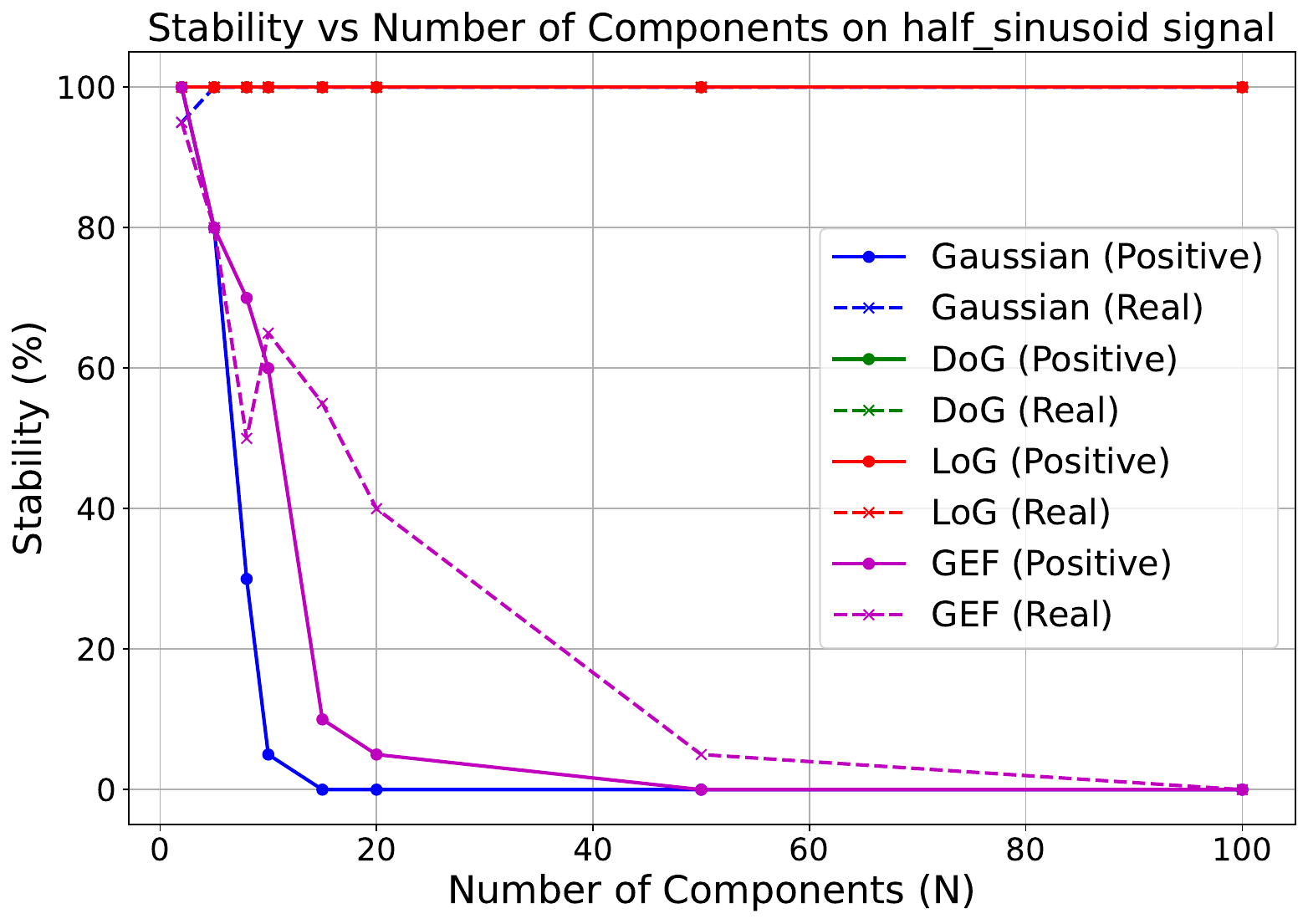} \\
    \multicolumn{2}{c|}{(e) Gaussian signal} & \multicolumn{2}{c}{(f) Half sinusoid signal} \\ 
\end{tabular}
\caption{\textbf{Numerical Simulation Results of Different Mixtures.} We show a comparison of average loss and stability (percentage of successful runs) for different mixture models optimized with gradient-based optimizers across varying numbers of components and weight configurations (positive \vs real weights) on various signal types (a-f). }
\label{figsup:loss-stability}
\end{figure*}

%% file: figures/fitting/fitting_square_p.tex
\begin{figure*}[h]
    \centering
    \resizebox{1.0\linewidth}{!}{
    \begin{tabular}{cccc}
    \tabcolsep=0.01cm
    Gaussian Mixture& LoG Mixture & DoG Mixture & GEF Mixture \\ 
    \includegraphics[width=0.24\linewidth]{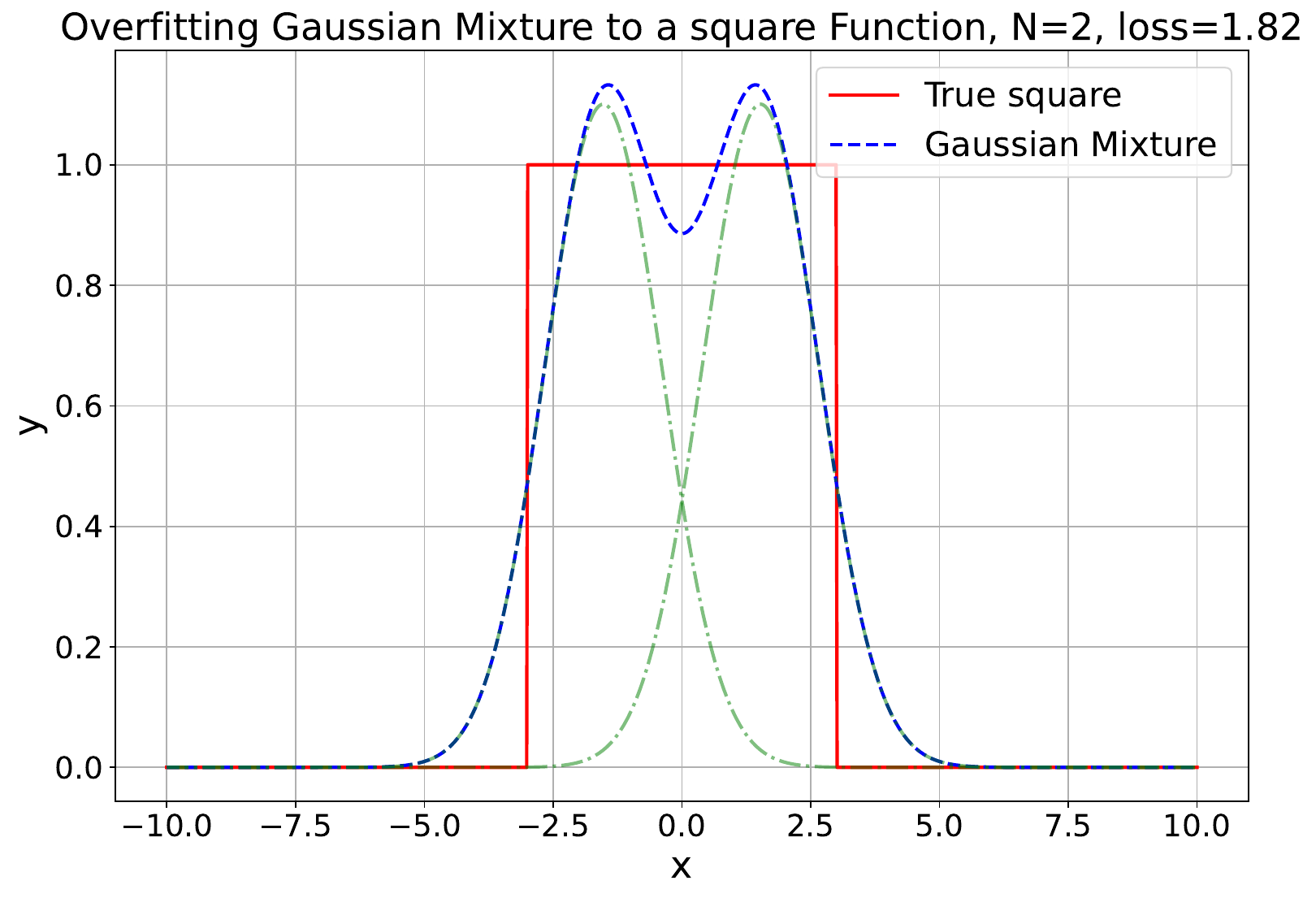} & 
    \includegraphics[width=0.24\linewidth]{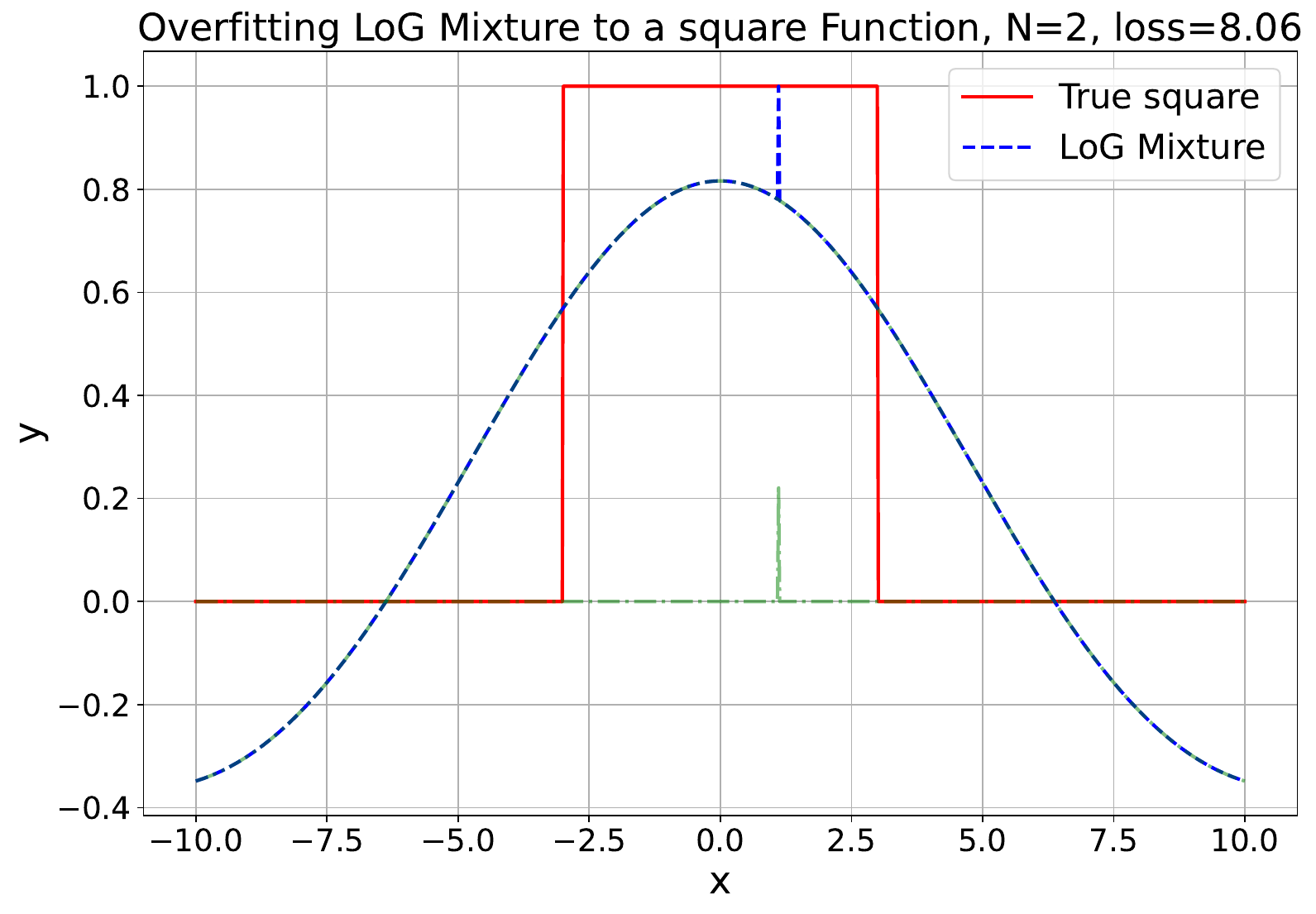} & 
    \includegraphics[width=0.24\linewidth]{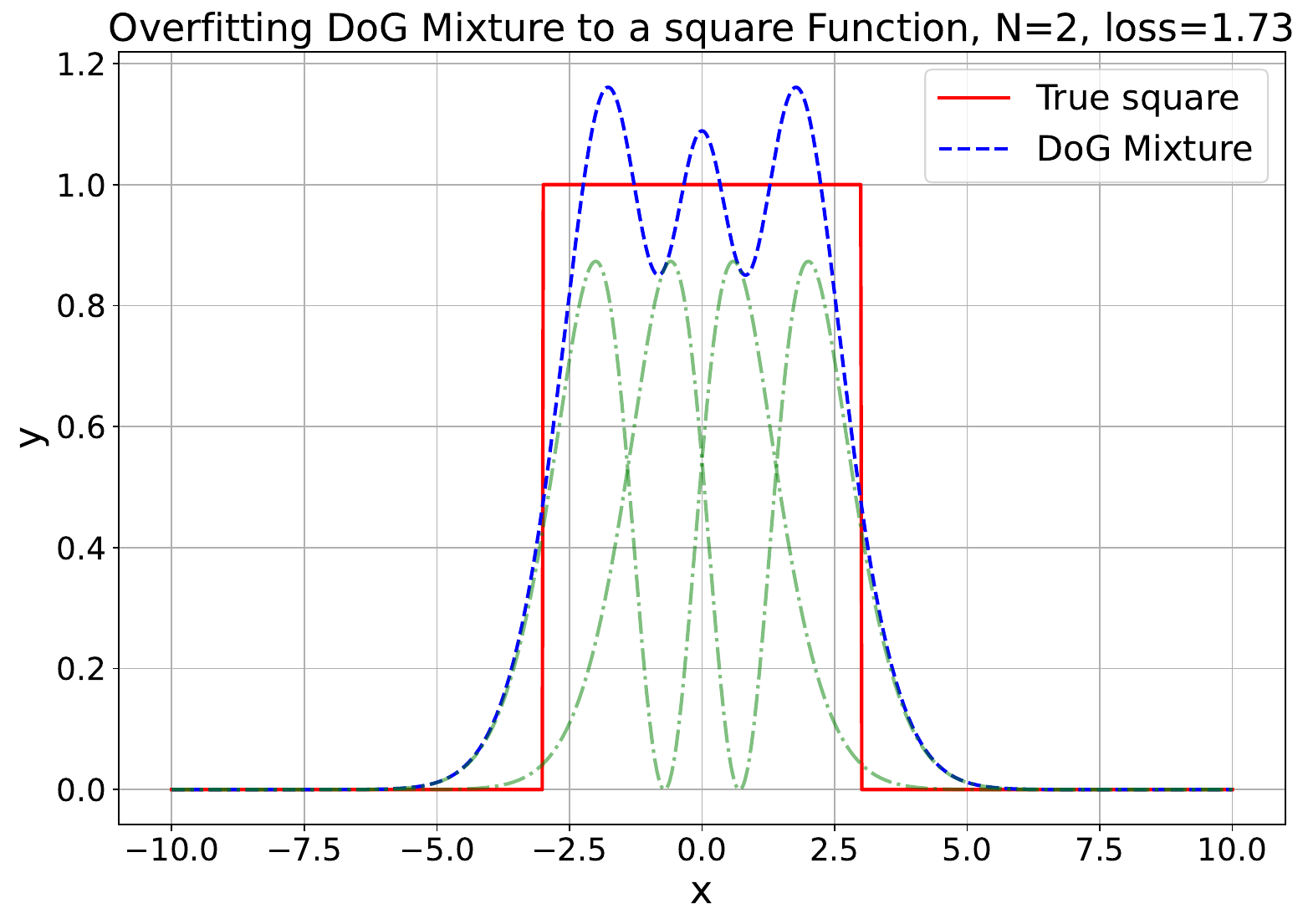} & 
    \includegraphics[width=0.24\linewidth]{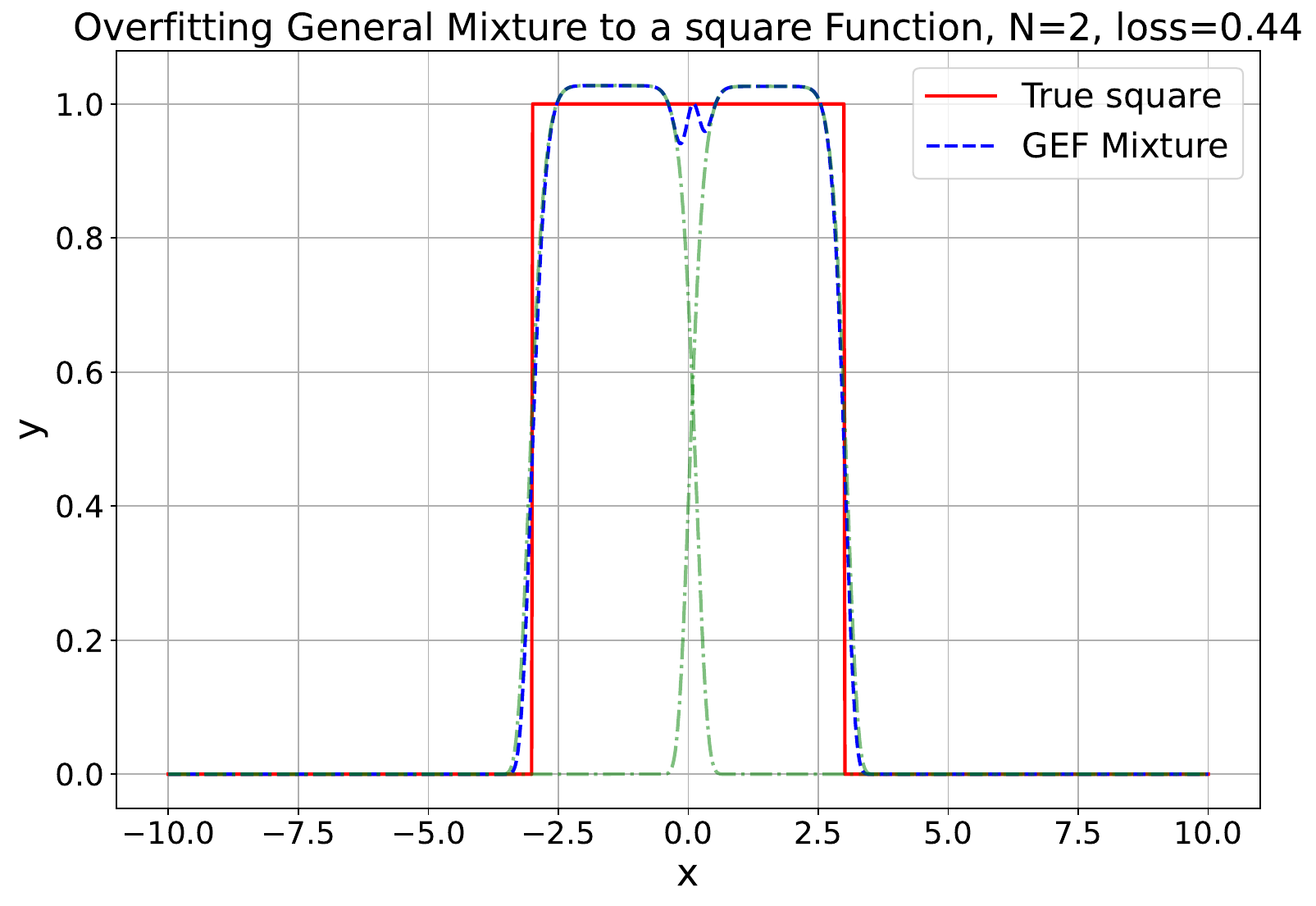}\\ 
    \includegraphics[width=0.24\linewidth]{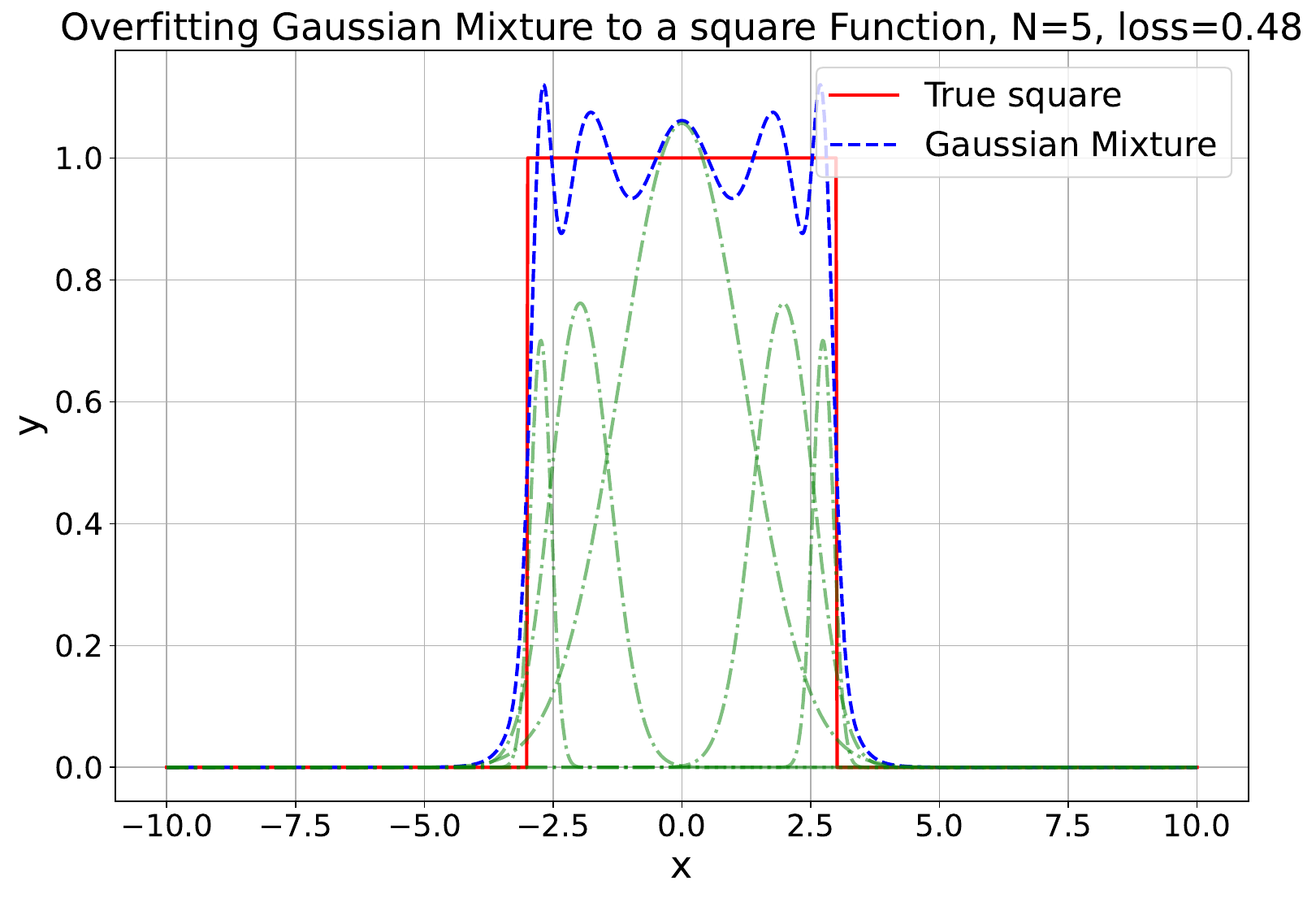} & 
    \includegraphics[width=0.24\linewidth]{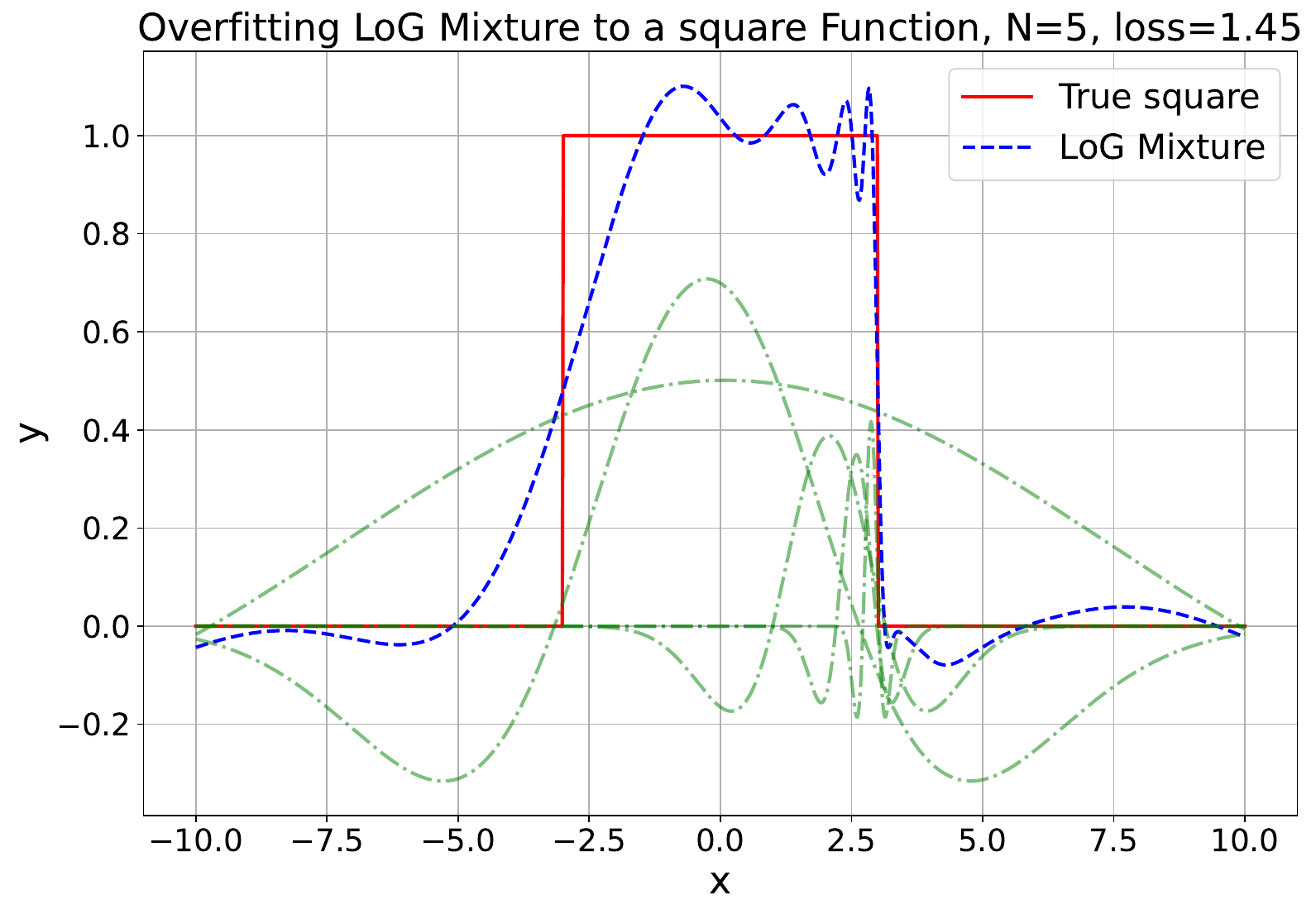} & 
    \includegraphics[width=0.24\linewidth]{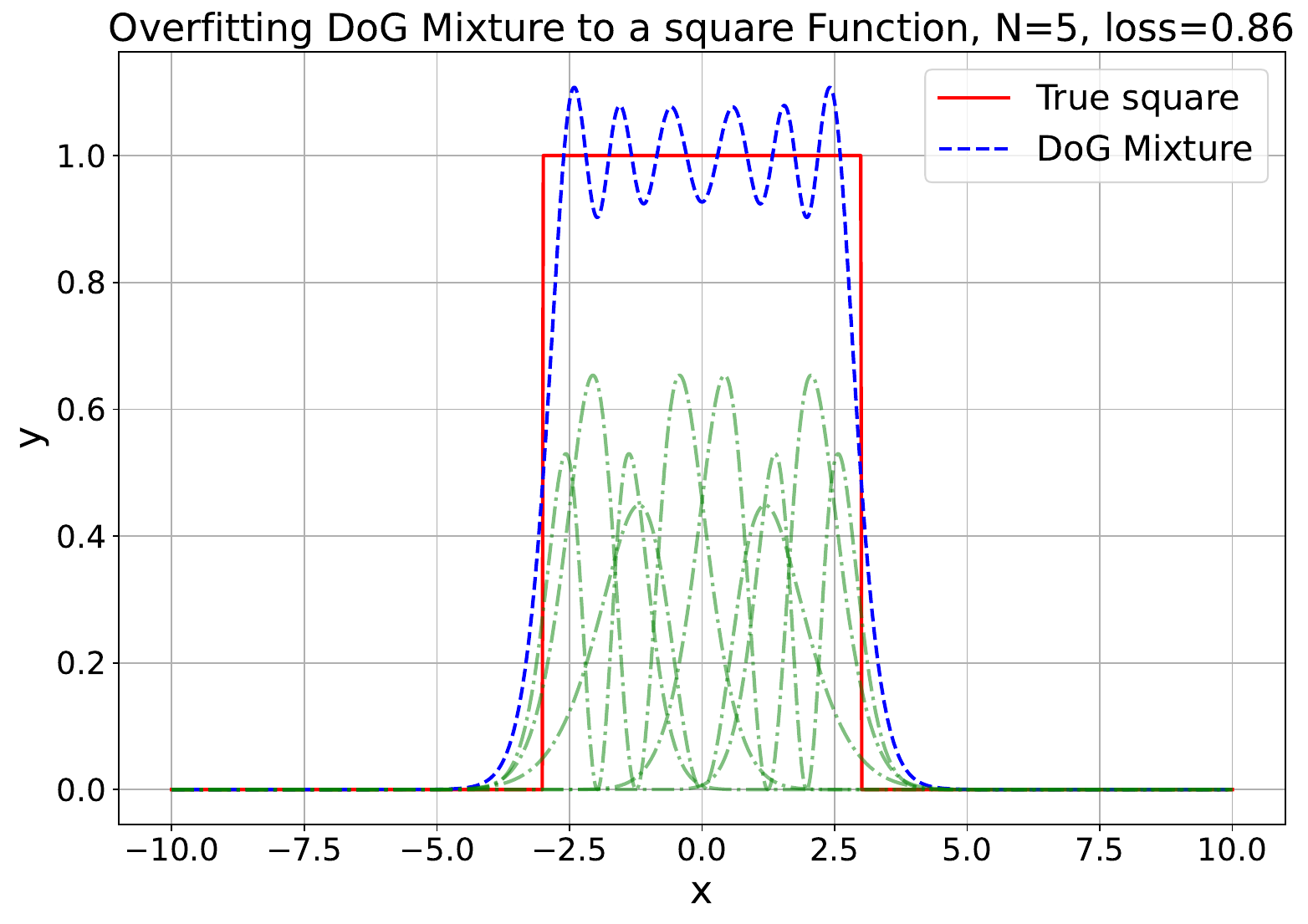} & 
    \includegraphics[width=0.24\linewidth]{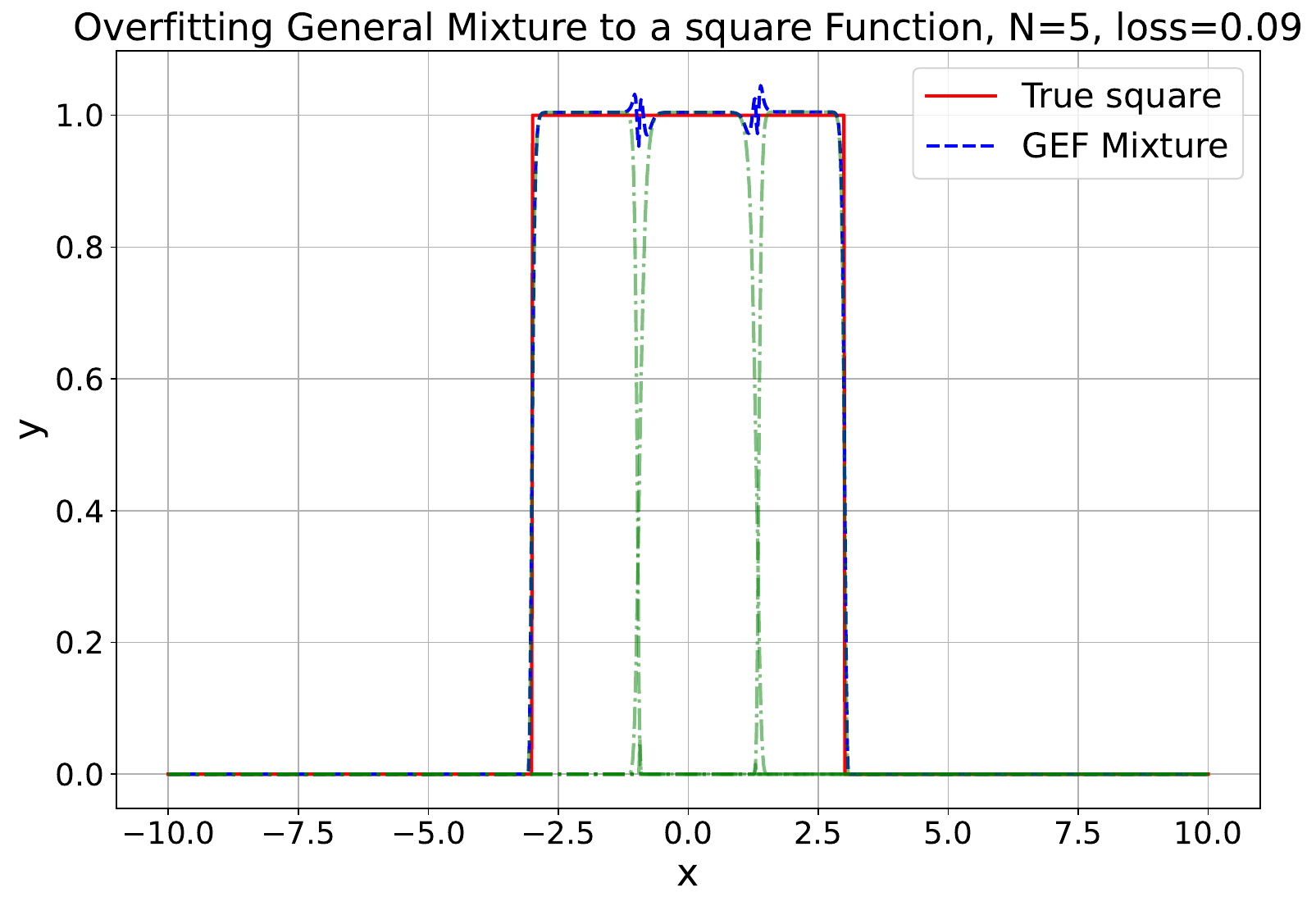}\\ 
    \includegraphics[width=0.24\linewidth]{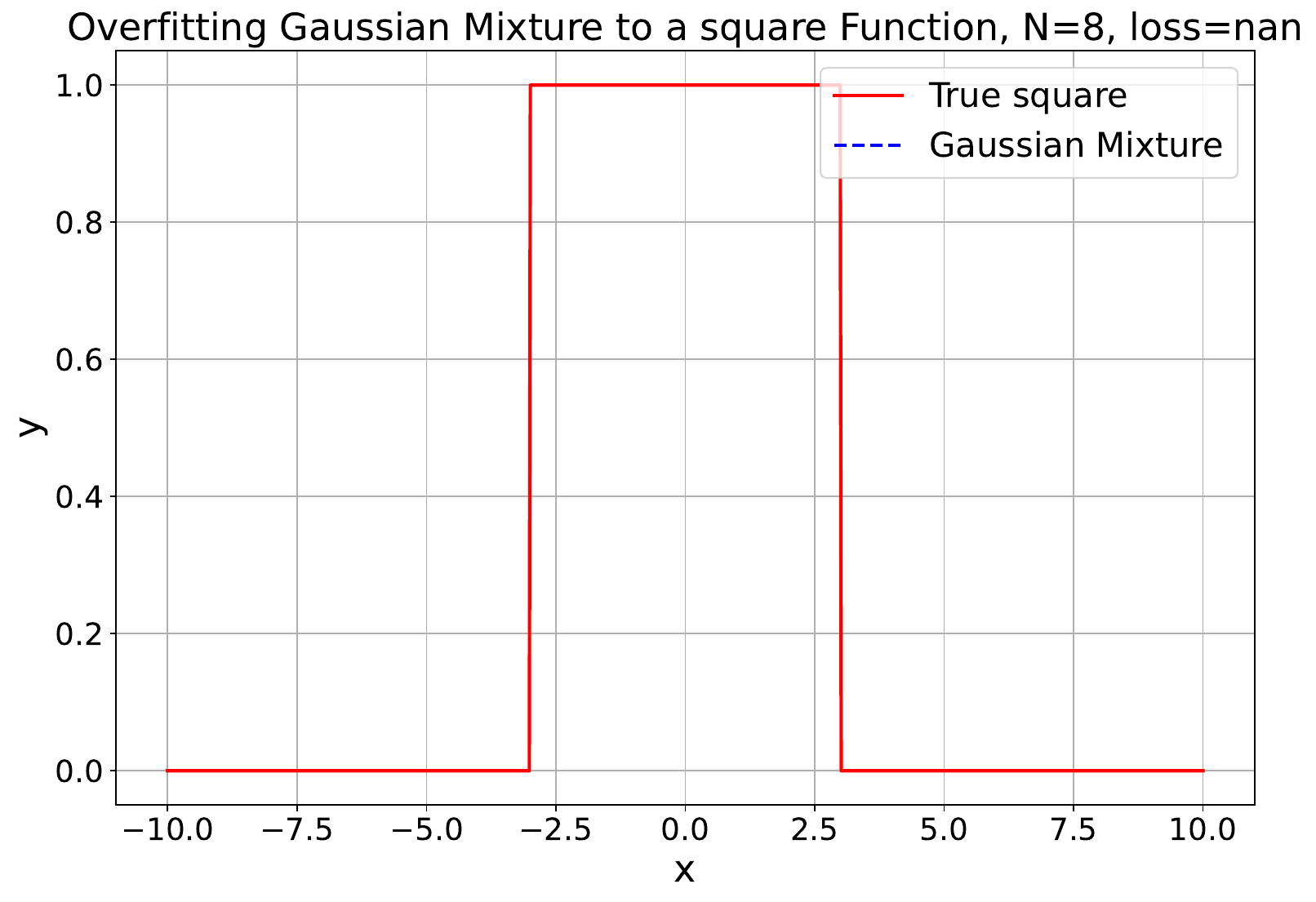} & 
    \includegraphics[width=0.24\linewidth]{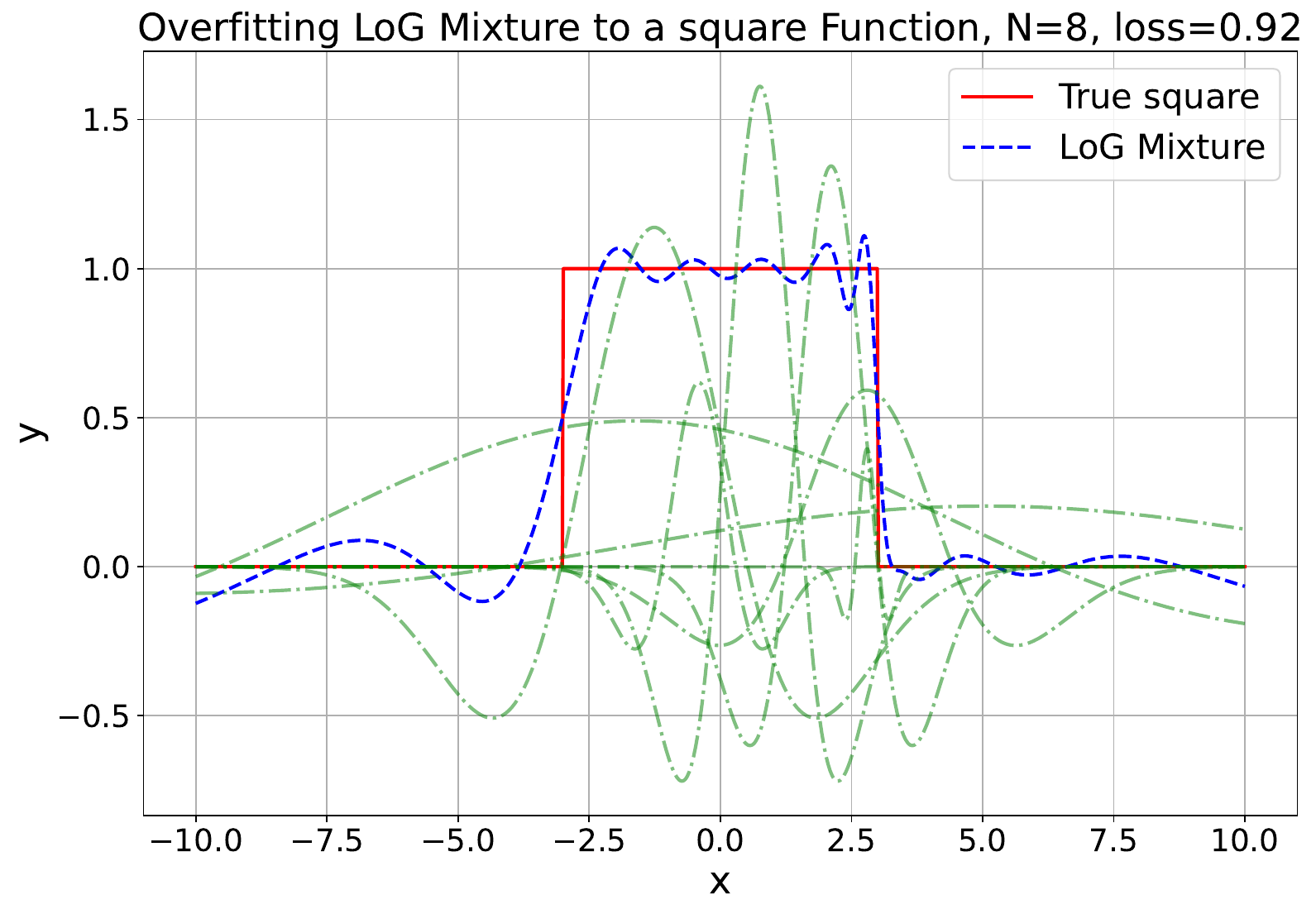} & 
    \includegraphics[width=0.24\linewidth]{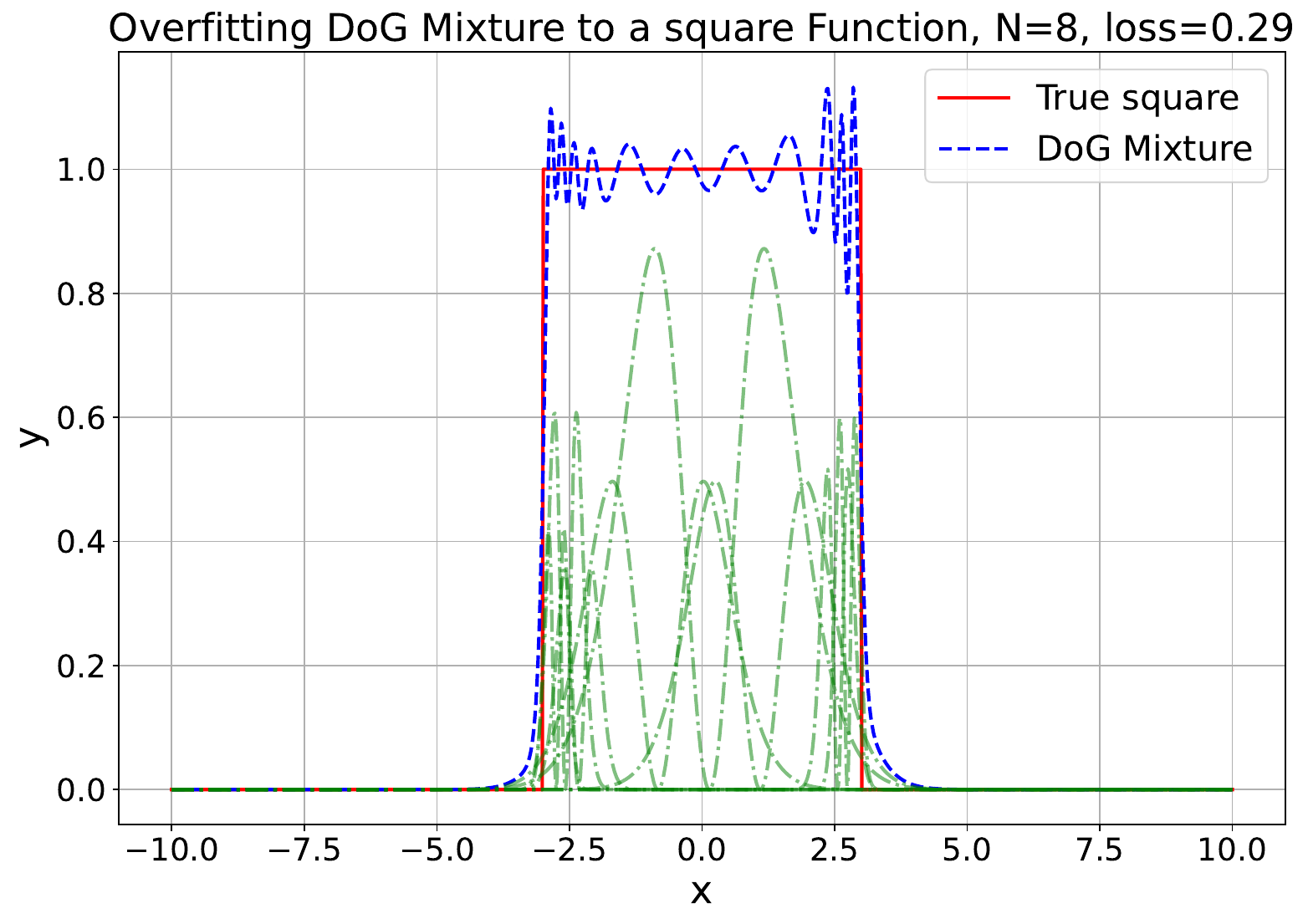} & 
    \includegraphics[width=0.24\linewidth]{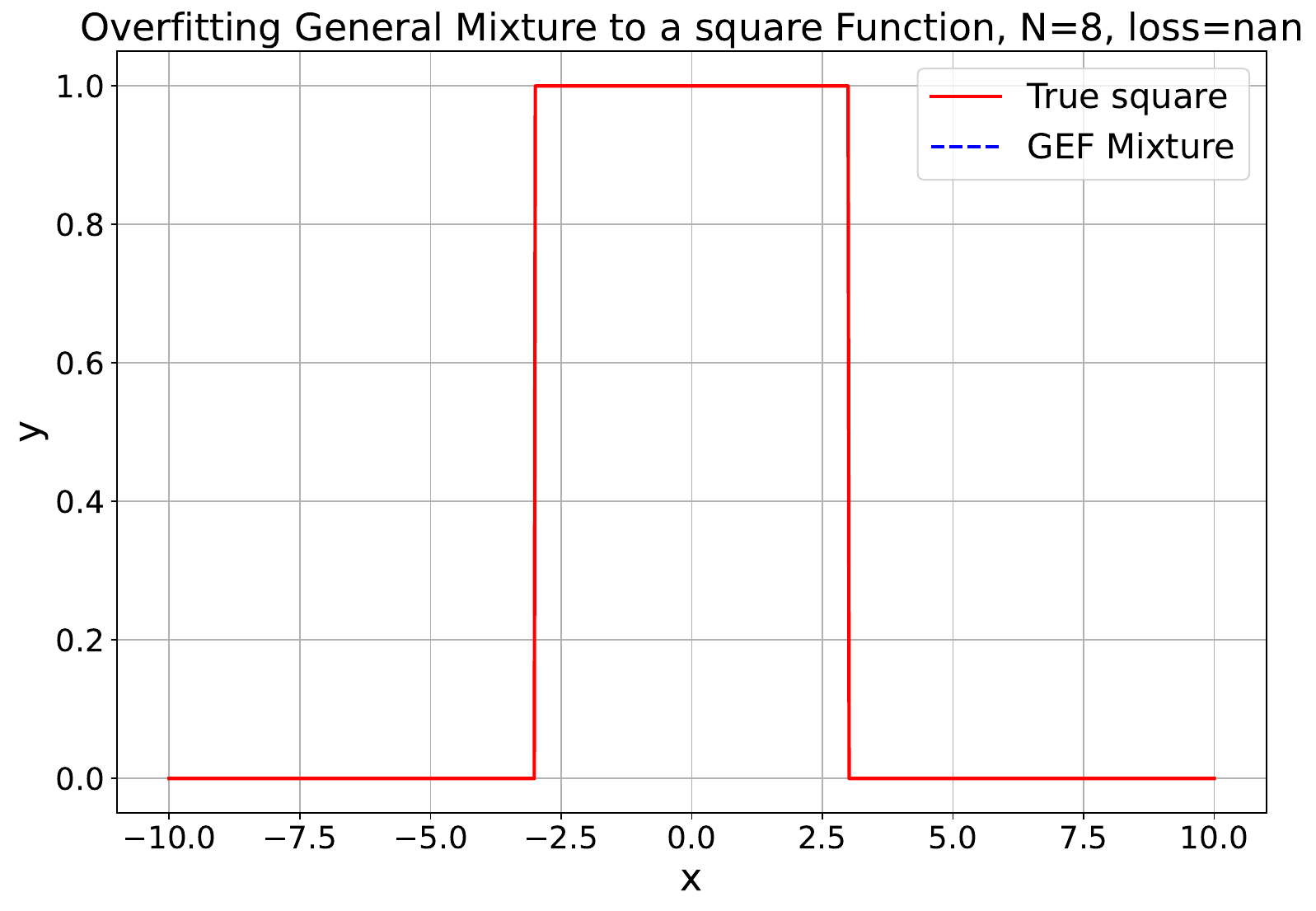}\\ 
    \includegraphics[width=0.24\linewidth]{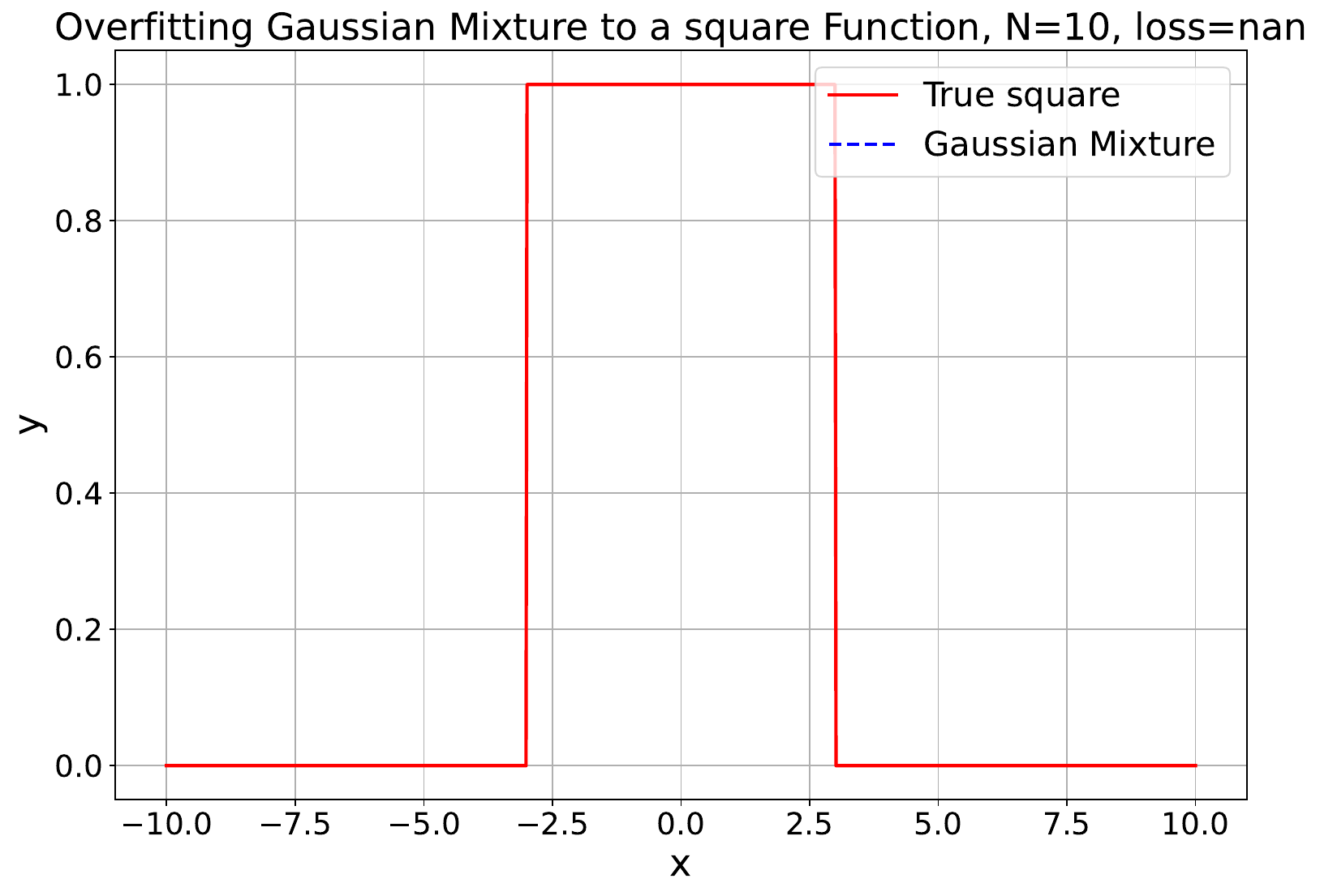} & 
    \includegraphics[width=0.24\linewidth]{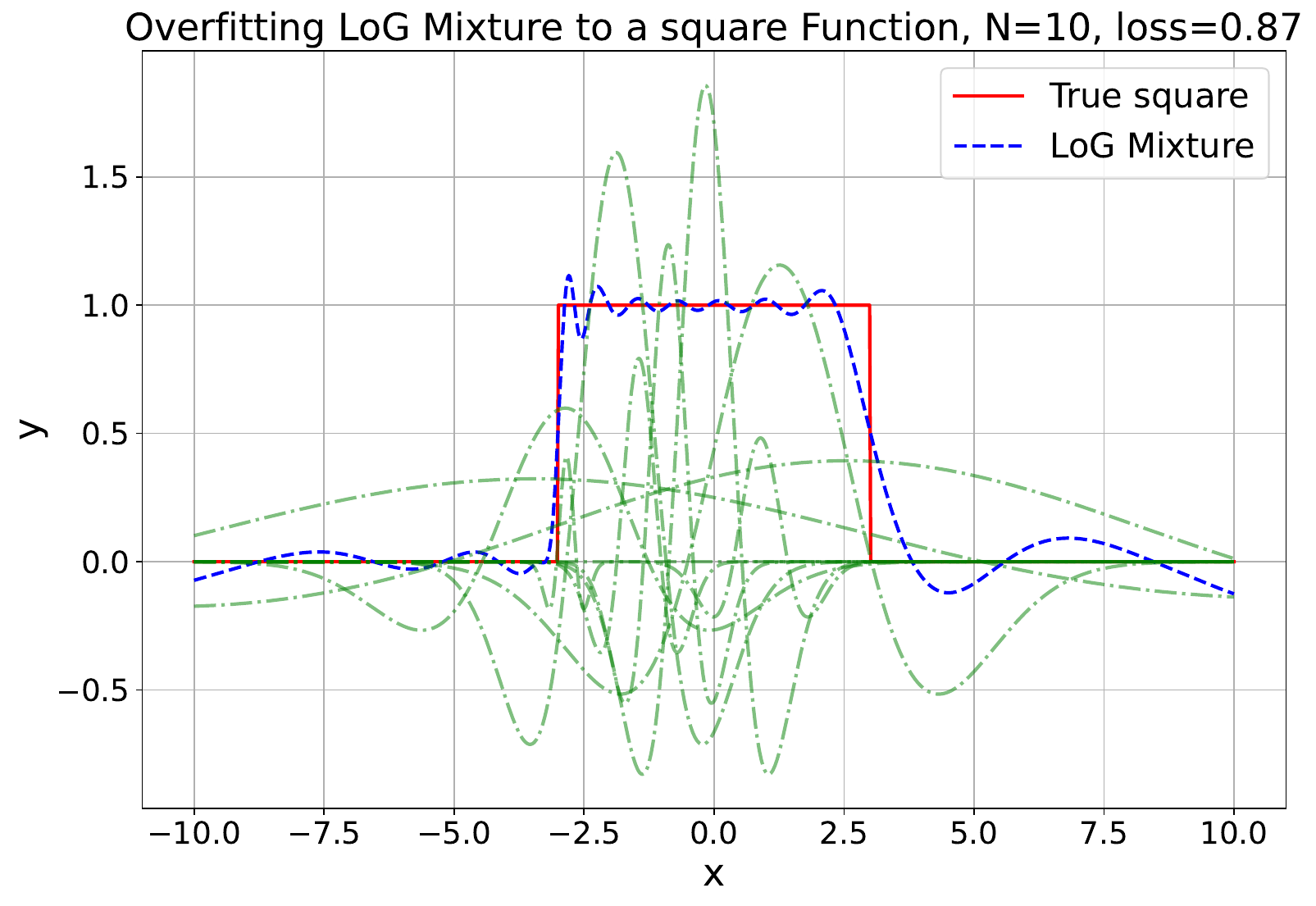} & 
    \includegraphics[width=0.24\linewidth]{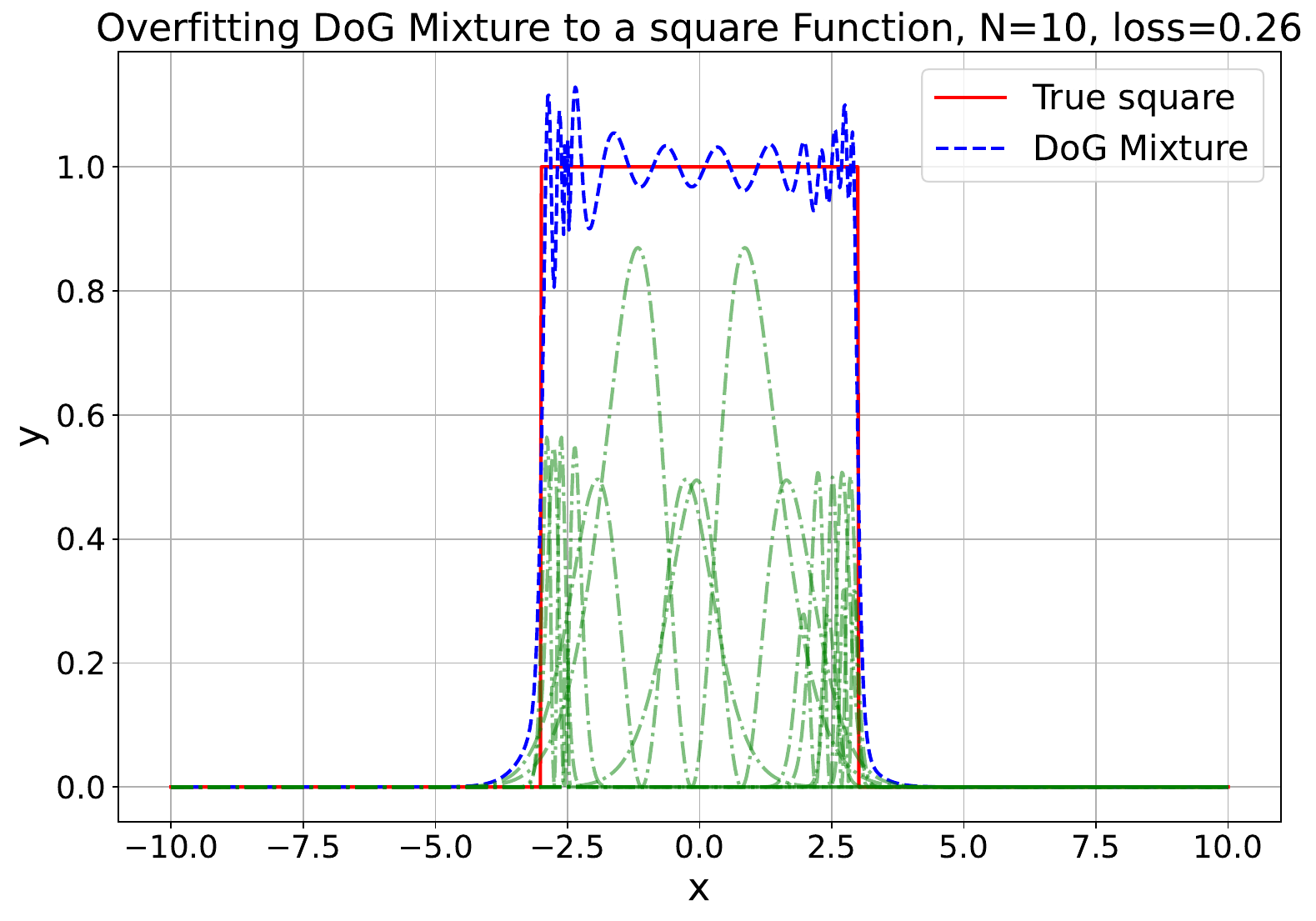} & 
    \includegraphics[width=0.24\linewidth]{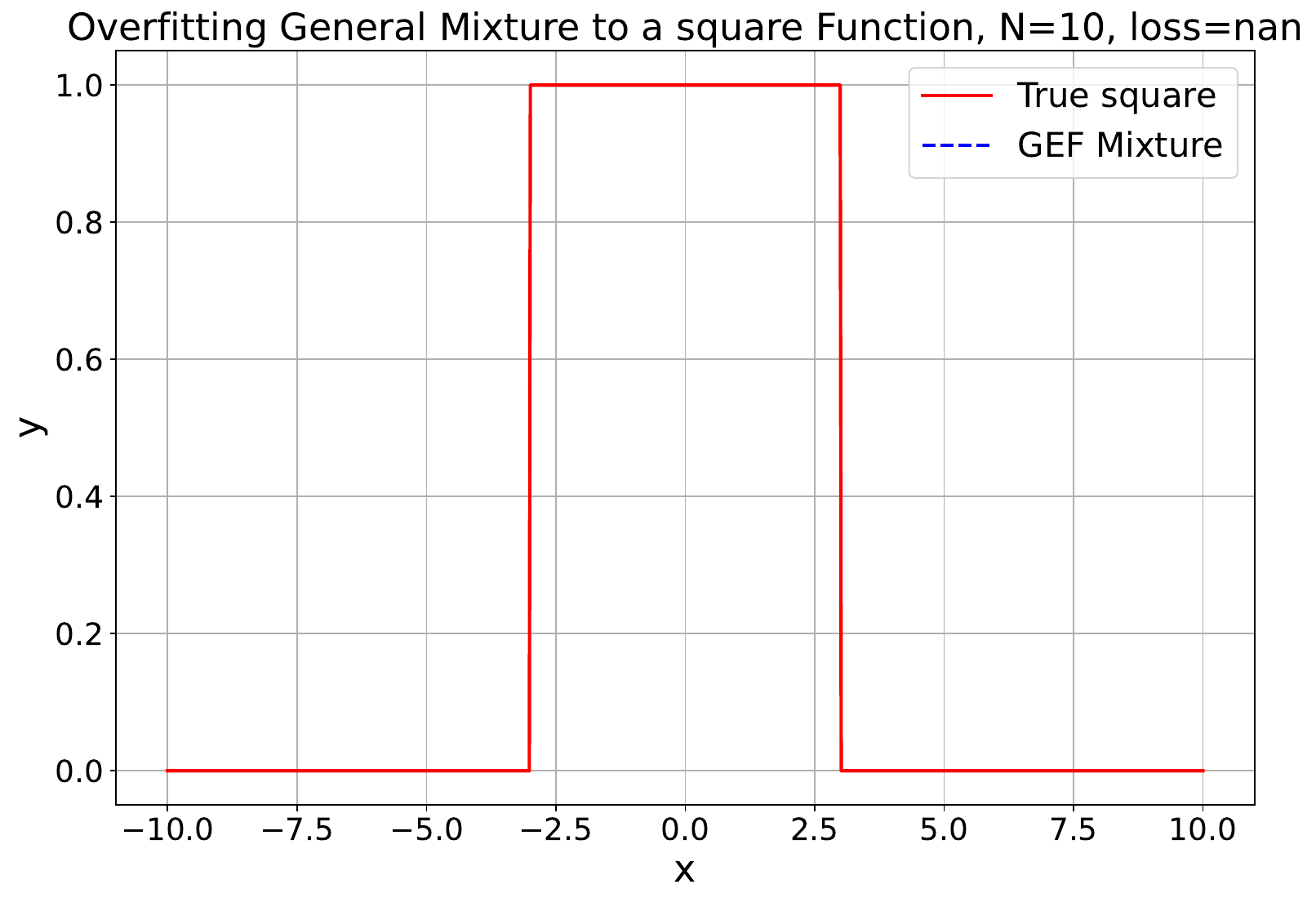}\\ 
    \includegraphics[width=0.24\linewidth]{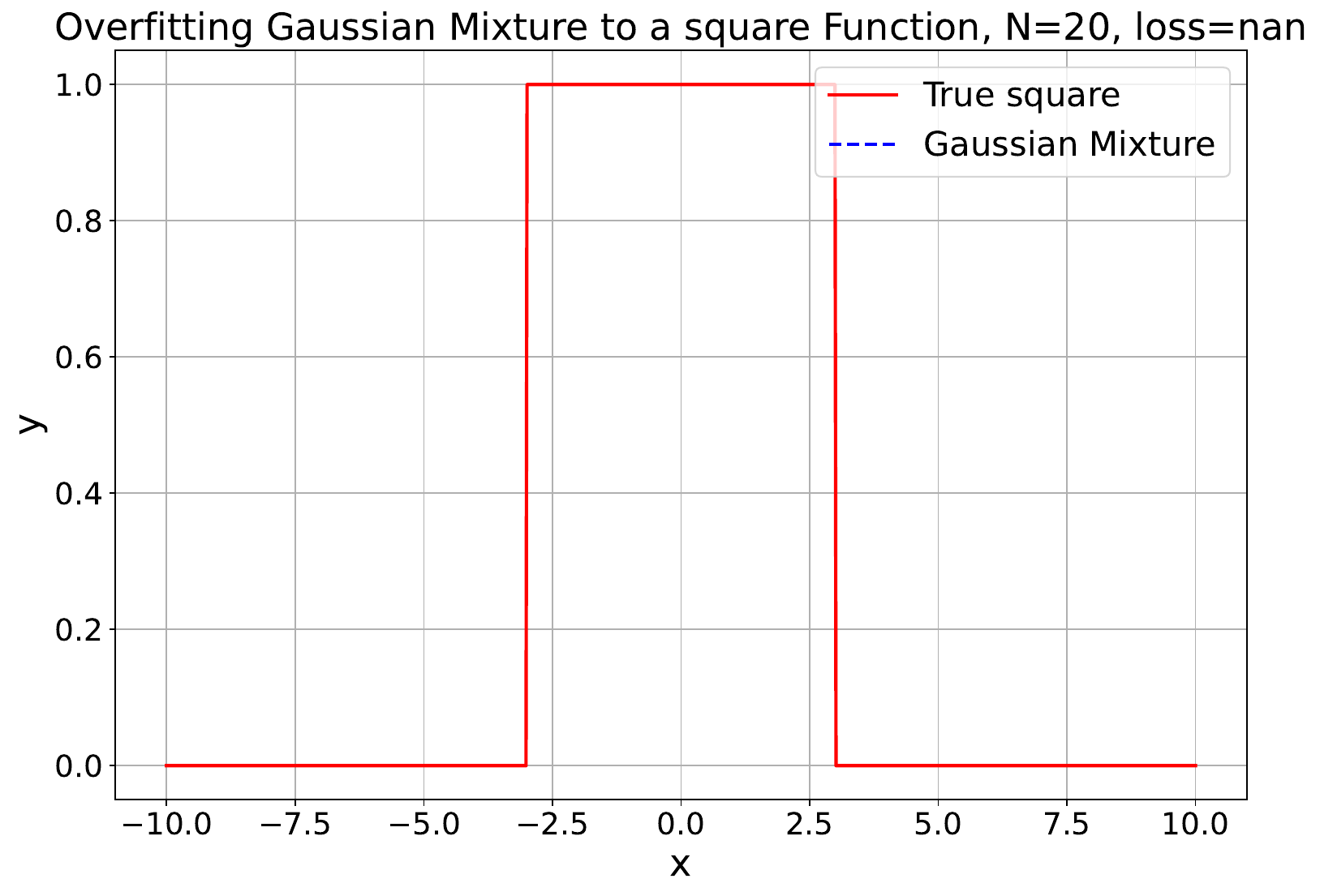} & 
    \includegraphics[width=0.24\linewidth]{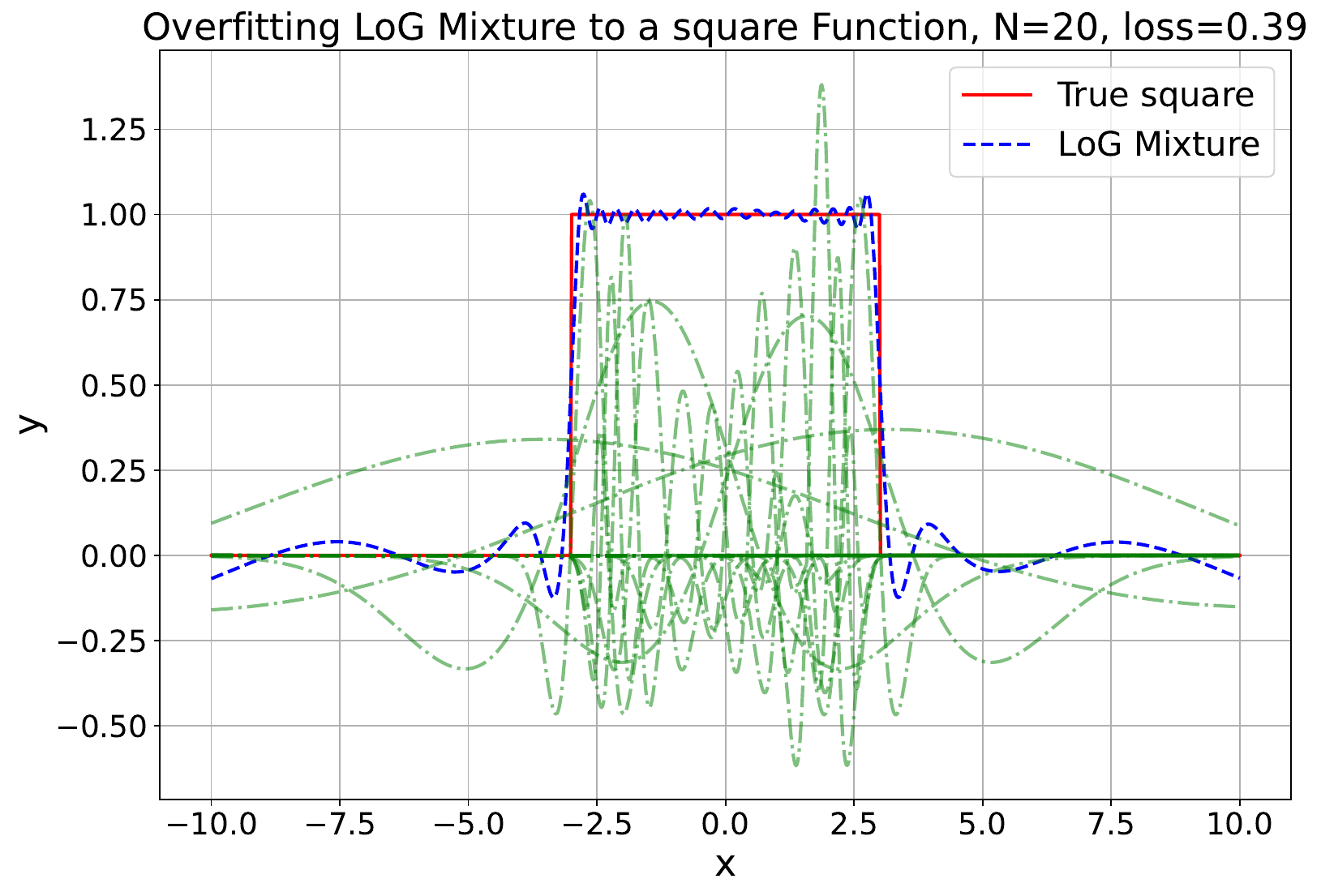} & 
    \includegraphics[width=0.24\linewidth]{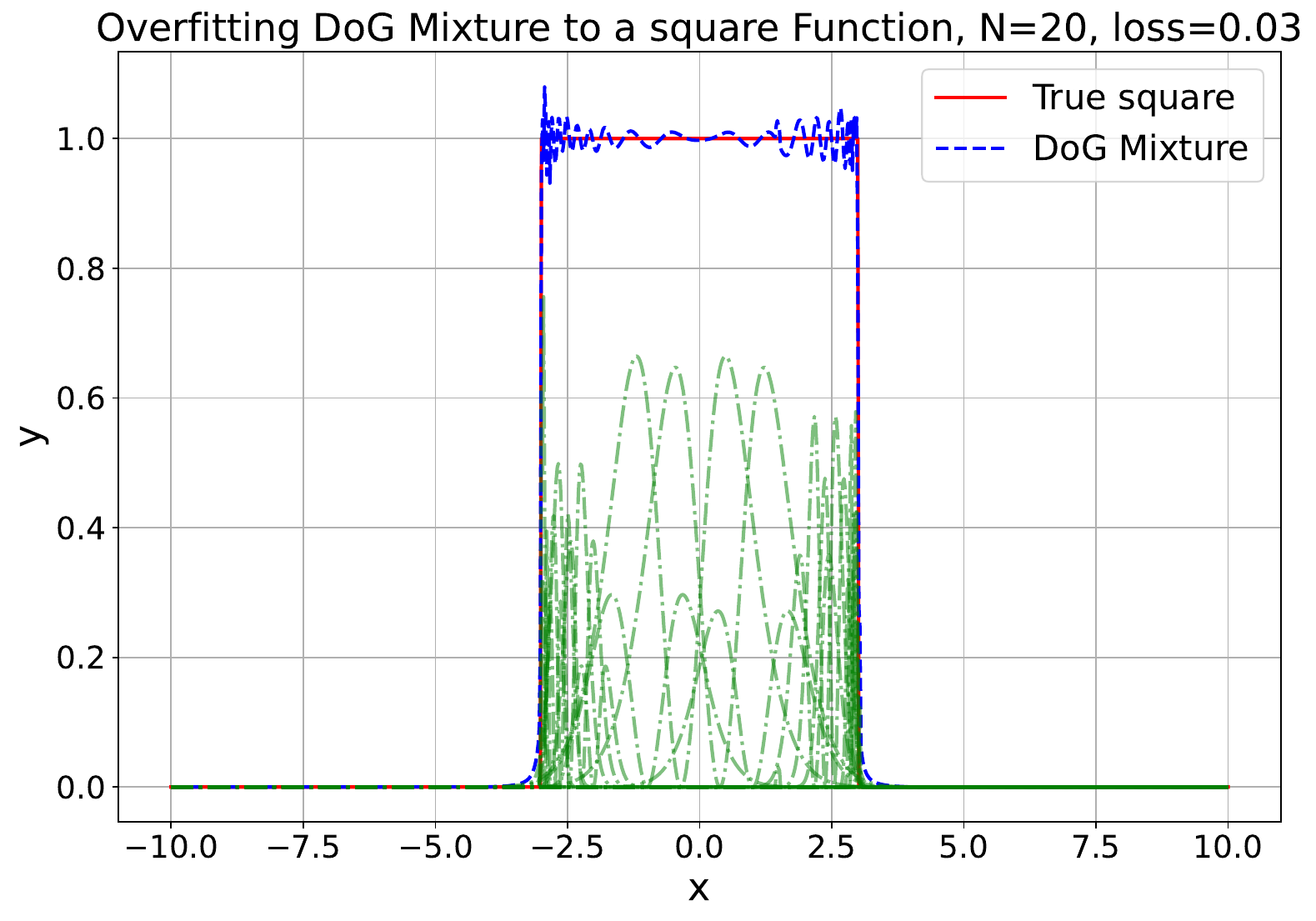} & 
    \includegraphics[width=0.24\linewidth]{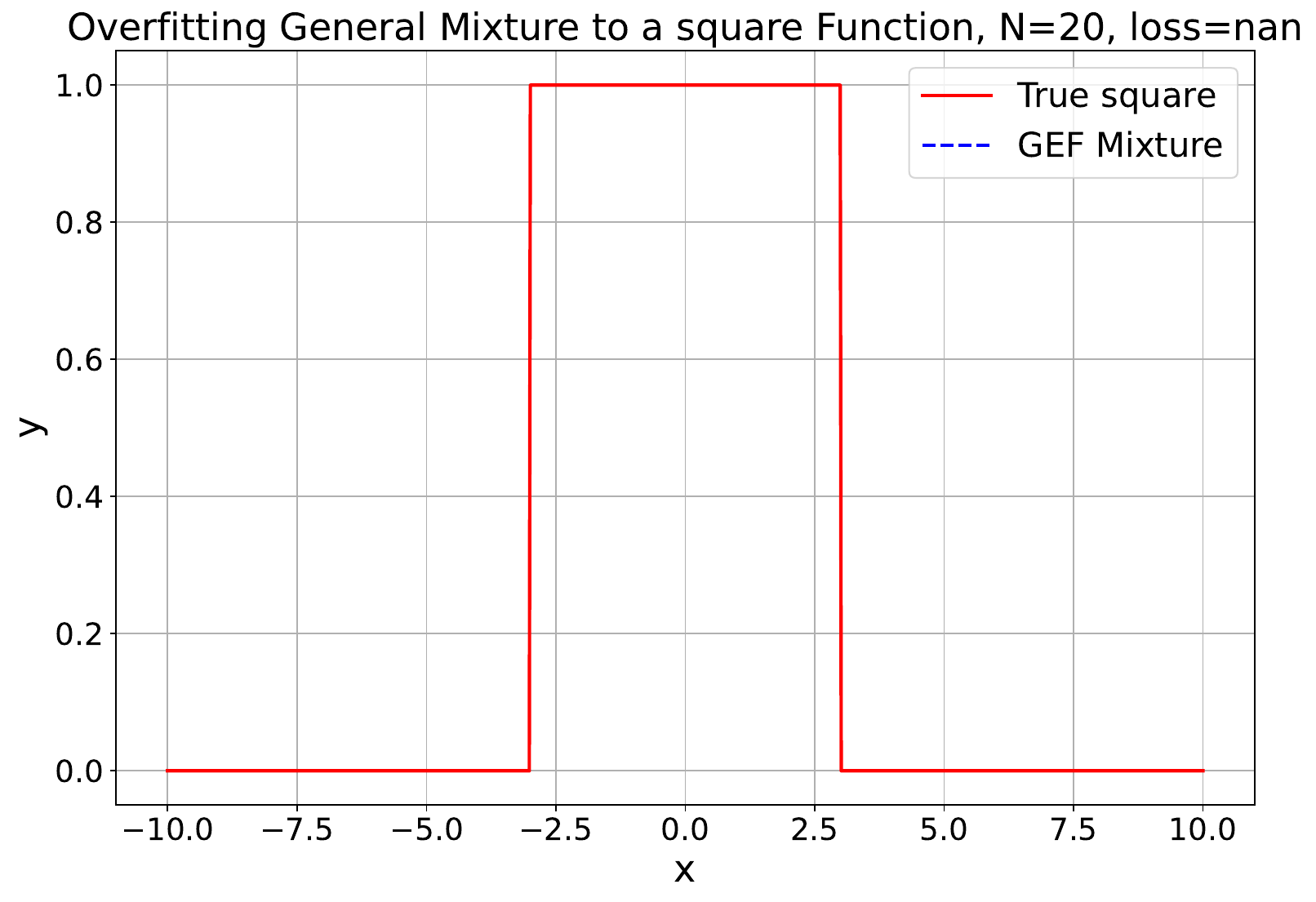}\\ 
    
    \end{tabular}
    }
    \caption{\textbf{Numerical Simulation Examples of Fitting Squares with Positive Weights Mixtures ( N= 2, 5, 8, and 10 )}. We show some fitting examples for Square signals with positive weights mixtures. The four mixtures used from left to right are Gaussians, LoG, DoG, and General mixtures. From top to bottom: N = 2, 8, and 10 components. The optimized individual components are shown in green. Some examples fail to optimize due to numerical instability in both Gaussians and GEF mixtures. Note that GEF is very efficient in fitting the Square with few components while LoG and DoG are more stable for a larger number of components. }
    \label{supfig:fitting_square_p}
    \end{figure*}
    

%% file: figures/fitting/fitting_square_n.tex
\begin{figure*}[h]
    \centering
    \resizebox{1.0\linewidth}{!}{
    \begin{tabular}{cccc}
    \tabcolsep=0.01cm
    Gaussian Mixture& LoG Mixture & DoG Mixture & GEF Mixture \\ 
    \includegraphics[width=0.24\linewidth]{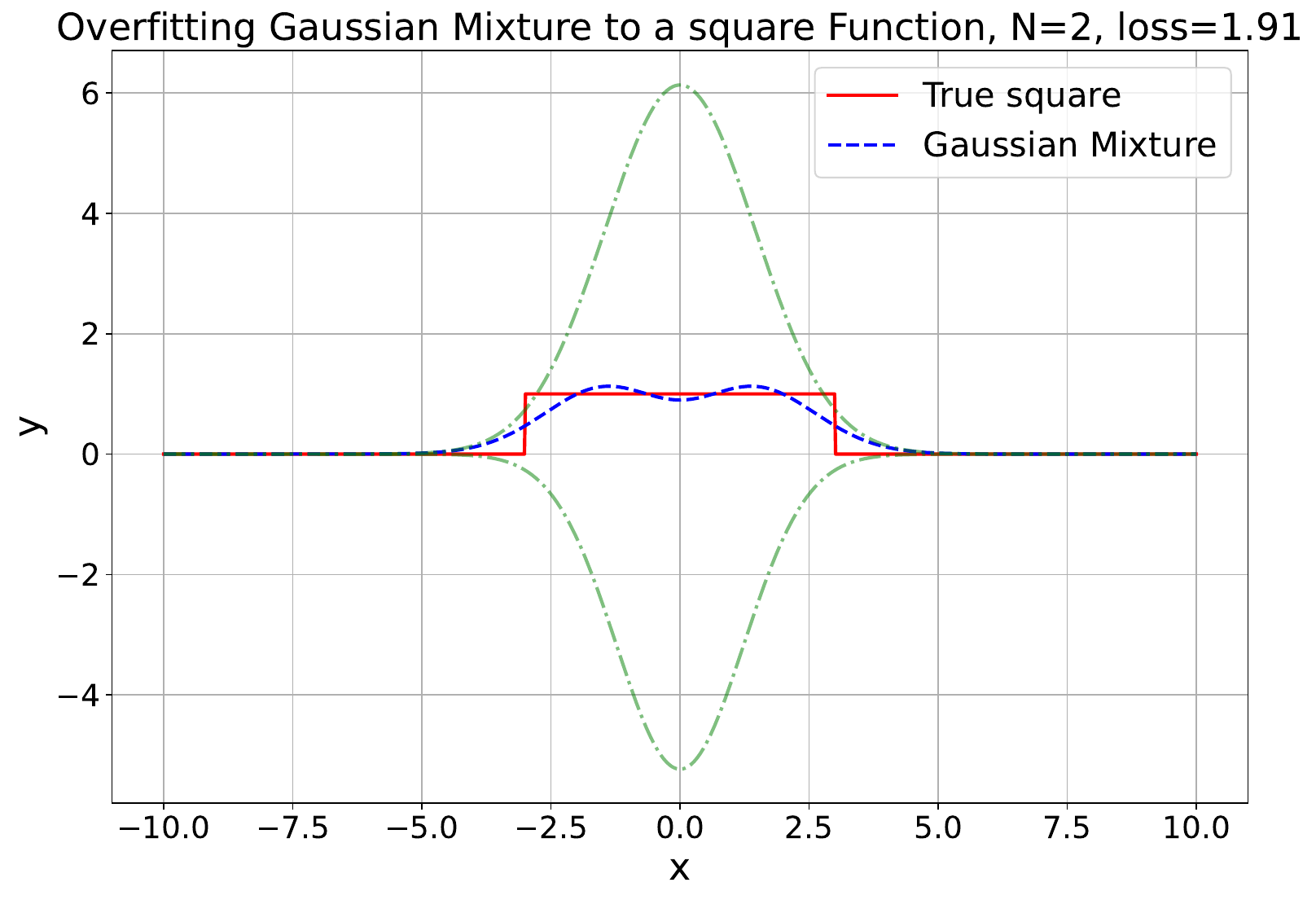} & 
    \includegraphics[width=0.24\linewidth]{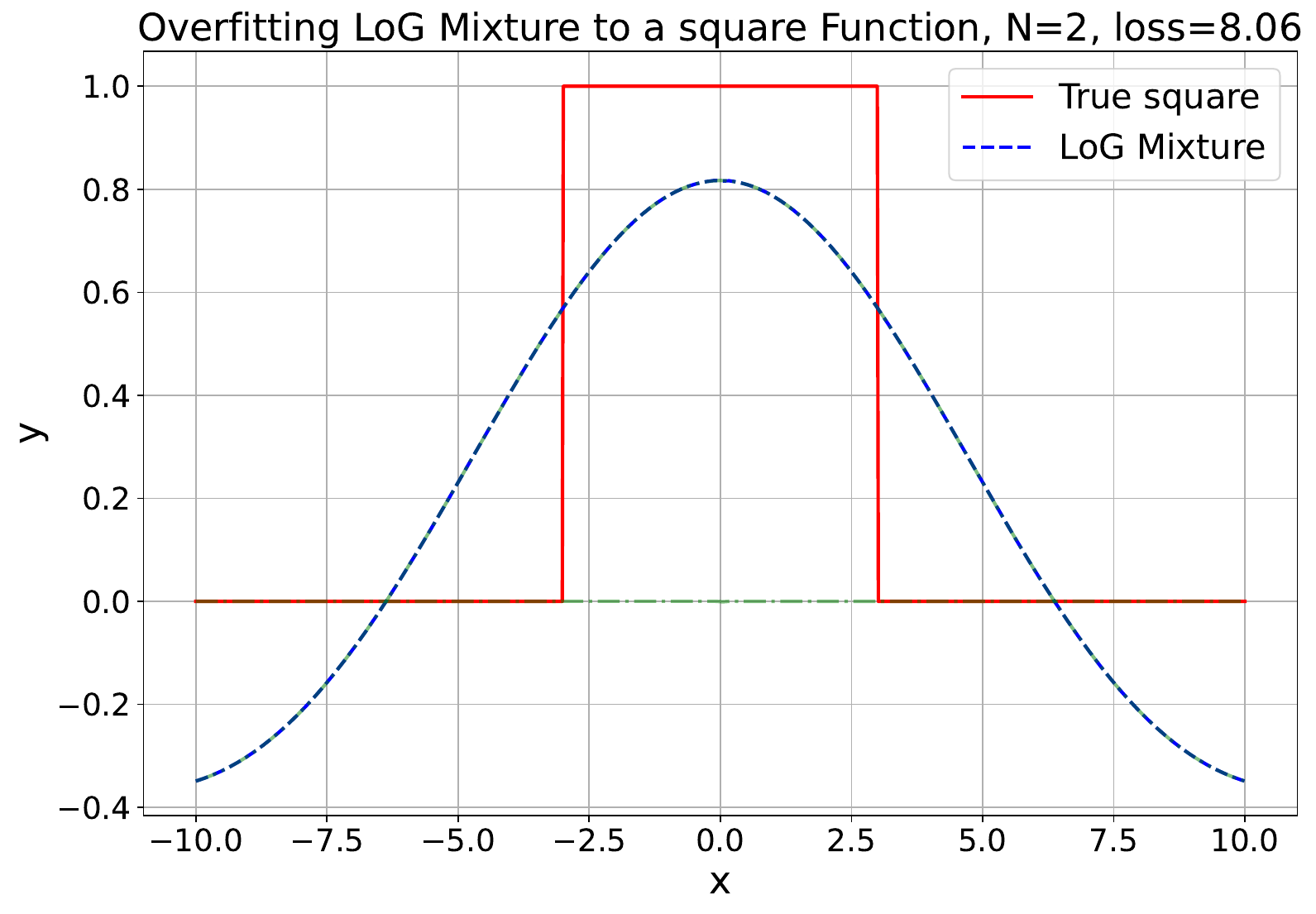} & 
    \includegraphics[width=0.24\linewidth]{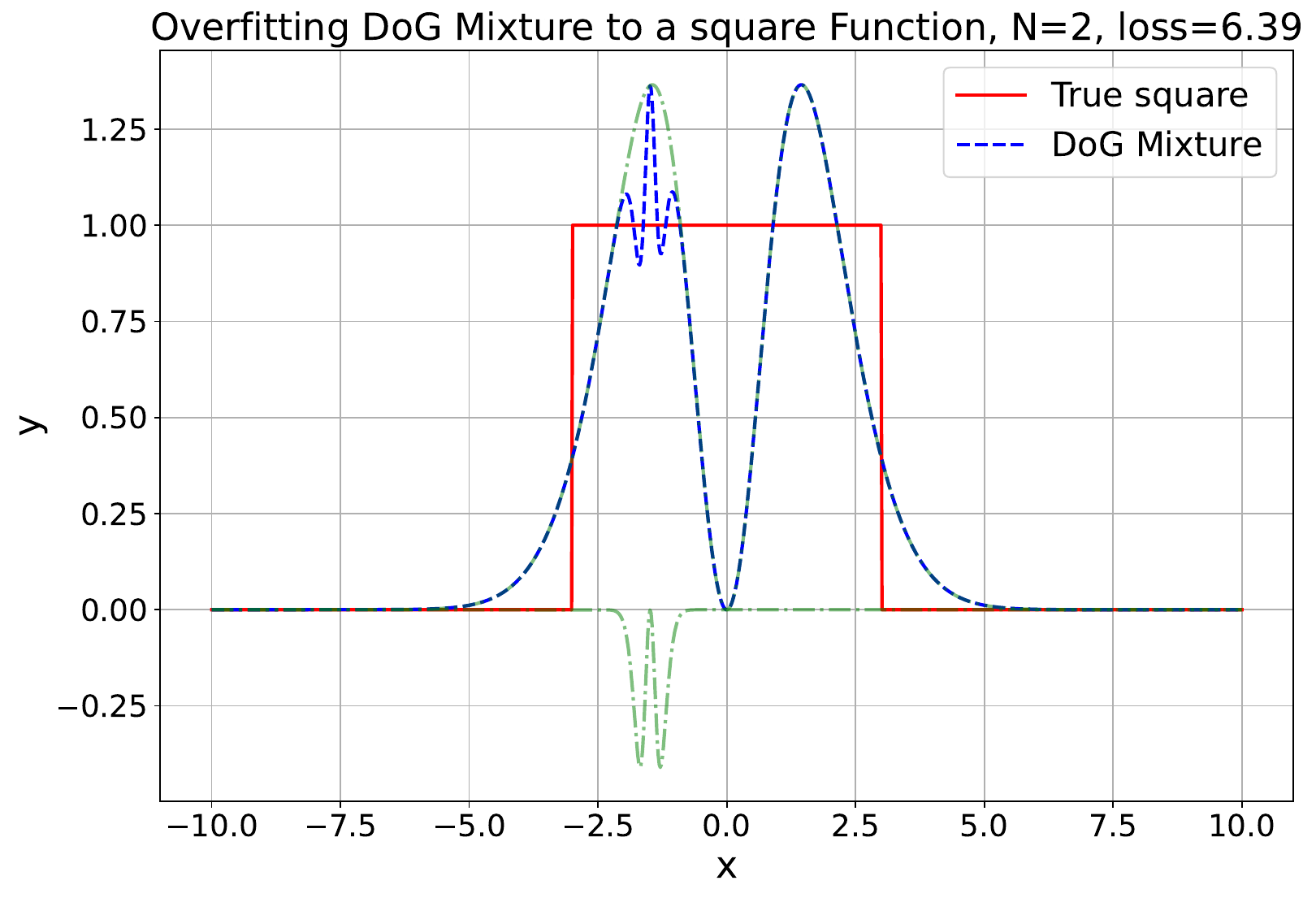} & 
    \includegraphics[width=0.24\linewidth]{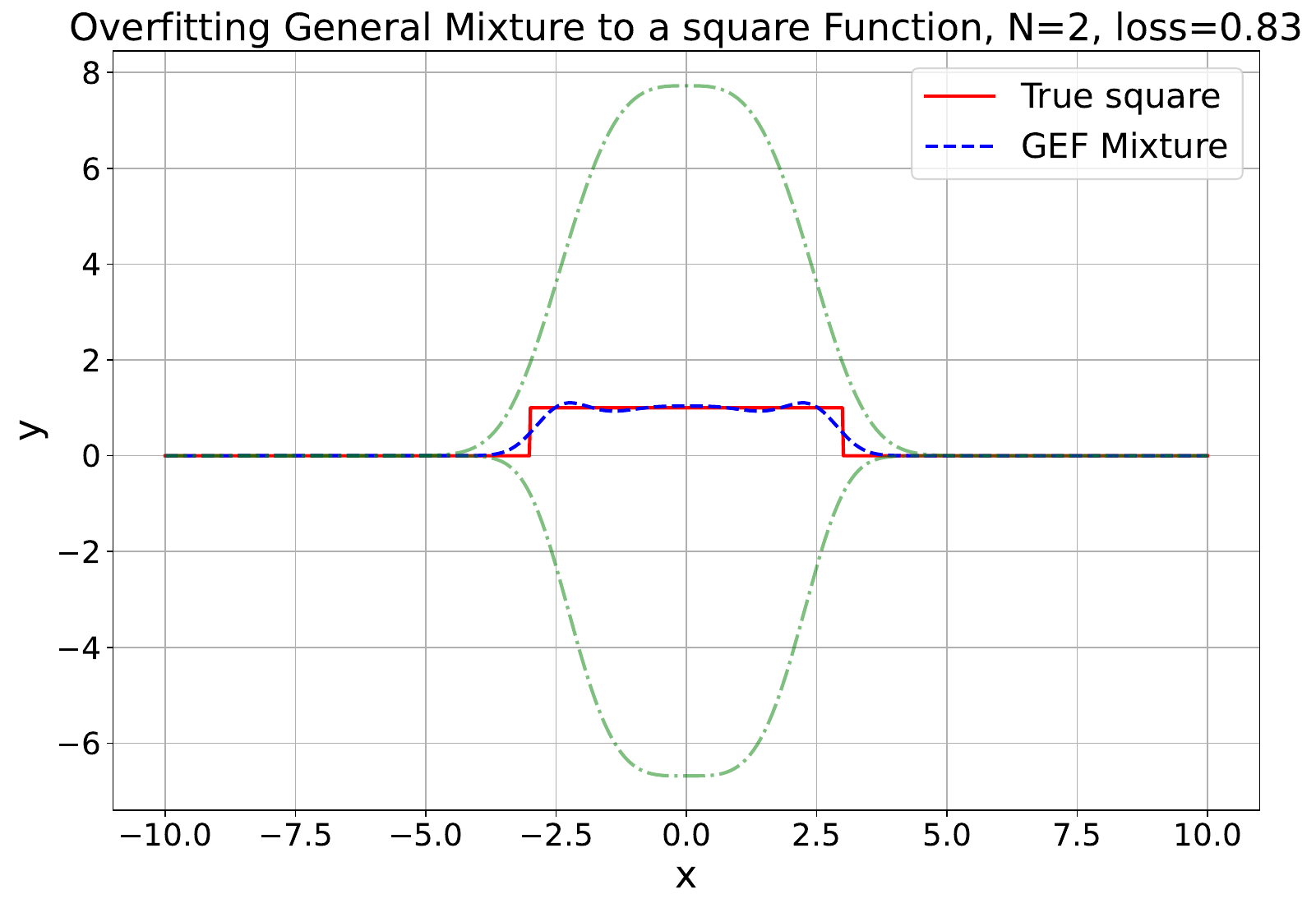}\\ 
    \includegraphics[width=0.24\linewidth]{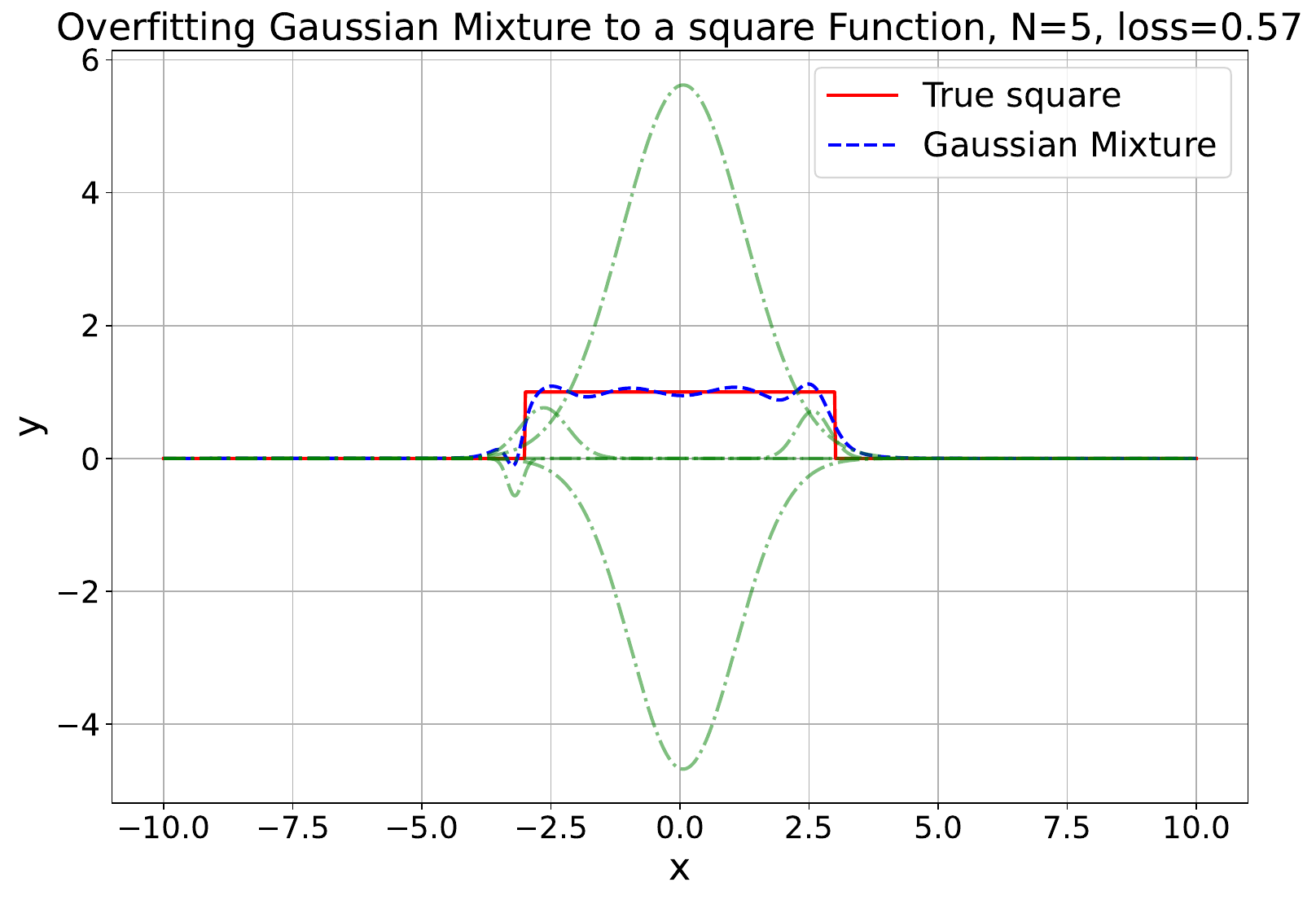} & 
    \includegraphics[width=0.24\linewidth]{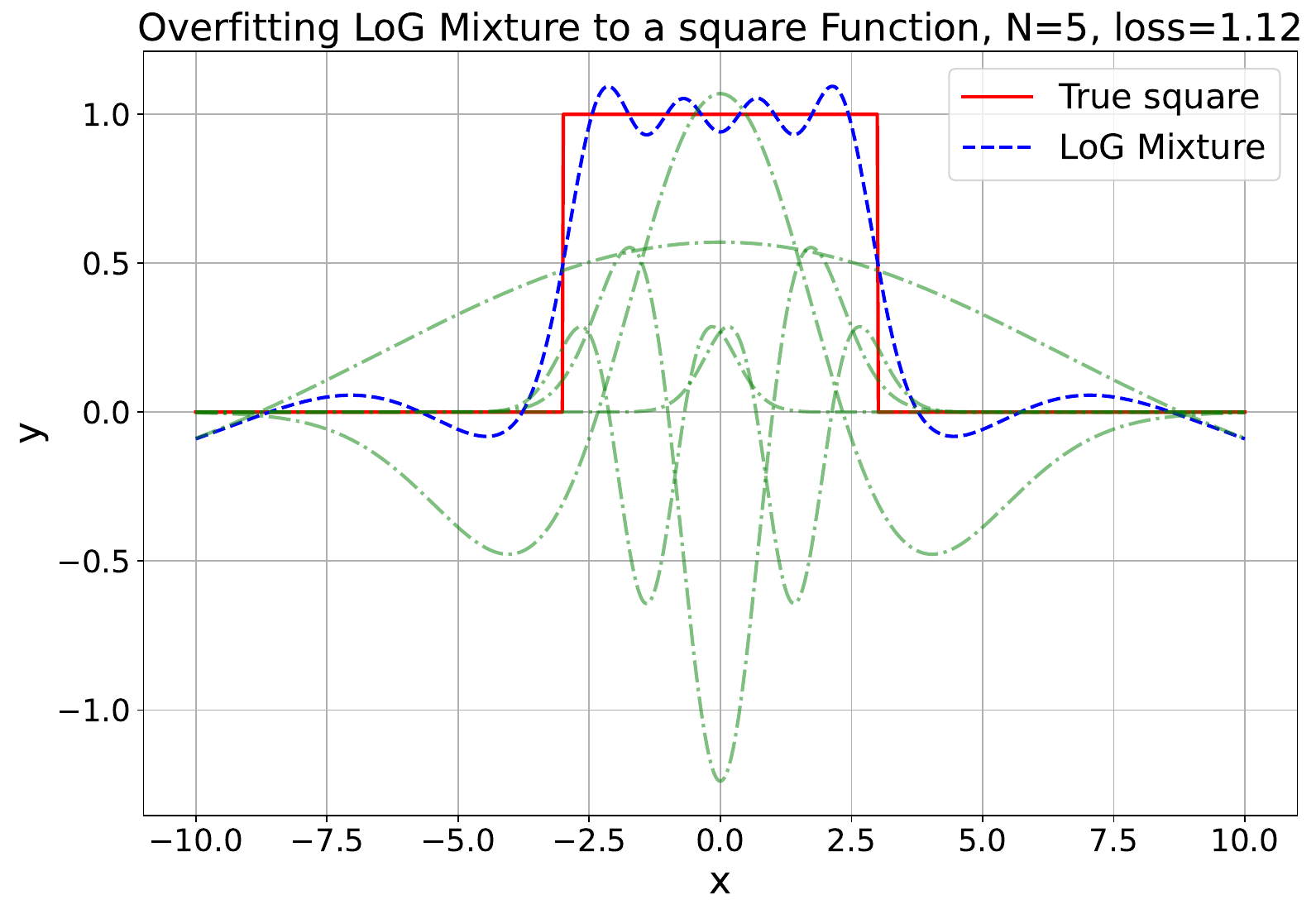} & 
    \includegraphics[width=0.24\linewidth]{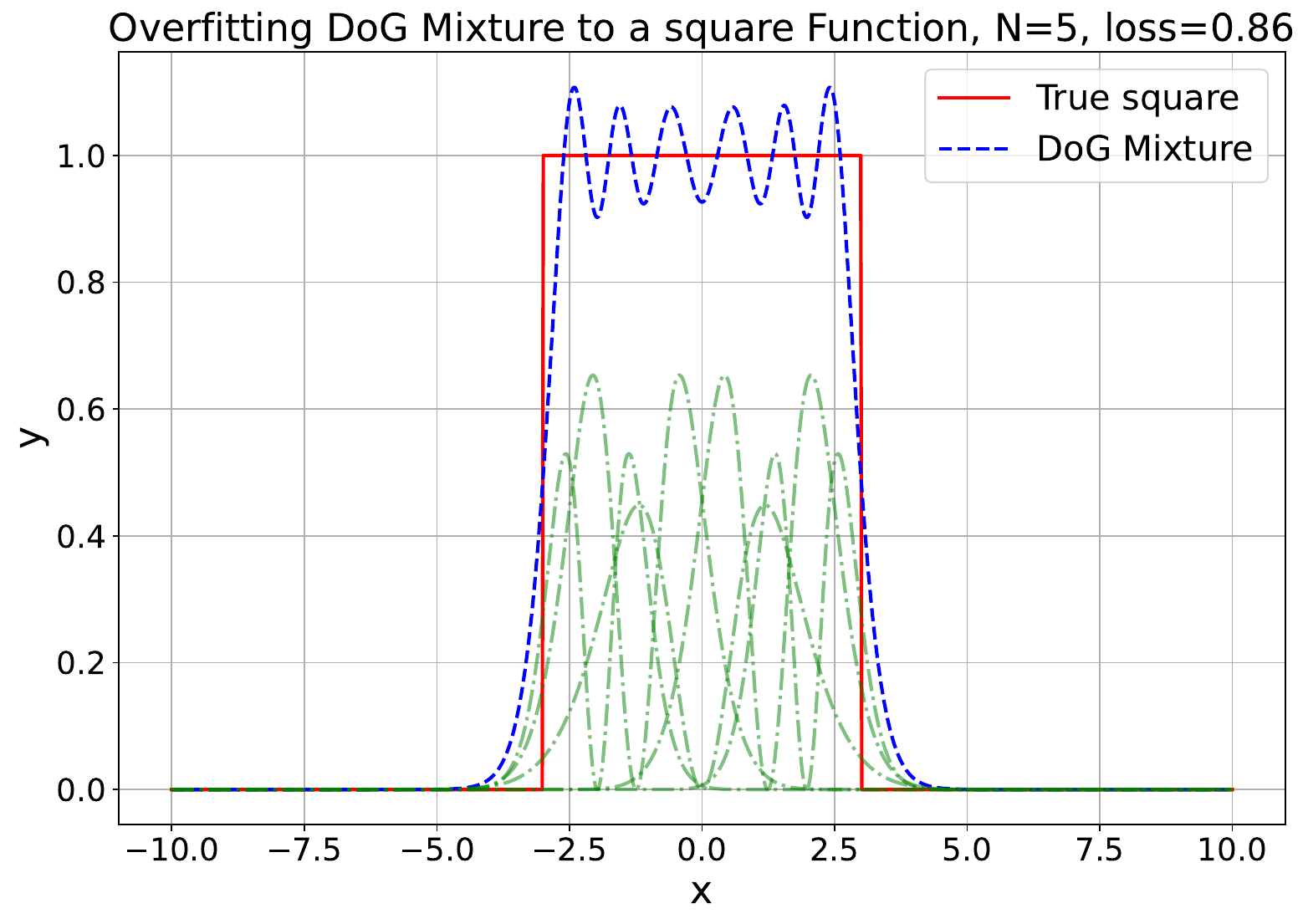} & 
    \includegraphics[width=0.24\linewidth]{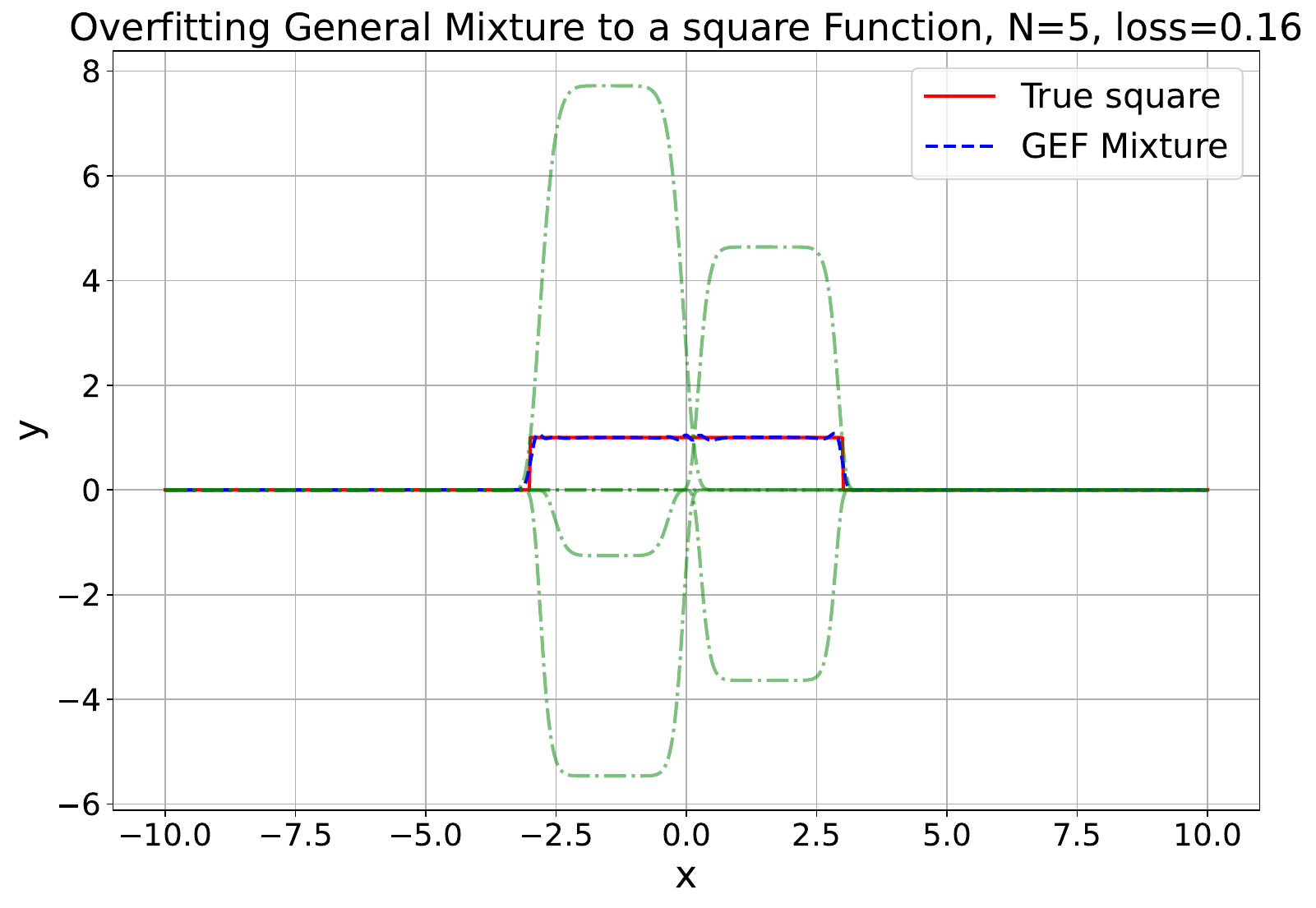}\\ 
    \includegraphics[width=0.24\linewidth]{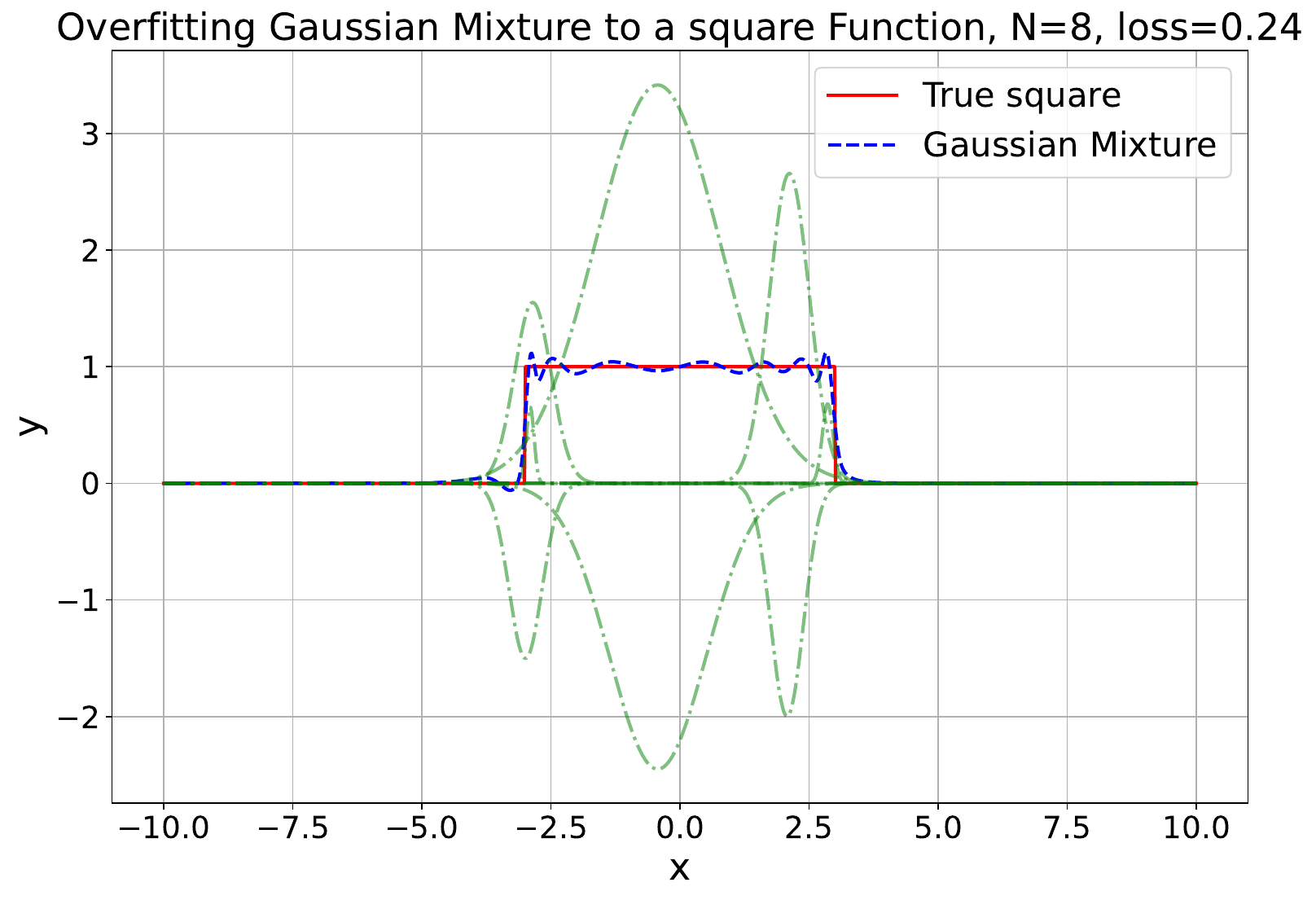} & 
    \includegraphics[width=0.24\linewidth]{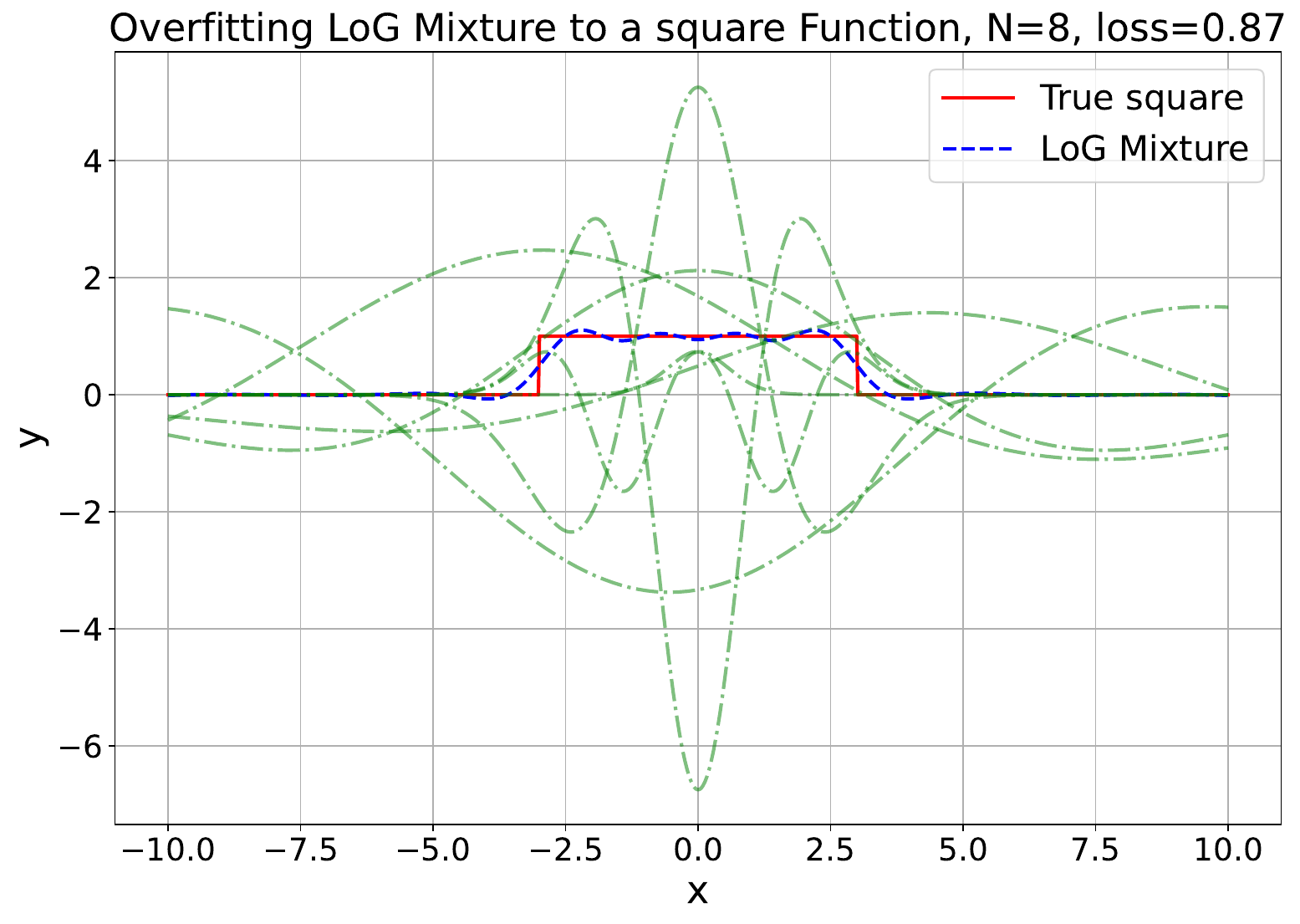} & 
    \includegraphics[width=0.24\linewidth]{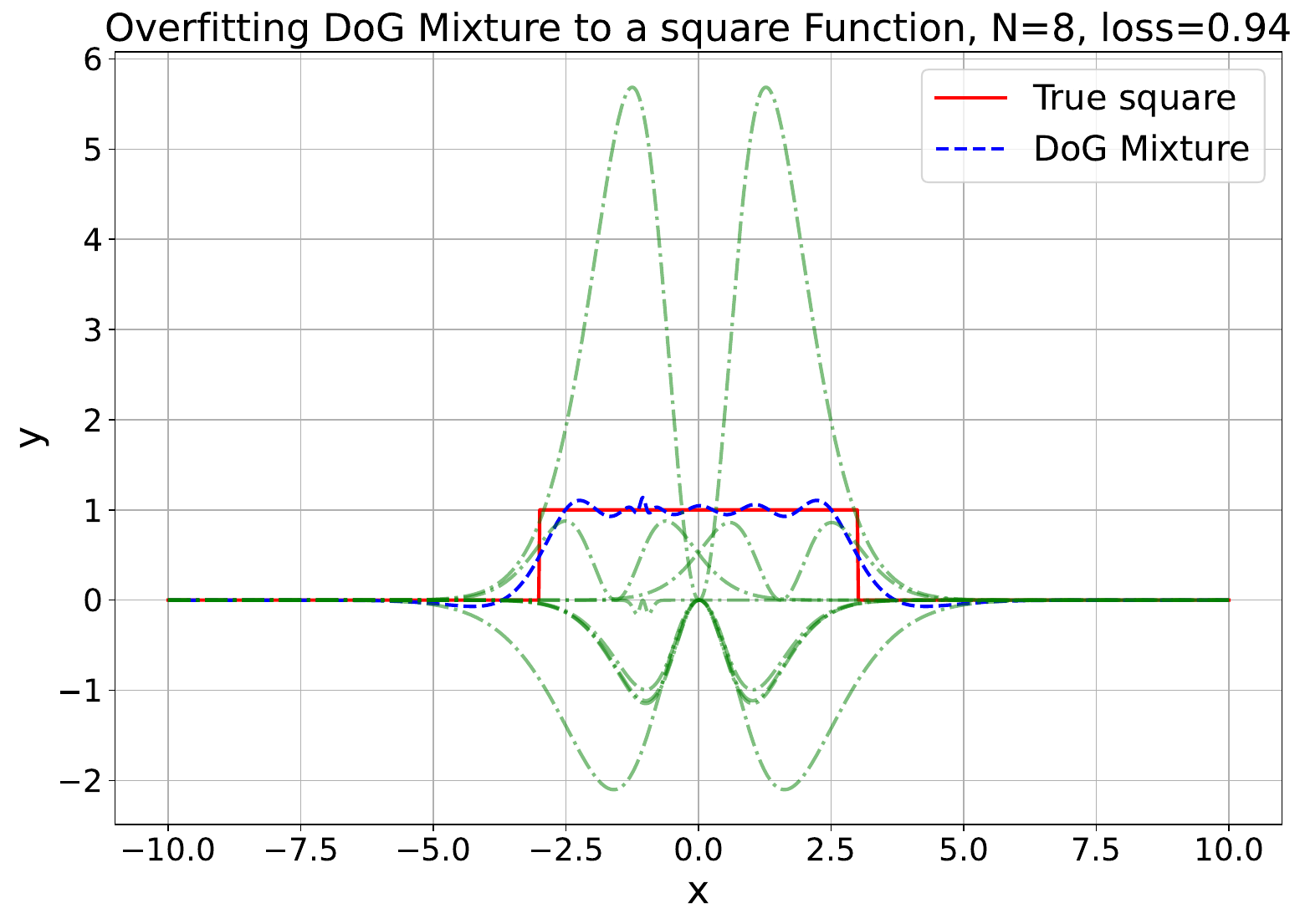} & 
    \includegraphics[width=0.24\linewidth]{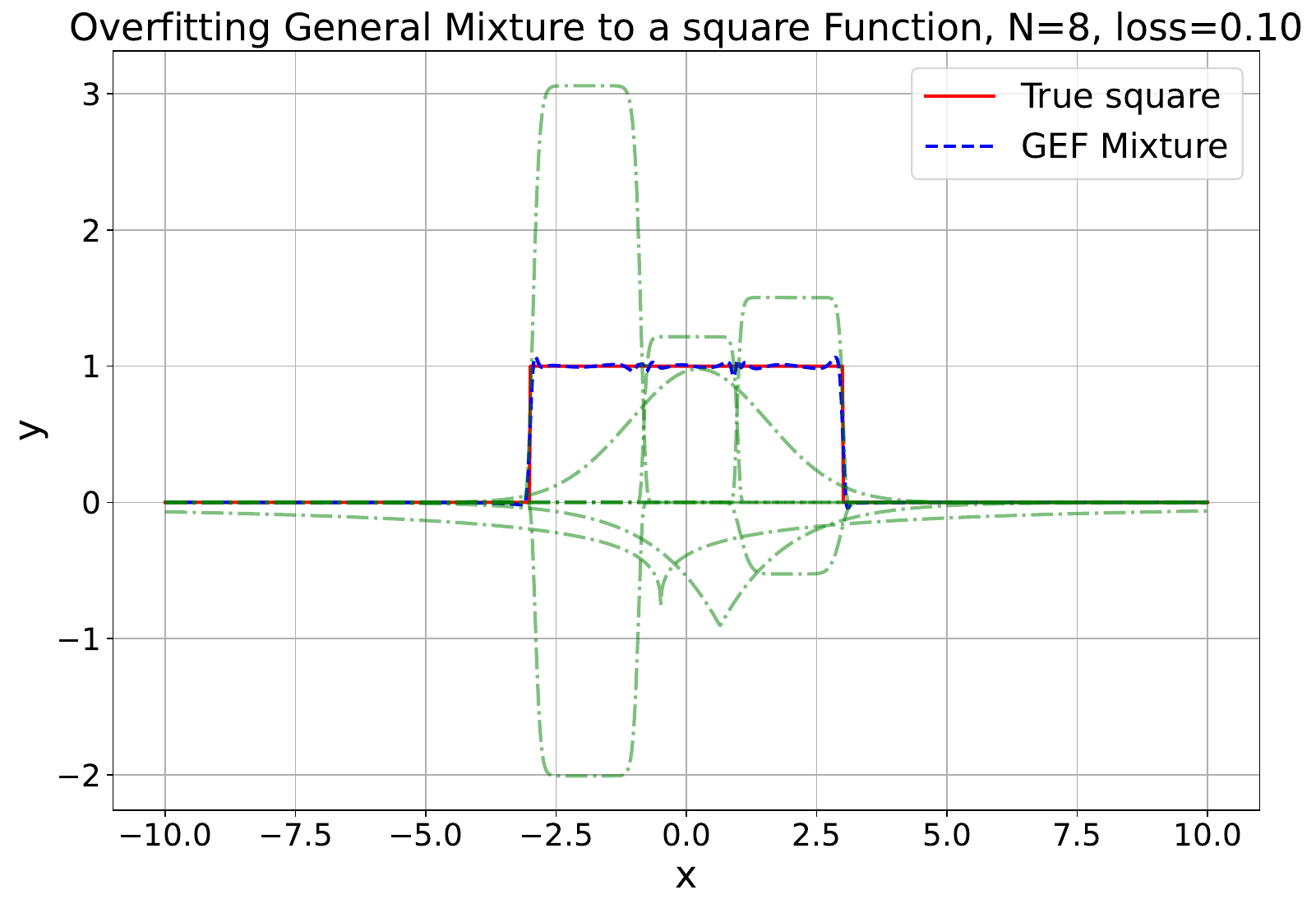}\\ 
    \includegraphics[width=0.24\linewidth]{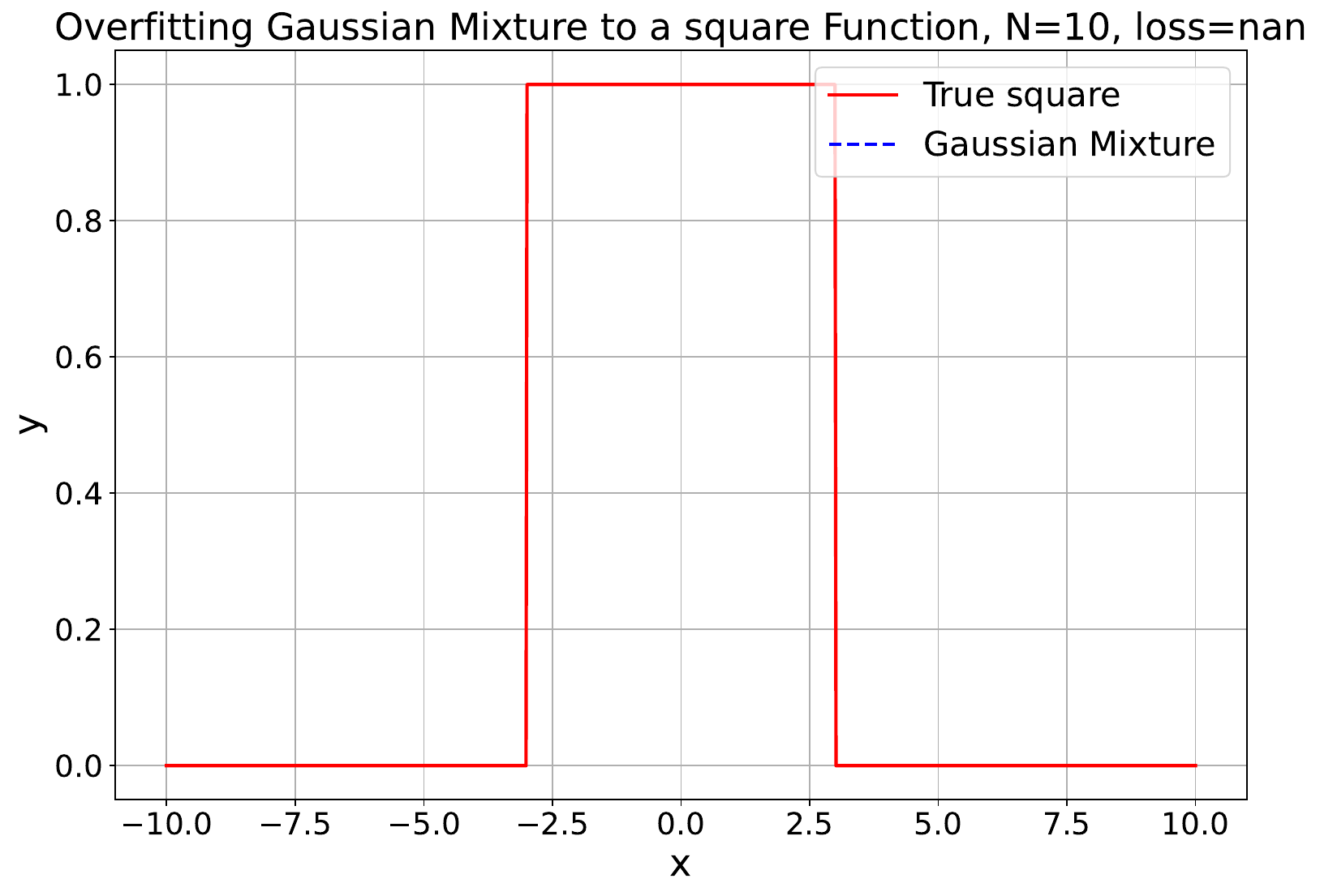} & 
    \includegraphics[width=0.24\linewidth]{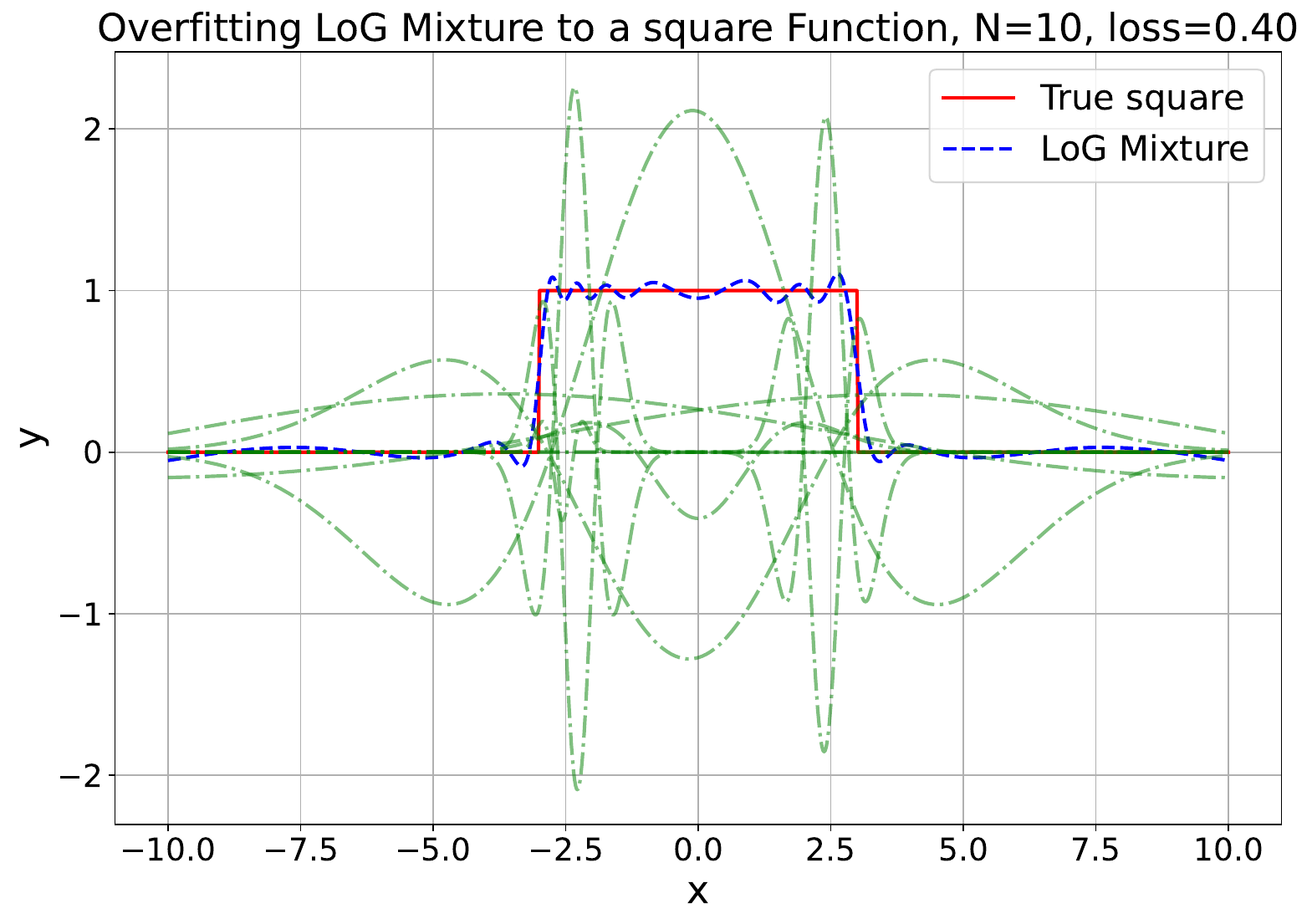} & 
    \includegraphics[width=0.24\linewidth]{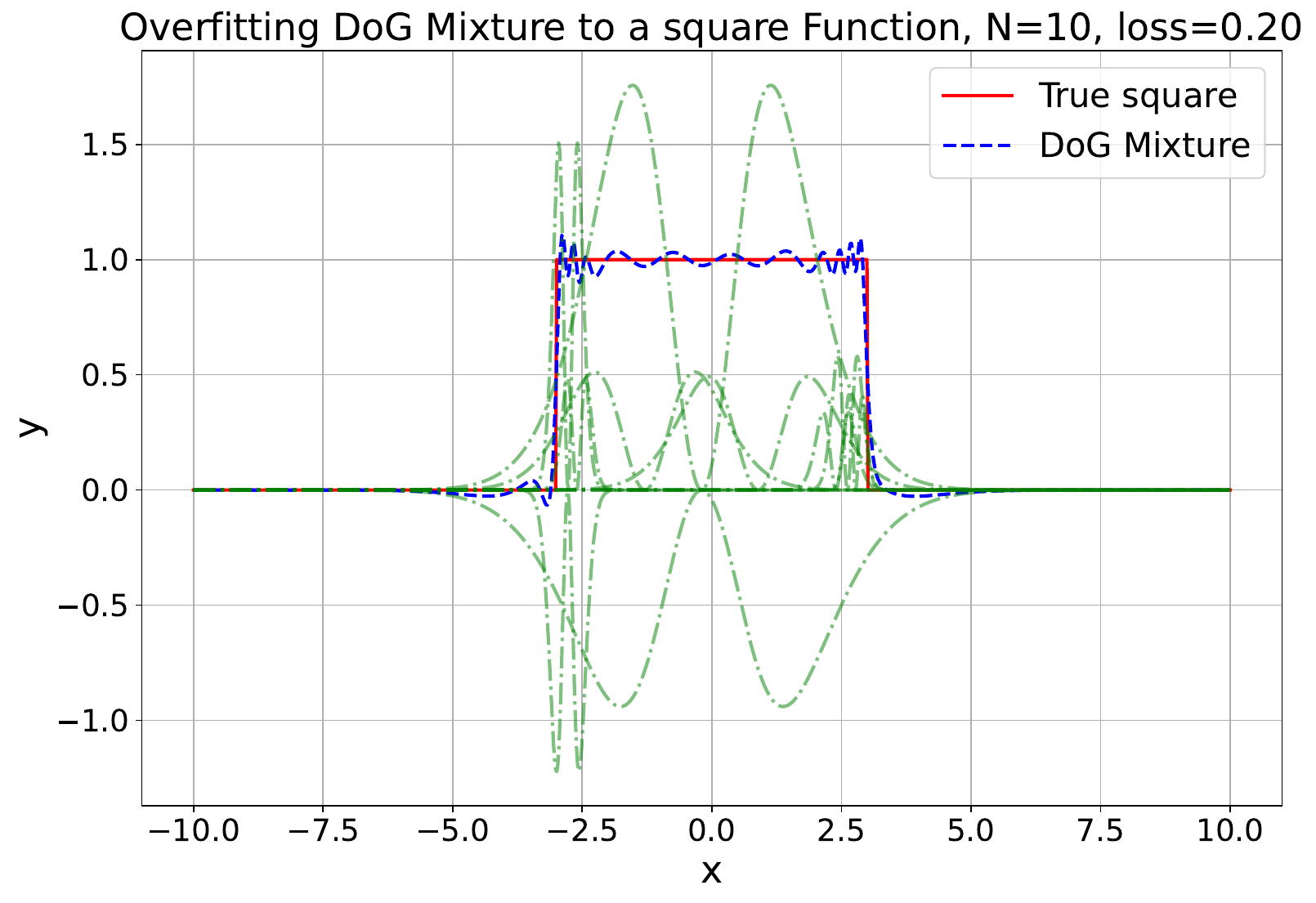} & 
    \includegraphics[width=0.24\linewidth]{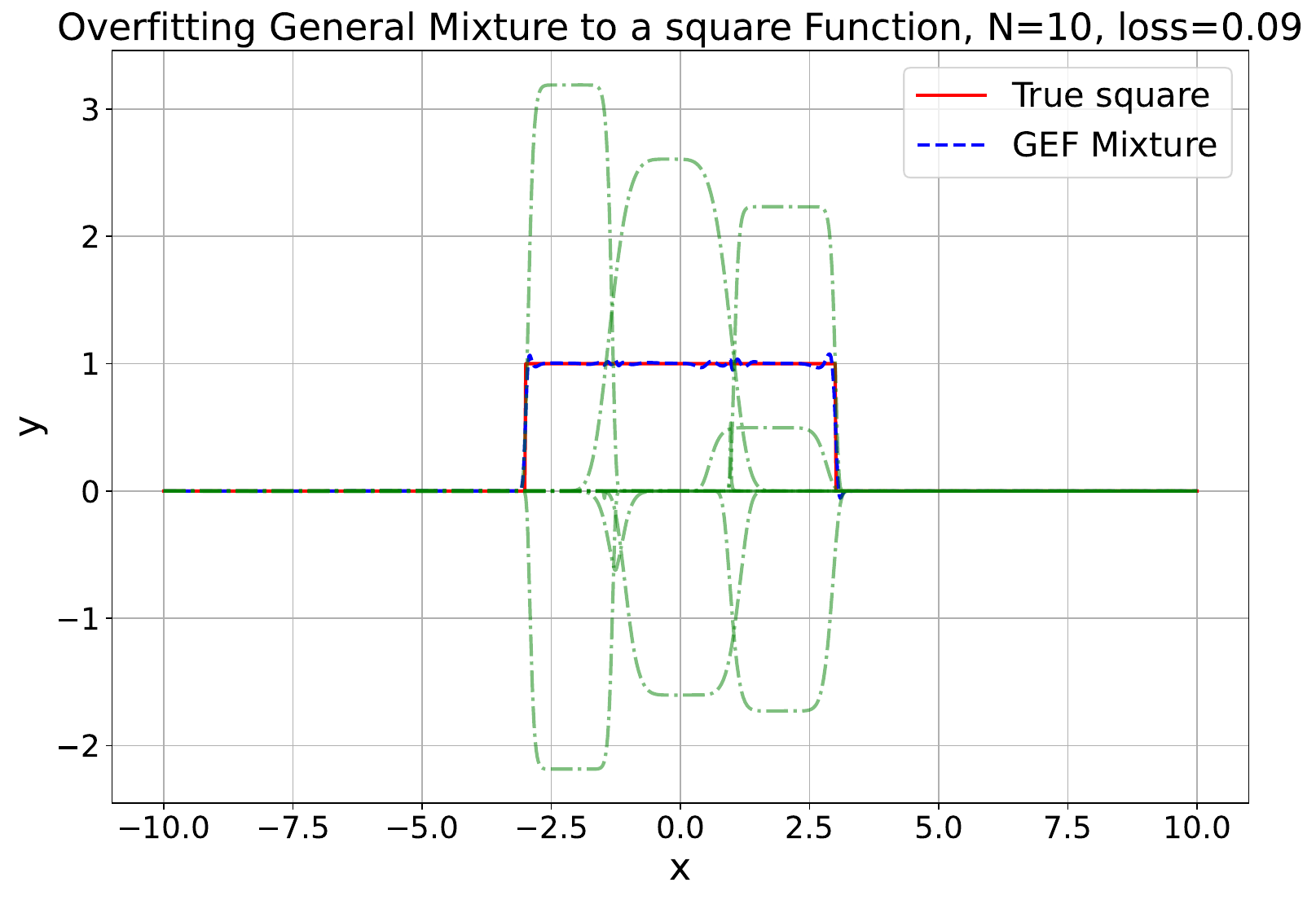}\\ 
    \includegraphics[width=0.24\linewidth]{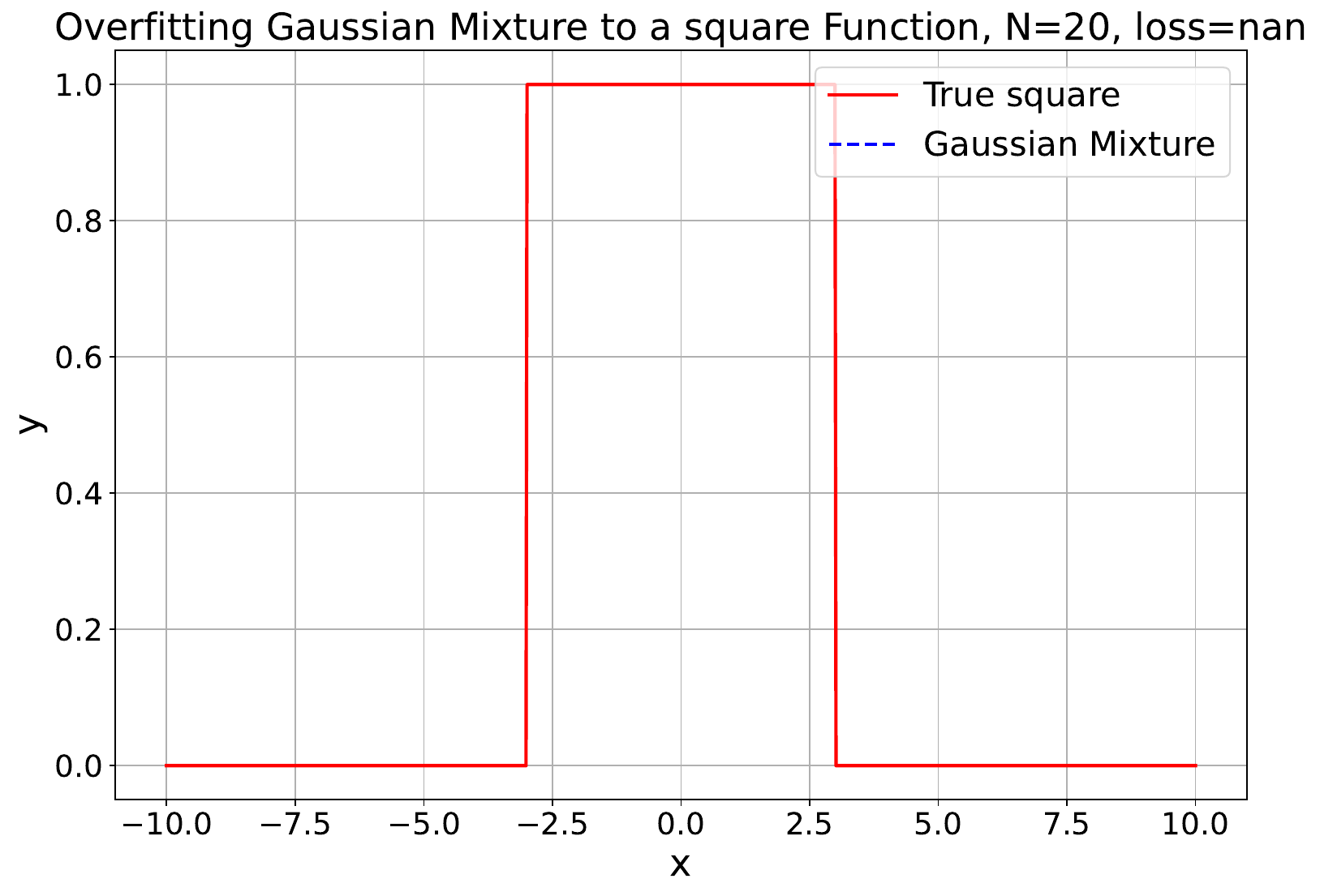} & 
    \includegraphics[width=0.24\linewidth]{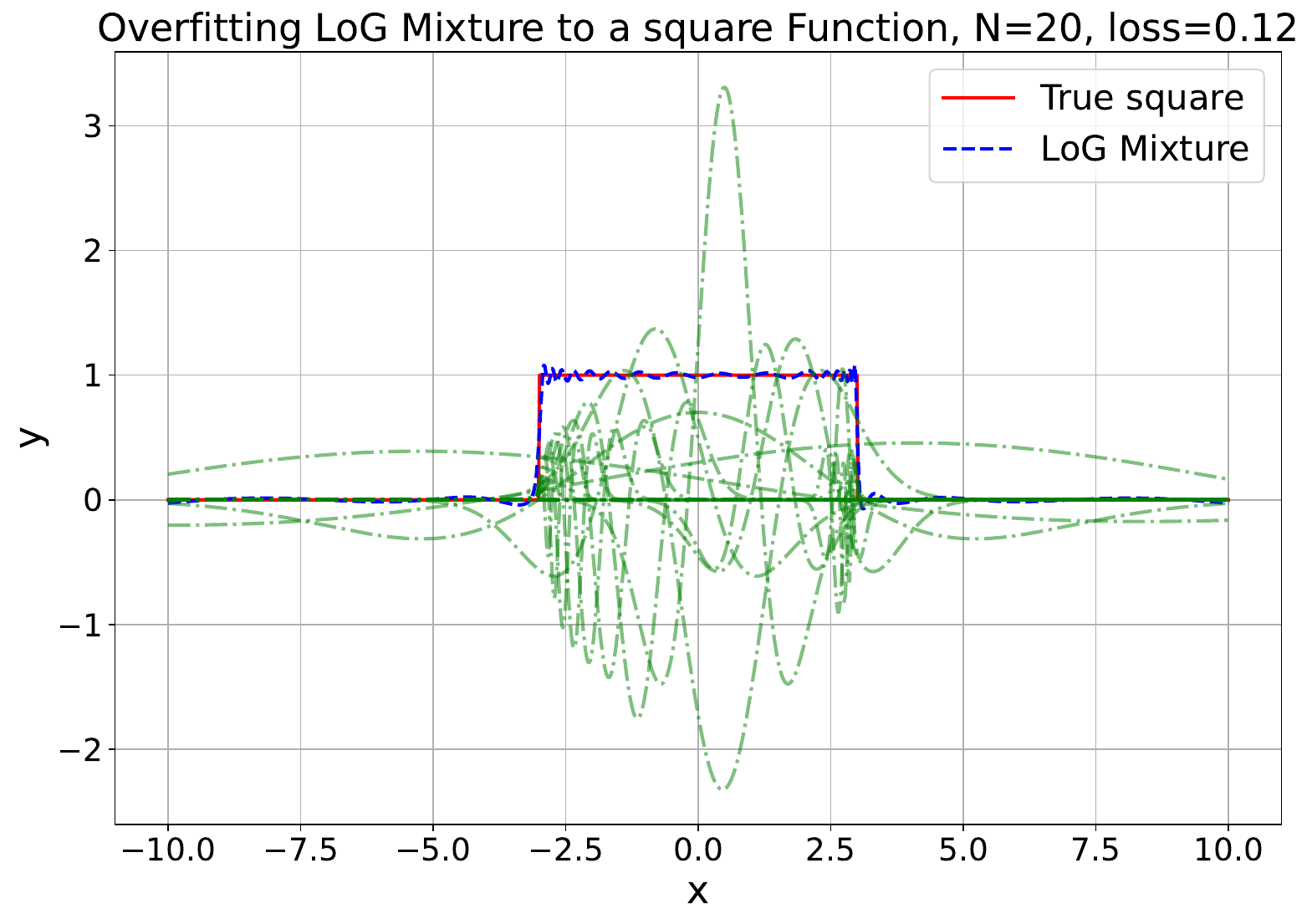} & 
    \includegraphics[width=0.24\linewidth]{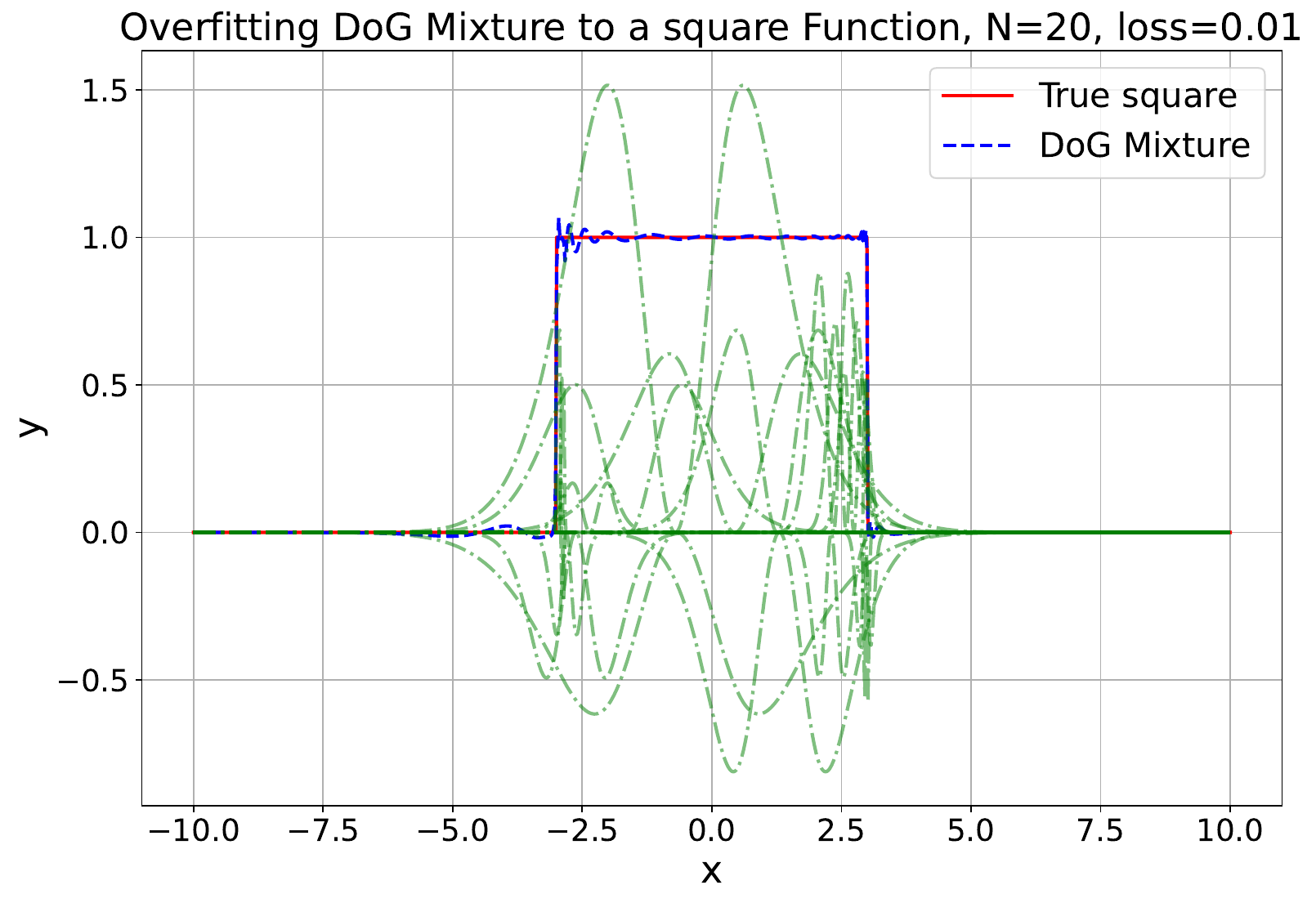} & 
    \includegraphics[width=0.24\linewidth]{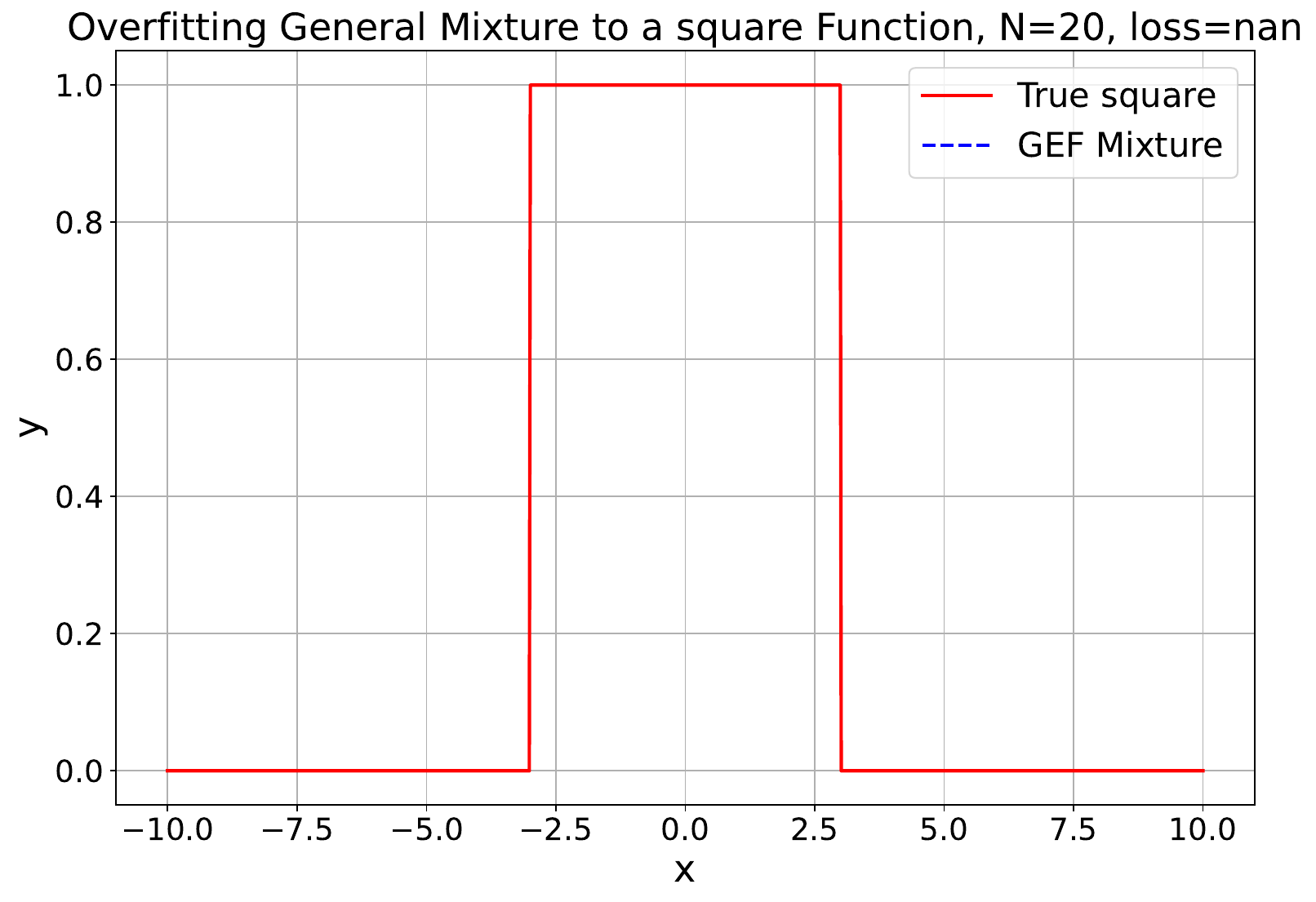}\\ 
    
    \end{tabular}
    }
    \caption{\textbf{Numerical Simulation Examples of Fitting Squares with Real Weights Mixtures ( N= 2, 5, 8, and 10 )}. We show some fitting examples for Square signals with Real weights mixtures (can be negative). The four mixtures used from left to right are Gaussians, LoG, DoG, and General mixtures. From top to bottom: N = 2, 8, and 10 components. The optimized individual components are shown in green. Some examples fail to optimize due to numerical instability in both Gaussians and GEF mixtures. Note that GEF is very efficient in fitting the Square with few components while LoG and DoG are more stable for a larger number of components. }
    \label{supfig:fitting_square_N}
    \end{figure*}
    

%% file: figures/fitting/fitting_parabola_p.tex
\begin{figure*}[h]
    \centering
    \resizebox{1.0\linewidth}{!}{
    \begin{tabular}{cccc}
    \tabcolsep=0.01cm
    Gaussian Mixture& LoG Mixture & DoG Mixture & GEF Mixture \\ 
    \includegraphics[width=0.24\linewidth]{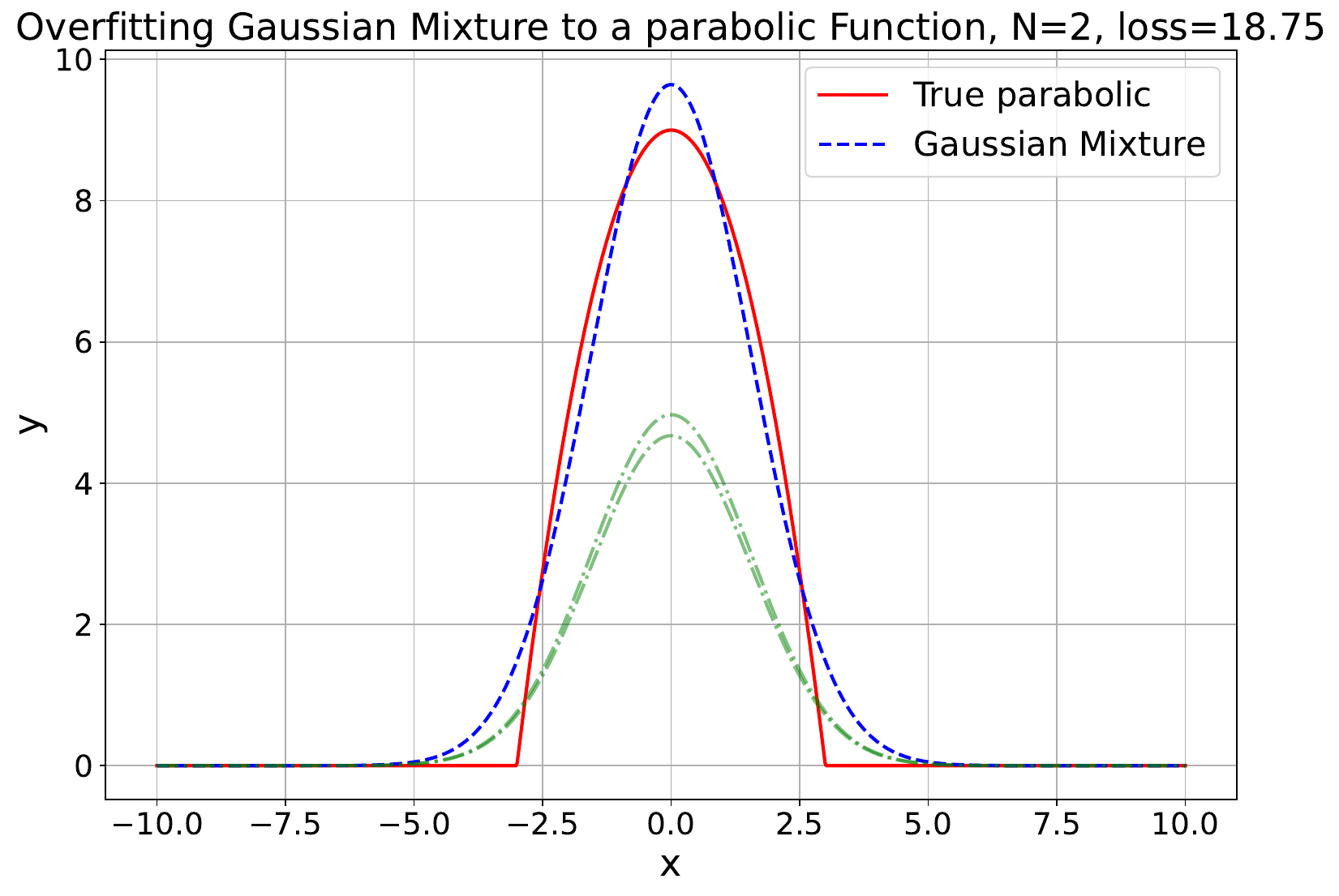} & 
    \includegraphics[width=0.24\linewidth]{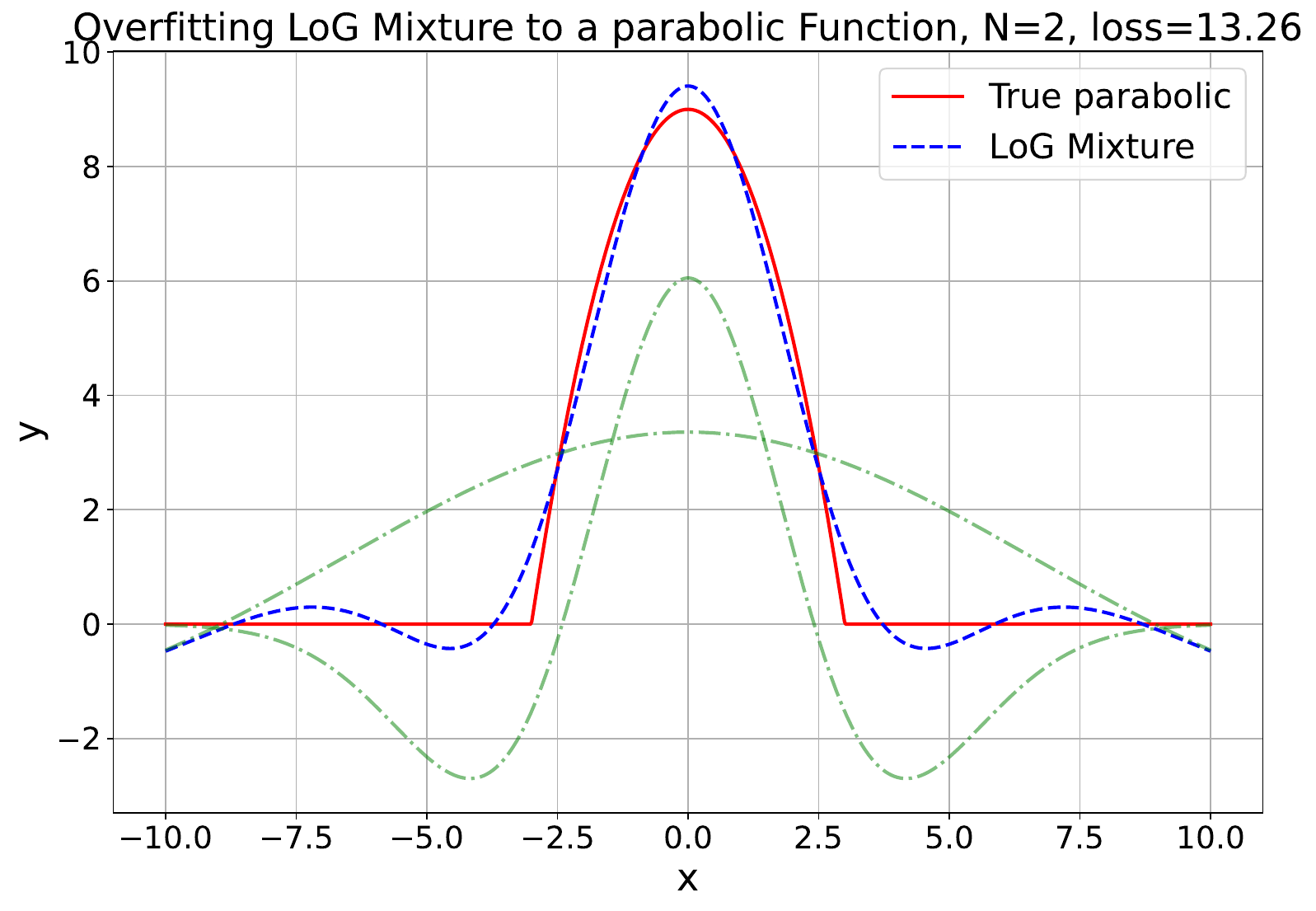} & 
    \includegraphics[width=0.24\linewidth]{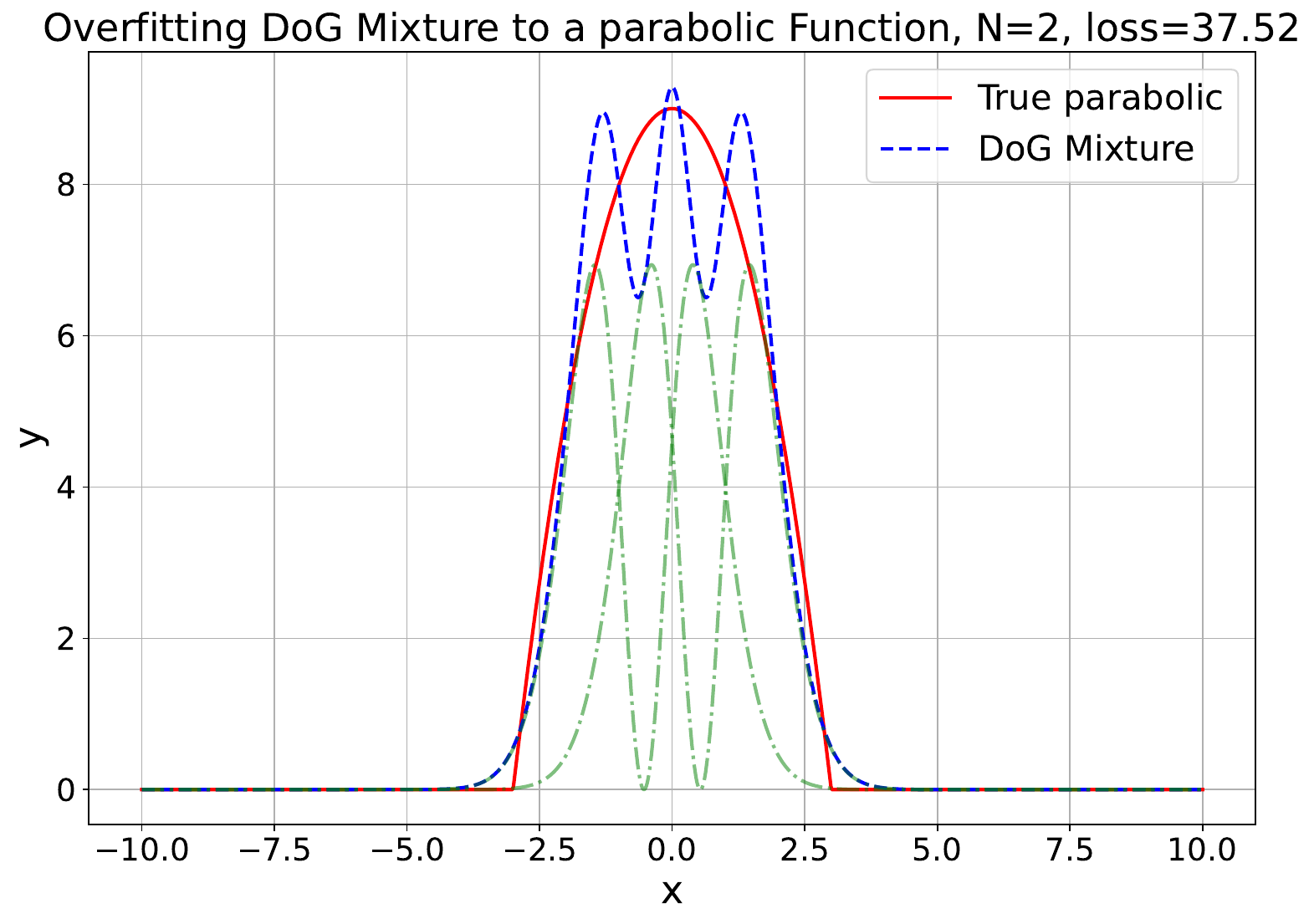} & 
    \includegraphics[width=0.24\linewidth]{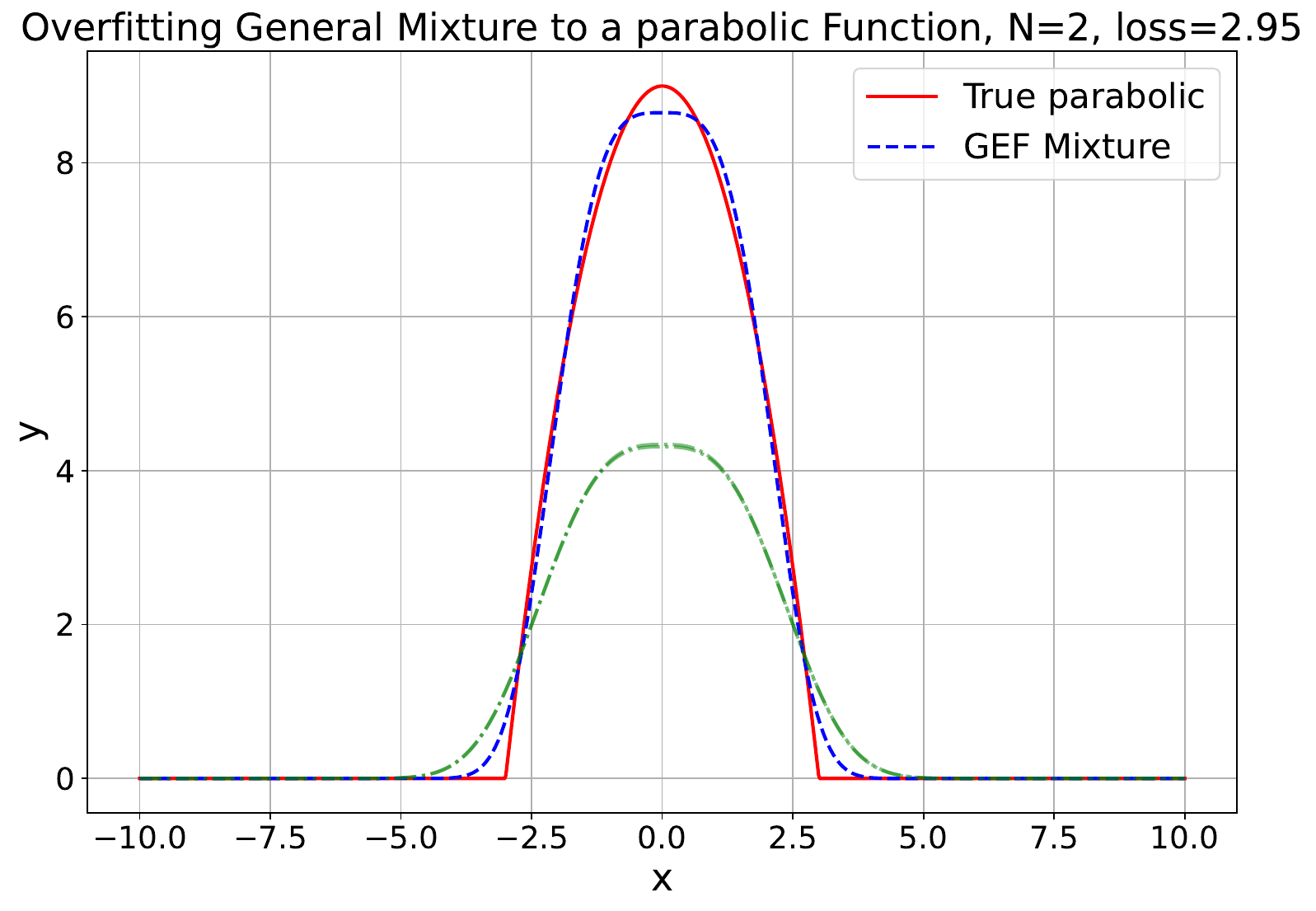}\\ 
    \includegraphics[width=0.24\linewidth]{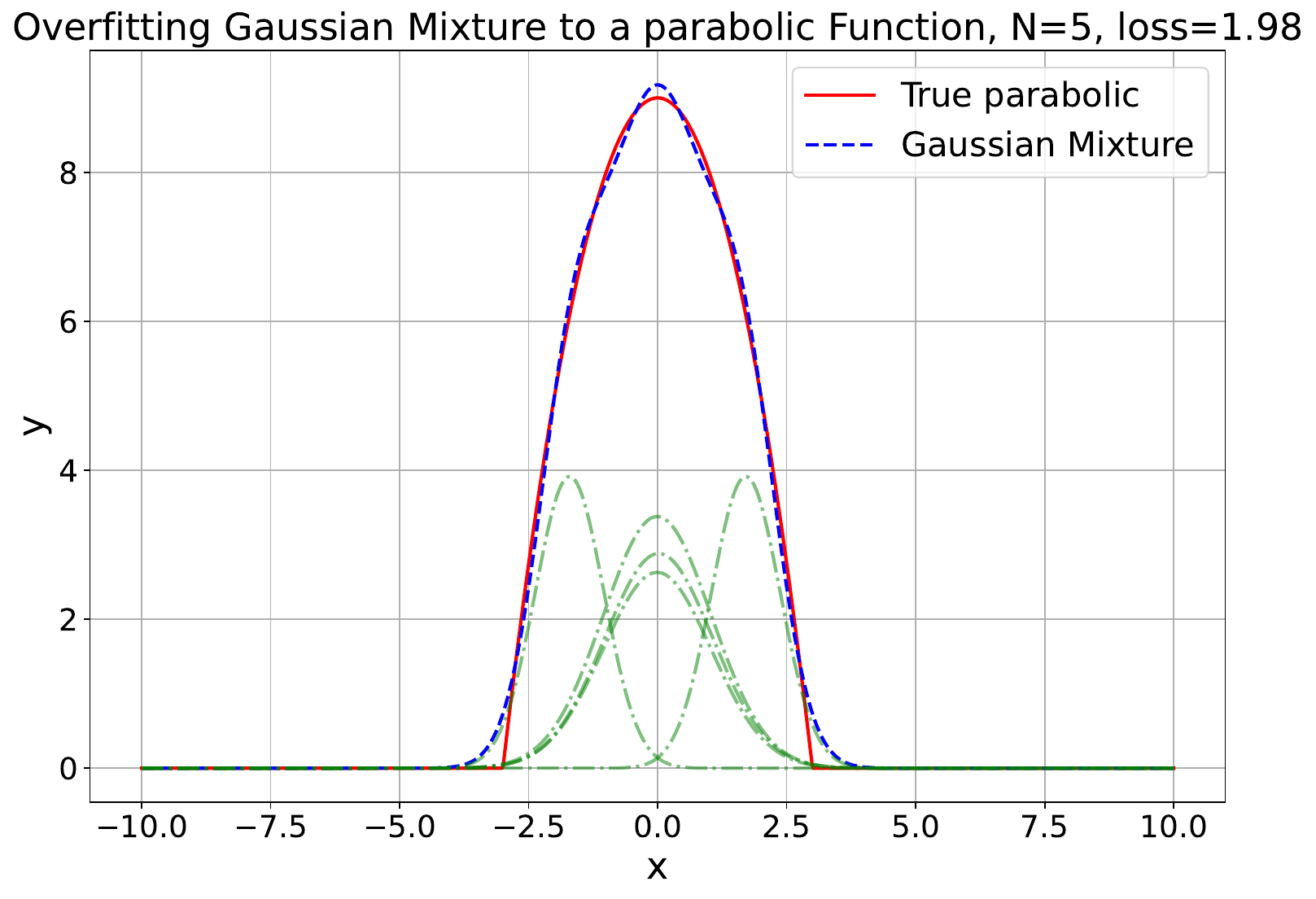} & 
    \includegraphics[width=0.24\linewidth]{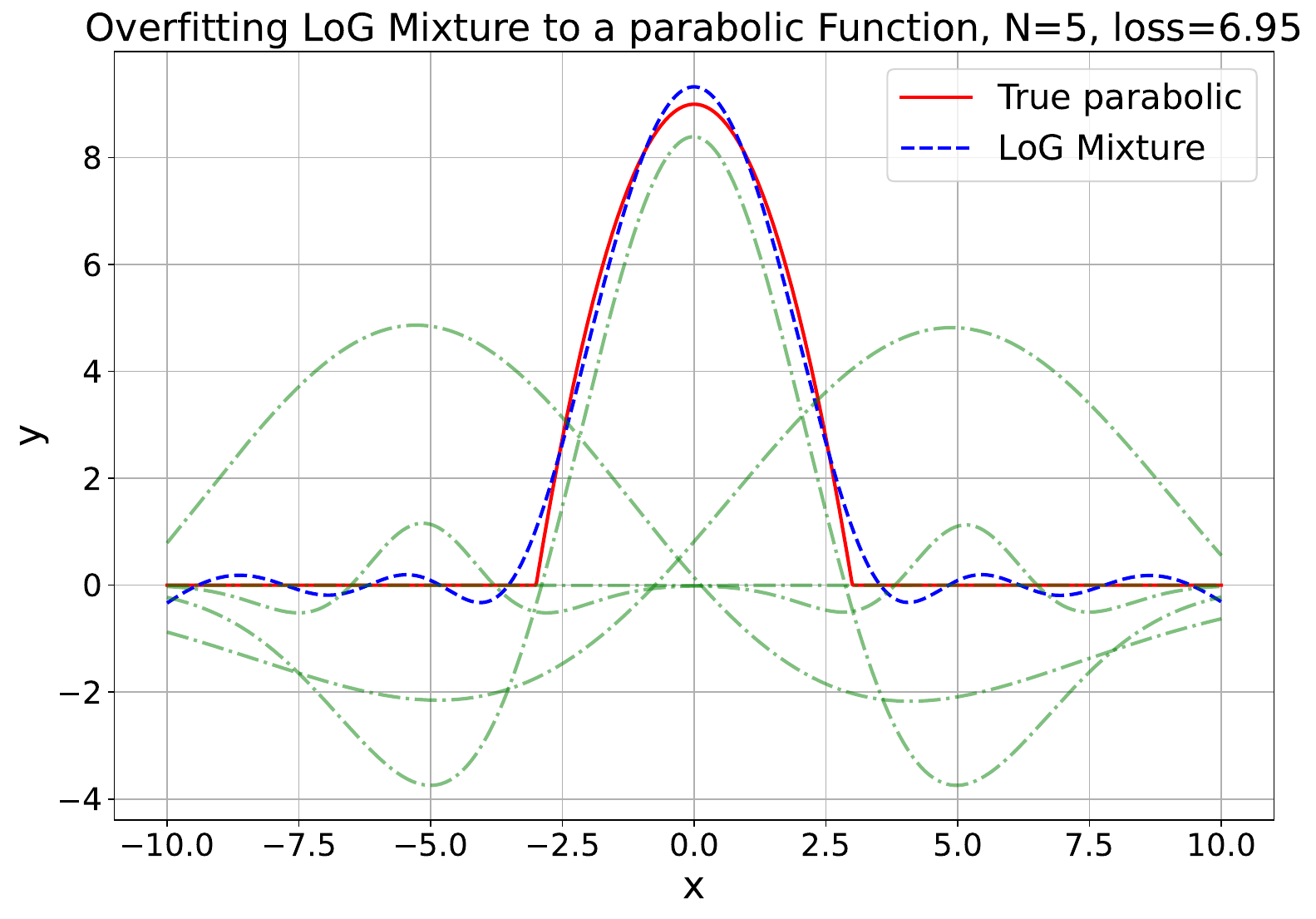} & 
    \includegraphics[width=0.24\linewidth]{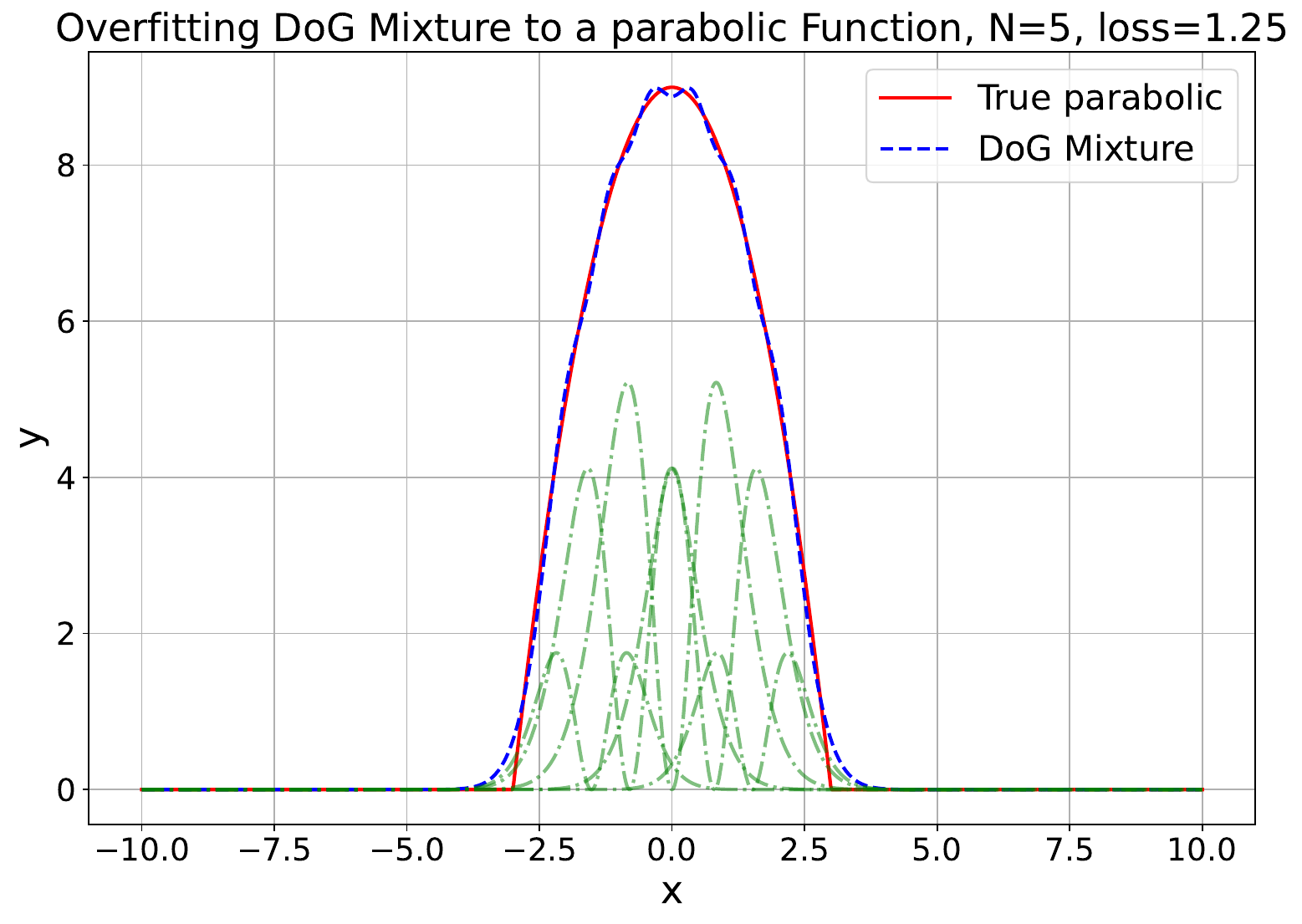} & 
    \includegraphics[width=0.24\linewidth]{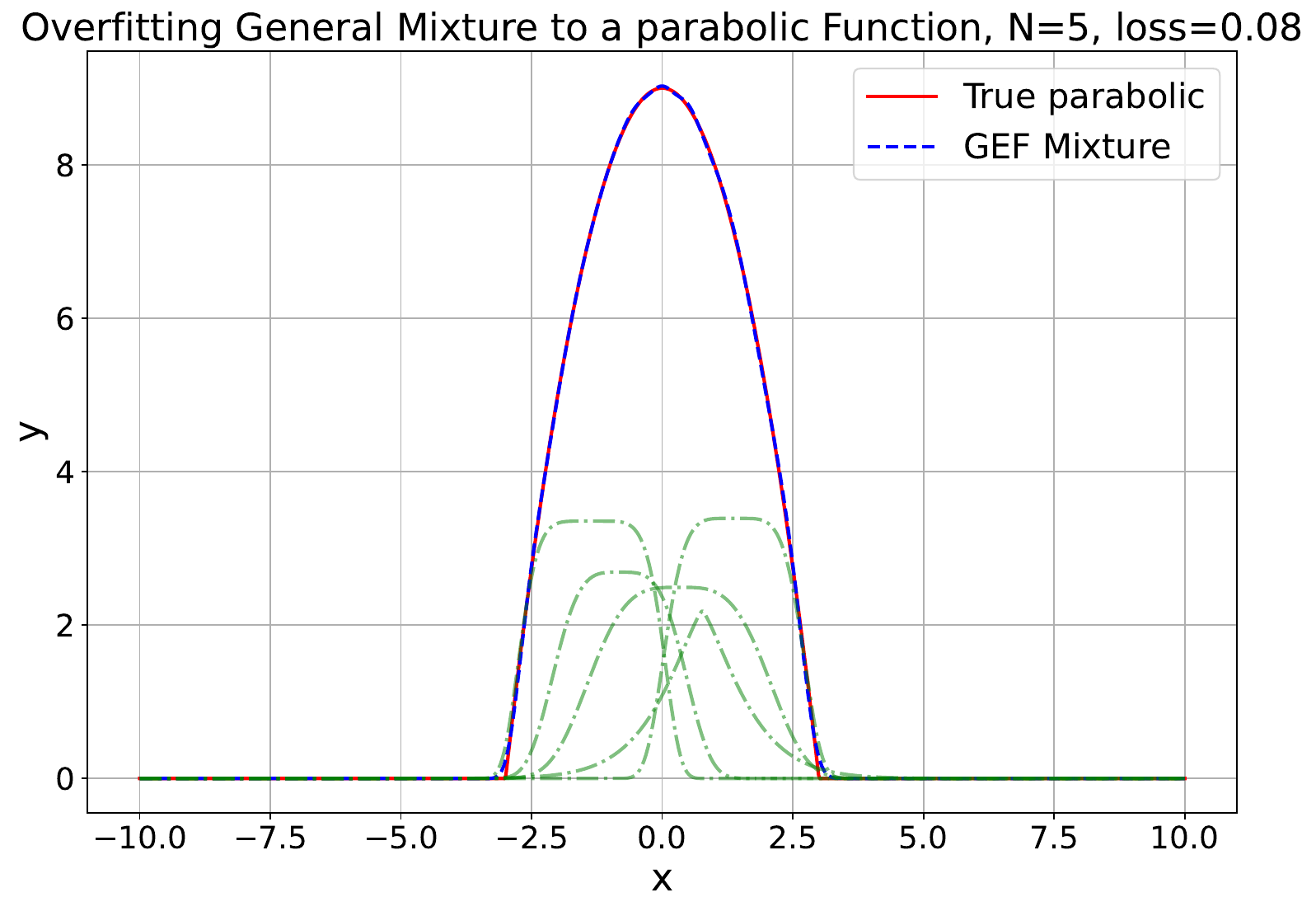}\\ 
    \includegraphics[width=0.24\linewidth]{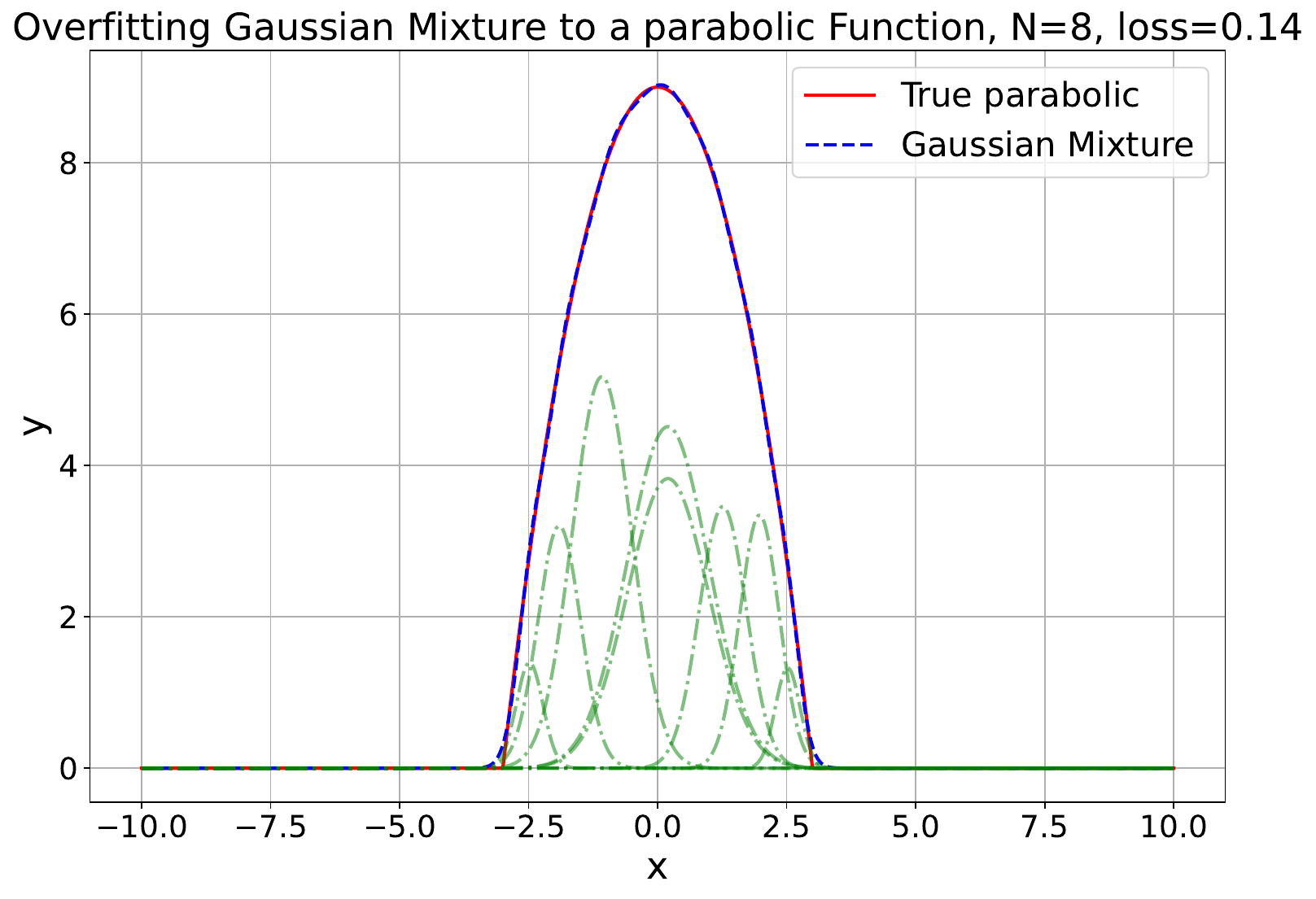} & 
    \includegraphics[width=0.24\linewidth]{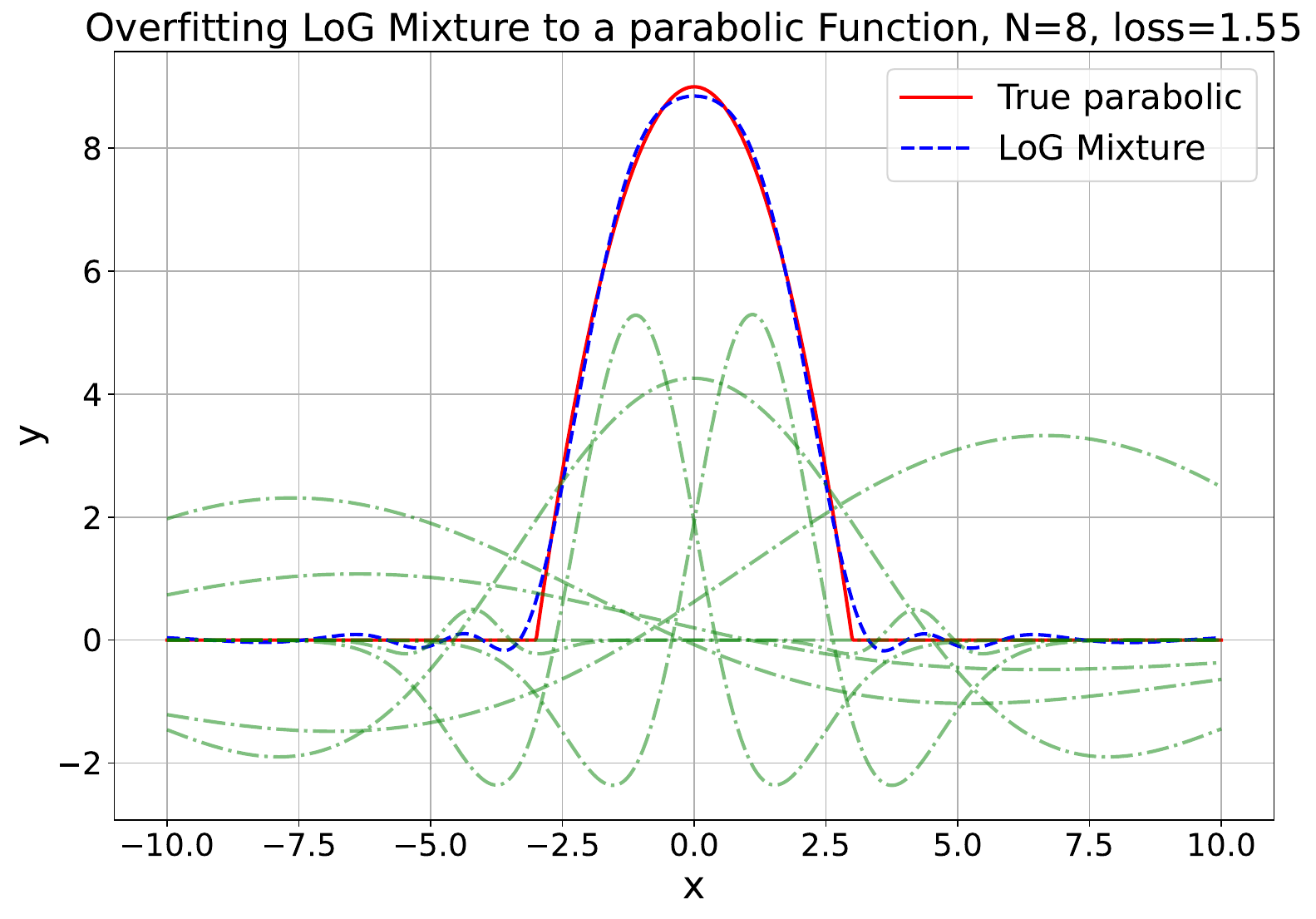} & 
    \includegraphics[width=0.24\linewidth]{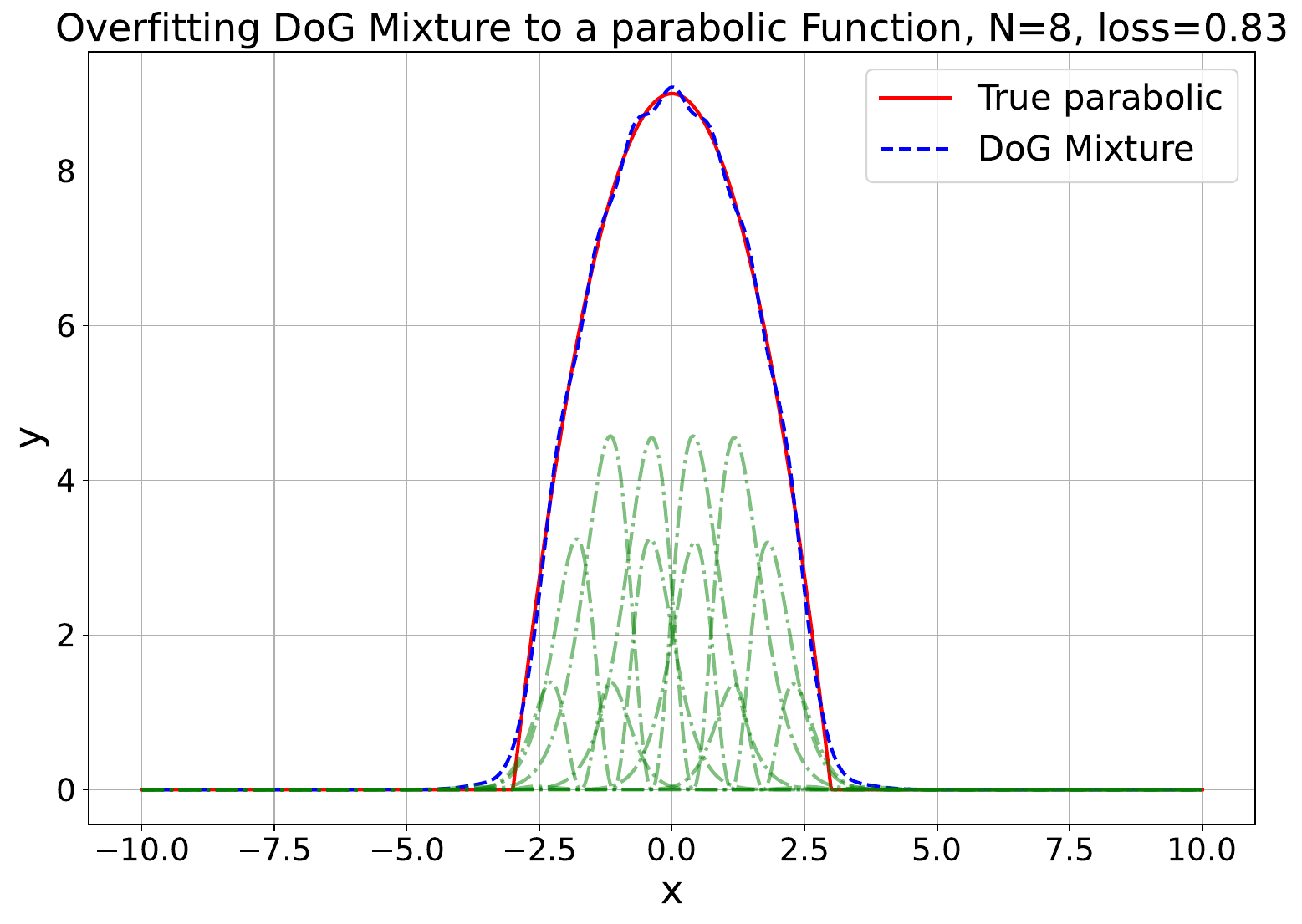} & 
    \includegraphics[width=0.24\linewidth]{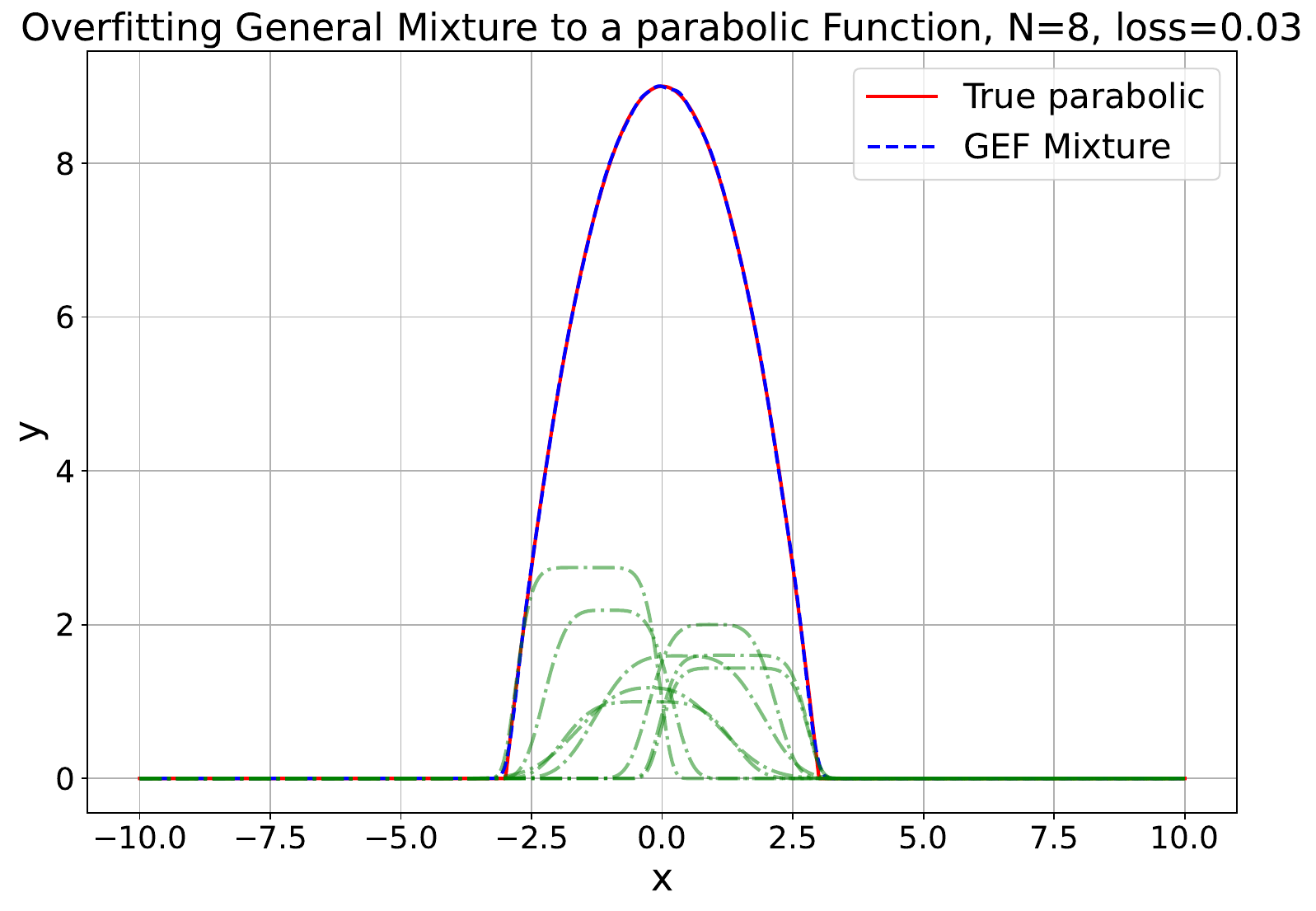}\\ 
    \includegraphics[width=0.24\linewidth]{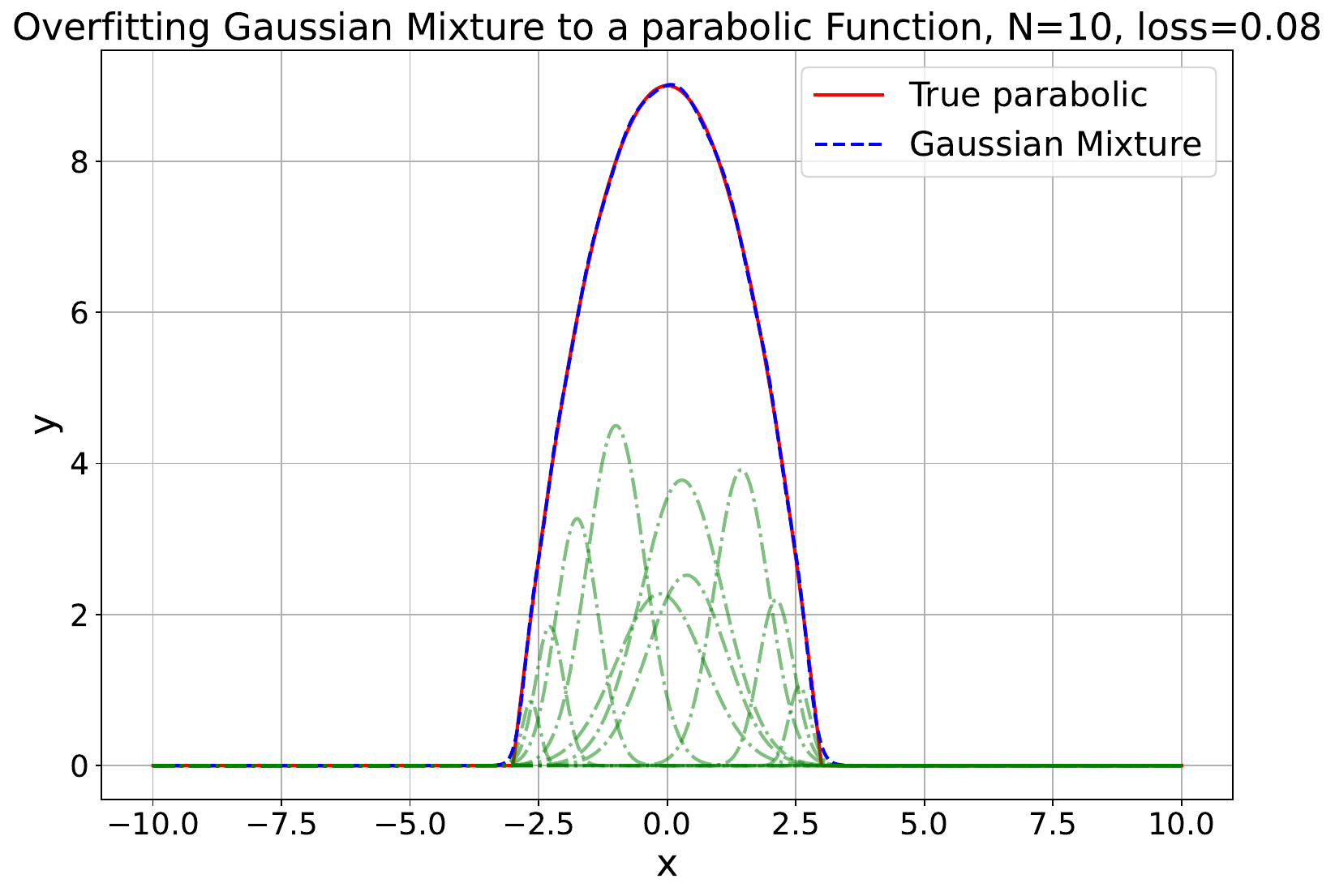} & 
    \includegraphics[width=0.24\linewidth]{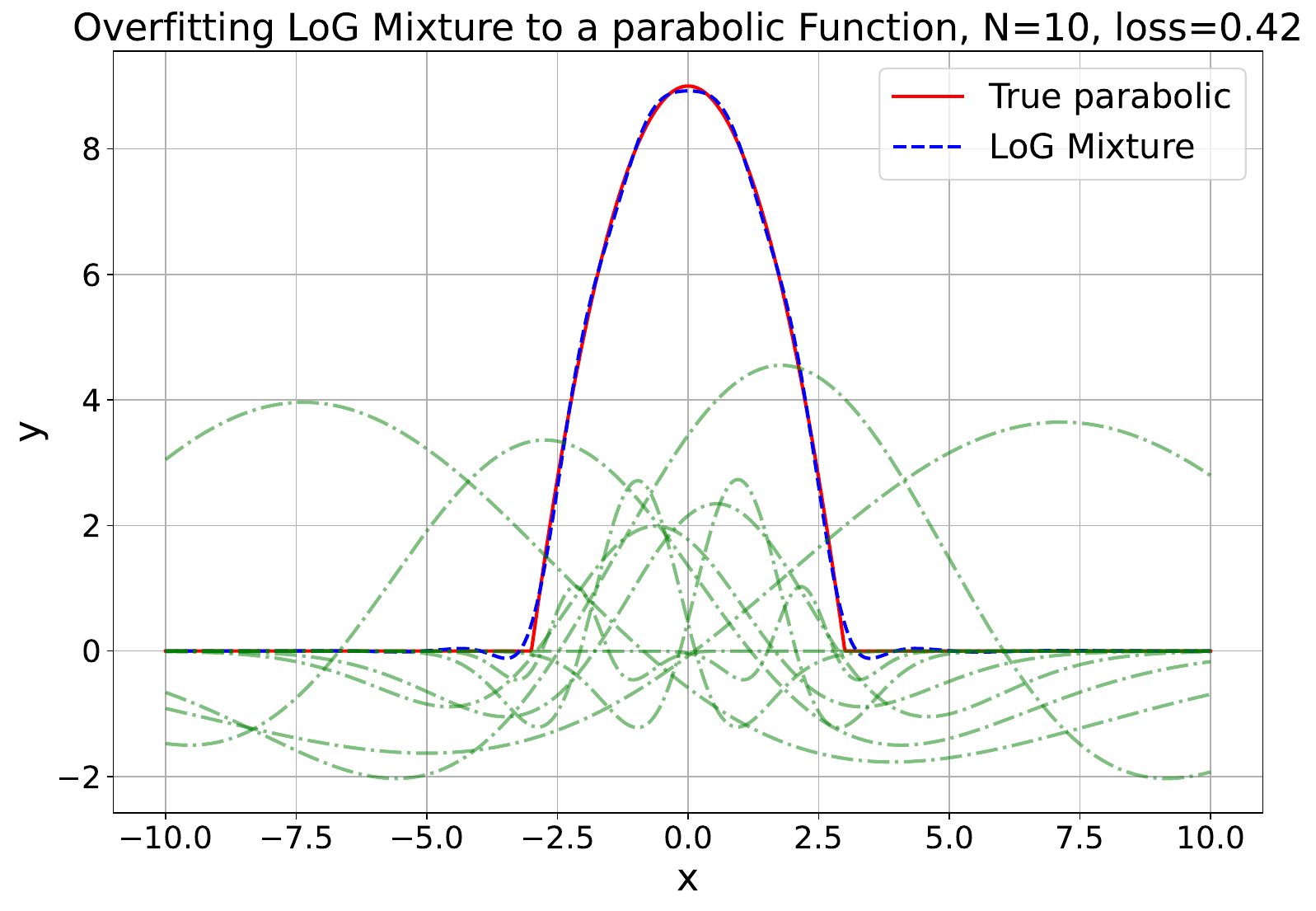} & 
    \includegraphics[width=0.24\linewidth]{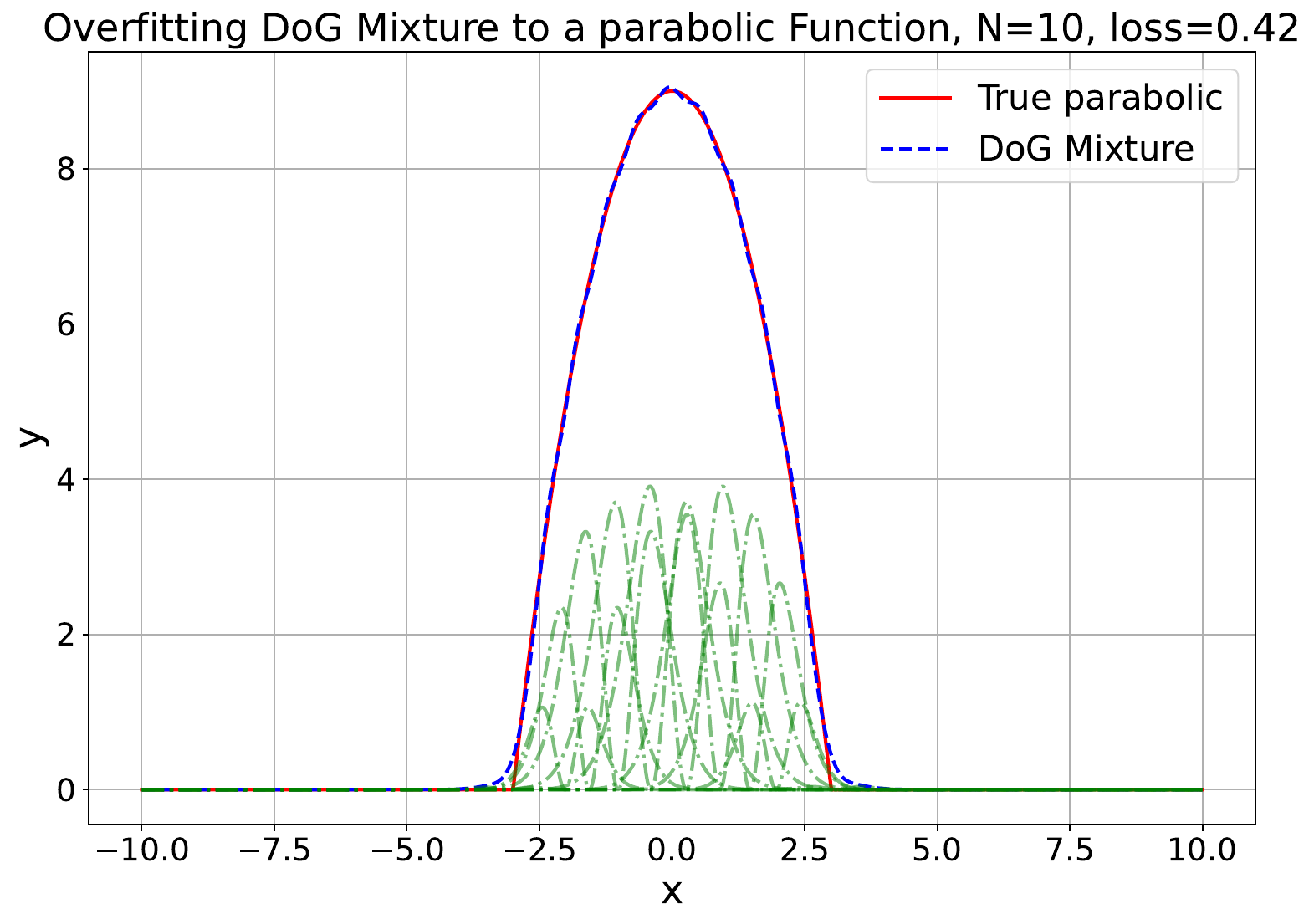} & 
    \includegraphics[width=0.24\linewidth]{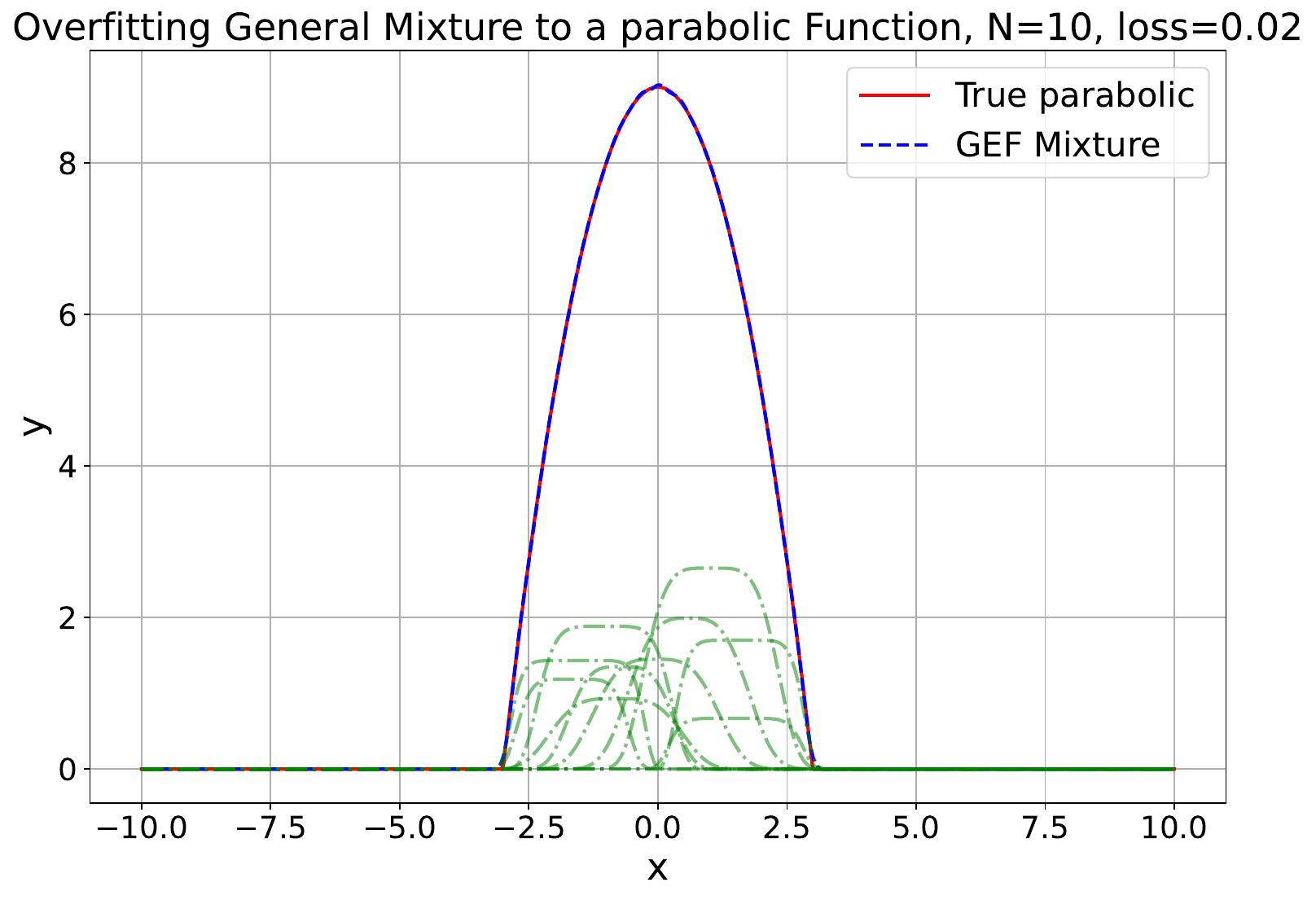}\\ 
    \includegraphics[width=0.24\linewidth]{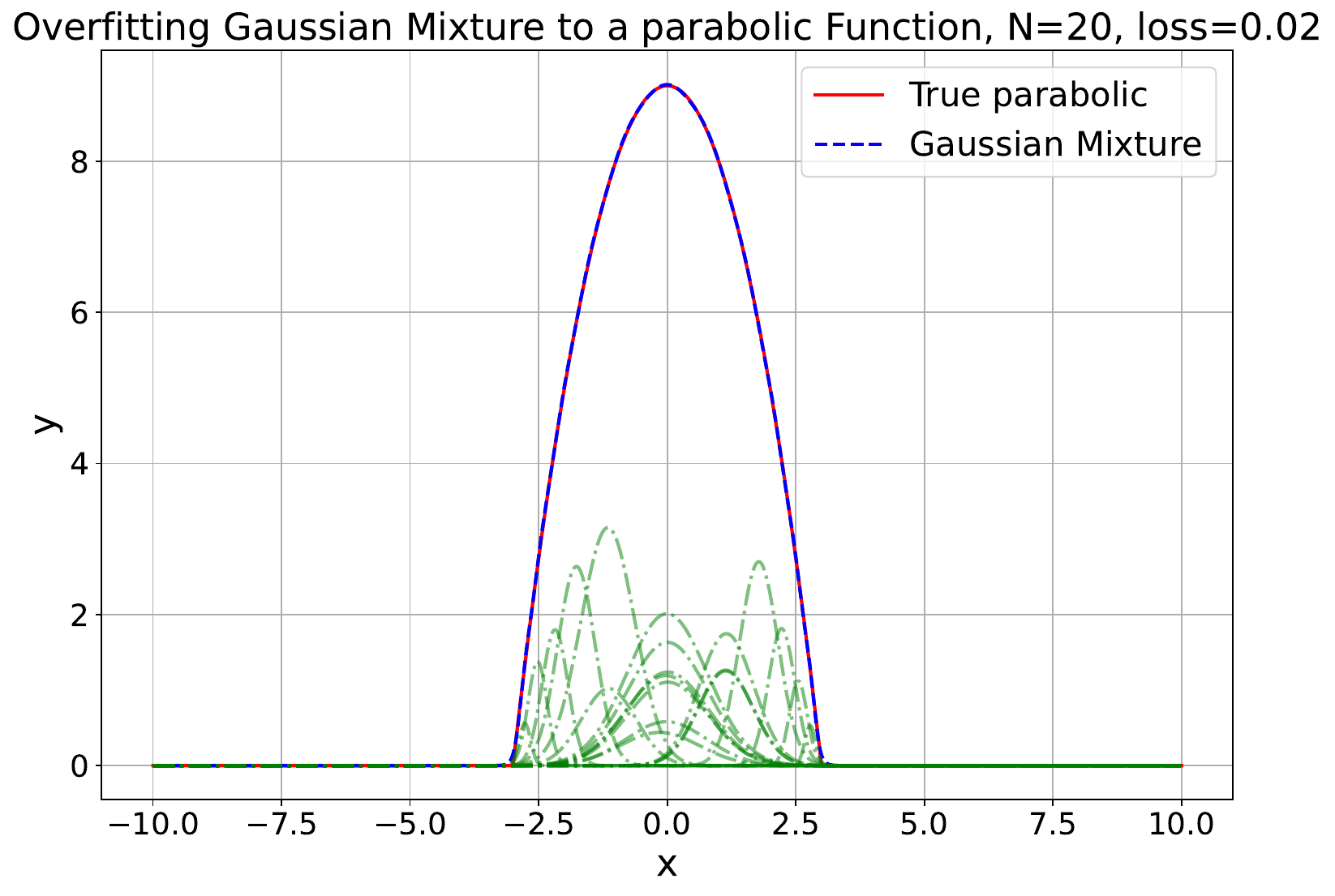} & 
    \includegraphics[width=0.24\linewidth]{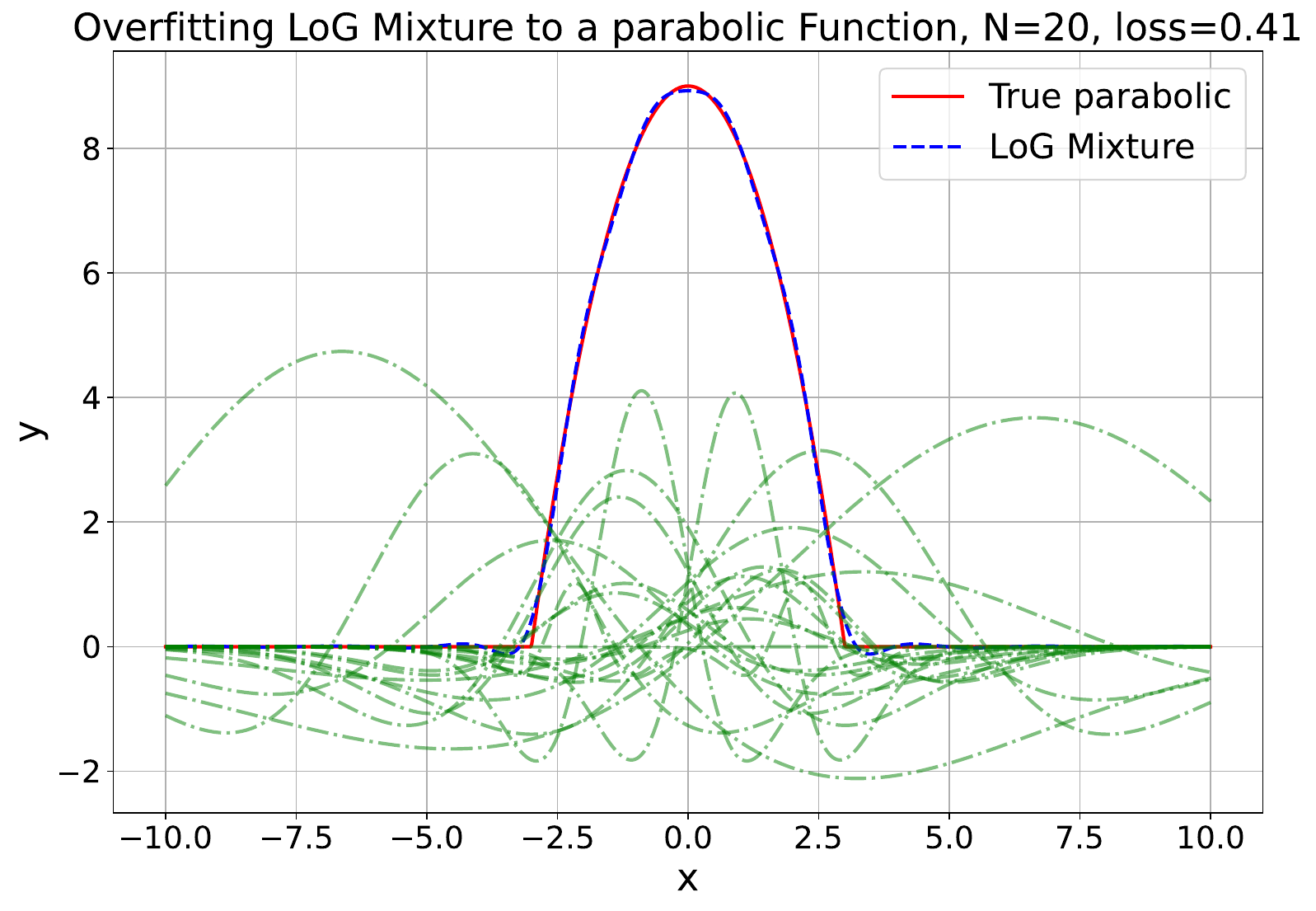} & 
    \includegraphics[width=0.24\linewidth]{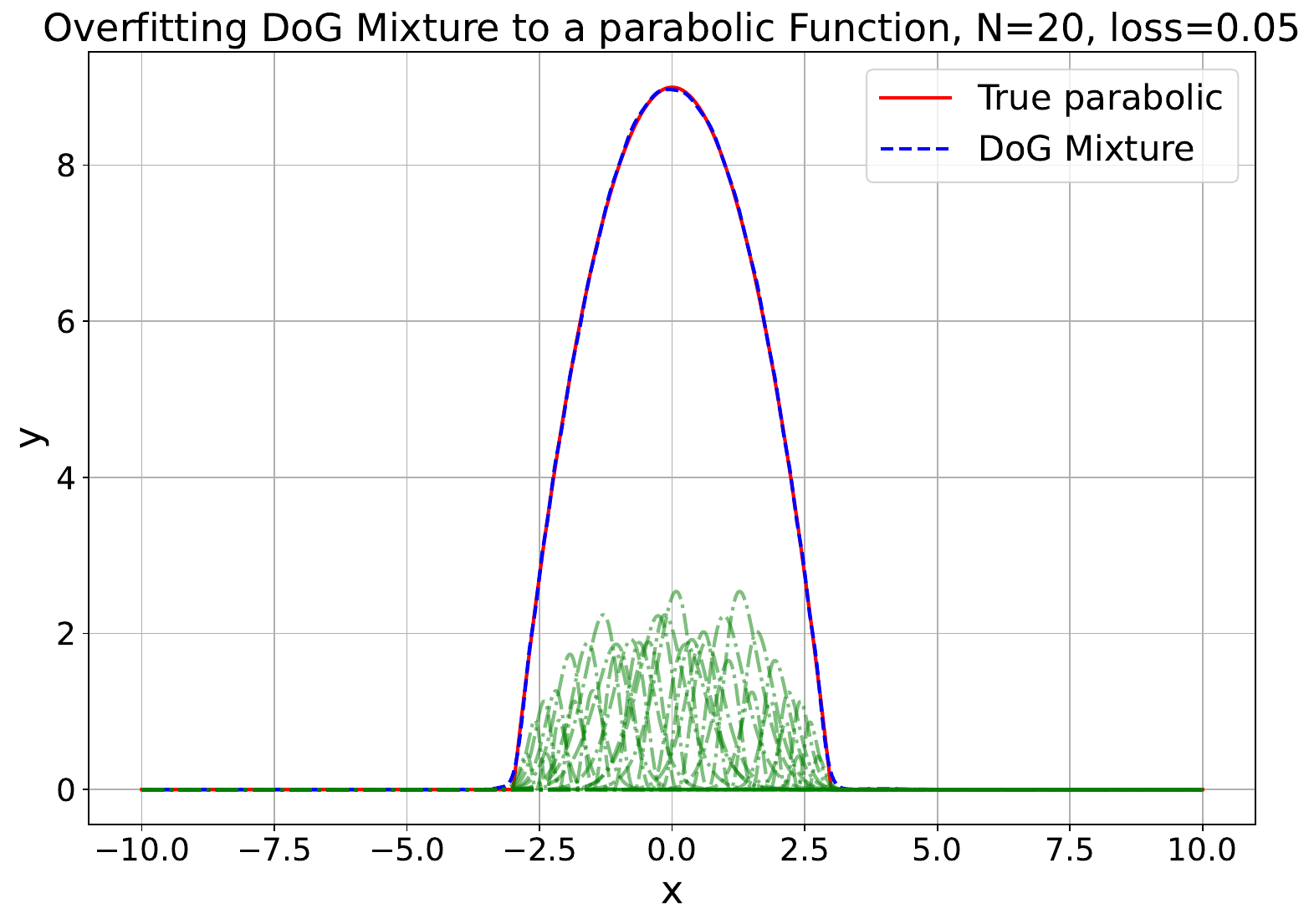} & 
    \includegraphics[width=0.24\linewidth]{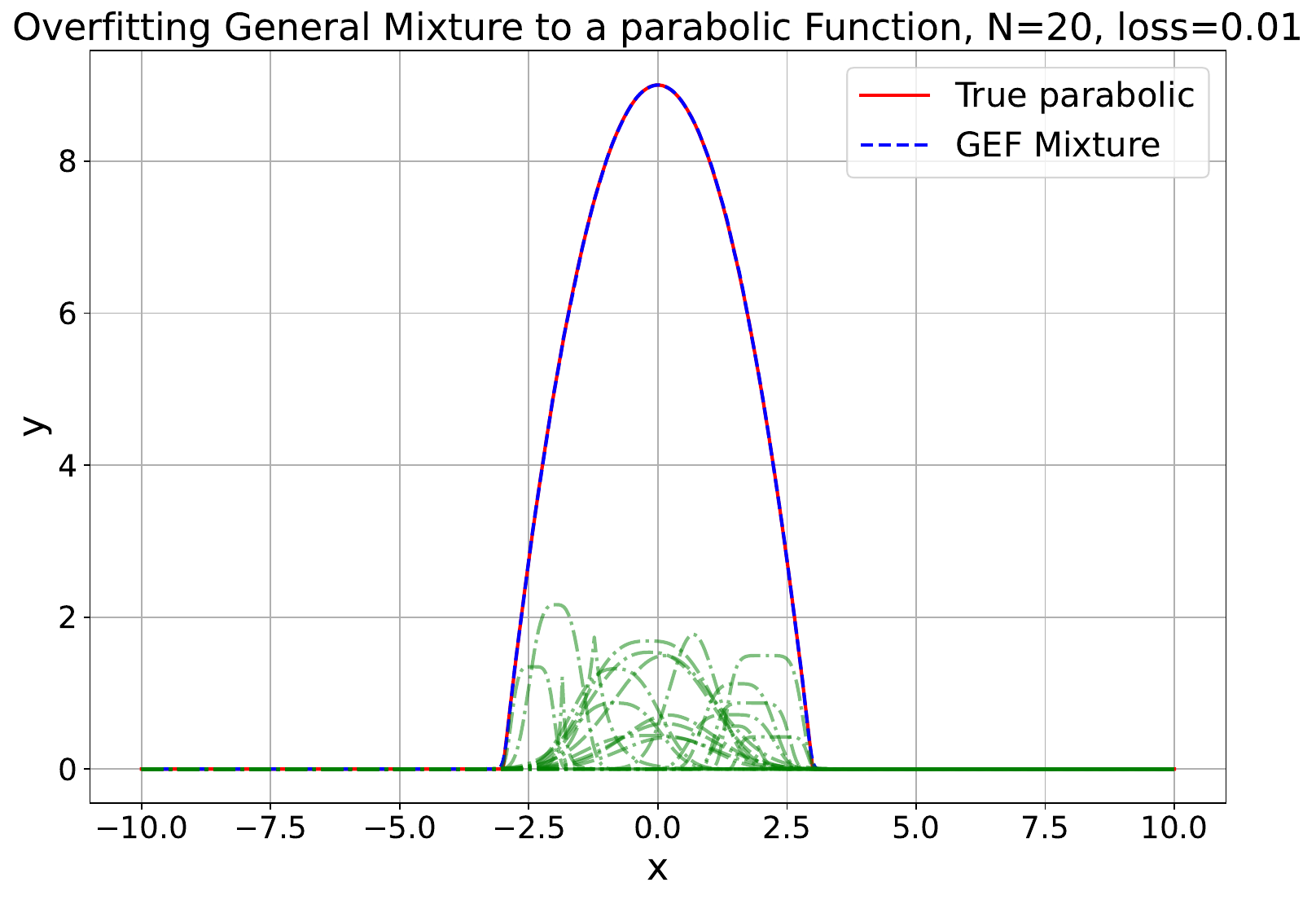}\\ 
    
    \end{tabular}
    }
    \caption{\textbf{Numerical Simulation Examples of Fitting parabolics with Positive Weights Mixtures ( N= 2, 5, 8, and 10 )}. We show some fitting examples for parabolic signals with positive weights mixtures. The four mixtures used from left to right are Gaussians, LoG, DoG, and General mixtures. From top to bottom: N = 2, 8, and 10 components. The optimized individual components are shown in green. Some examples fail to optimize due to numerical instability in both Gaussians and GEF mixtures. Note that GEF is very efficient in fitting the parabolic with few components while LoG and DoG are more stable for a larger number of components. }
    \label{supfig:fitting_parabolic_p}
    \end{figure*}
    

%% file: figures/fitting/fitting_parabola_n.tex
\begin{figure*}[h]
    \centering
    \resizebox{1.0\linewidth}{!}{
    \begin{tabular}{cccc}
    \tabcolsep=0.01cm
    Gaussian Mixture& LoG Mixture & DoG Mixture & GEF Mixture \\ 
    \includegraphics[width=0.24\linewidth]{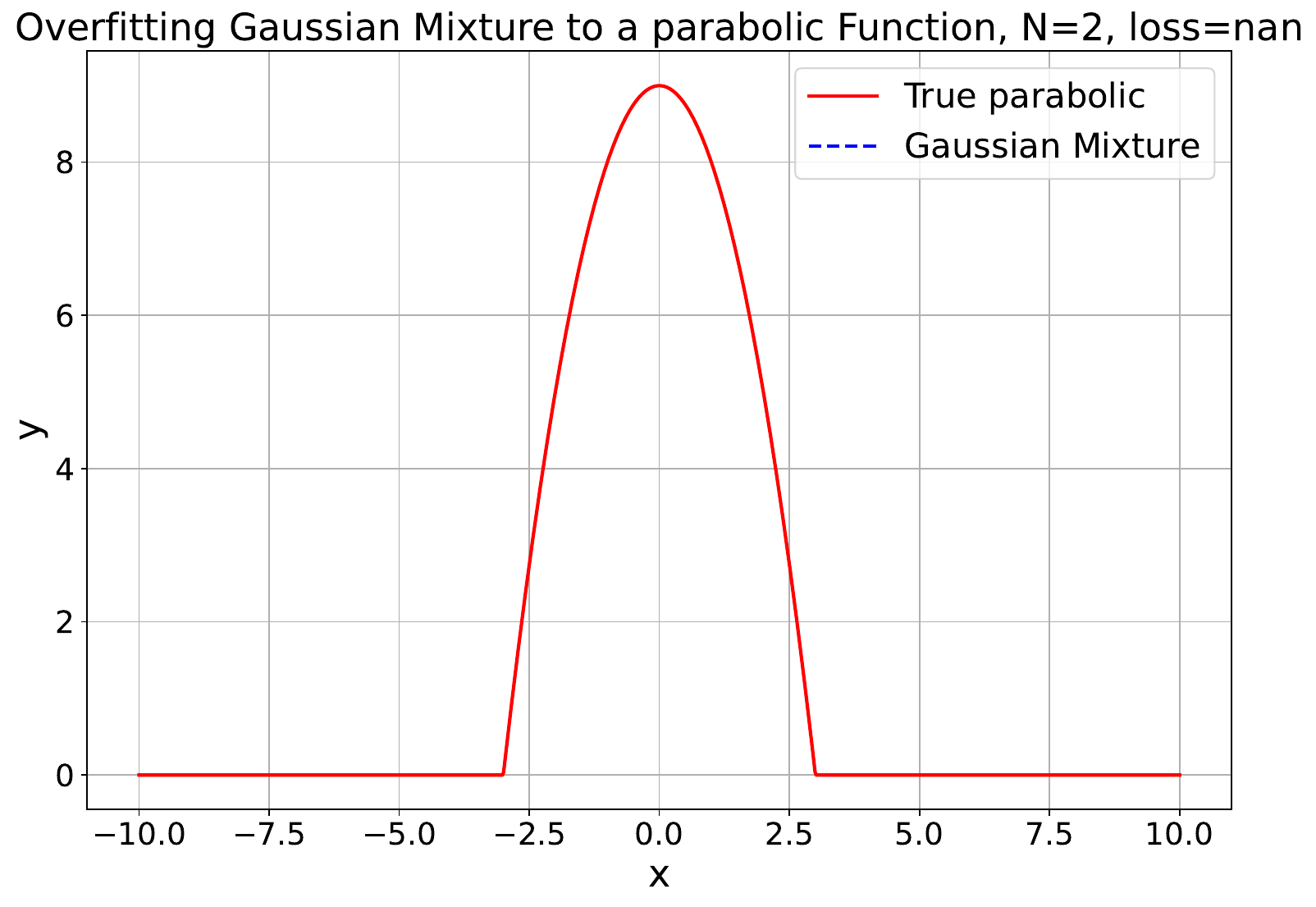} & 
    \includegraphics[width=0.24\linewidth]{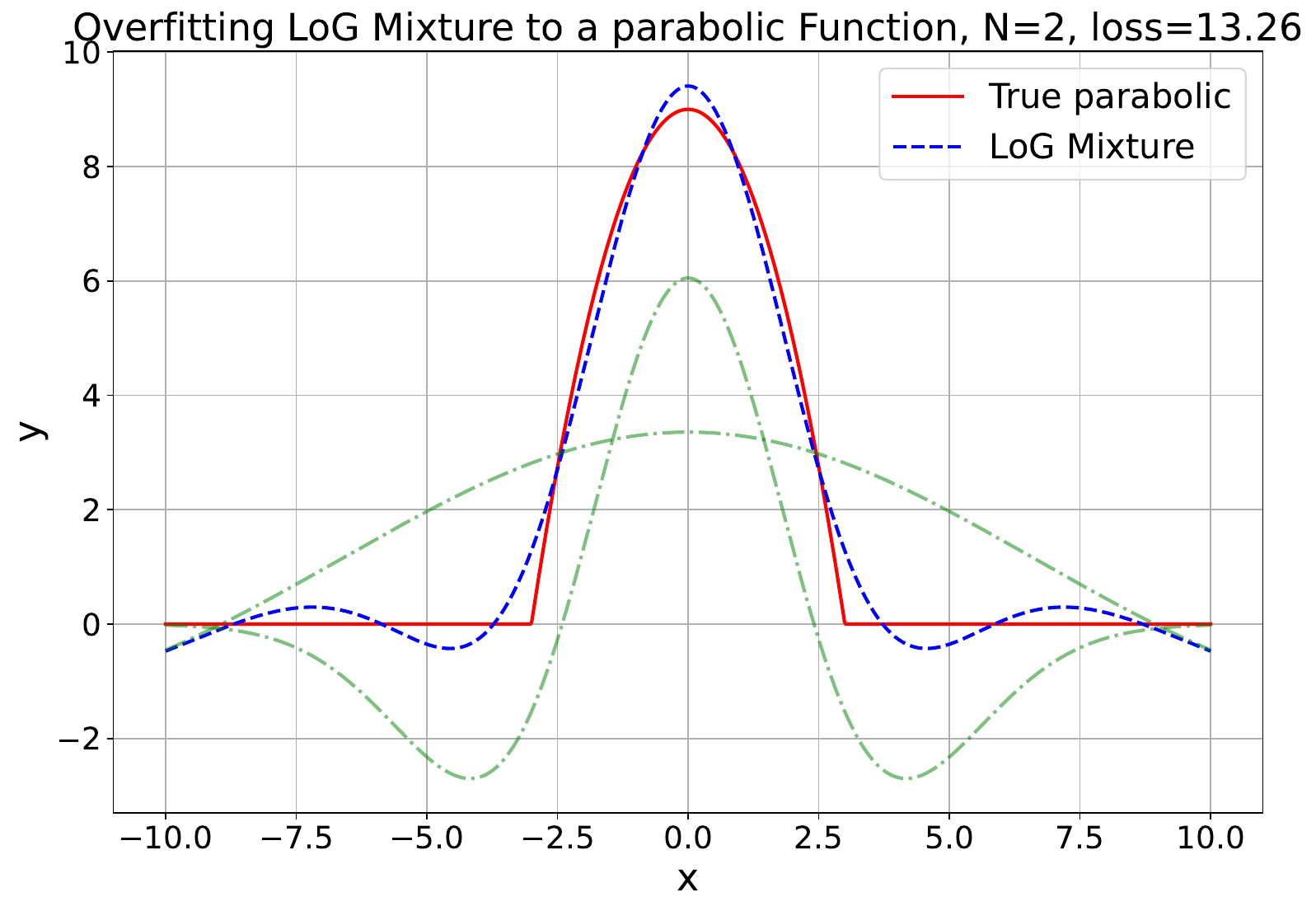} & 
    \includegraphics[width=0.24\linewidth]{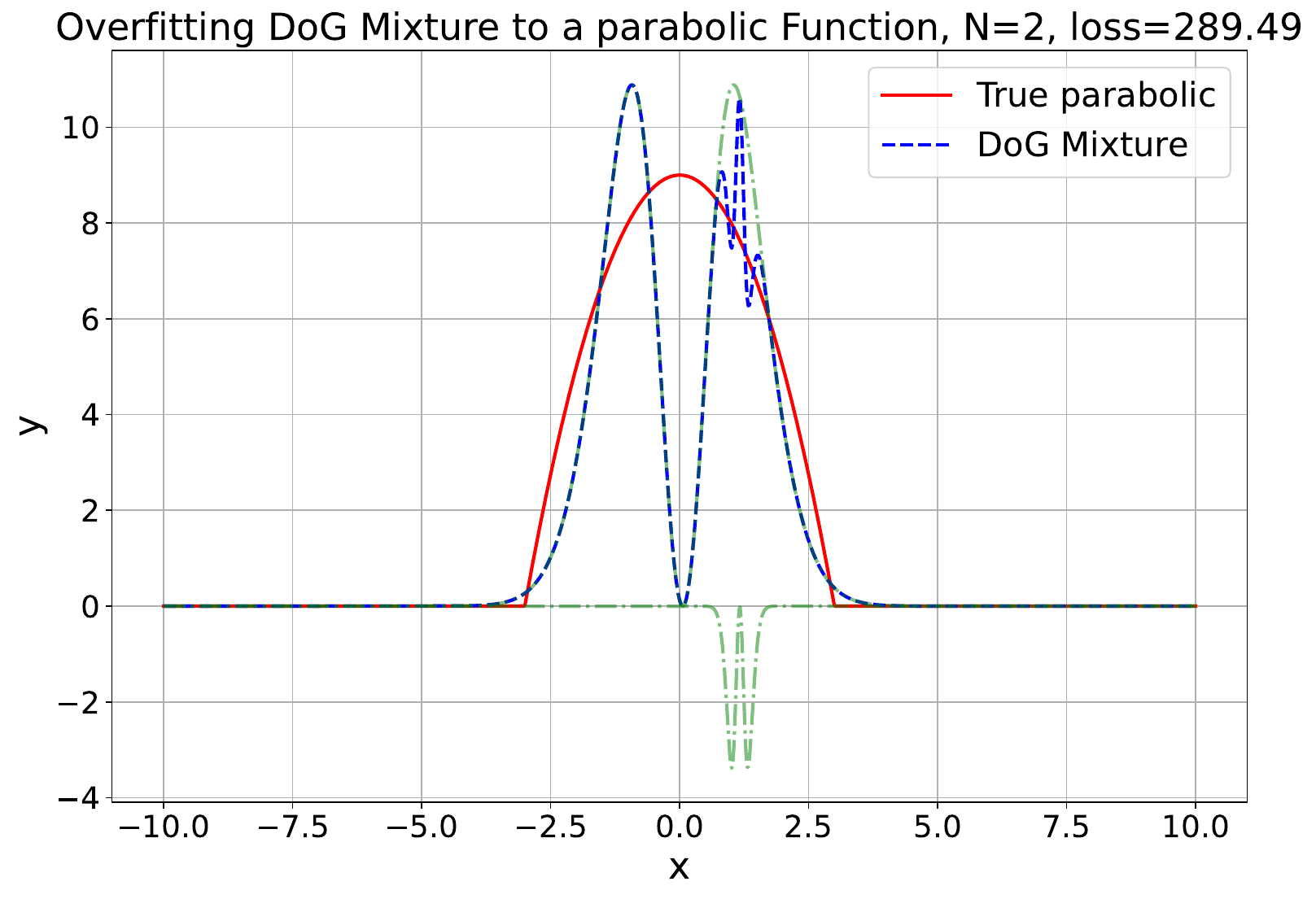} & 
    \includegraphics[width=0.24\linewidth]{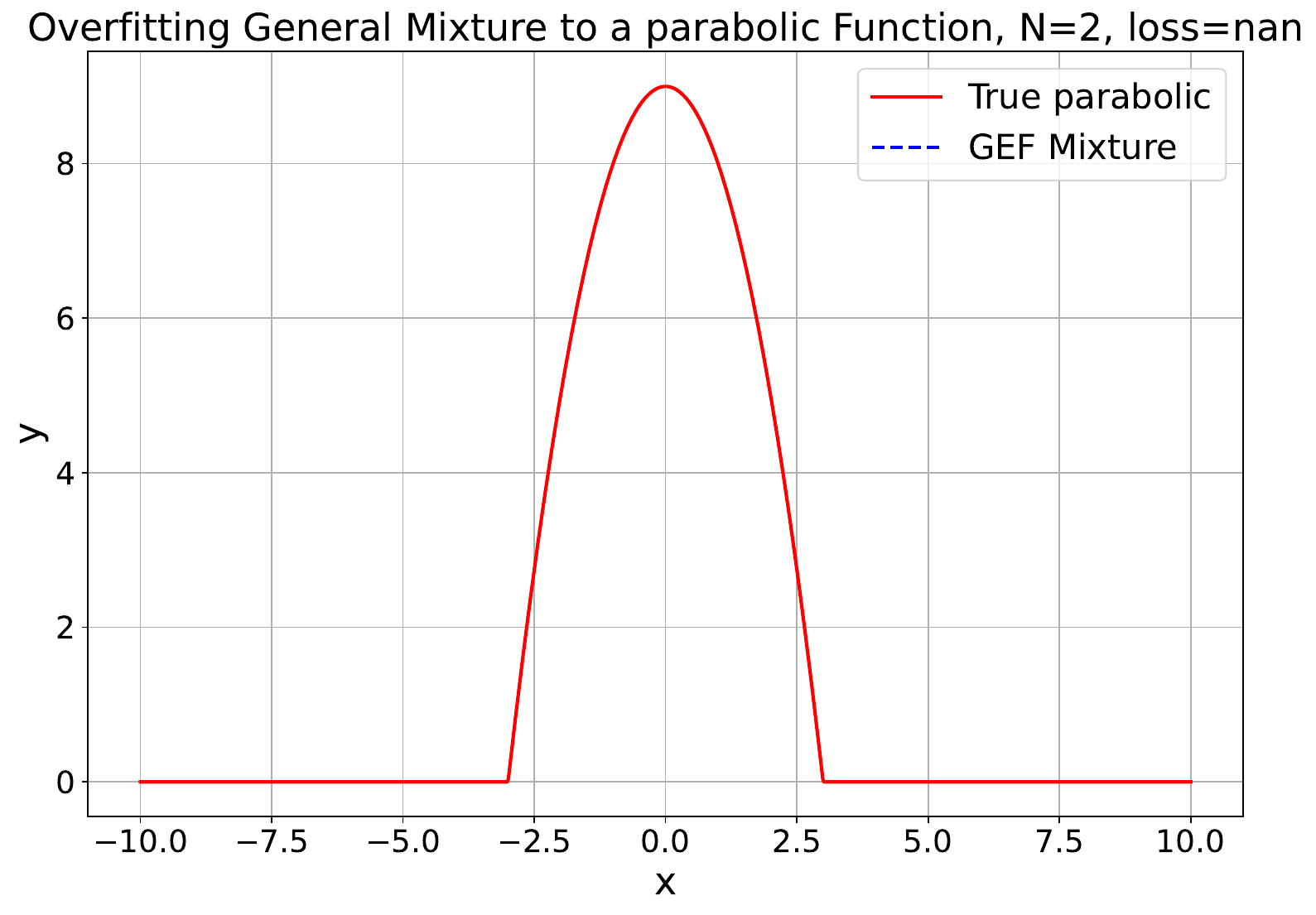}\\ 
    \includegraphics[width=0.24\linewidth]{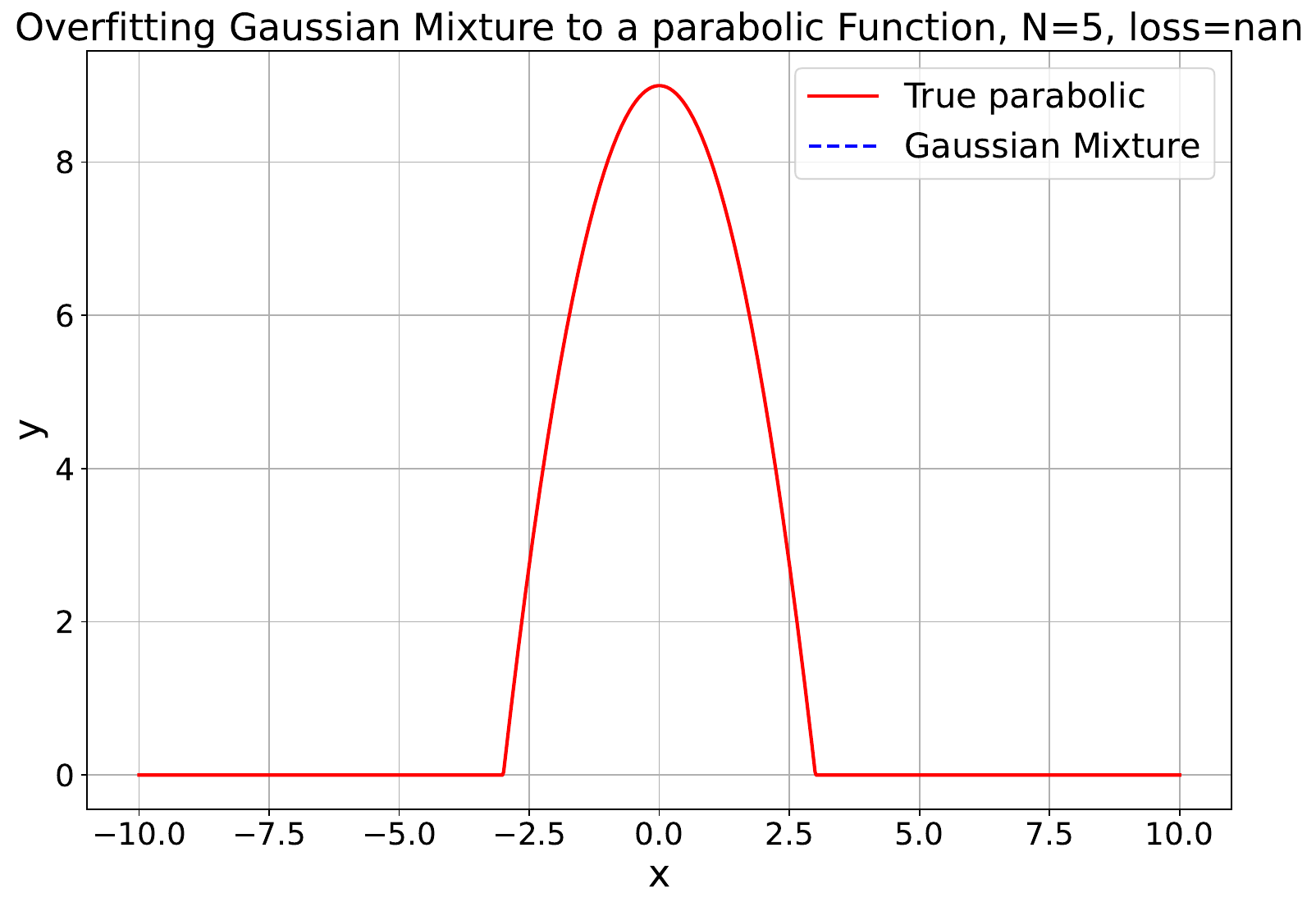} & 
    \includegraphics[width=0.24\linewidth]{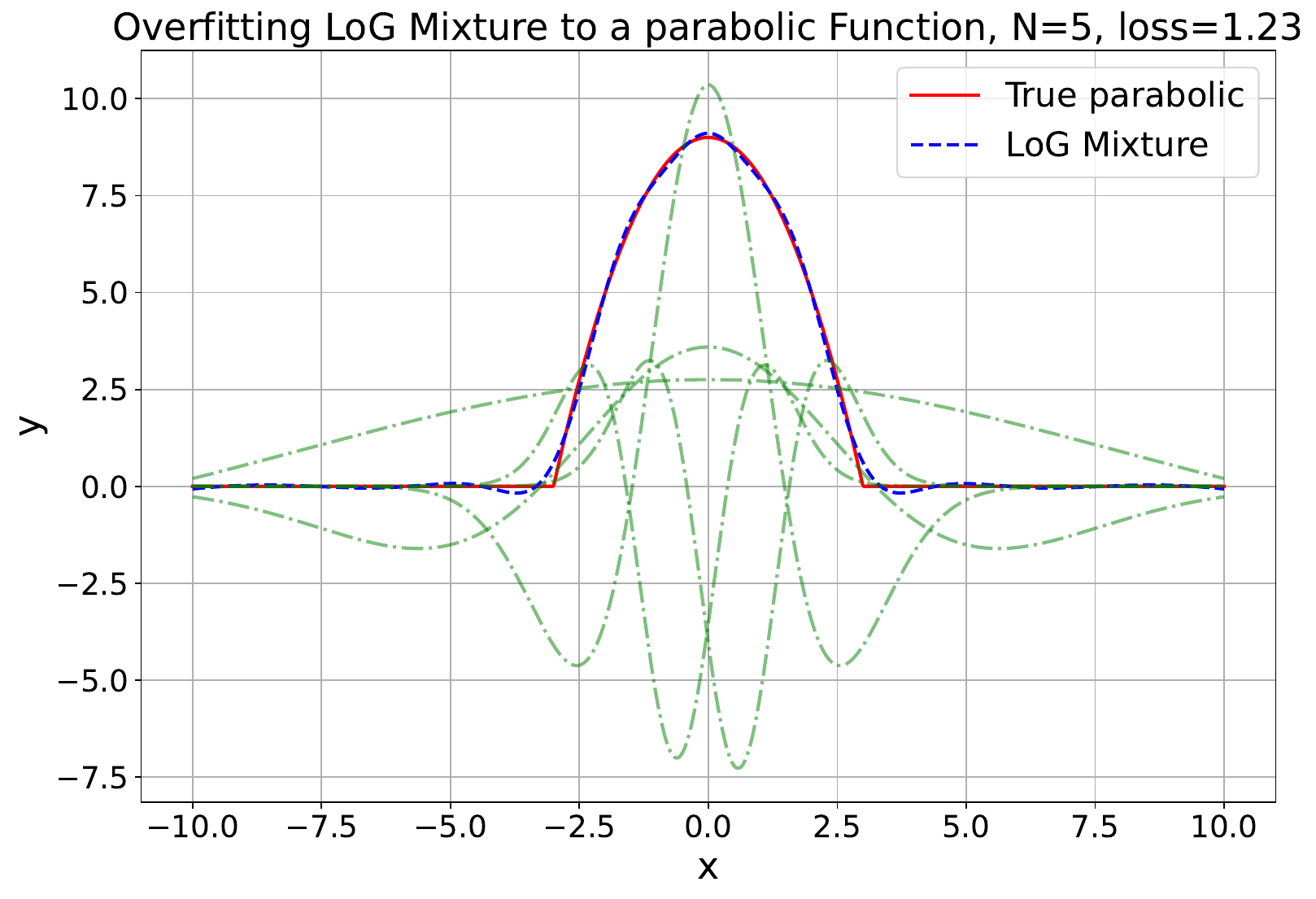} & 
    \includegraphics[width=0.24\linewidth]{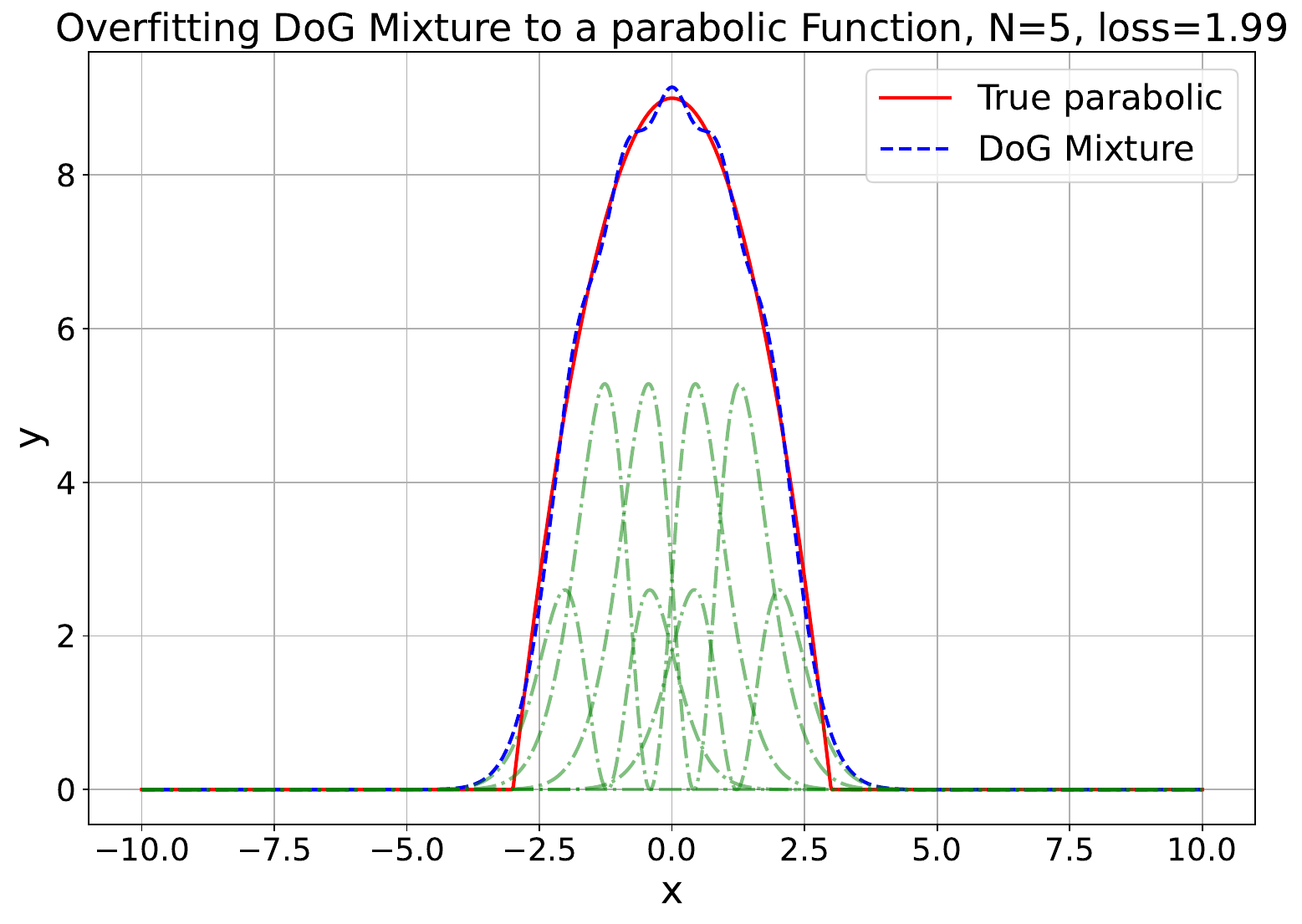} & 
    \includegraphics[width=0.24\linewidth]{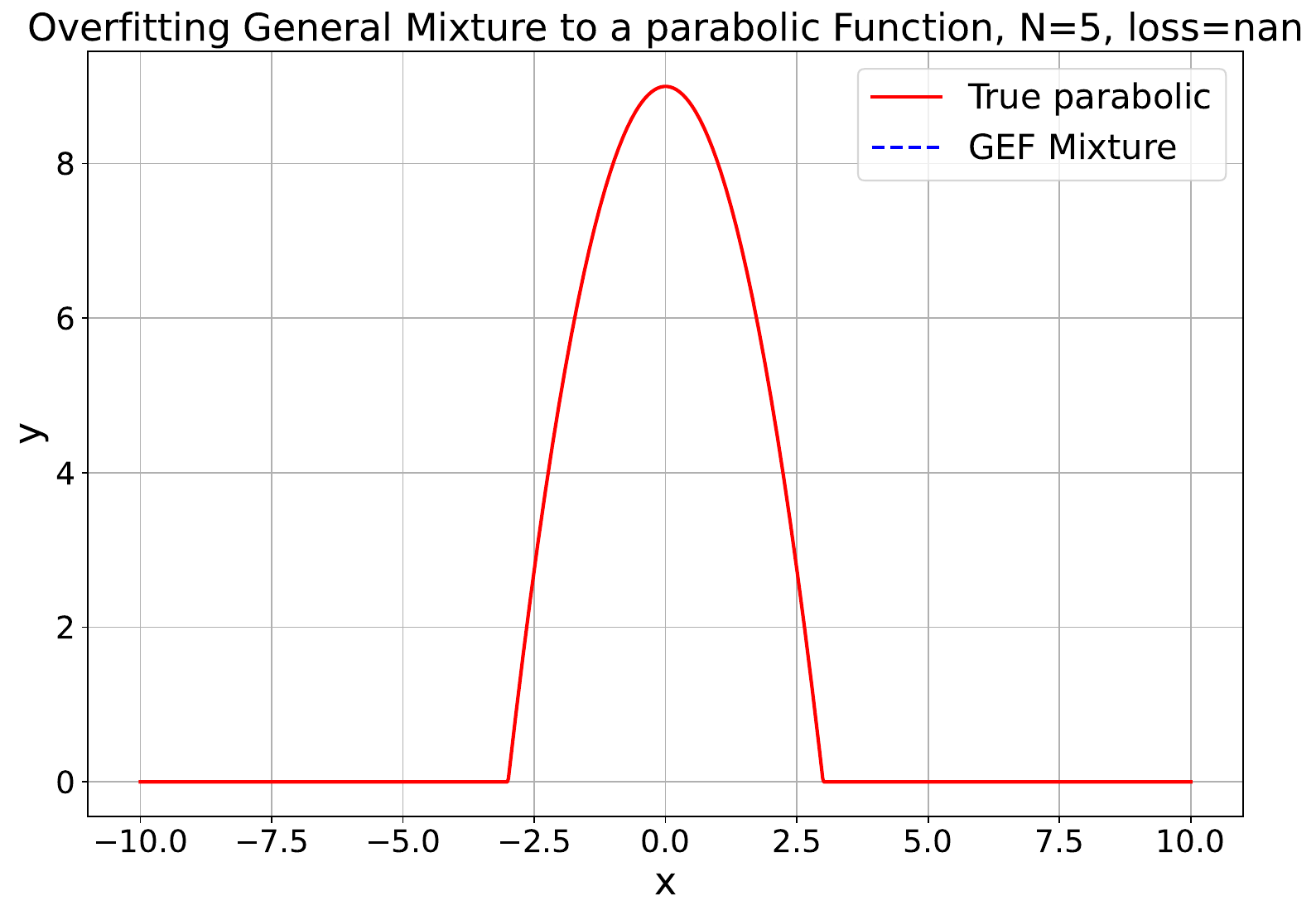}\\ 
    \includegraphics[width=0.24\linewidth]{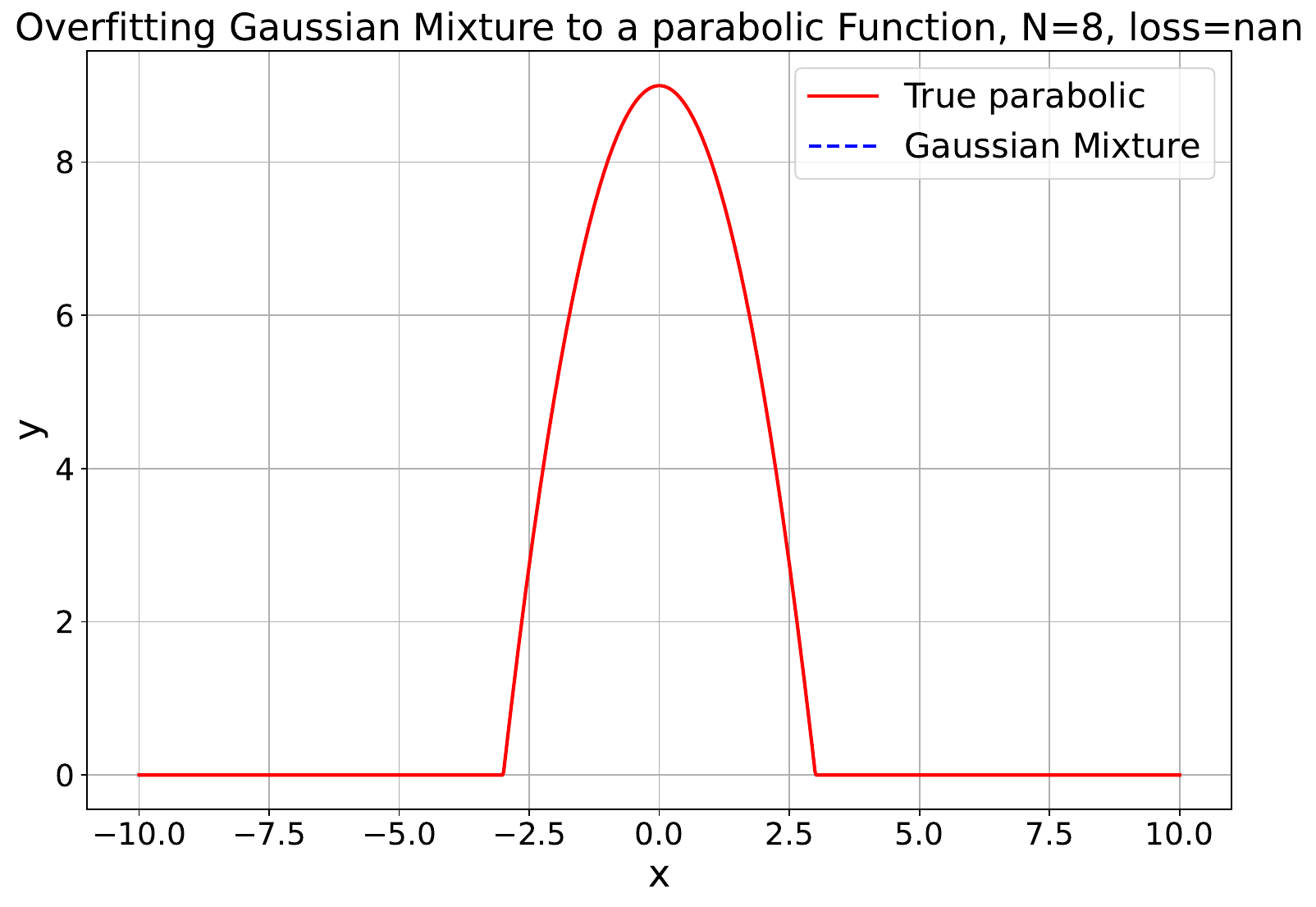} & 
    \includegraphics[width=0.24\linewidth]{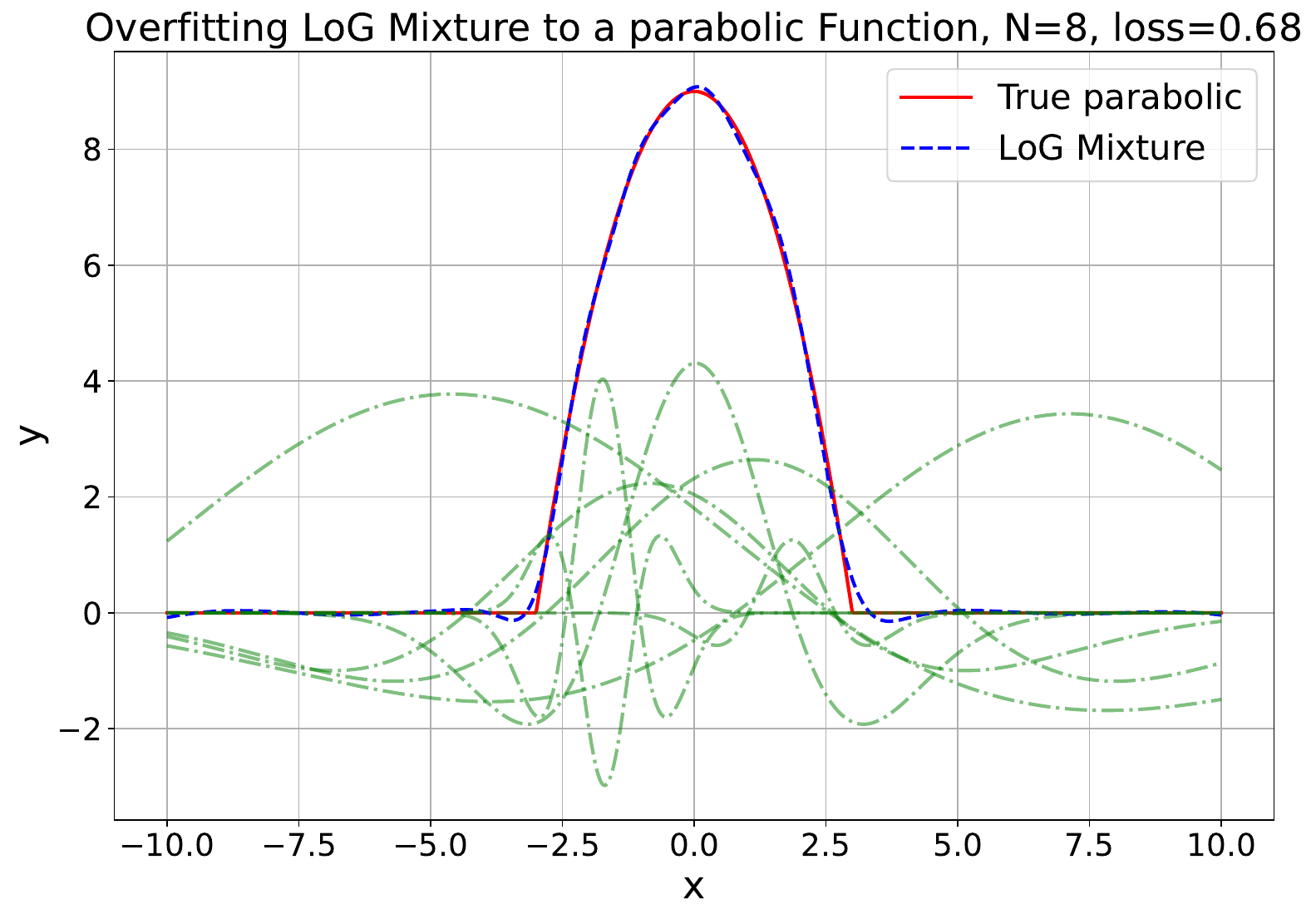} & 
    \includegraphics[width=0.24\linewidth]{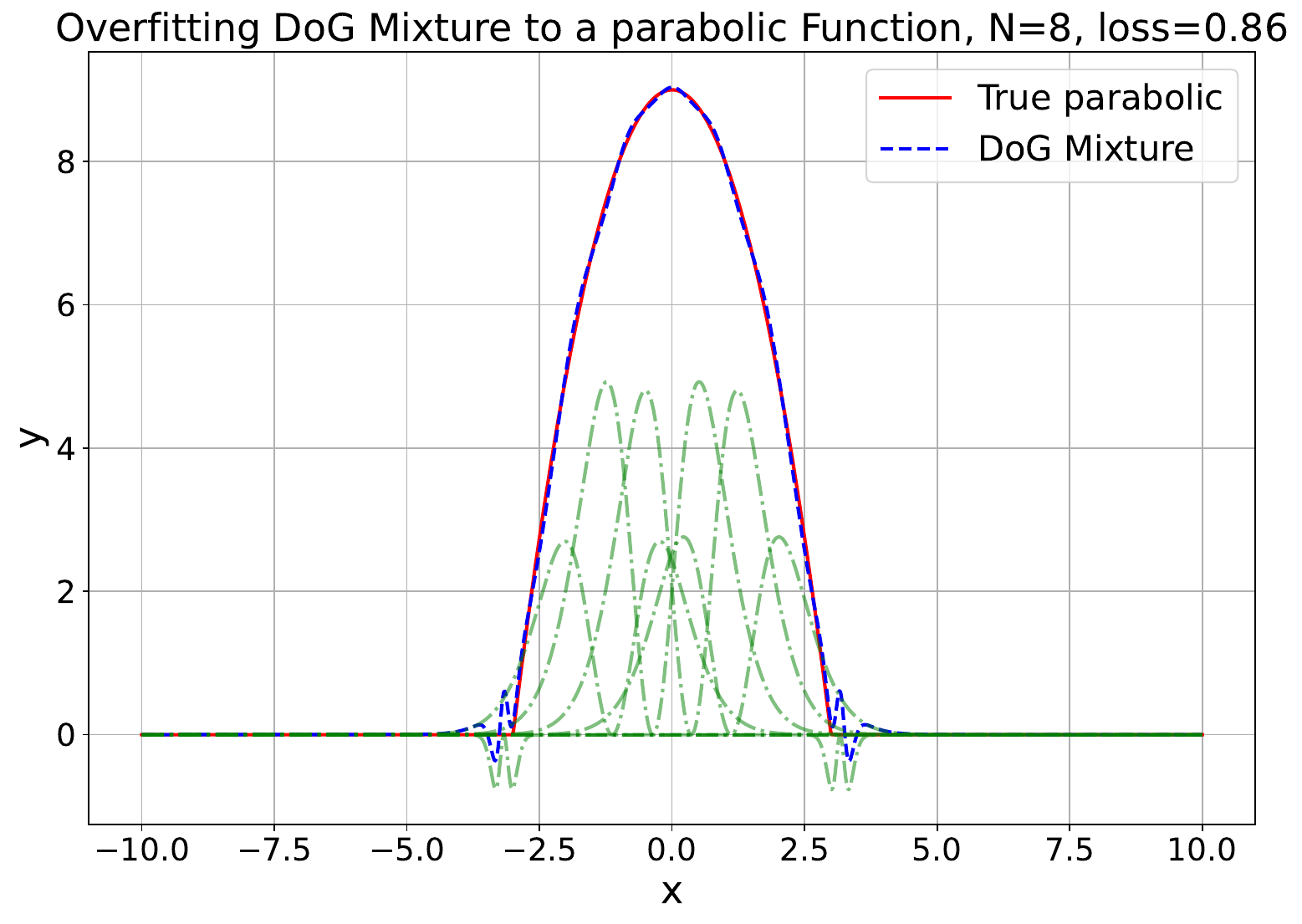} & 
    \includegraphics[width=0.24\linewidth]{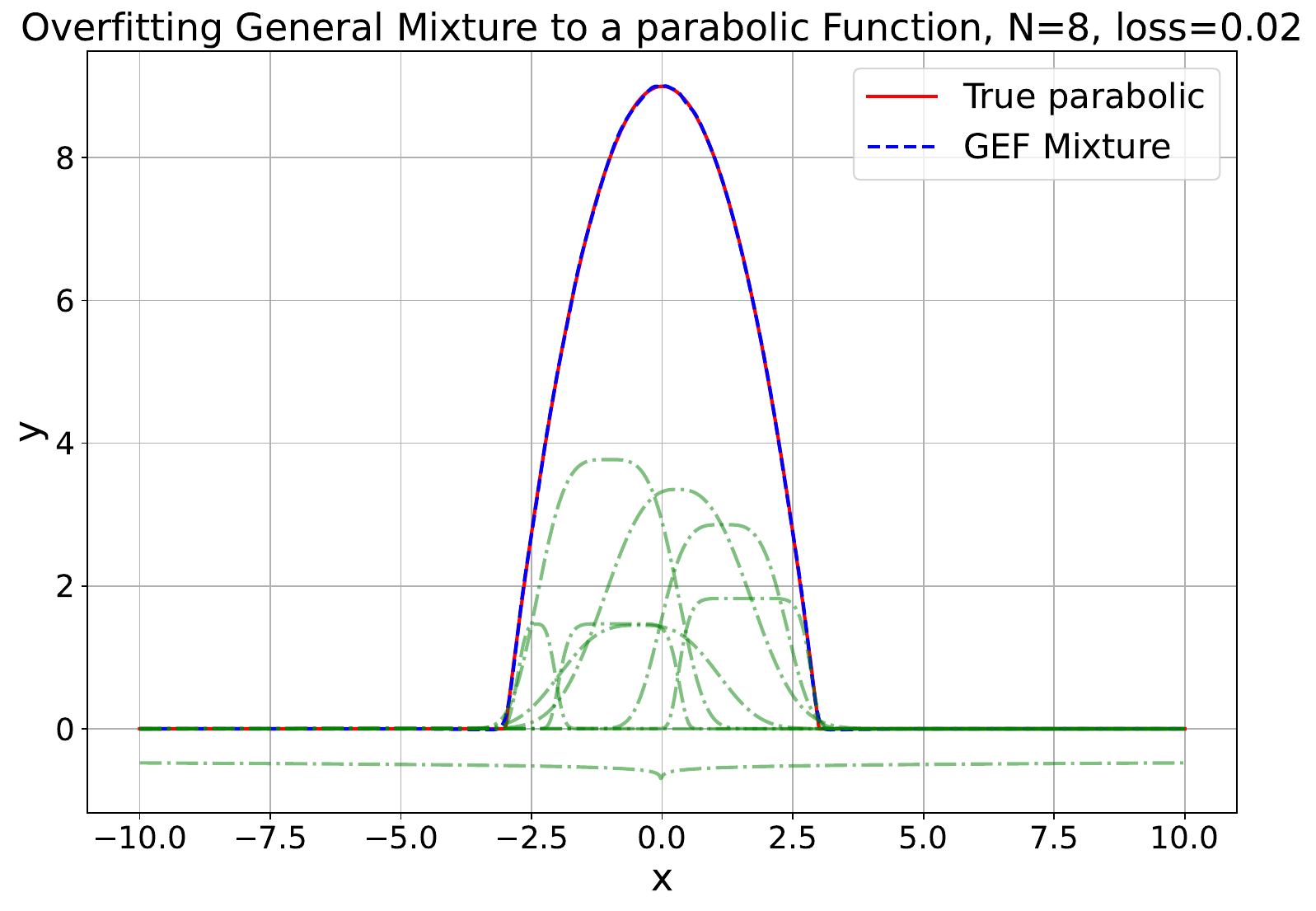}\\ 
    \includegraphics[width=0.24\linewidth]{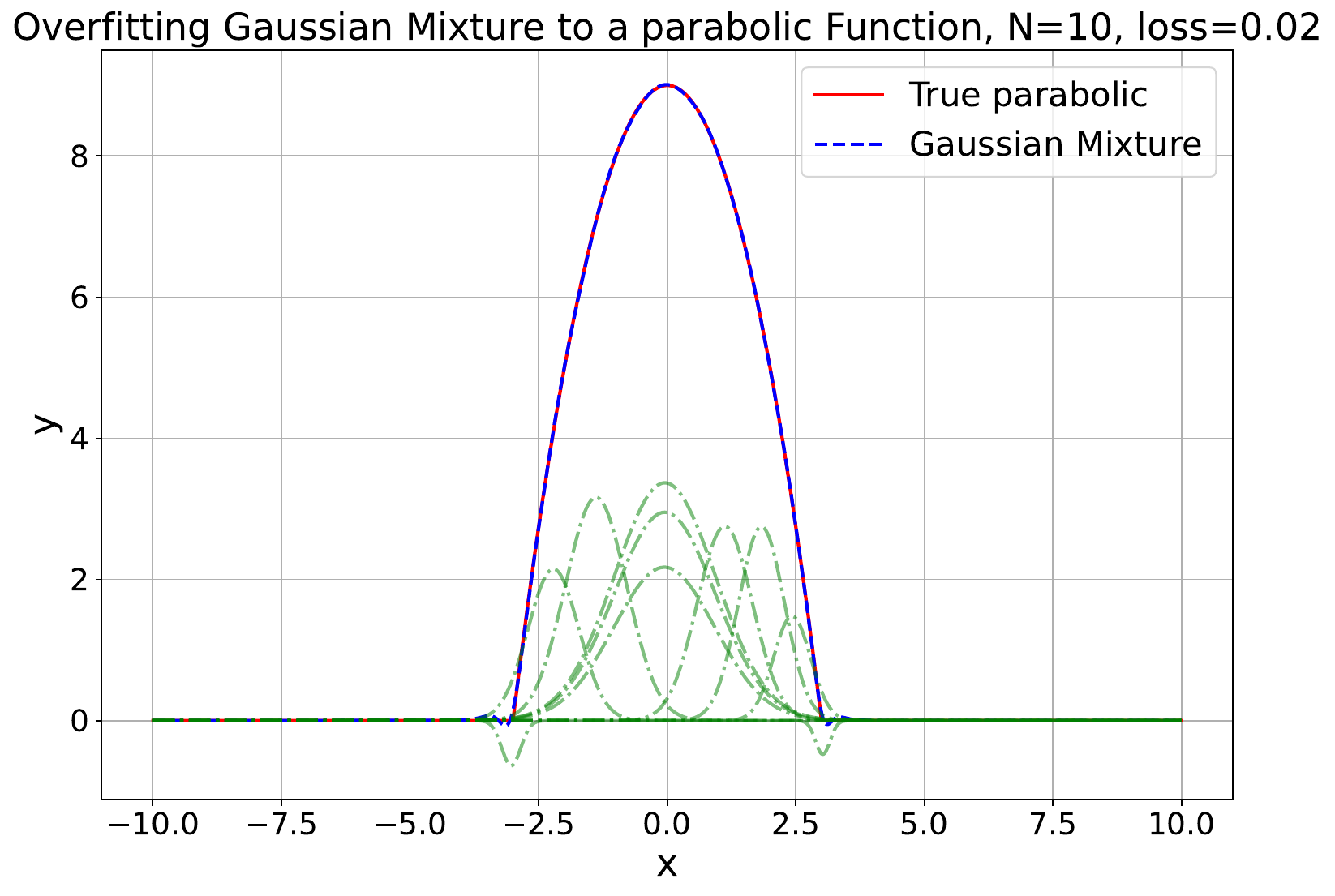} & 
    \includegraphics[width=0.24\linewidth]{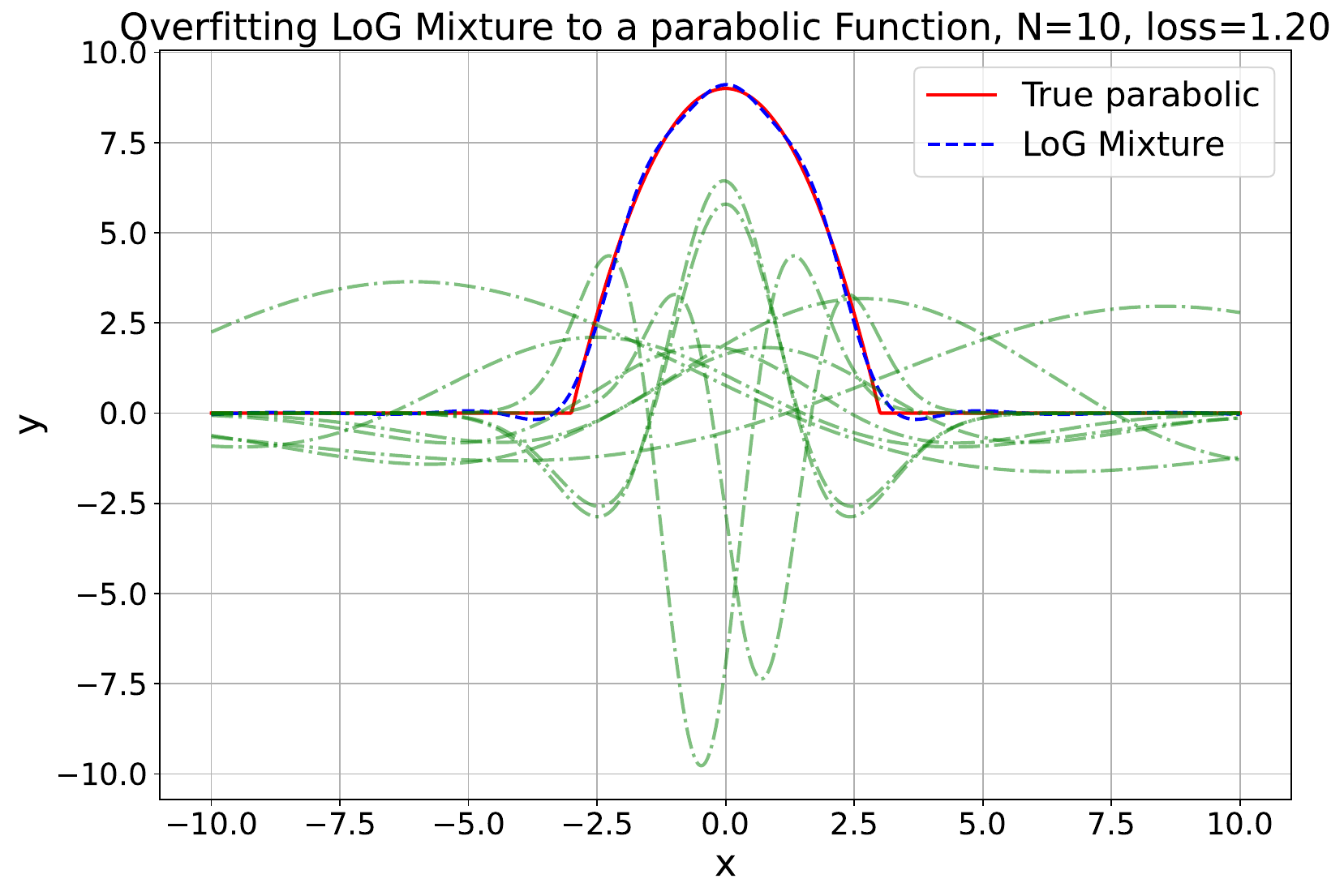} & 
    \includegraphics[width=0.24\linewidth]{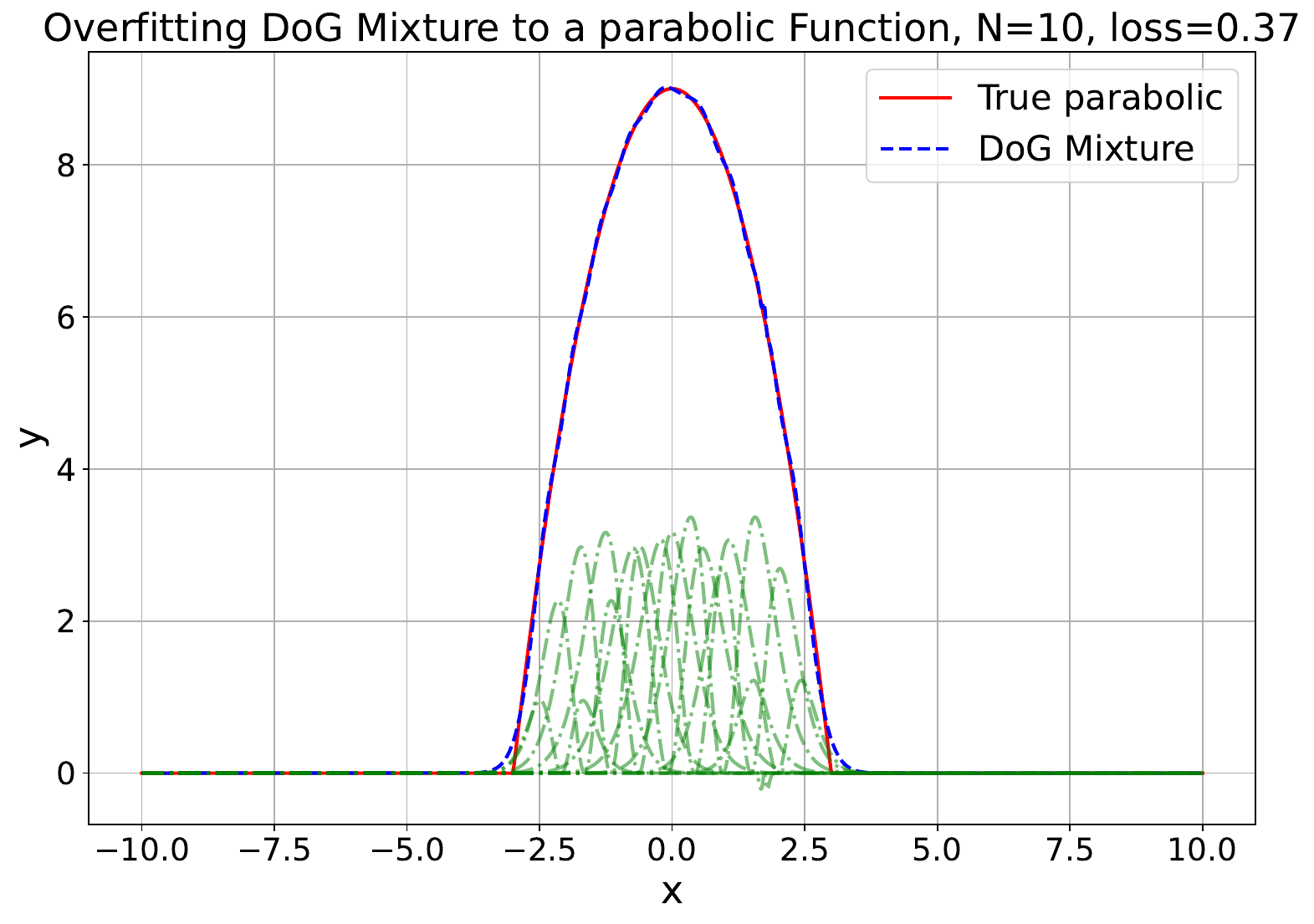} & 
    \includegraphics[width=0.24\linewidth]{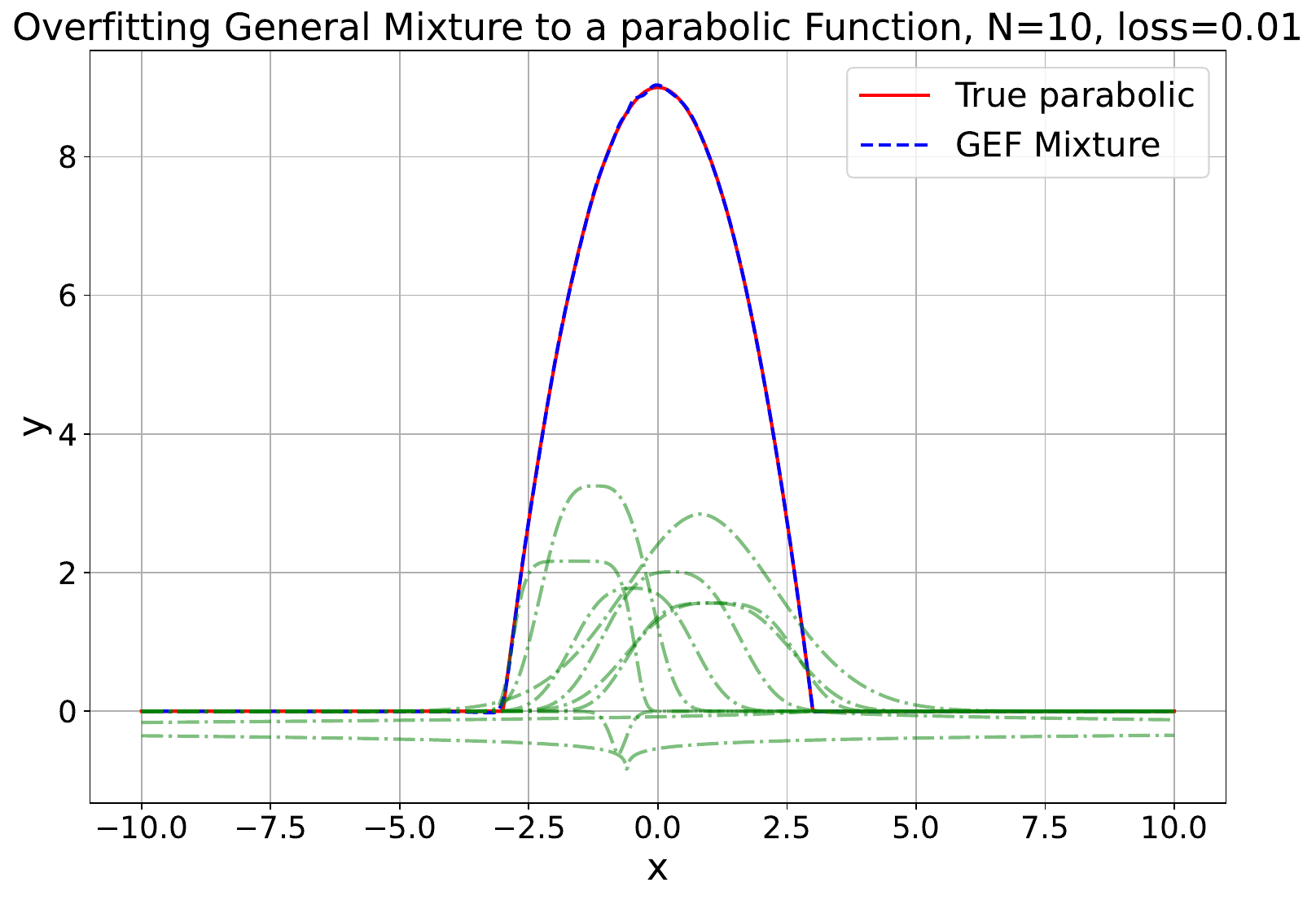}\\ 
    \includegraphics[width=0.24\linewidth]{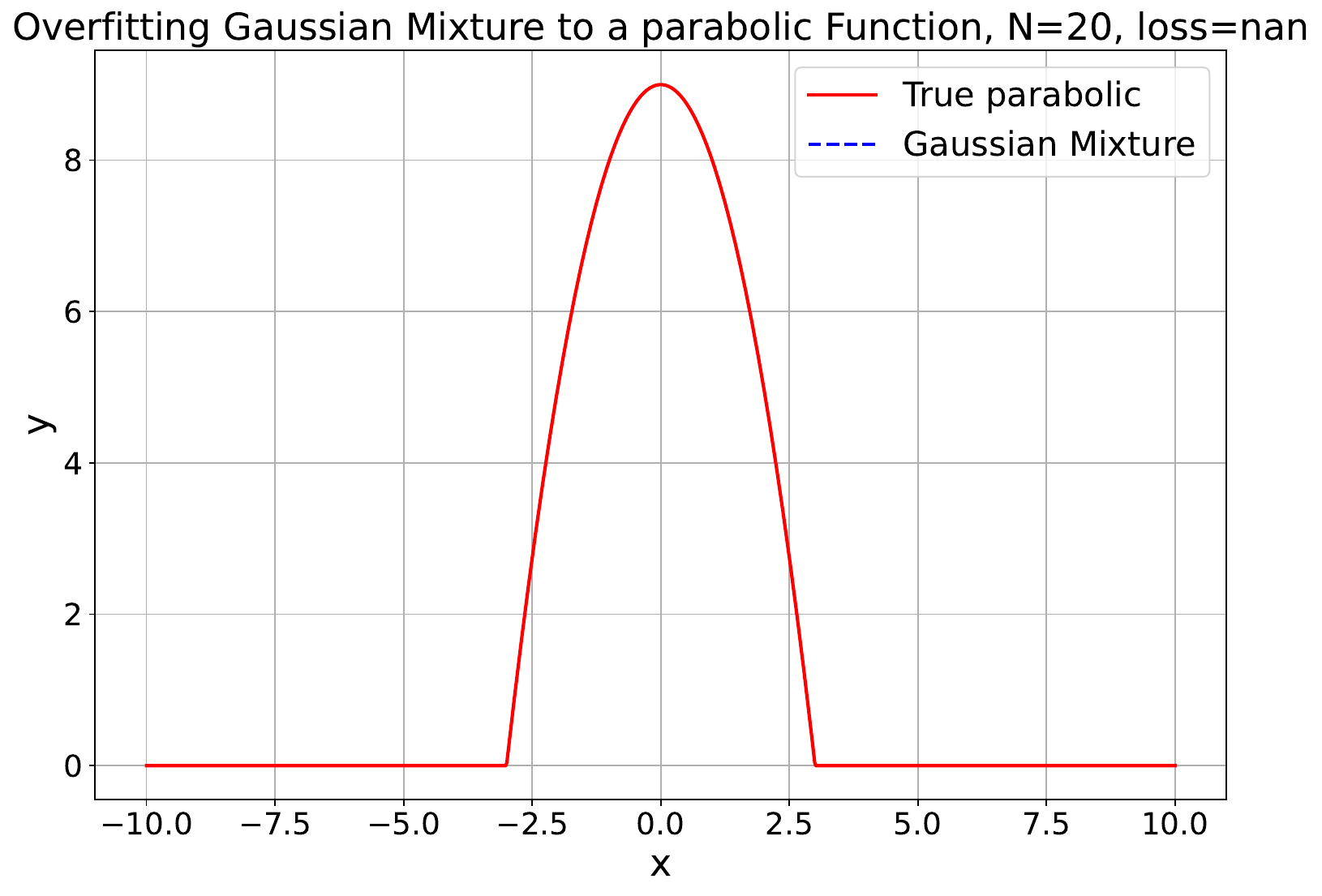} & 
    \includegraphics[width=0.24\linewidth]{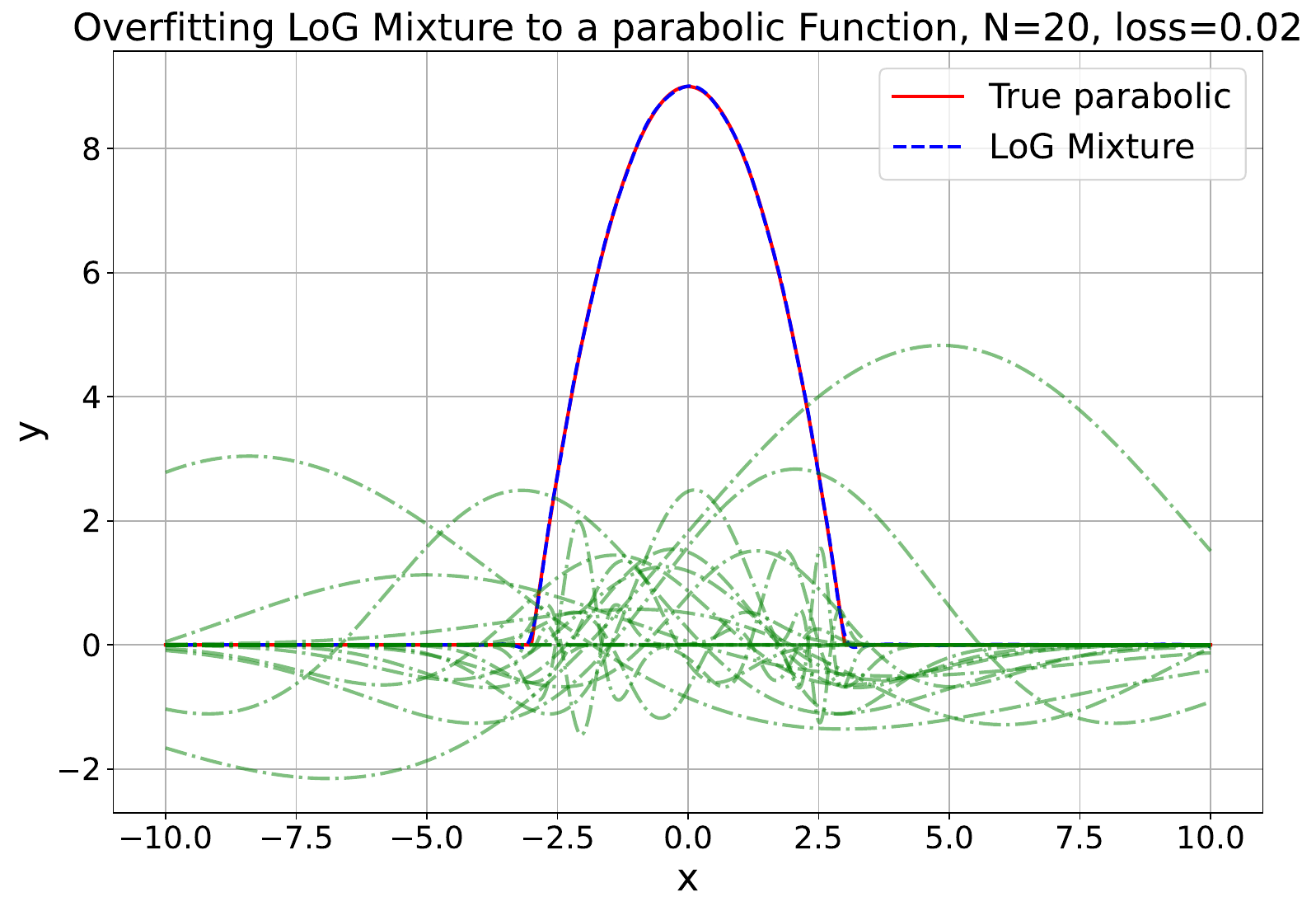} & 
    \includegraphics[width=0.24\linewidth]{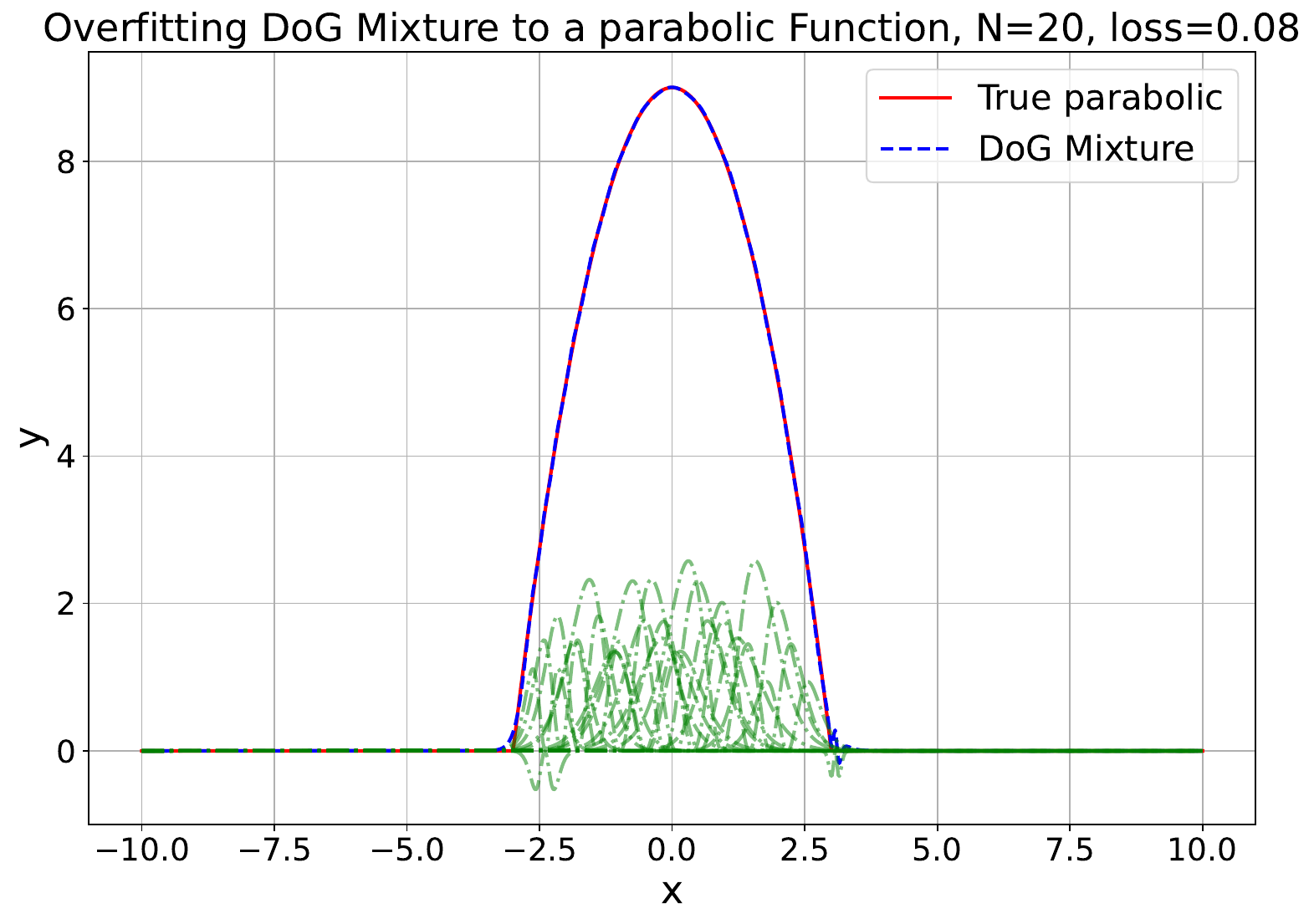} & 
    \includegraphics[width=0.24\linewidth]{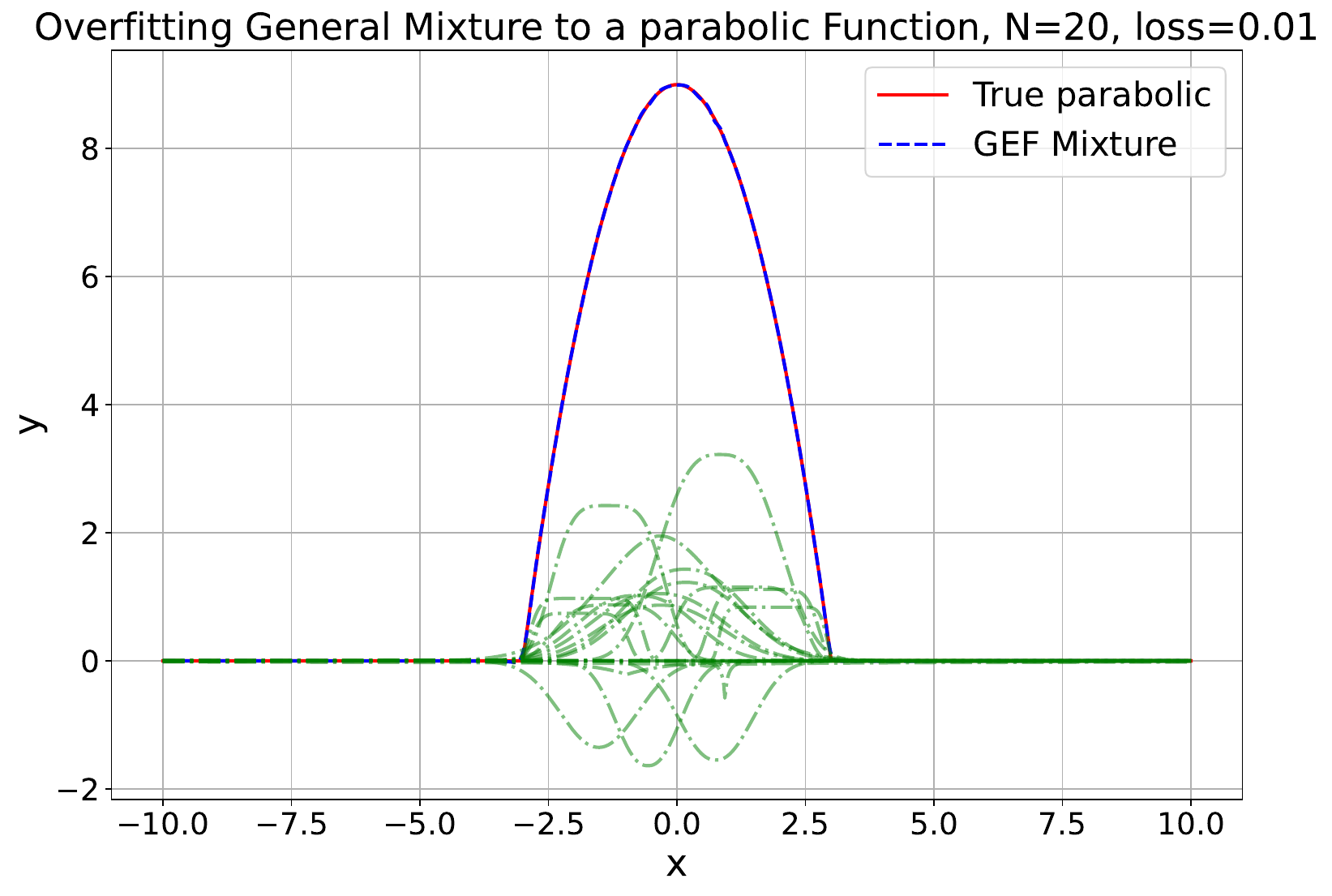}\\ 
    
    \end{tabular}
    }
    \caption{\textbf{Numerical Simulation Examples of Fitting Parabolics with Real Weights Mixtures ( N= 2, 5, 8, and 10 )}. We show some fitting examples for parabolic signals with Real weights mixtures (can be negative). The four mixtures used from left to right are Gaussians, LoG, DoG, and General mixtures. From top to bottom: N = 2, 8, and 10 components. The optimized individual components are shown in green. Some examples fail to optimize due to numerical instability in both Gaussians and GEF mixtures. Note that GEF is very efficient in fitting the parabolic with few components while LoG and DoG are more stable for a larger number of components. }
    \label{supfig:fitting_parabolic_N}
    \end{figure*}
    

%% file: figures/fitting/fitting_exponential_p.tex
\begin{figure*}[h]
    \centering
    \resizebox{1.0\linewidth}{!}{
    \begin{tabular}{cccc}
    \tabcolsep=0.01cm
    Gaussian Mixture& LoG Mixture & DoG Mixture & GEF Mixture \\ 
    \includegraphics[width=0.24\linewidth]{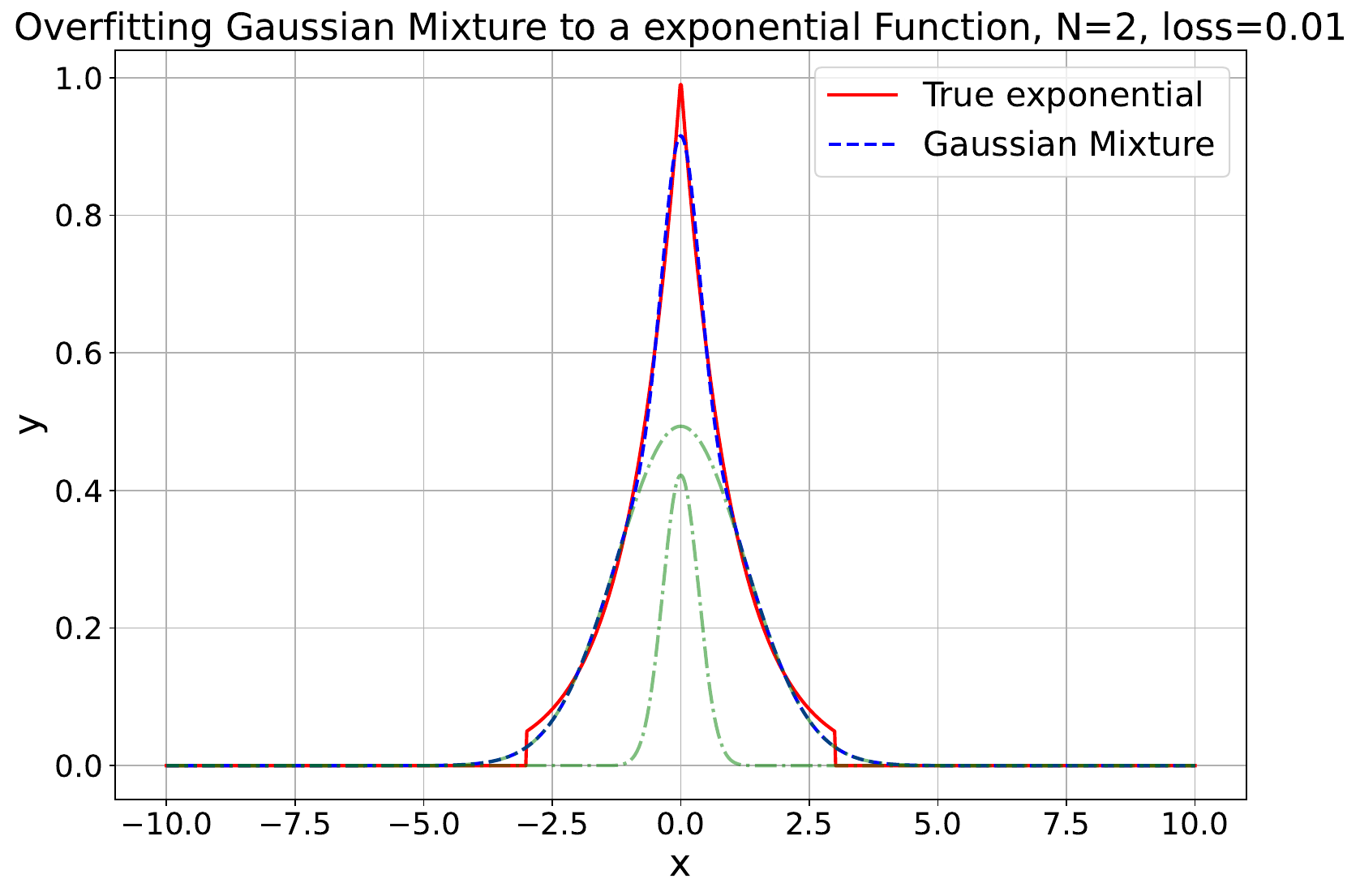} & 
    \includegraphics[width=0.24\linewidth]{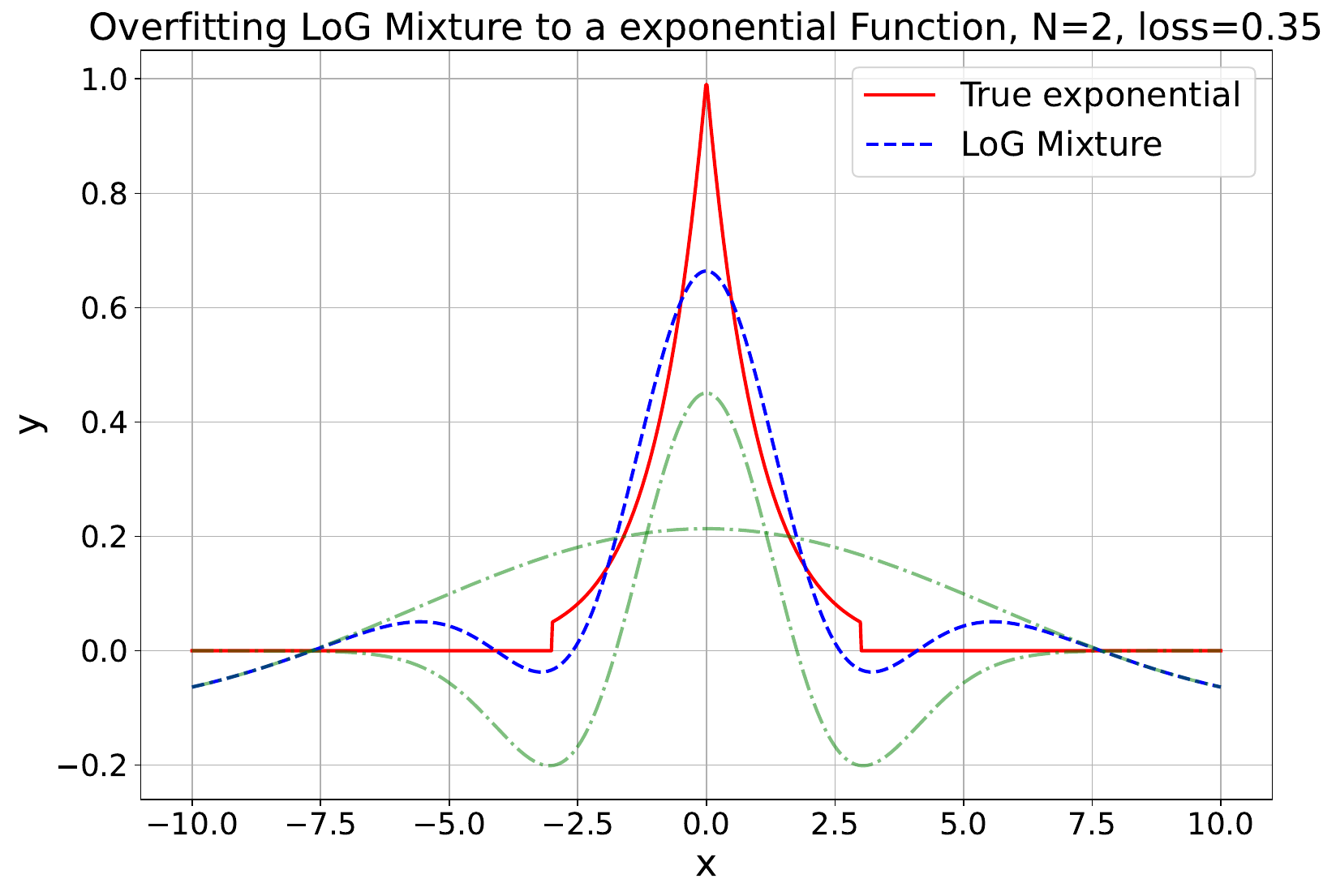} & 
    \includegraphics[width=0.24\linewidth]{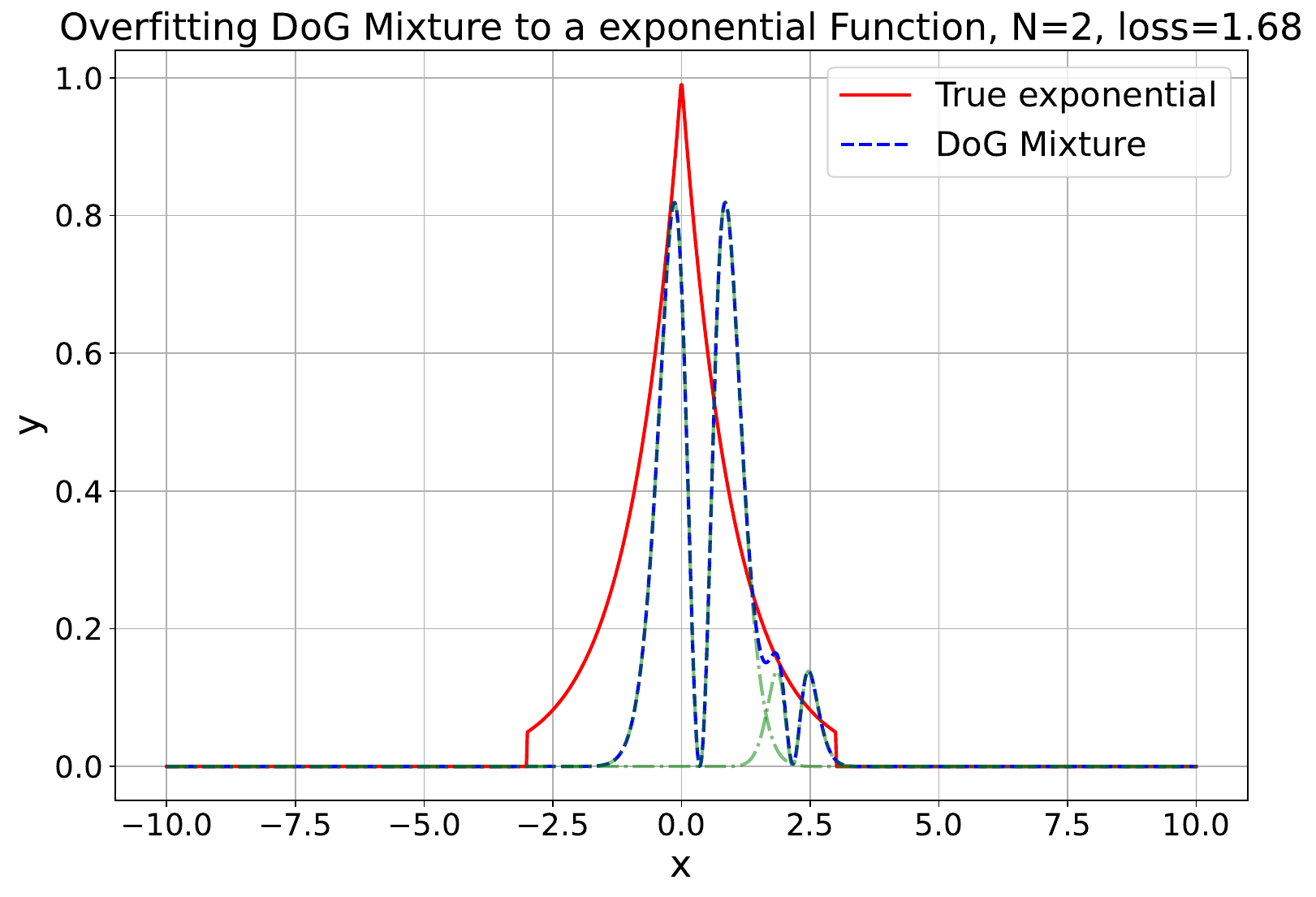} & 
    \includegraphics[width=0.24\linewidth]{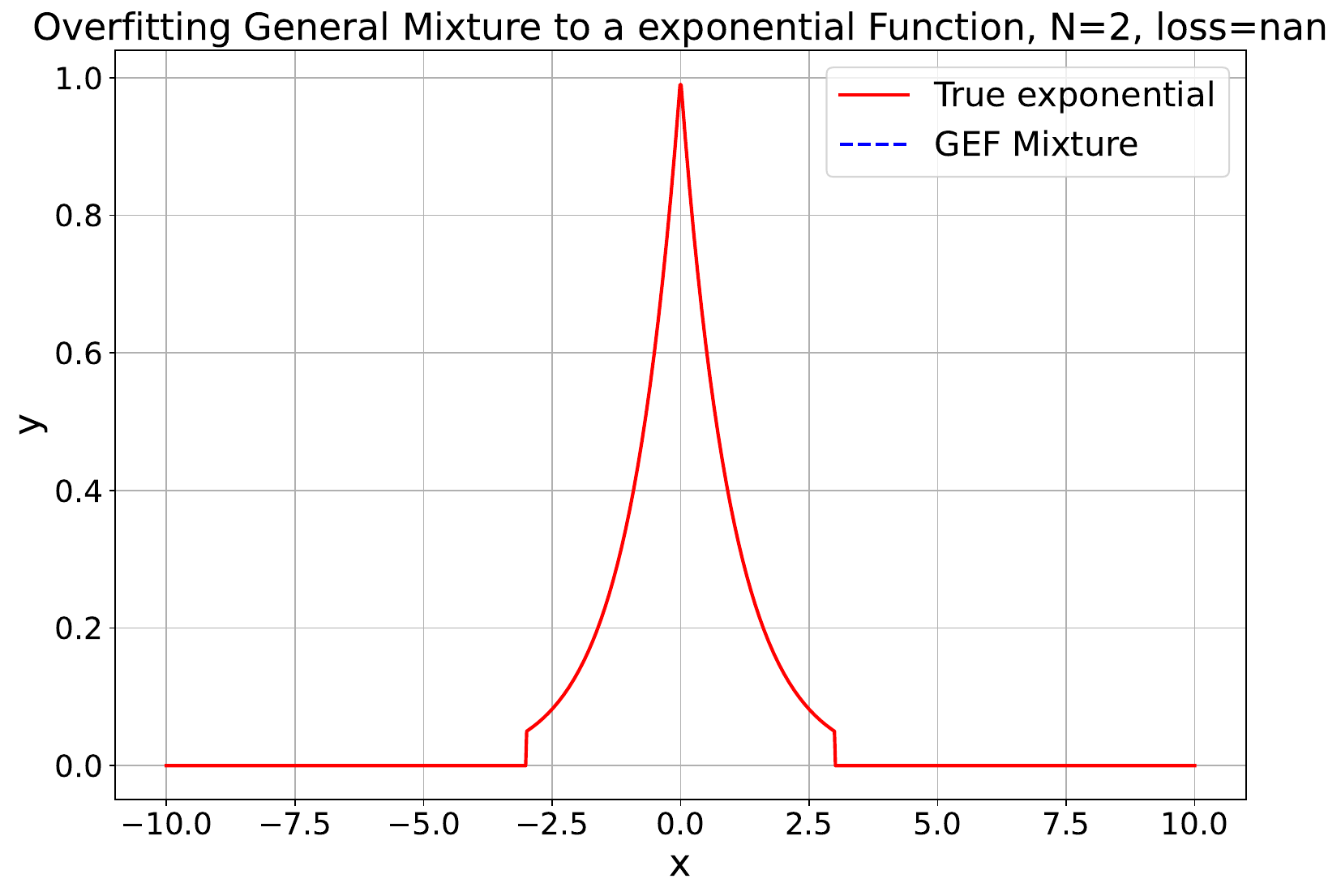}\\ 
    \includegraphics[width=0.24\linewidth]{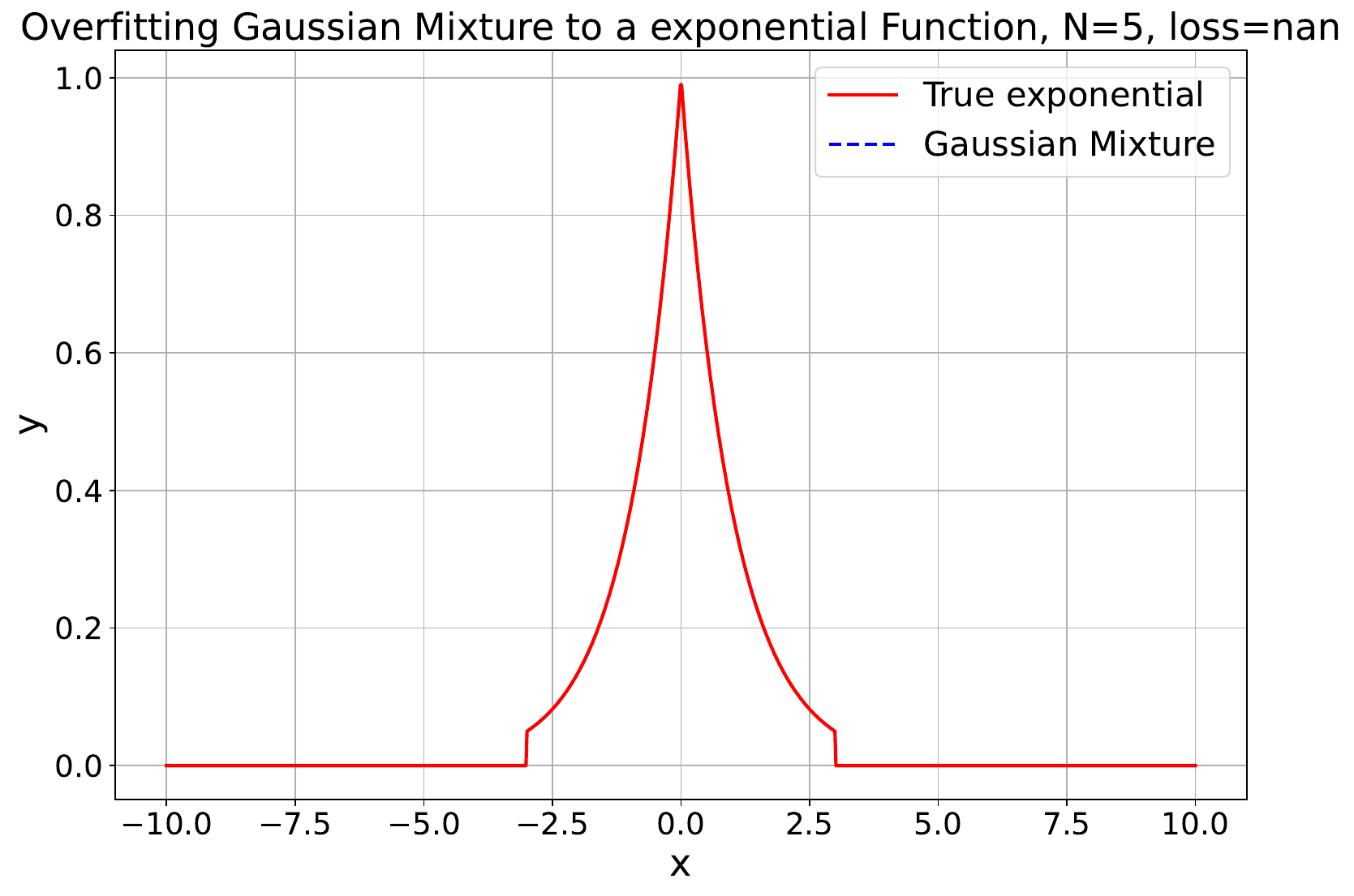} & 
    \includegraphics[width=0.24\linewidth]{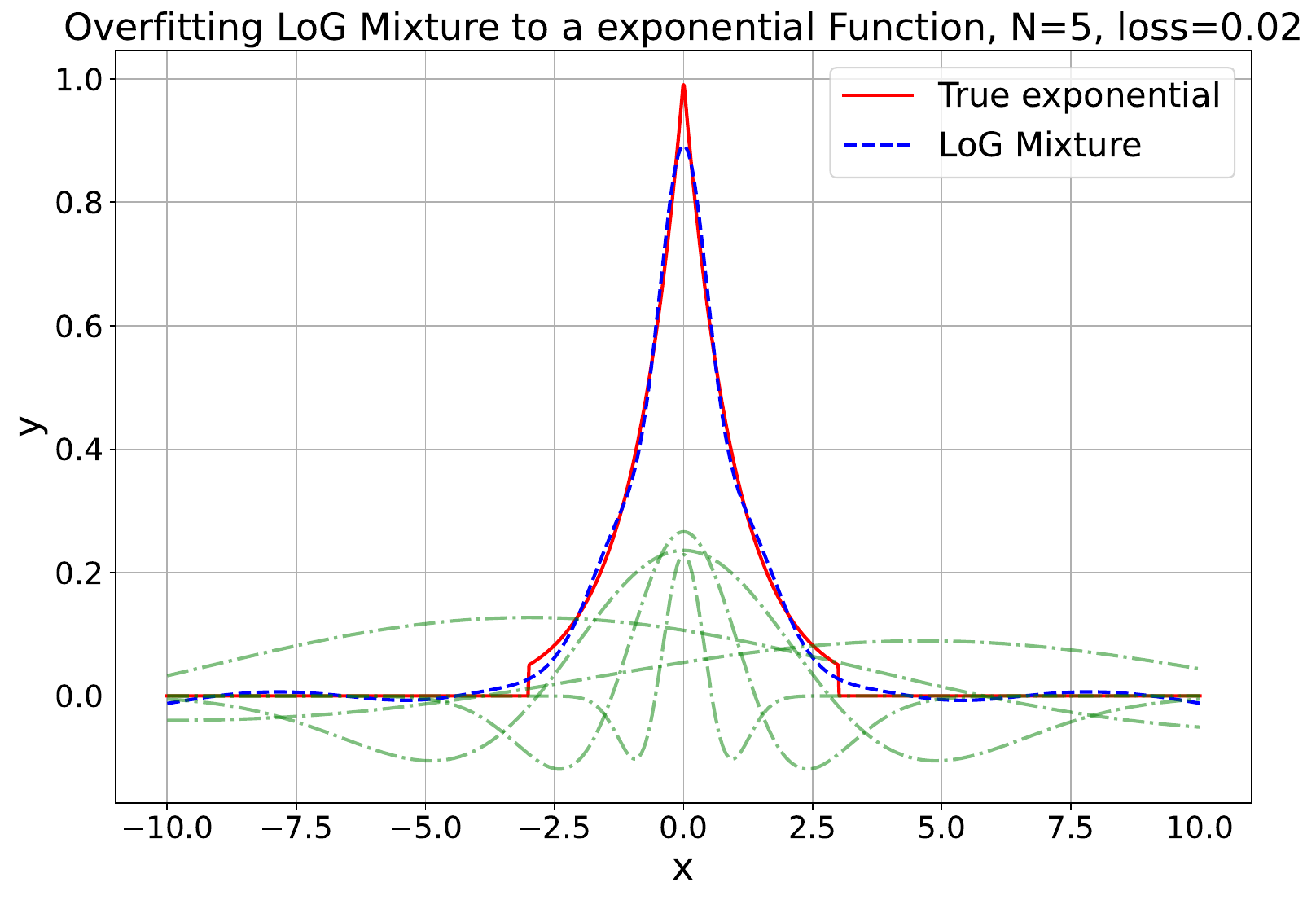} & 
    \includegraphics[width=0.24\linewidth]{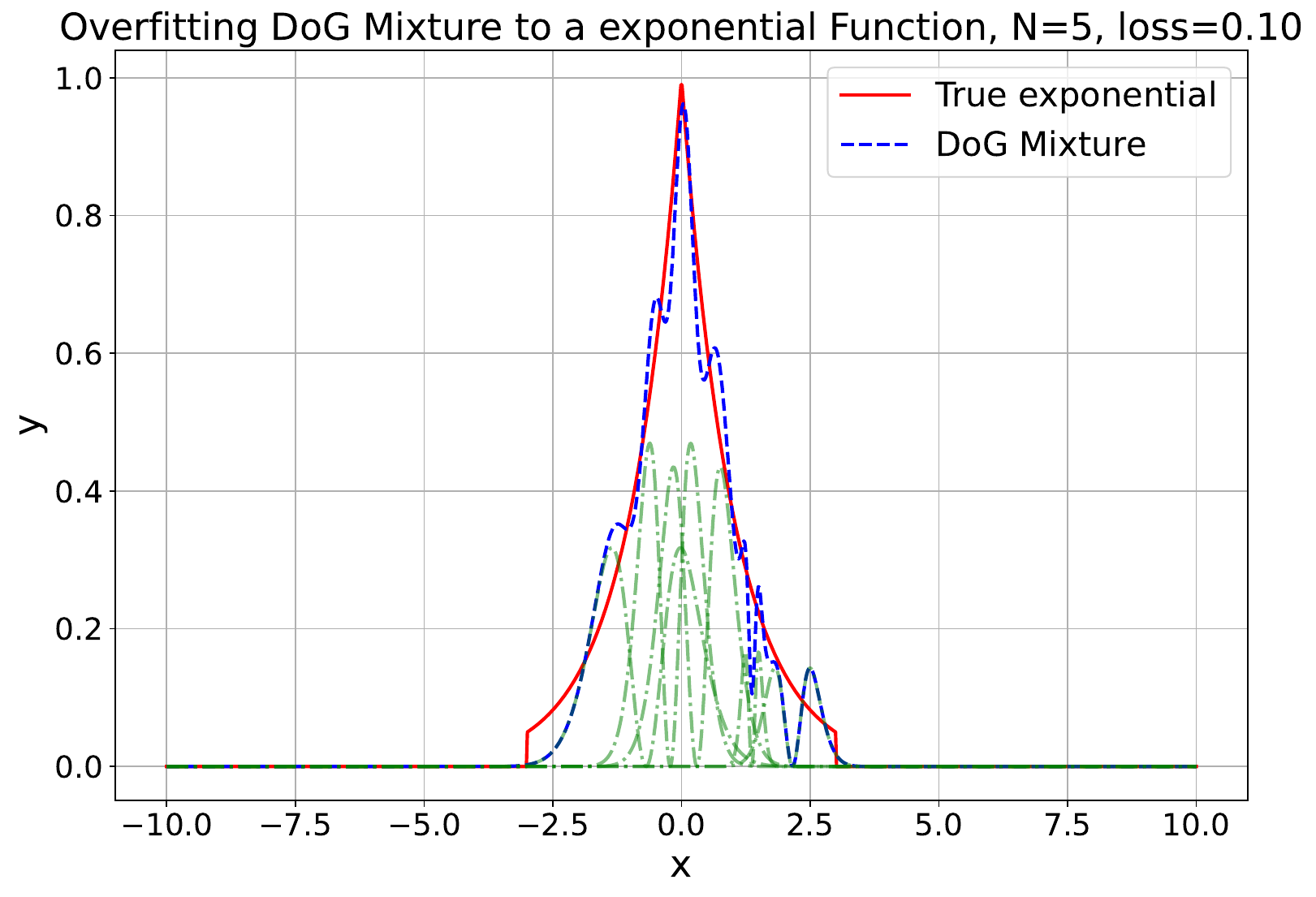} & 
    \includegraphics[width=0.24\linewidth]{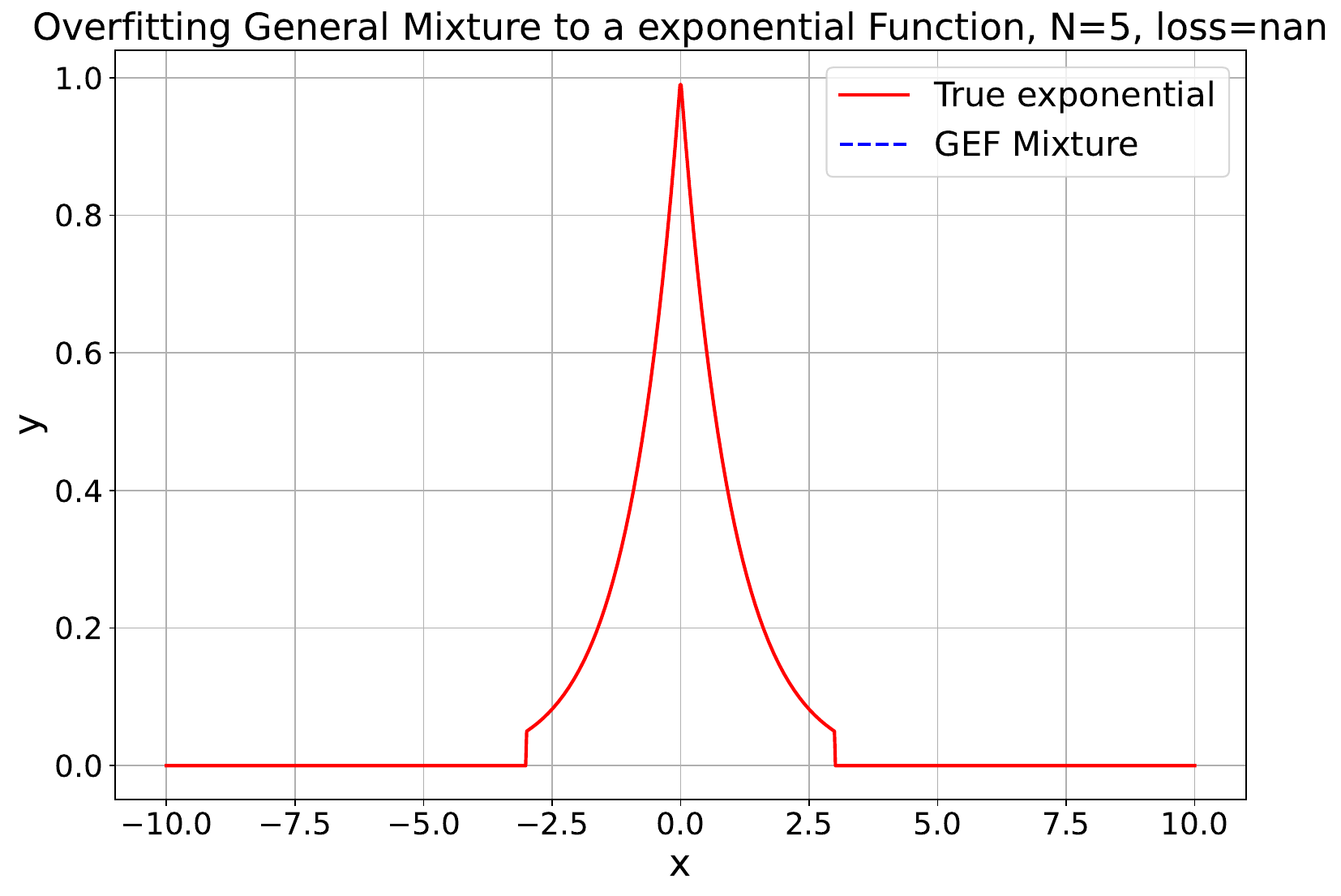}\\ 
    \includegraphics[width=0.24\linewidth]{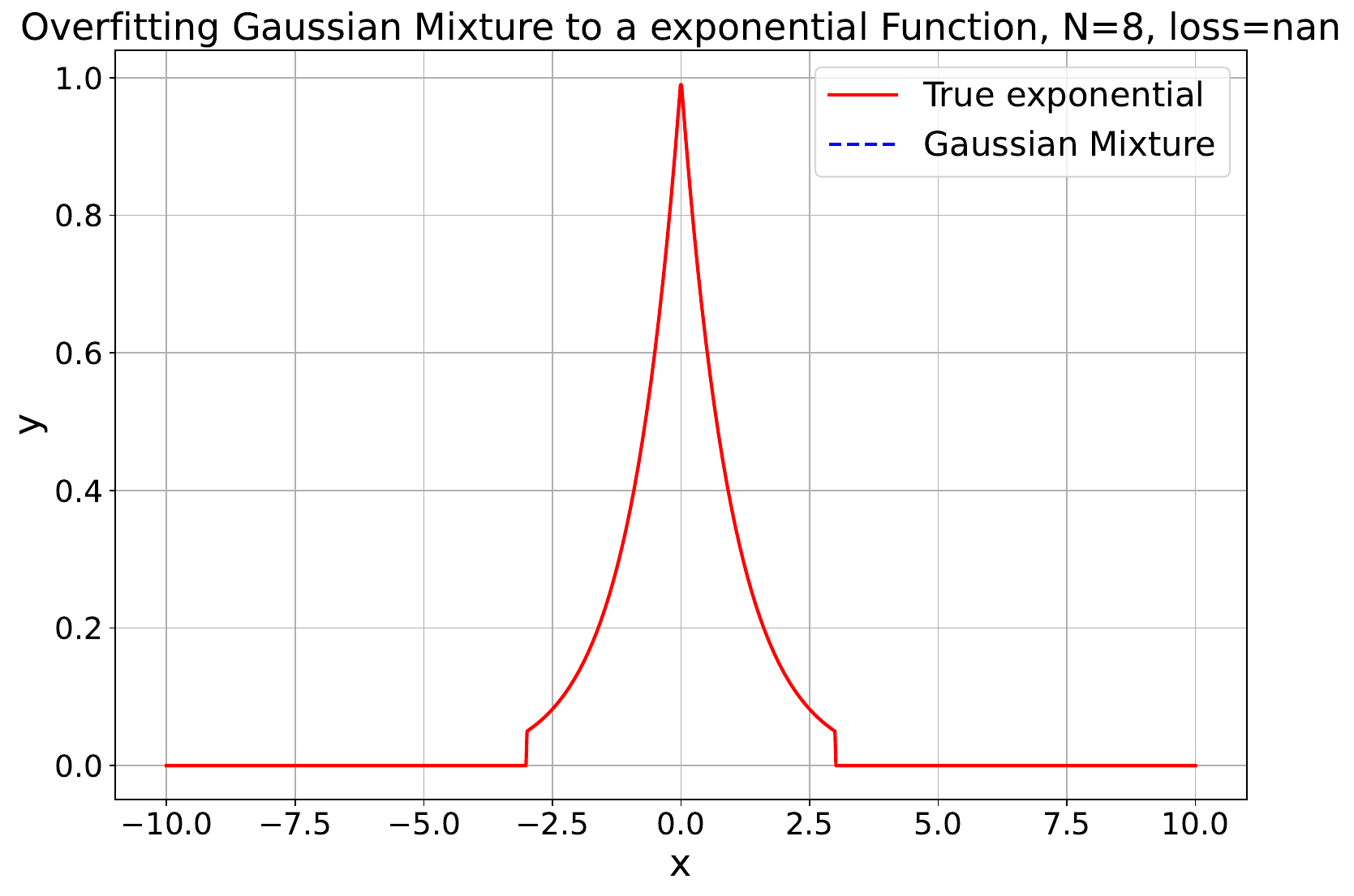} & 
    \includegraphics[width=0.24\linewidth]{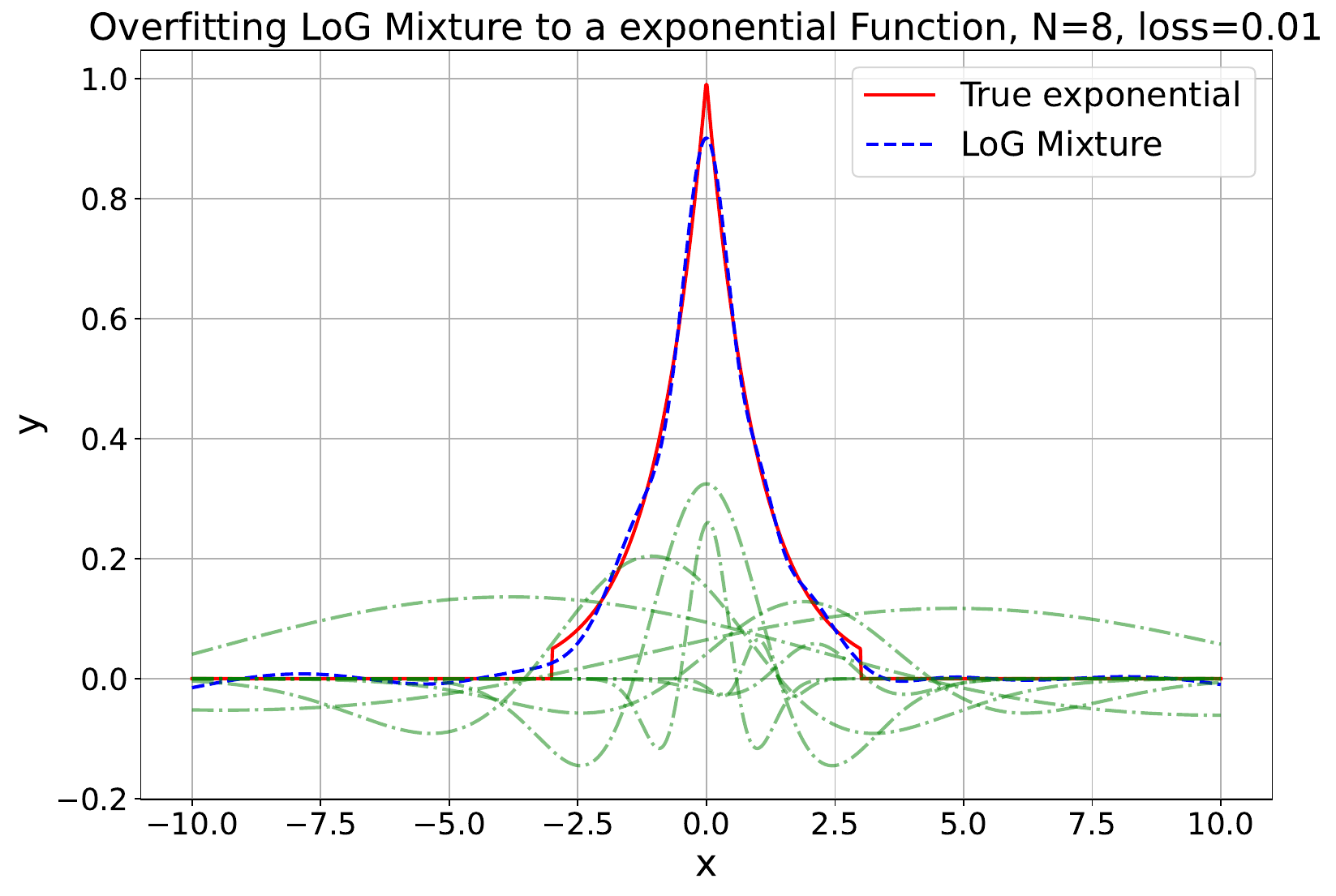} & 
    \includegraphics[width=0.24\linewidth]{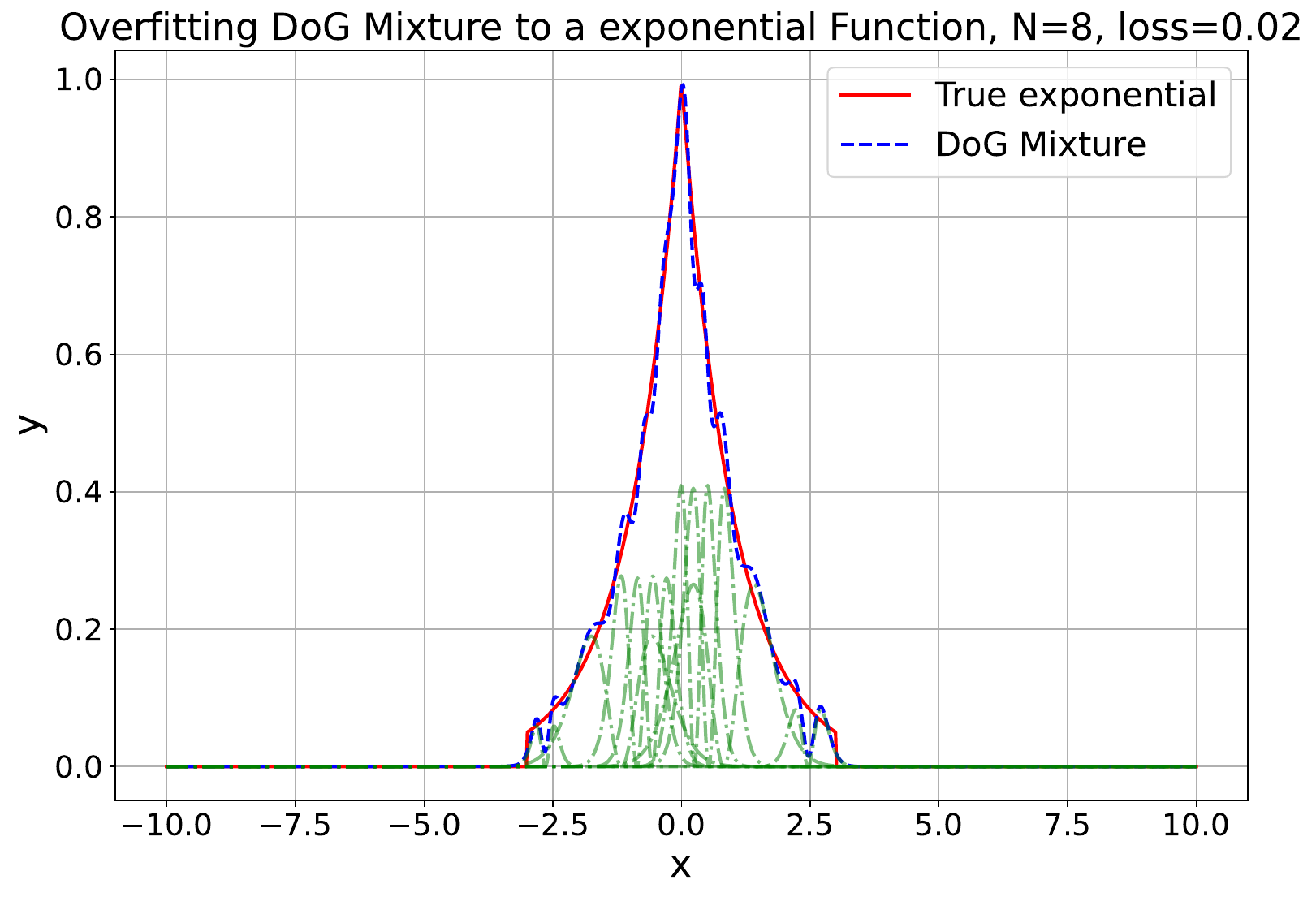} & 
    \includegraphics[width=0.24\linewidth]{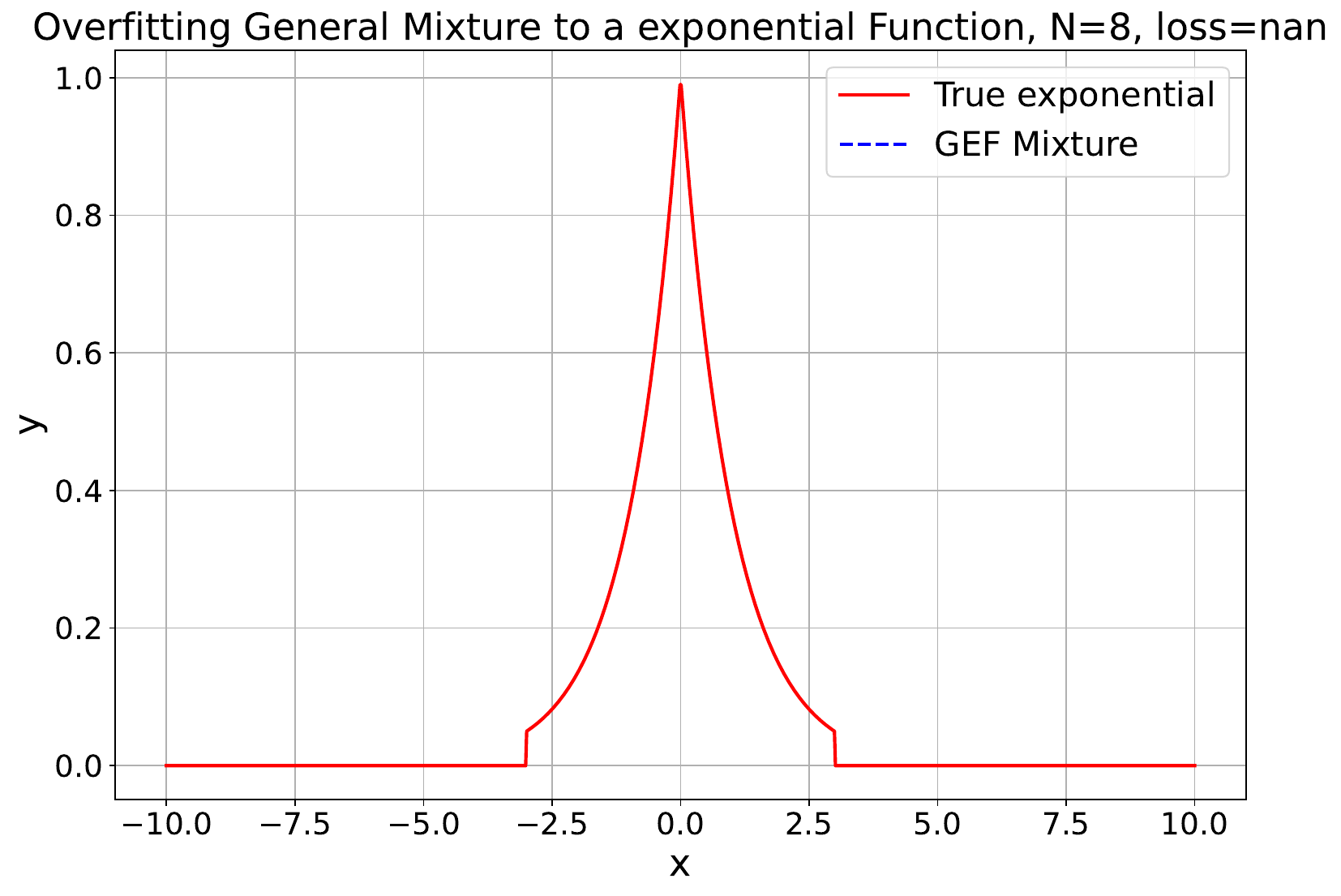}\\ 
    \includegraphics[width=0.24\linewidth]{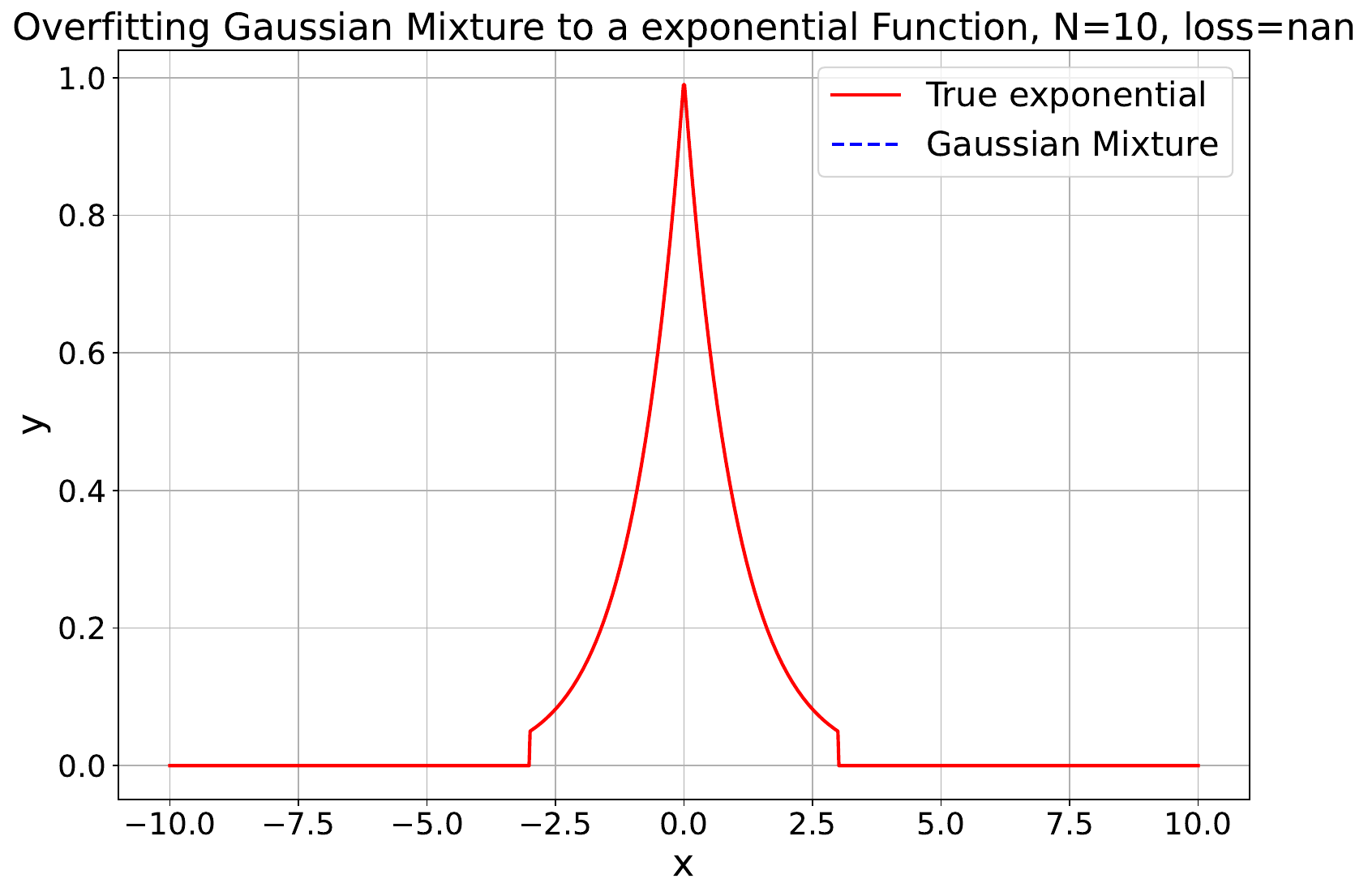} & 
    \includegraphics[width=0.24\linewidth]{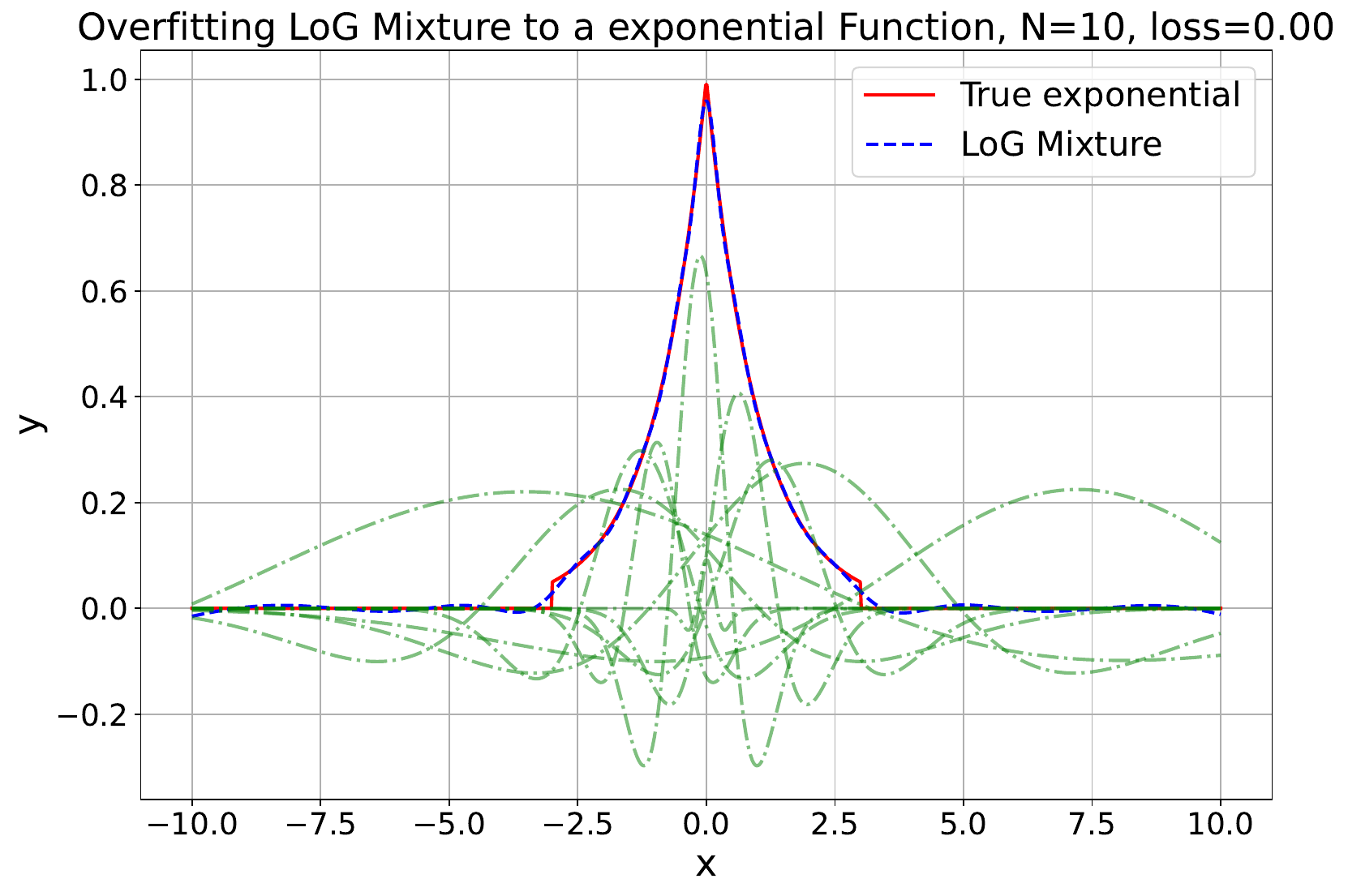} & 
    \includegraphics[width=0.24\linewidth]{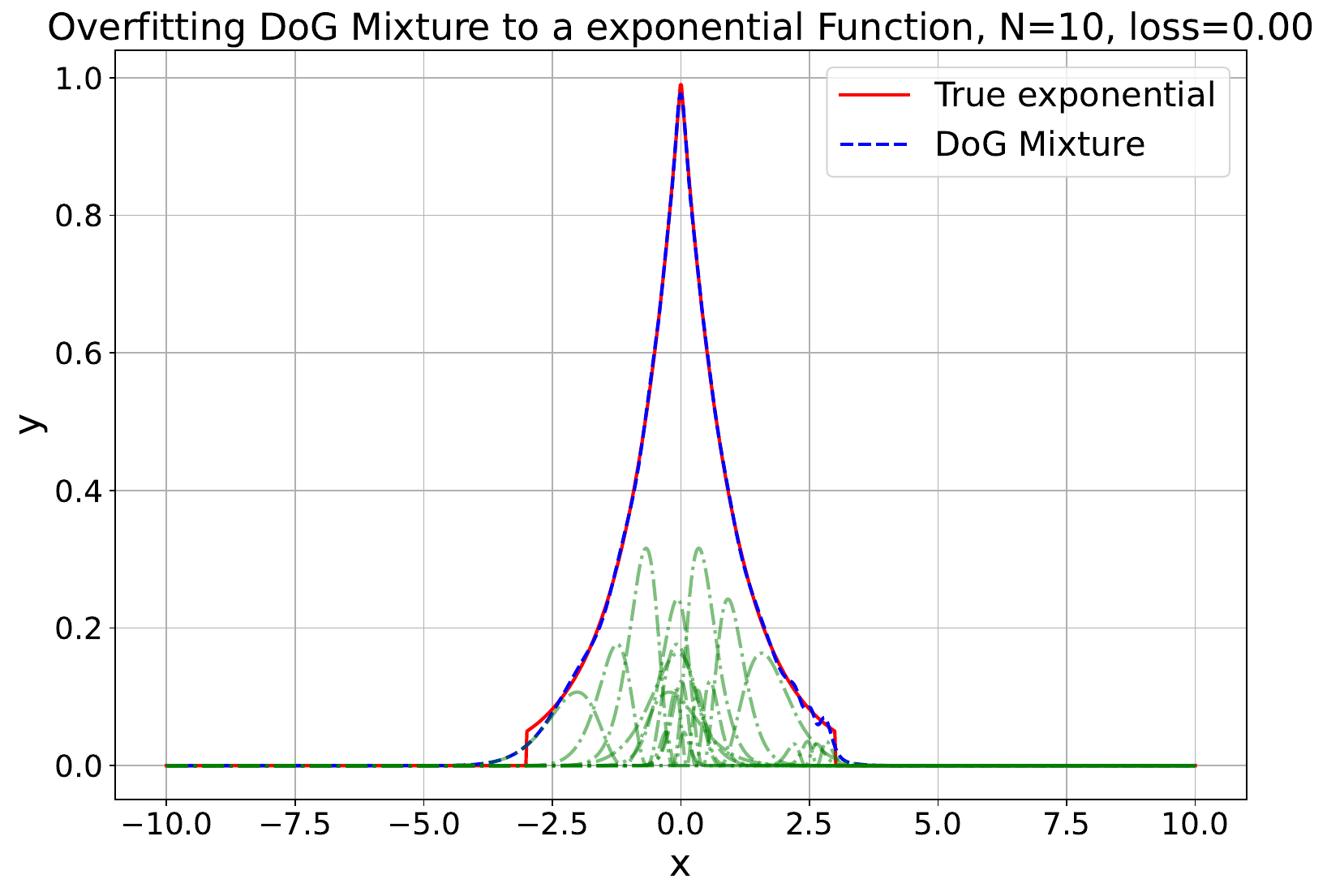} & 
    \includegraphics[width=0.24\linewidth]{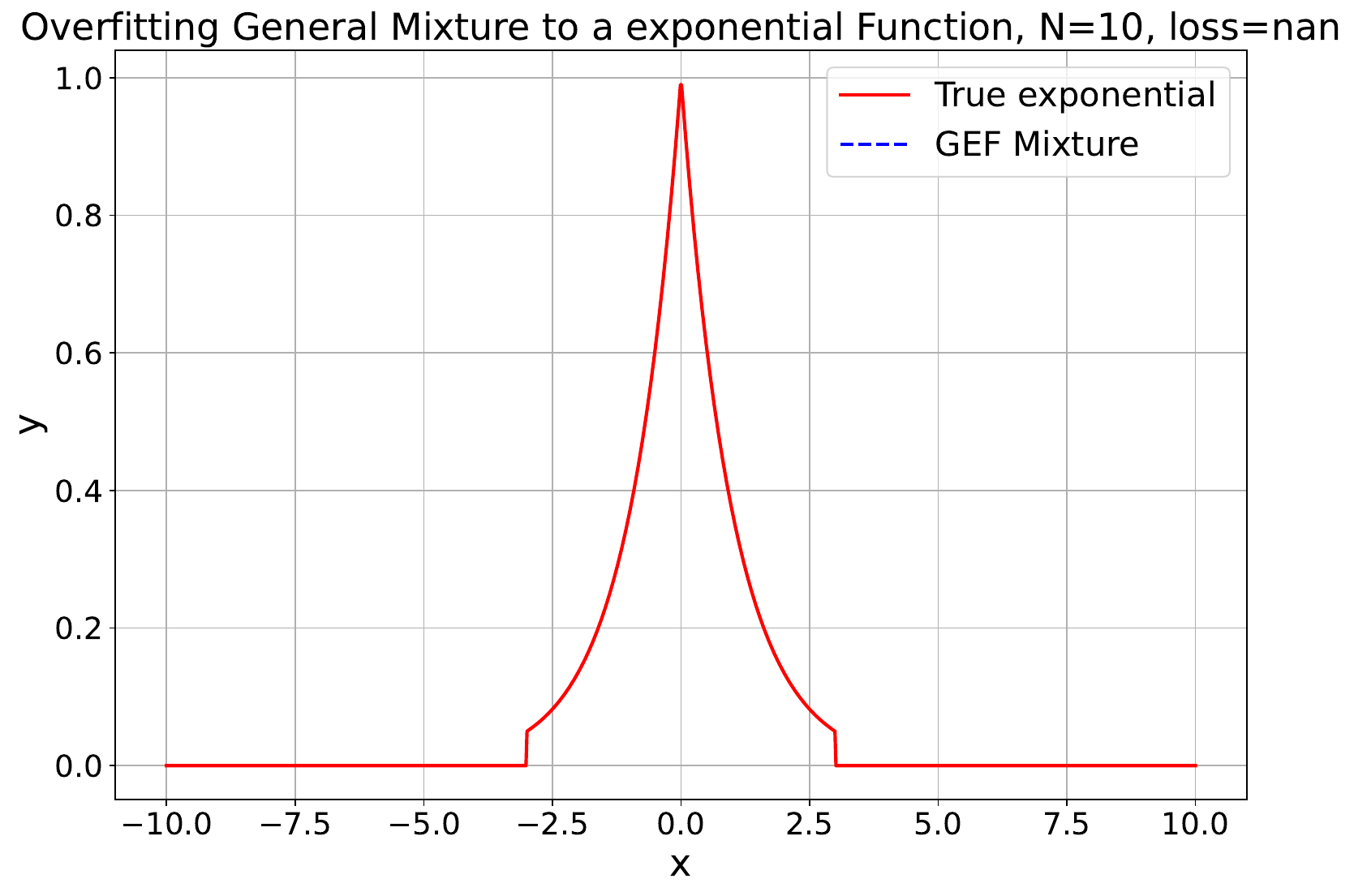}\\ 
    \includegraphics[width=0.24\linewidth]{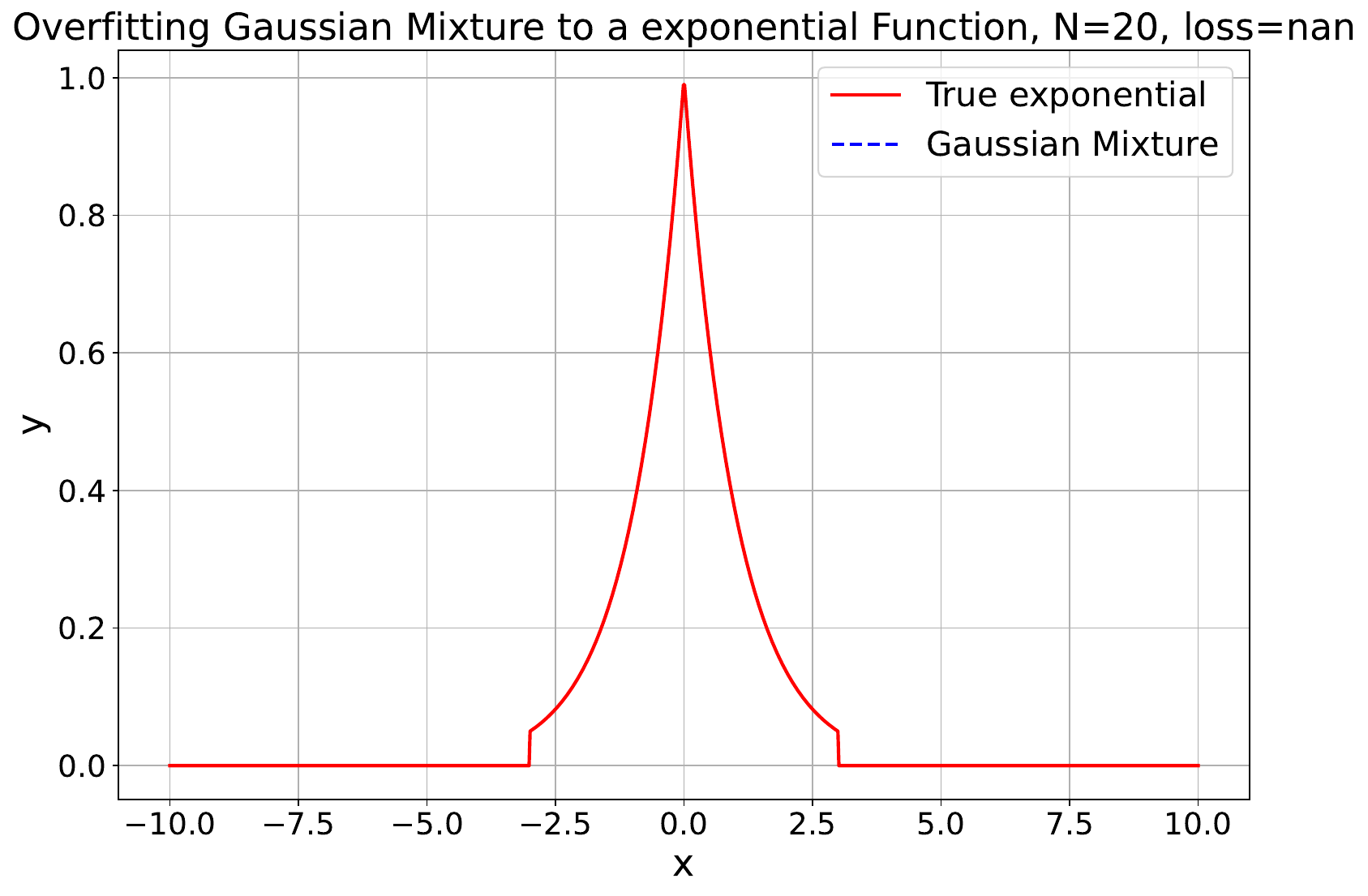} & 
    \includegraphics[width=0.24\linewidth]{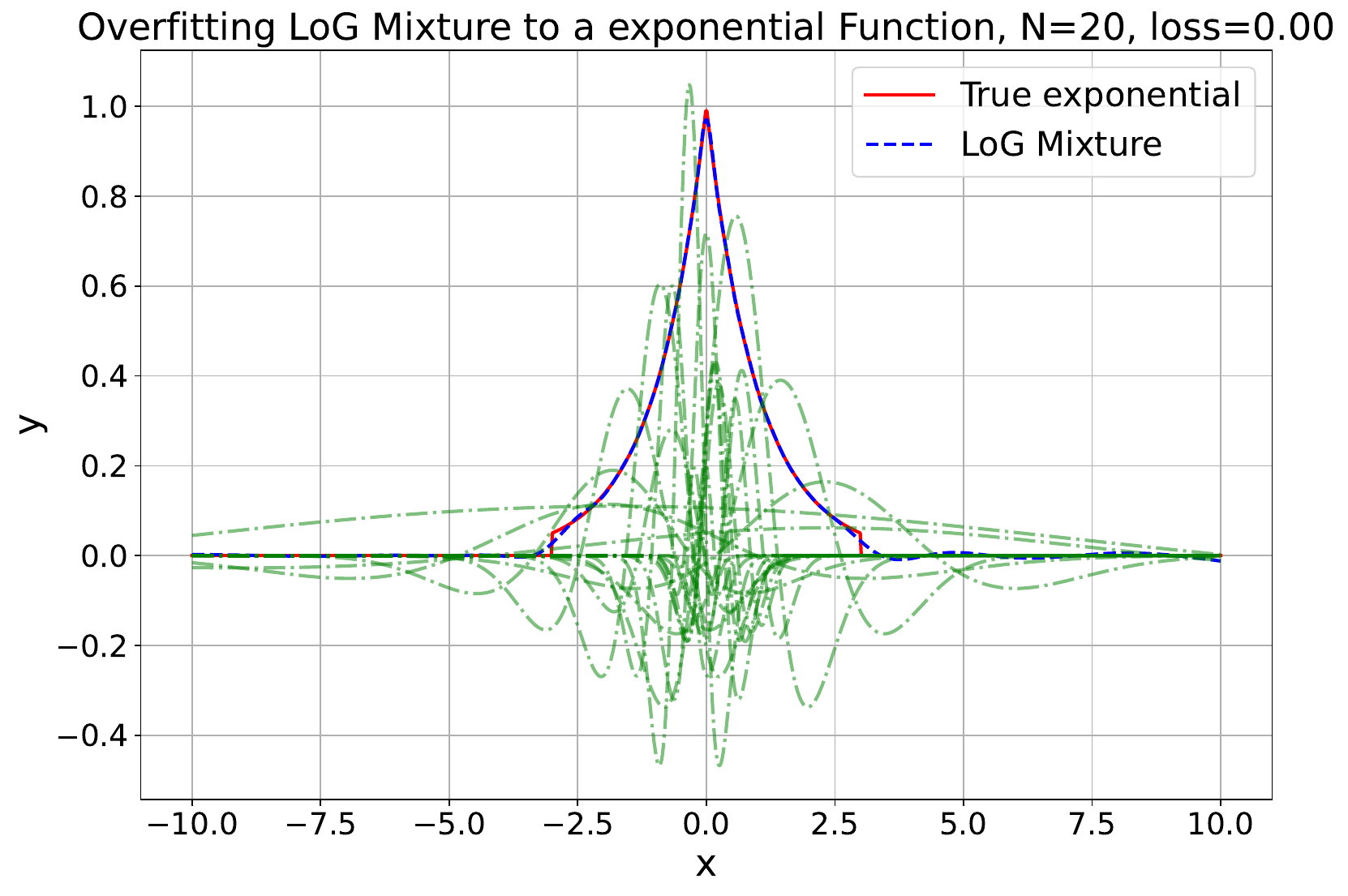} & 
    \includegraphics[width=0.24\linewidth]{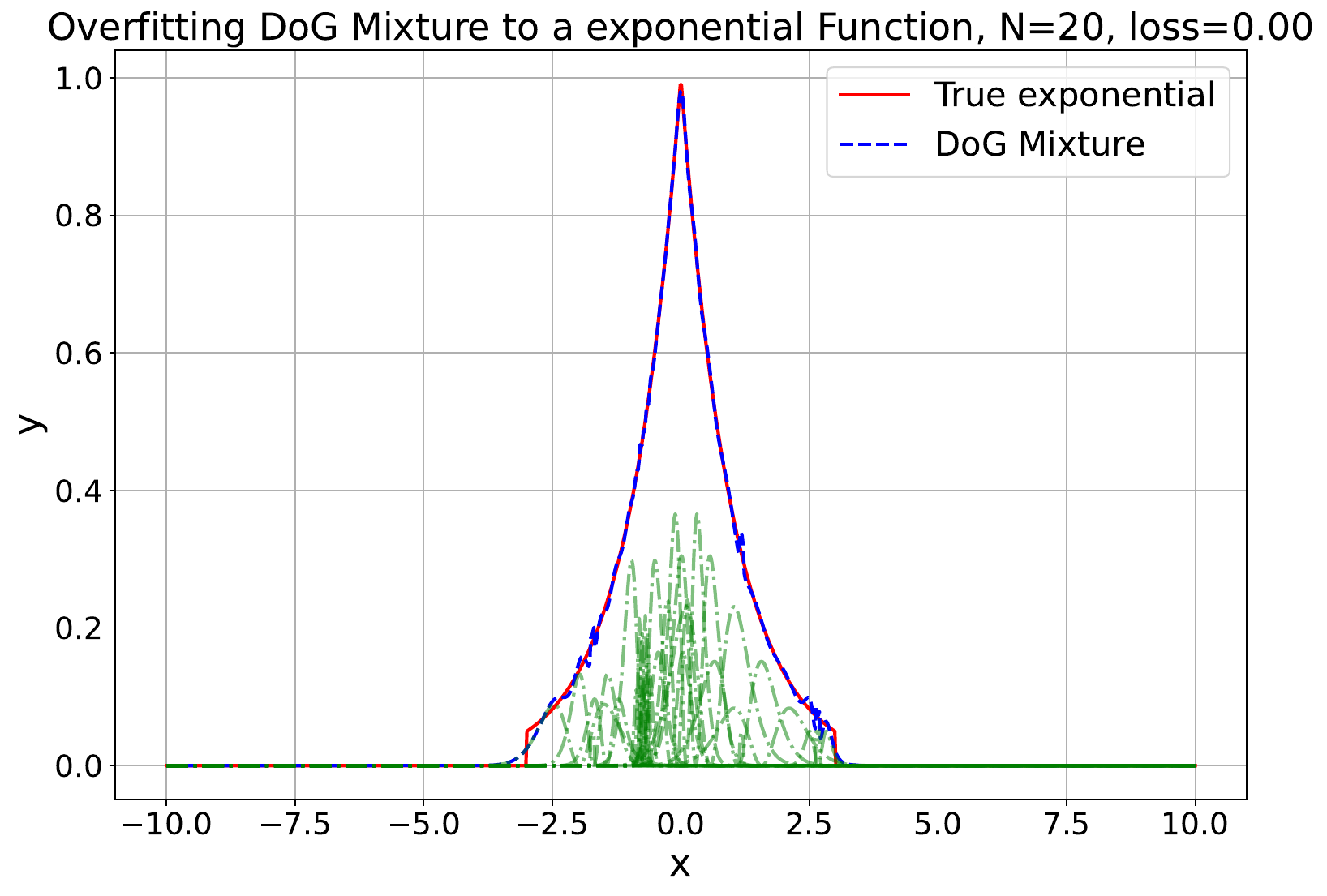} & 
    \includegraphics[width=0.24\linewidth]{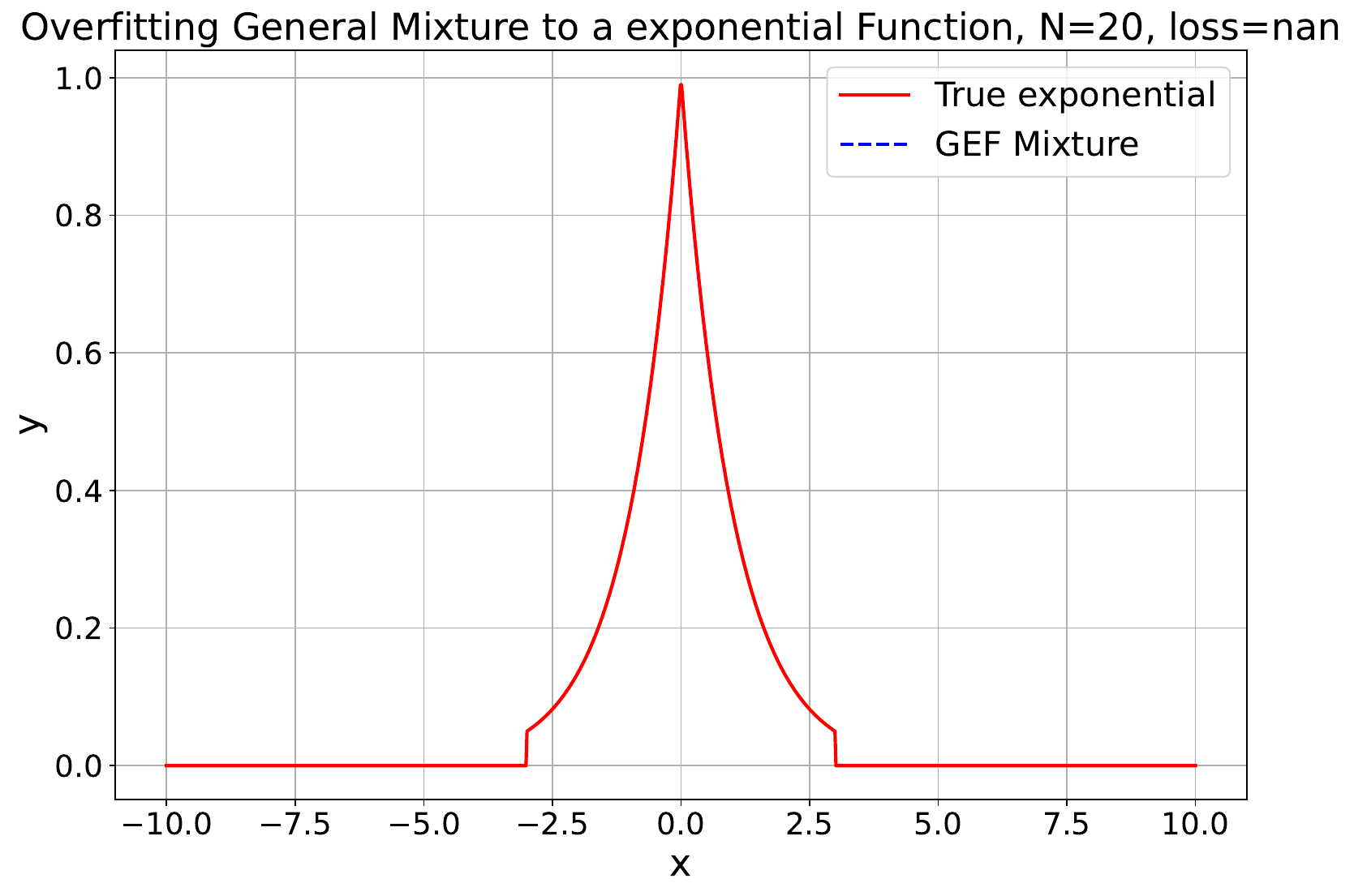}\\ 
    
    \end{tabular}
    }
    \caption{\textbf{Numerical Simulation Examples of Fitting Exponentials with Positive Weights Mixtures ( N= 2, 5, 8, and 10 )}. We show some fitting examples for exponential signals with positive weight mixtures. The four mixtures used from left to right are Gaussians, LoG, DoG, and General mixtures. From top to bottom: N = 2, 8, and 10 components. The optimized individual components are shown in green. Some examples fail to optimize due to numerical instability in both Gaussians and GEF mixtures. Note that GEF is very efficient in fitting the exponential with few components while LoG and DoG are more stable for a larger number of components. }
    \label{supfig:fitting_exponential_p}
    \end{figure*}
    

%% file: figures/fitting/fitting_exponential_n.tex
\begin{figure*}[h]
    \centering
    \resizebox{1.0\linewidth}{!}{
    \begin{tabular}{cccc}
    \tabcolsep=0.01cm
    Gaussian Mixture& LoG Mixture & DoG Mixture & GEF Mixture \\ 
    \includegraphics[width=0.24\linewidth]{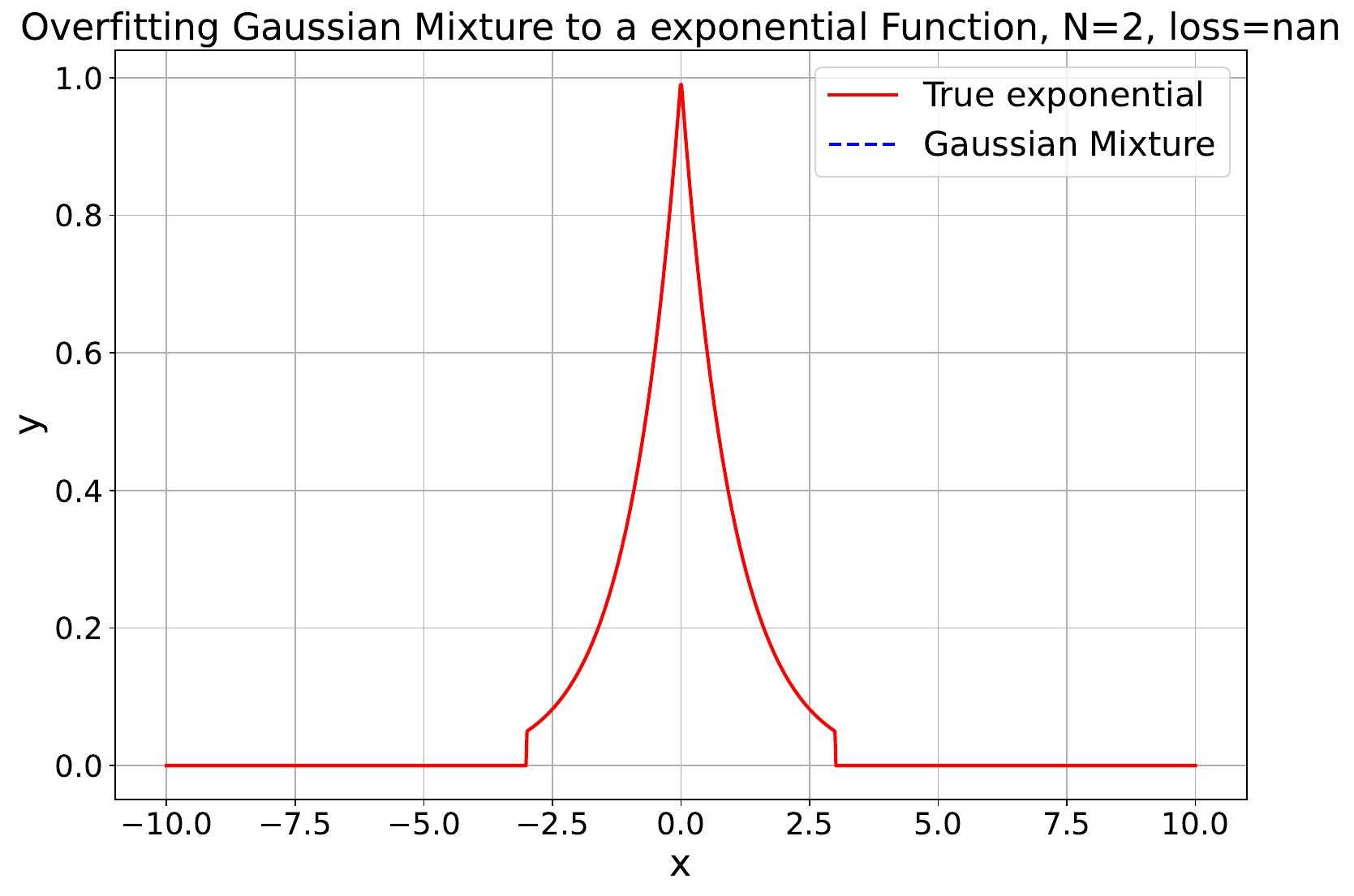} & 
    \includegraphics[width=0.24\linewidth]{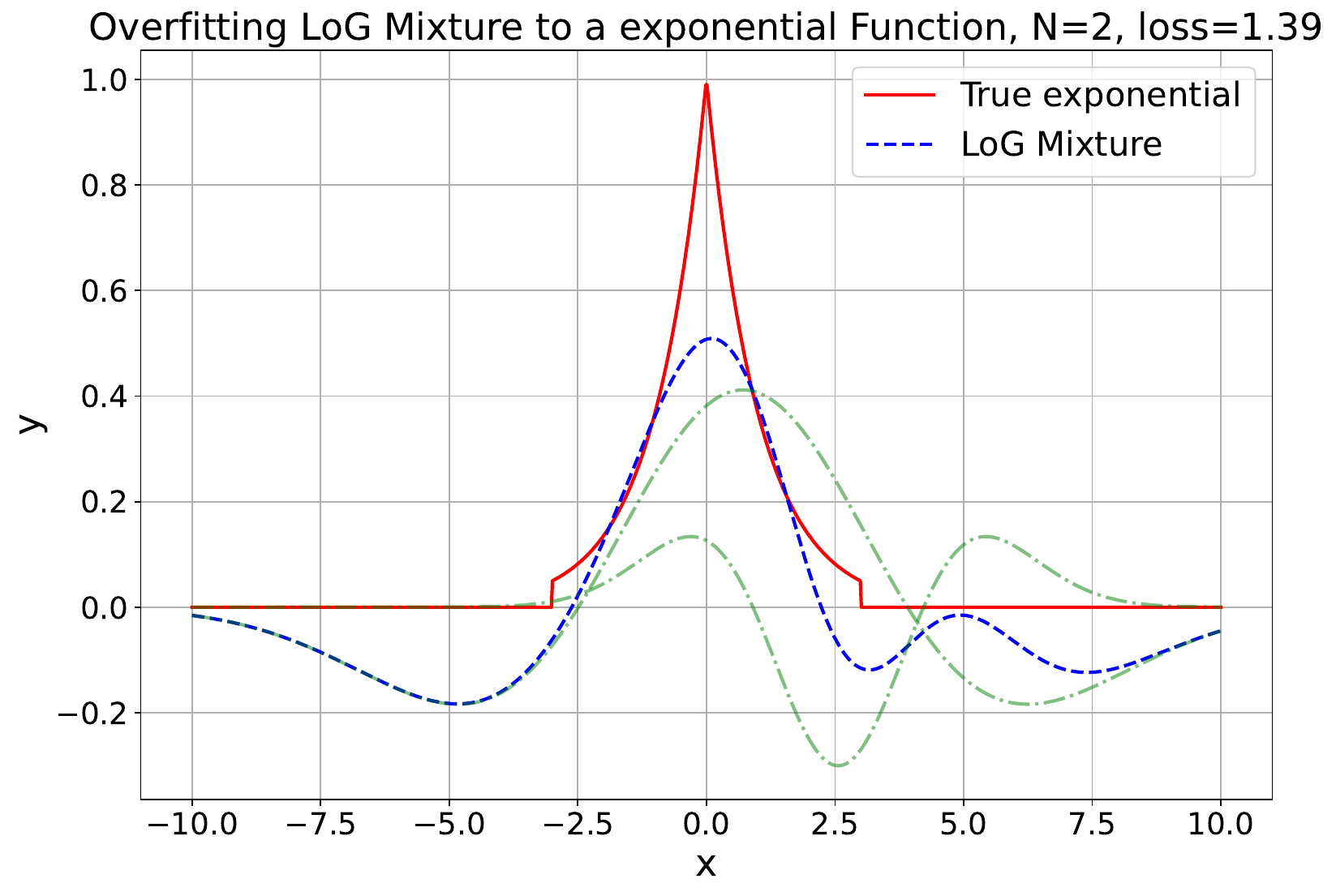} & 
    \includegraphics[width=0.24\linewidth]{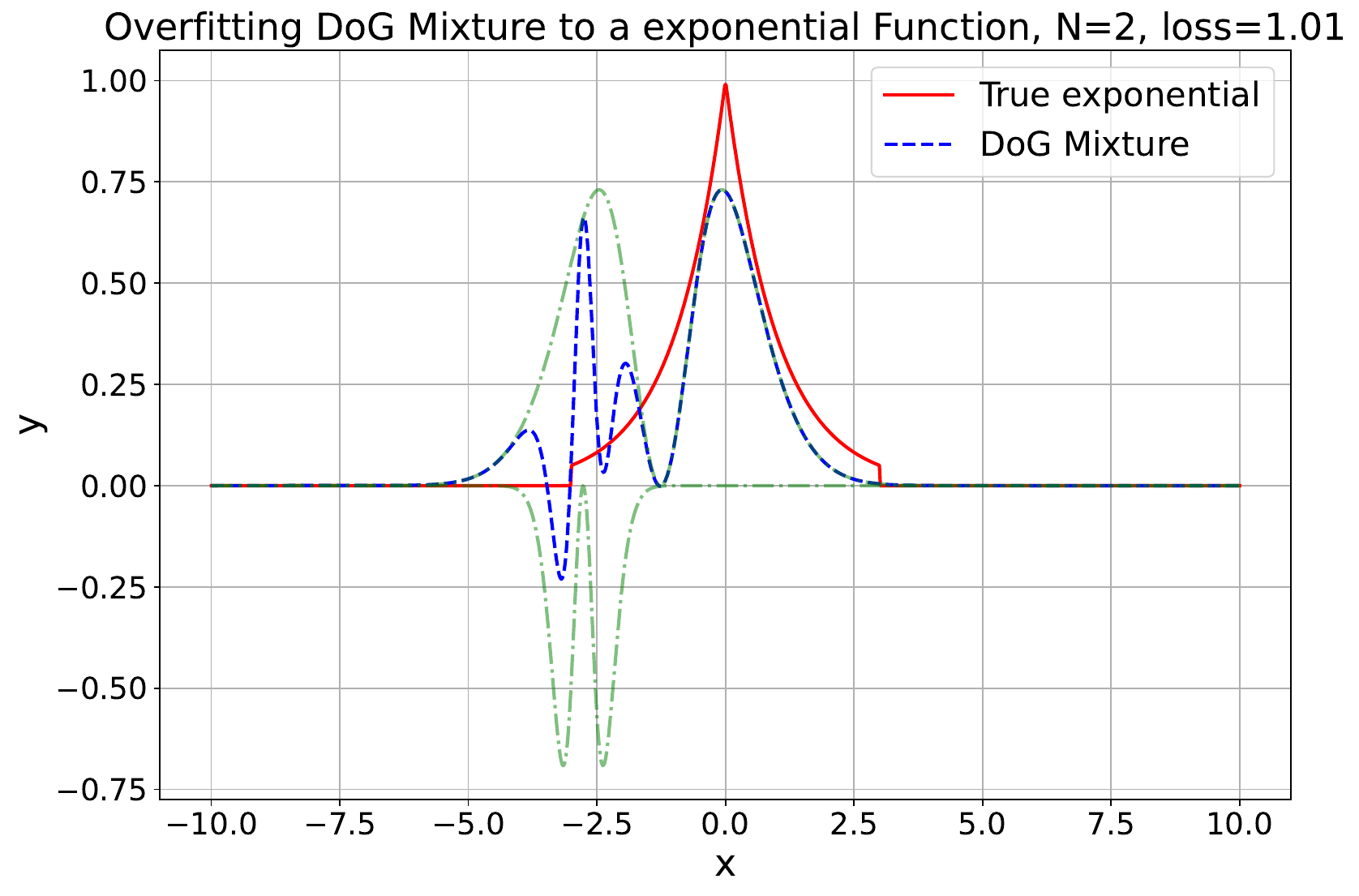} & 
    \includegraphics[width=0.24\linewidth]{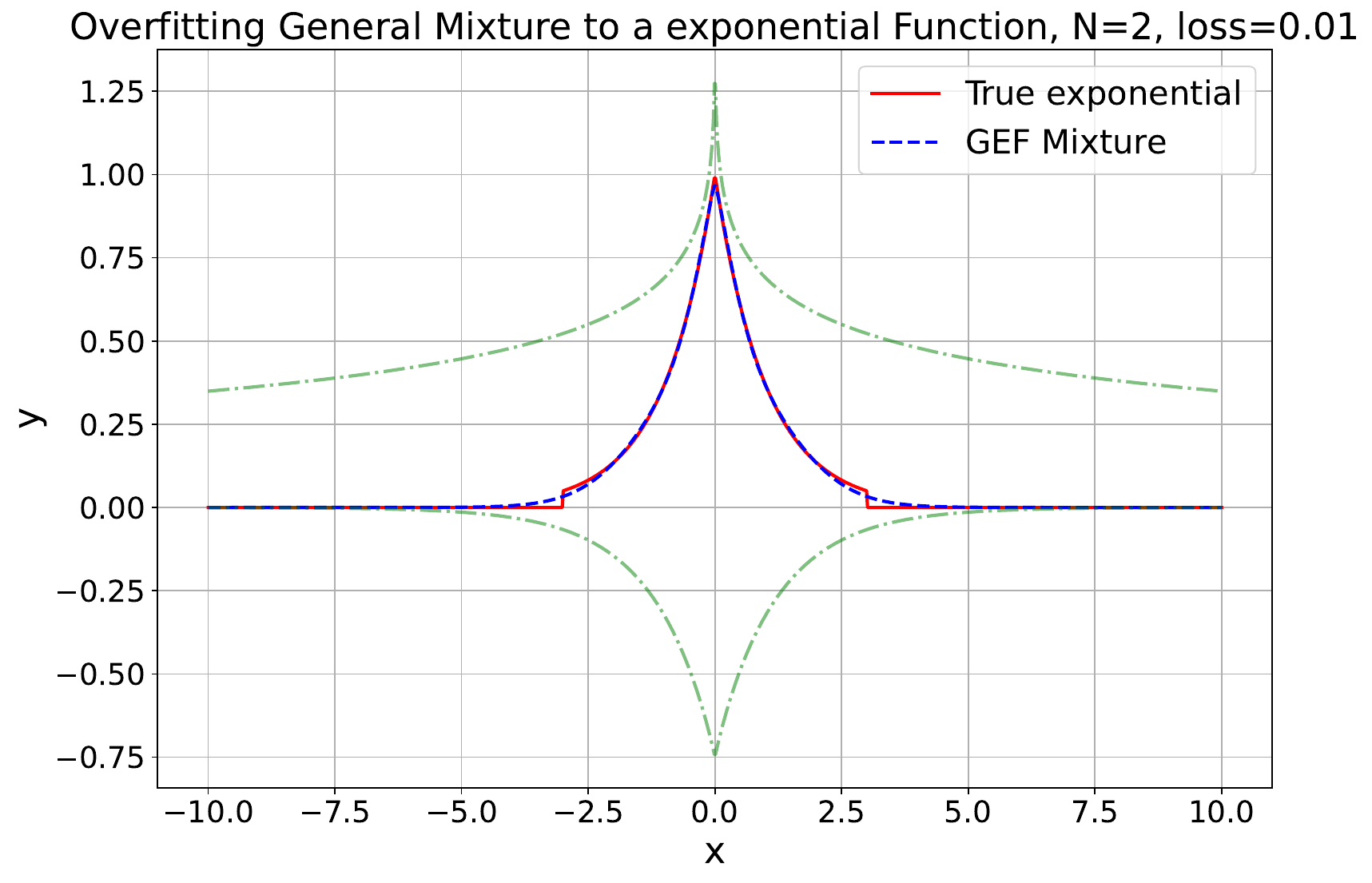}\\ 
    \includegraphics[width=0.24\linewidth]{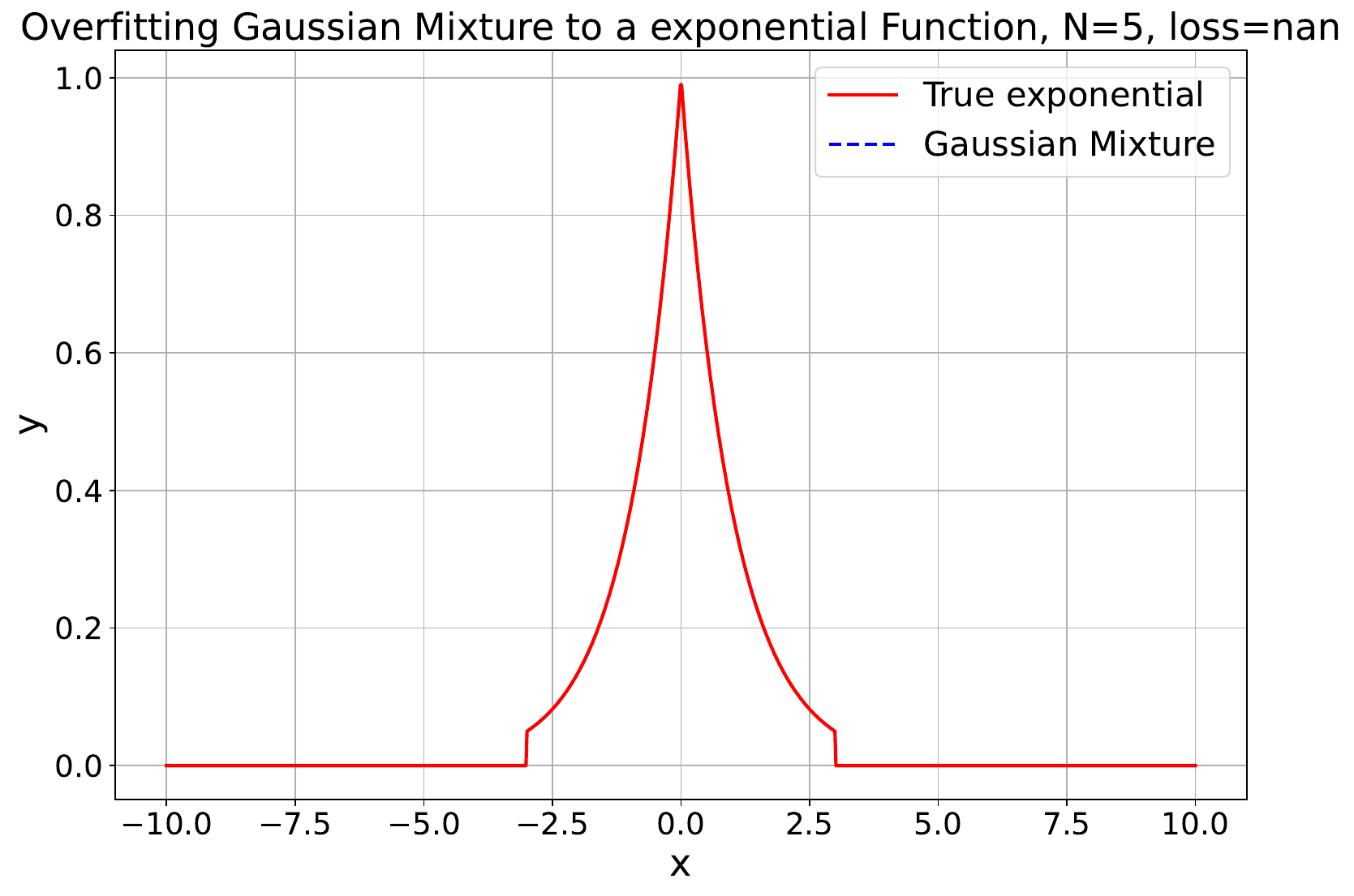} & 
    \includegraphics[width=0.24\linewidth]{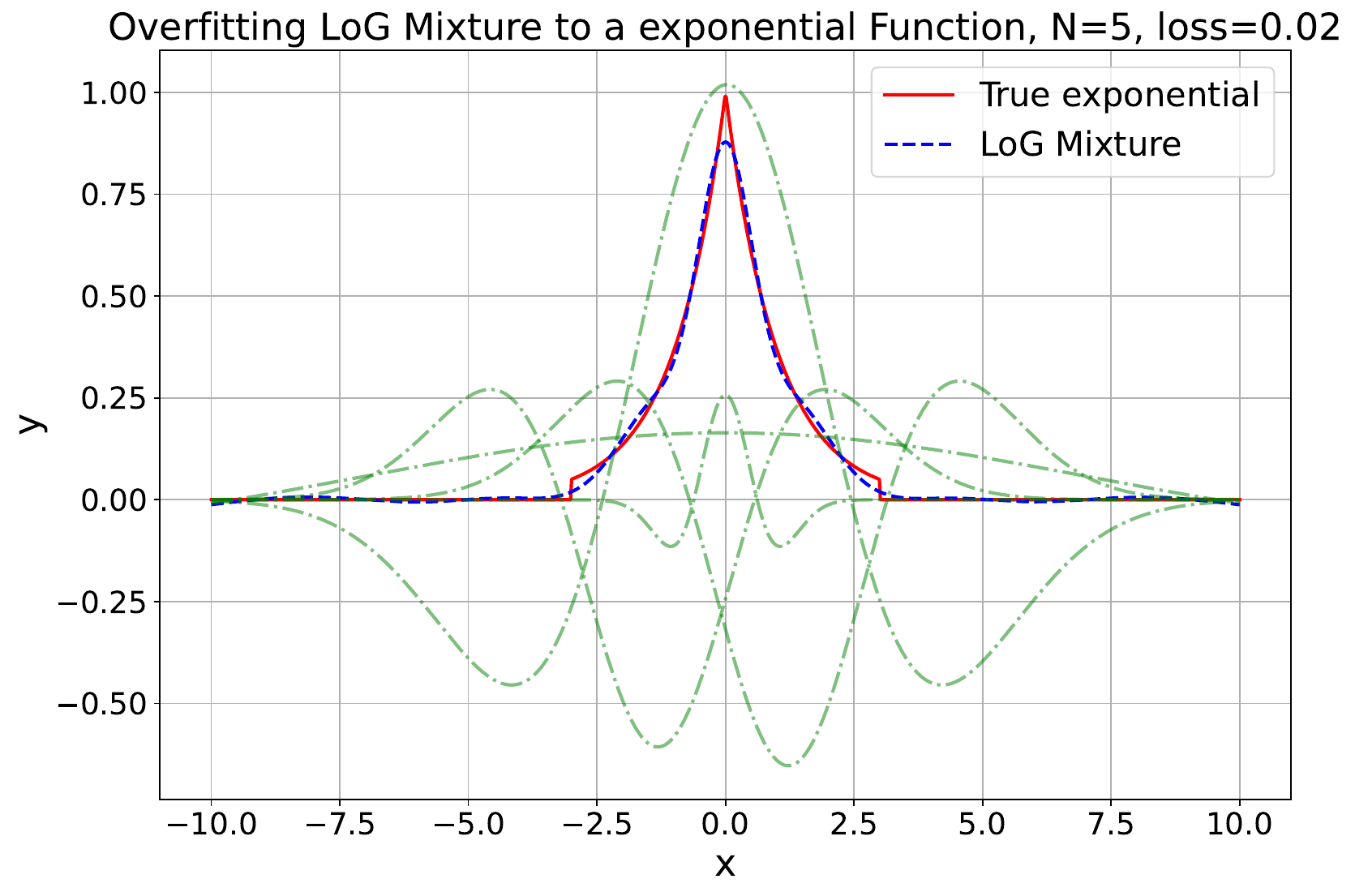} & 
    \includegraphics[width=0.24\linewidth]{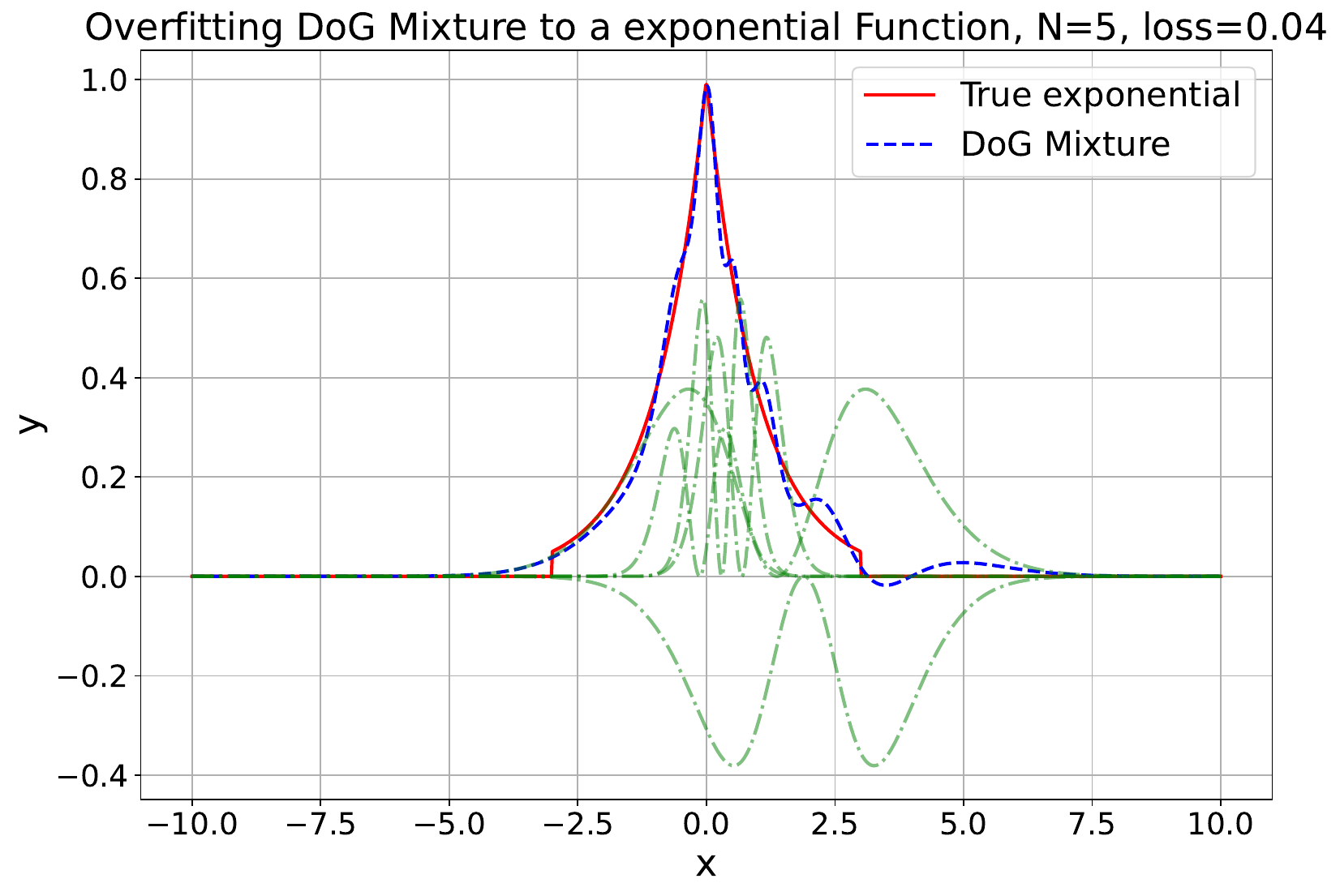} & 
    \includegraphics[width=0.24\linewidth]{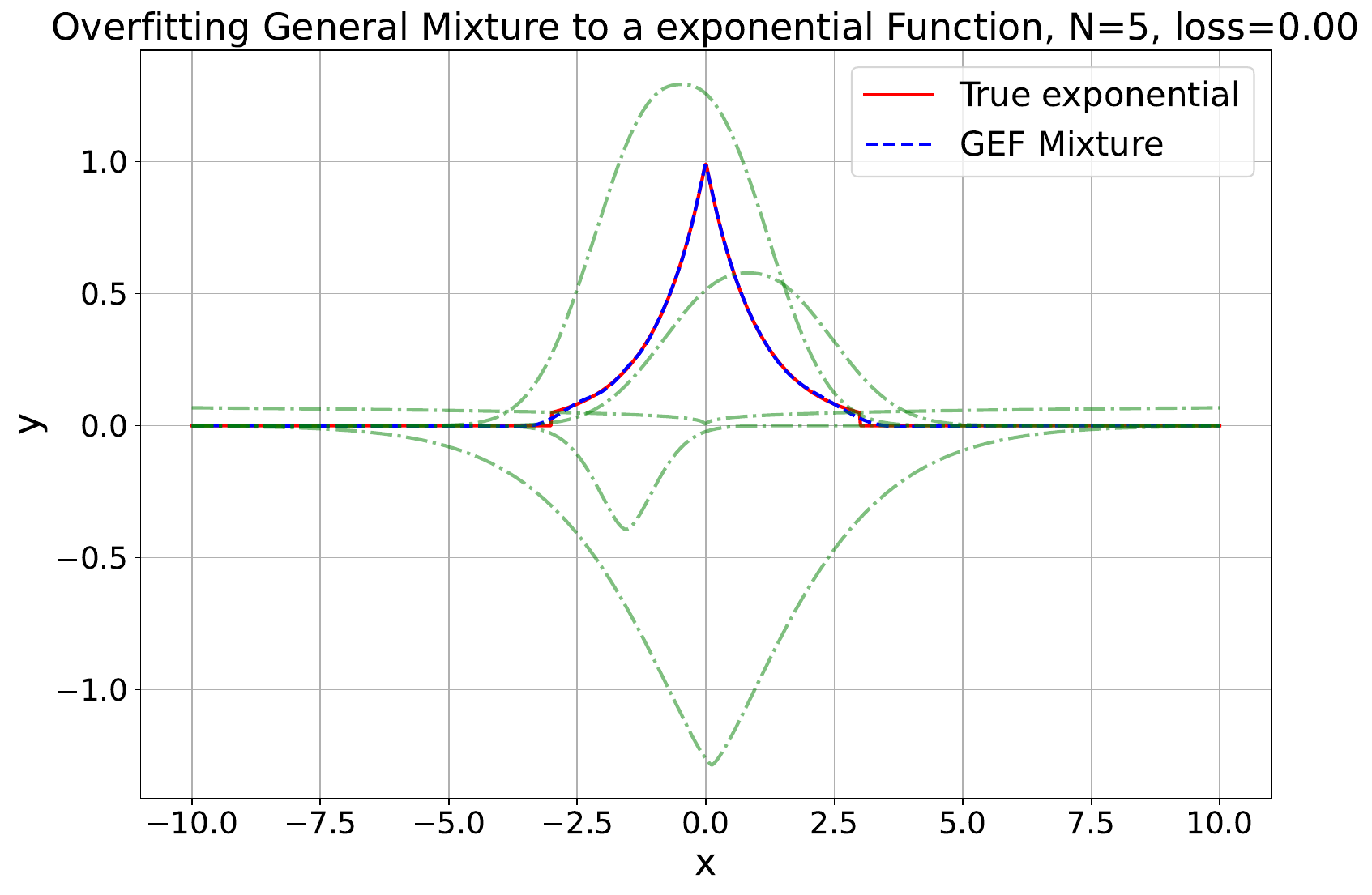}\\ 
    \includegraphics[width=0.24\linewidth]{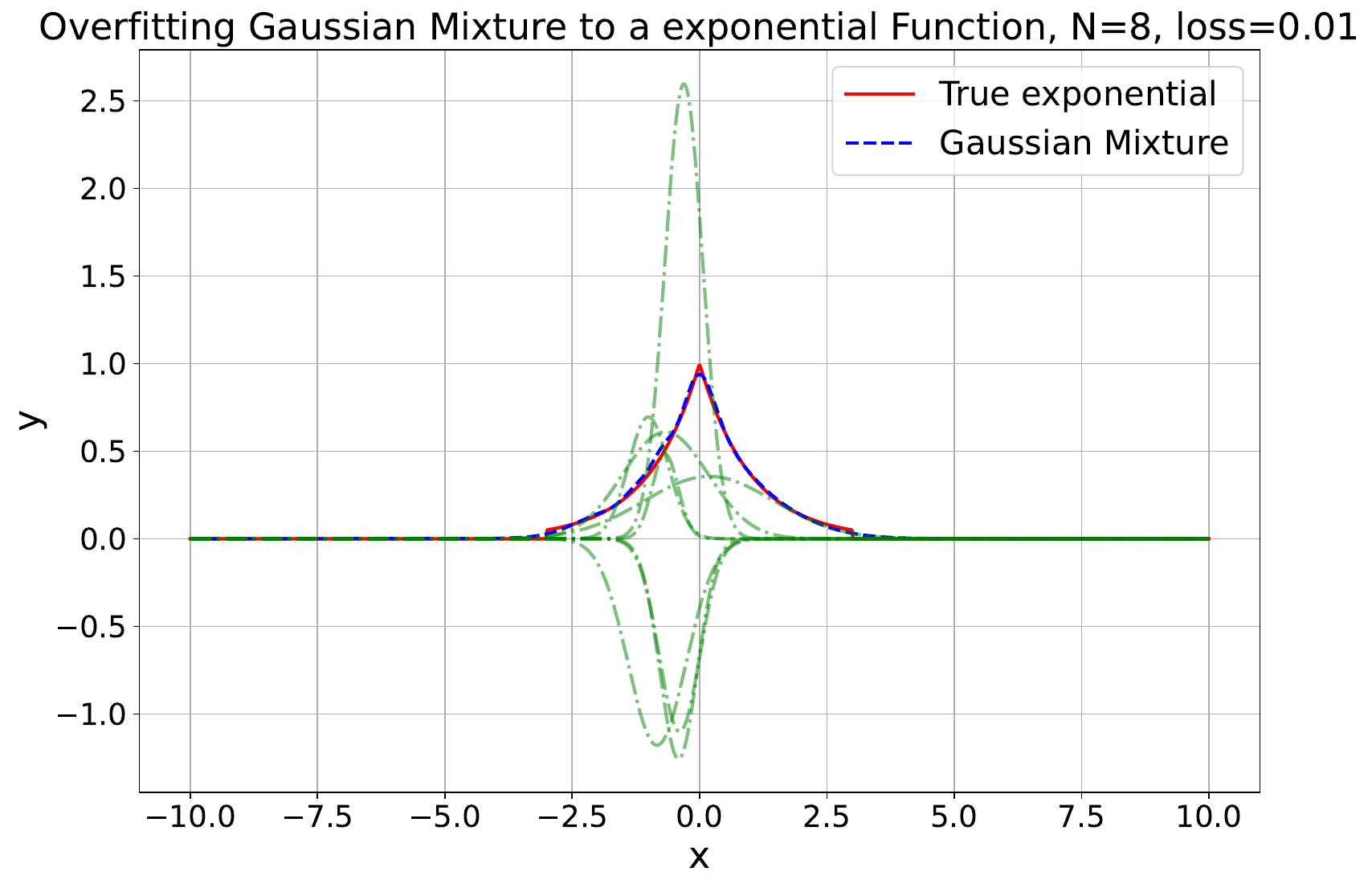} & 
    \includegraphics[width=0.24\linewidth]{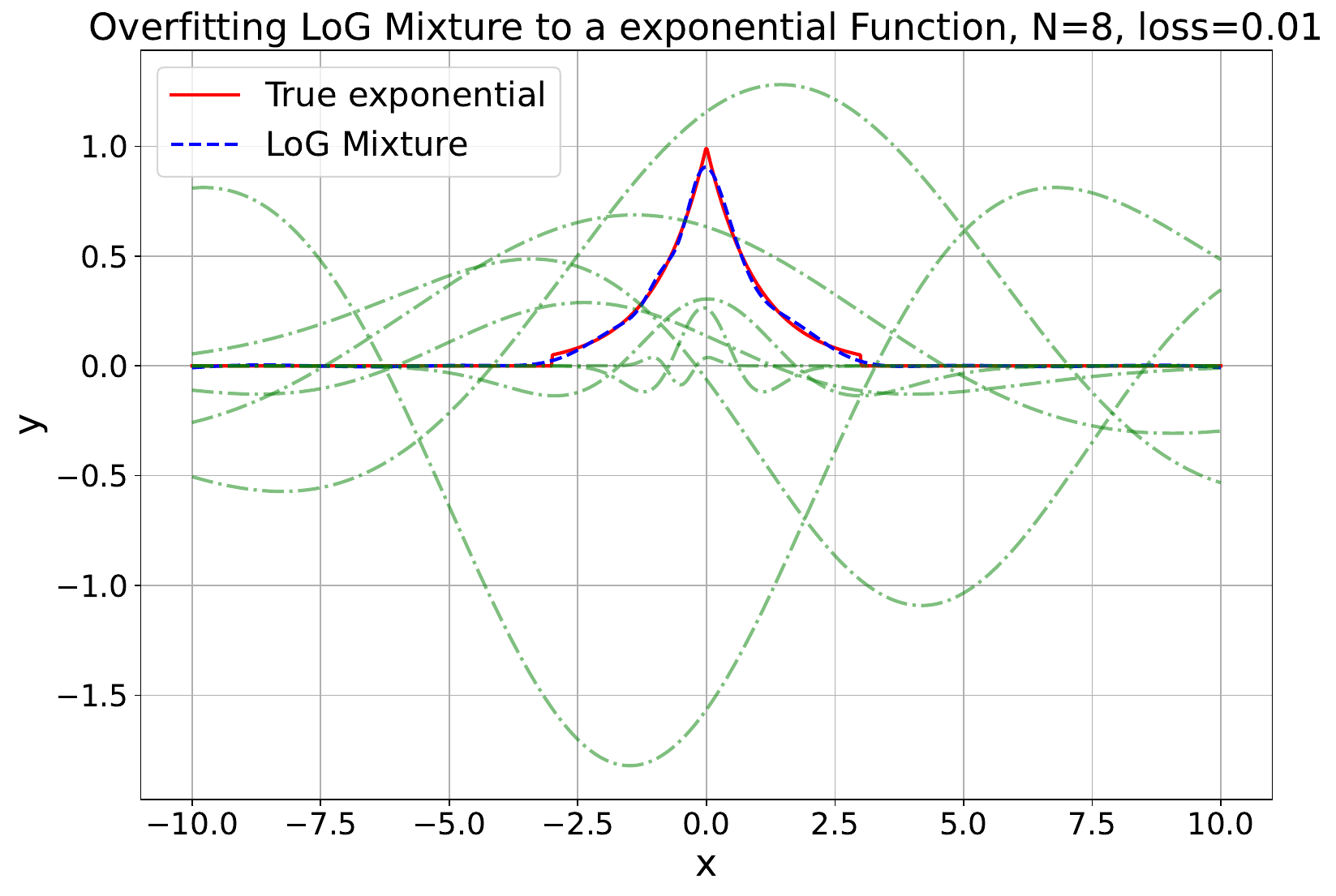} & 
    \includegraphics[width=0.24\linewidth]{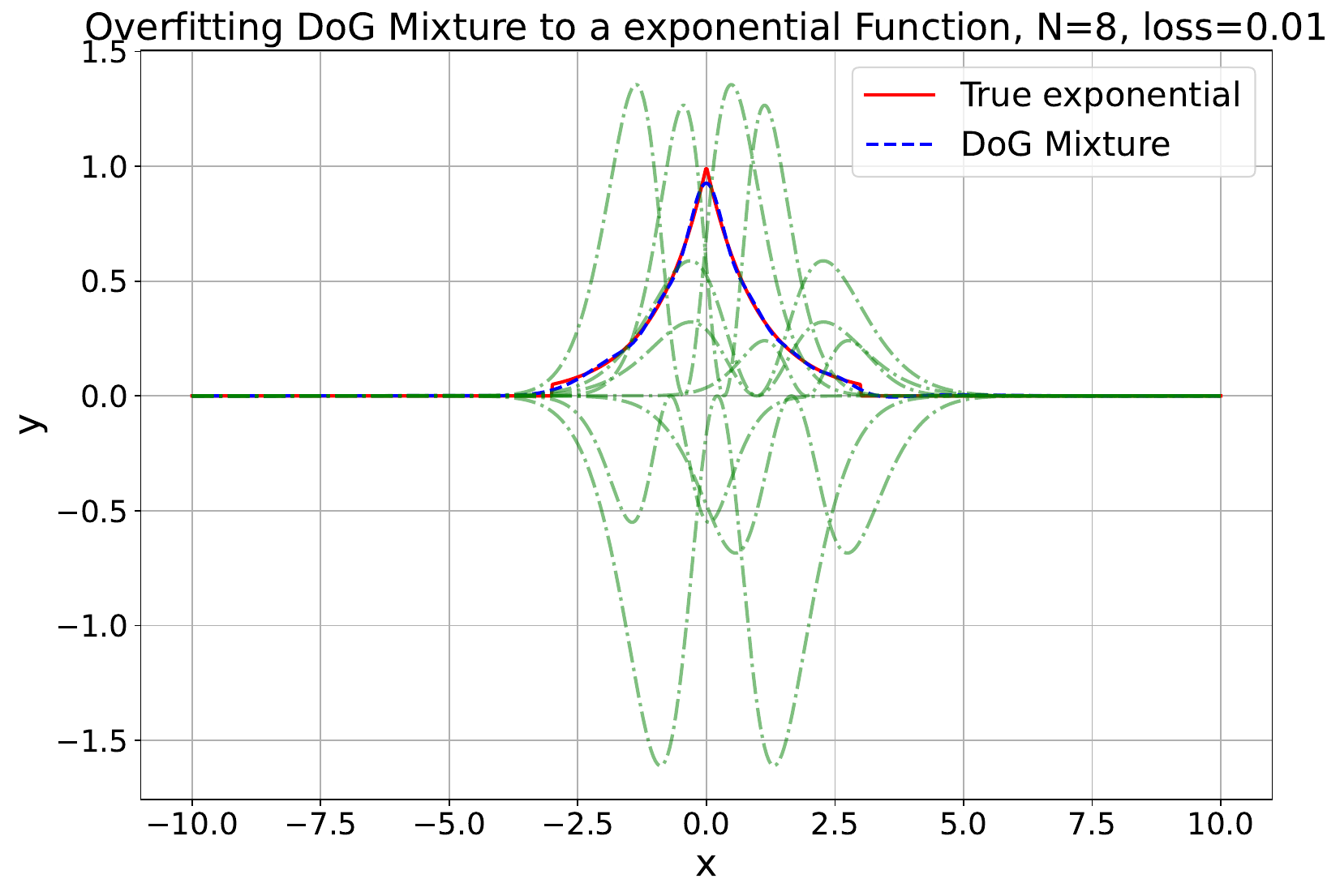} & 
    \includegraphics[width=0.24\linewidth]{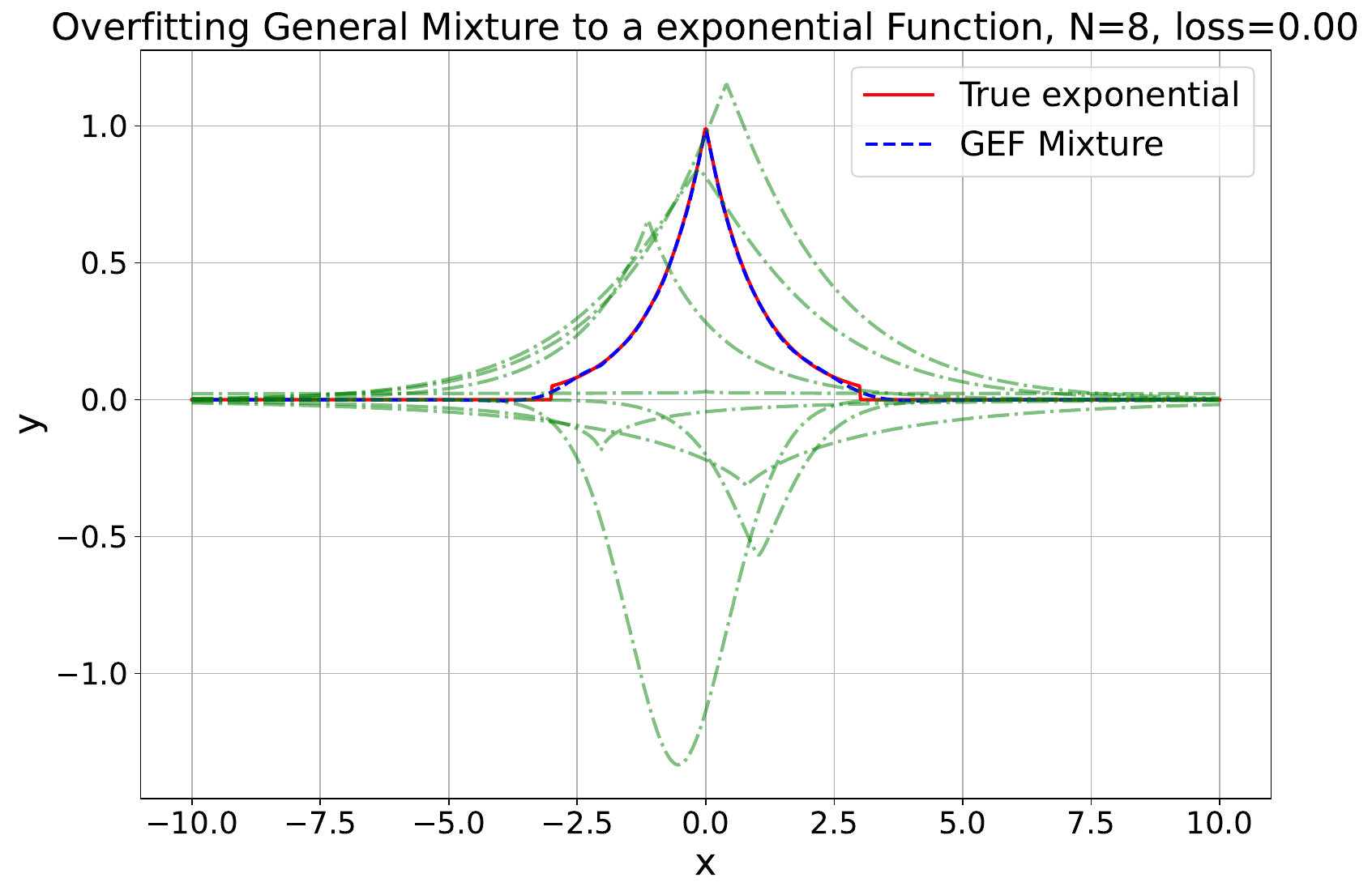}\\ 
    \includegraphics[width=0.24\linewidth]{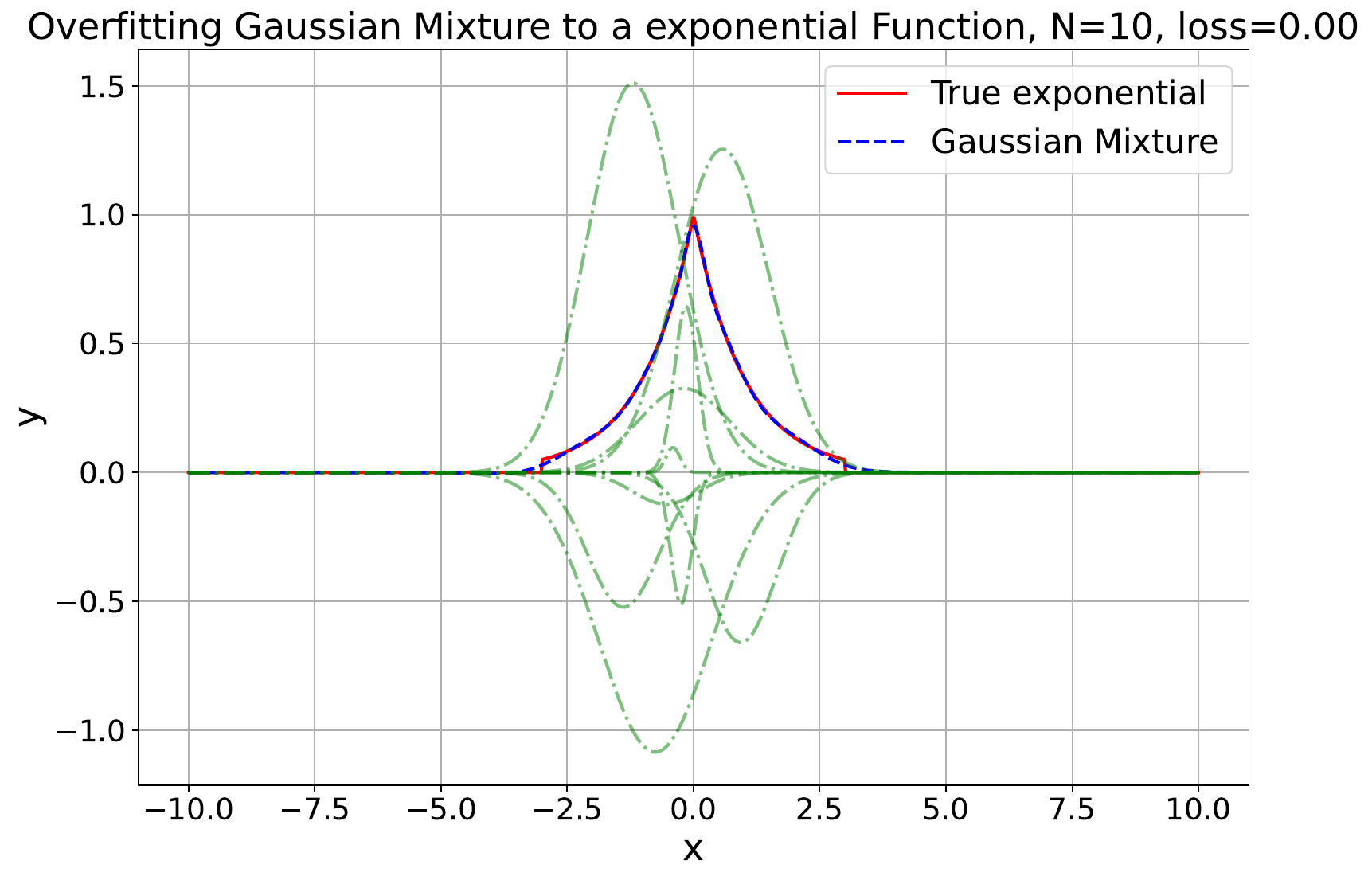} & 
    \includegraphics[width=0.24\linewidth]{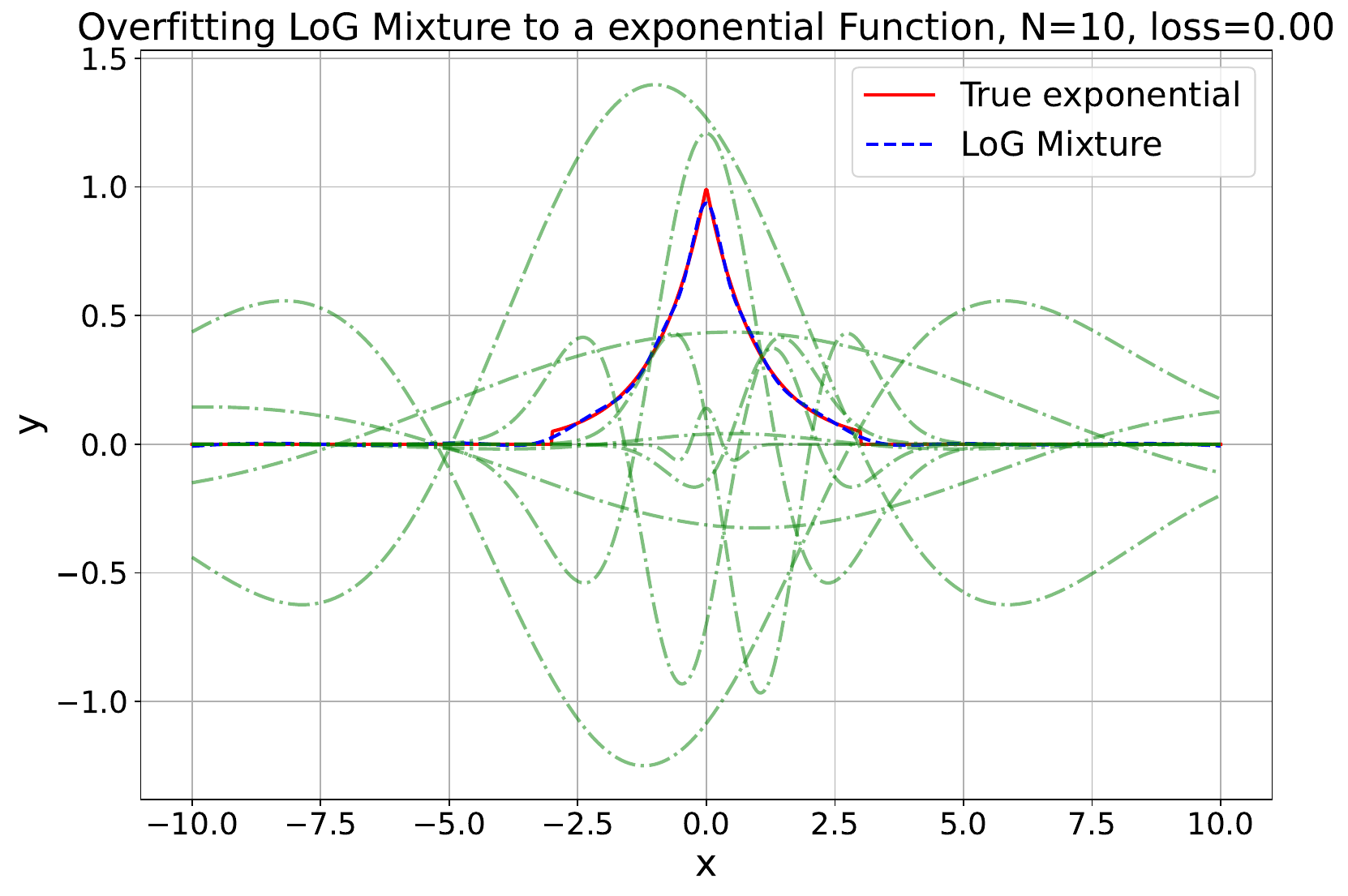} & 
    \includegraphics[width=0.24\linewidth]{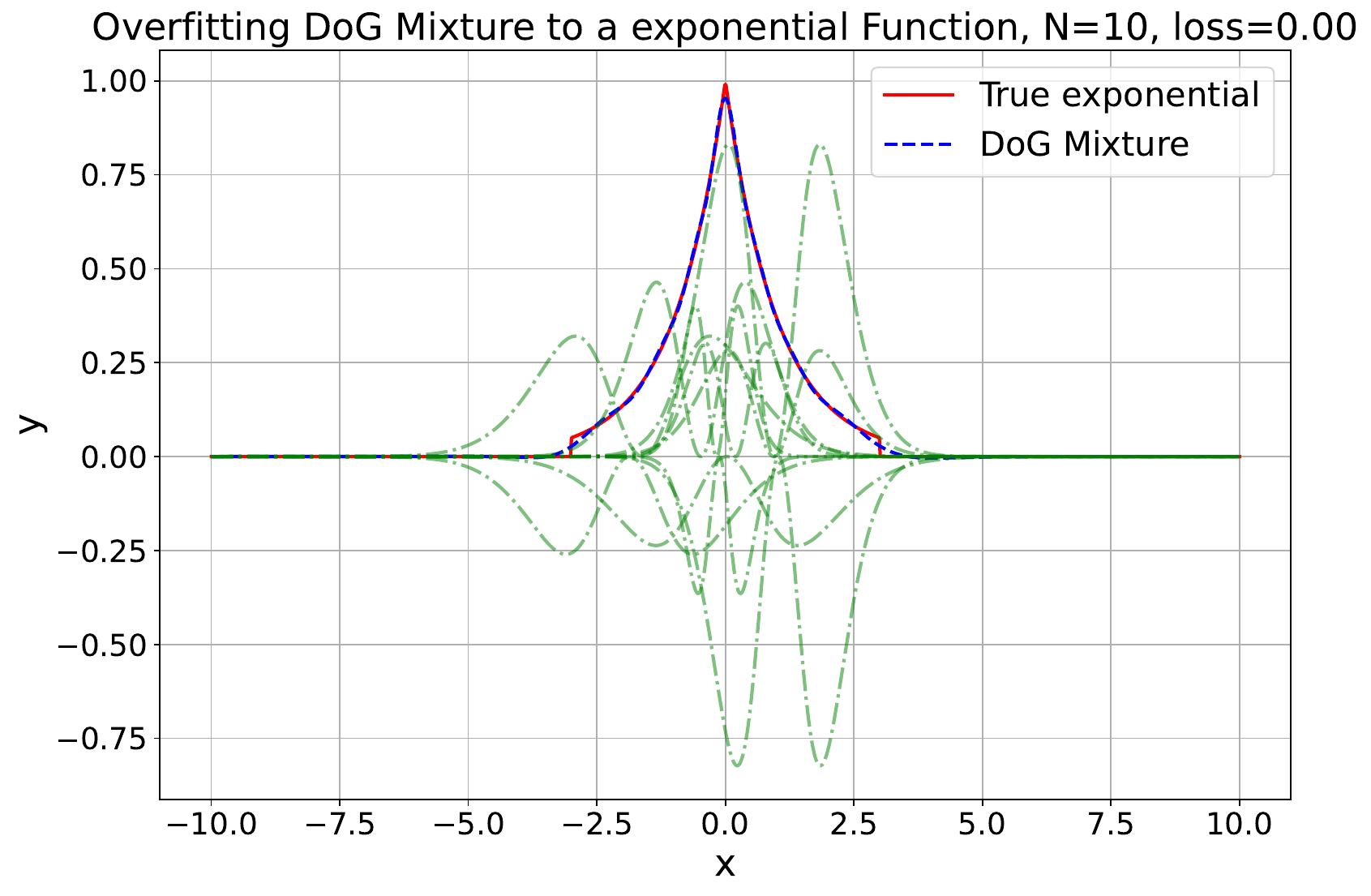} & 
    \includegraphics[width=0.24\linewidth]{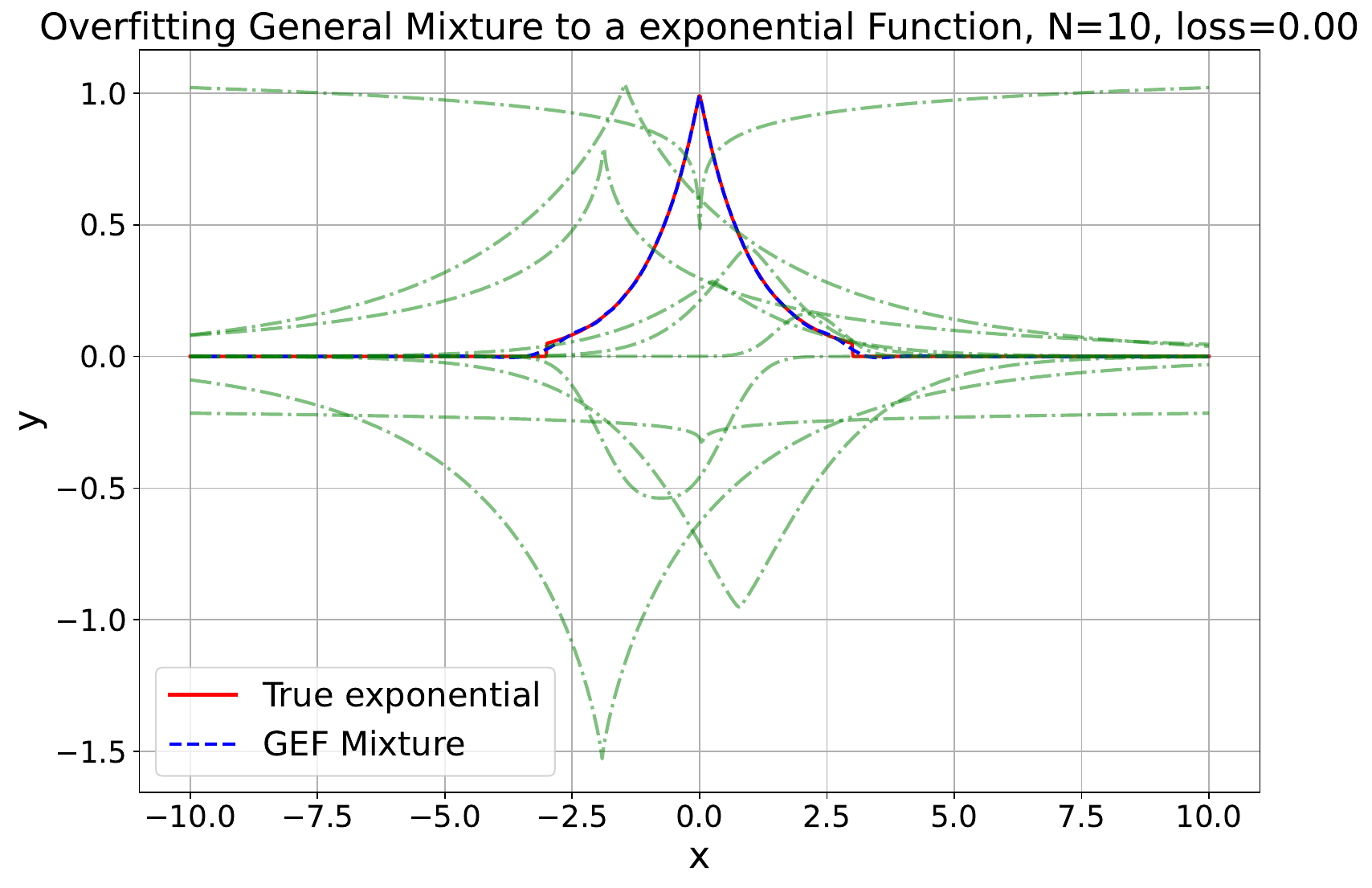}\\ 
    \includegraphics[width=0.24\linewidth]{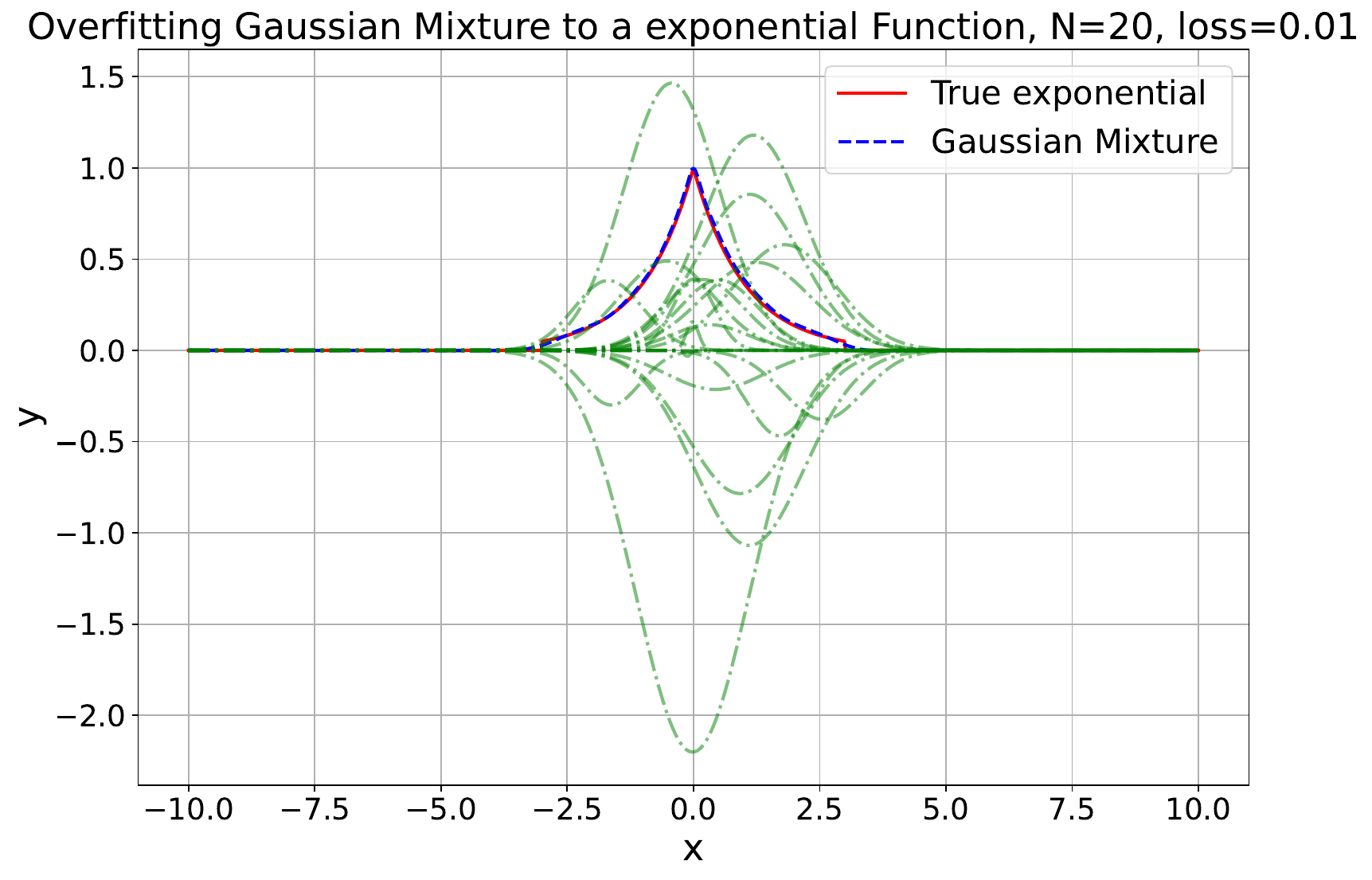} & 
    \includegraphics[width=0.24\linewidth]{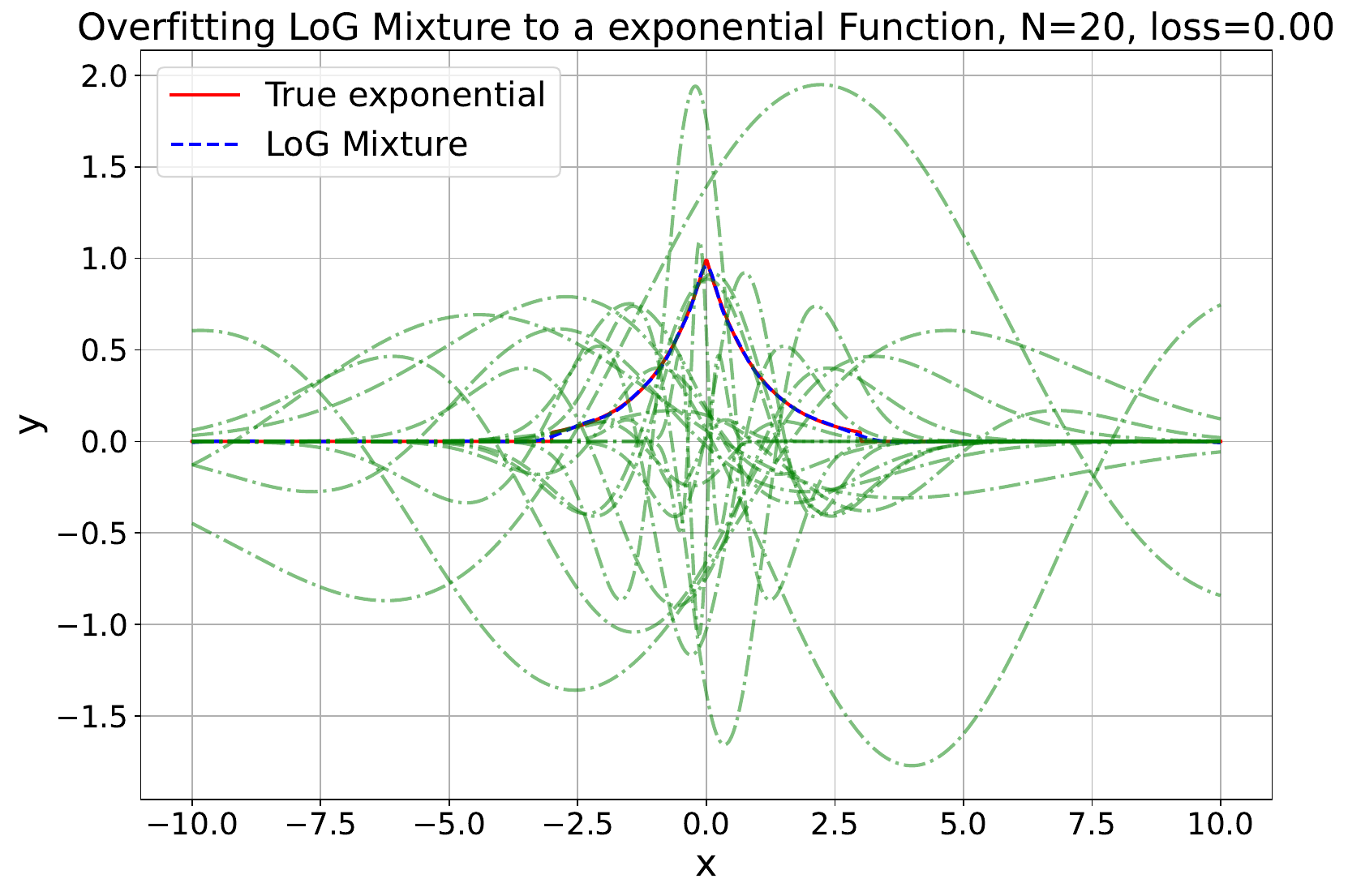} & 
    \includegraphics[width=0.24\linewidth]{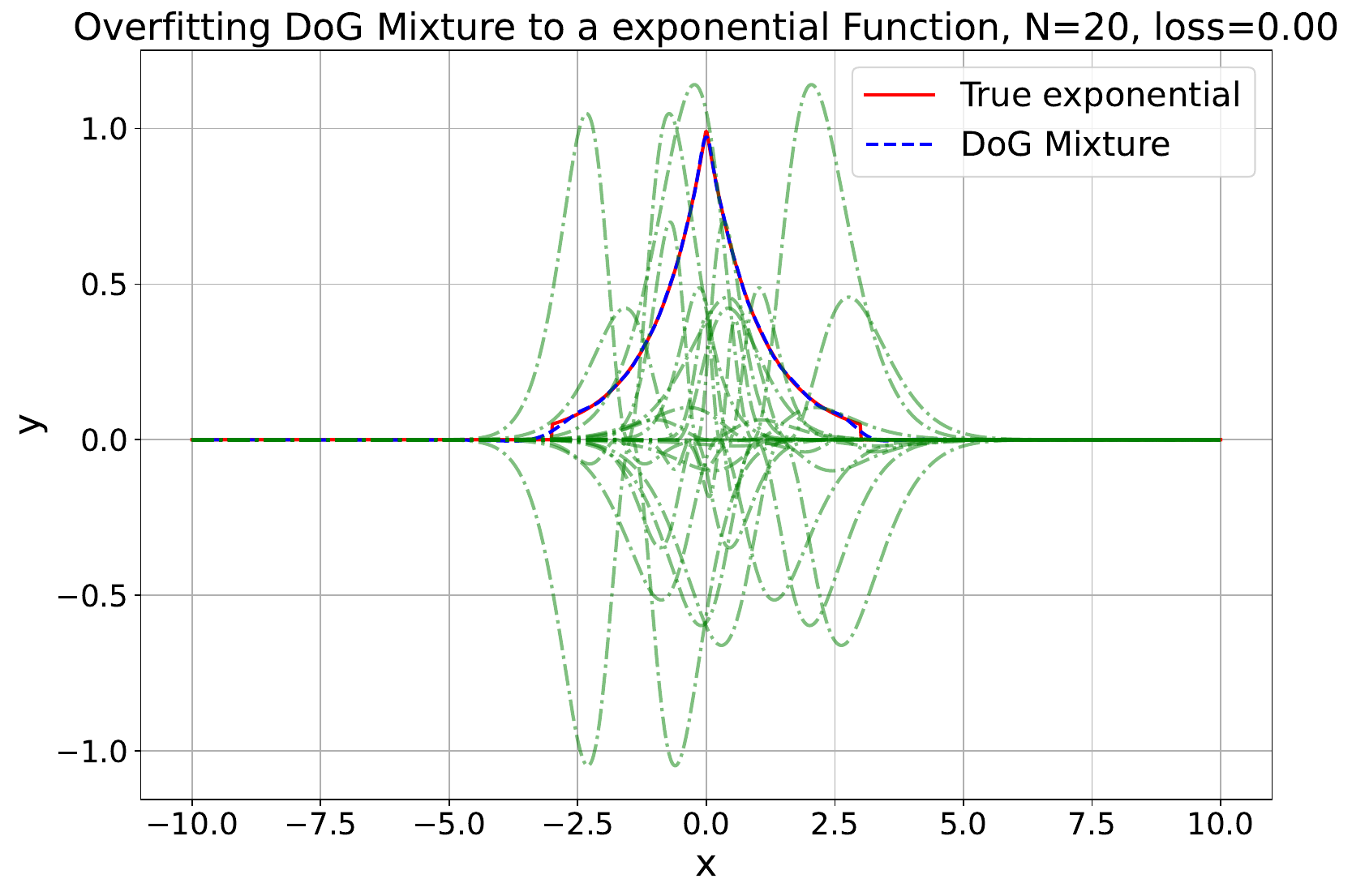} & 
    \includegraphics[width=0.24\linewidth]{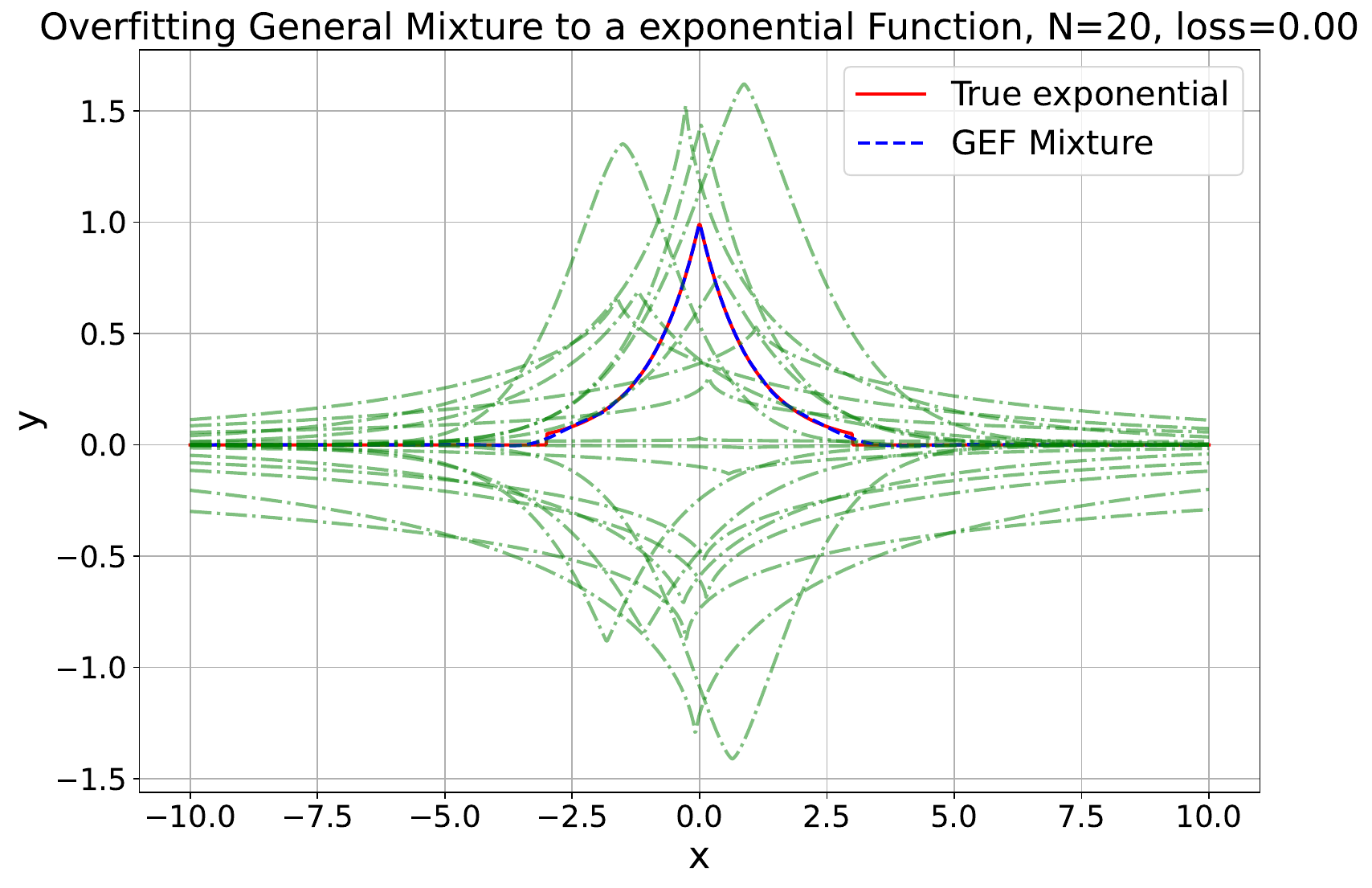}\\ 
    
    \end{tabular}
    }
    \caption{\textbf{Numerical Simulation Examples of Fitting Exponentials with Real Weights Mixtures ( N= 2, 5, 8, and 10 )}. We show some fitting examples for exponential signals with Real weights mixtures (can be negative). The four mixtures used from left to right are Gaussians, LoG, DoG, and General mixtures. From top to bottom: N = 2, 8, and 10 components. The optimized individual components are shown in green. Some examples fail to optimize due to numerical instability in both Gaussians and GEF mixtures. Note that GEF is very efficient in fitting the exponential with few components while LoG and DoG are more stable for a larger number of components. }
    \label{supfig:fitting_exponential_N}
    \end{figure*}
    

%% file: figures/fitting/fitting_traingle_p.tex
\begin{figure*}[h]
    \centering
    \resizebox{1.0\linewidth}{!}{
    \begin{tabular}{cccc}
    \tabcolsep=0.01cm
    Gaussian Mixture& LoG Mixture & DoG Mixture & GEF Mixture \\ 
    \includegraphics[width=0.24\linewidth]{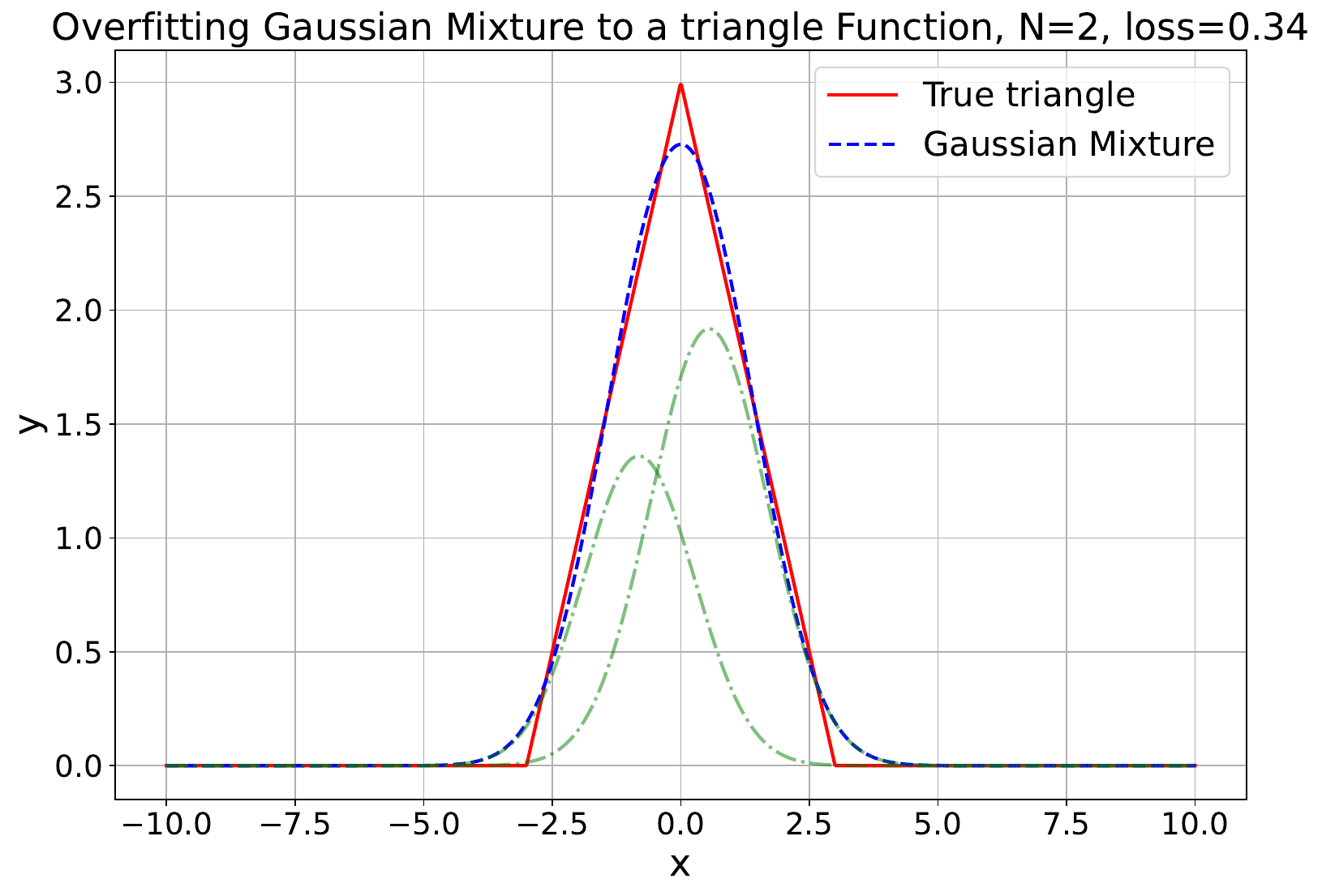} & 
    \includegraphics[width=0.24\linewidth]{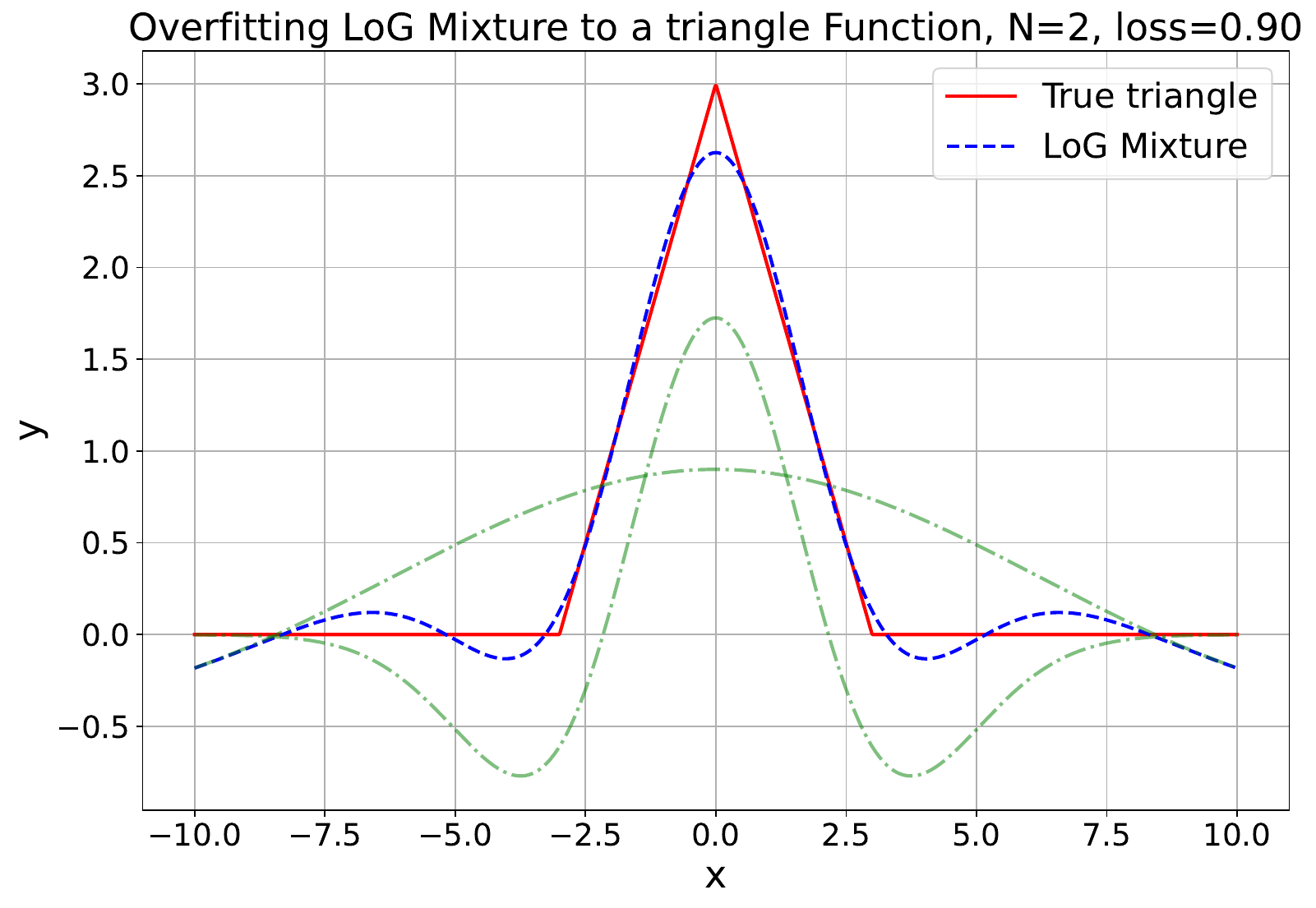} & 
    \includegraphics[width=0.24\linewidth]{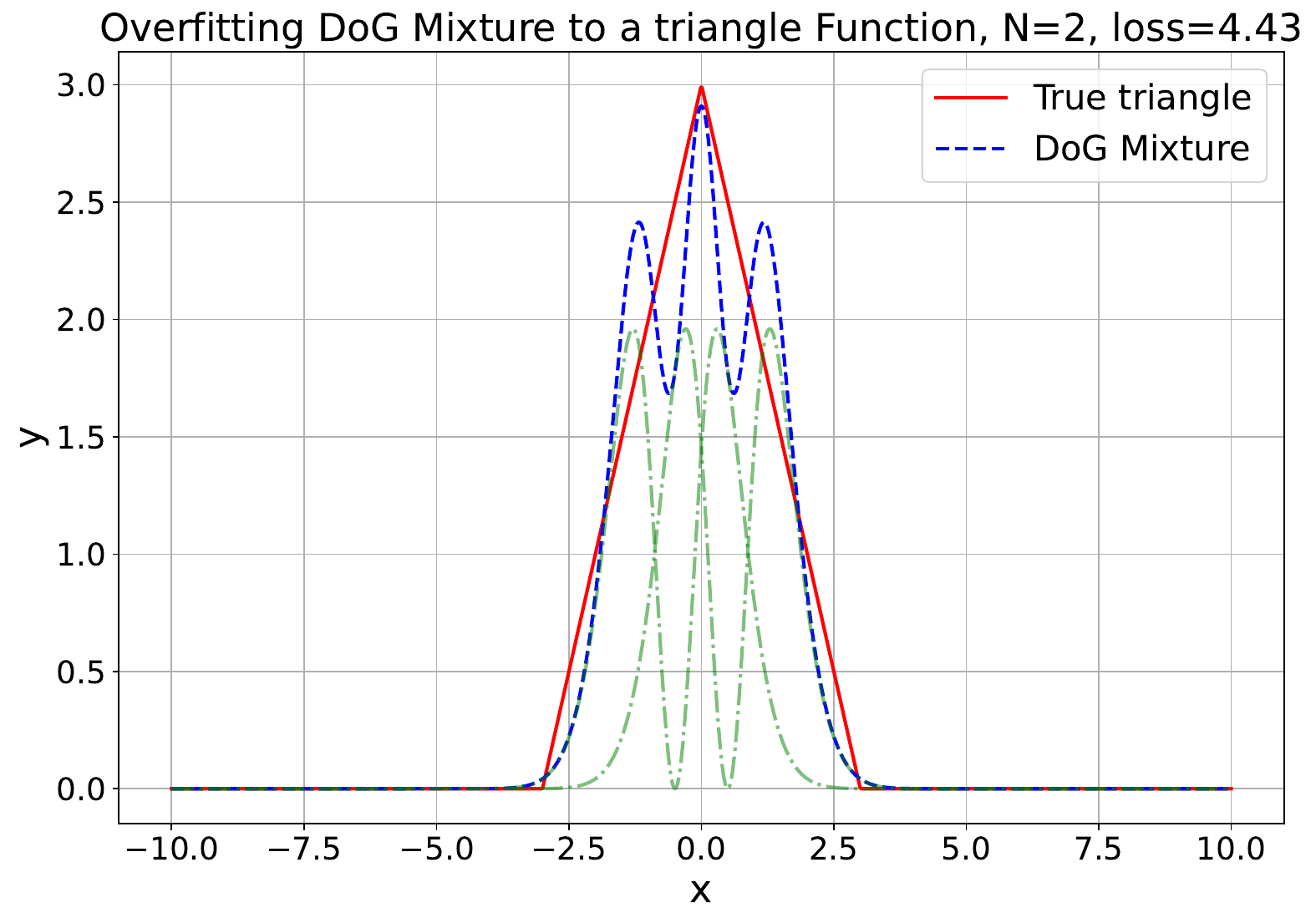} & 
    \includegraphics[width=0.24\linewidth]{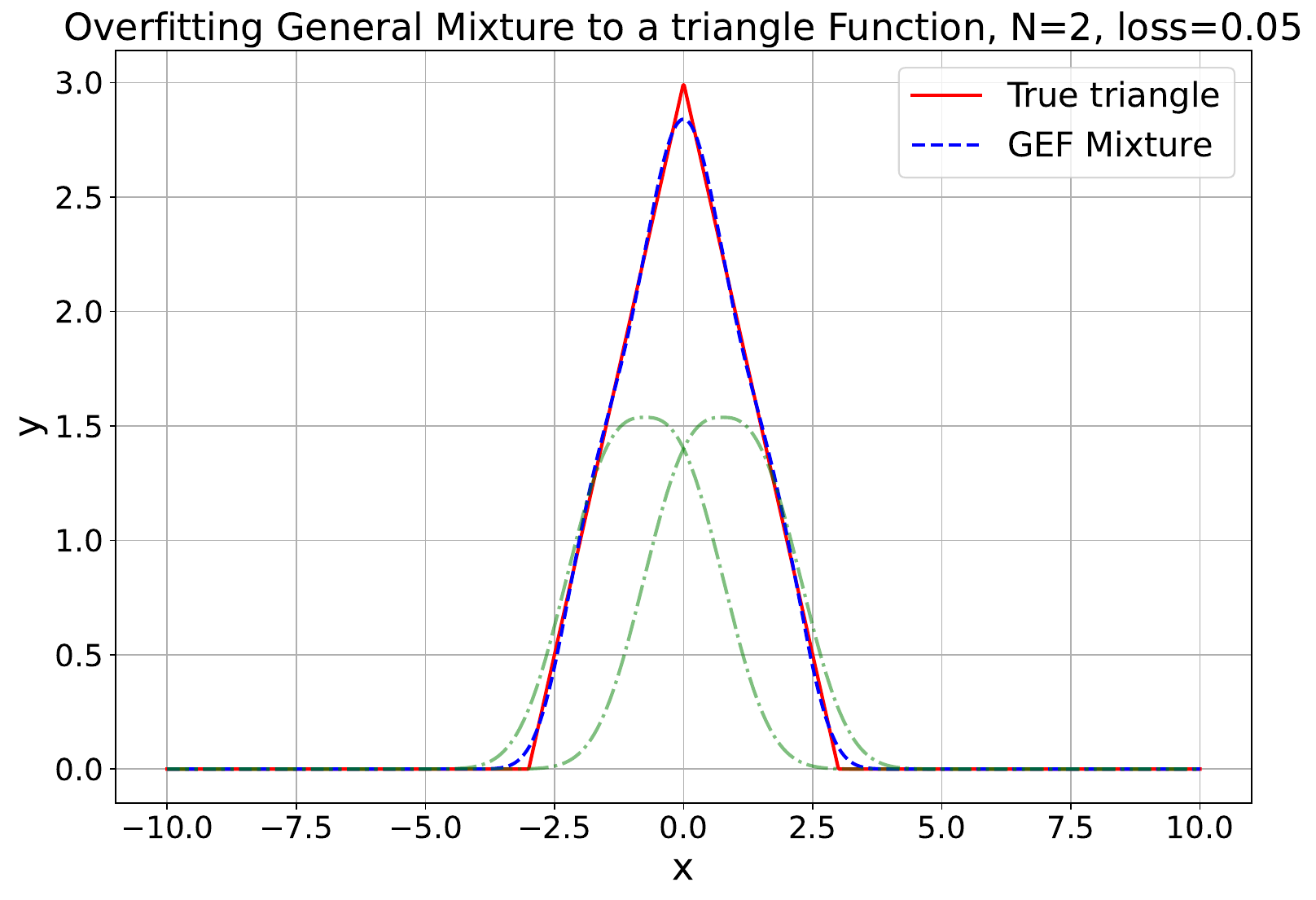}\\ 
    \includegraphics[width=0.24\linewidth]{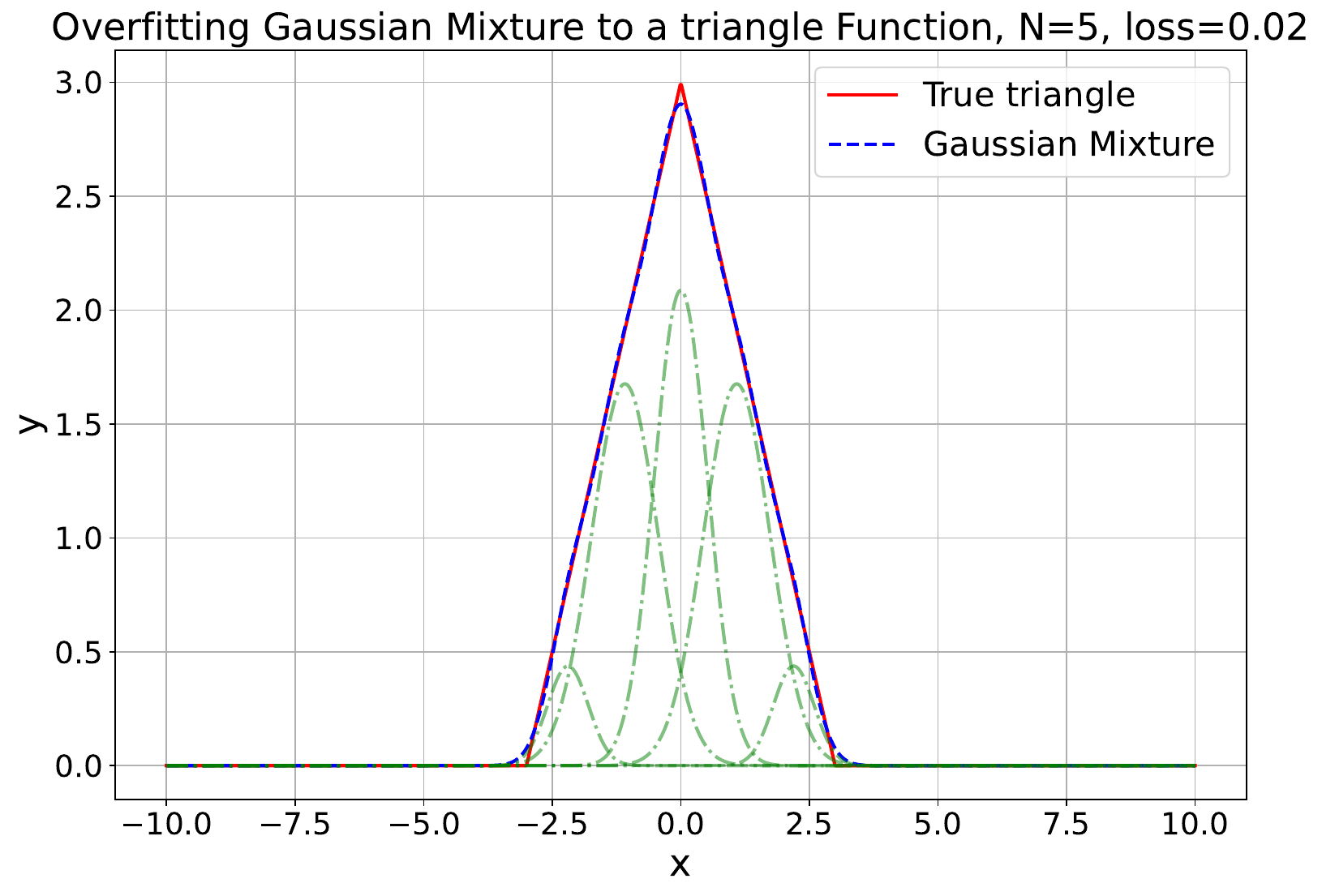} & 
    \includegraphics[width=0.24\linewidth]{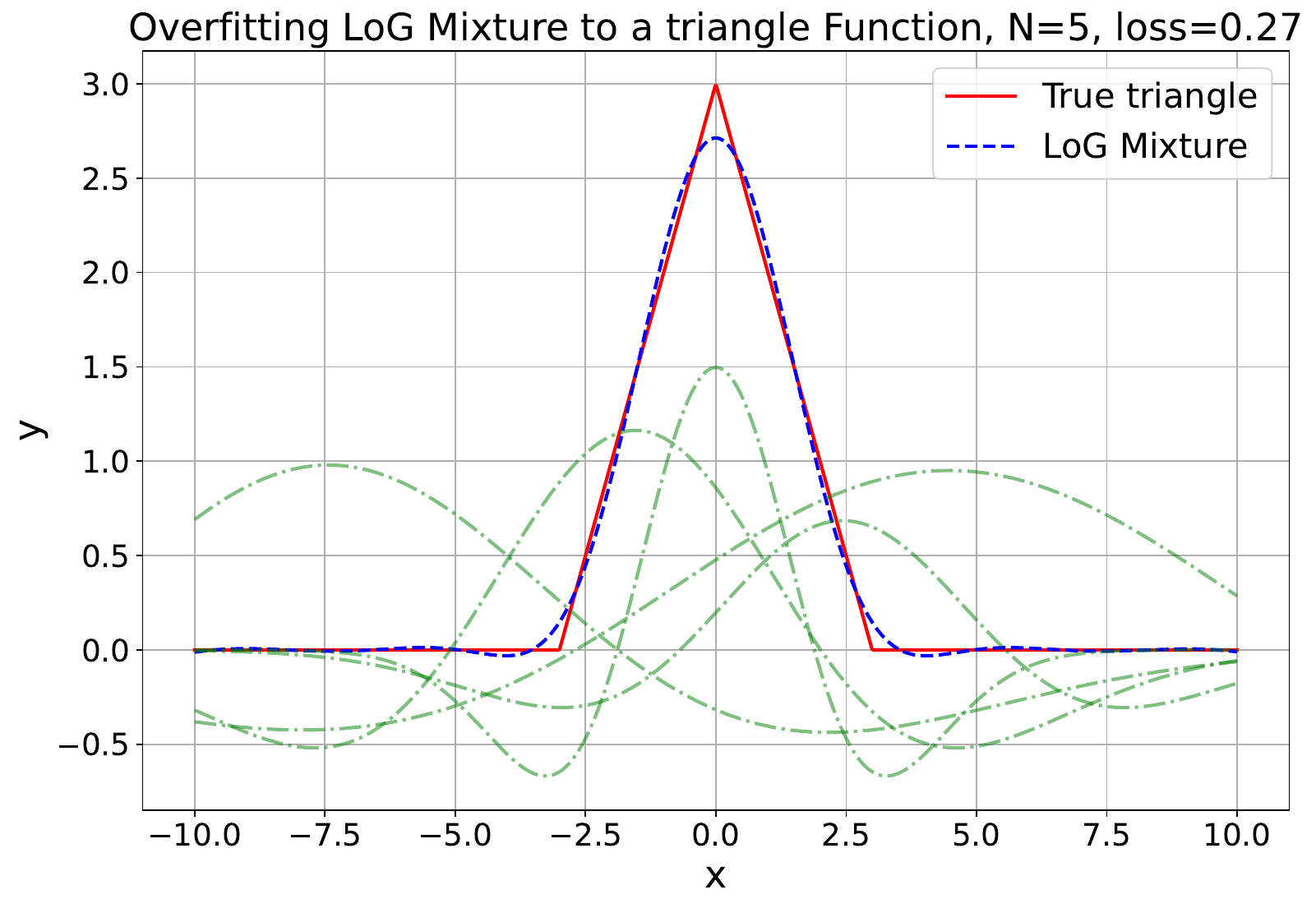} & 
    \includegraphics[width=0.24\linewidth]{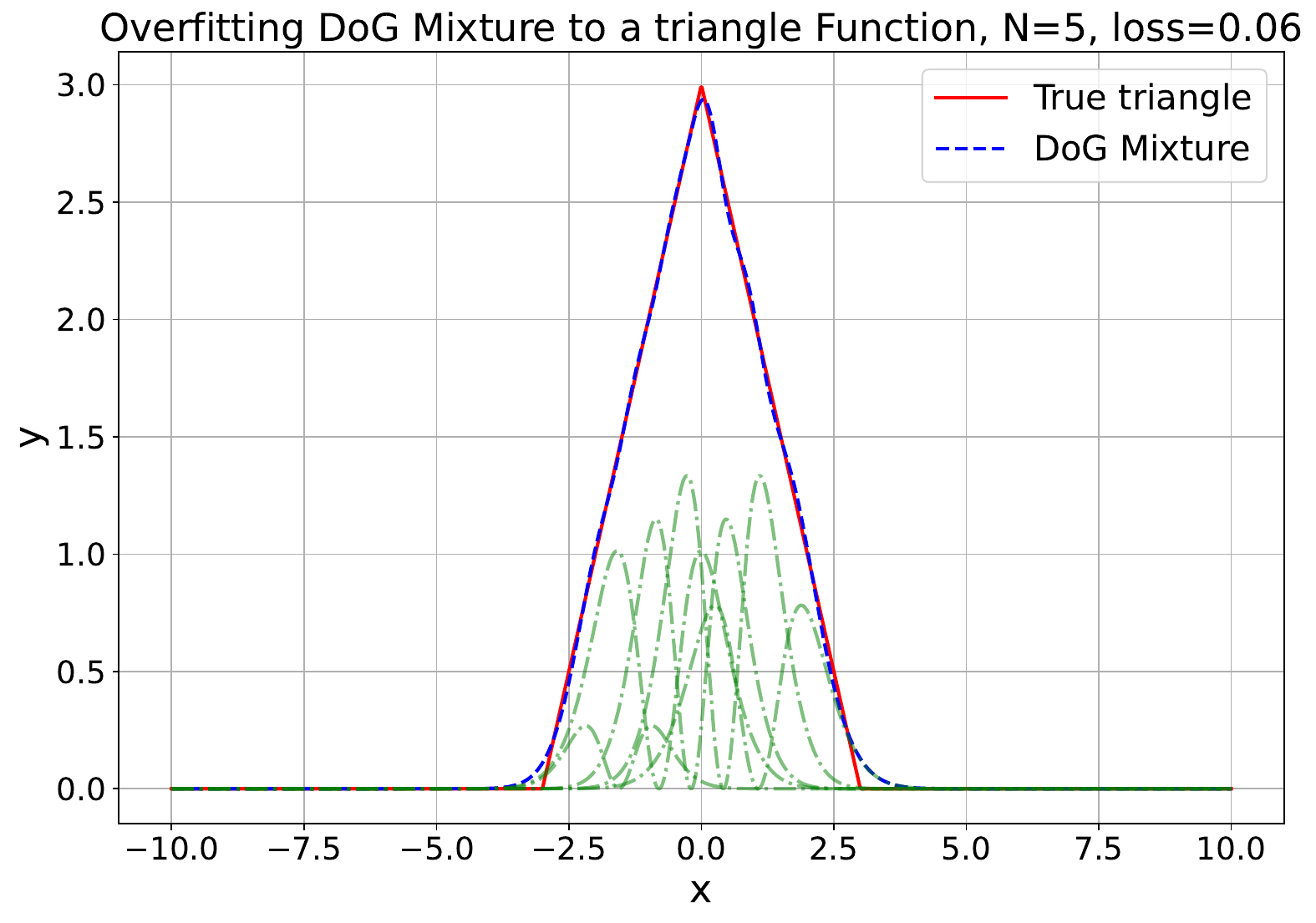} & 
    \includegraphics[width=0.24\linewidth]{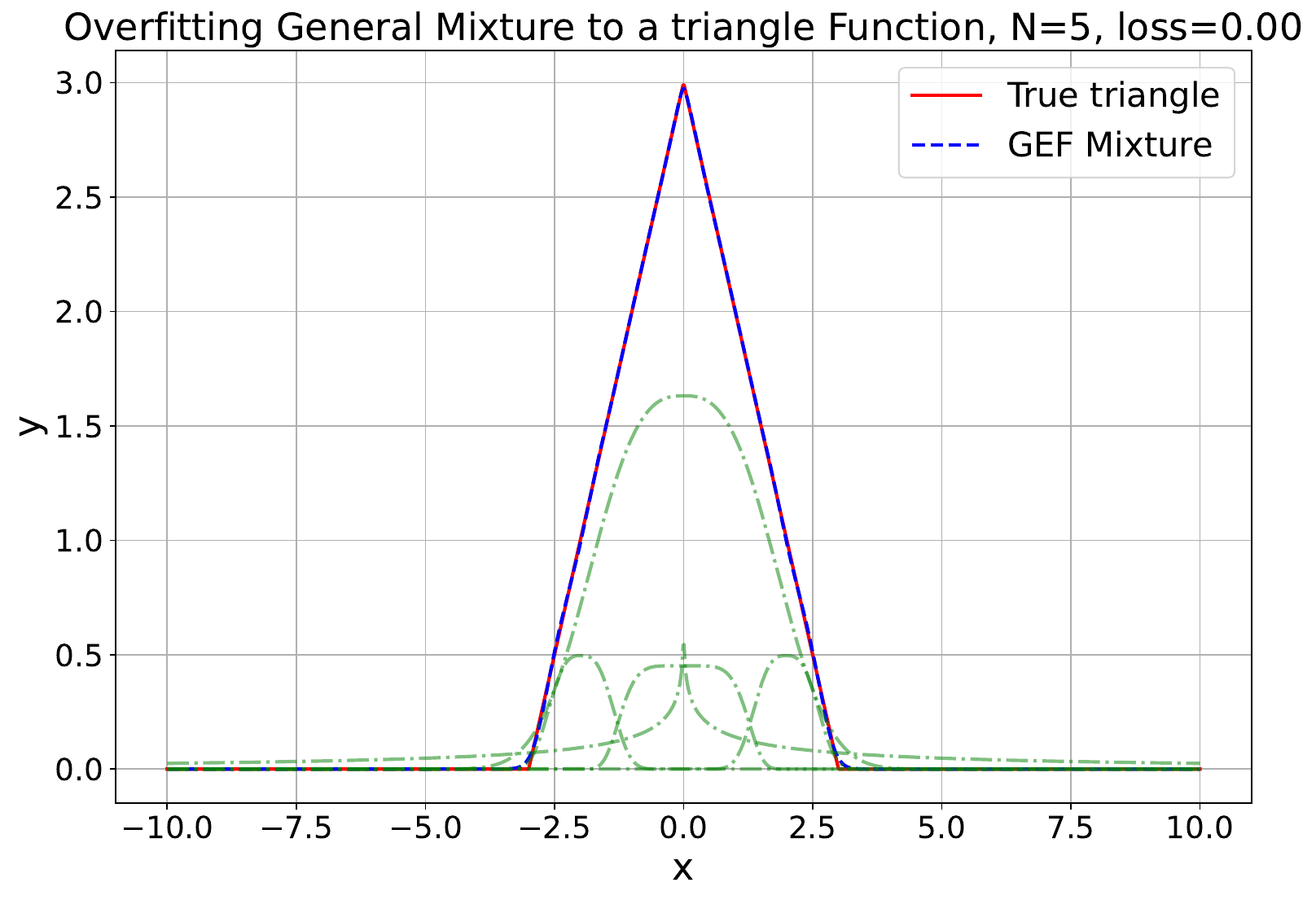}\\ 
    \includegraphics[width=0.24\linewidth]{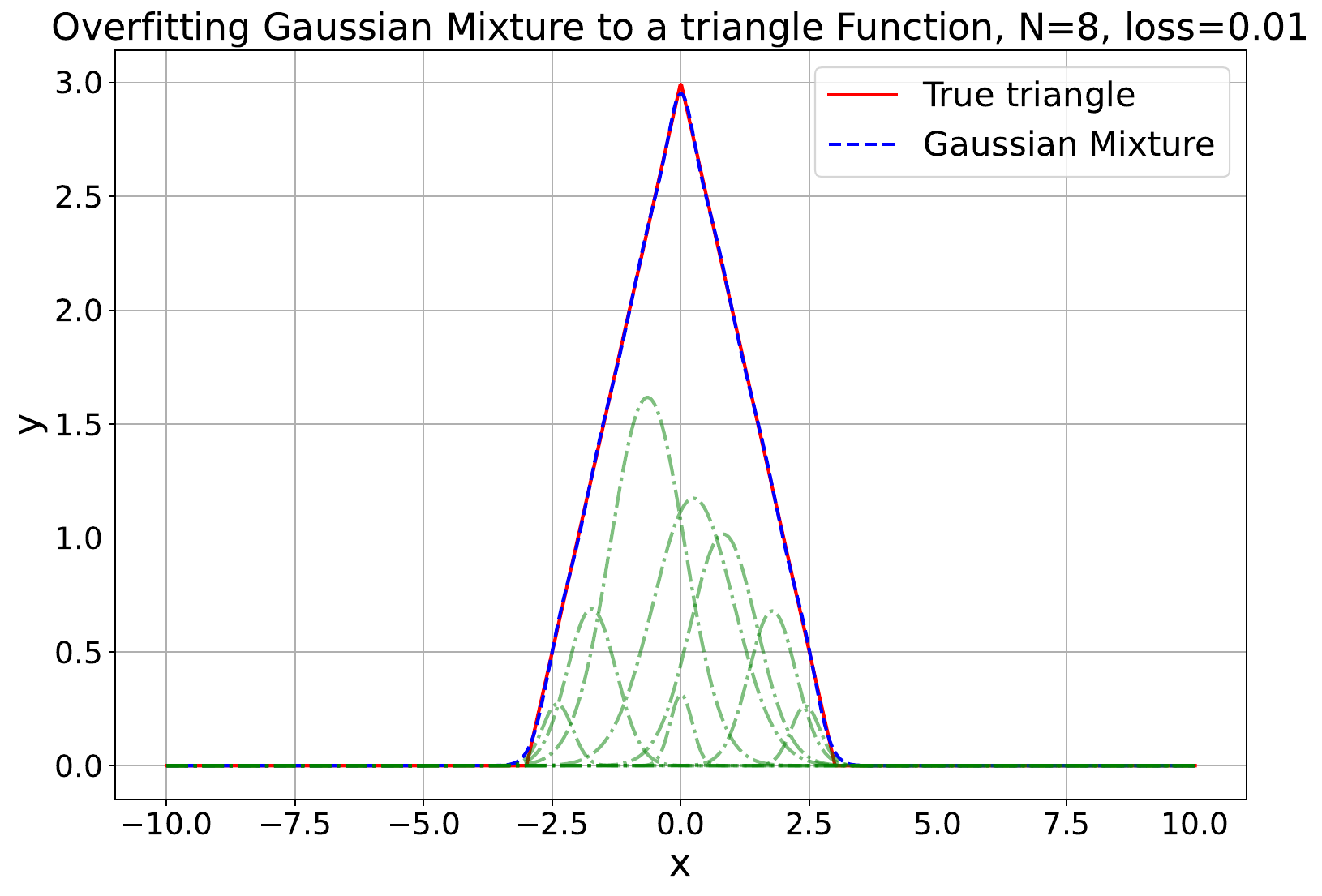} & 
    \includegraphics[width=0.24\linewidth]{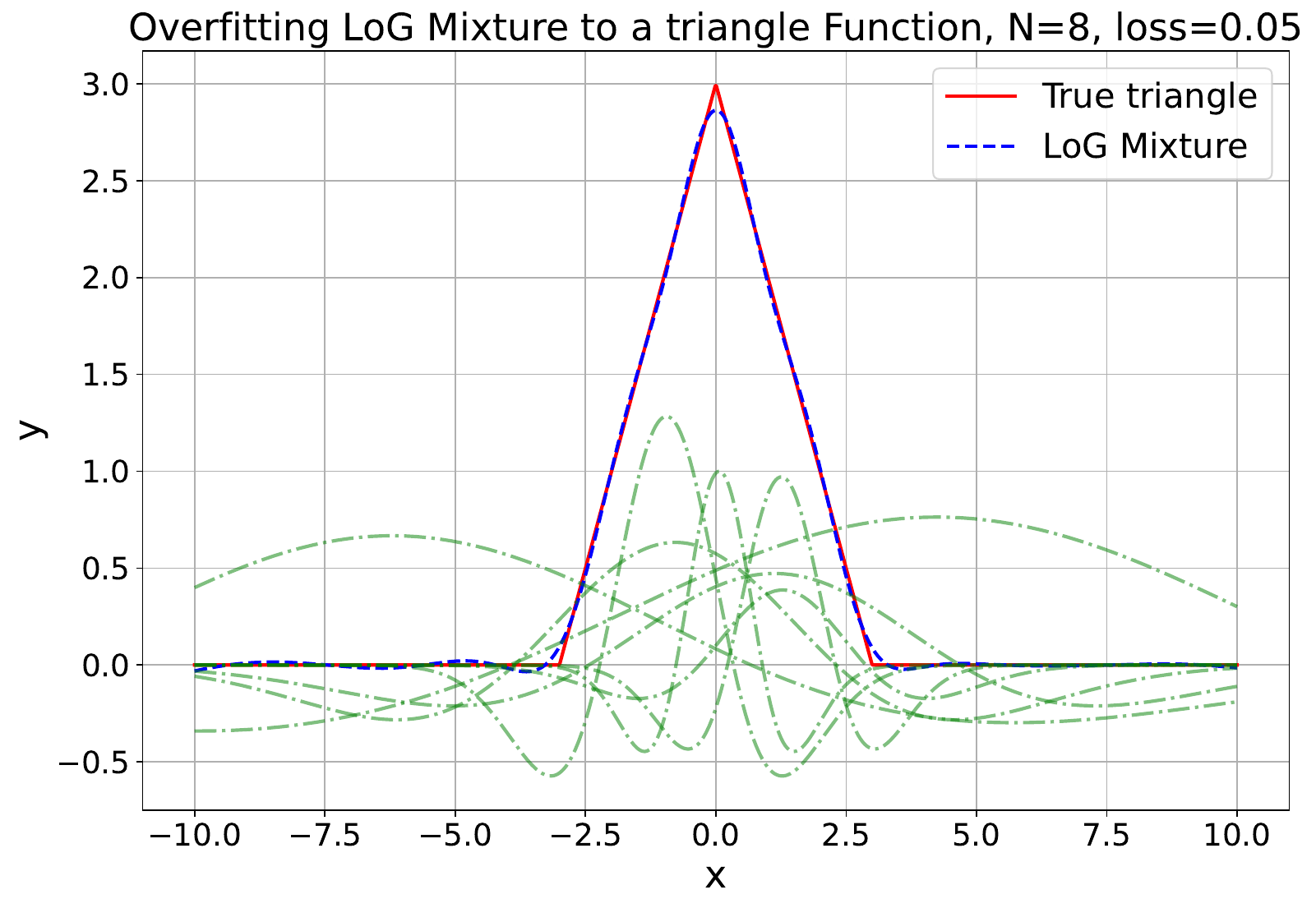} & 
    \includegraphics[width=0.24\linewidth]{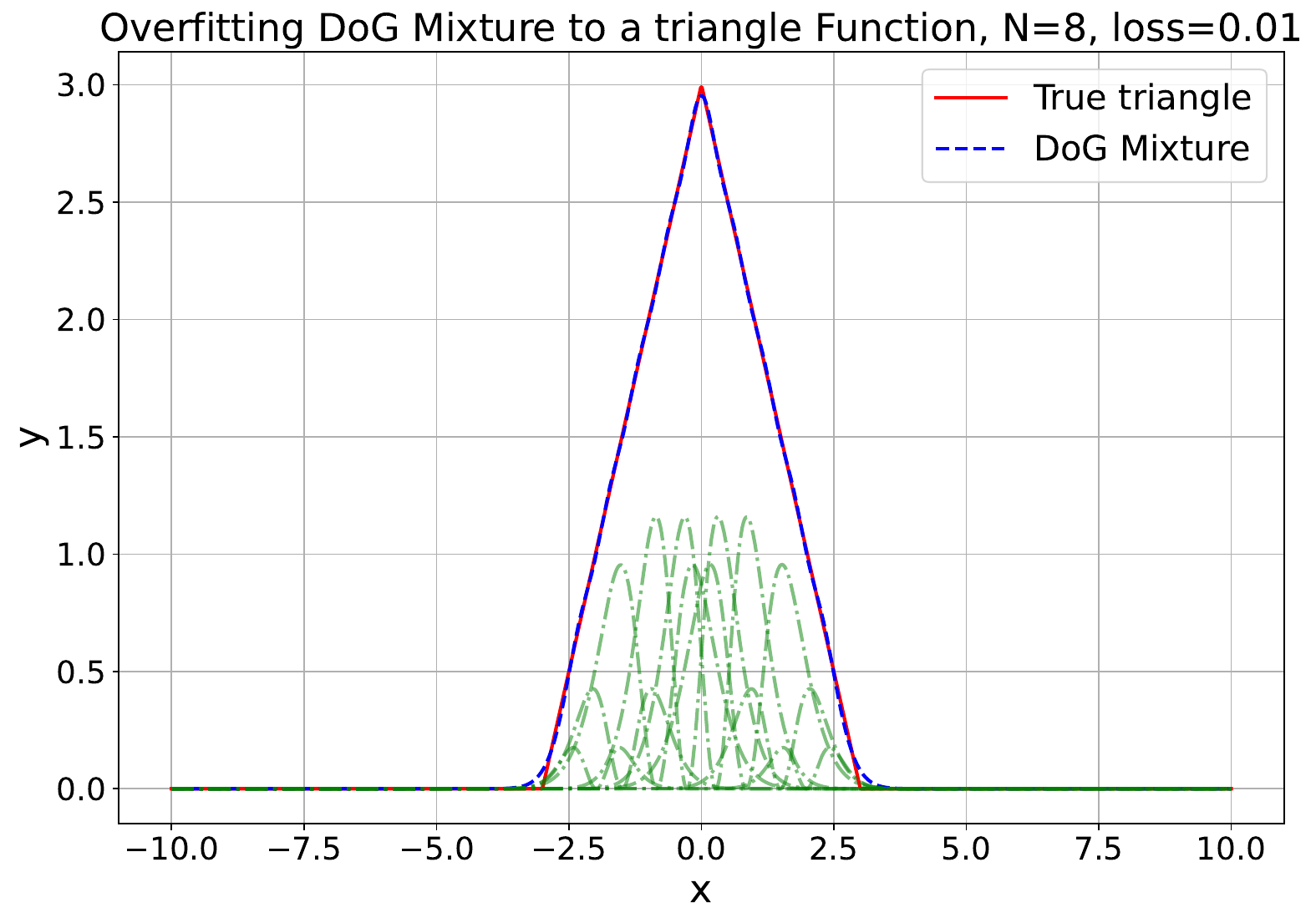} & 
    \includegraphics[width=0.24\linewidth]{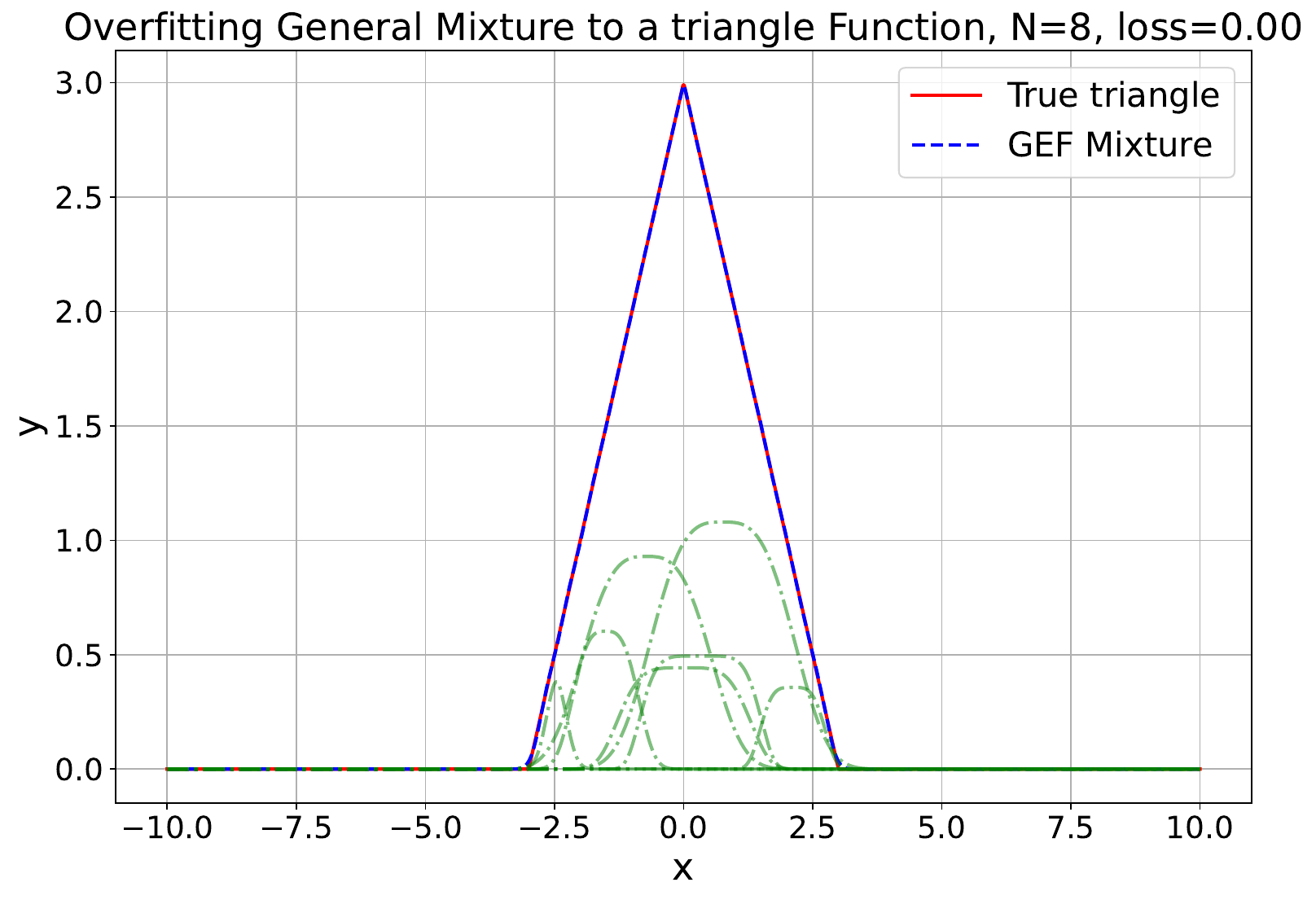}\\ 
    \includegraphics[width=0.24\linewidth]{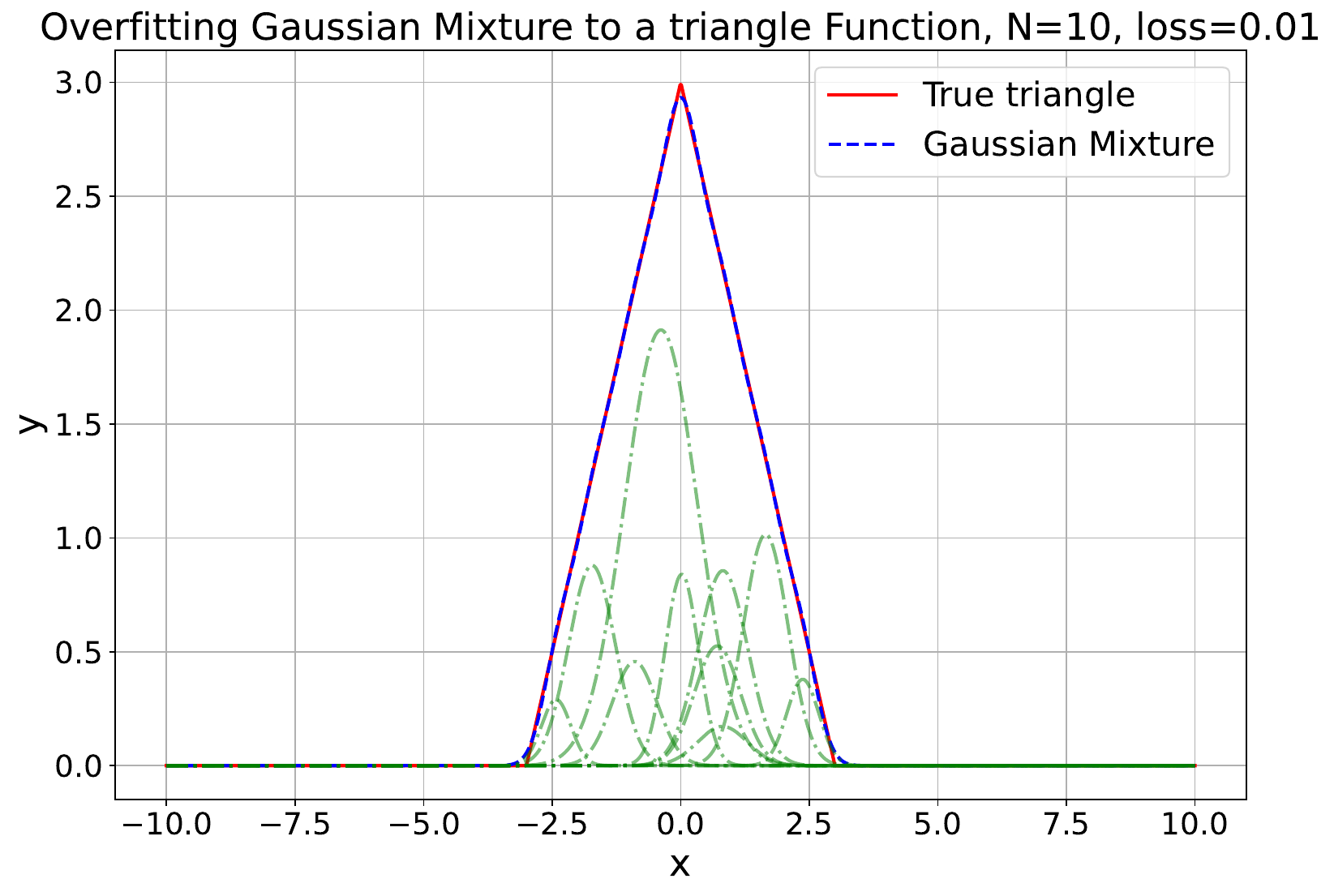} & 
    \includegraphics[width=0.24\linewidth]{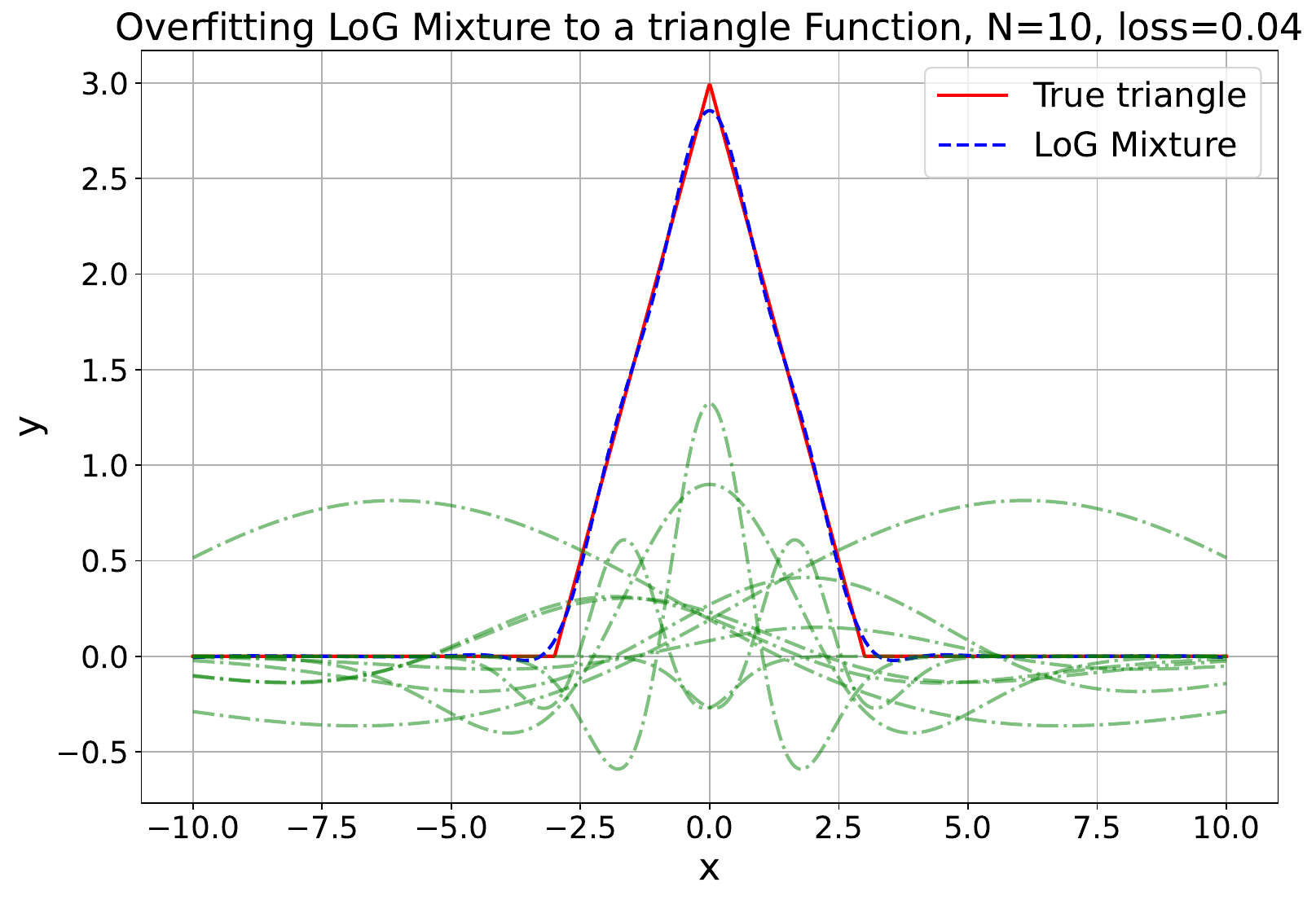} & 
    \includegraphics[width=0.24\linewidth]{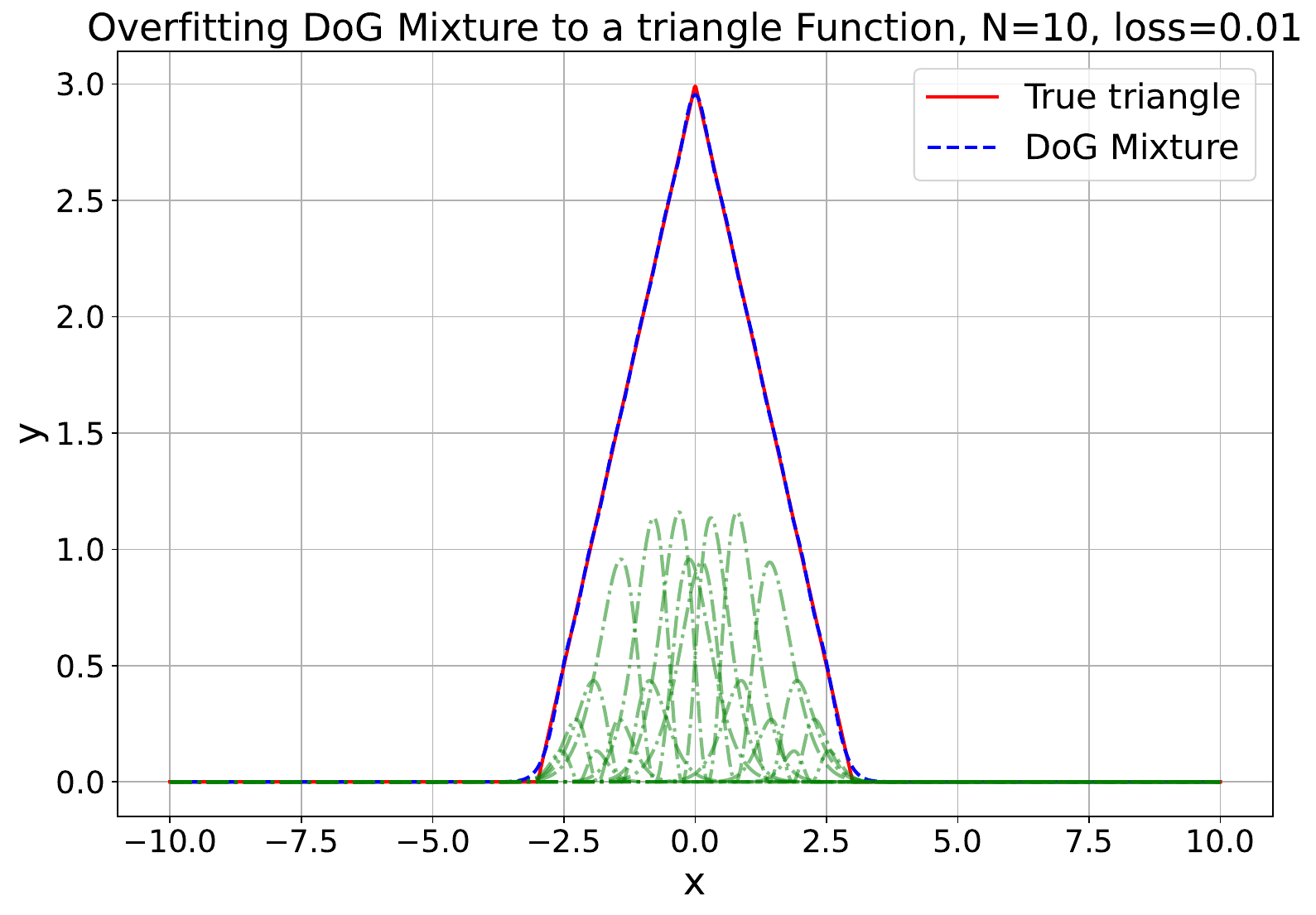} & 
    \includegraphics[width=0.24\linewidth]{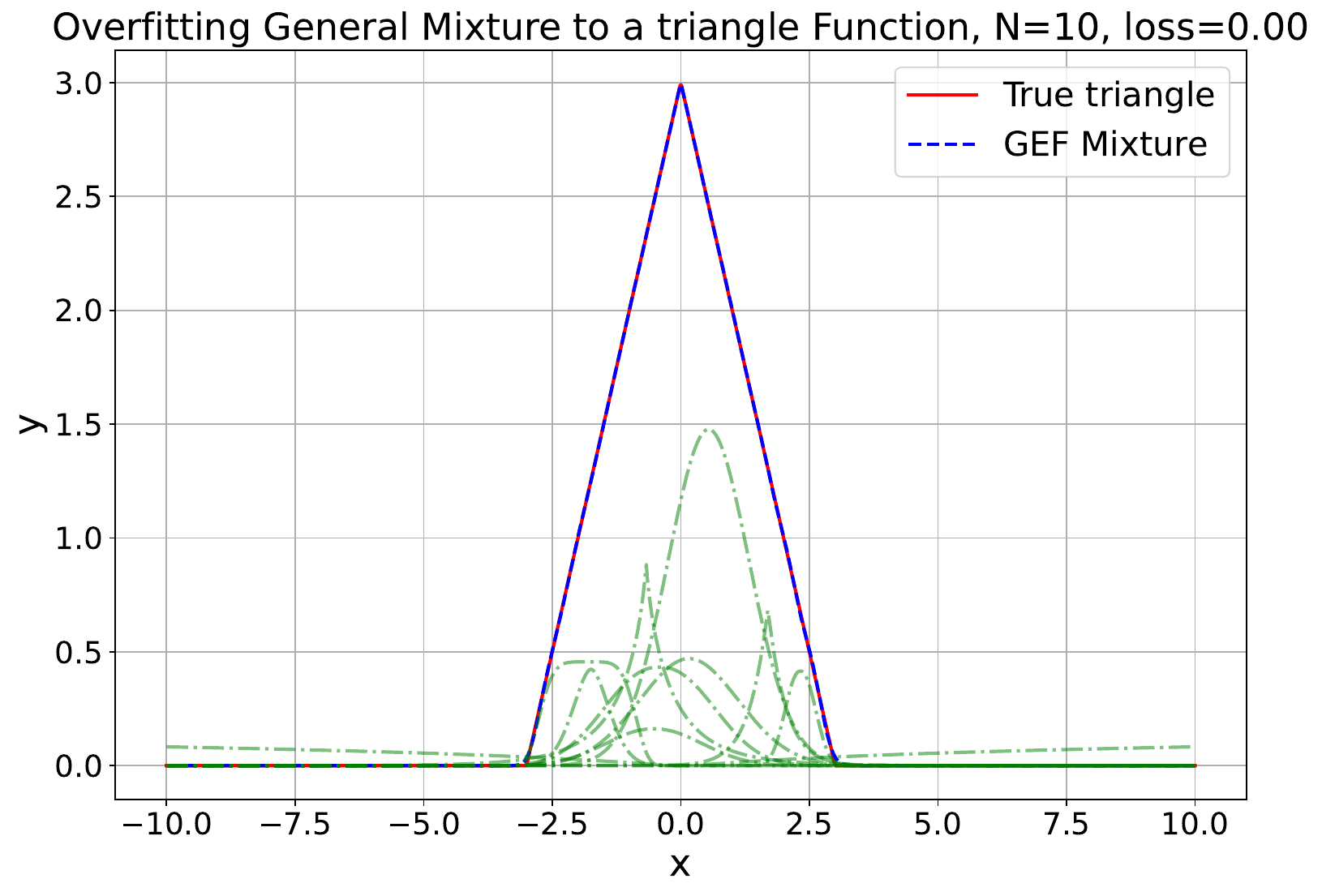}\\ 
    \includegraphics[width=0.24\linewidth]{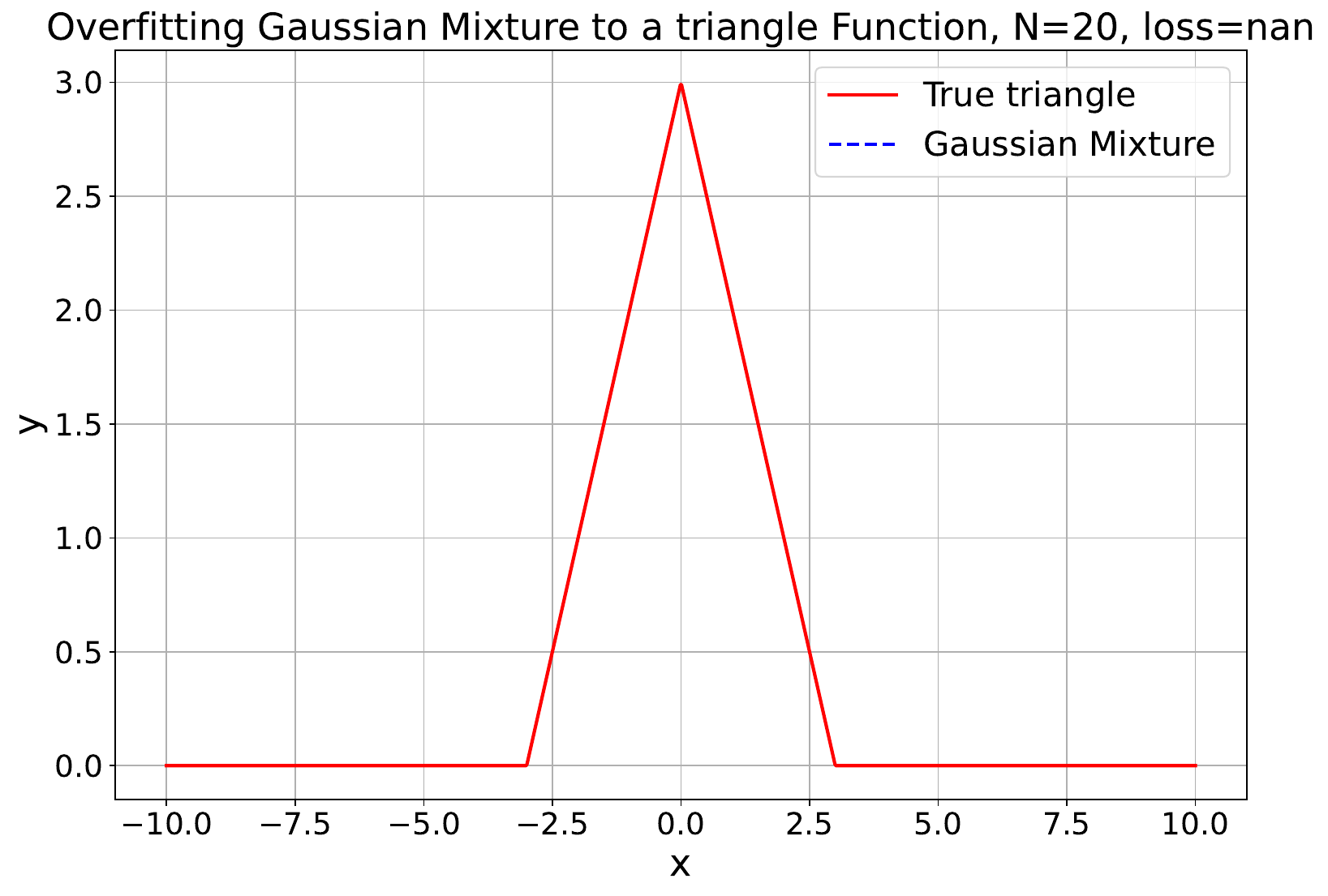} & 
    \includegraphics[width=0.24\linewidth]{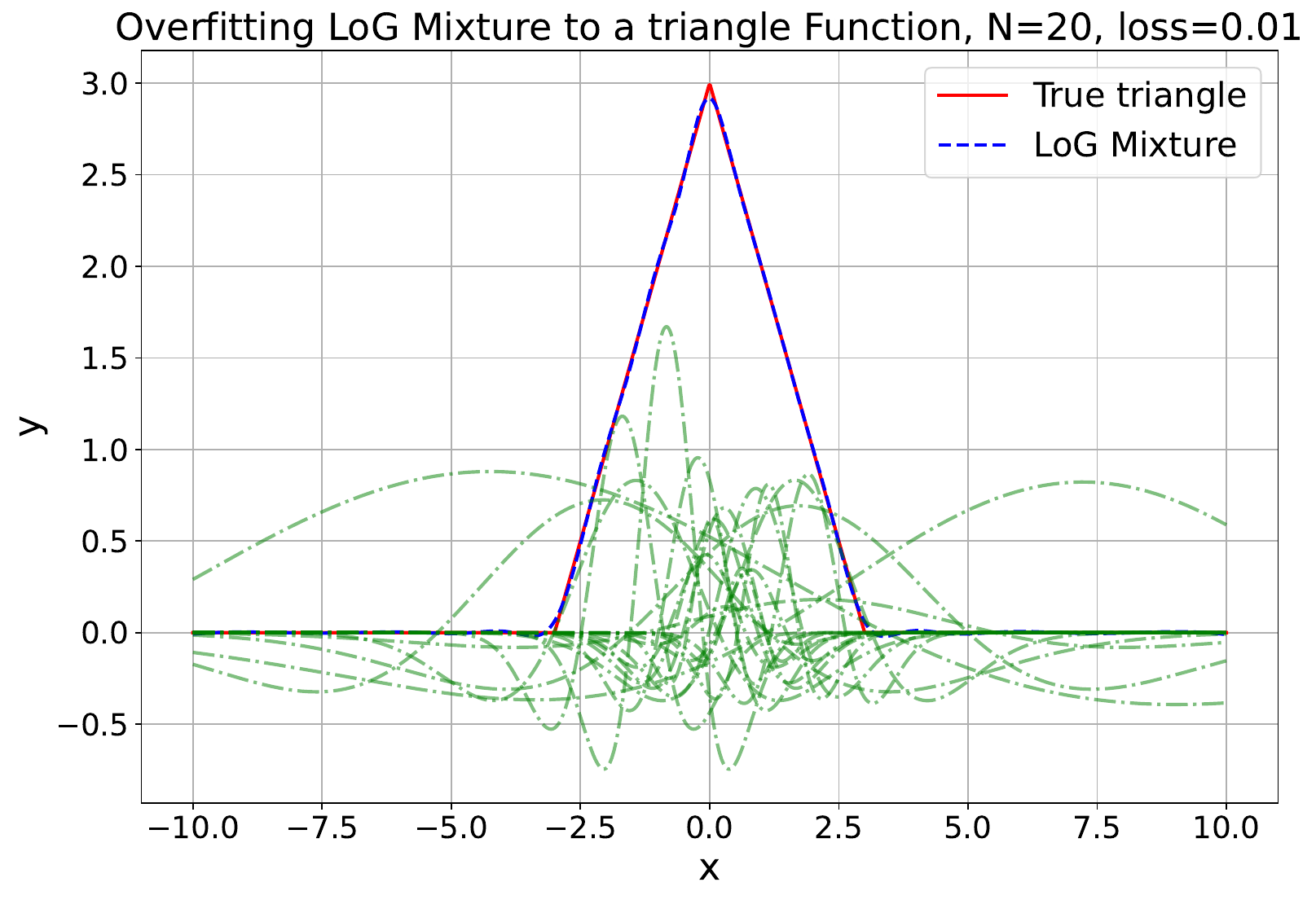} & 
    \includegraphics[width=0.24\linewidth]{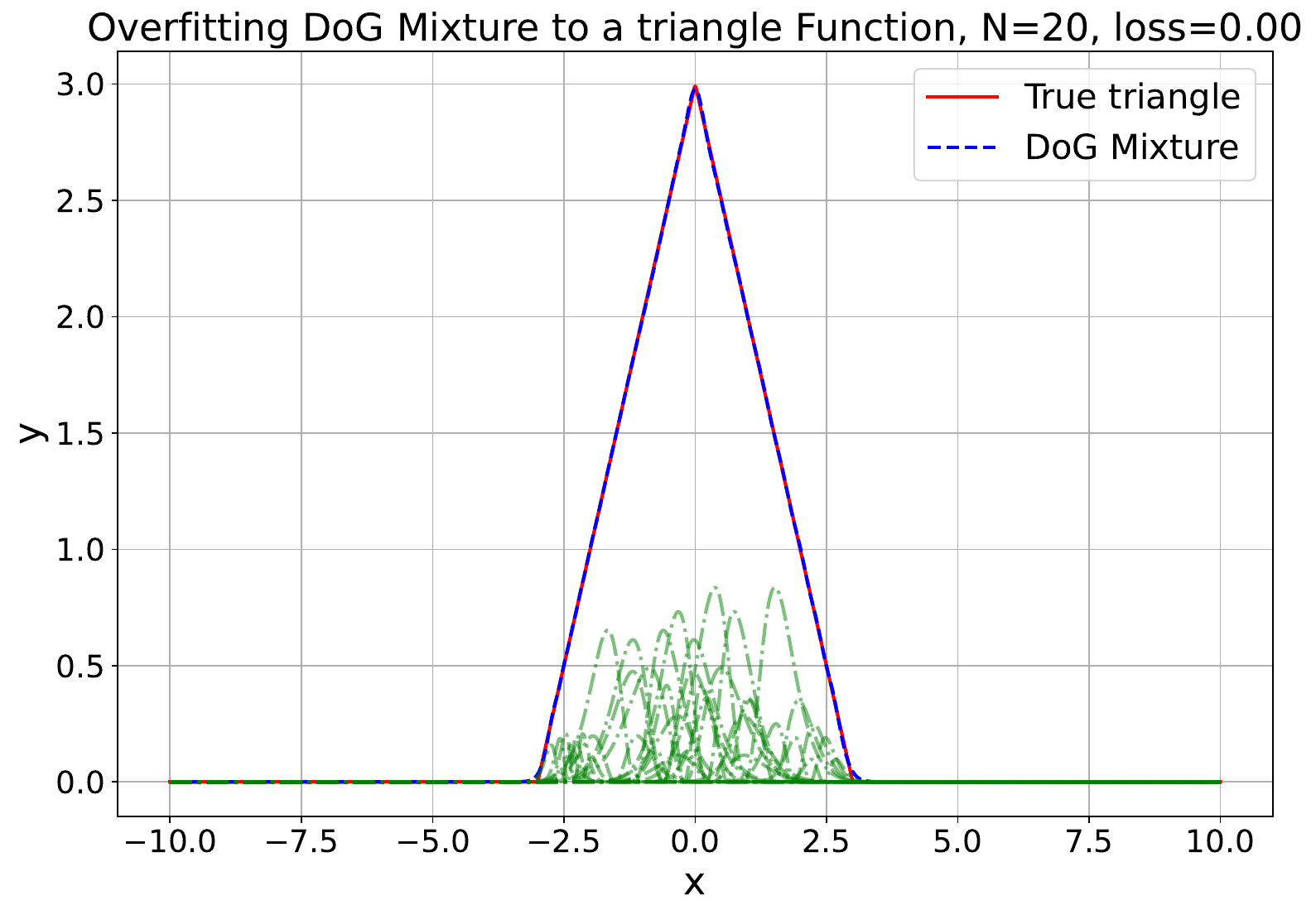} & 
    \includegraphics[width=0.24\linewidth]{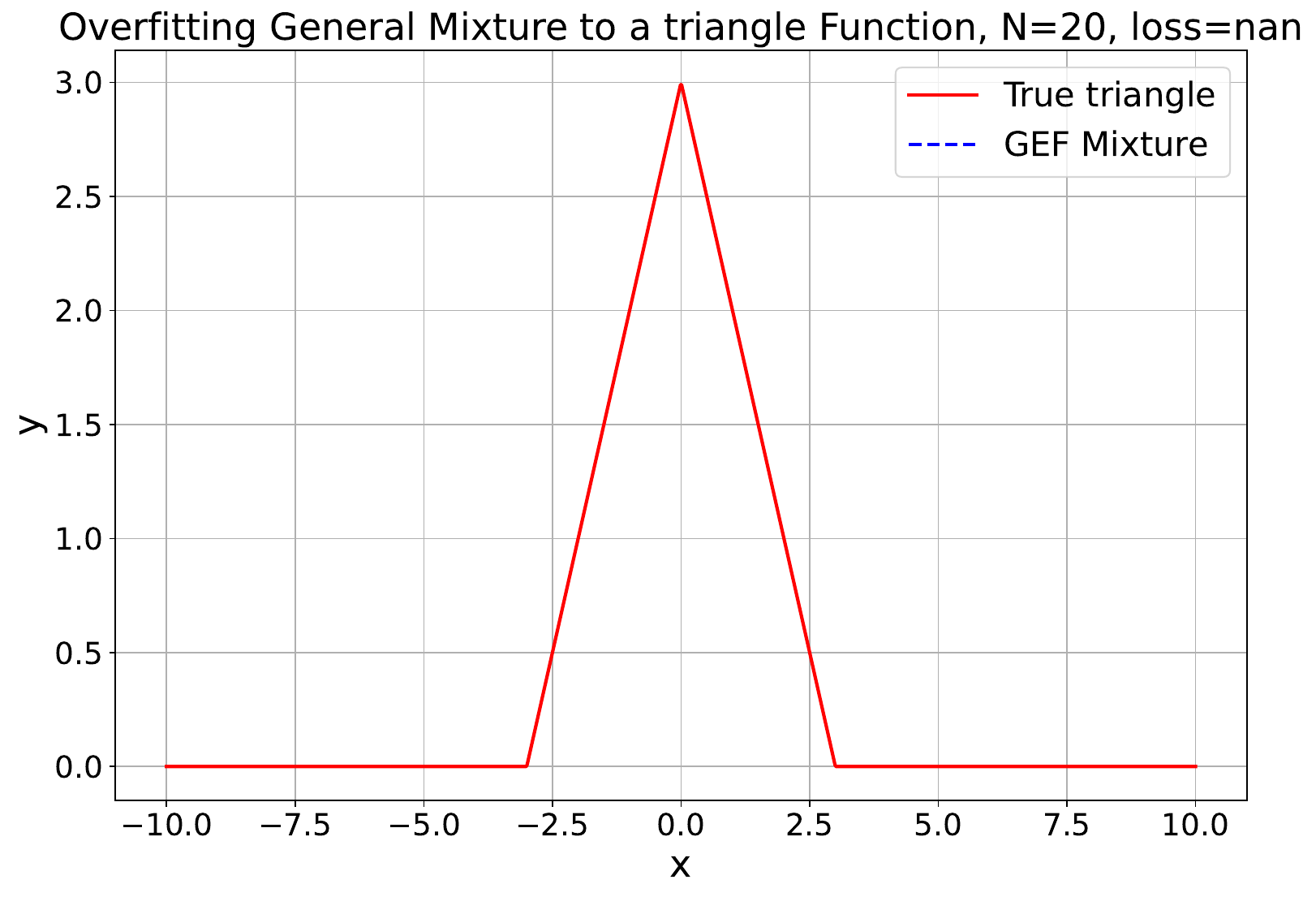}\\ 
    
    \end{tabular}
    }
    \caption{\textbf{Numerical Simulation Examples of Fitting Triangles with Positive Weights Mixtures ( N= 2, 5, 8, and 10 )}. We show some fitting examples for triangle signals with positive weight mixtures. The four mixtures used from left to right are Gaussians, LoG, DoG, and General mixtures. From top to bottom: N = 2, 8, and 10 components. The optimized individual components are shown in green. Some examples fail to optimize due to numerical instability in both Gaussians and GEF mixtures. Note that GEF is very efficient in fitting the triangle with few components while LoG and DoG are more stable for a larger number of components. }
    \label{supfig:fitting_triangle_p}
    \end{figure*}
    

%% file: figures/fitting/fitting_traingle_n.tex
\begin{figure*}[h]
    \centering
    \resizebox{1.0\linewidth}{!}{
    \begin{tabular}{cccc}
    \tabcolsep=0.01cm
    Gaussian Mixture& LoG Mixture & DoG Mixture & GEF Mixture \\ 
    \includegraphics[width=0.24\linewidth]{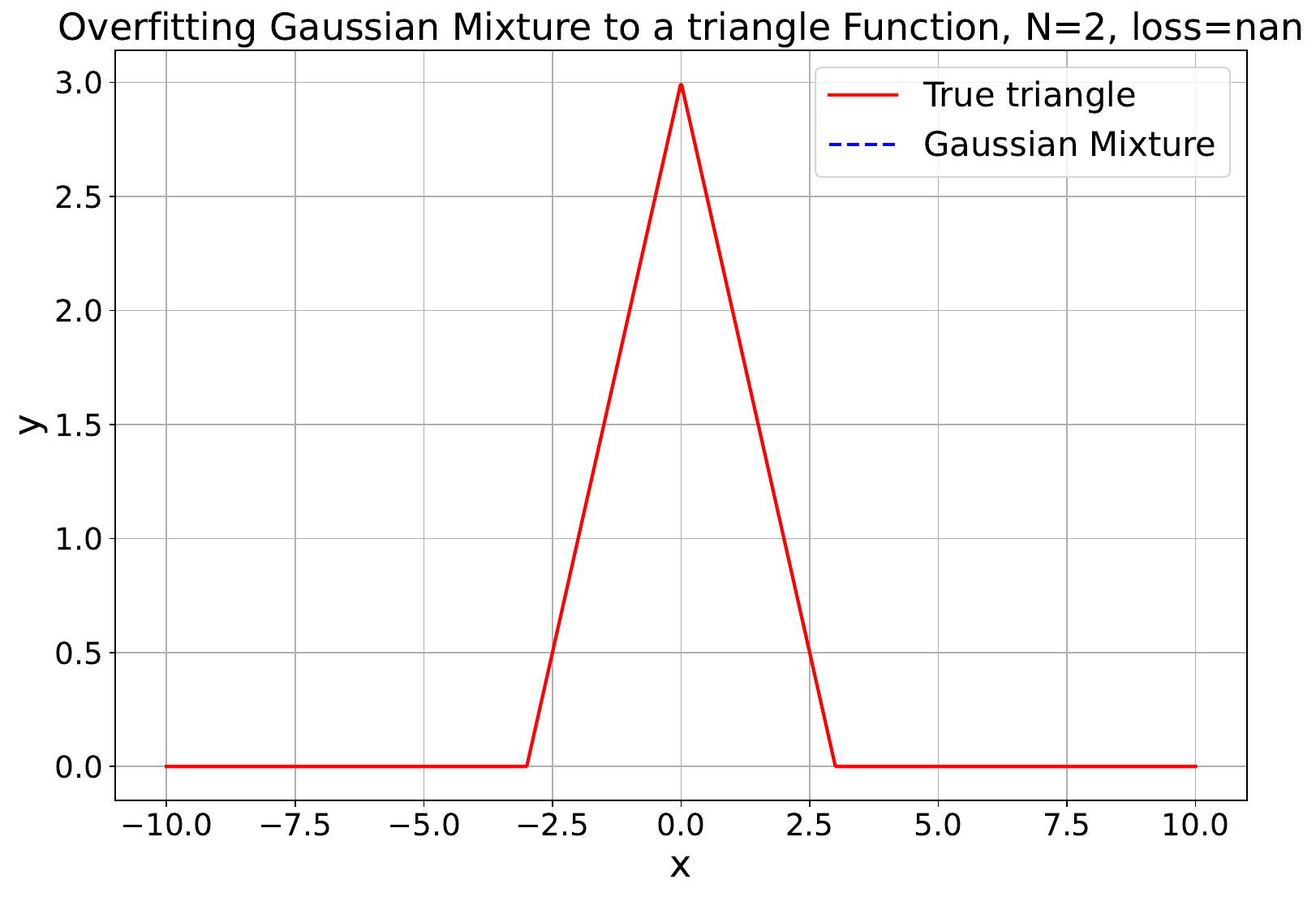} & 
    \includegraphics[width=0.24\linewidth]{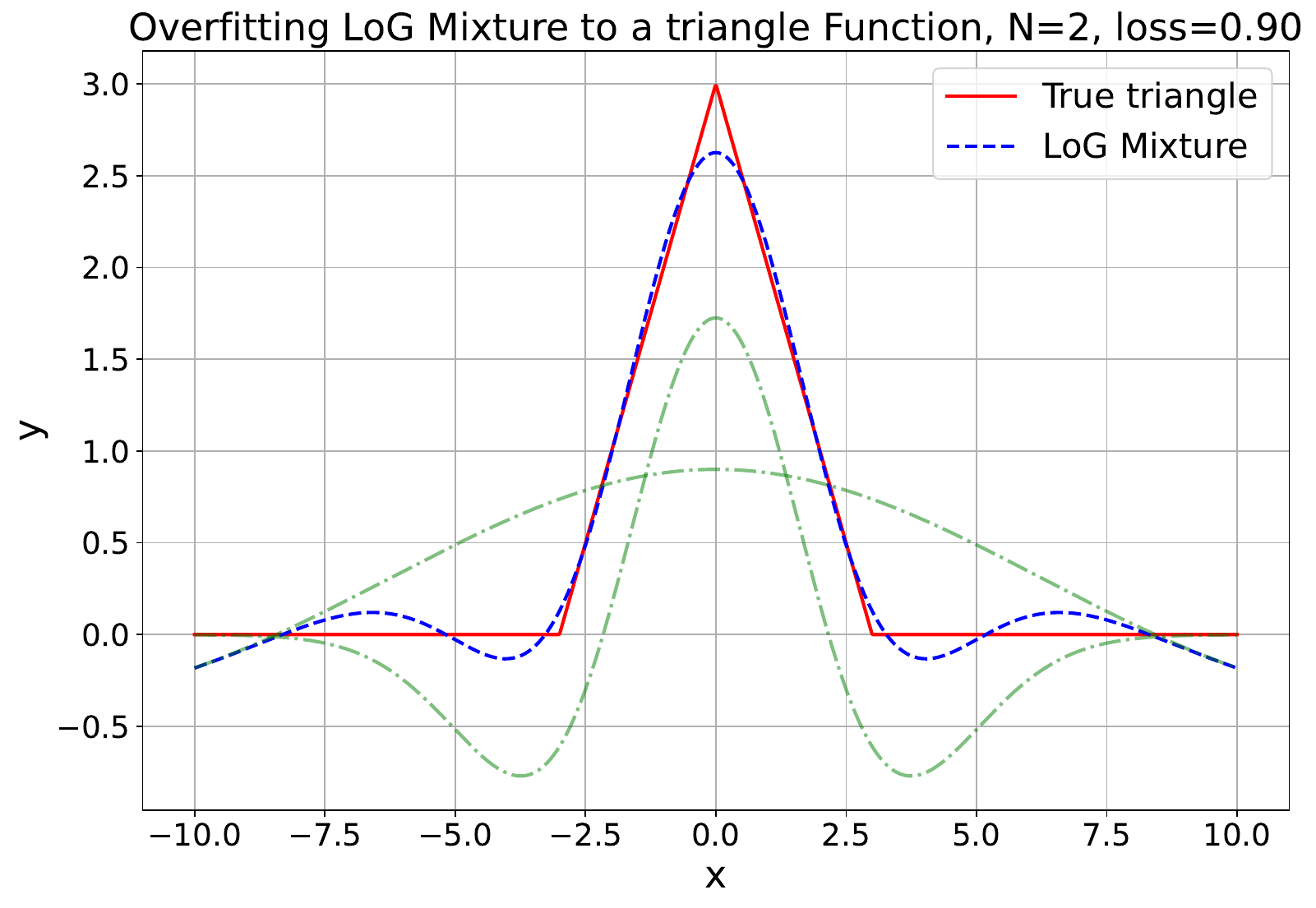} & 
    \includegraphics[width=0.24\linewidth]{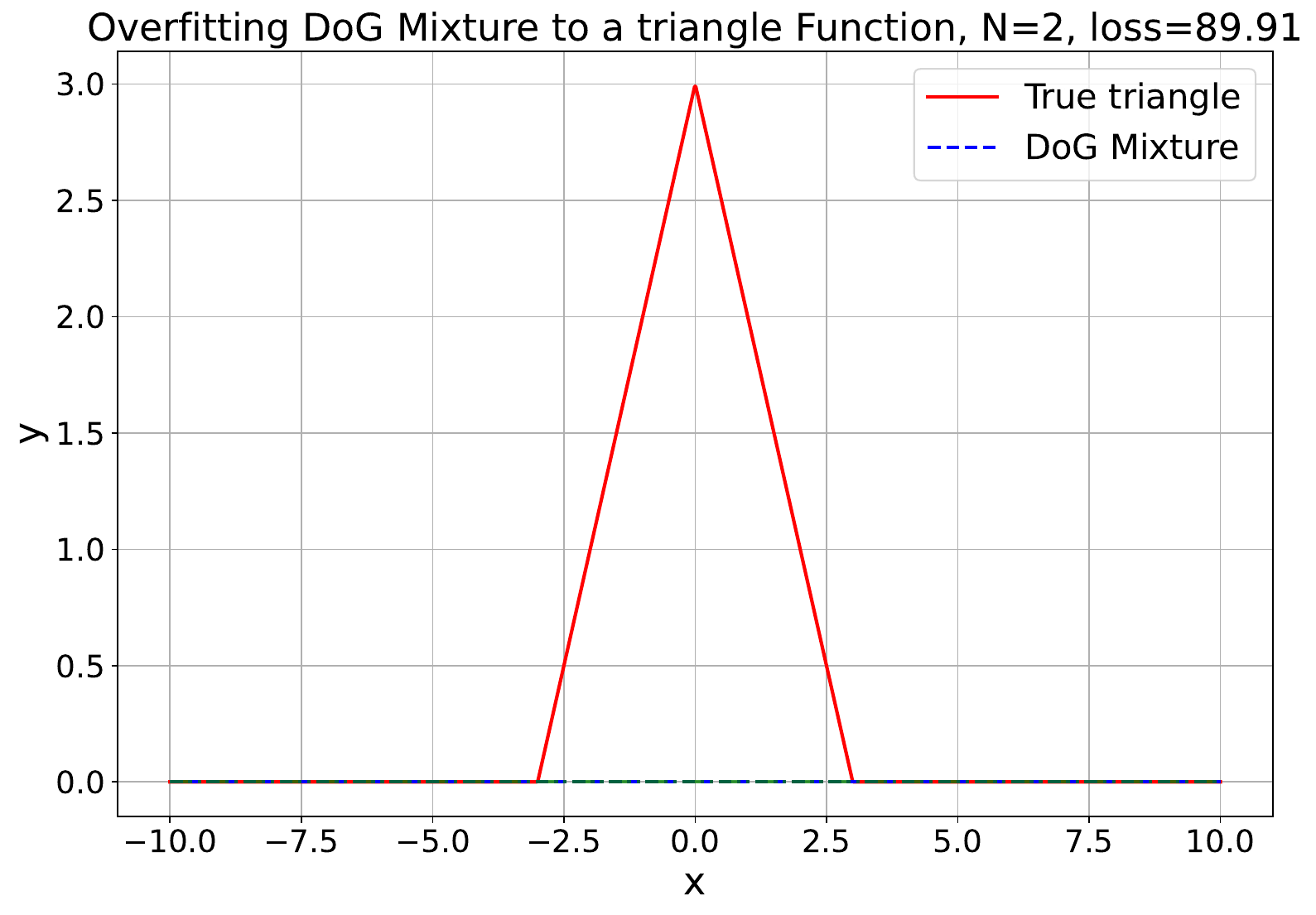} & 
    \includegraphics[width=0.24\linewidth]{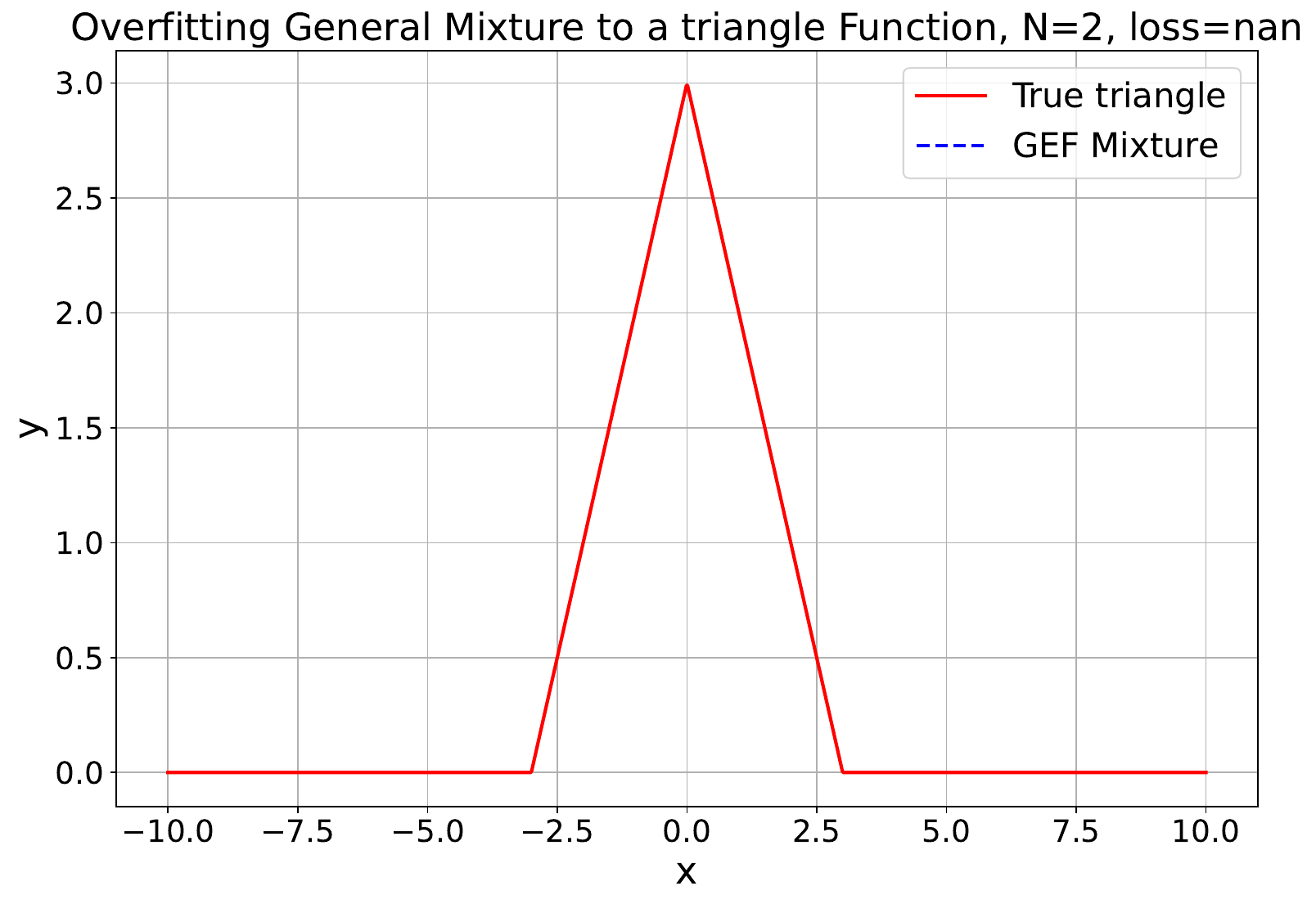}\\ 
    \includegraphics[width=0.24\linewidth]{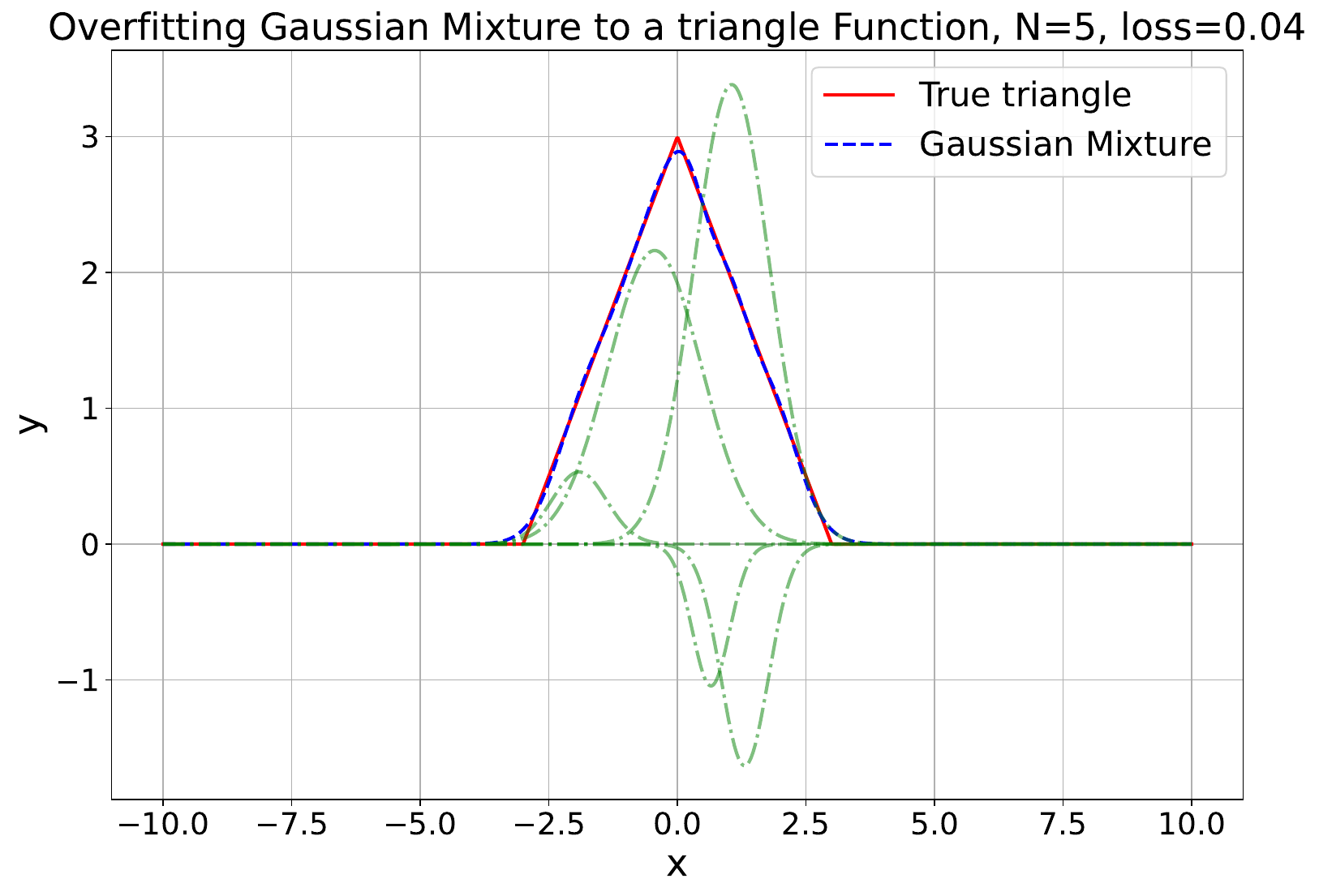} & 
    \includegraphics[width=0.24\linewidth]{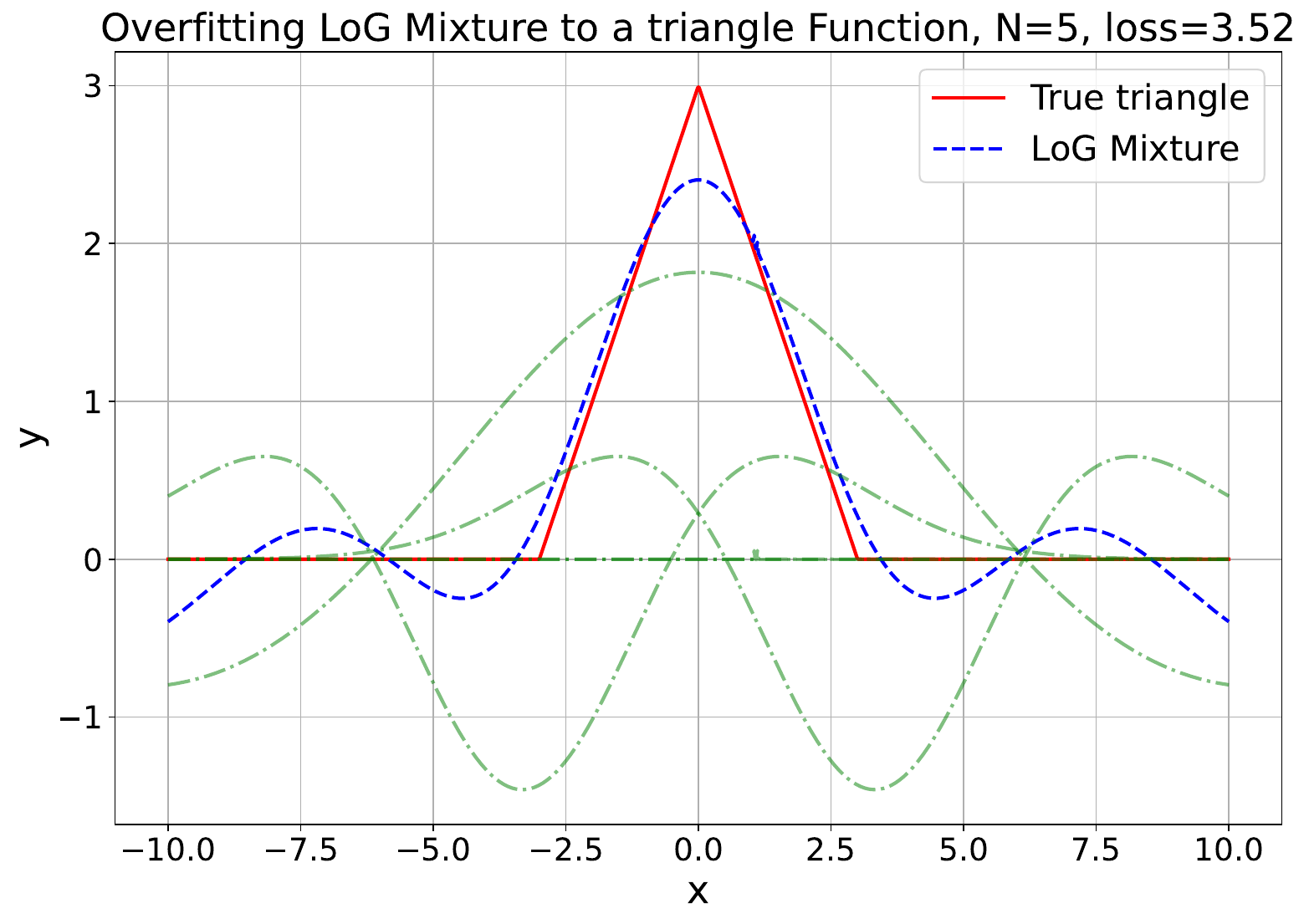} & 
    \includegraphics[width=0.24\linewidth]{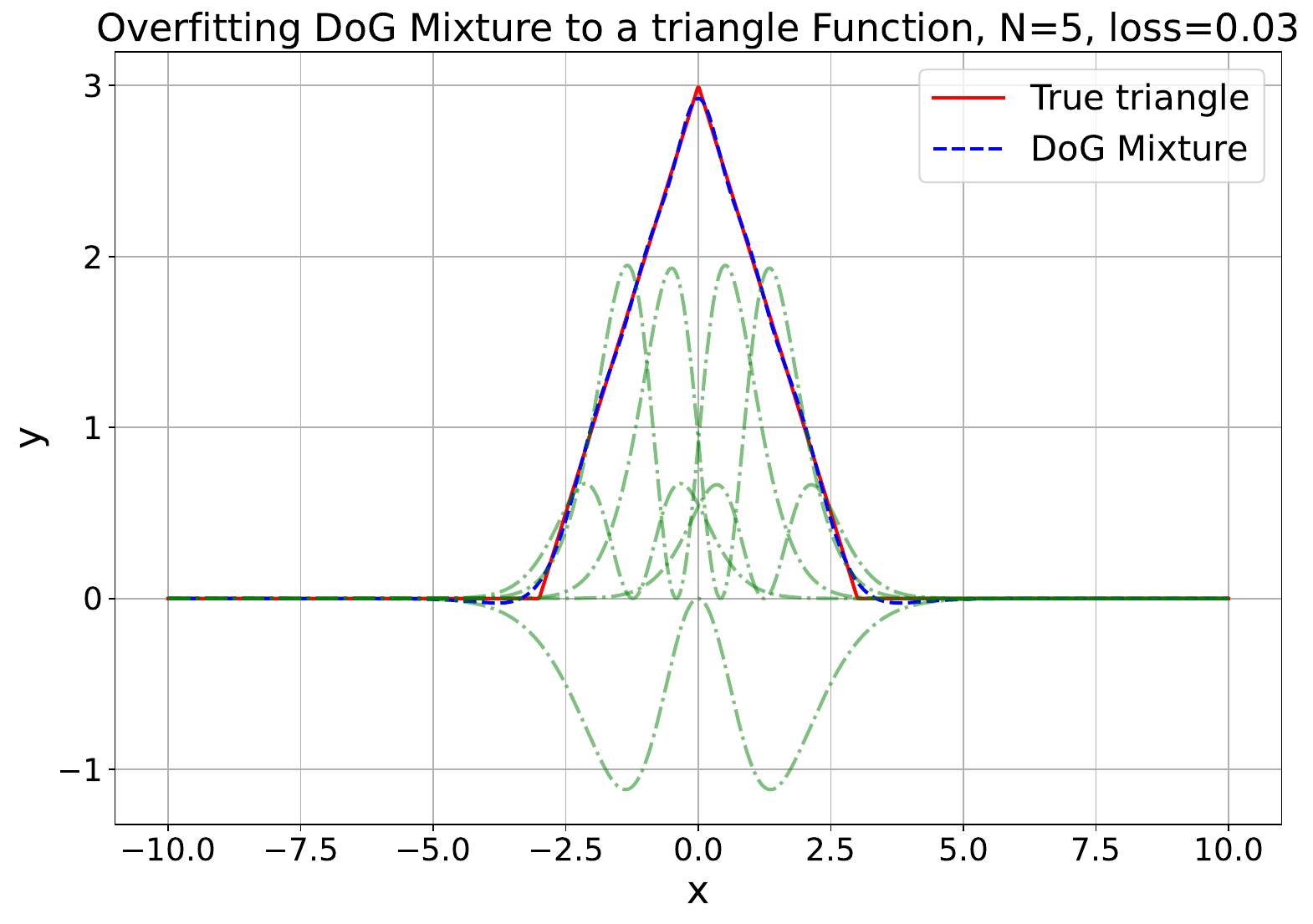} & 
    \includegraphics[width=0.24\linewidth]{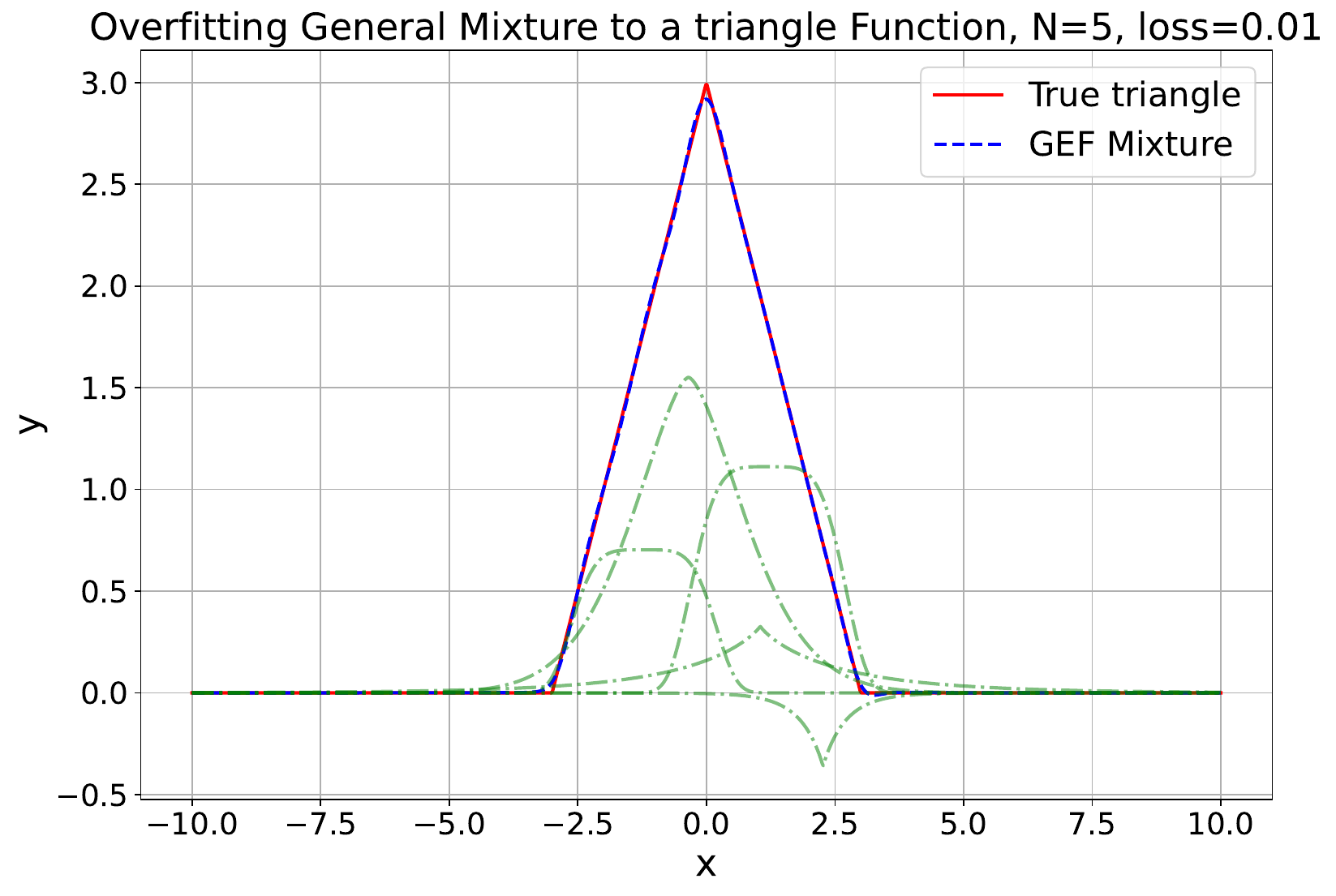}\\ 
    \includegraphics[width=0.24\linewidth]{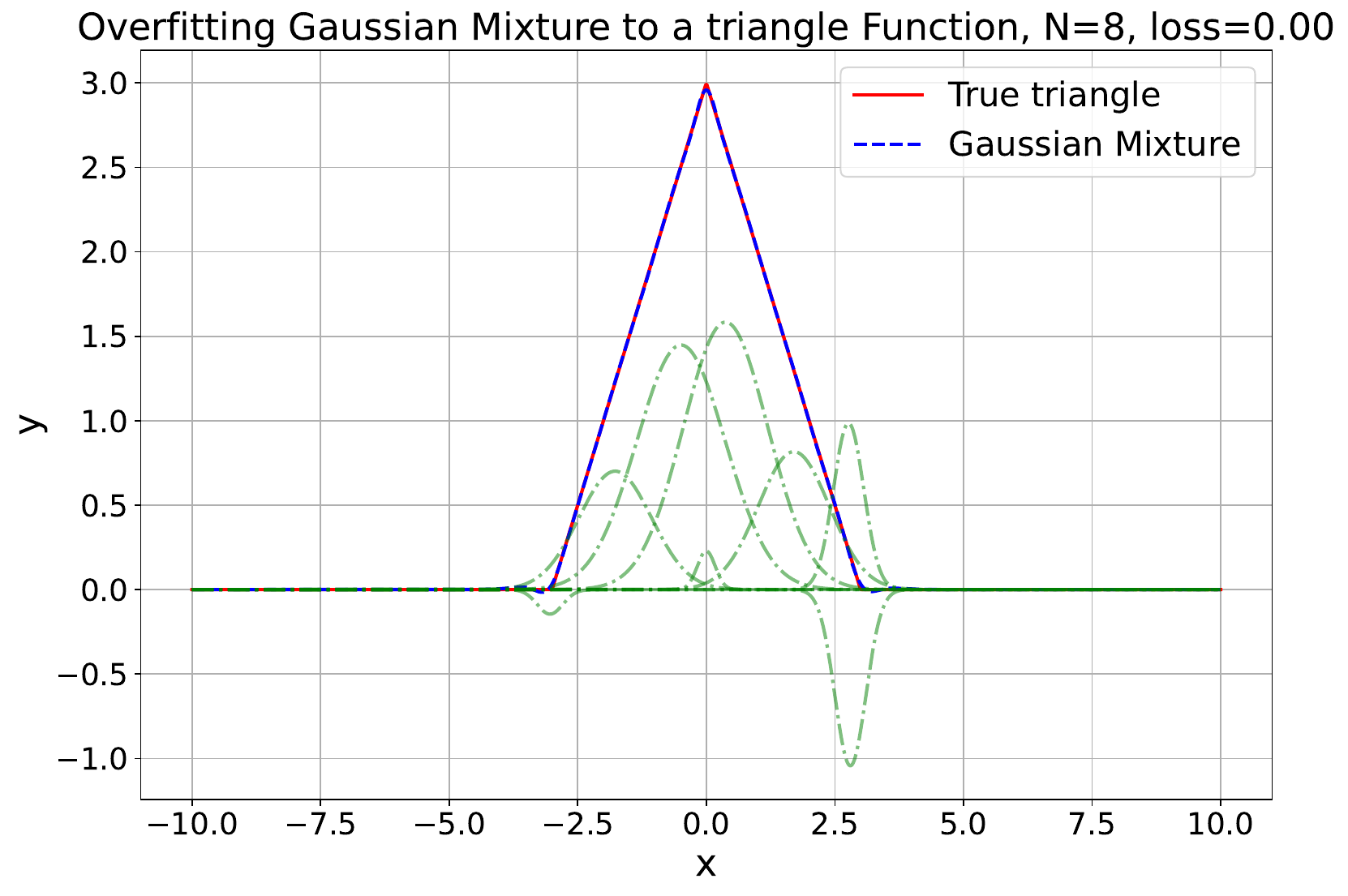} & 
    \includegraphics[width=0.24\linewidth]{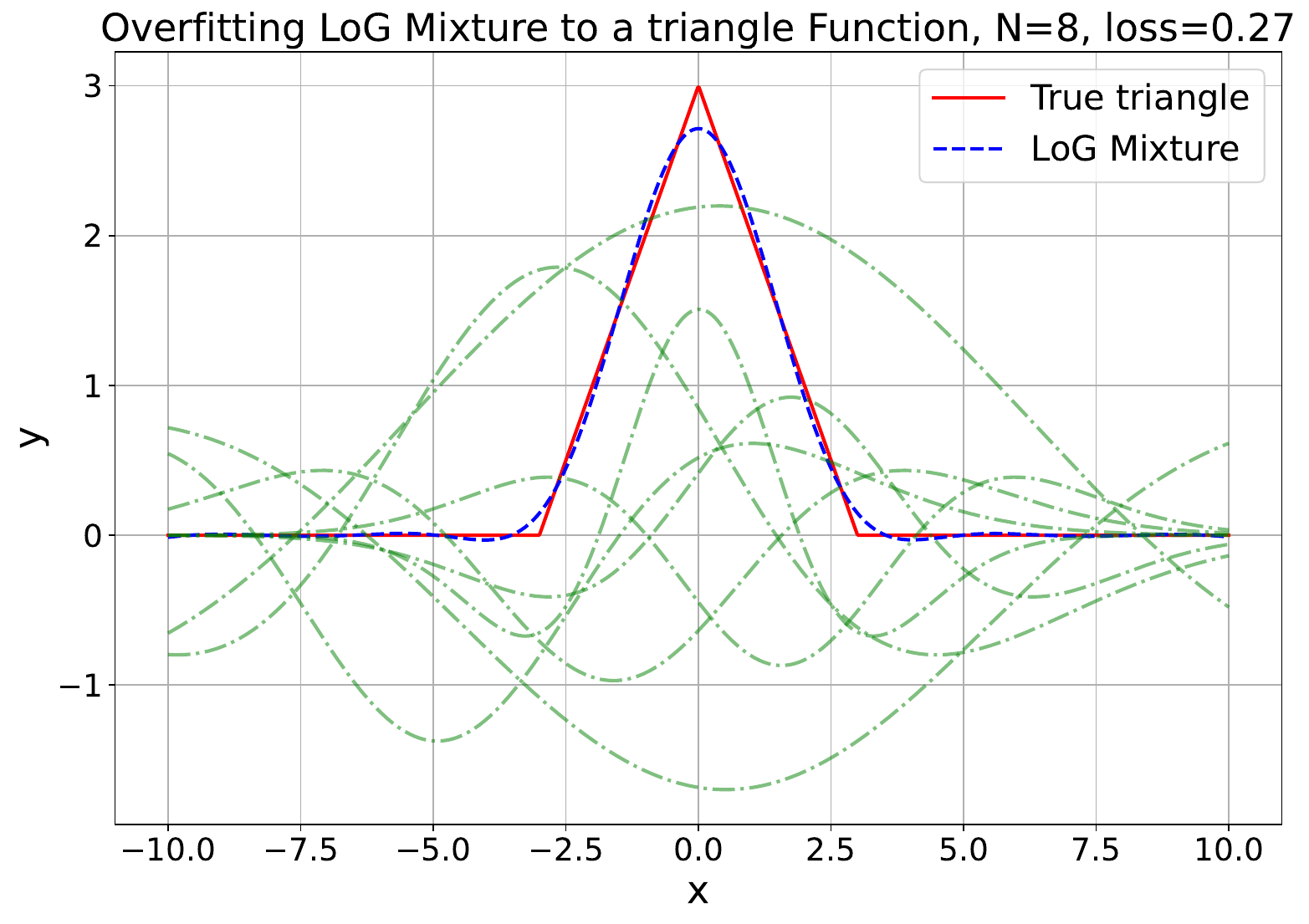} & 
    \includegraphics[width=0.24\linewidth]{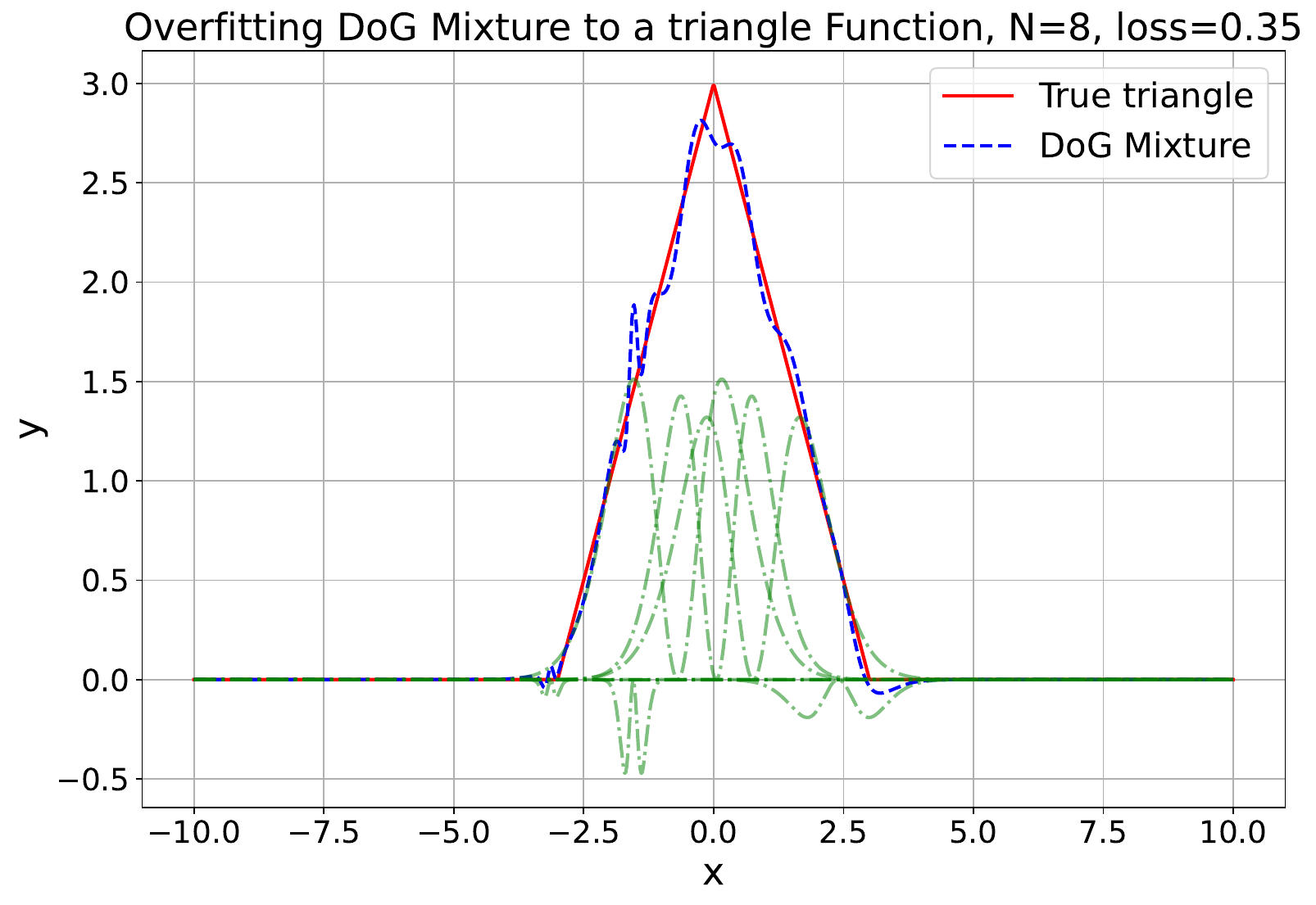} & 
    \includegraphics[width=0.24\linewidth]{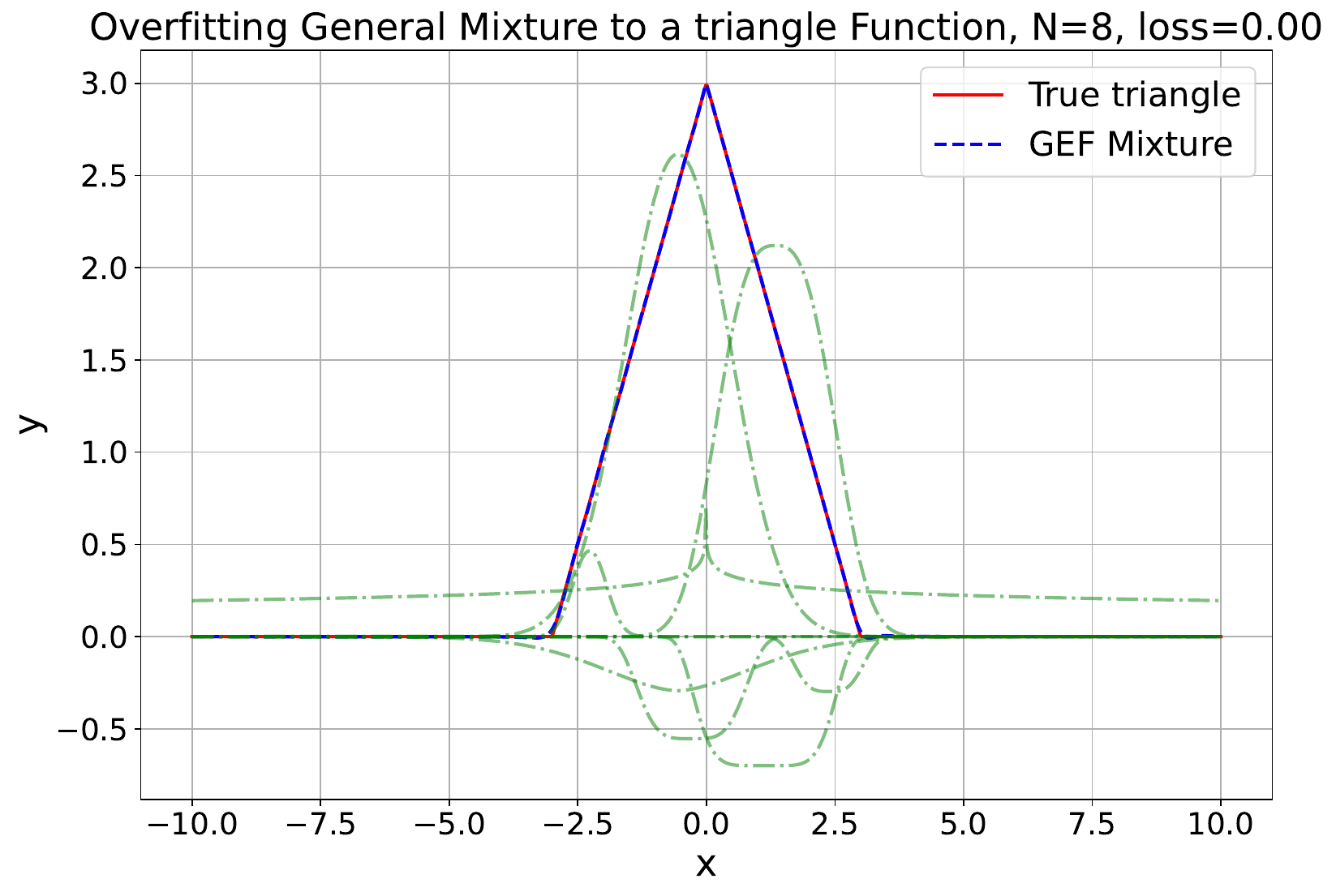}\\ 
    \includegraphics[width=0.24\linewidth]{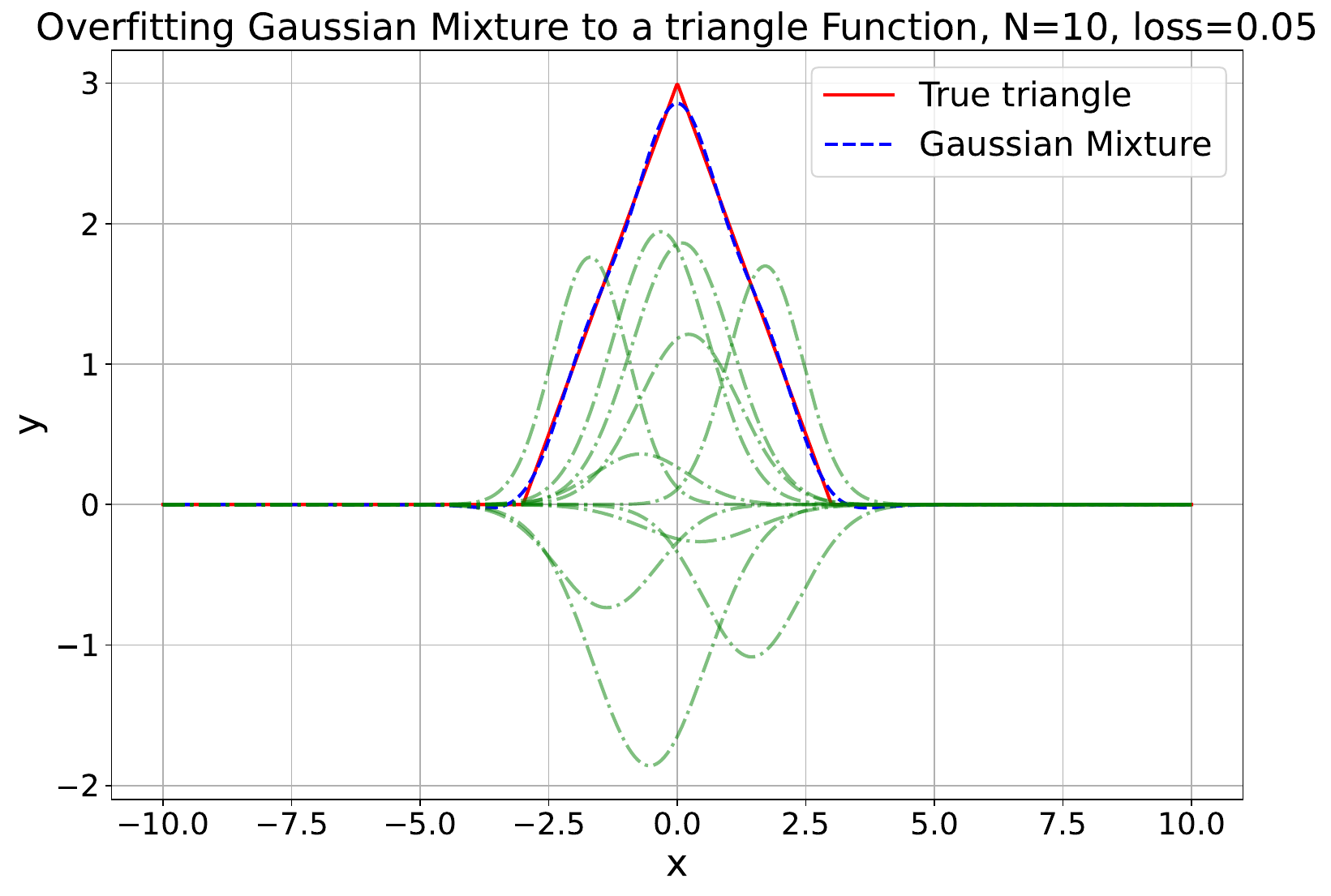} & 
    \includegraphics[width=0.24\linewidth]{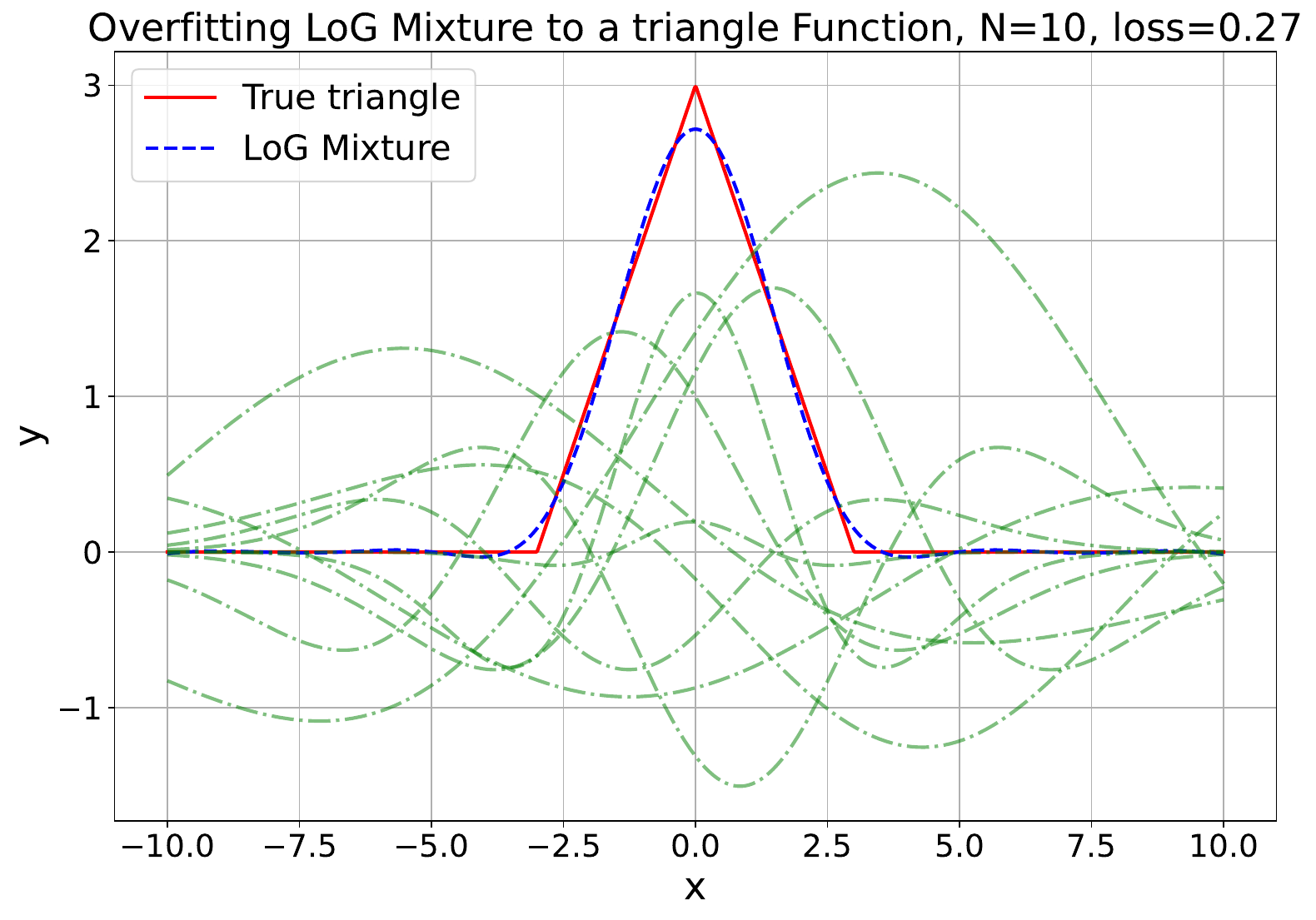} & 
    \includegraphics[width=0.24\linewidth]{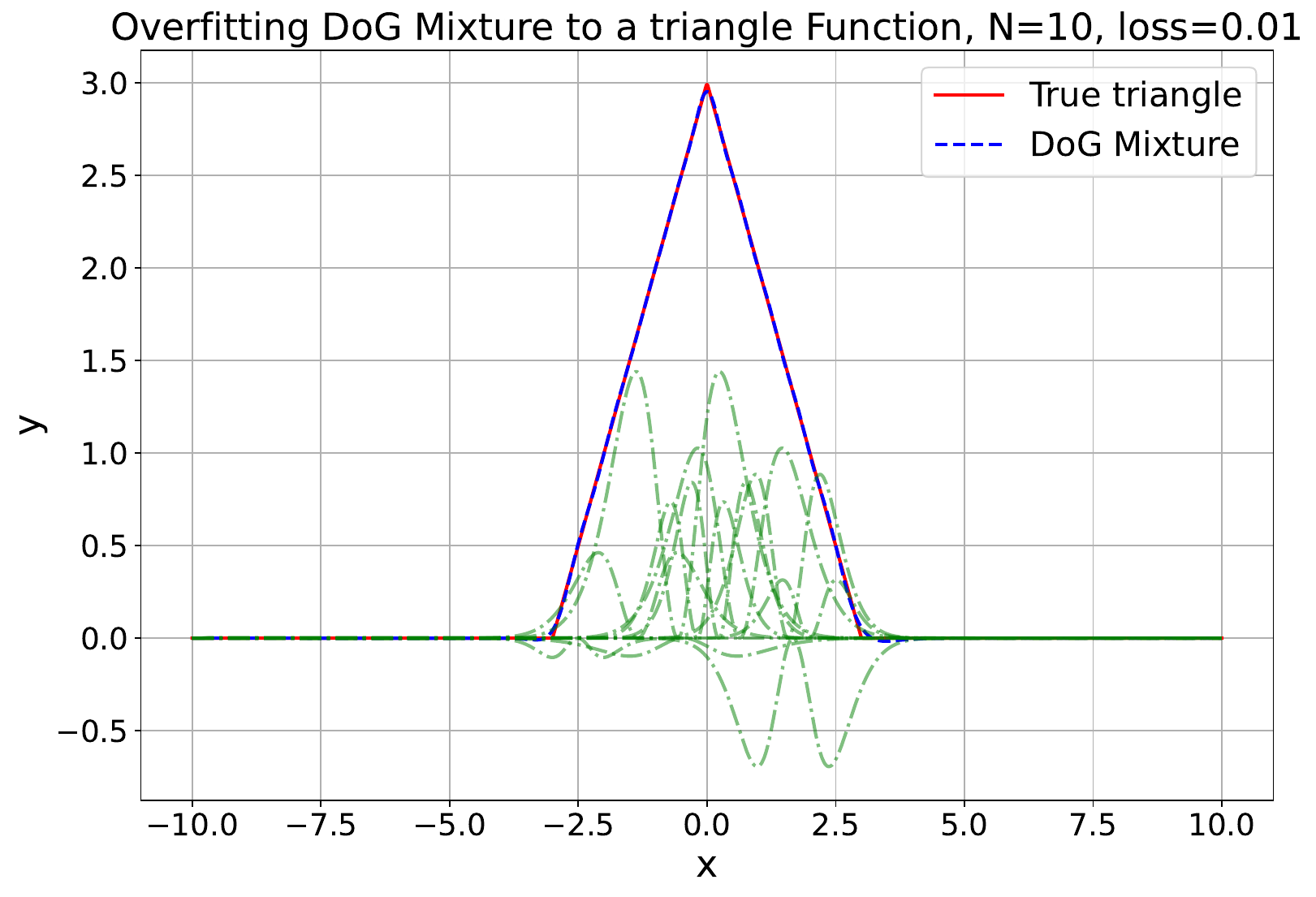} & 
    \includegraphics[width=0.24\linewidth]{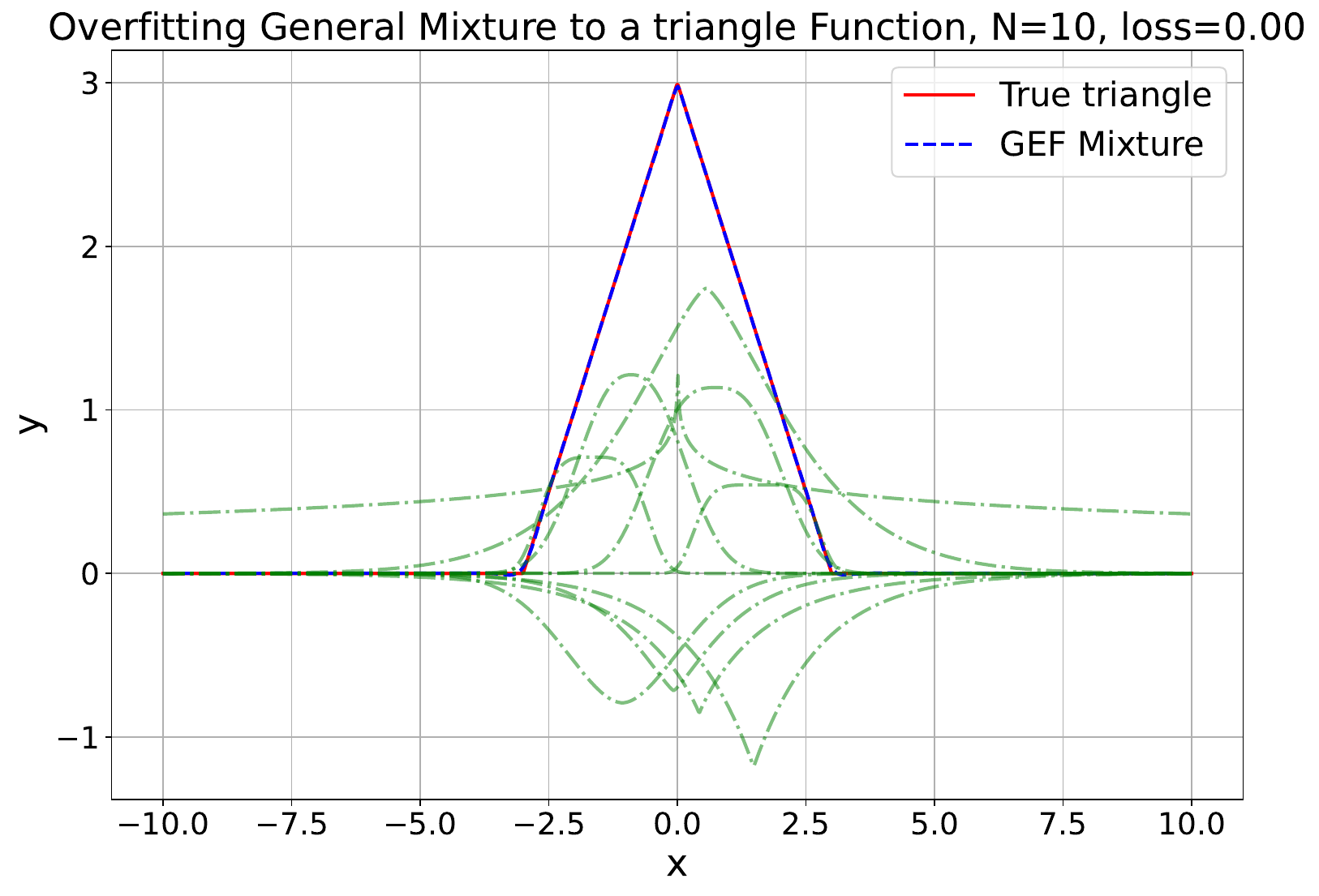}\\ 
    \includegraphics[width=0.24\linewidth]{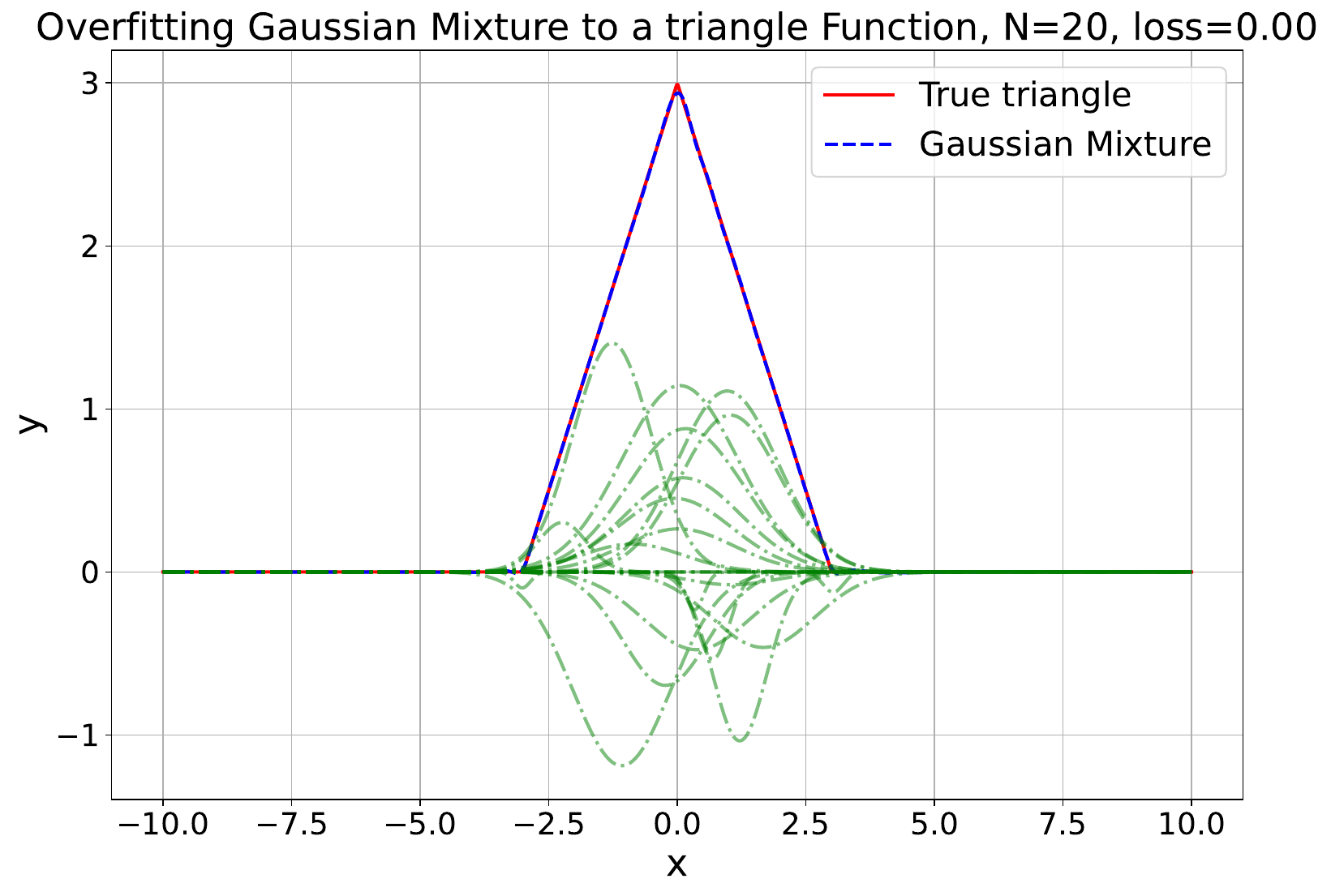} & 
    \includegraphics[width=0.24\linewidth]{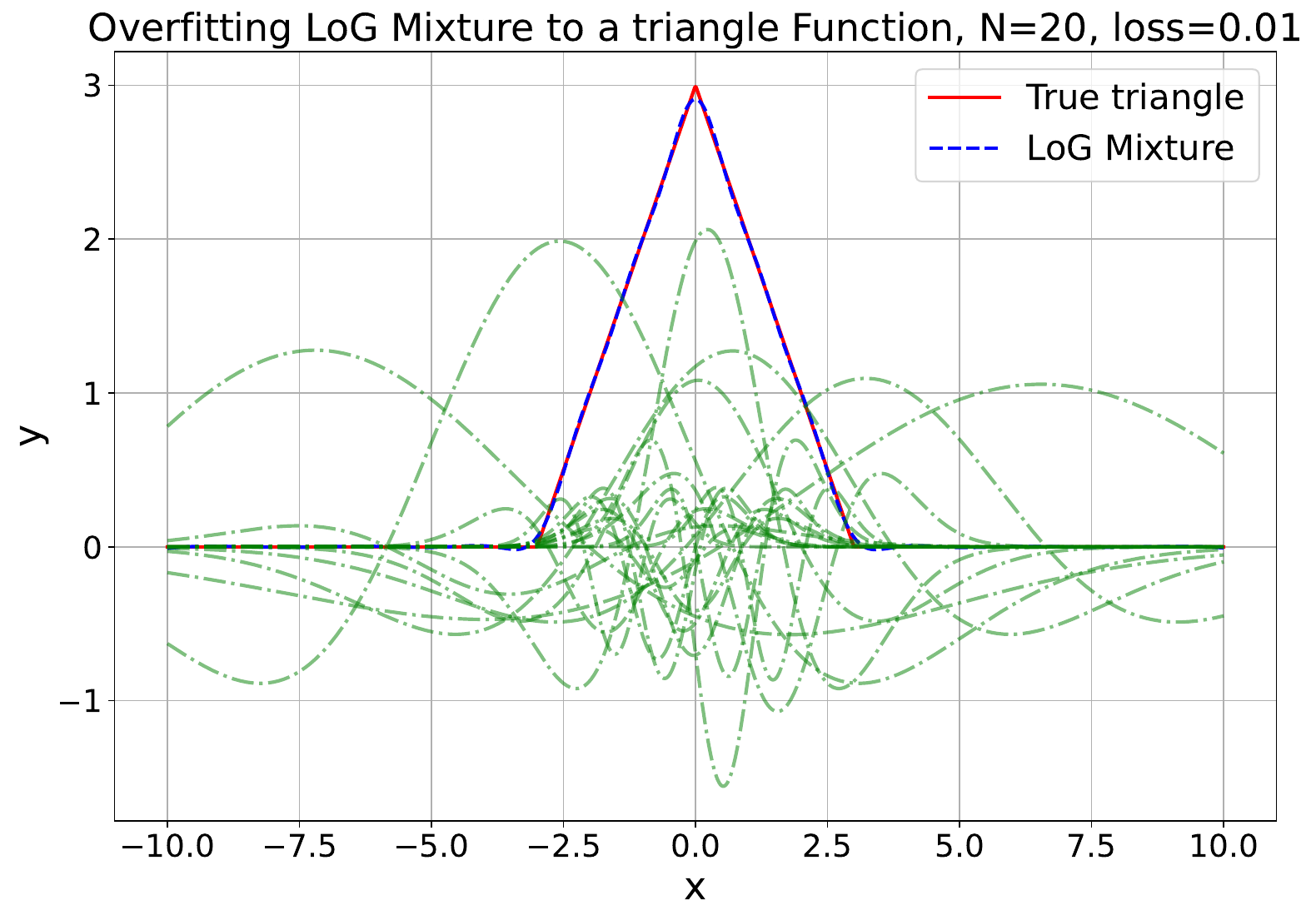} & 
    \includegraphics[width=0.24\linewidth]{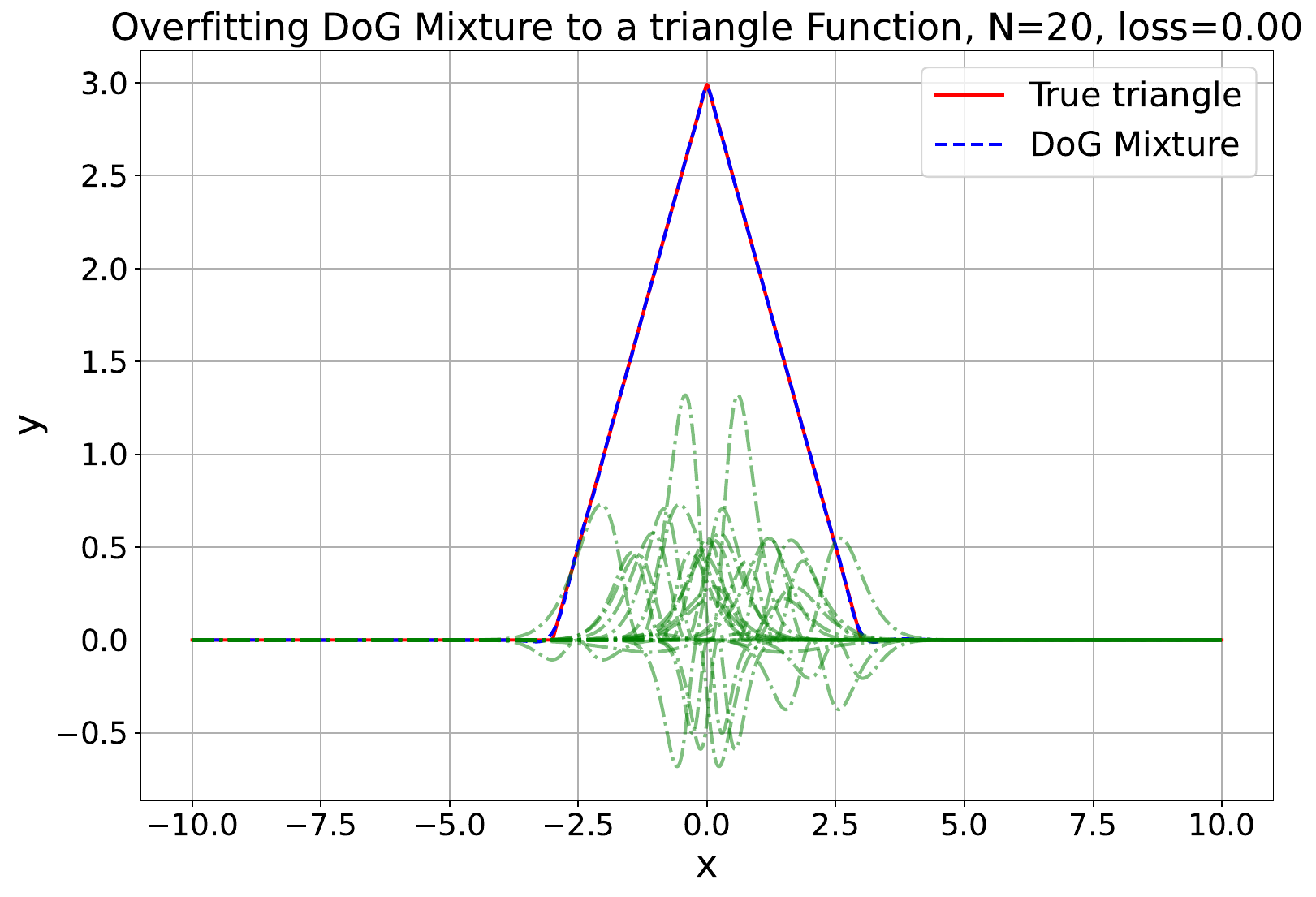} & 
    \includegraphics[width=0.24\linewidth]{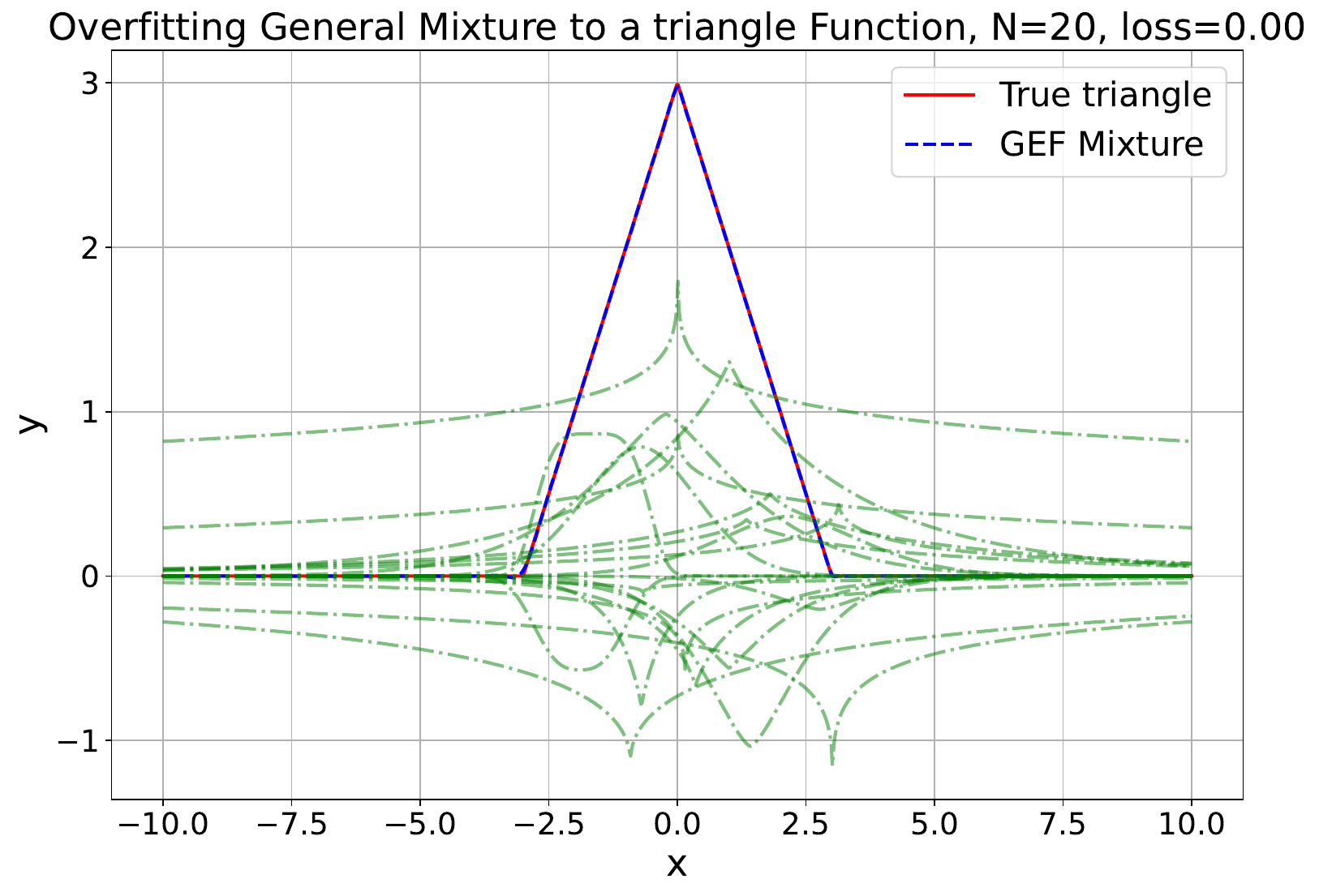}\\ 
    
    \end{tabular}
    }
    \caption{\textbf{Numerical Simulation Examples of Fitting Triangles with Real Weights Mixtures ( N= 2, 5, 8, and 10 )}. We show some fitting examples for triangle signals with Real weights mixtures (can be negative). The four mixtures used from left to right are Gaussians, LoG, DoG, and General mixtures. From top to bottom: N = 2, 8, and 10 components. The optimized individual components are shown in green. Some examples fail to optimize due to numerical instability in both Gaussians and GEF mixtures. Note that GEF is very efficient in fitting the triangle with few components while LoG and DoG are more stable for a larger number of components. }
    \label{supfig:fitting_triangle_N}
    \end{figure*}
    

%% file: figures/fitting/fitting_gaussian_p.tex
\begin{figure*}[h]
    \centering
    \resizebox{1.0\linewidth}{!}{
    \begin{tabular}{cccc}
    \tabcolsep=0.01cm
    Gaussian Mixture& LoG Mixture & DoG Mixture & GEF Mixture \\ 
    \includegraphics[width=0.24\linewidth]{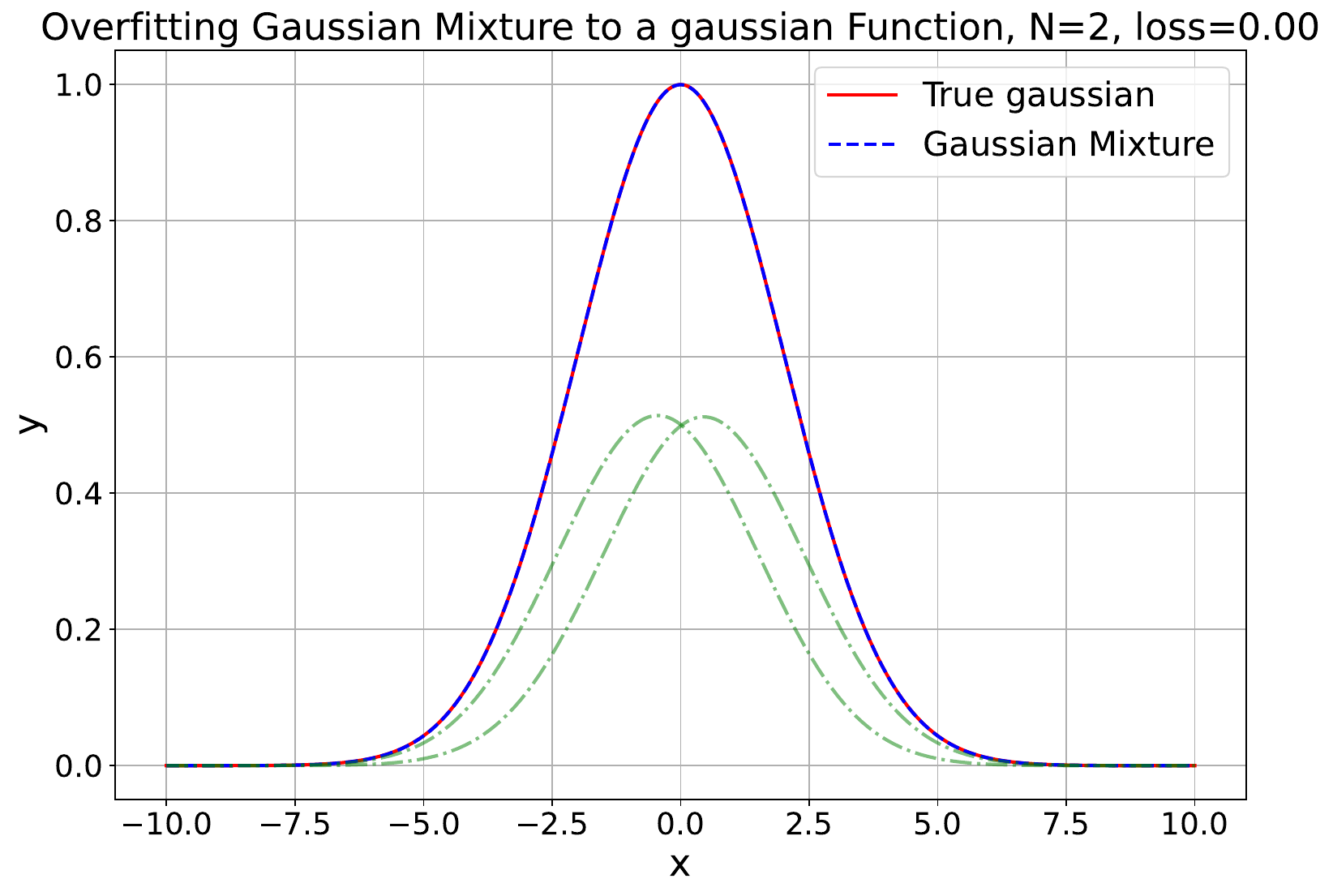} & 
    \includegraphics[width=0.24\linewidth]{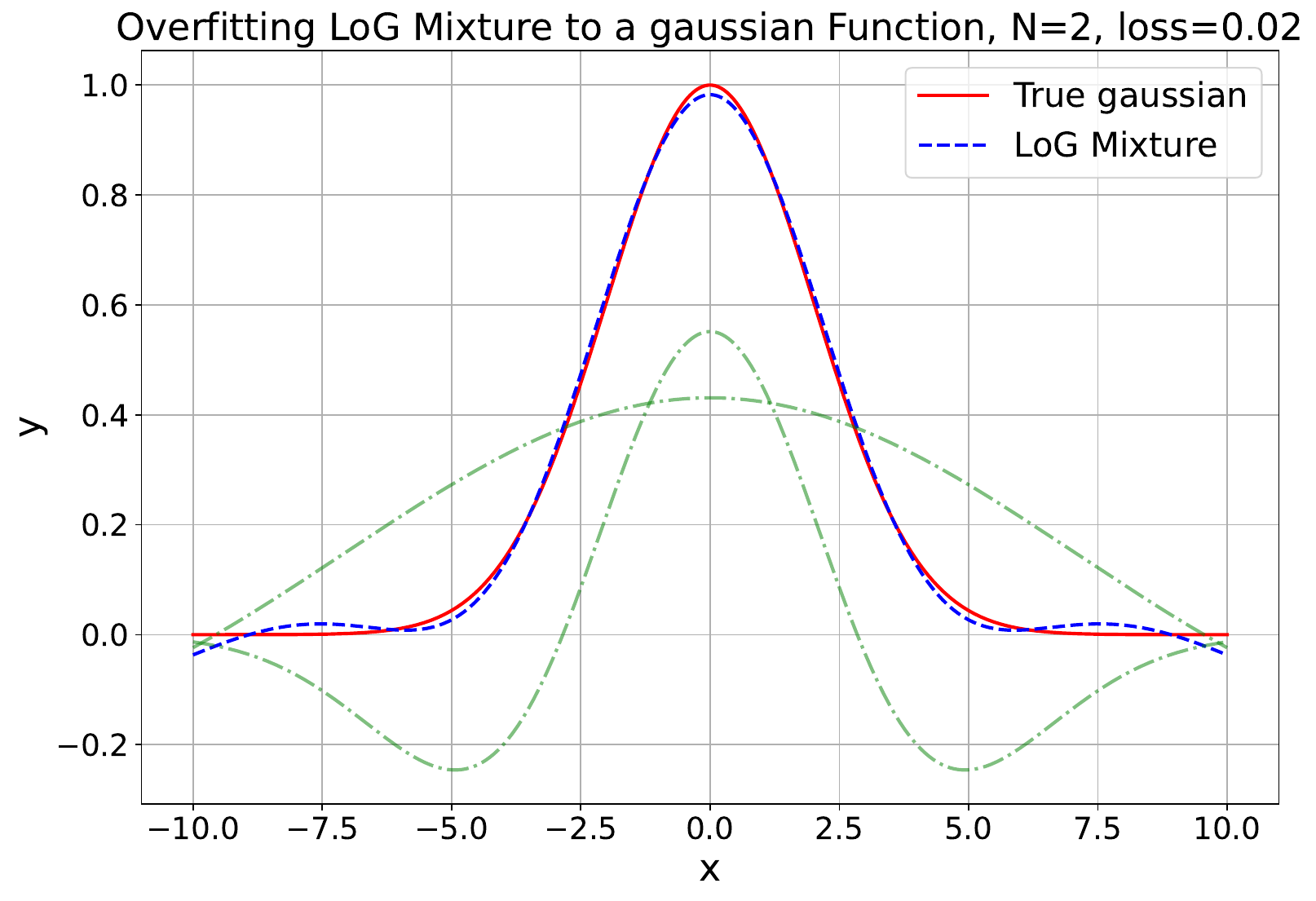} & 
    \includegraphics[width=0.24\linewidth]{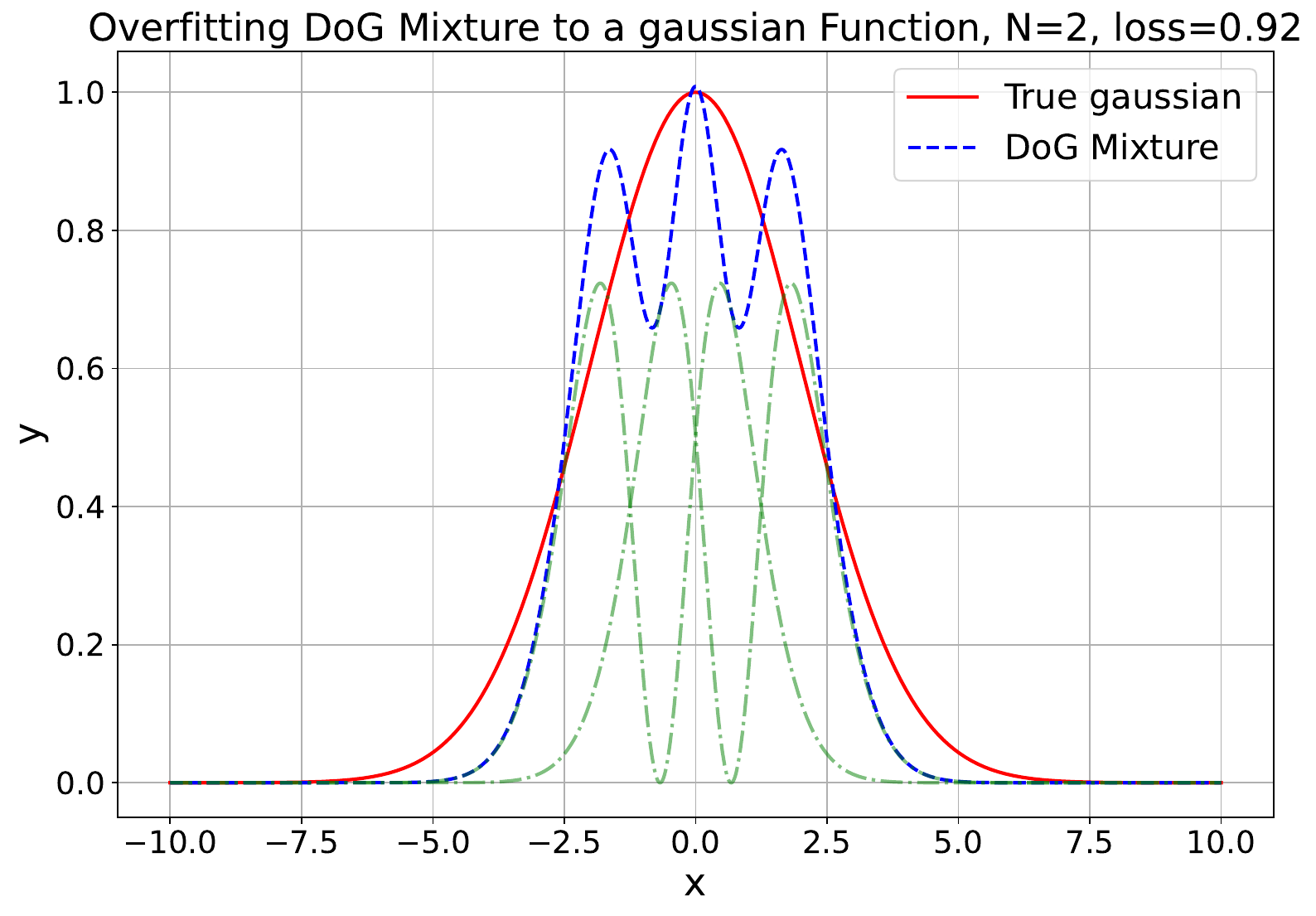} & 
    \includegraphics[width=0.24\linewidth]{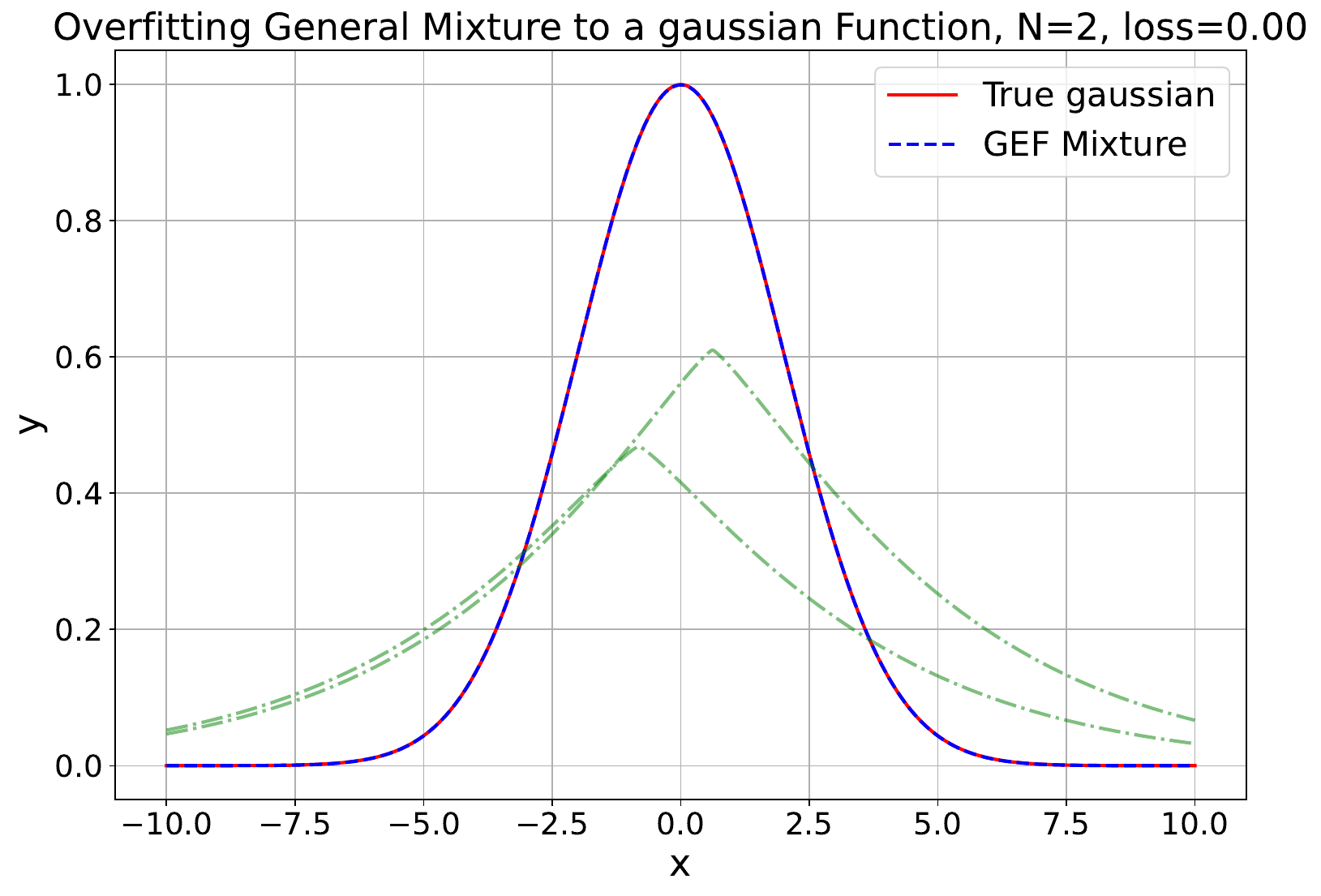}\\ 
    \includegraphics[width=0.24\linewidth]{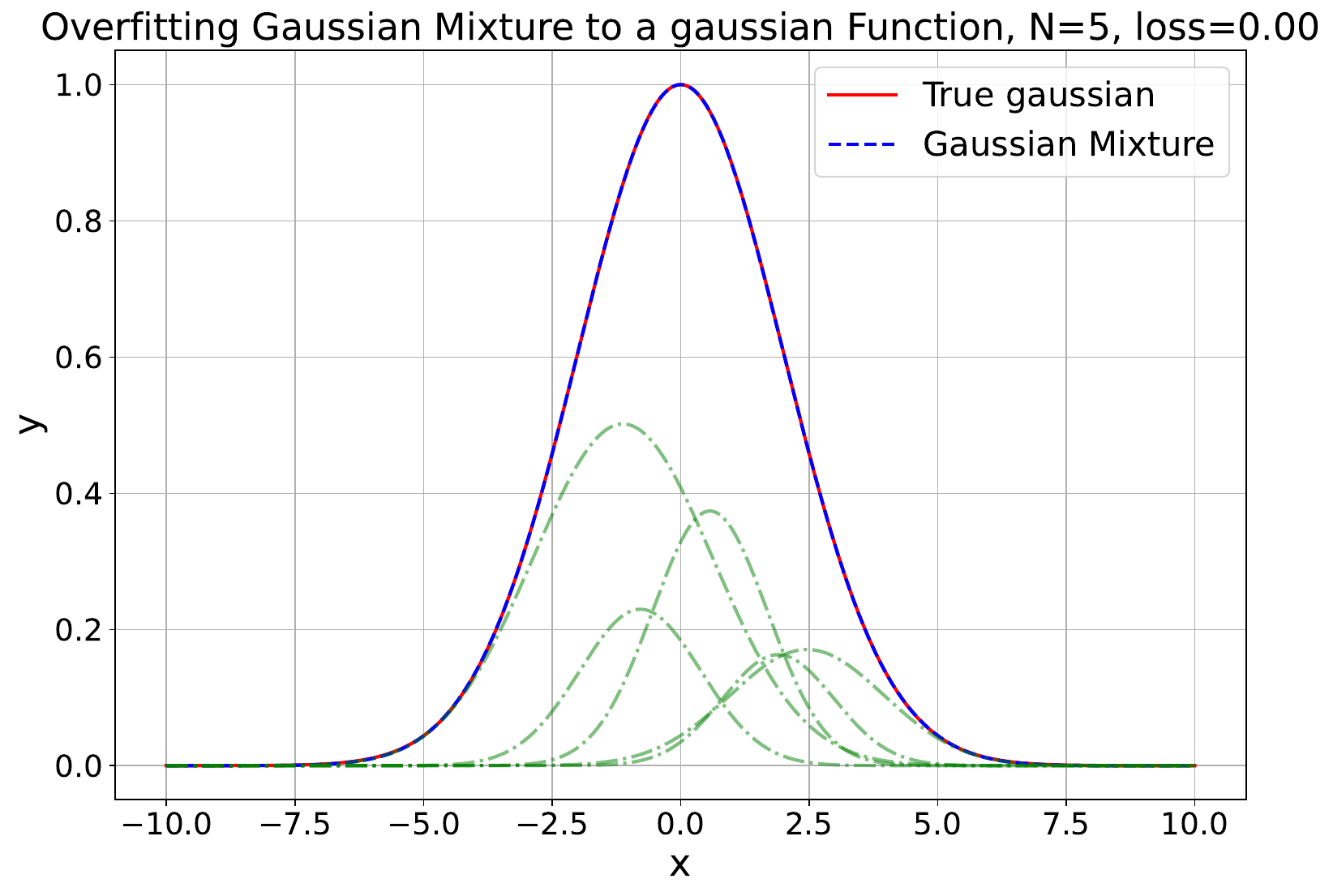} & 
    \includegraphics[width=0.24\linewidth]{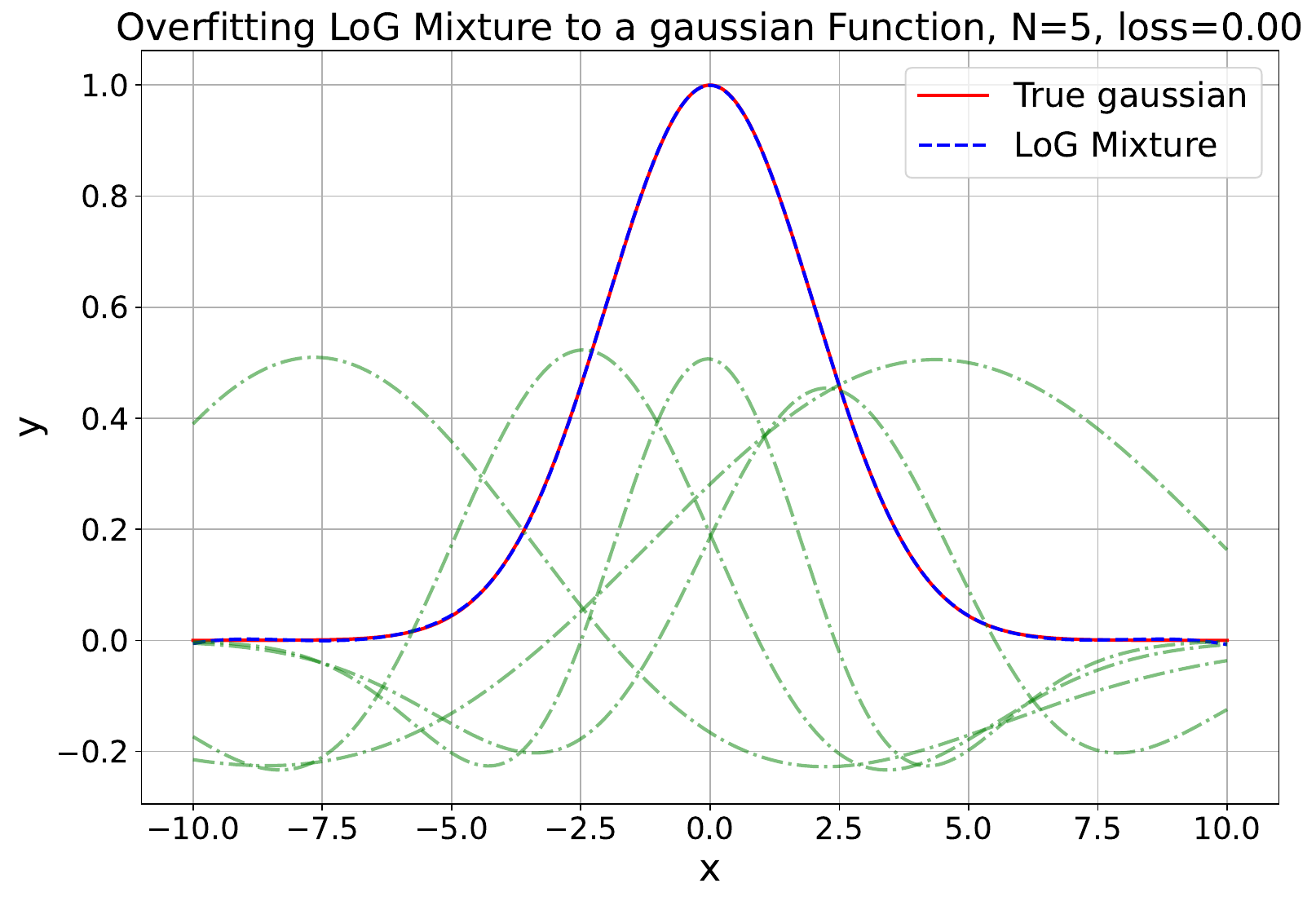} & 
    \includegraphics[width=0.24\linewidth]{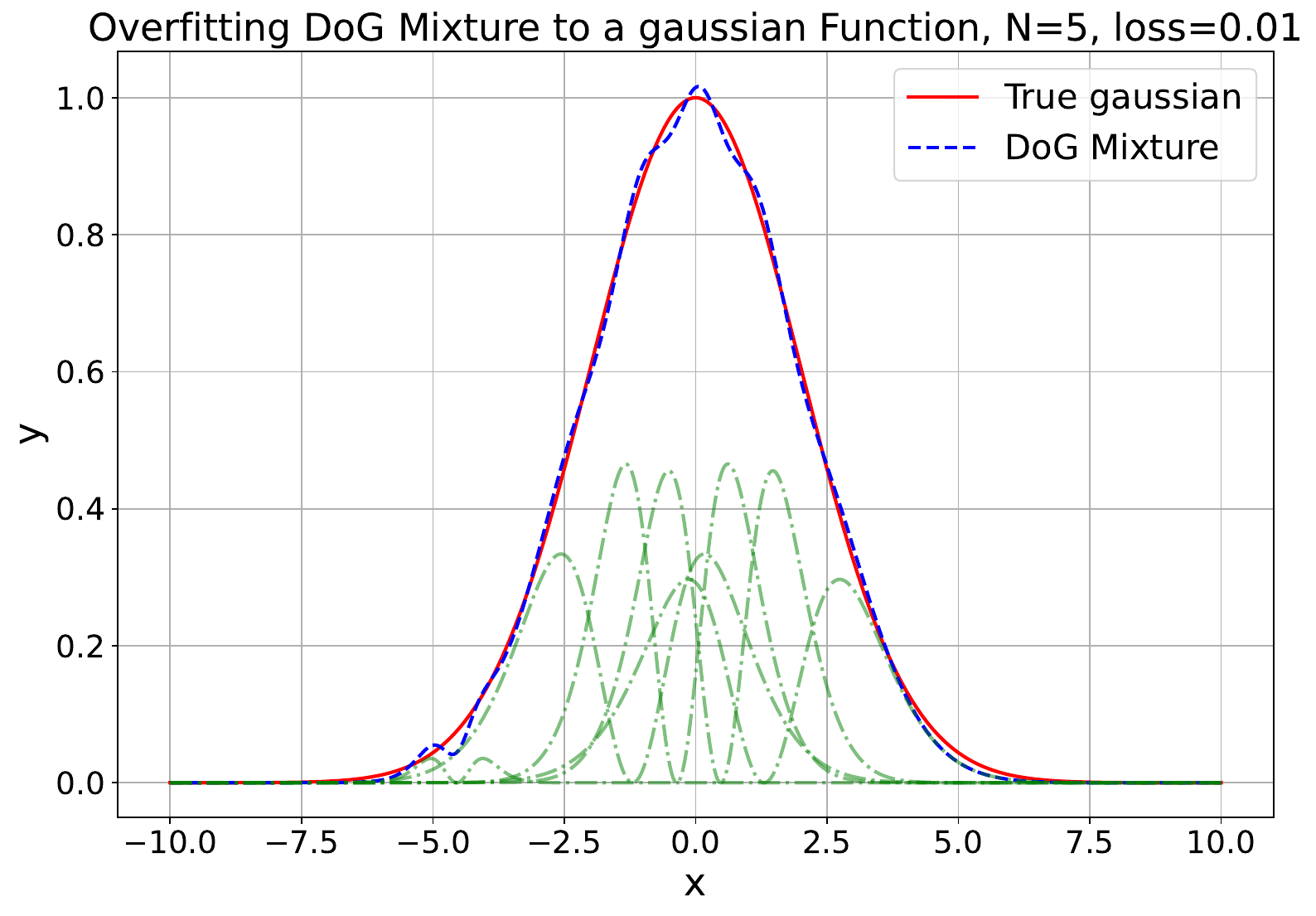} & 
    \includegraphics[width=0.24\linewidth]{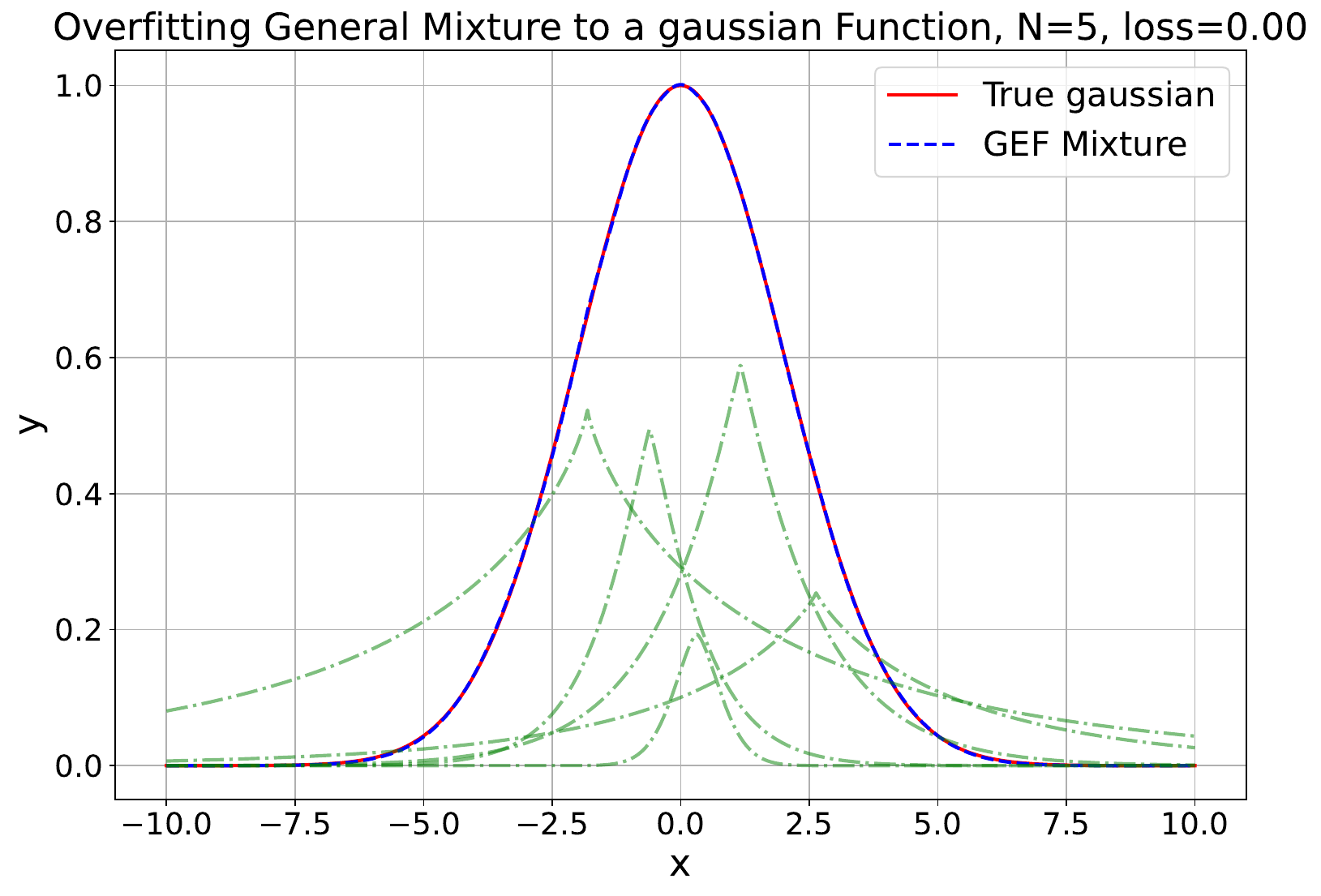}\\ 
    \includegraphics[width=0.24\linewidth]{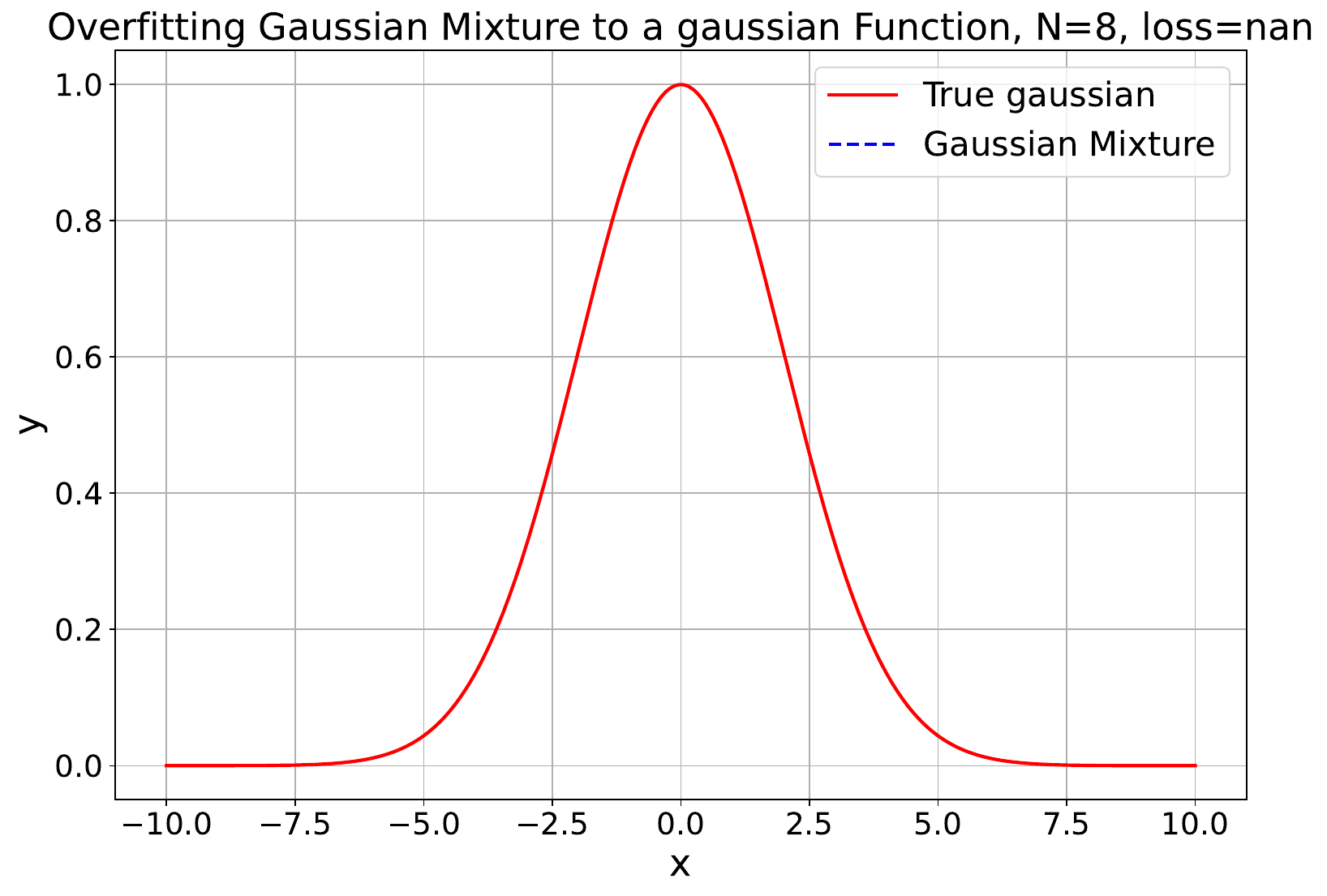} & 
    \includegraphics[width=0.24\linewidth]{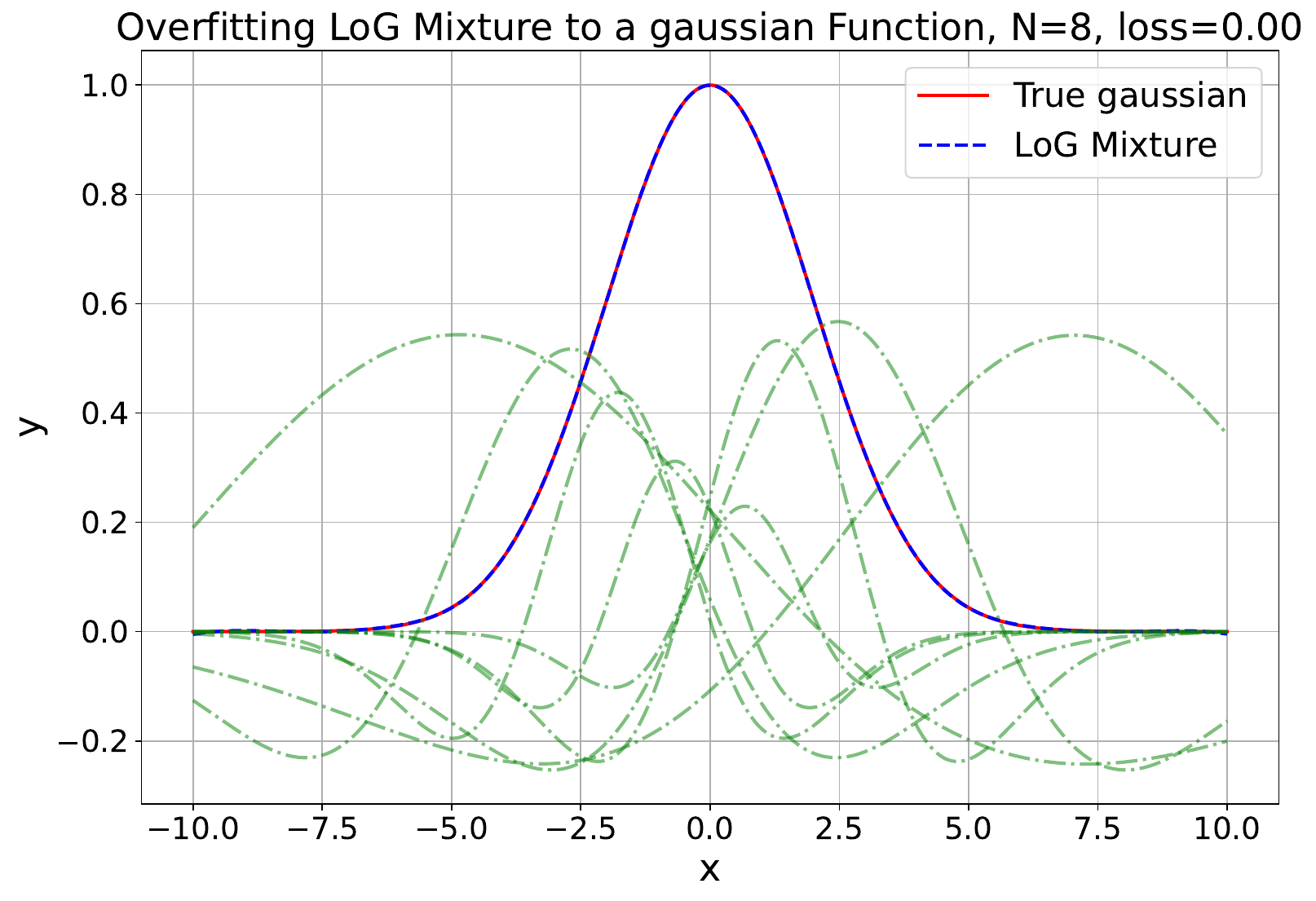} & 
    \includegraphics[width=0.24\linewidth]{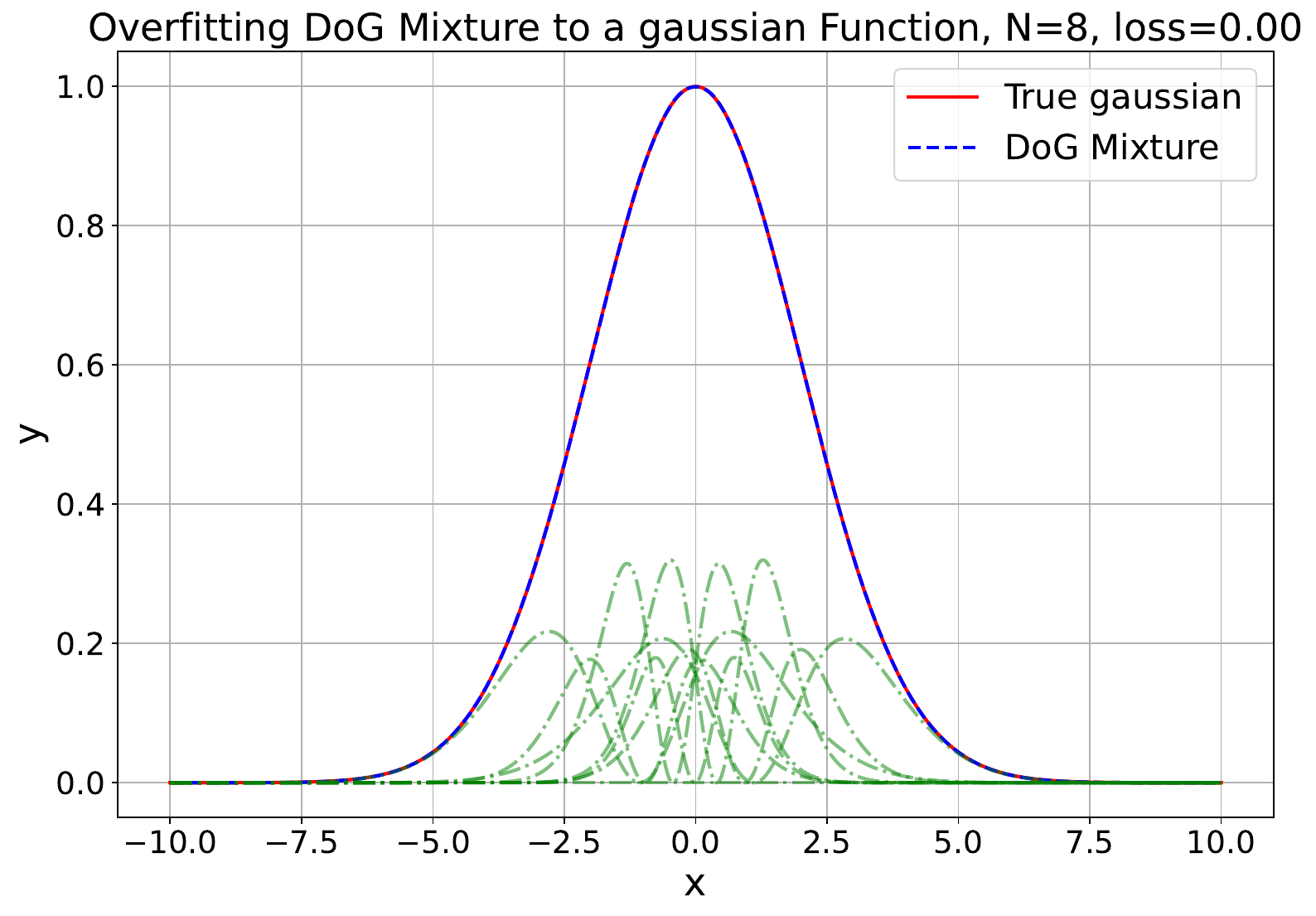} & 
    \includegraphics[width=0.24\linewidth]{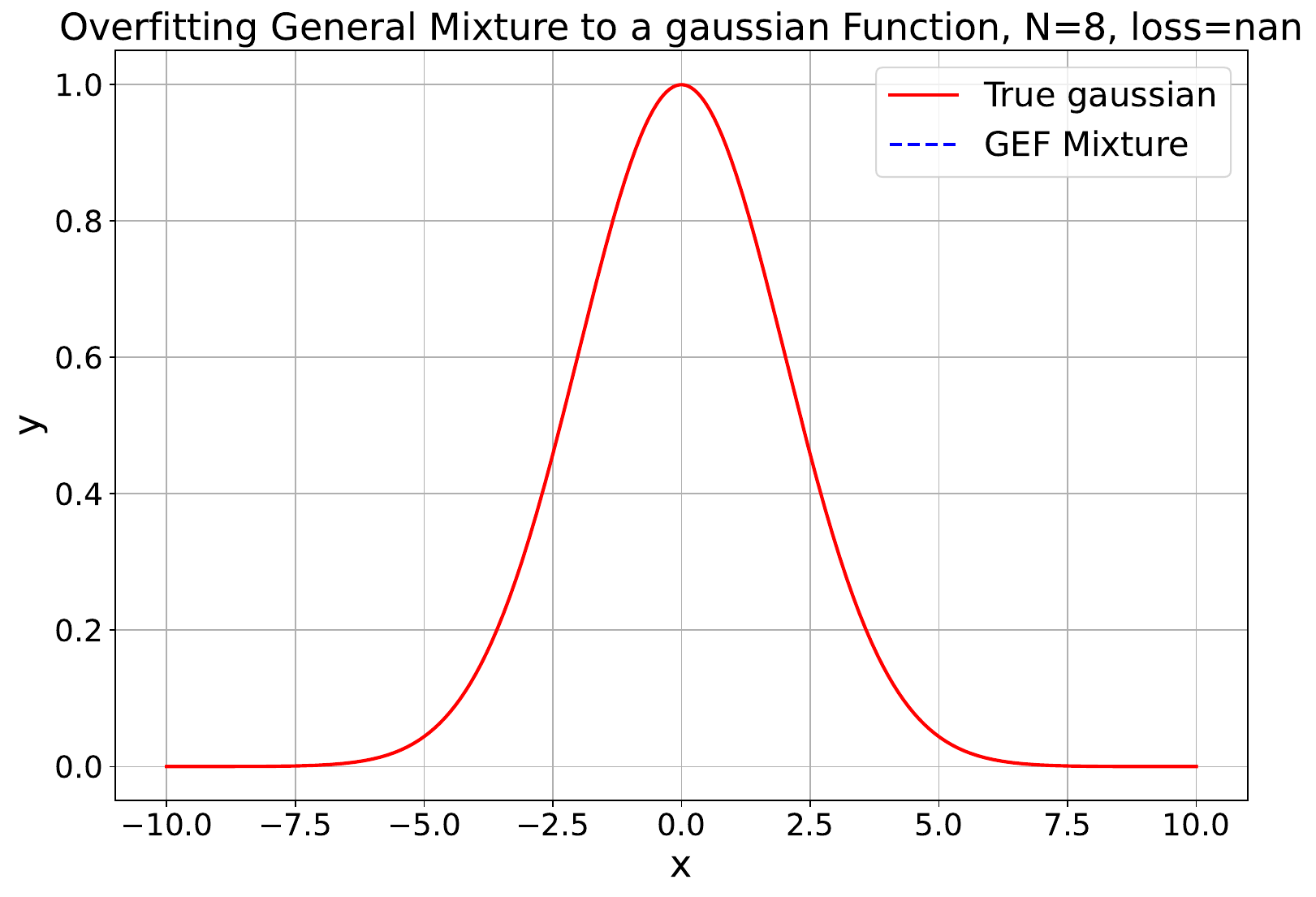}\\ 
    \includegraphics[width=0.24\linewidth]{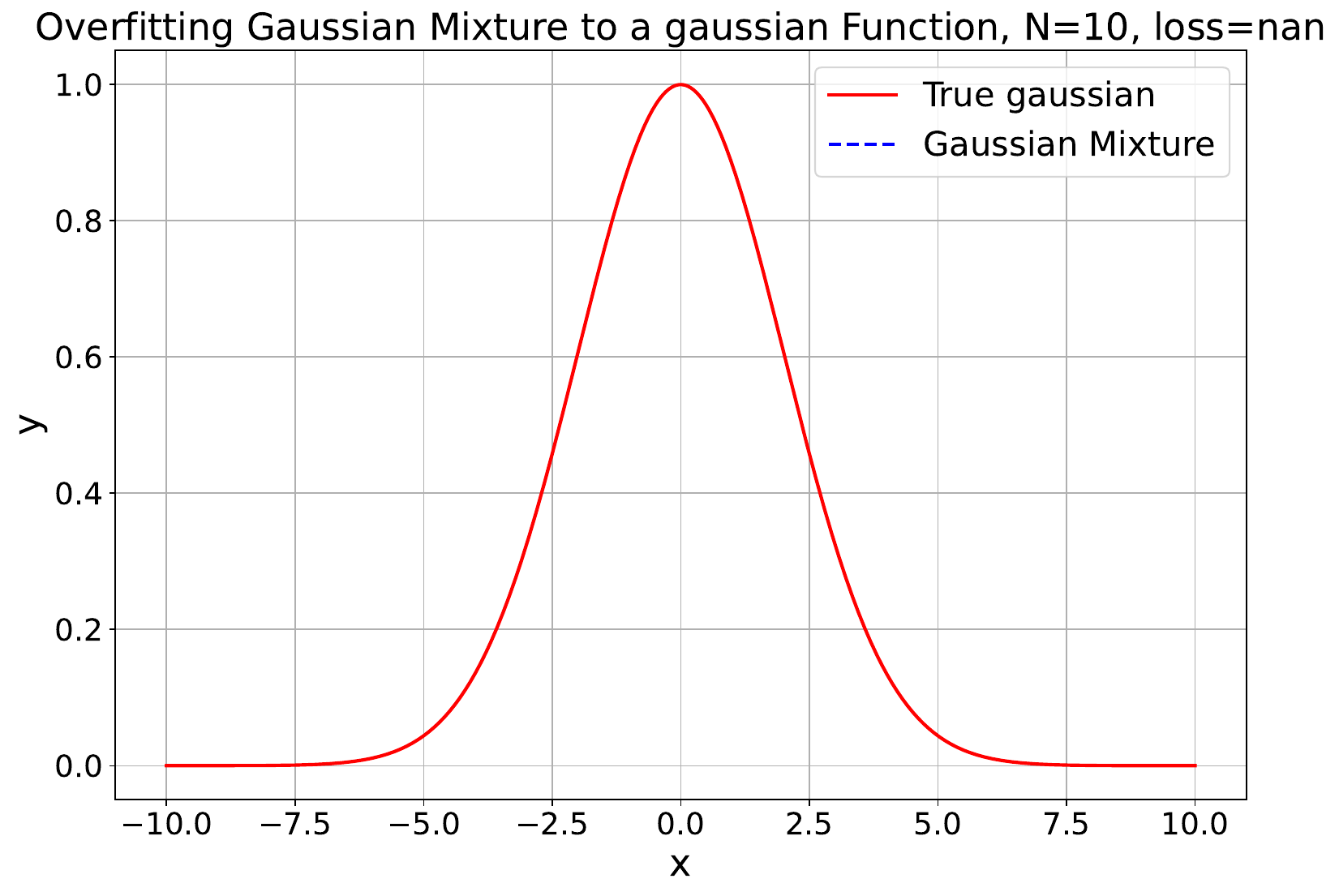} & 
    \includegraphics[width=0.24\linewidth]{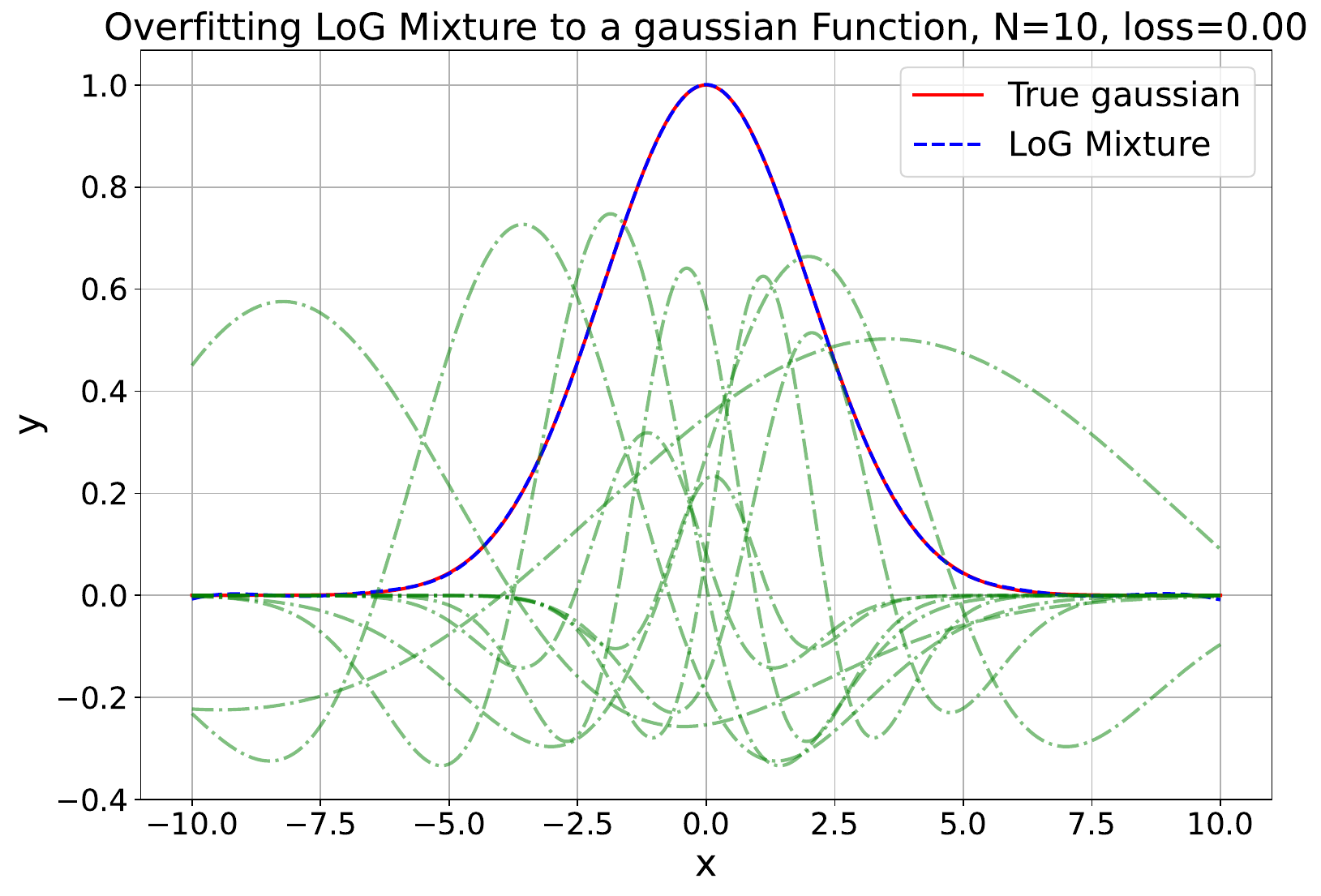} & 
    \includegraphics[width=0.24\linewidth]{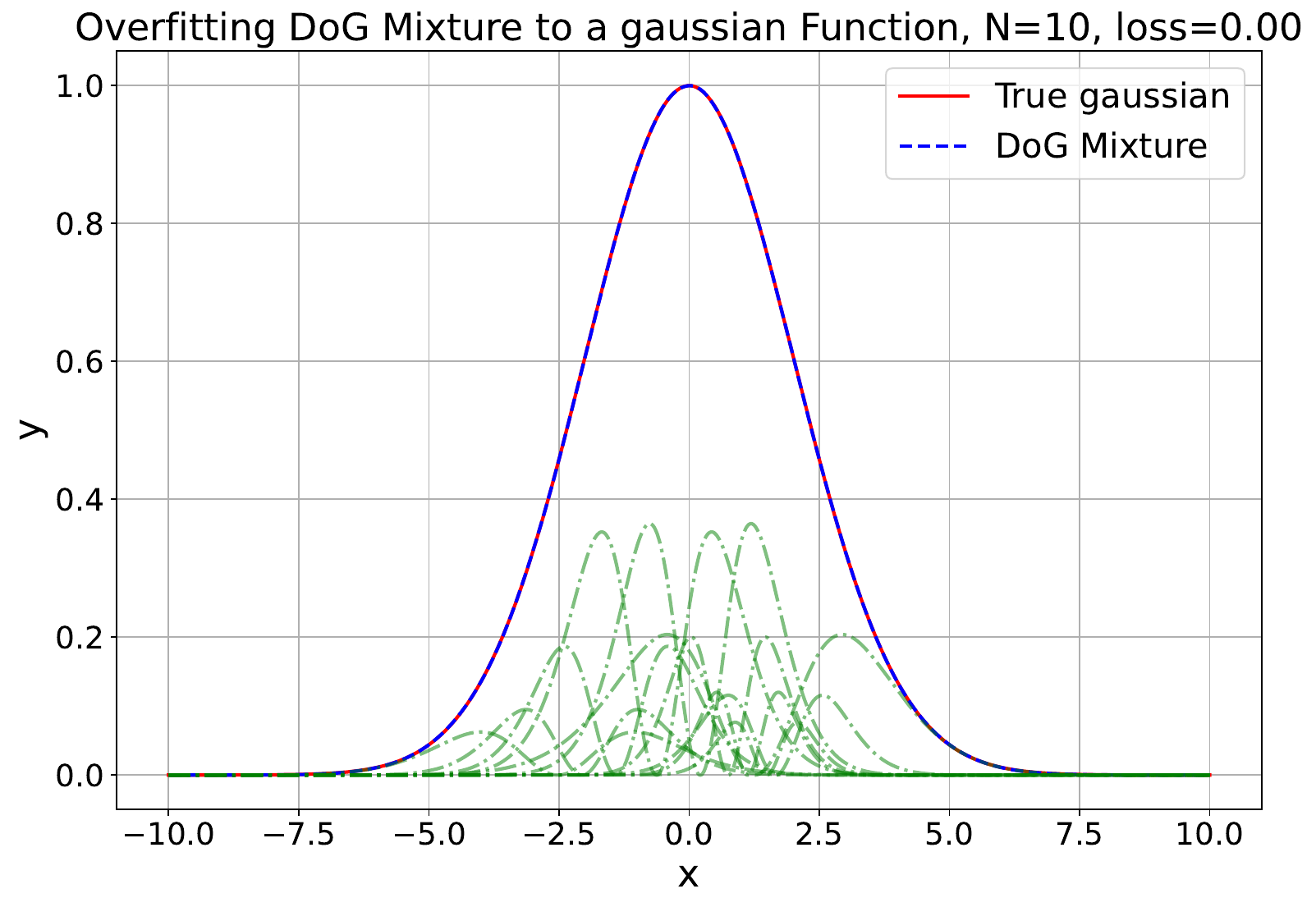} & 
    \includegraphics[width=0.24\linewidth]{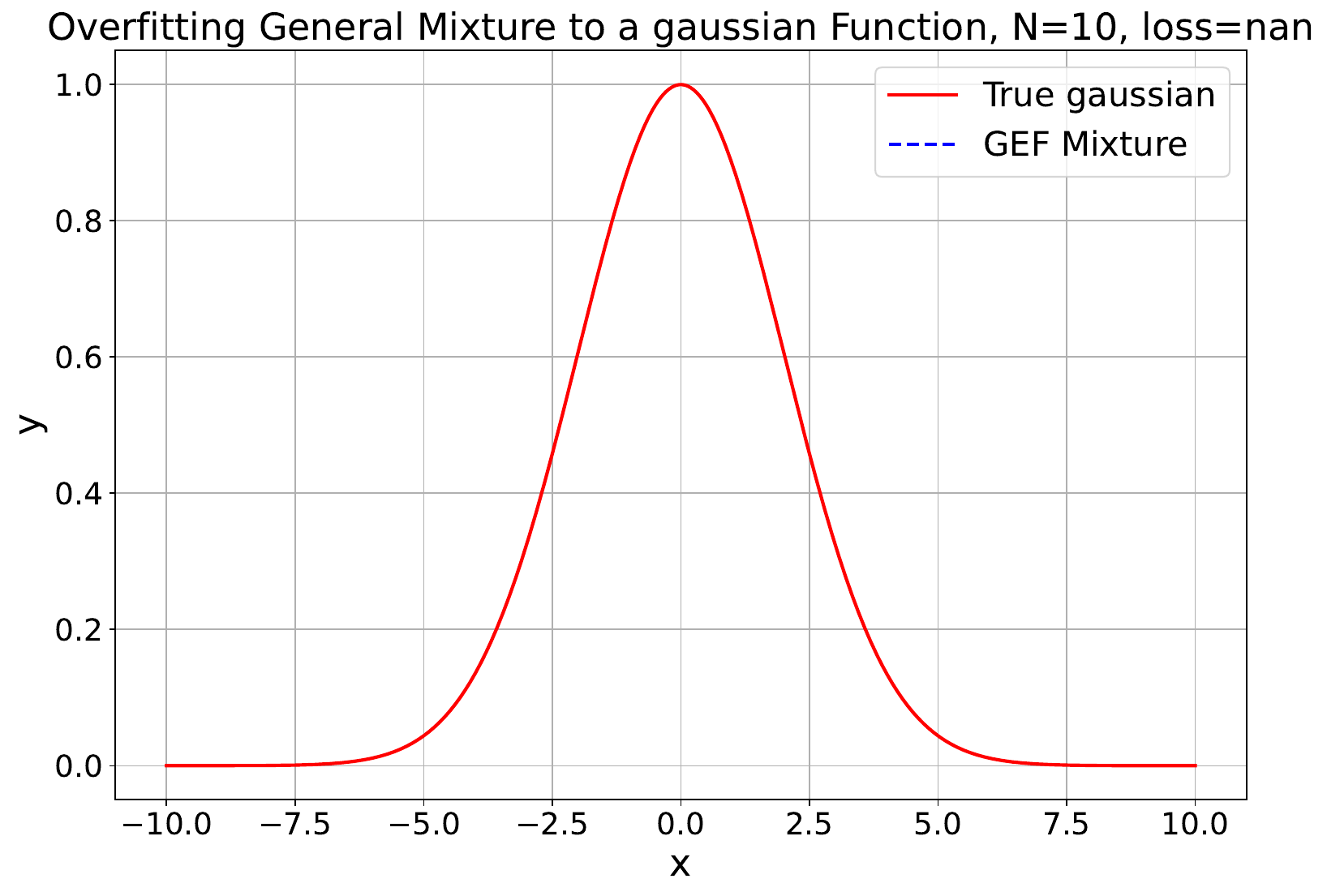}\\ 
    \includegraphics[width=0.24\linewidth]{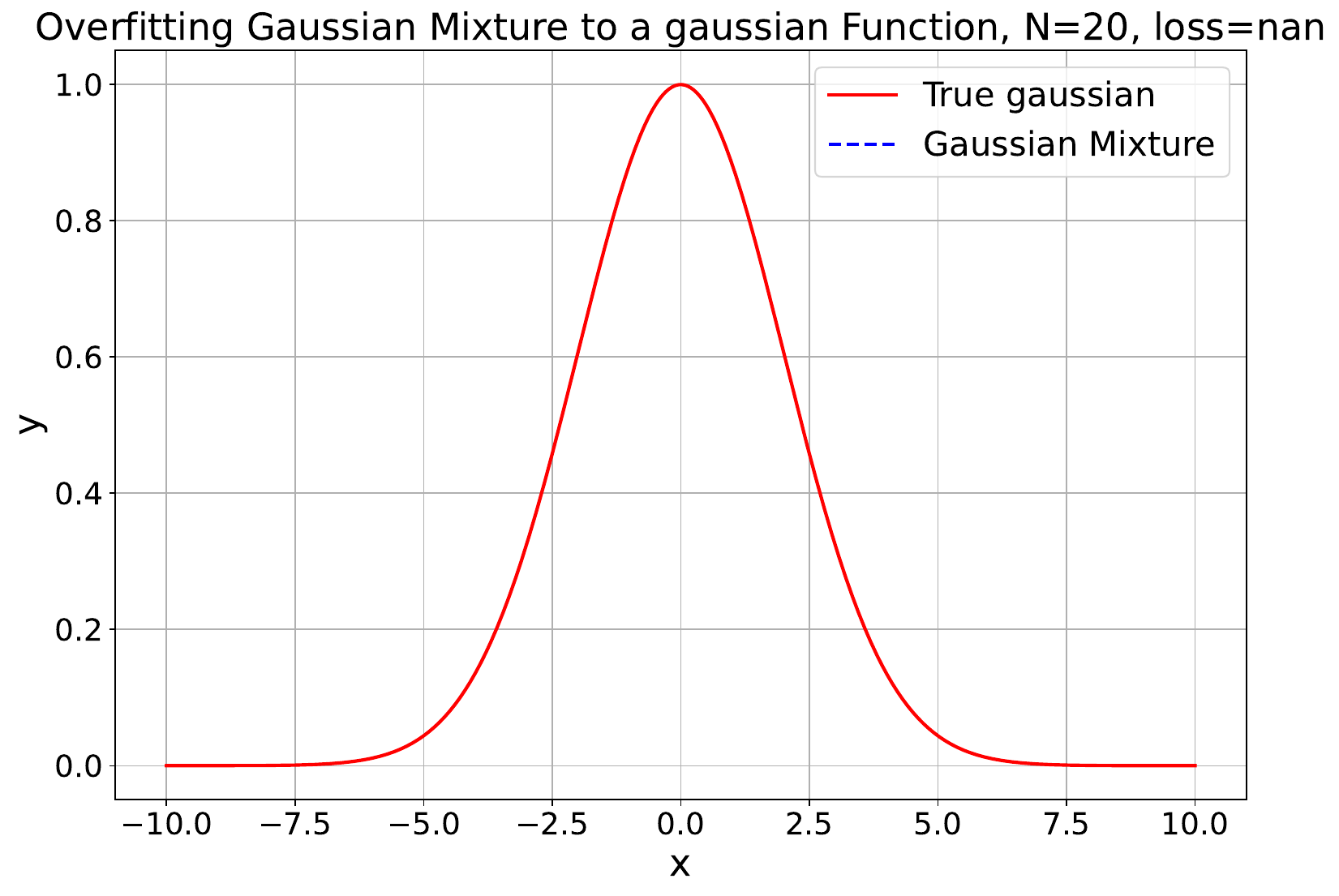} & 
    \includegraphics[width=0.24\linewidth]{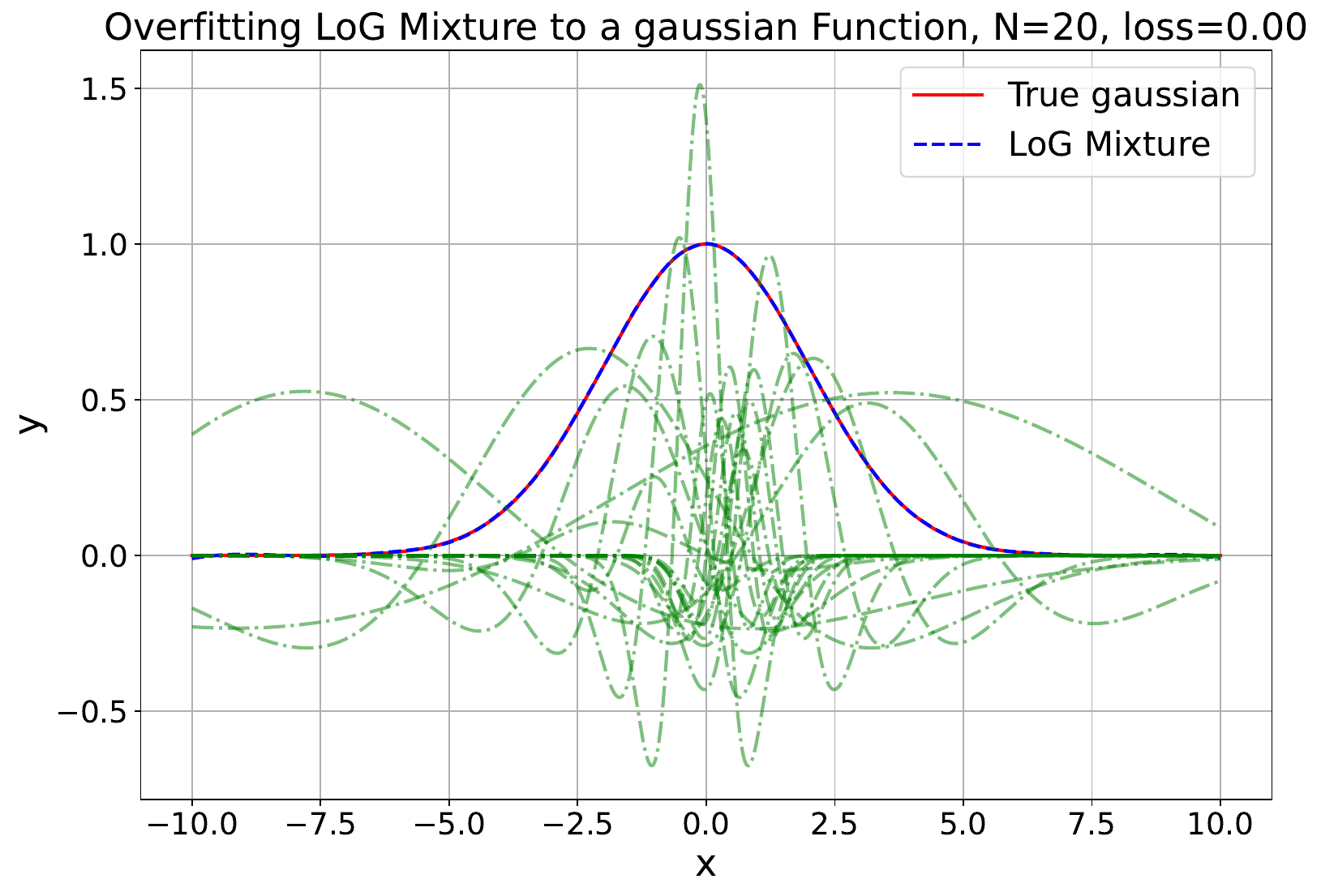} & 
    \includegraphics[width=0.24\linewidth]{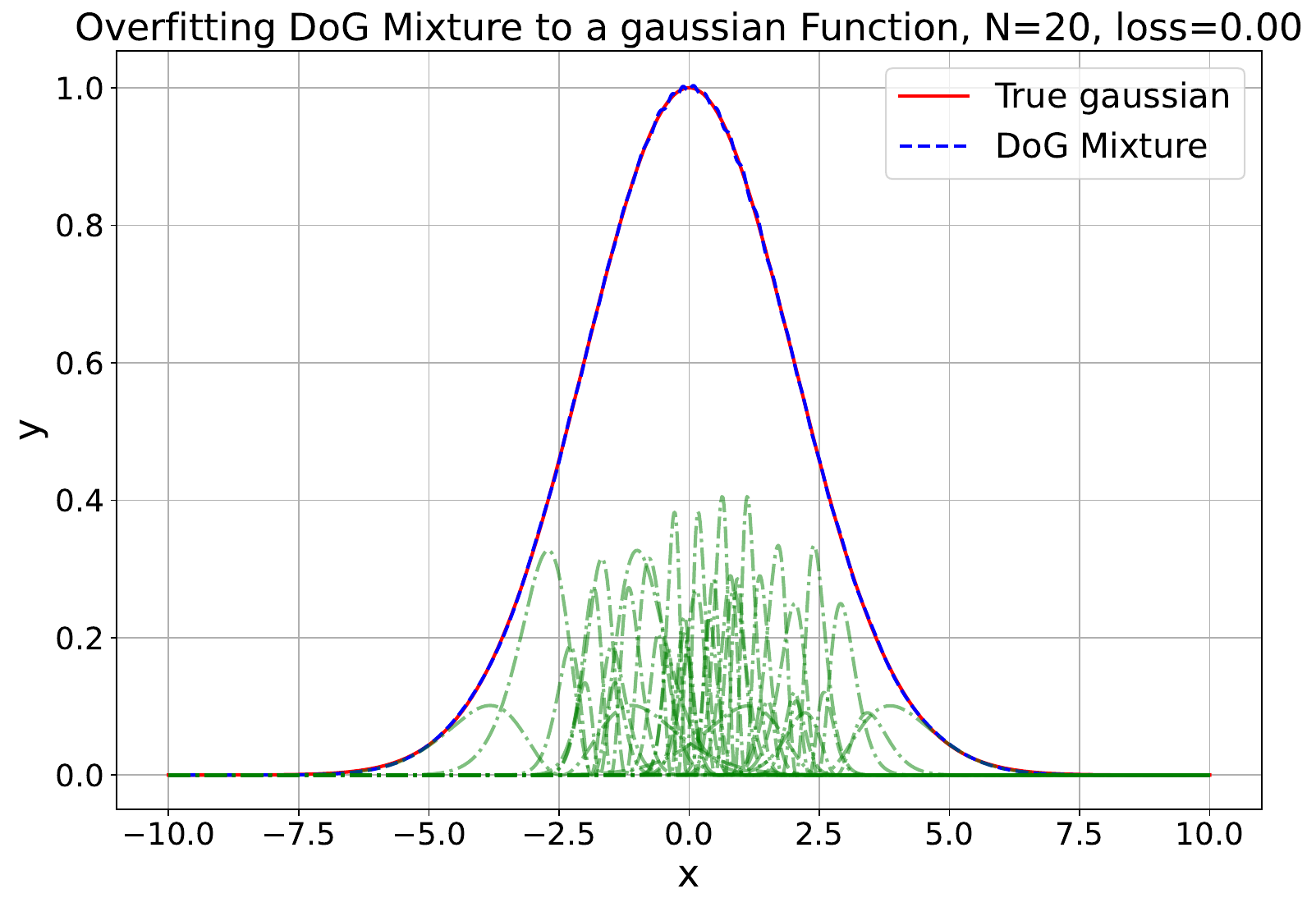} & 
    \includegraphics[width=0.24\linewidth]{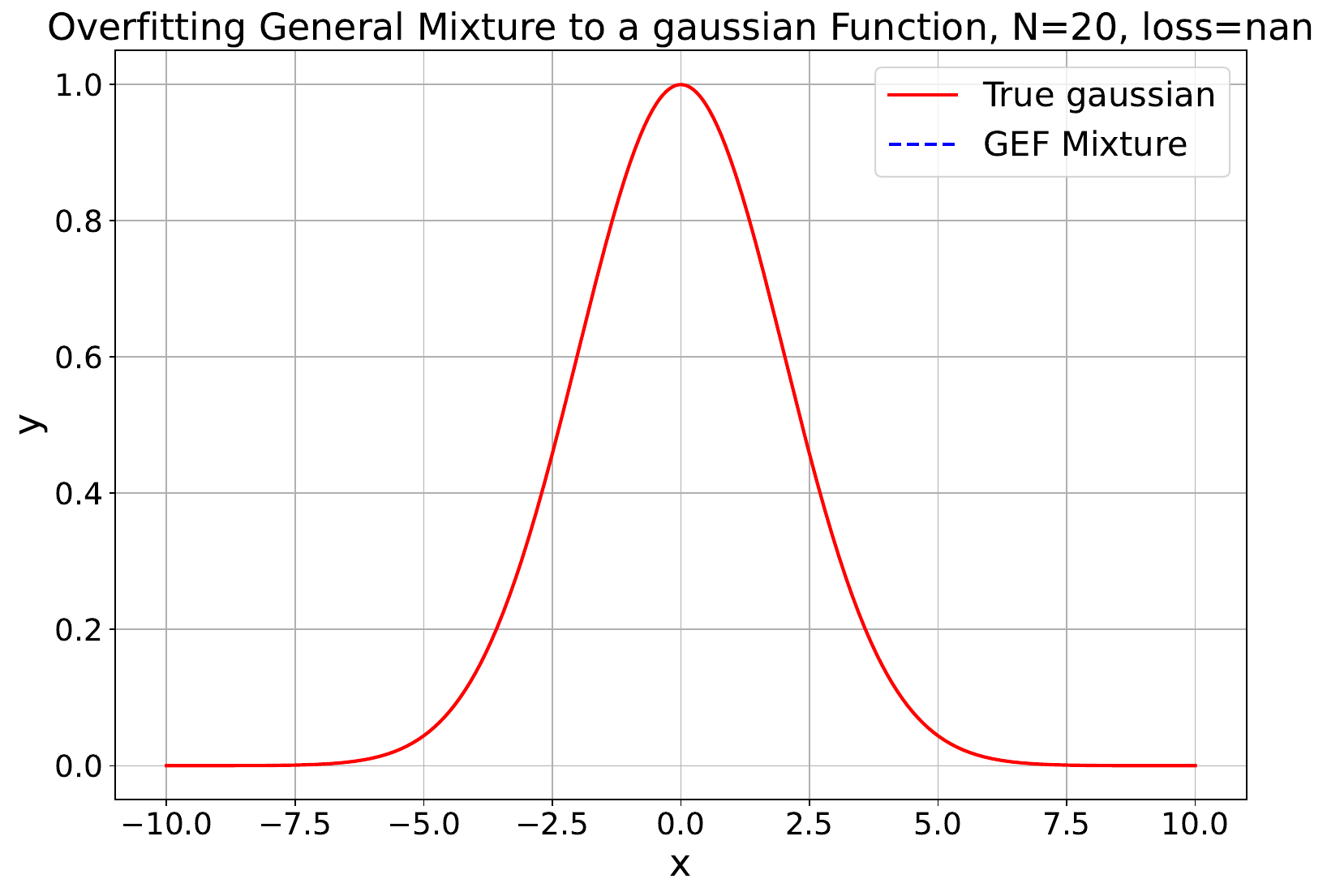}\\ 
    
    \end{tabular}
    }
    \caption{\textbf{Numerical Simulation Examples of Fitting Gaussians with Positive Weights Mixtures ( N= 2, 5, 8, and 10 )}. We show some fitting examples for Gaussian signals with positive weight mixtures. The four mixtures used from left to right are Gaussians, LoG, DoG, and General mixtures. From top to bottom: N = 2, 8, and 10 components. The optimized individual components are shown in green. Some examples fail to optimize due to numerical instability in both Gaussians and GEF mixtures. Note that GEF is very efficient in fitting the Gaussian with few components while LoG and DoG are more stable for a larger number of components. }
    \label{supfig:fitting_gaussian_p}
    \end{figure*}
    

%% file: figures/fitting/fitting_gaussian_n.tex
\begin{figure*}[h]
    \centering
    \resizebox{1.0\linewidth}{!}{
    \begin{tabular}{cccc}
    \tabcolsep=0.01cm
    Gaussian Mixture& LoG Mixture & DoG Mixture & GEF Mixture \\ 
    \includegraphics[width=0.24\linewidth]{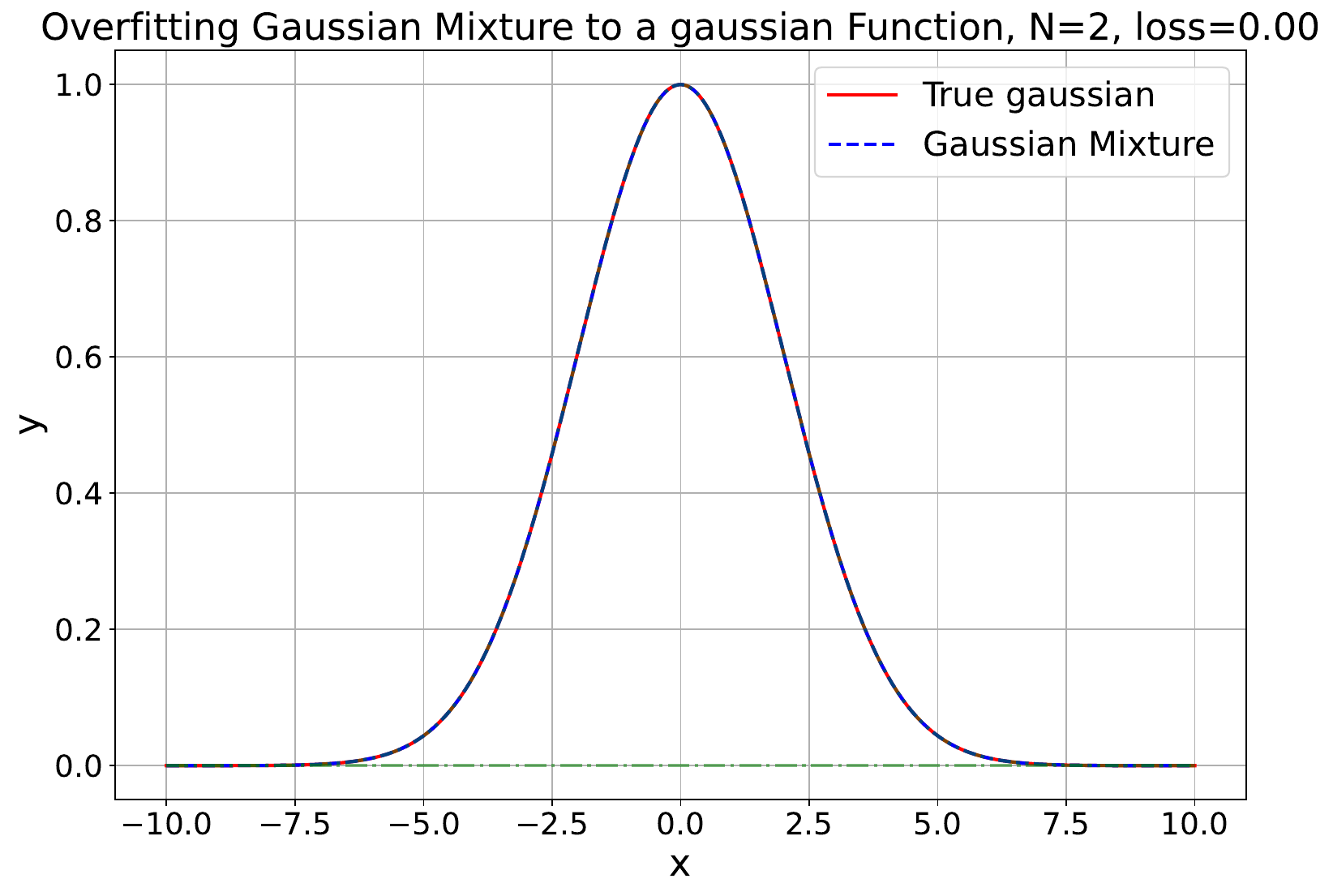} & 
    \includegraphics[width=0.24\linewidth]{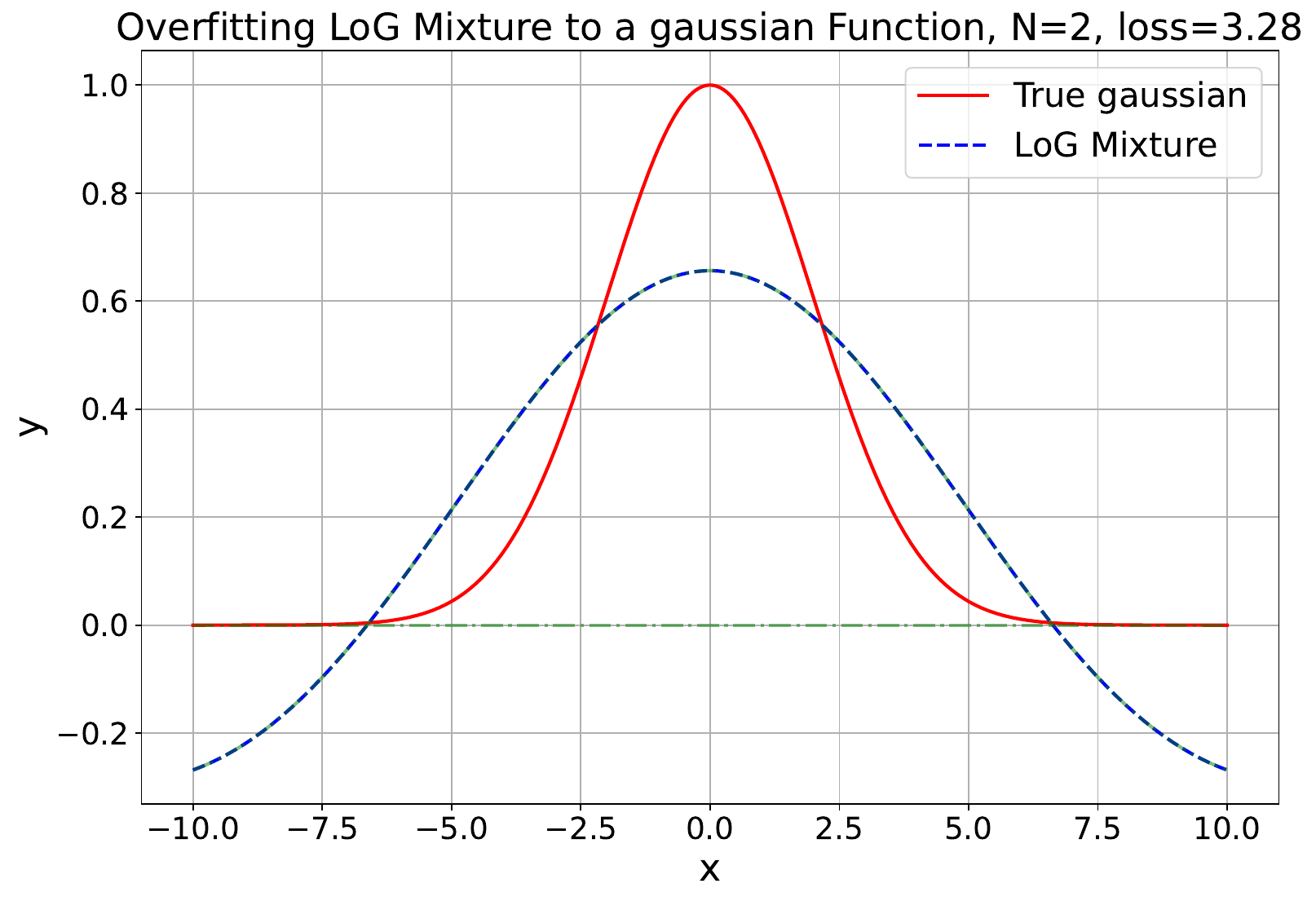} & 
    \includegraphics[width=0.24\linewidth]{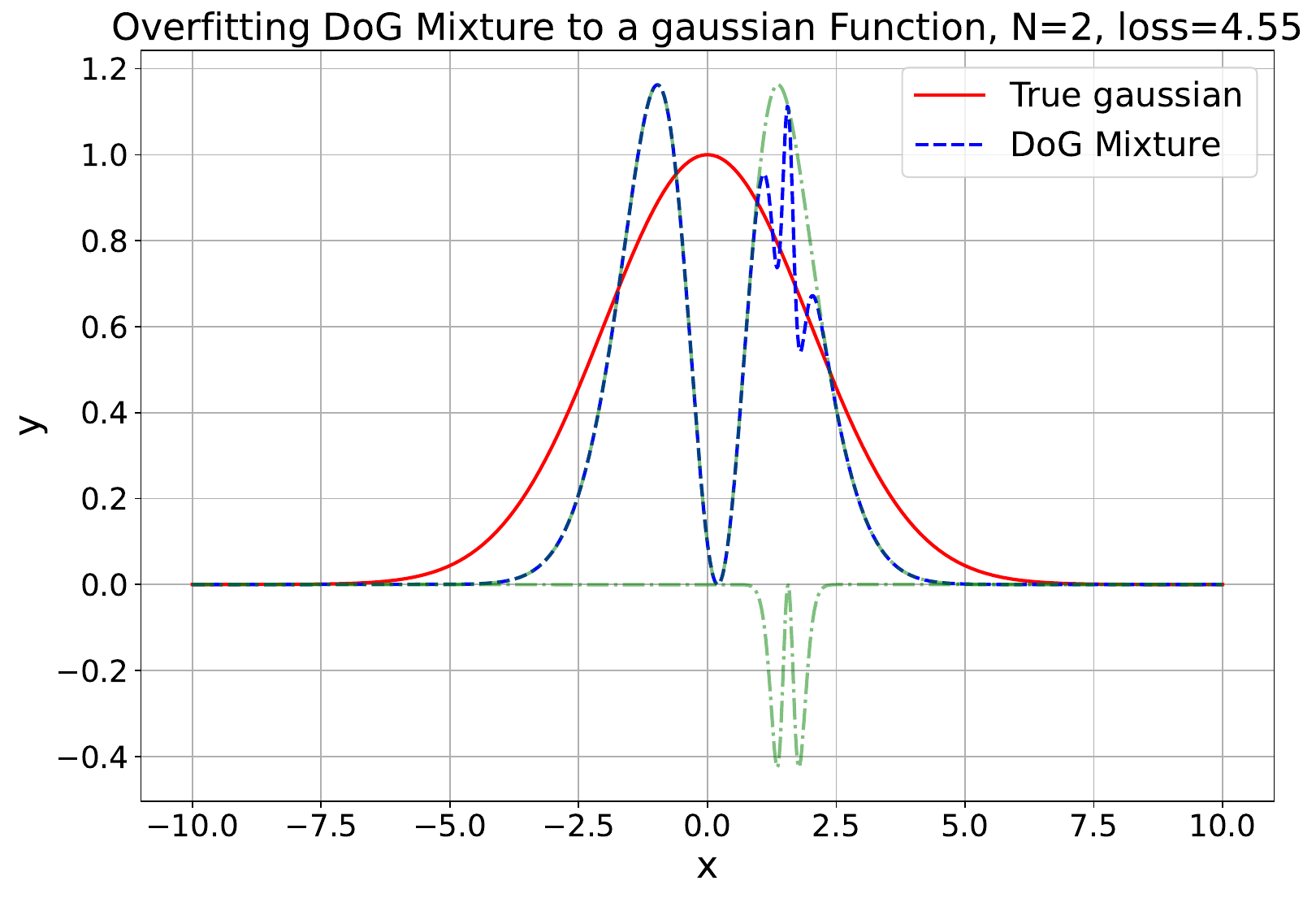} & 
    \includegraphics[width=0.24\linewidth]{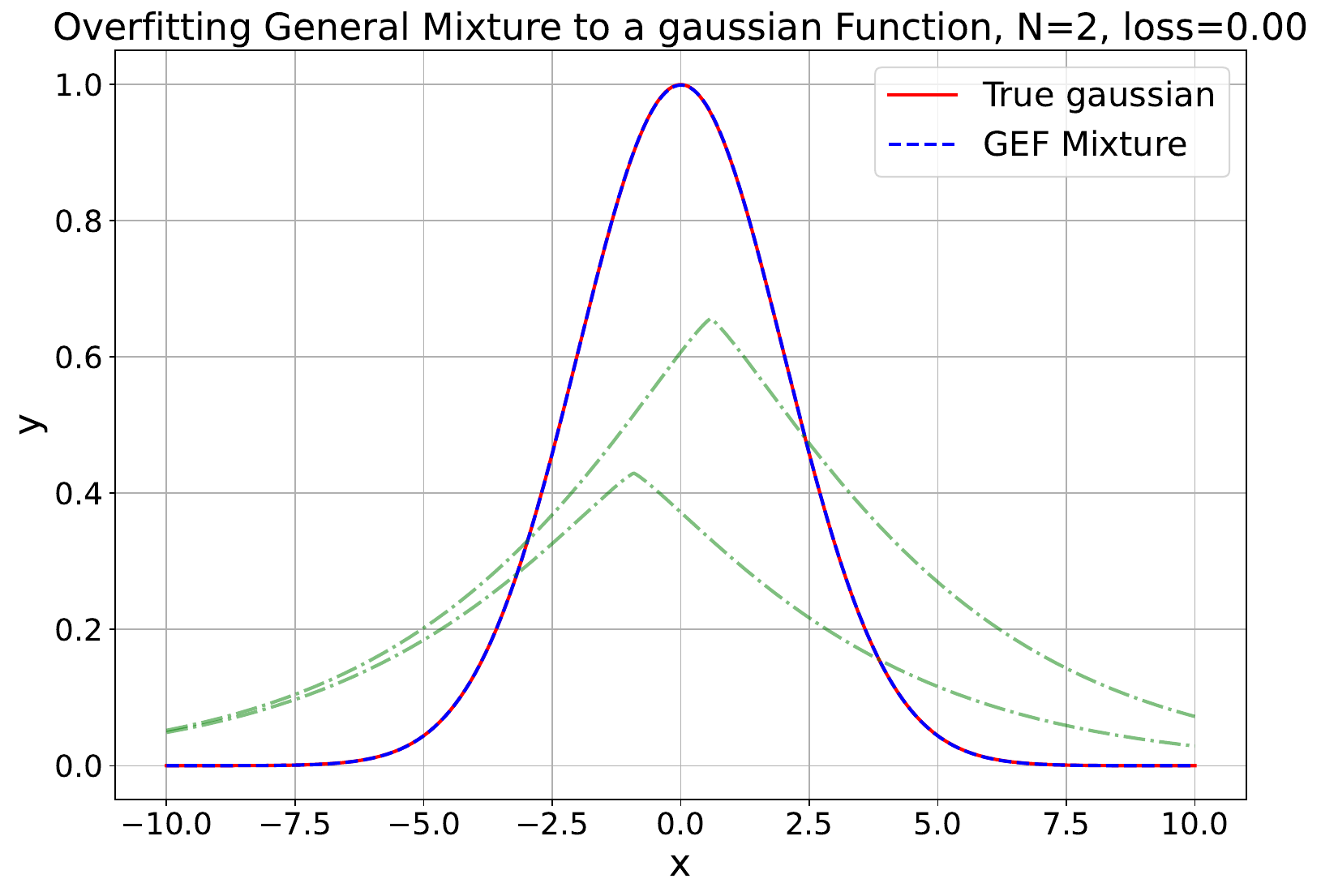}\\ 
    \includegraphics[width=0.24\linewidth]{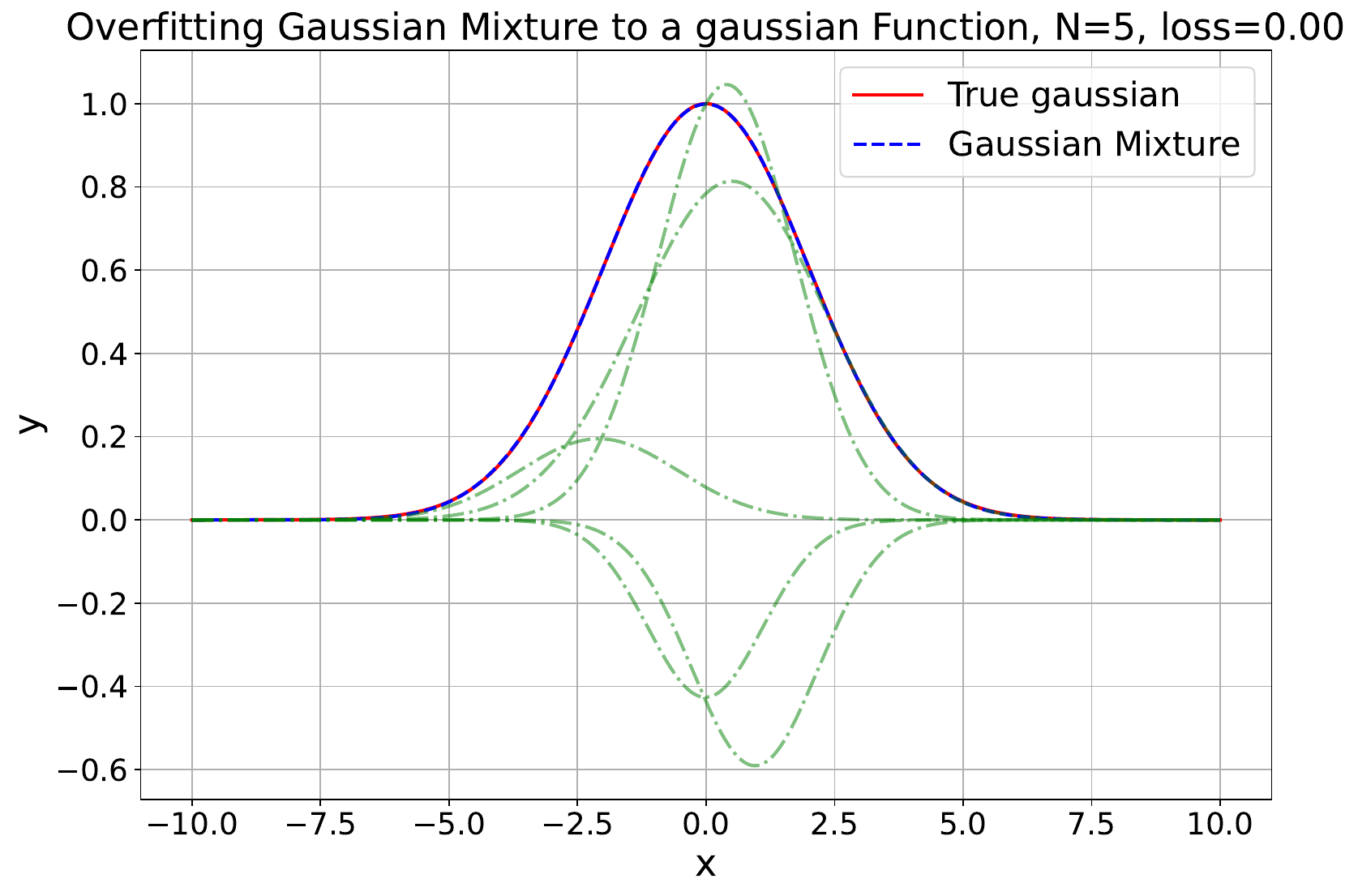} & 
    \includegraphics[width=0.24\linewidth]{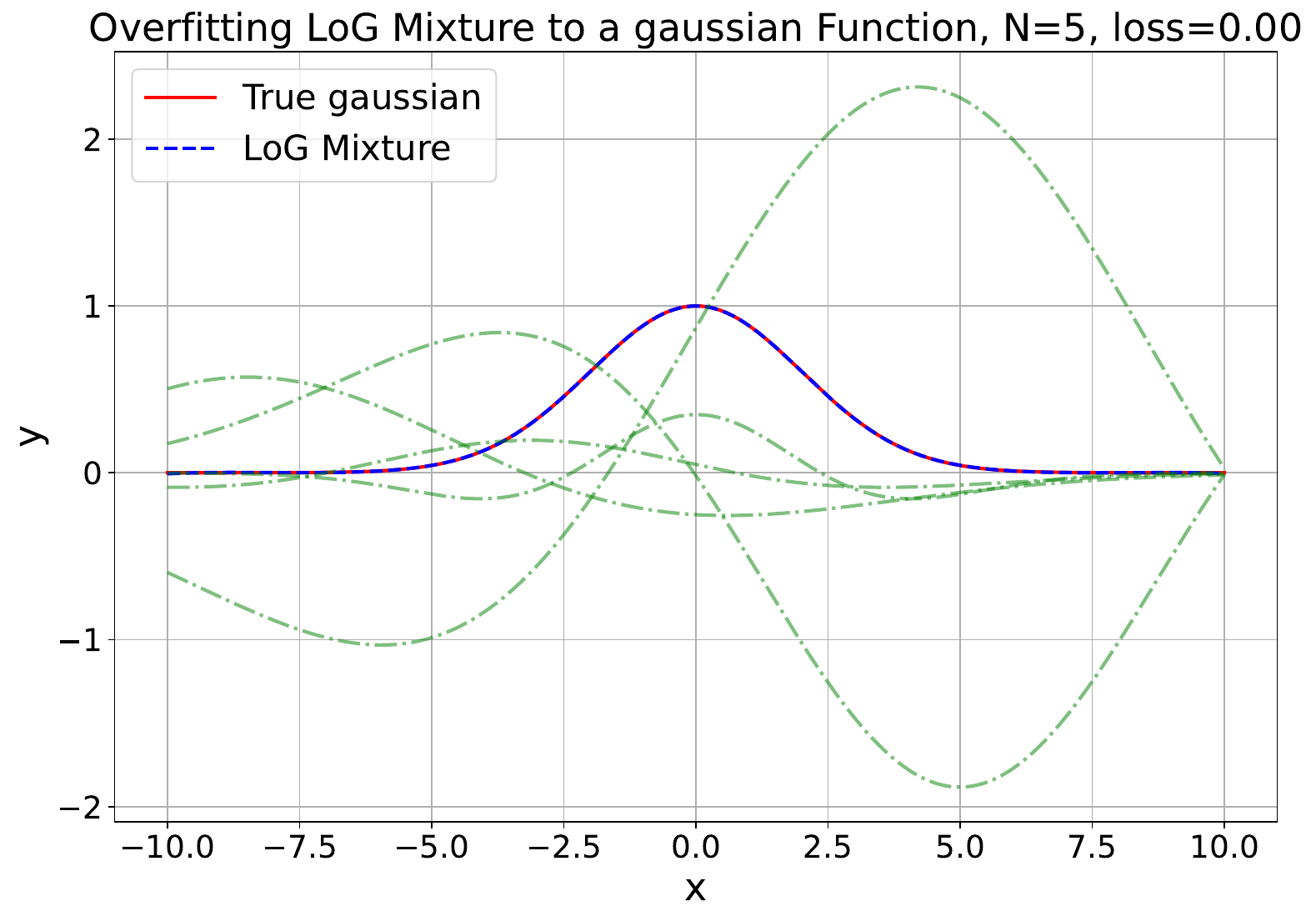} & 
    \includegraphics[width=0.24\linewidth]{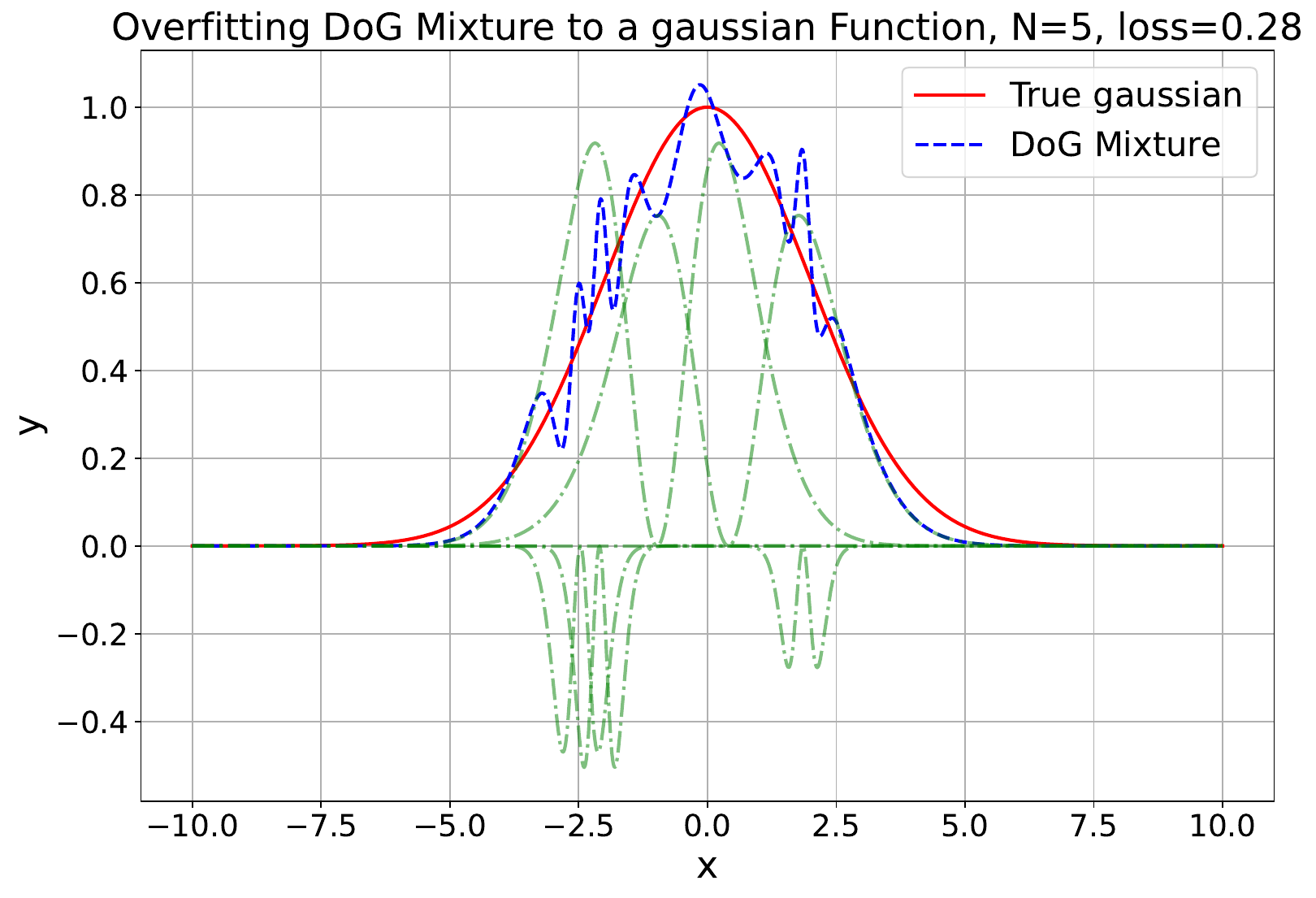} & 
    \includegraphics[width=0.24\linewidth]{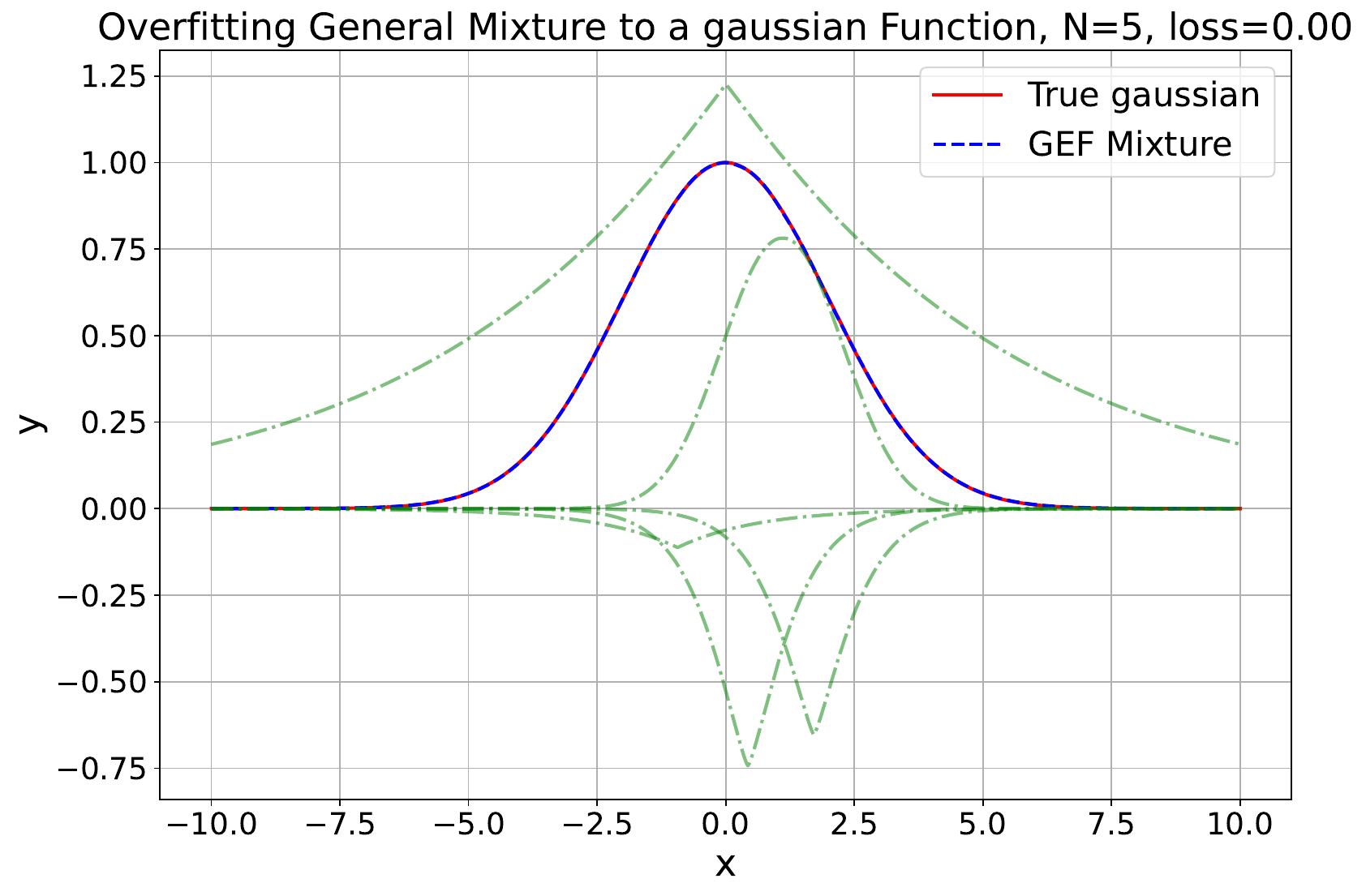}\\ 
    \includegraphics[width=0.24\linewidth]{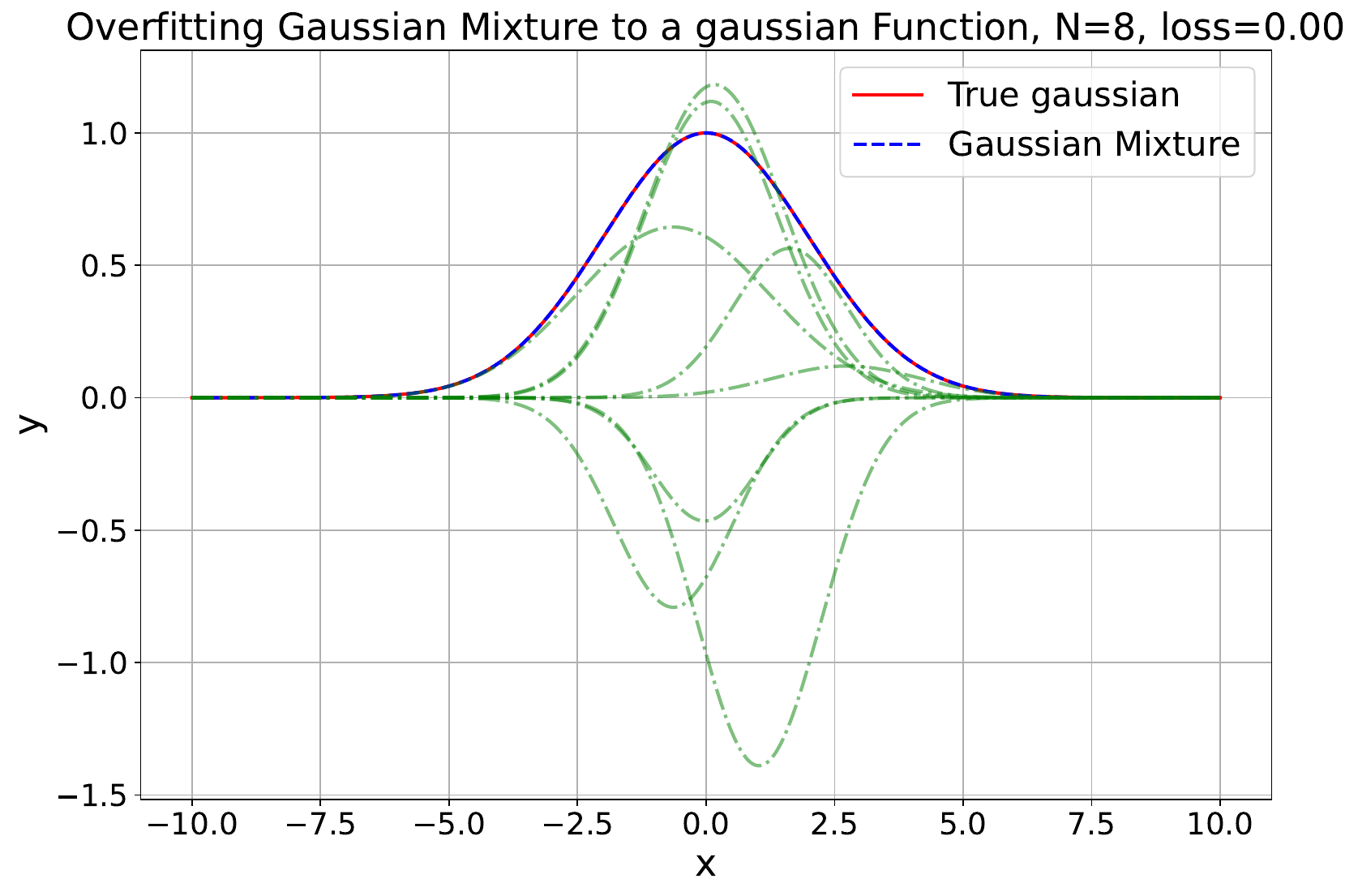} & 
    \includegraphics[width=0.24\linewidth]{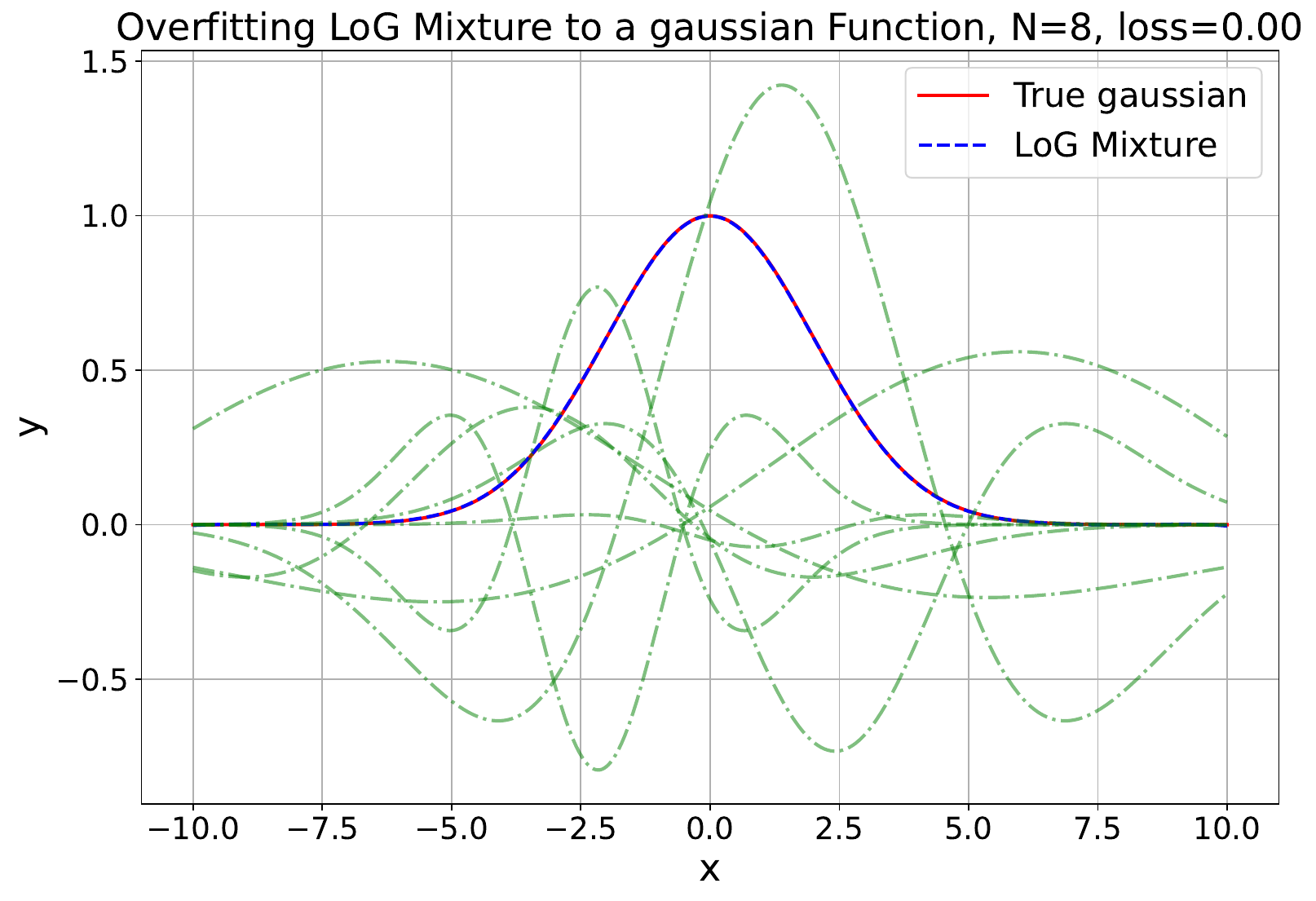} & 
    \includegraphics[width=0.24\linewidth]{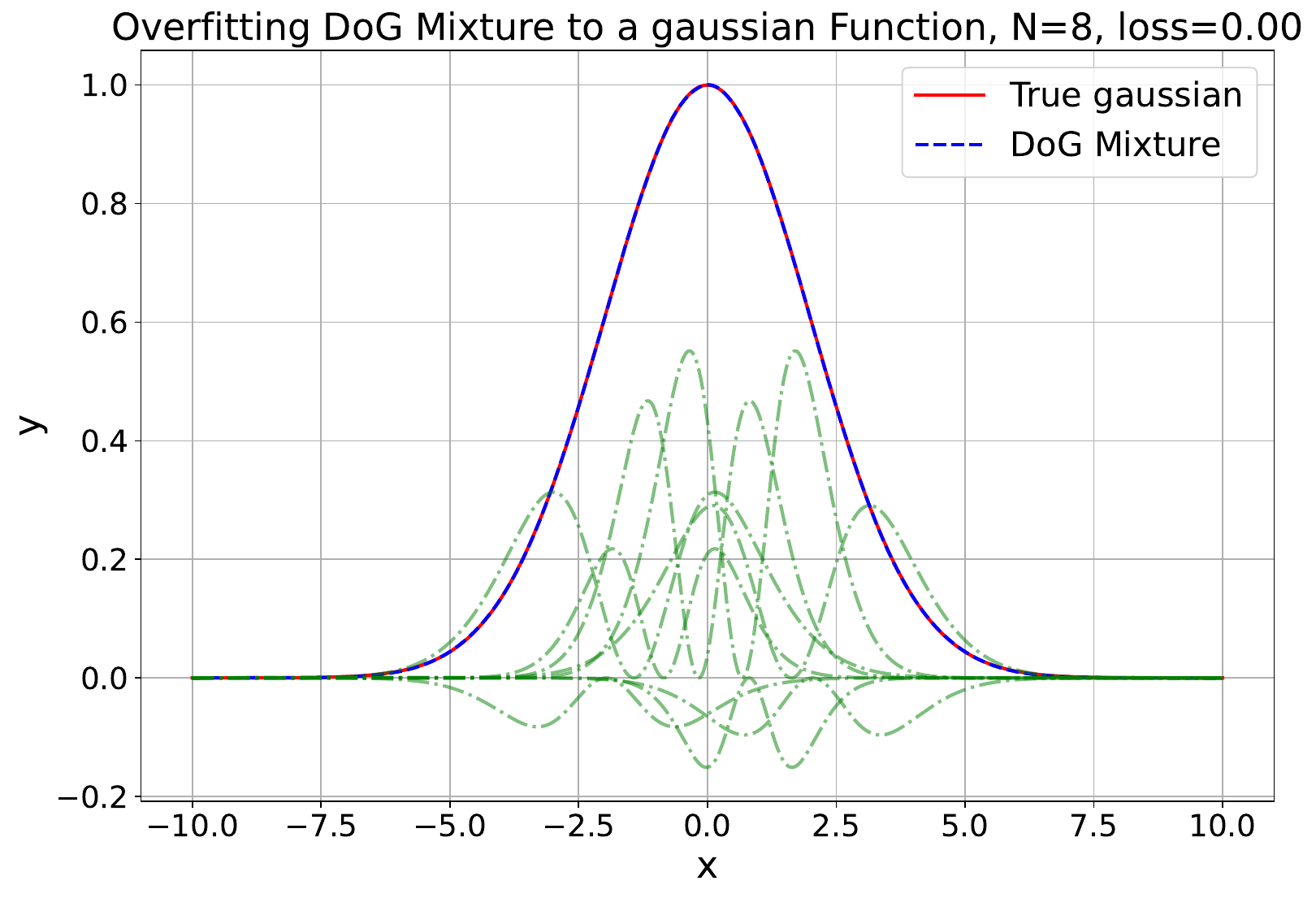} & 
    \includegraphics[width=0.24\linewidth]{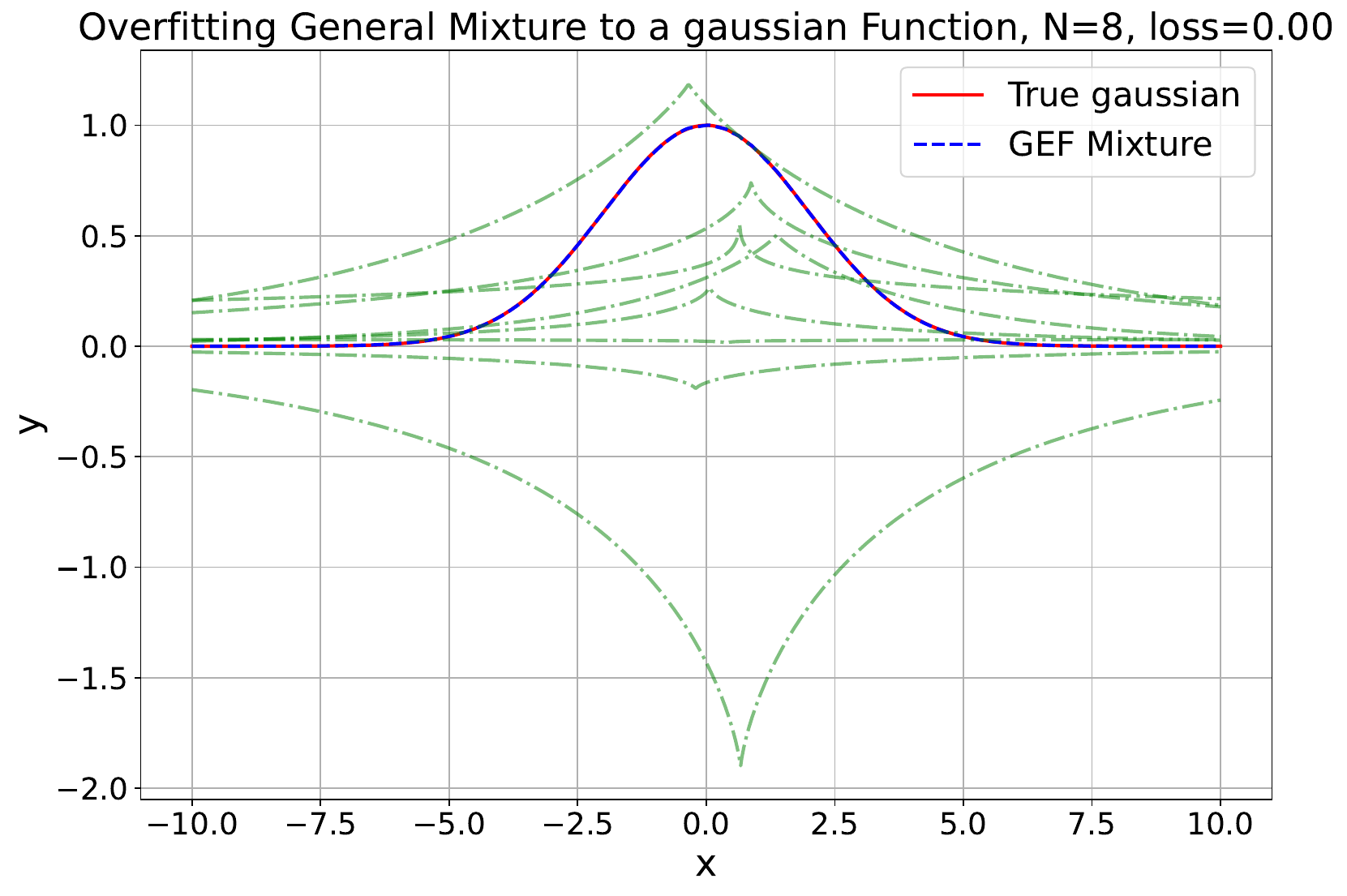}\\ 
    \includegraphics[width=0.24\linewidth]{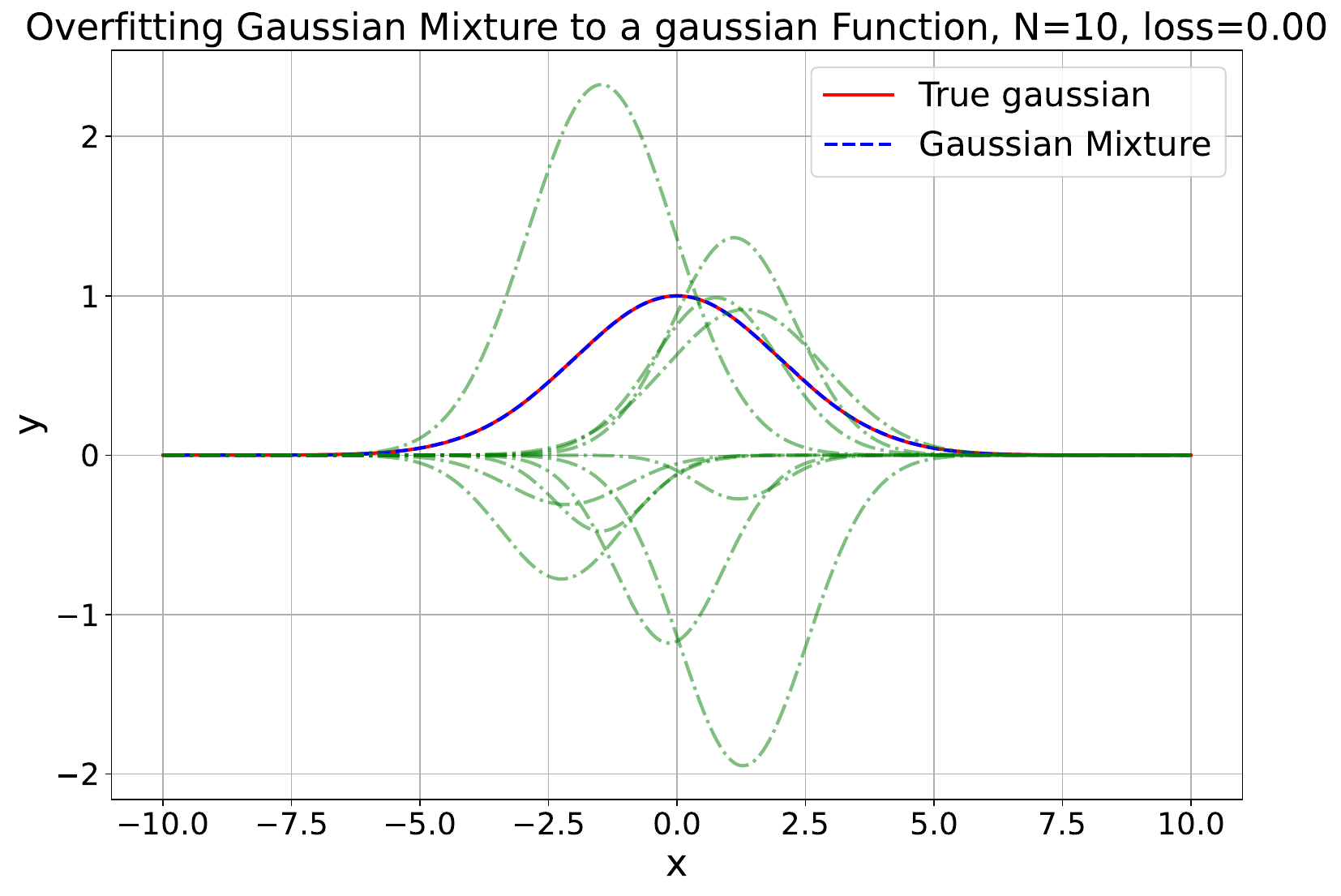} & 
    \includegraphics[width=0.24\linewidth]{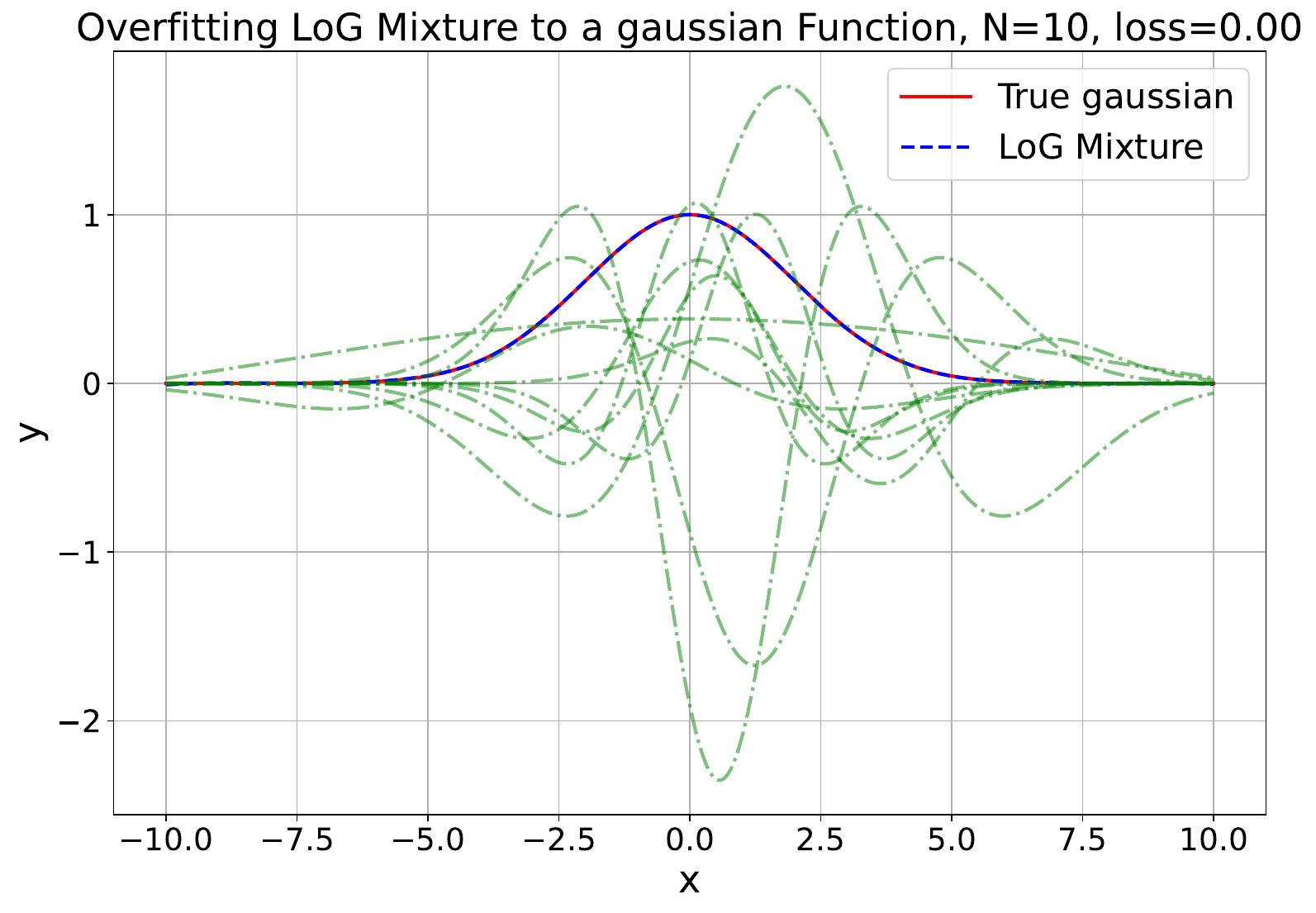} & 
    \includegraphics[width=0.24\linewidth]{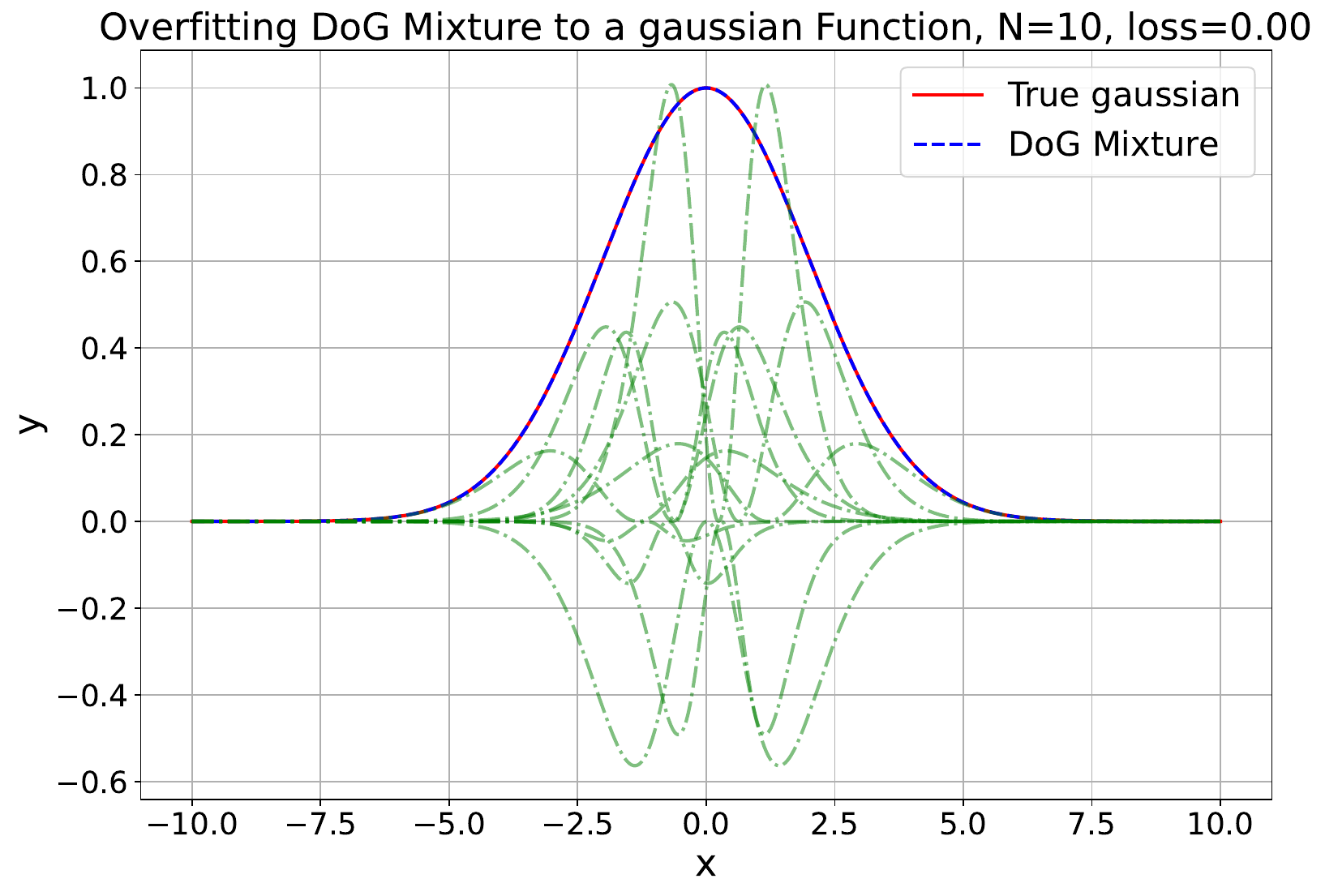} & 
    \includegraphics[width=0.24\linewidth]{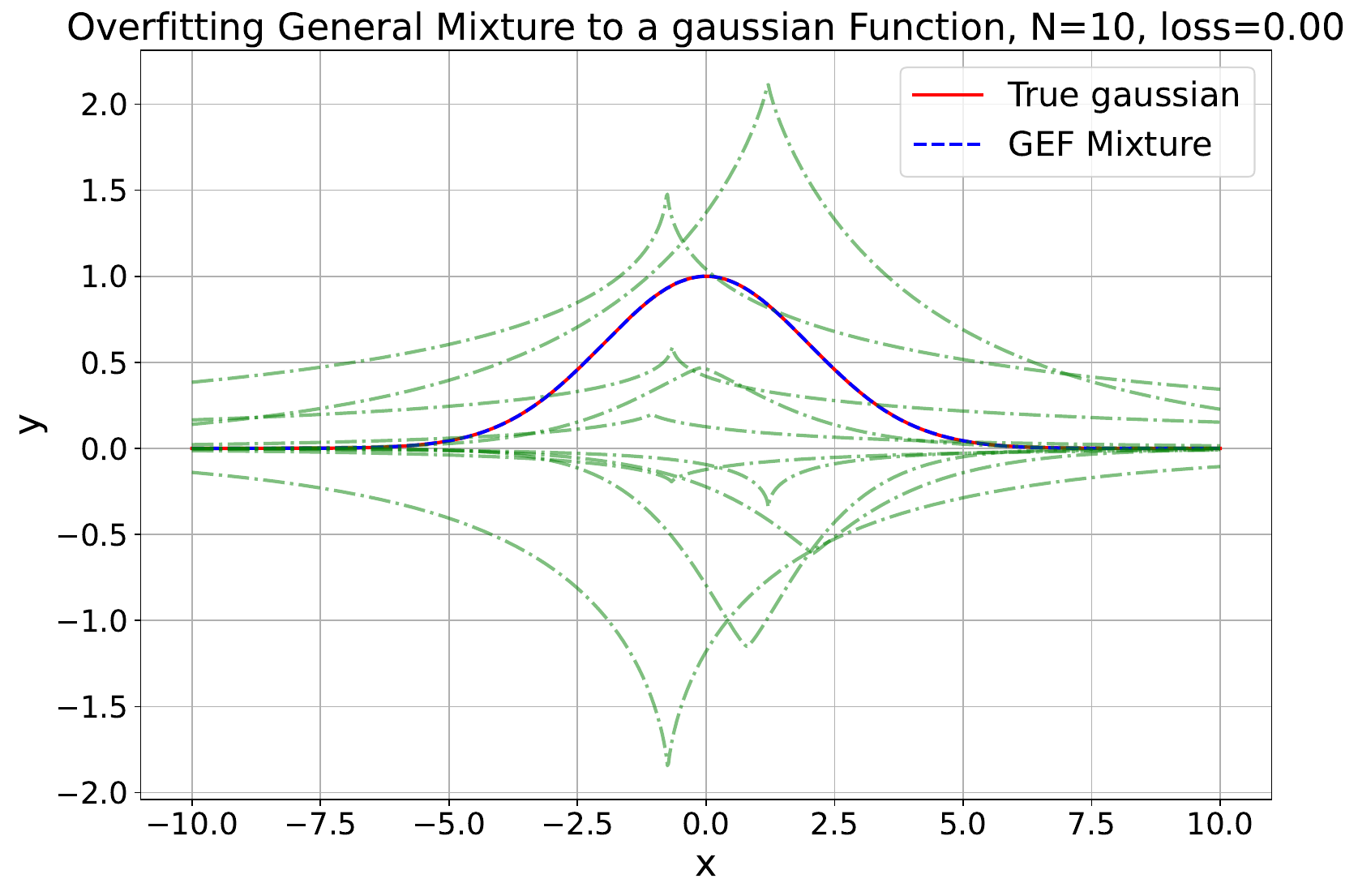}\\ 
    \includegraphics[width=0.24\linewidth]{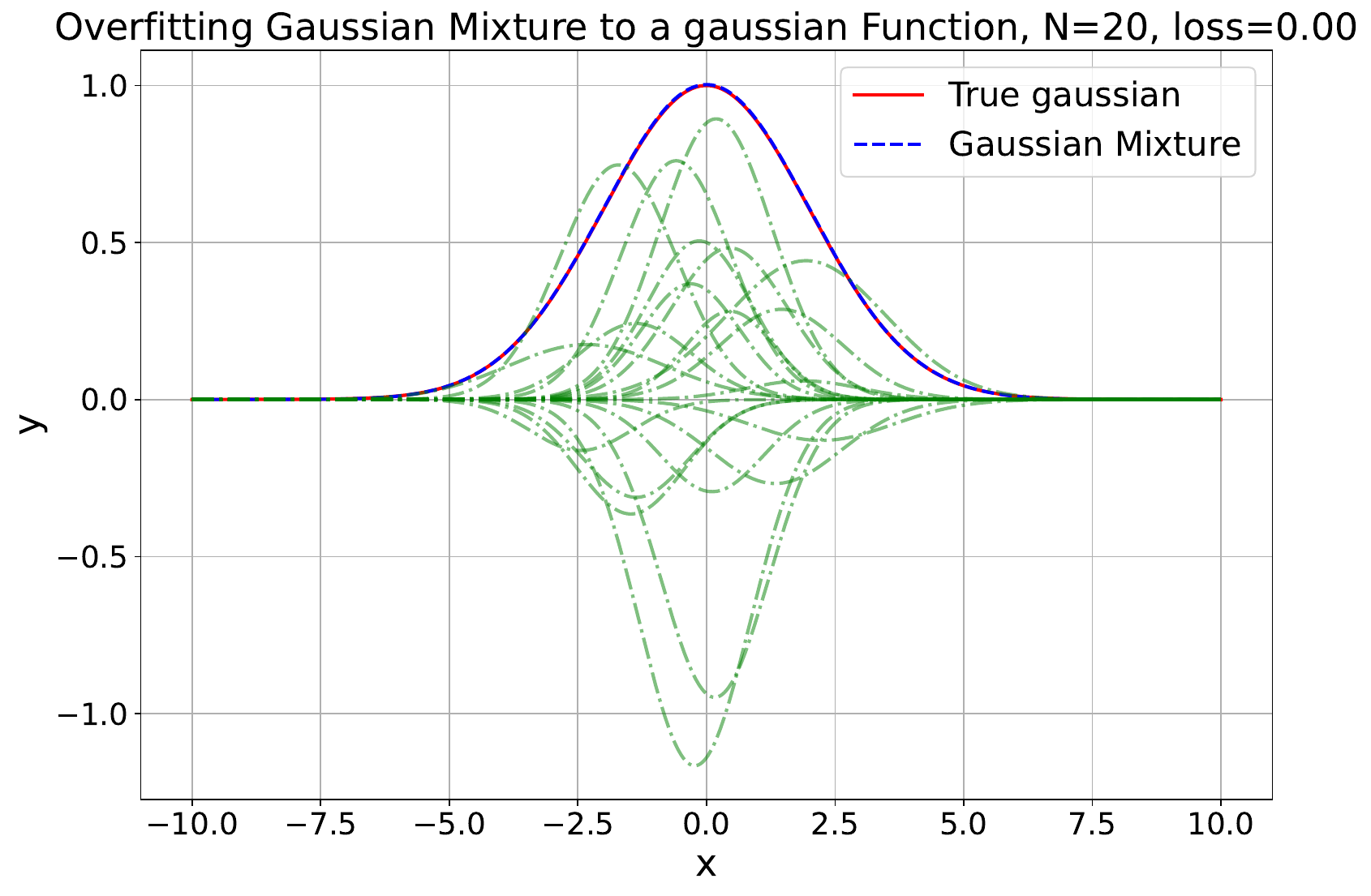} & 
    \includegraphics[width=0.24\linewidth]{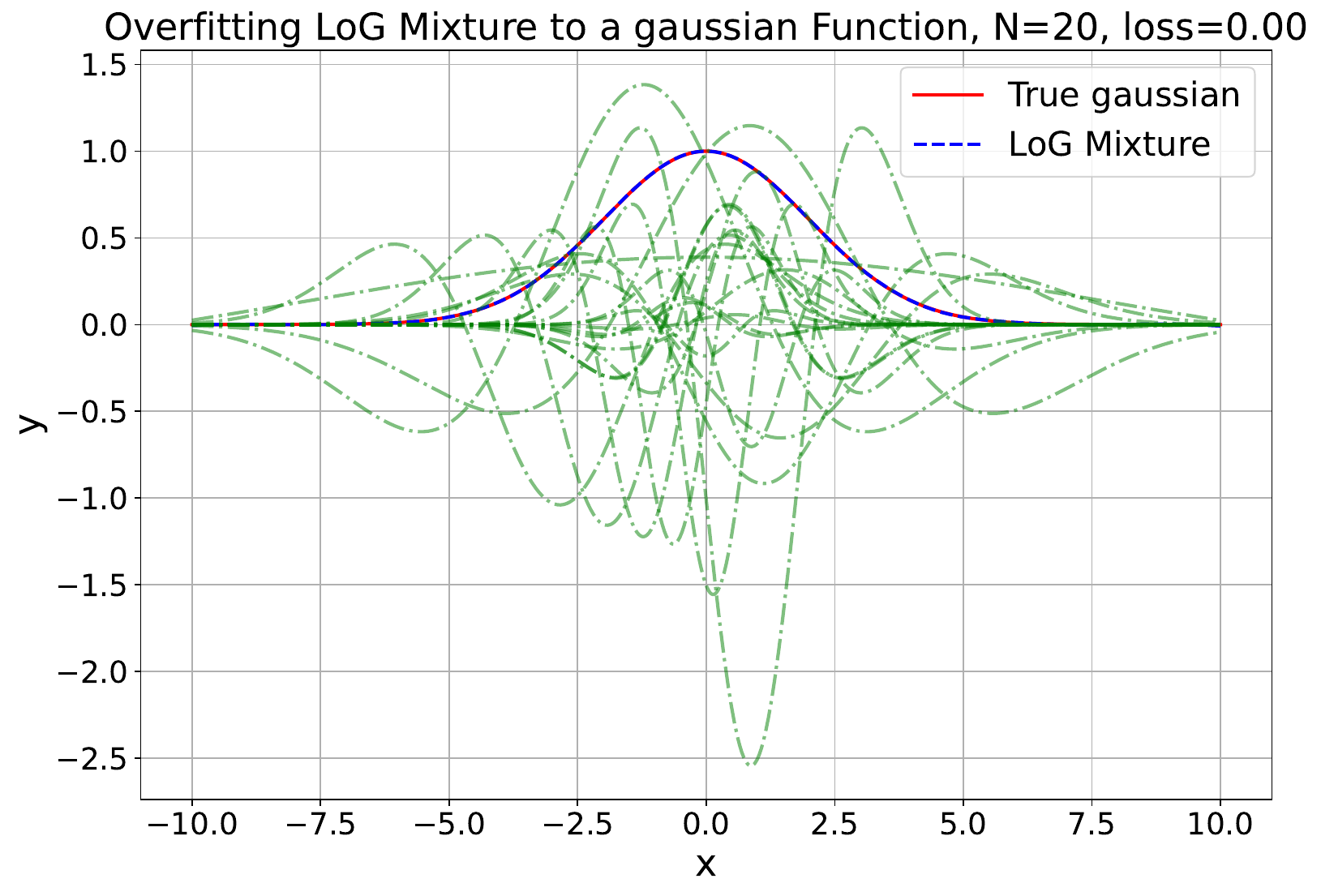} & 
    \includegraphics[width=0.24\linewidth]{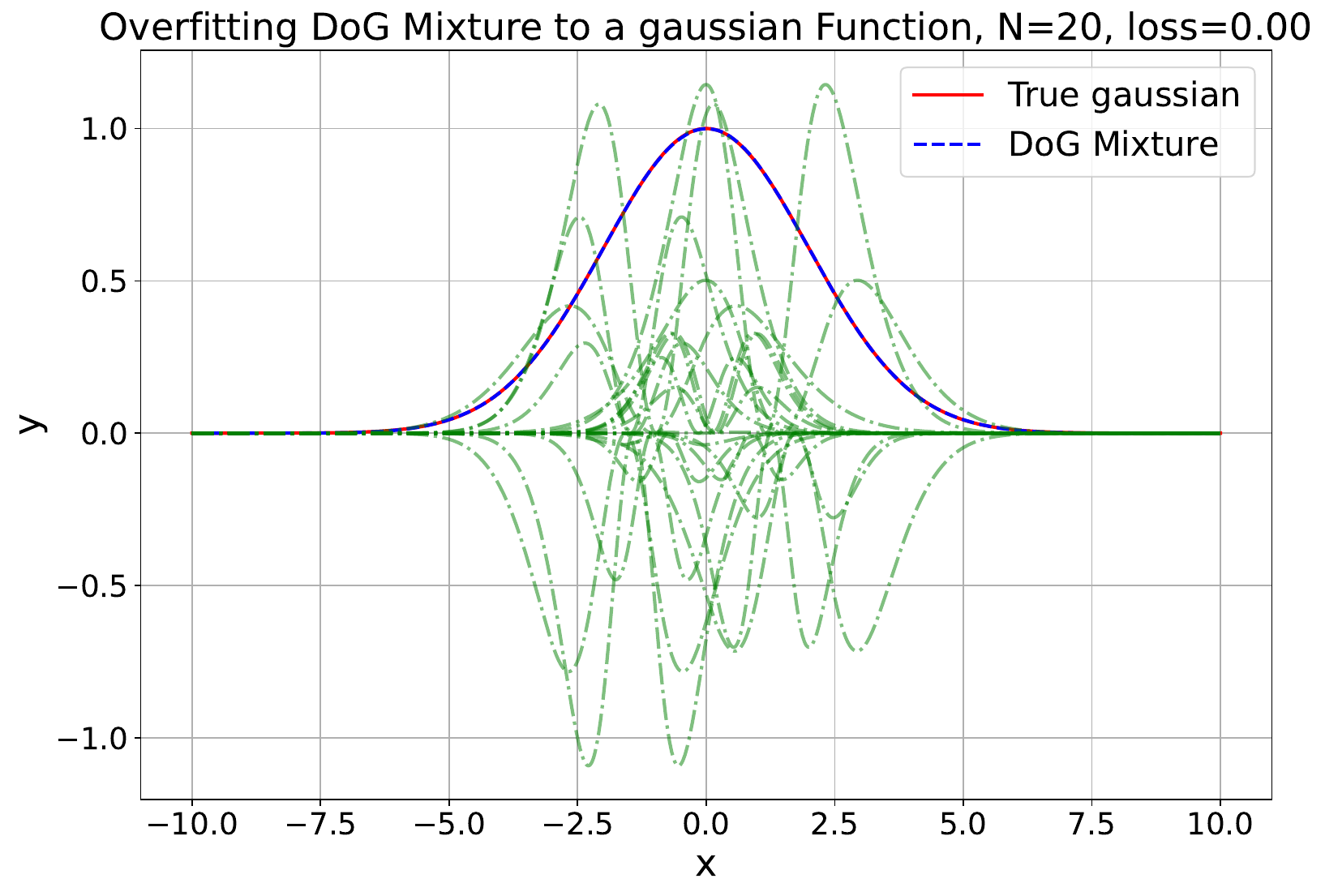} & 
    \includegraphics[width=0.24\linewidth]{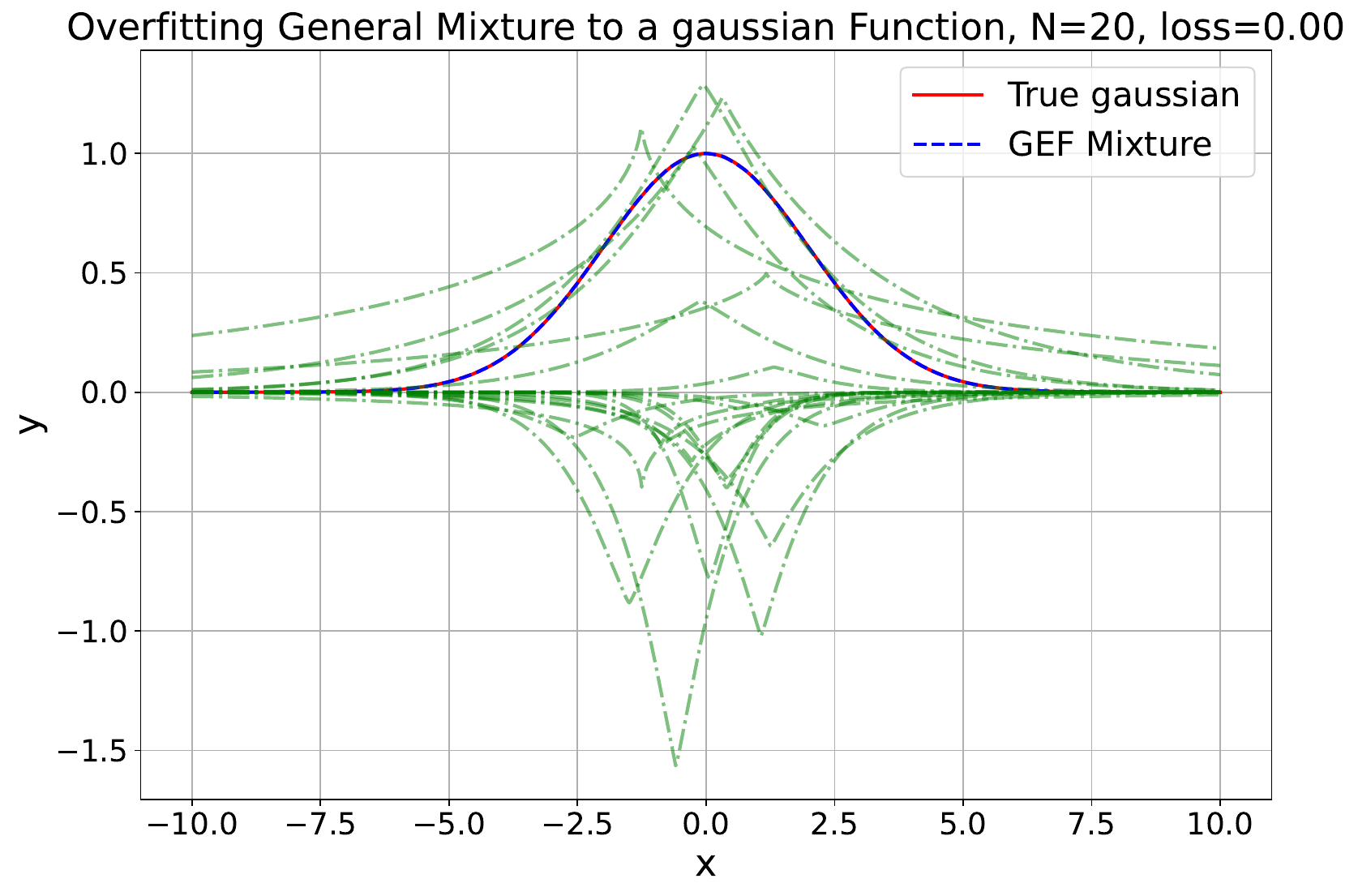}\\ 
    
    \end{tabular}
    }
    \caption{\textbf{Numerical Simulation Examples of Fitting Gaussians with Real Weights Mixtures ( N= 2, 5, 8, and 10 )}. We show some fitting examples for Gaussian signals with Real weights mixtures (can be negative). The four mixtures used from left to right are Gaussians, LoG, DoG, and General mixtures. From top to bottom: N = 2, 8, and 10 components. The optimized individual components are shown in green. Some examples fail to optimize due to numerical instability in both Gaussians and GEF mixtures. Note that GEF is very efficient in fitting the Gaussian with few components while LoG and DoG are more stable for a larger number of components. }
    \label{supfig:fitting_gaussian_N}
    \end{figure*}
    

%% file: figures/fitting/fitting_sinusoid_p.tex
\begin{figure*}[h]
    \centering
    \resizebox{1.0\linewidth}{!}{
    \begin{tabular}{cccc}
    \tabcolsep=0.01cm
    Gaussian Mixture& LoG Mixture & DoG Mixture & GEF Mixture \\ 
    \includegraphics[width=0.24\linewidth]{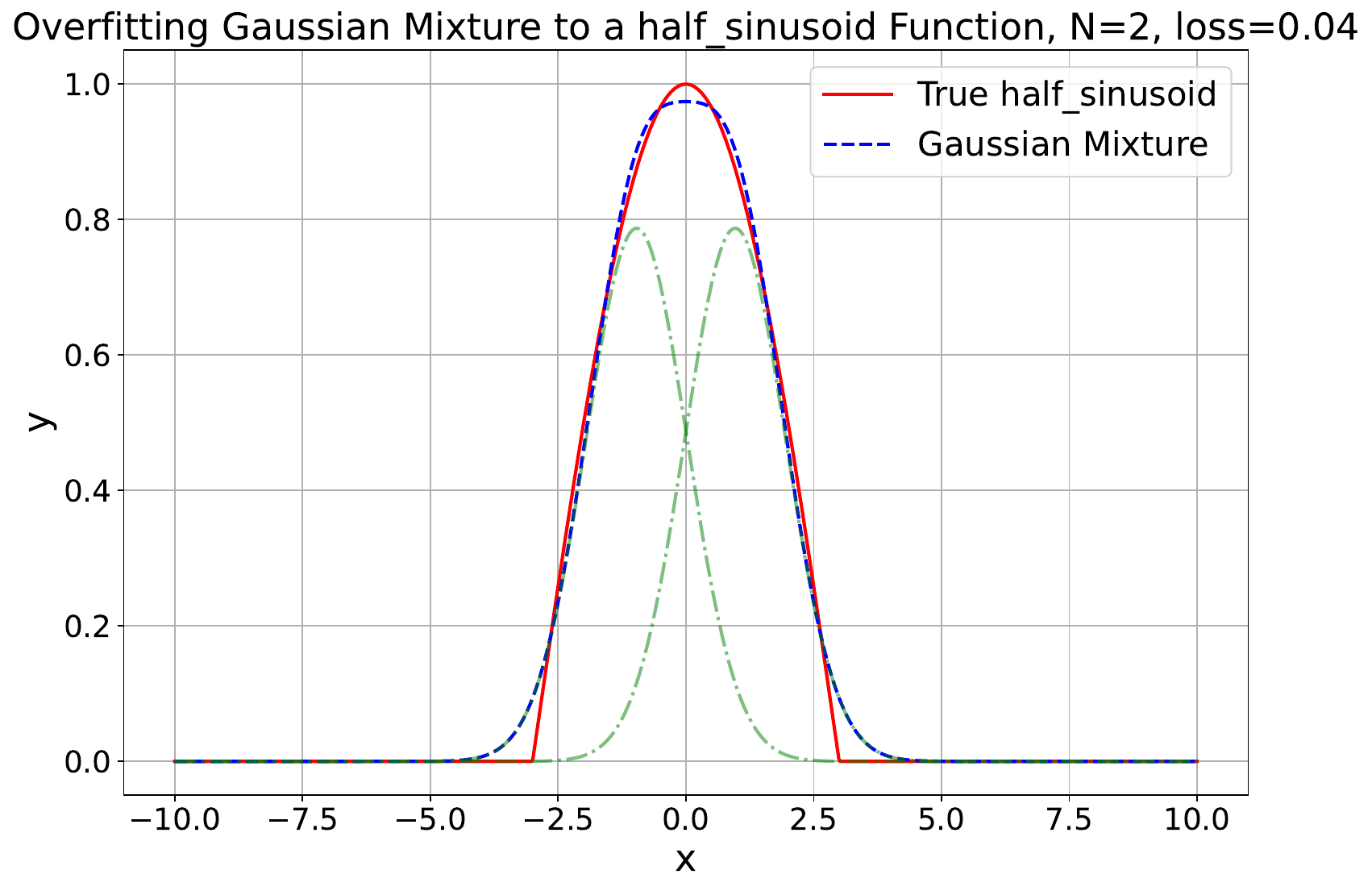} & 
    \includegraphics[width=0.24\linewidth]{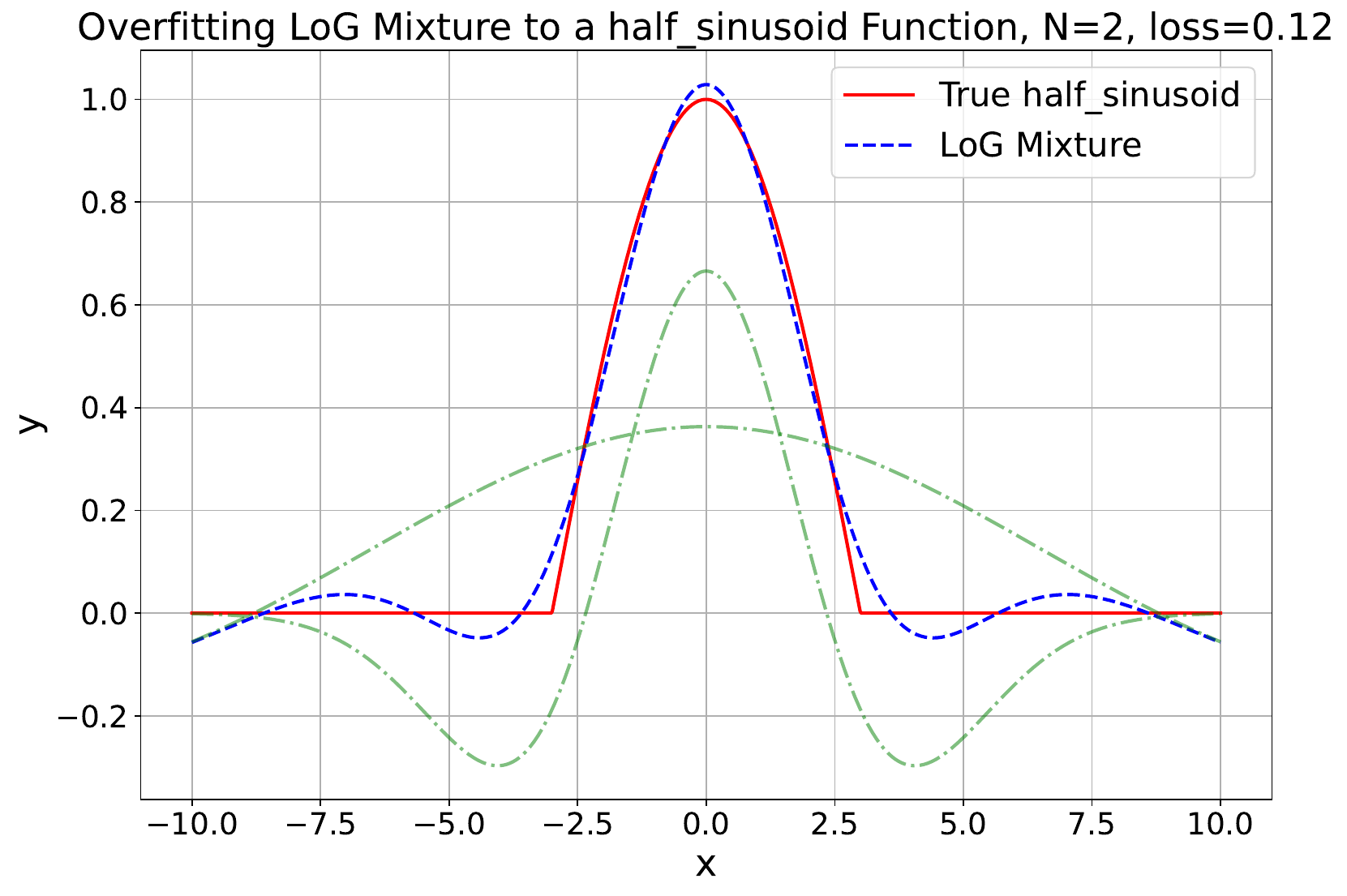} & 
    \includegraphics[width=0.24\linewidth]{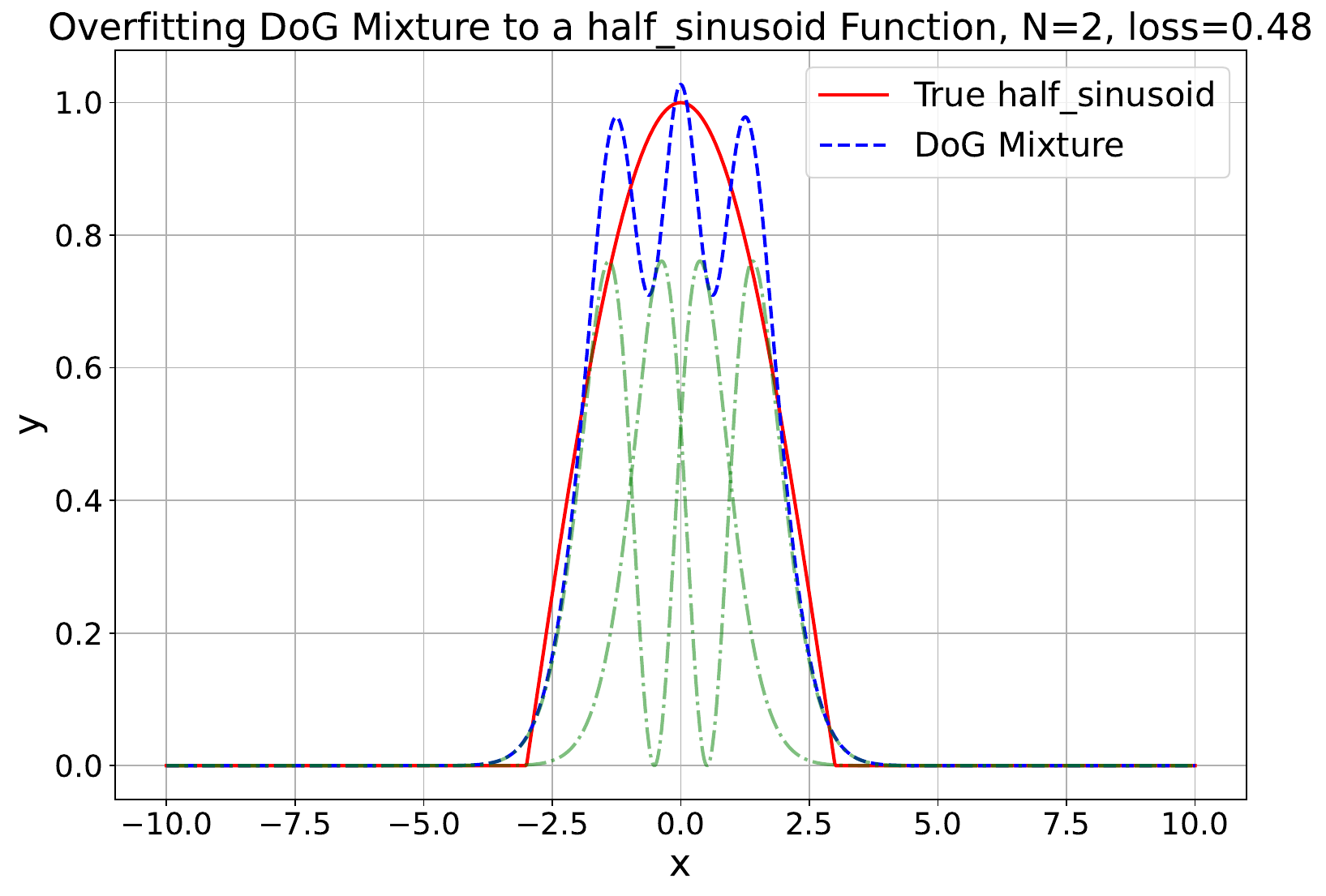} & 
    \includegraphics[width=0.24\linewidth]{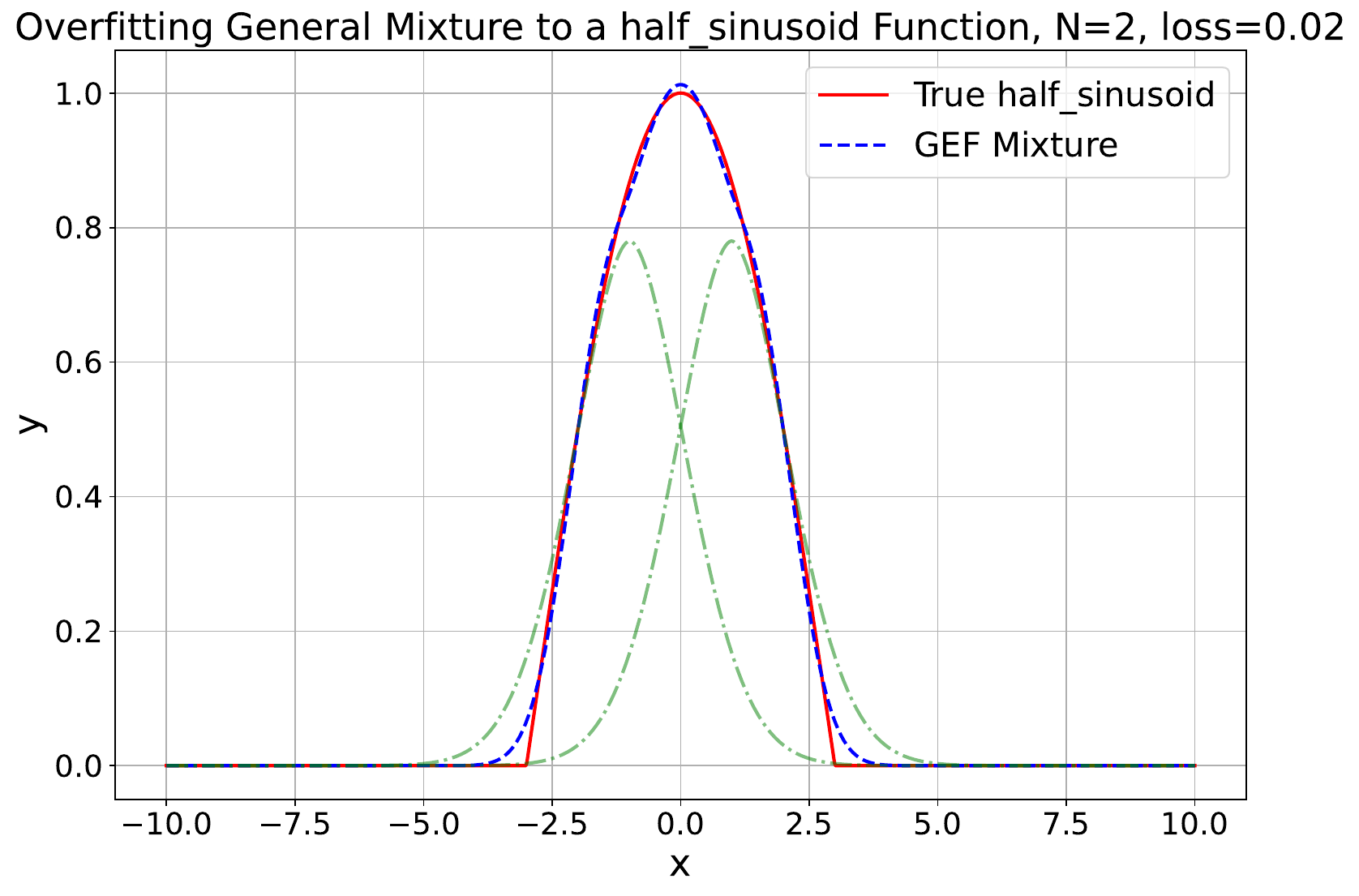}\\ 
    \includegraphics[width=0.24\linewidth]{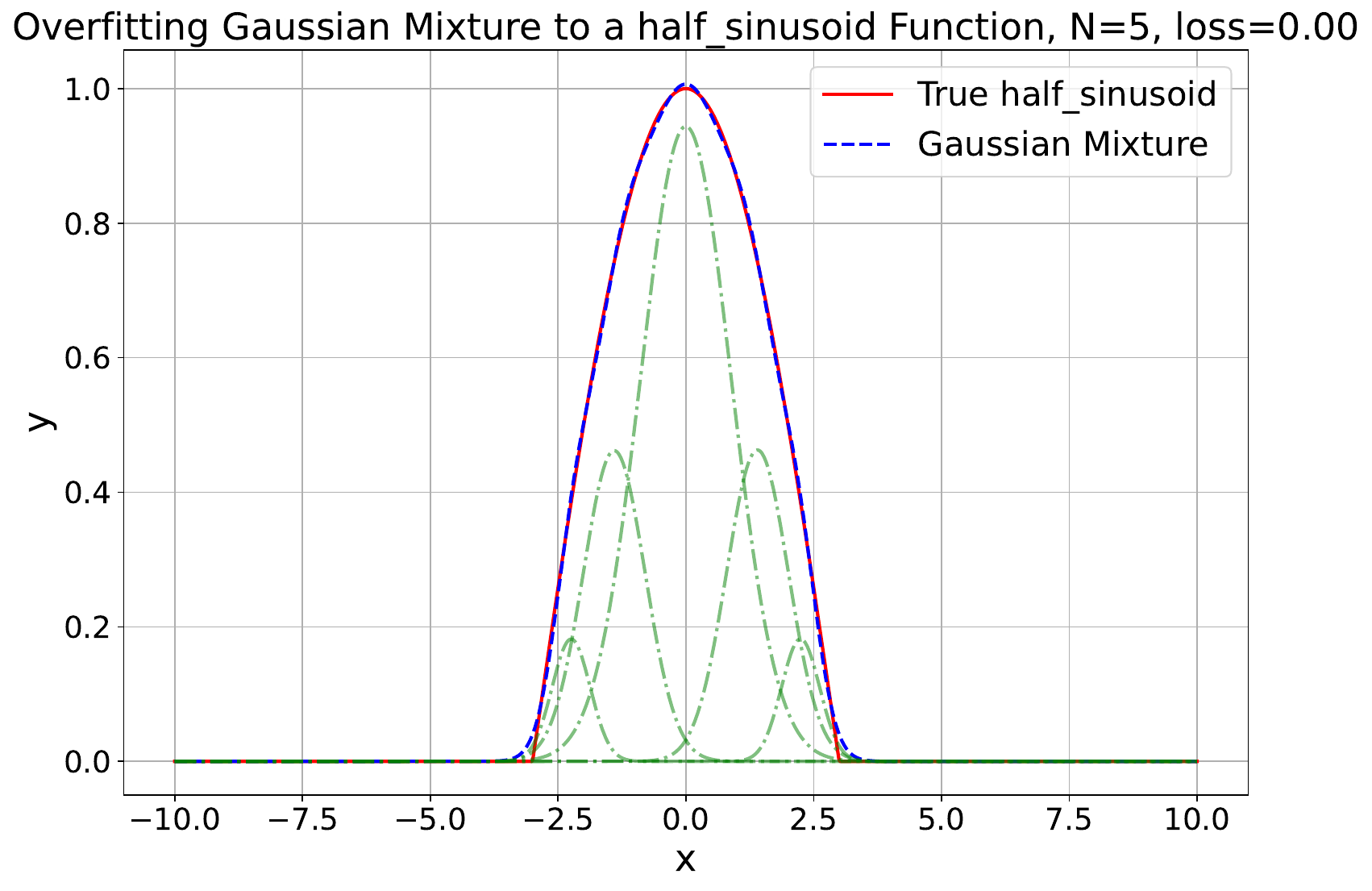} & 
    \includegraphics[width=0.24\linewidth]{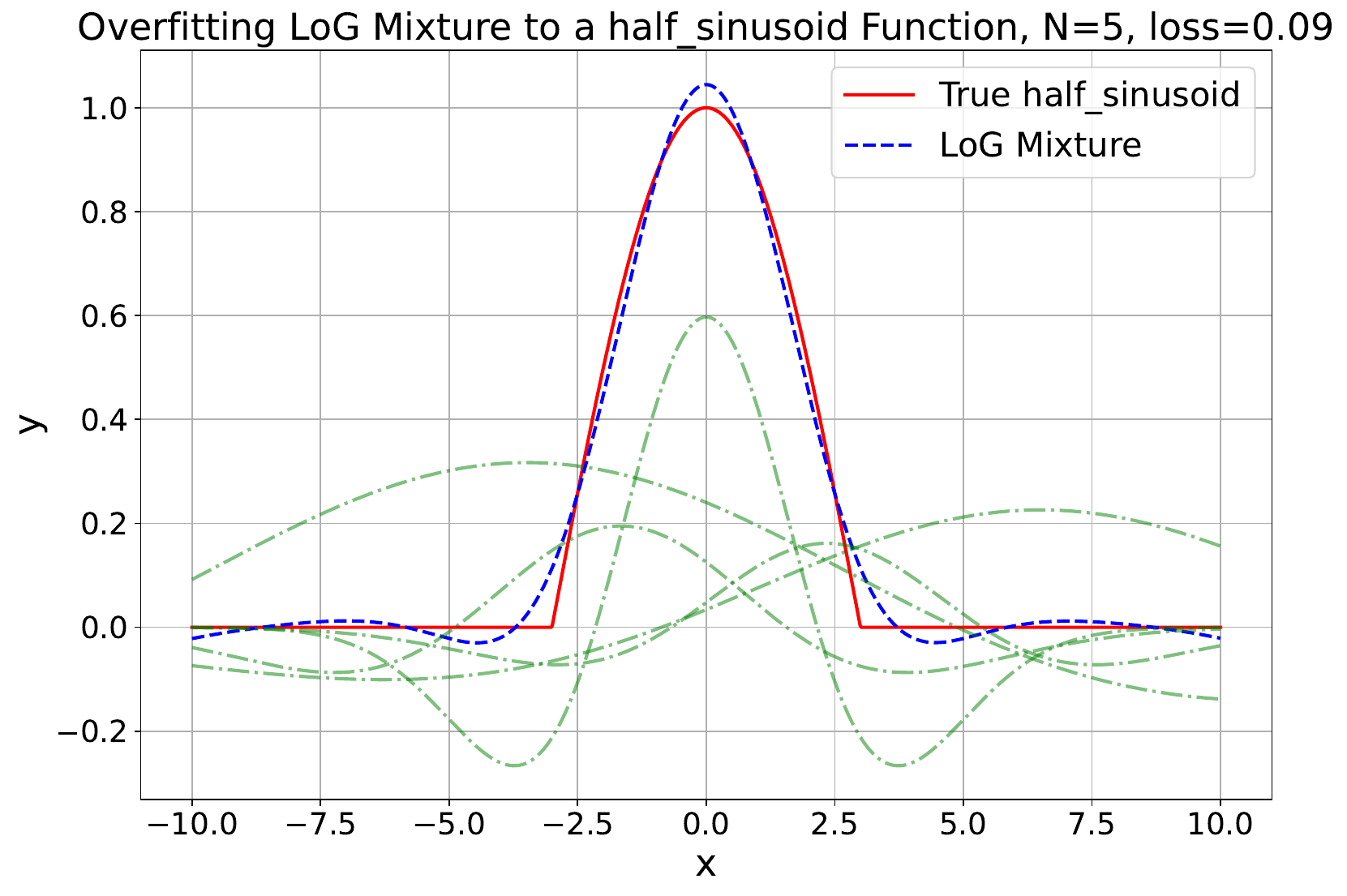} & 
    \includegraphics[width=0.24\linewidth]{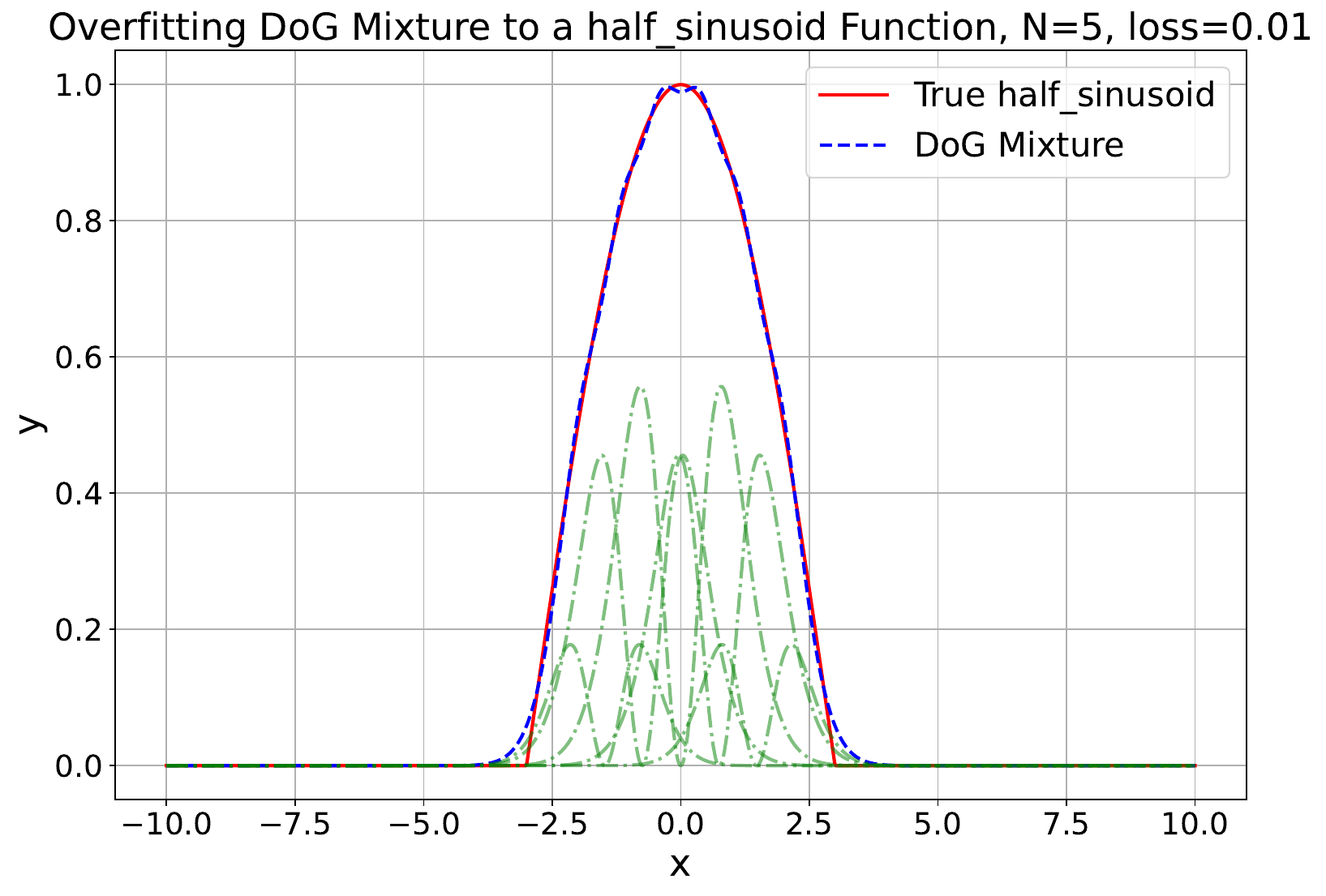} & 
    \includegraphics[width=0.24\linewidth]{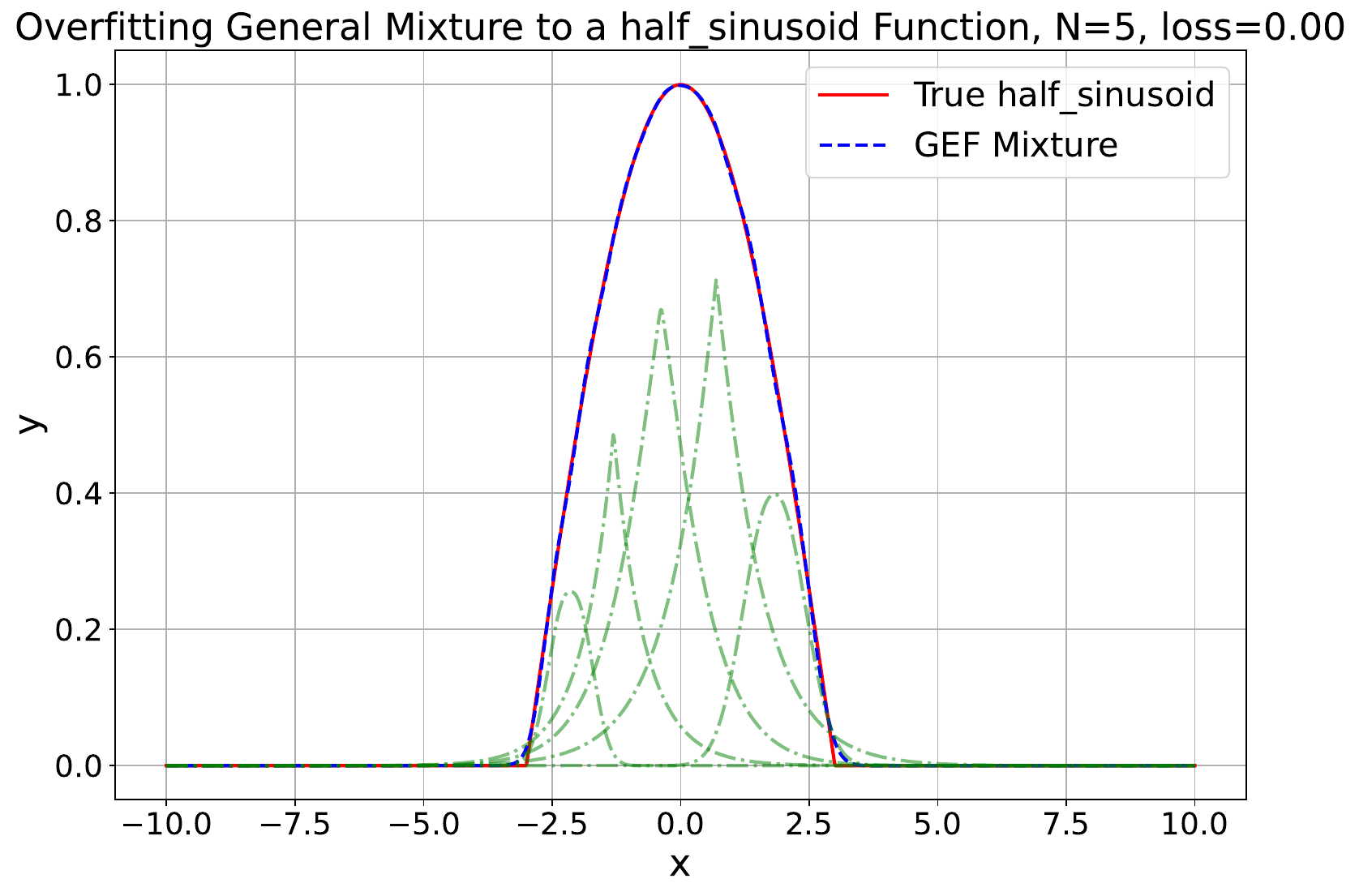}\\ 
    \includegraphics[width=0.24\linewidth]{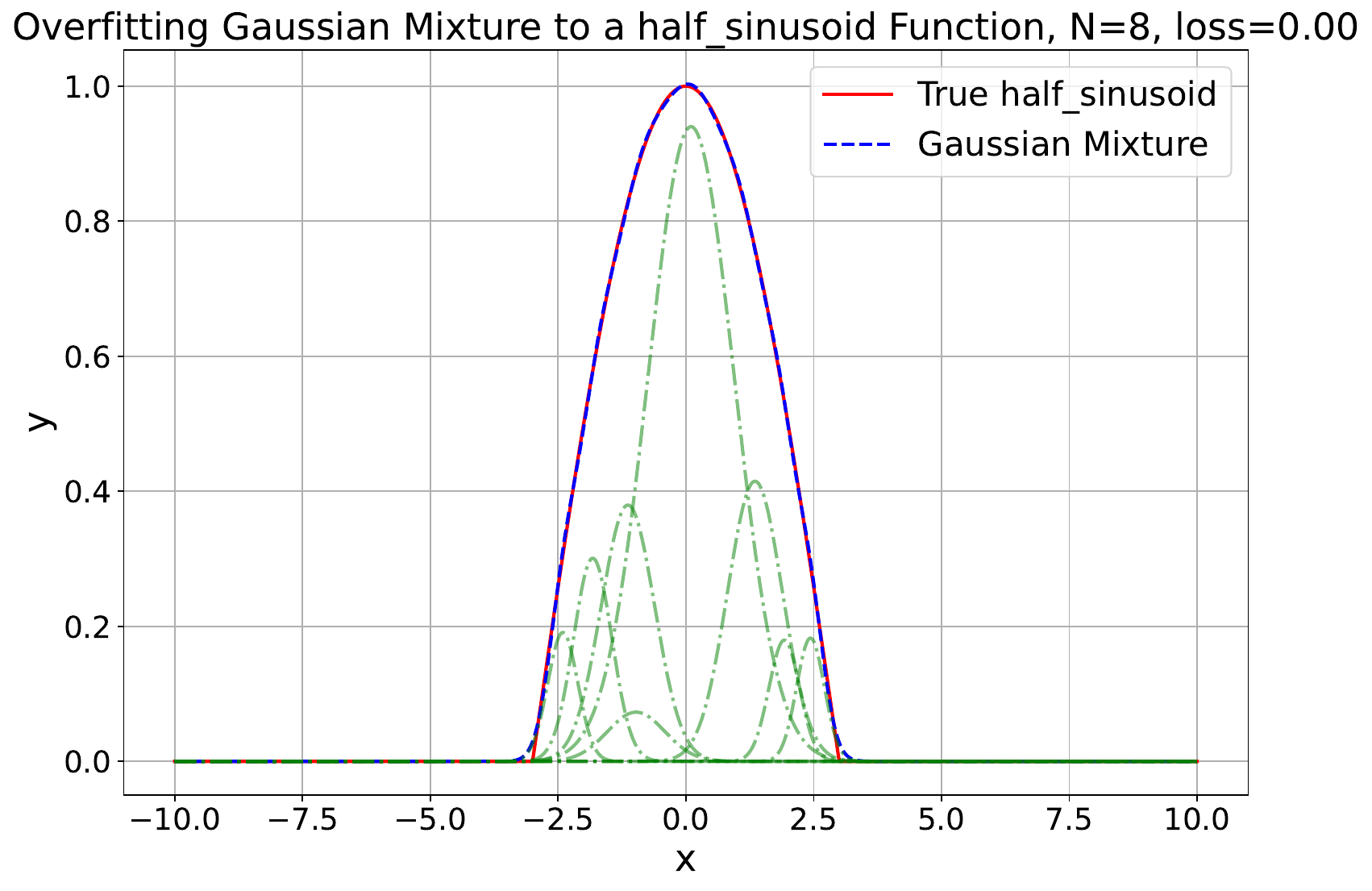} & 
    \includegraphics[width=0.24\linewidth]{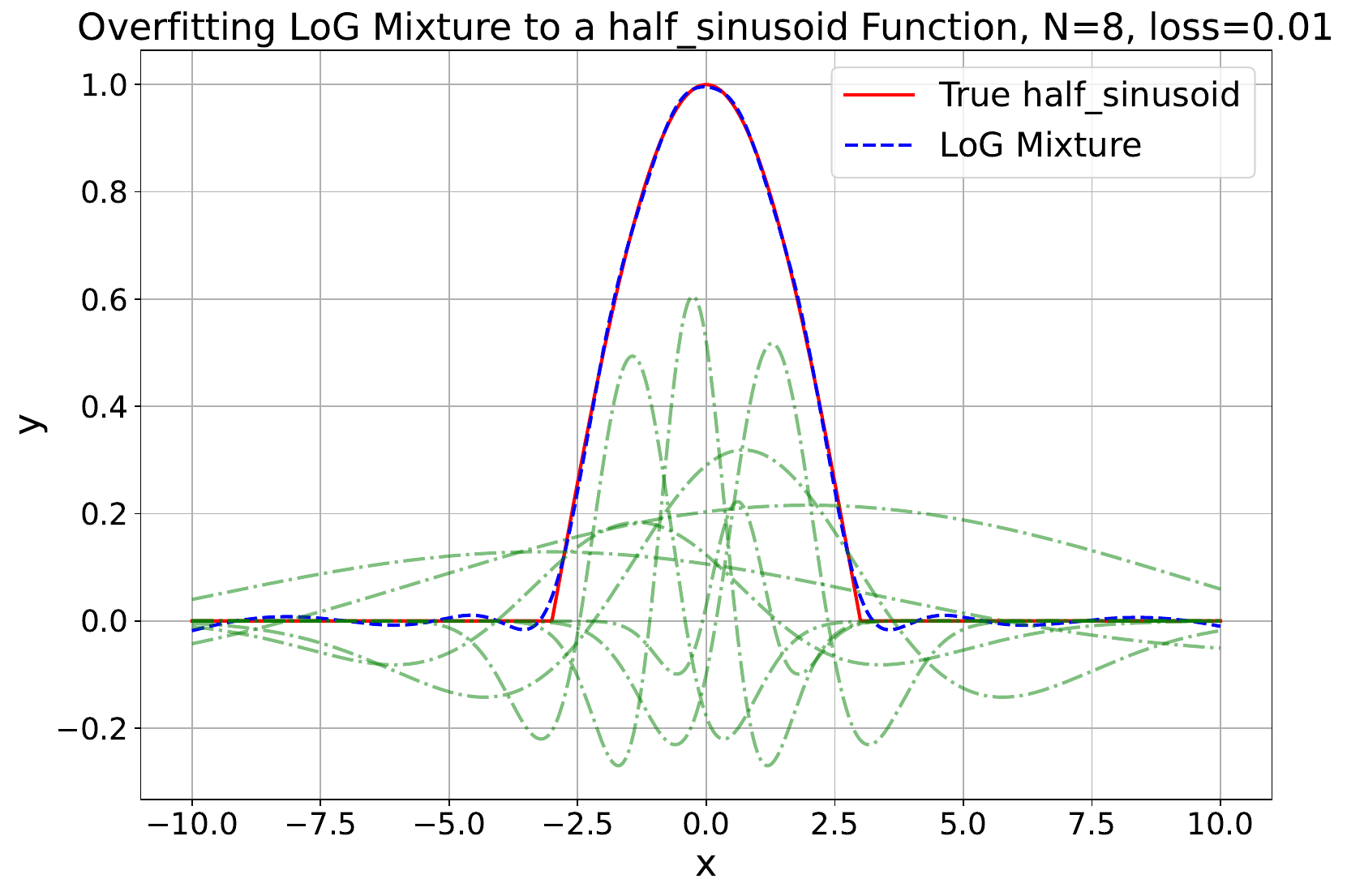} & 
    \includegraphics[width=0.24\linewidth]{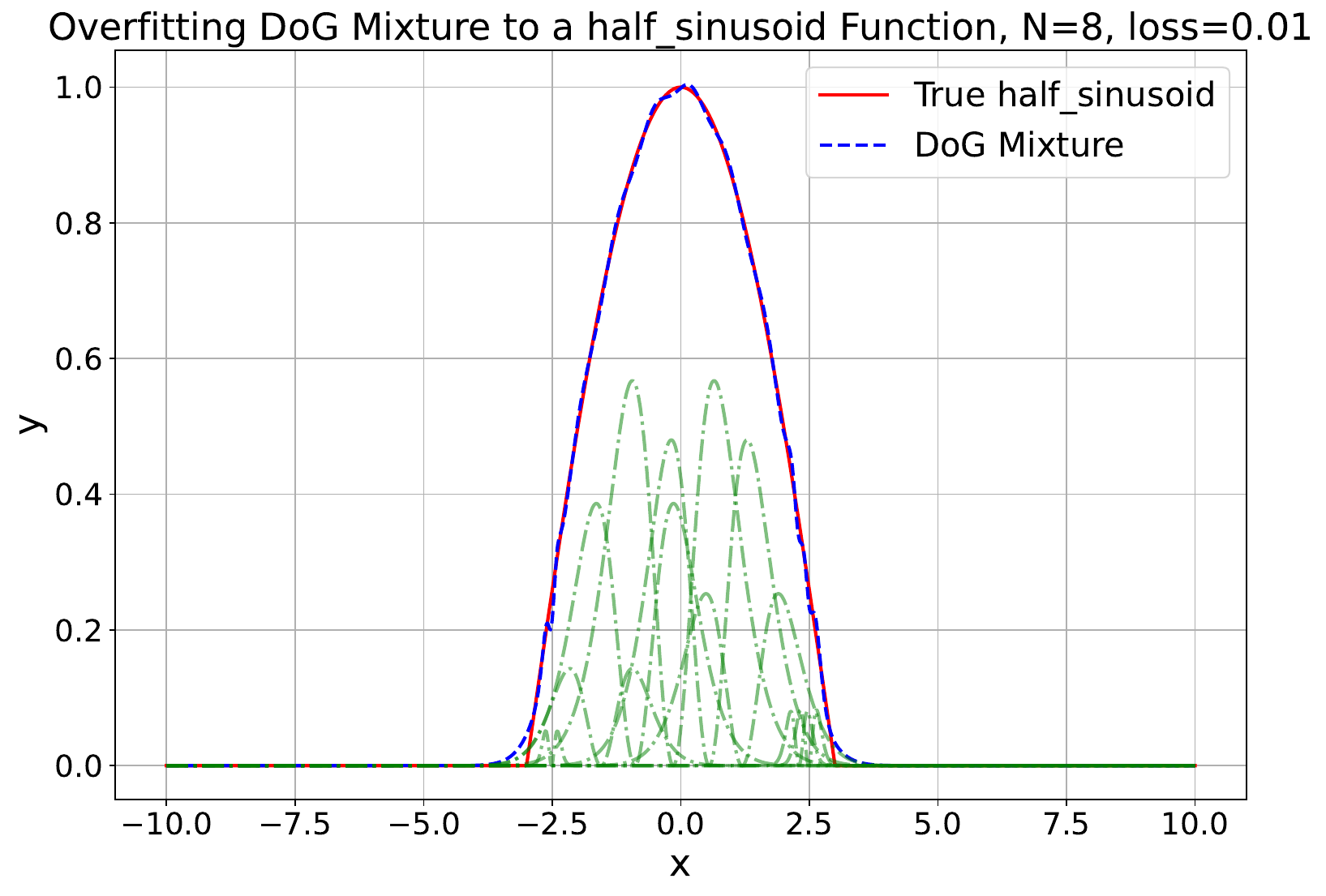} & 
    \includegraphics[width=0.24\linewidth]{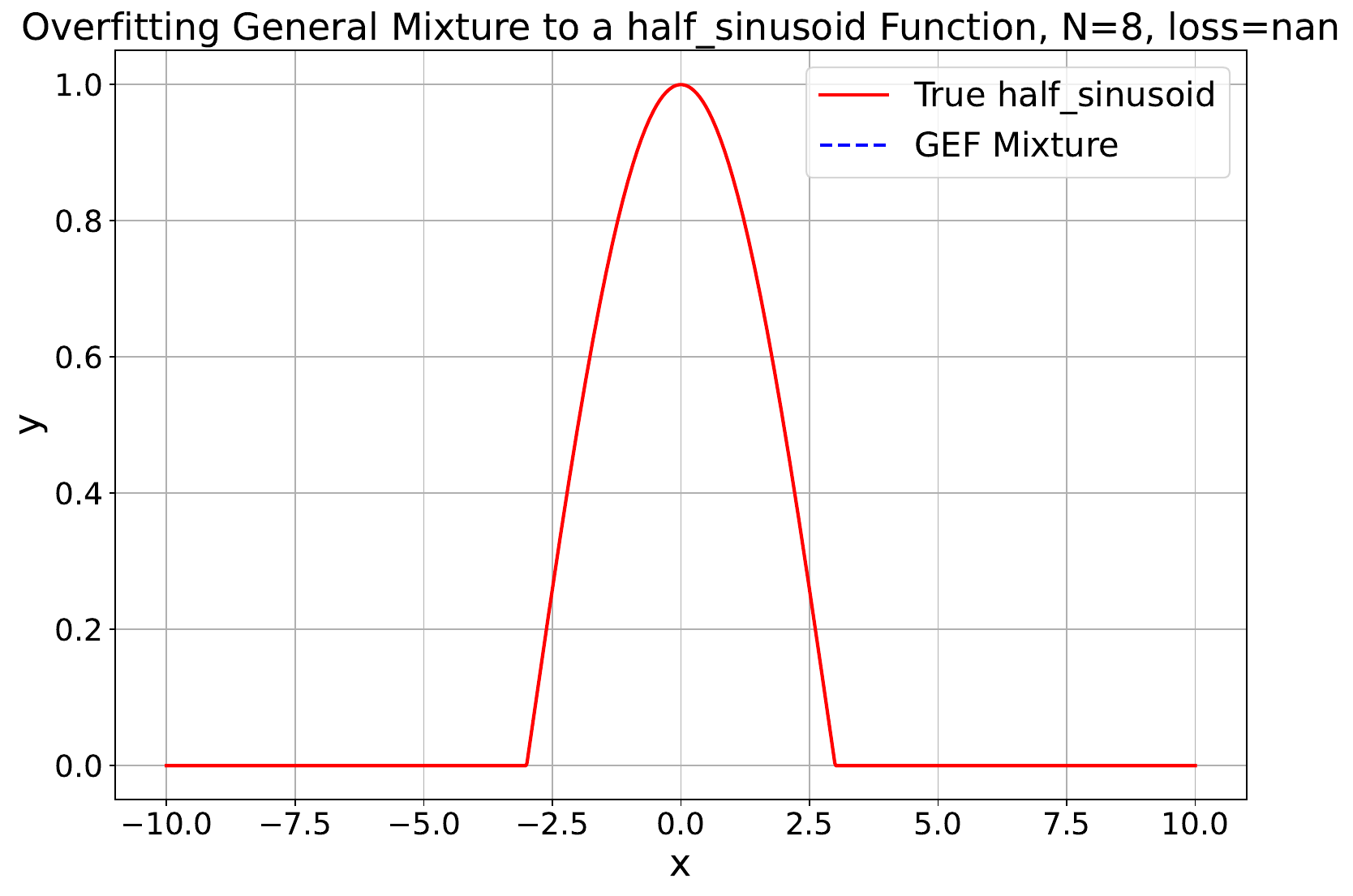}\\ 
    \includegraphics[width=0.24\linewidth]{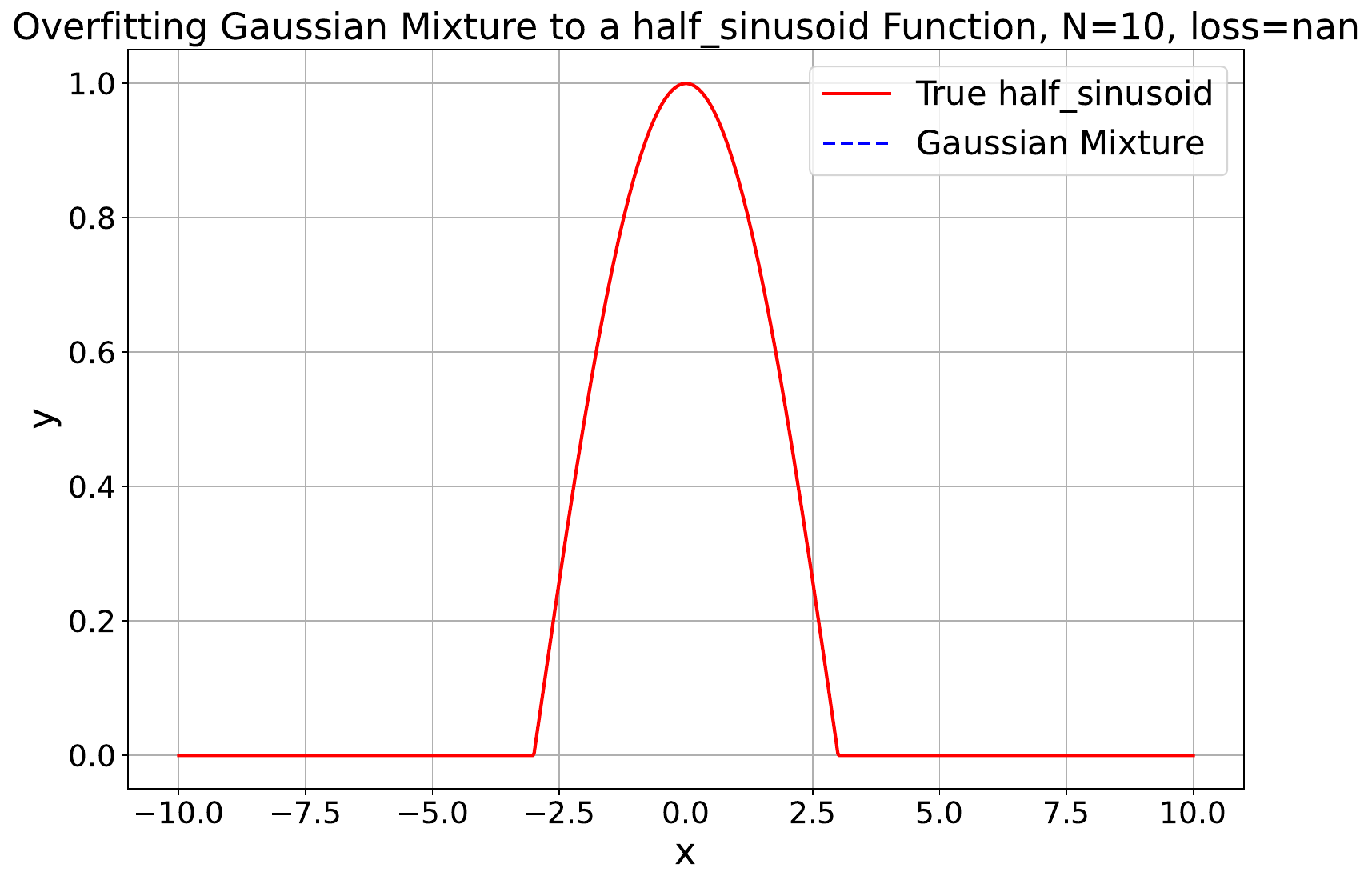} & 
    \includegraphics[width=0.24\linewidth]{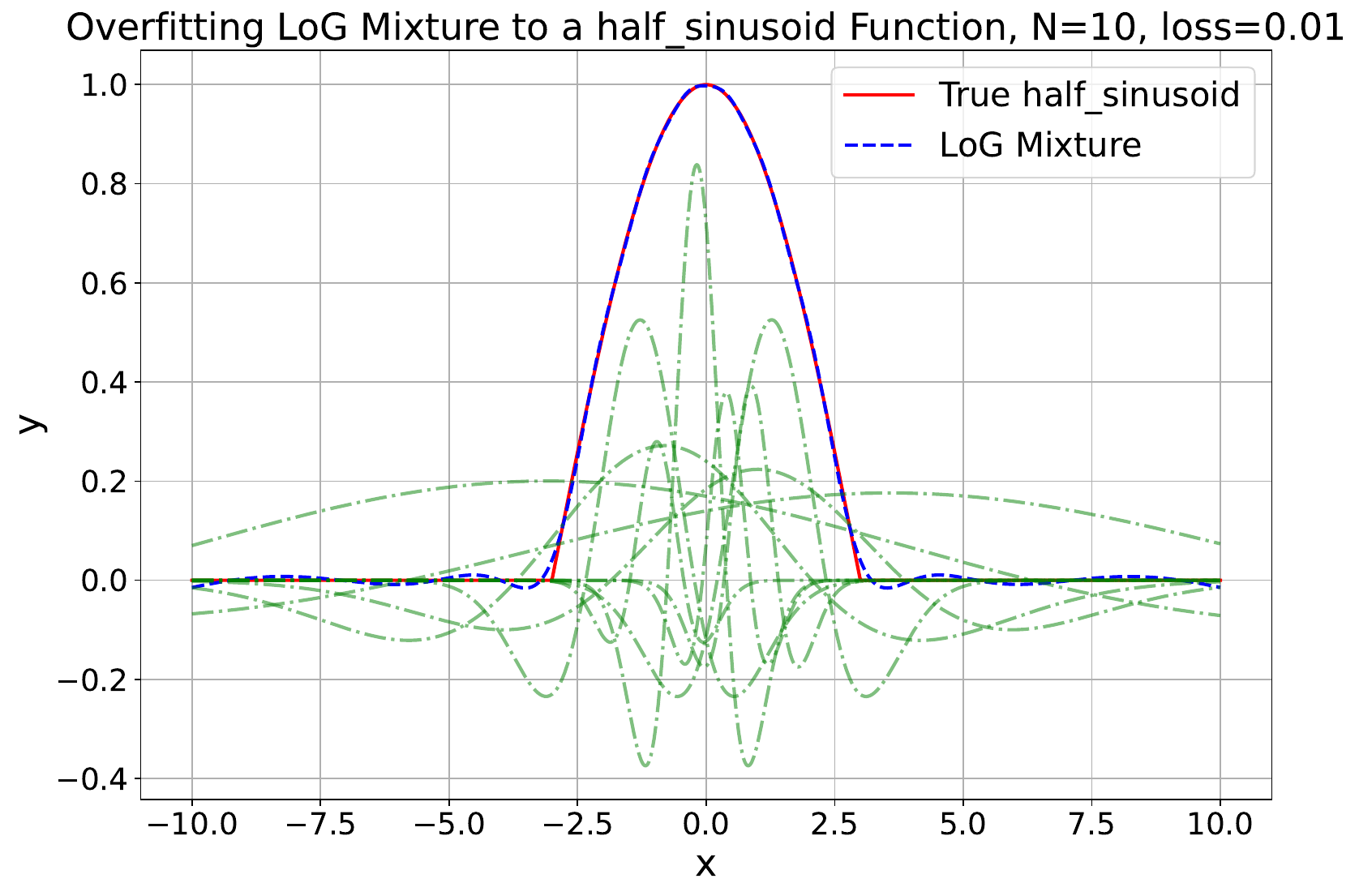} & 
    \includegraphics[width=0.24\linewidth]{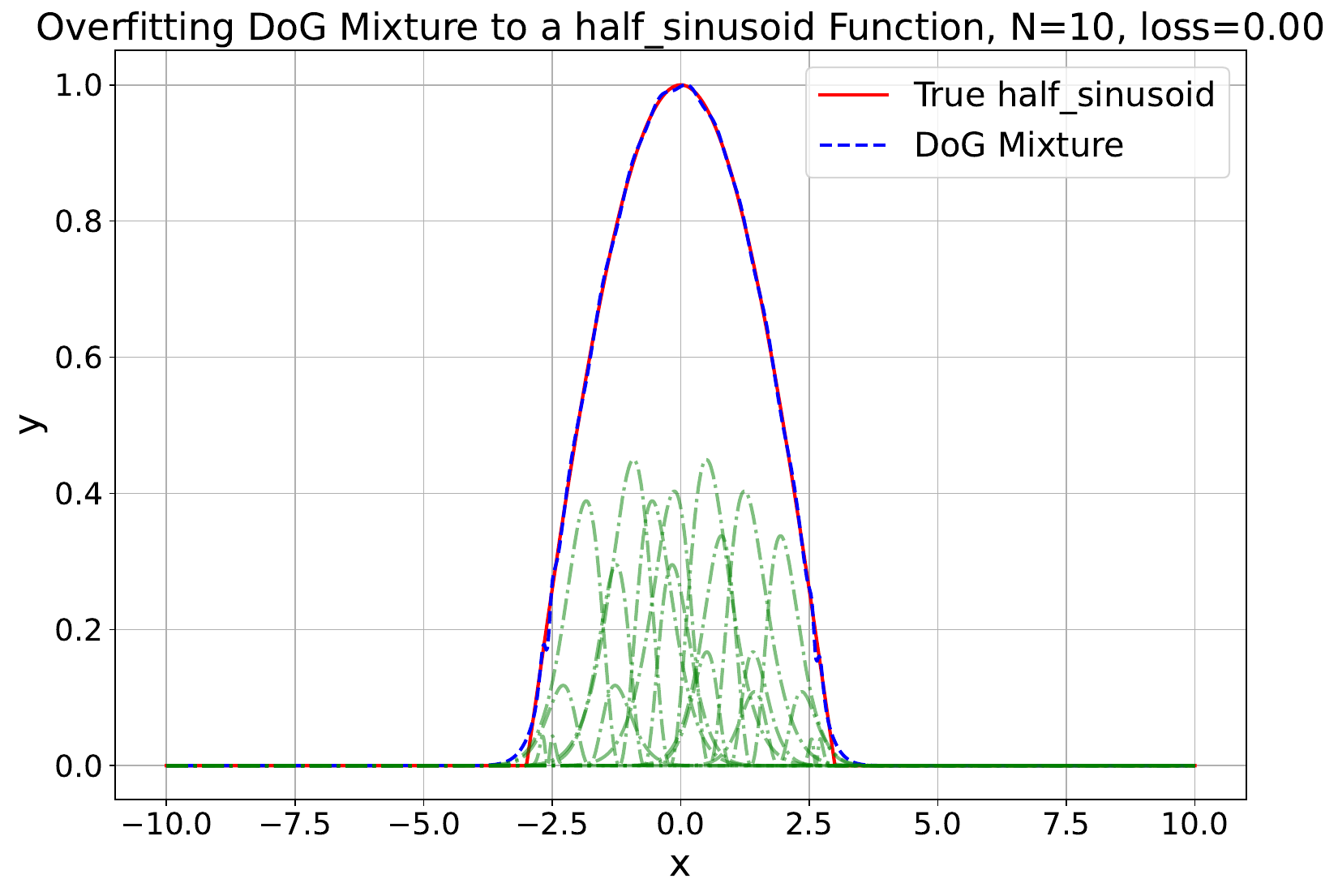} & 
    \includegraphics[width=0.24\linewidth]{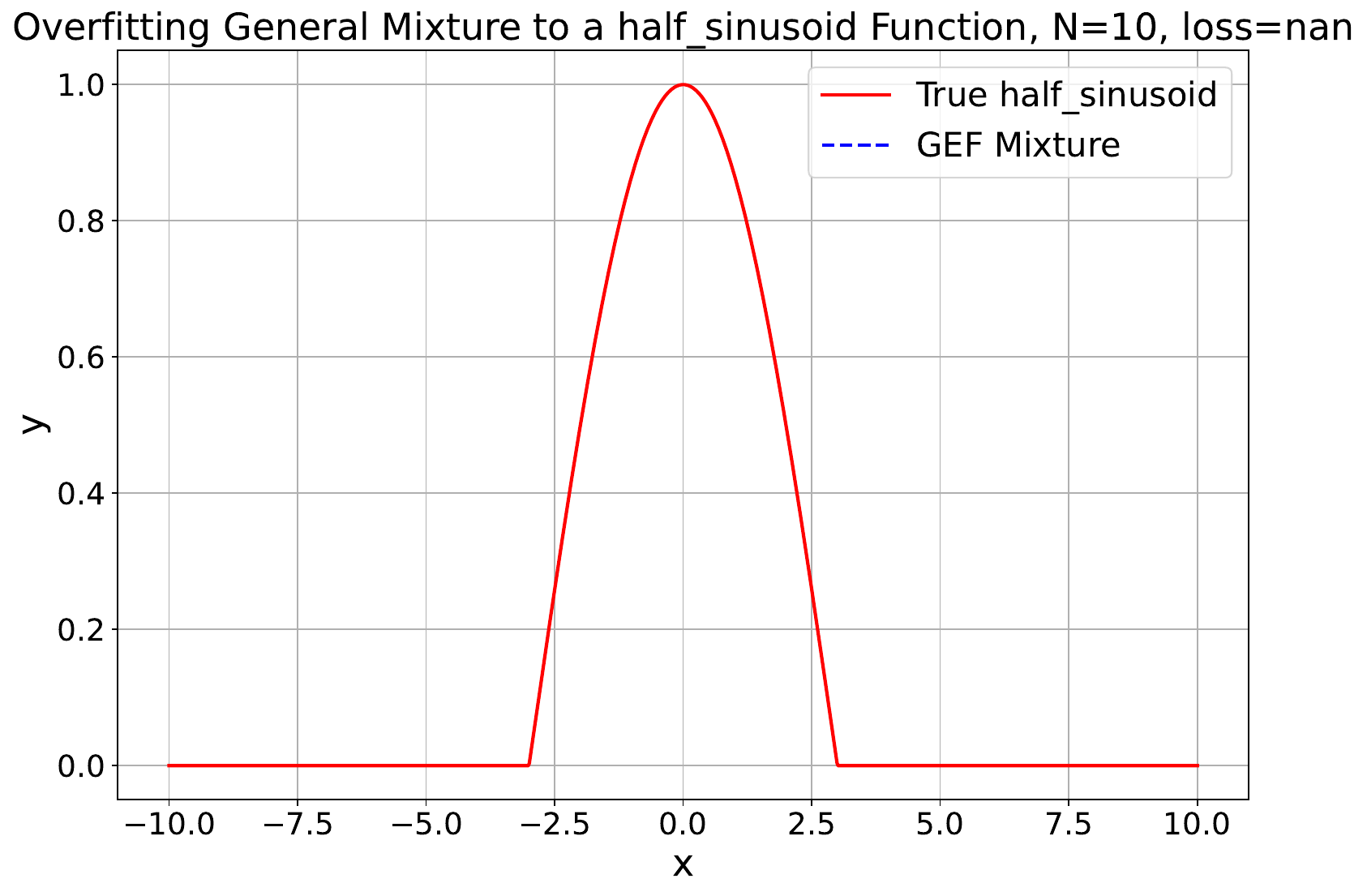}\\ 
    \includegraphics[width=0.24\linewidth]{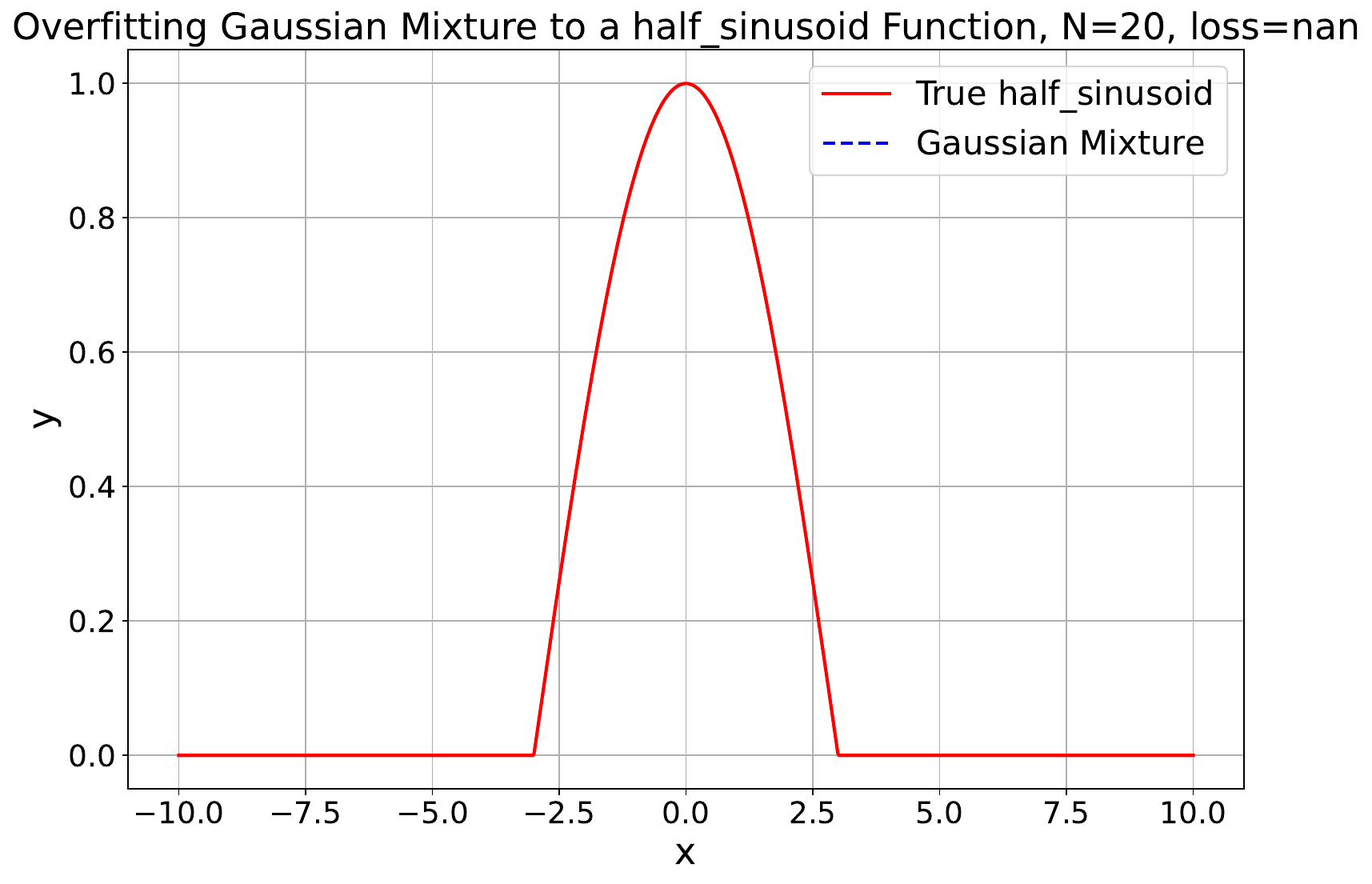} & 
    \includegraphics[width=0.24\linewidth]{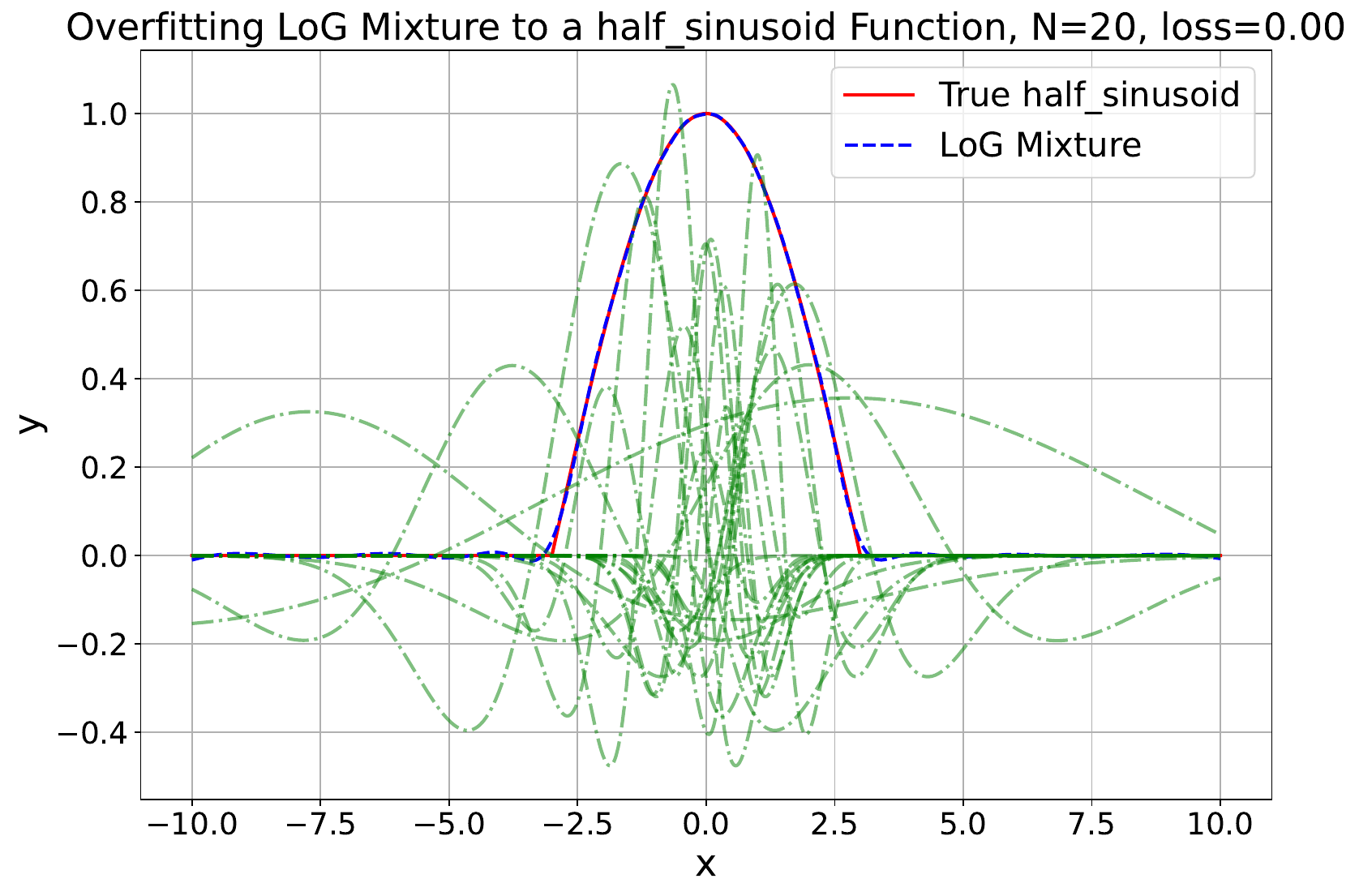} & 
    \includegraphics[width=0.24\linewidth]{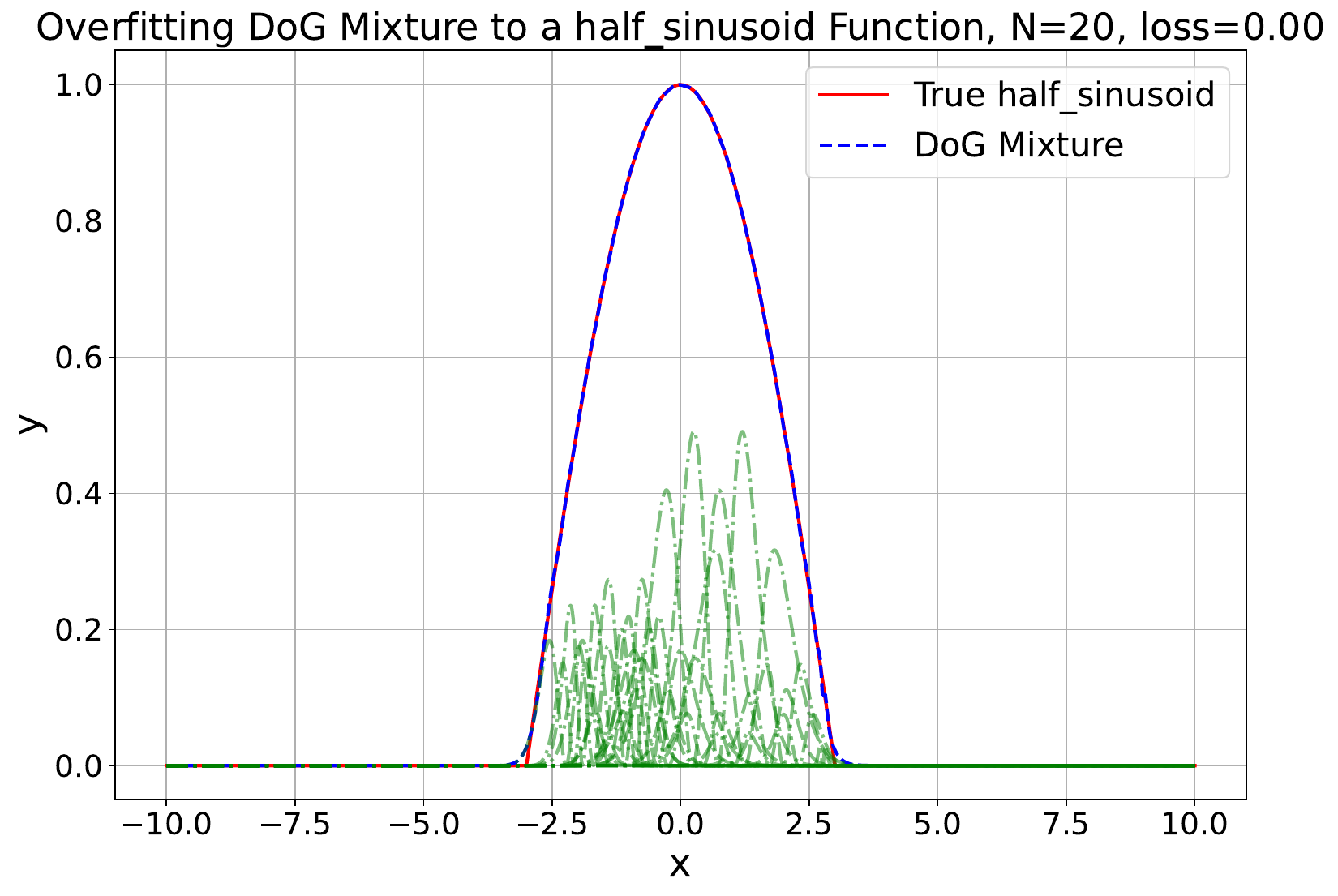} & 
    \includegraphics[width=0.24\linewidth]{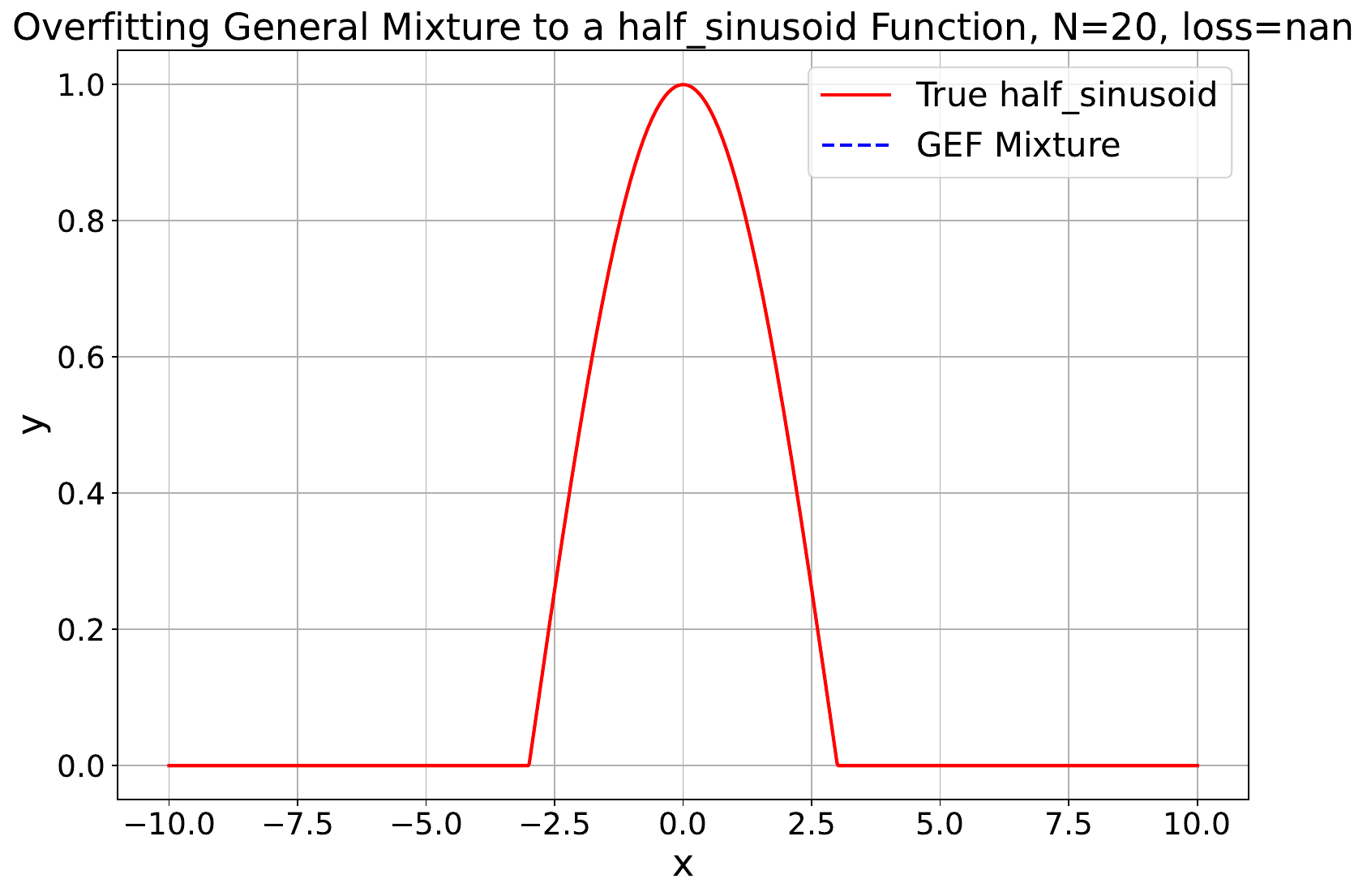}\\ 
    
    \end{tabular}
    }
    \caption{\textbf{Numerical Simulation Examples of Fitting Half sinusoids with Positive Weights Mixtures ( N= 2, 5, 8, and 10 )}. We show some fitting examples for half sinusoid signals with positive weights mixtures. The four mixtures used from left to right are Gaussians, LoG, DoG, and General mixtures. From top to bottom: N = 2, 8, and 10 components. The optimized individual components are shown in green. Some examples fail to optimize due to numerical instability in both Gaussians and GEF mixtures. Note that GEF is very efficient in fitting the half sinusoid with few components while LoG and DoG are more stable for a larger number of components. }
    \label{supfig:fitting_half_sinusoid_p}
    \end{figure*}
    

%% file: figures/fitting/fitting_sinusoid_n.tex
\begin{figure*}[h]
    \centering
    \resizebox{1.0\linewidth}{!}{
    \begin{tabular}{cccc}
    \tabcolsep=0.01cm
    Gaussian Mixture& LoG Mixture & DoG Mixture & GEF Mixture \\ 
    \includegraphics[width=0.24\linewidth]{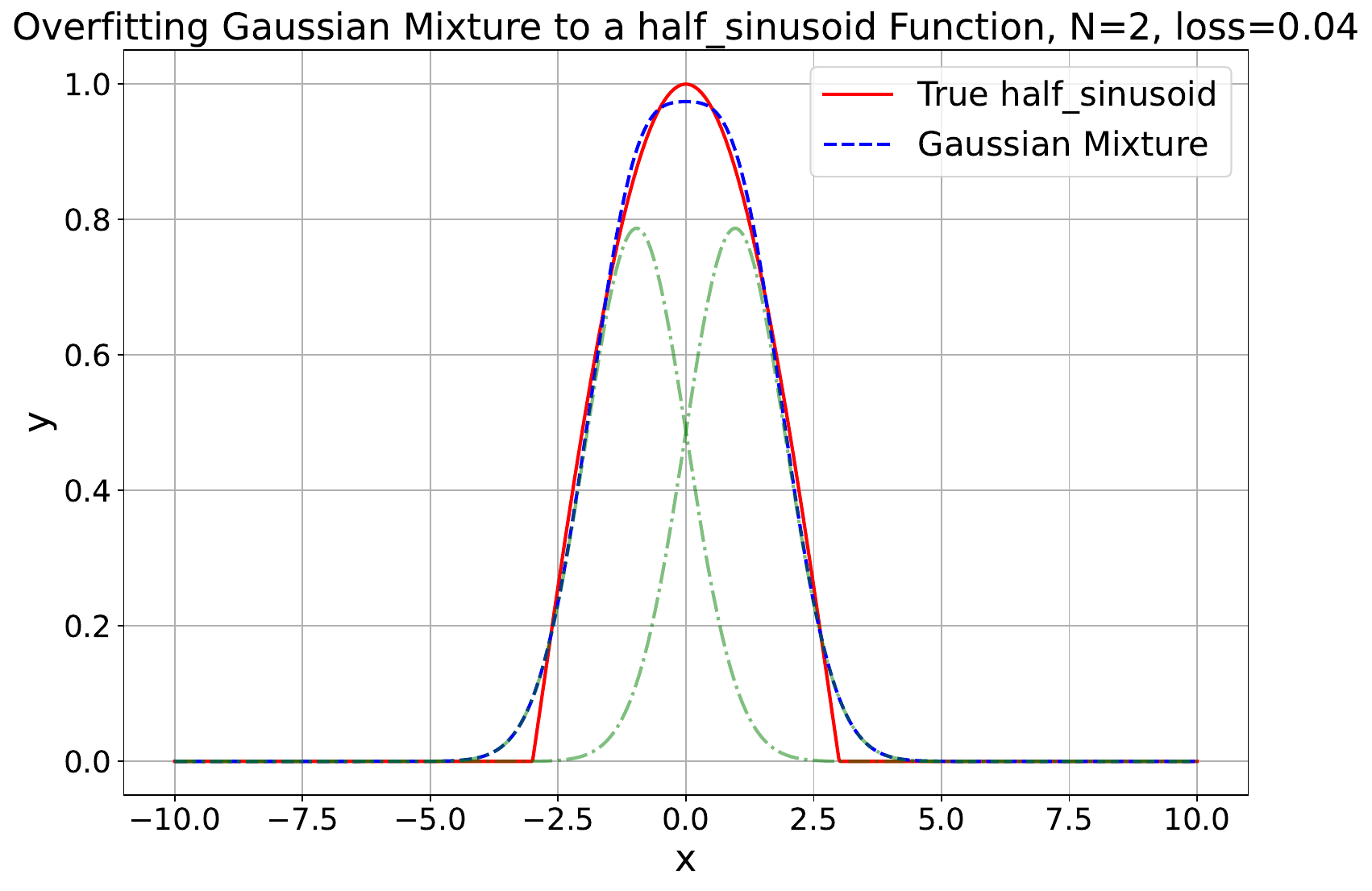} & 
    \includegraphics[width=0.24\linewidth]{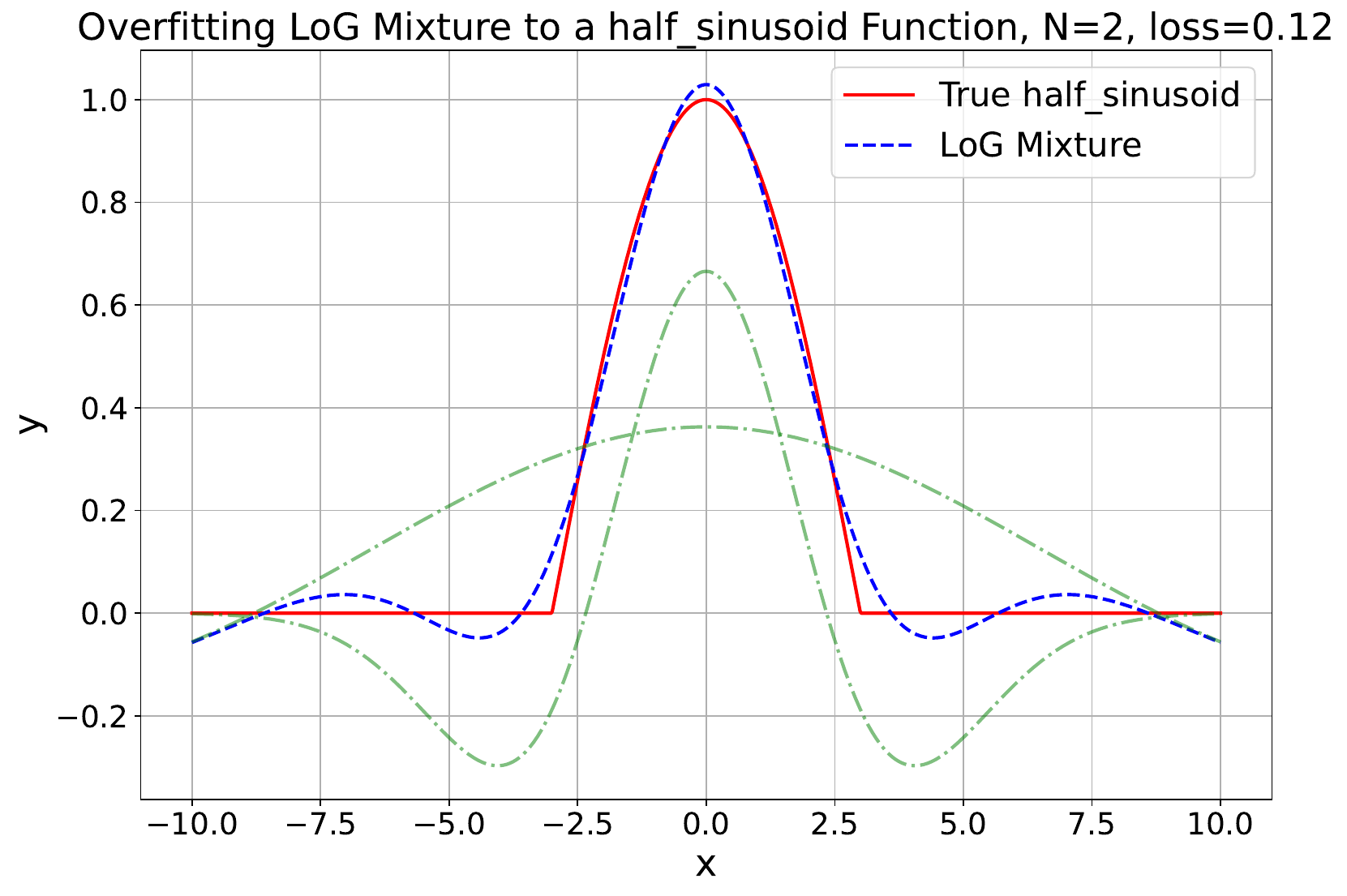} & 
    \includegraphics[width=0.24\linewidth]{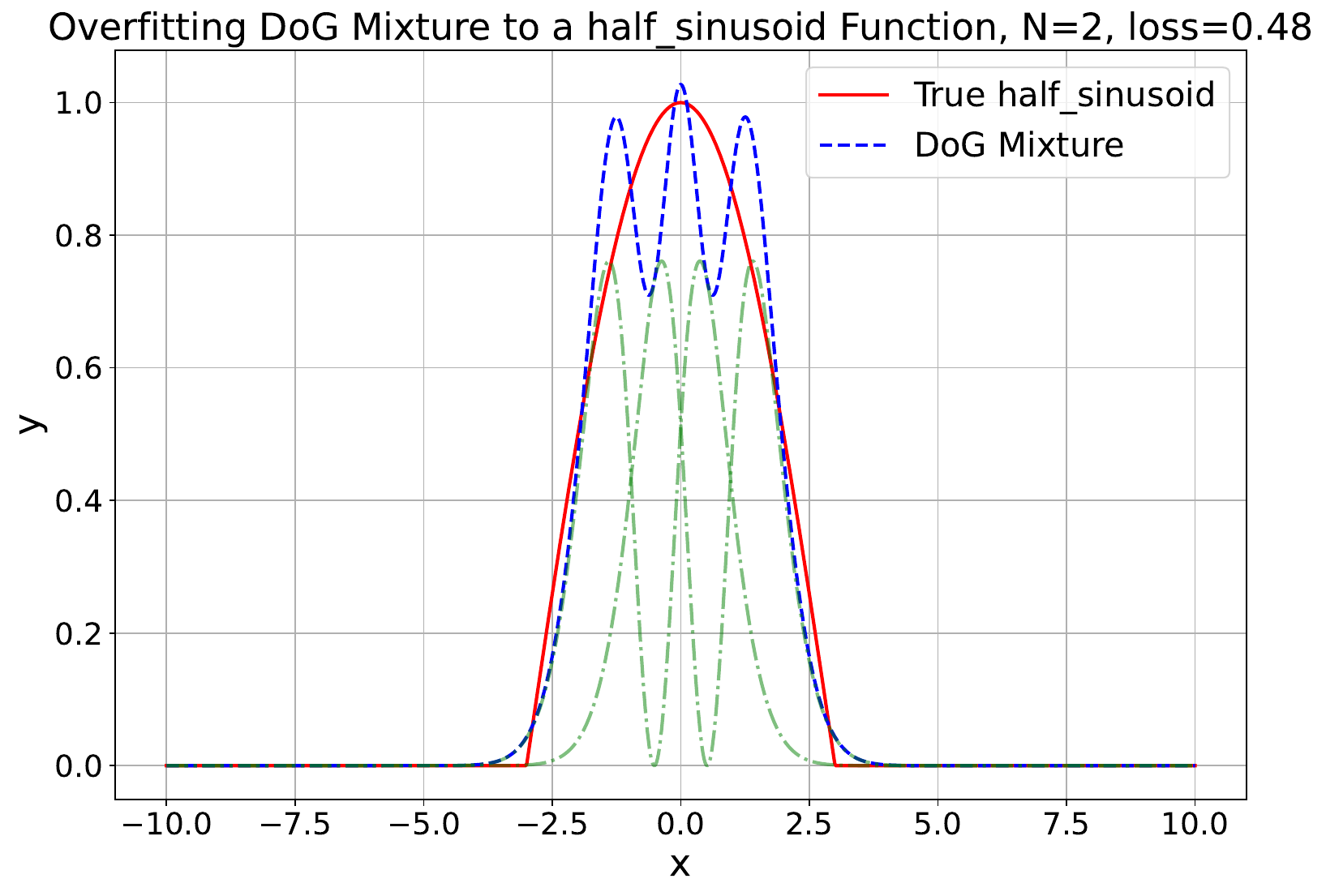} & 
    \includegraphics[width=0.24\linewidth]{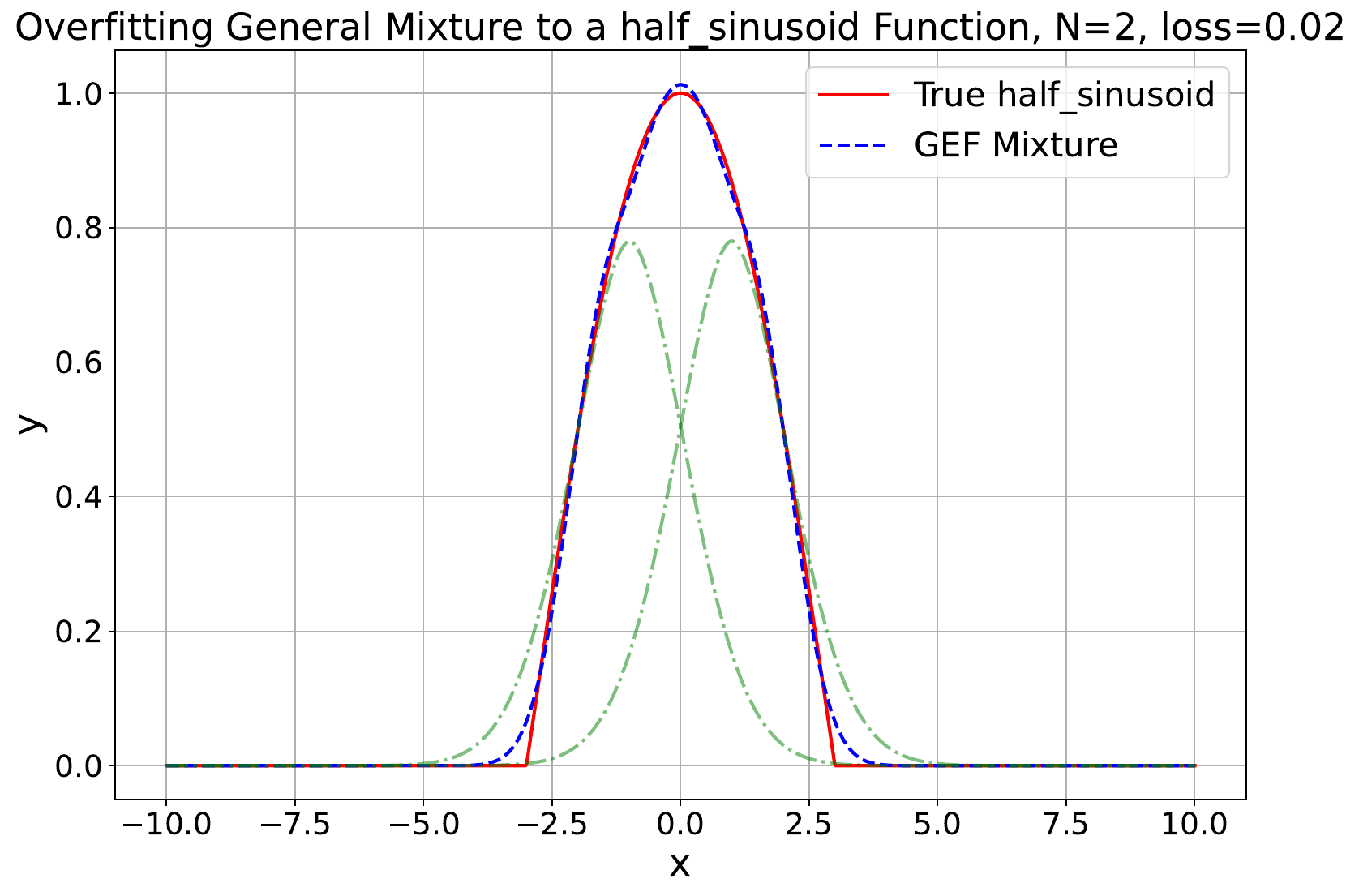}\\ 
    \includegraphics[width=0.24\linewidth]{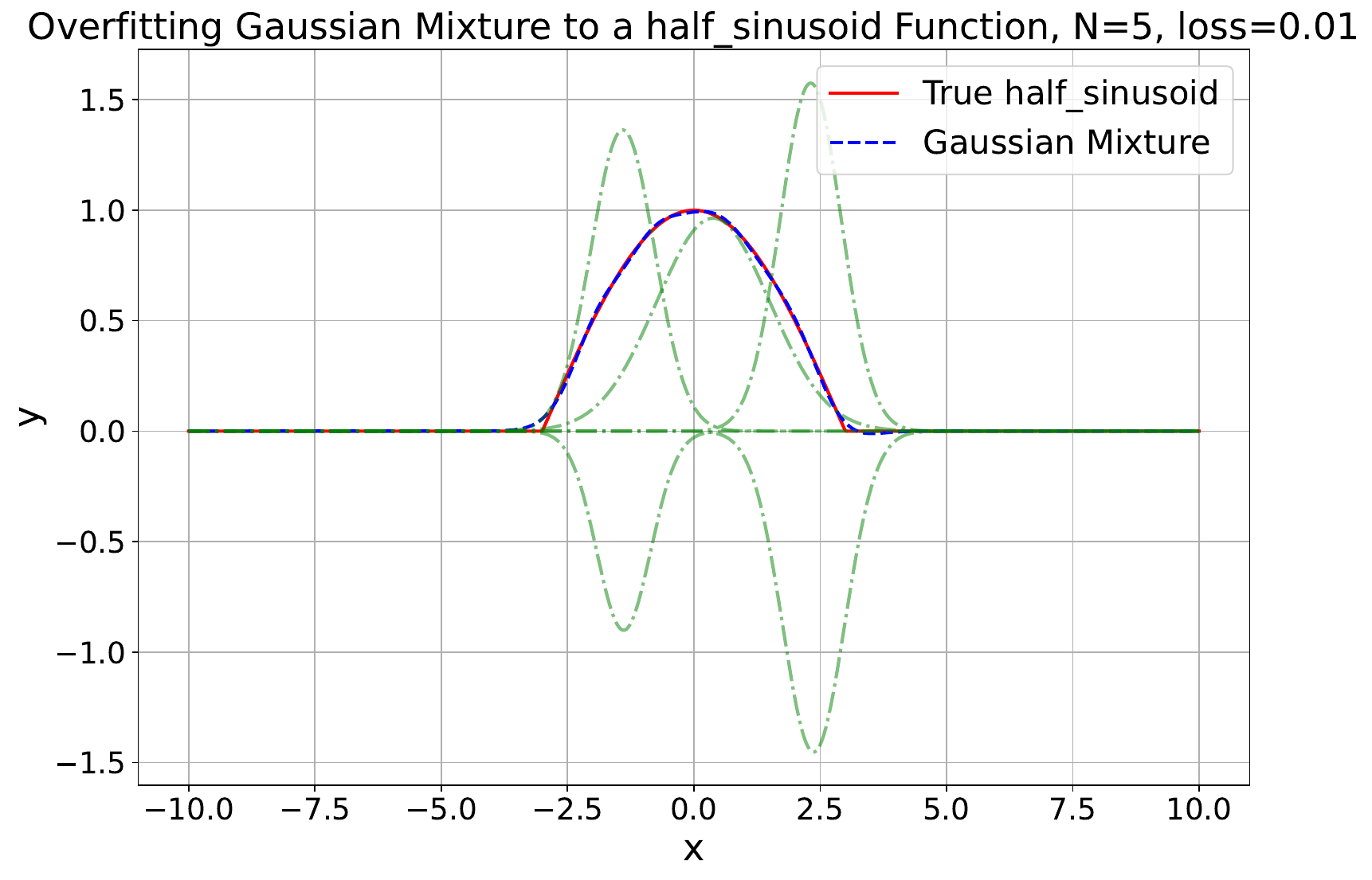} & 
    \includegraphics[width=0.24\linewidth]{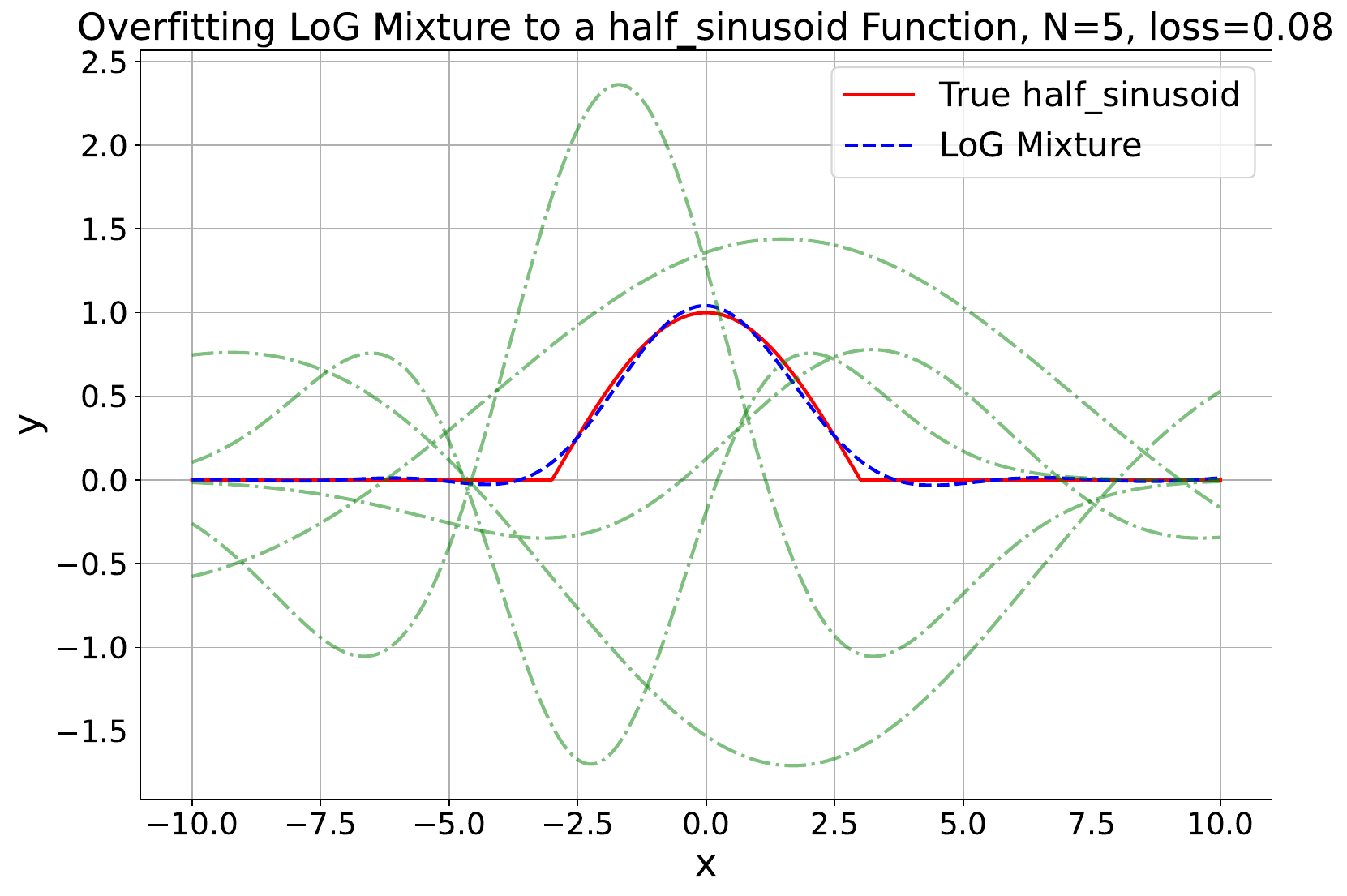} & 
    \includegraphics[width=0.24\linewidth]{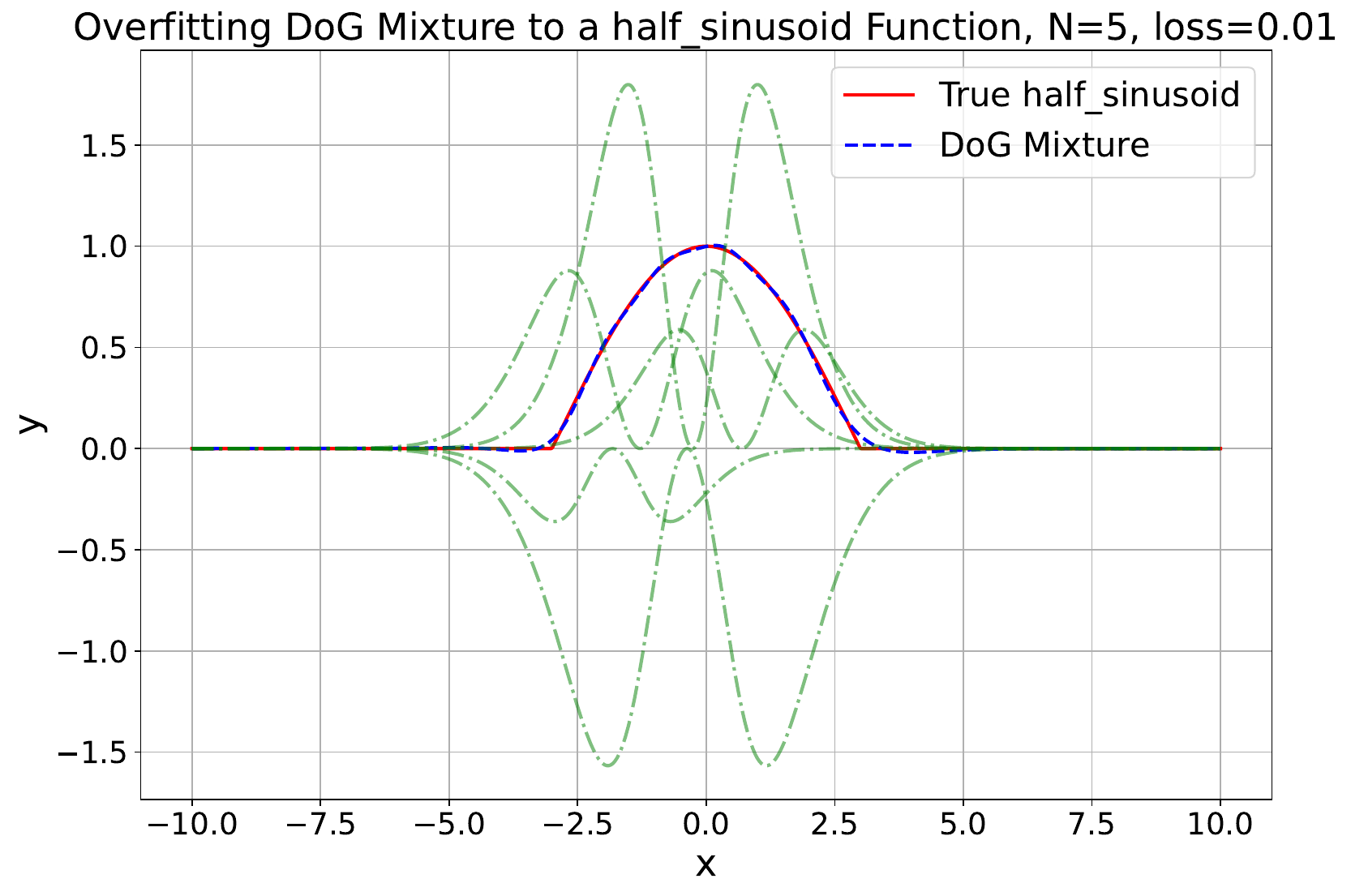} & 
    \includegraphics[width=0.24\linewidth]{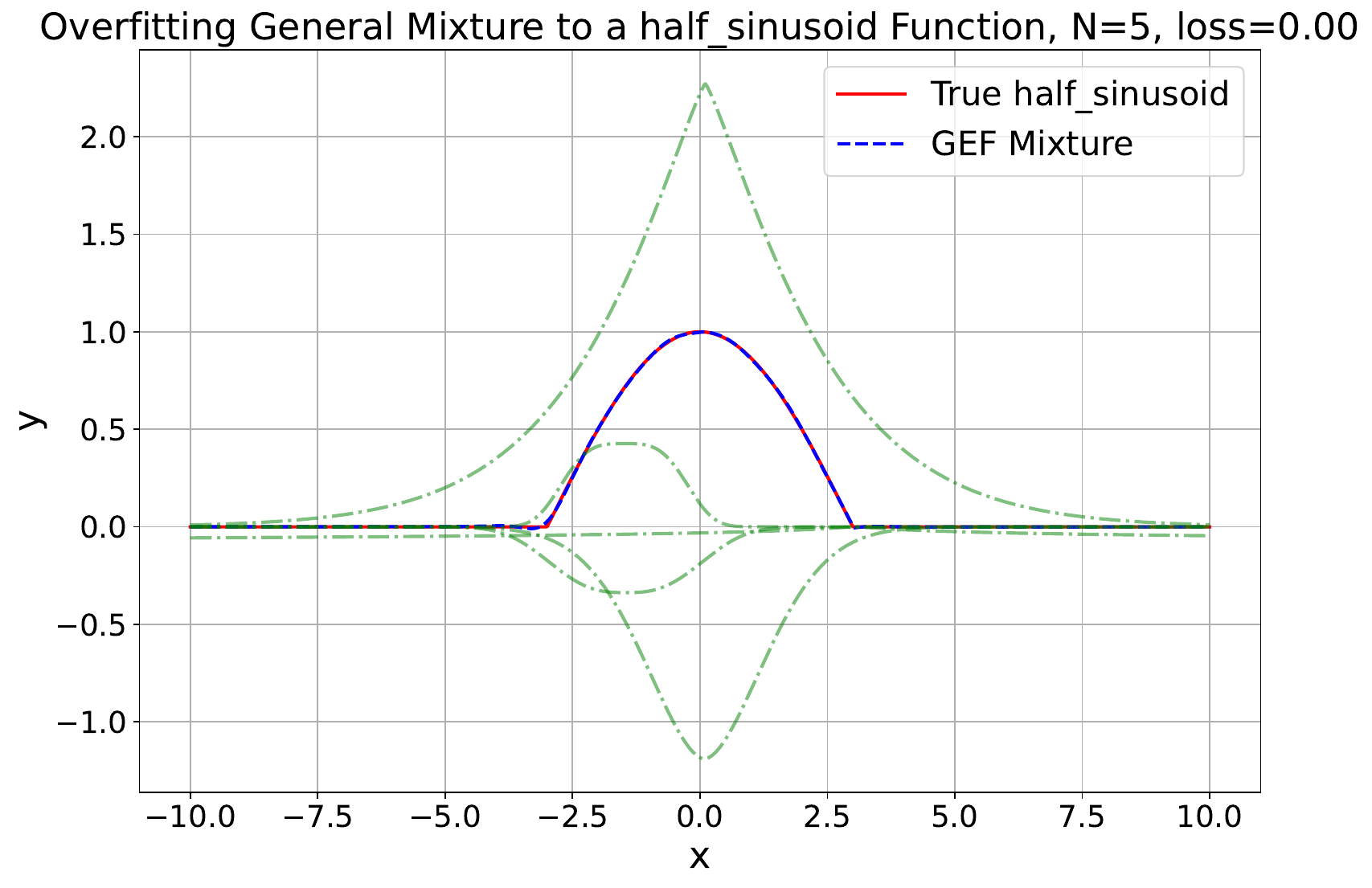}\\ 
    \includegraphics[width=0.24\linewidth]{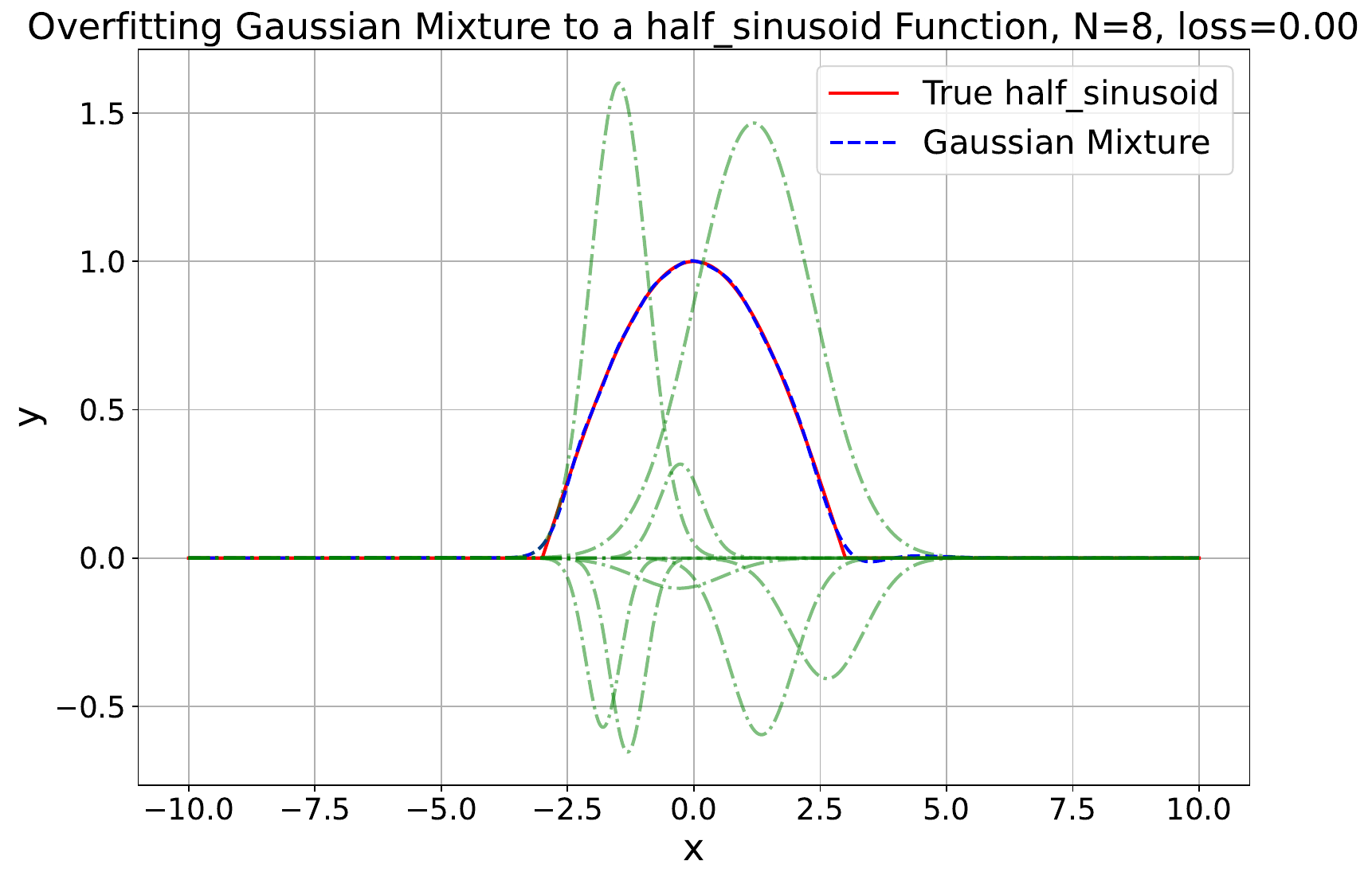} & 
    \includegraphics[width=0.24\linewidth]{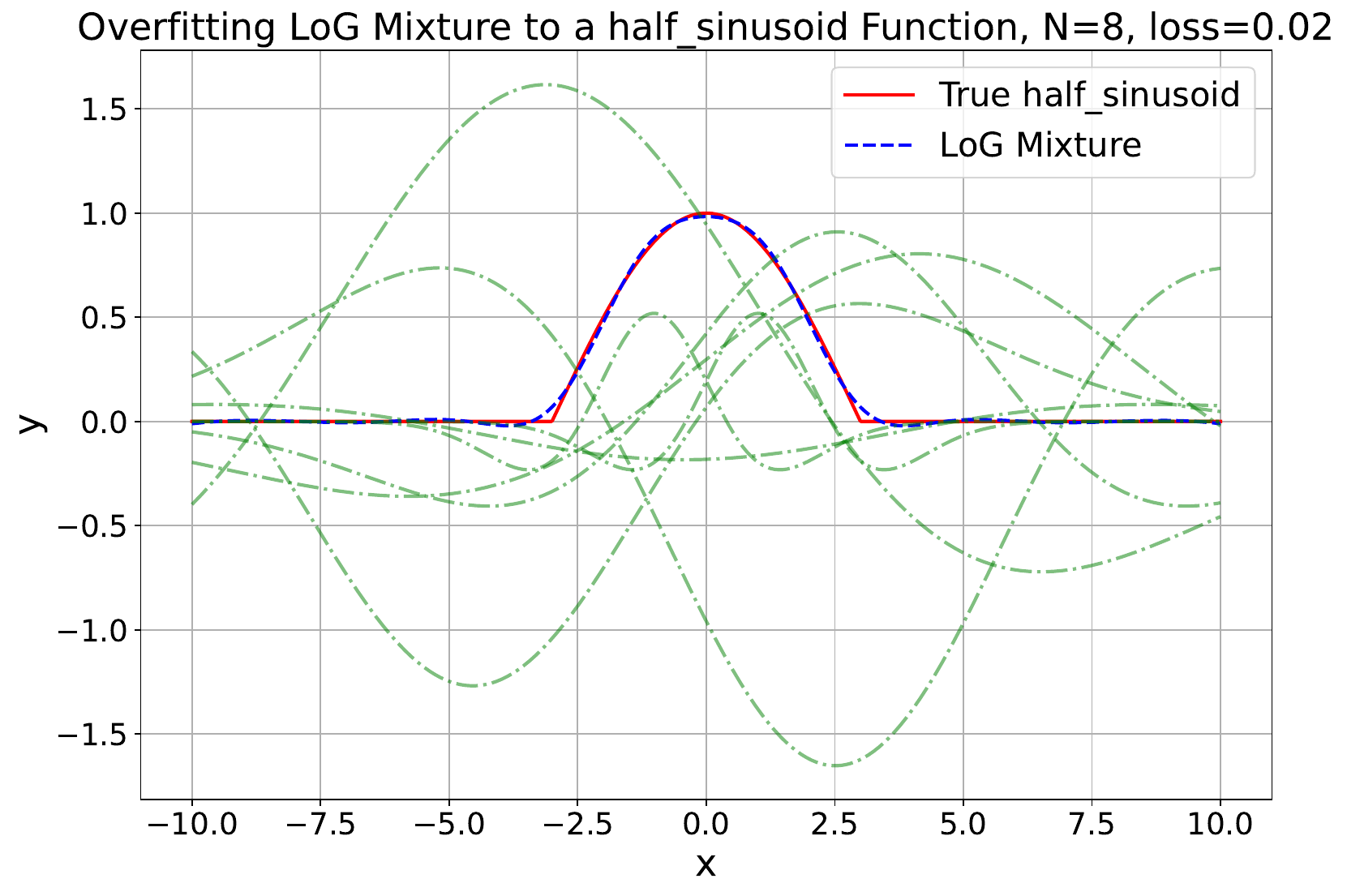} & 
    \includegraphics[width=0.24\linewidth]{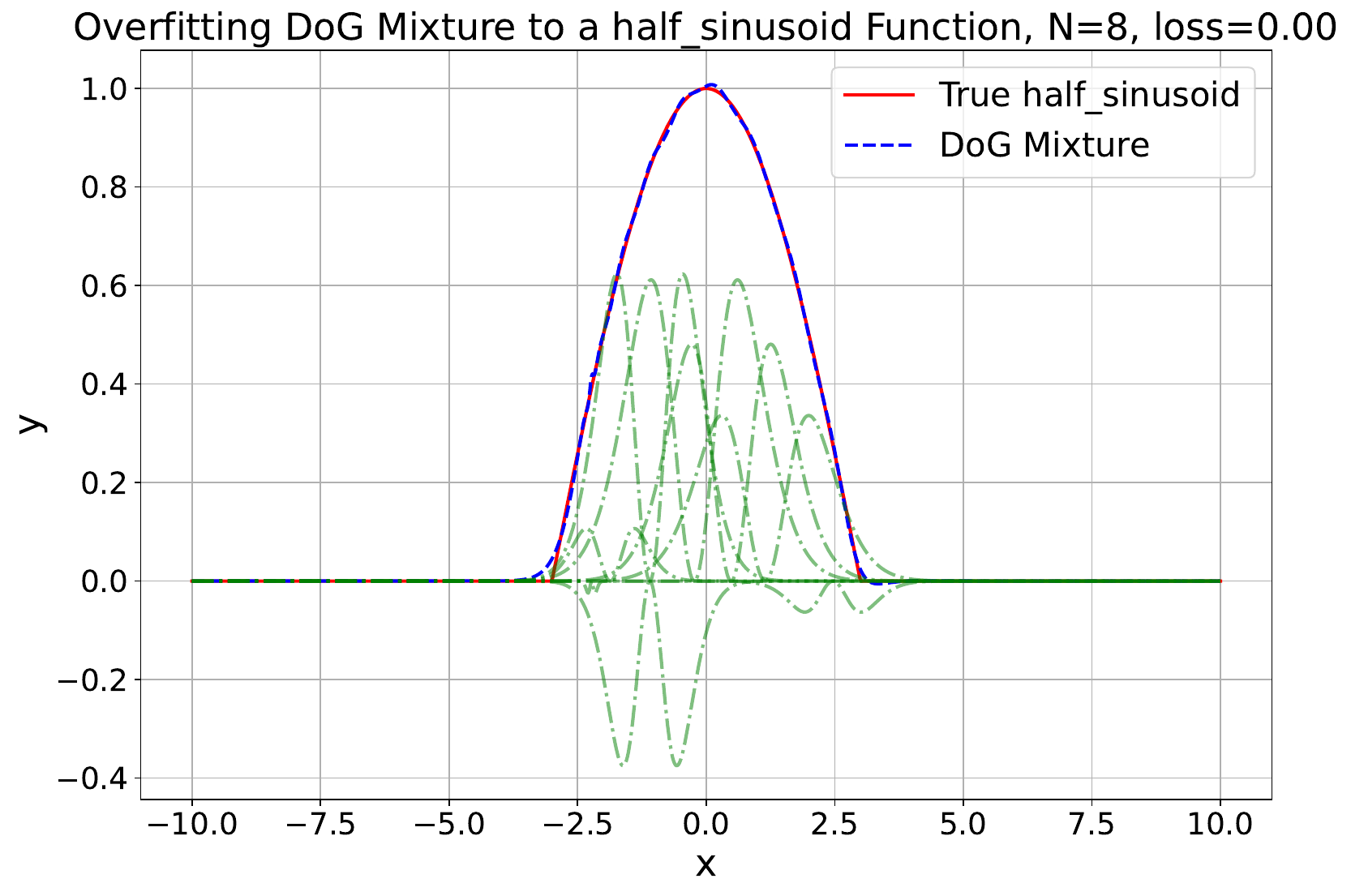} & 
    \includegraphics[width=0.24\linewidth]{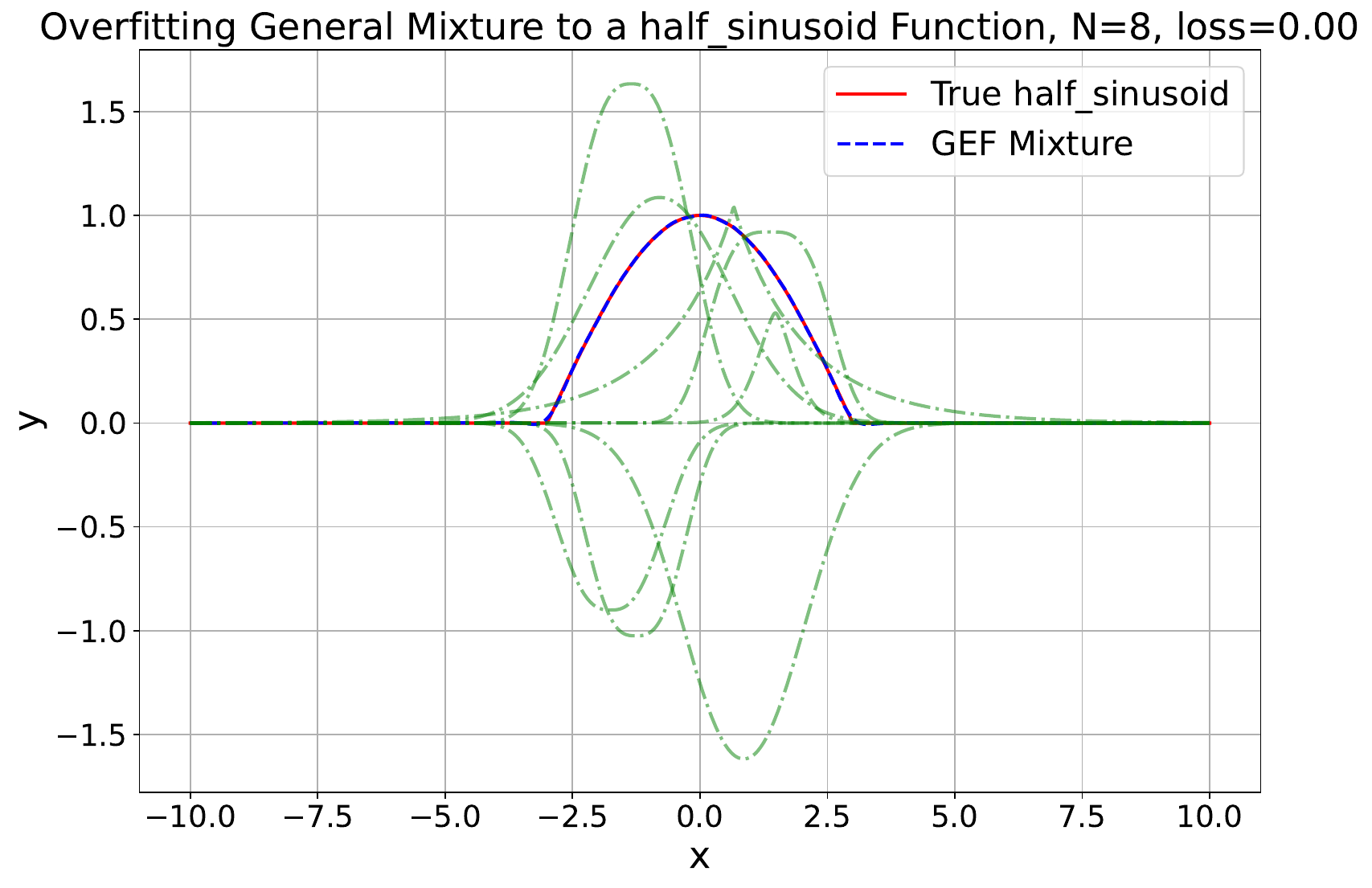}\\ 
    \includegraphics[width=0.24\linewidth]{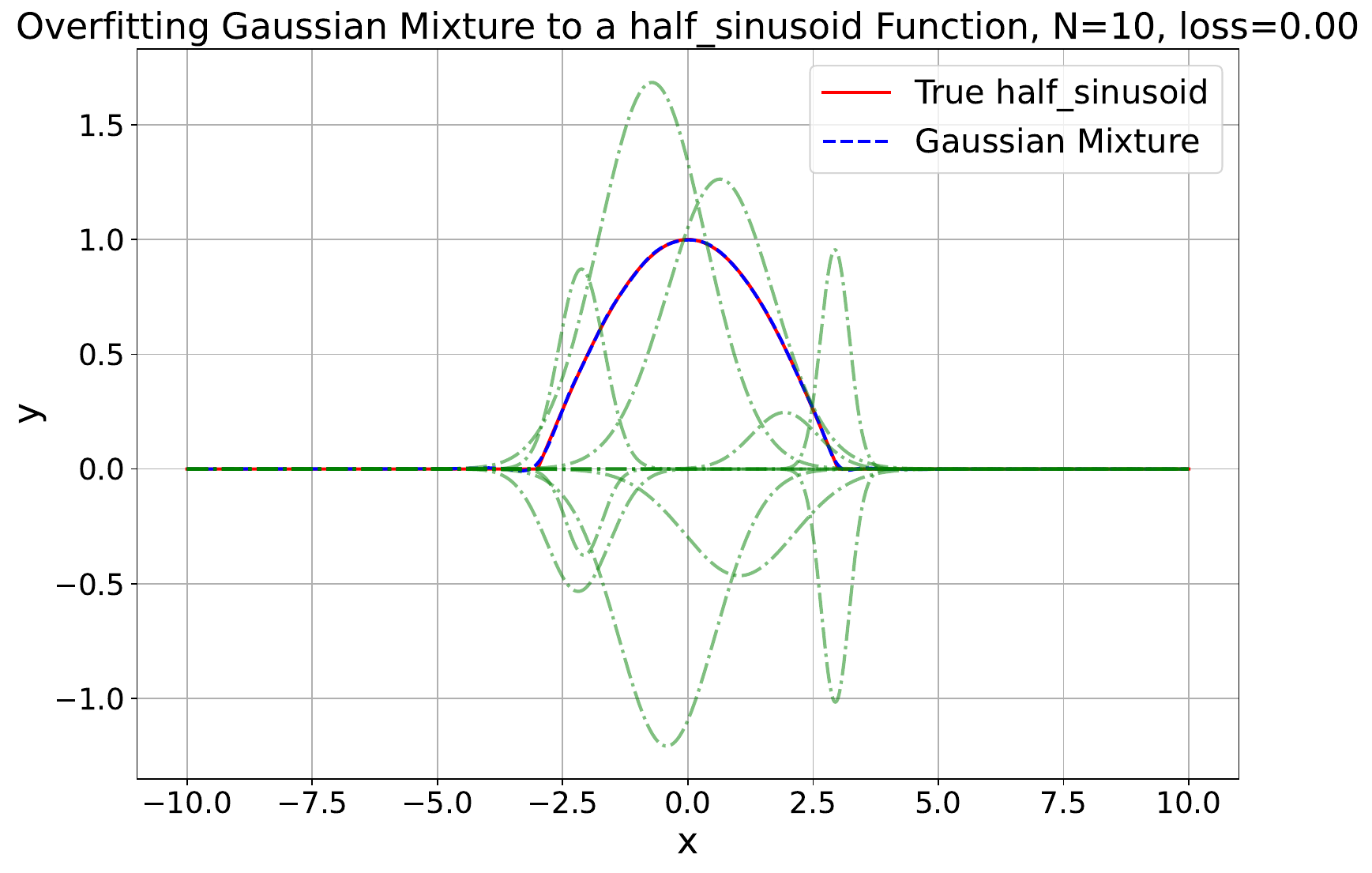} & 
    \includegraphics[width=0.24\linewidth]{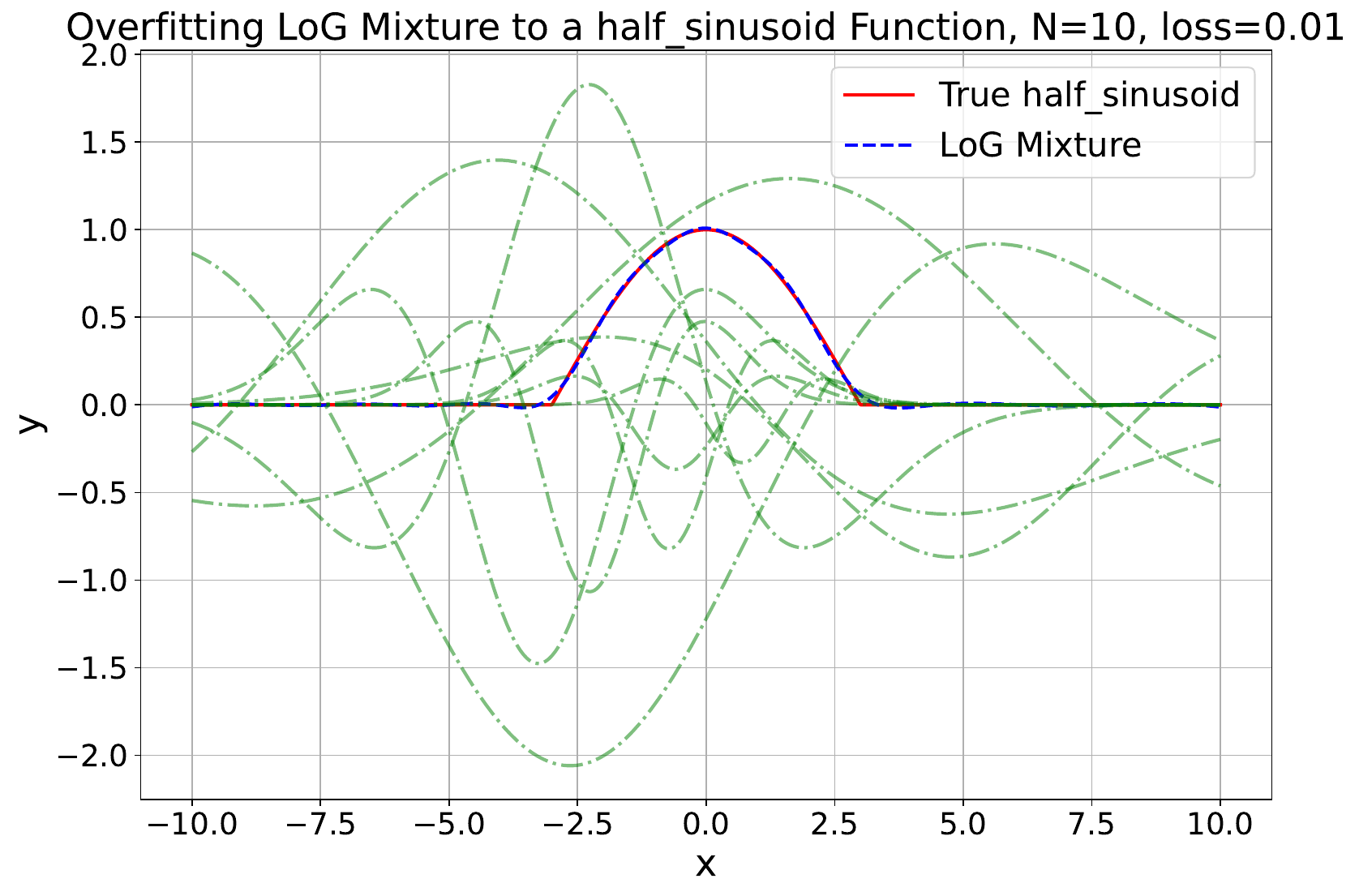} & 
    \includegraphics[width=0.24\linewidth]{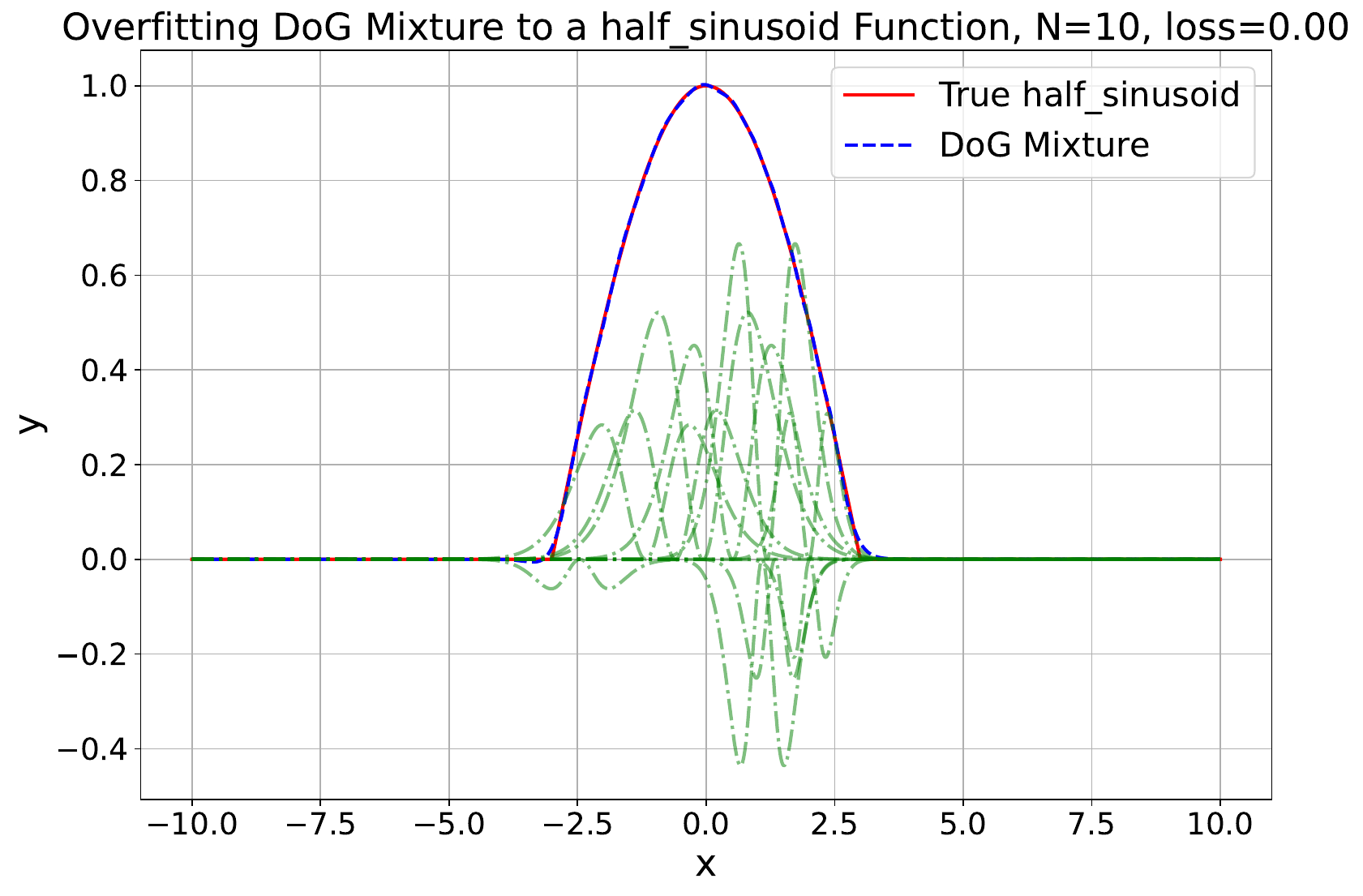} & 
    \includegraphics[width=0.24\linewidth]{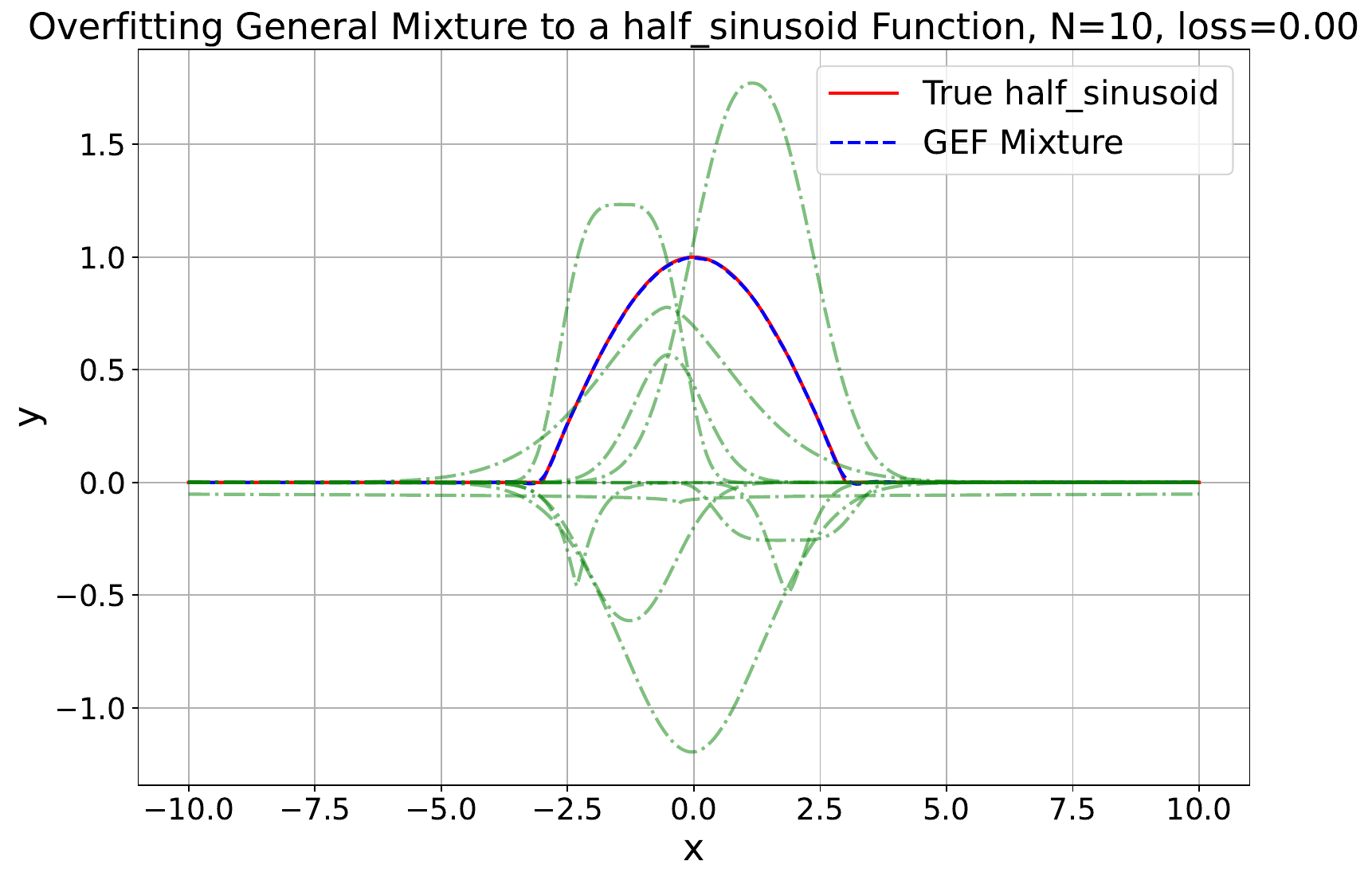}\\ 
    \includegraphics[width=0.24\linewidth]{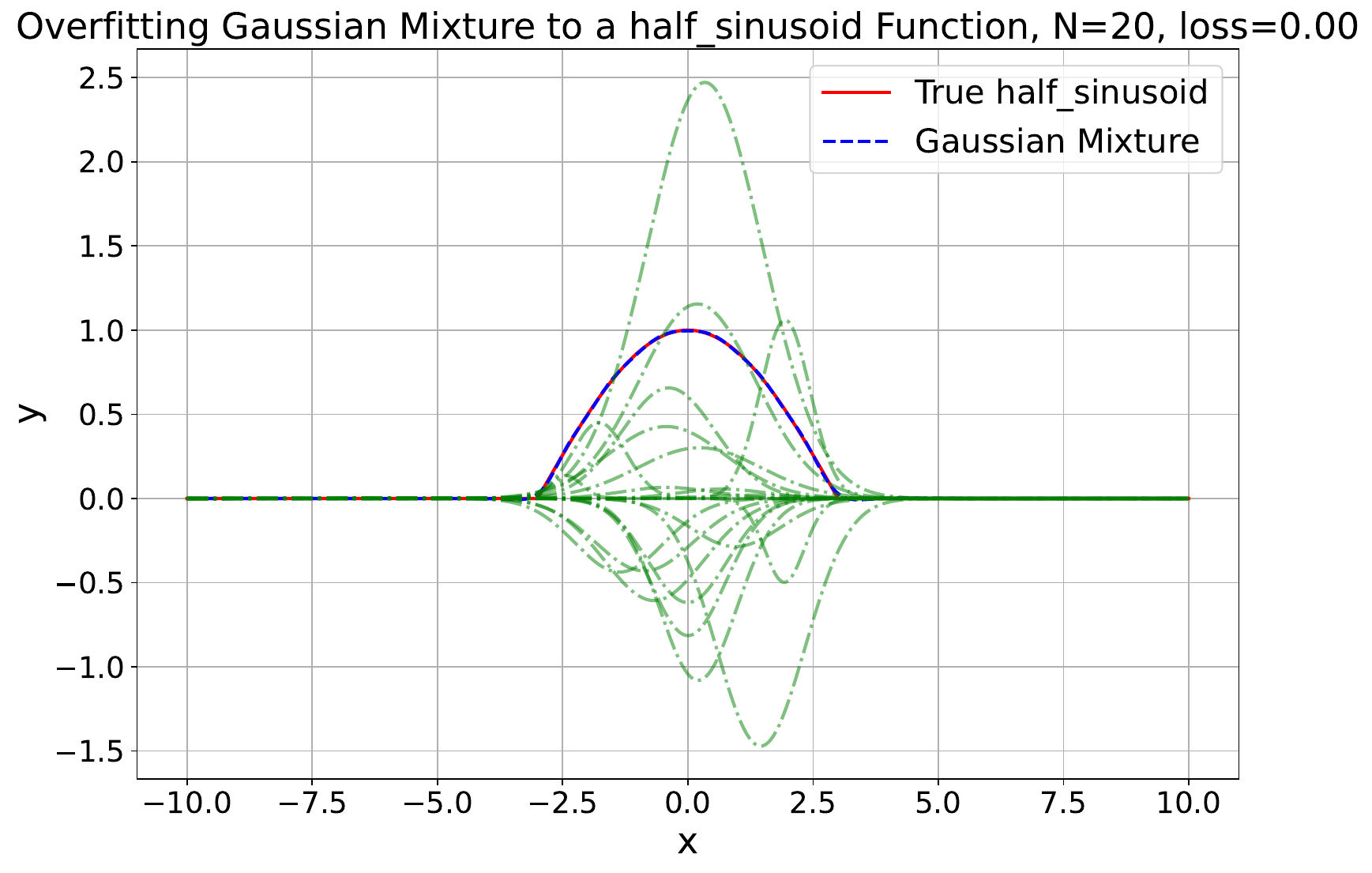} & 
    \includegraphics[width=0.24\linewidth]{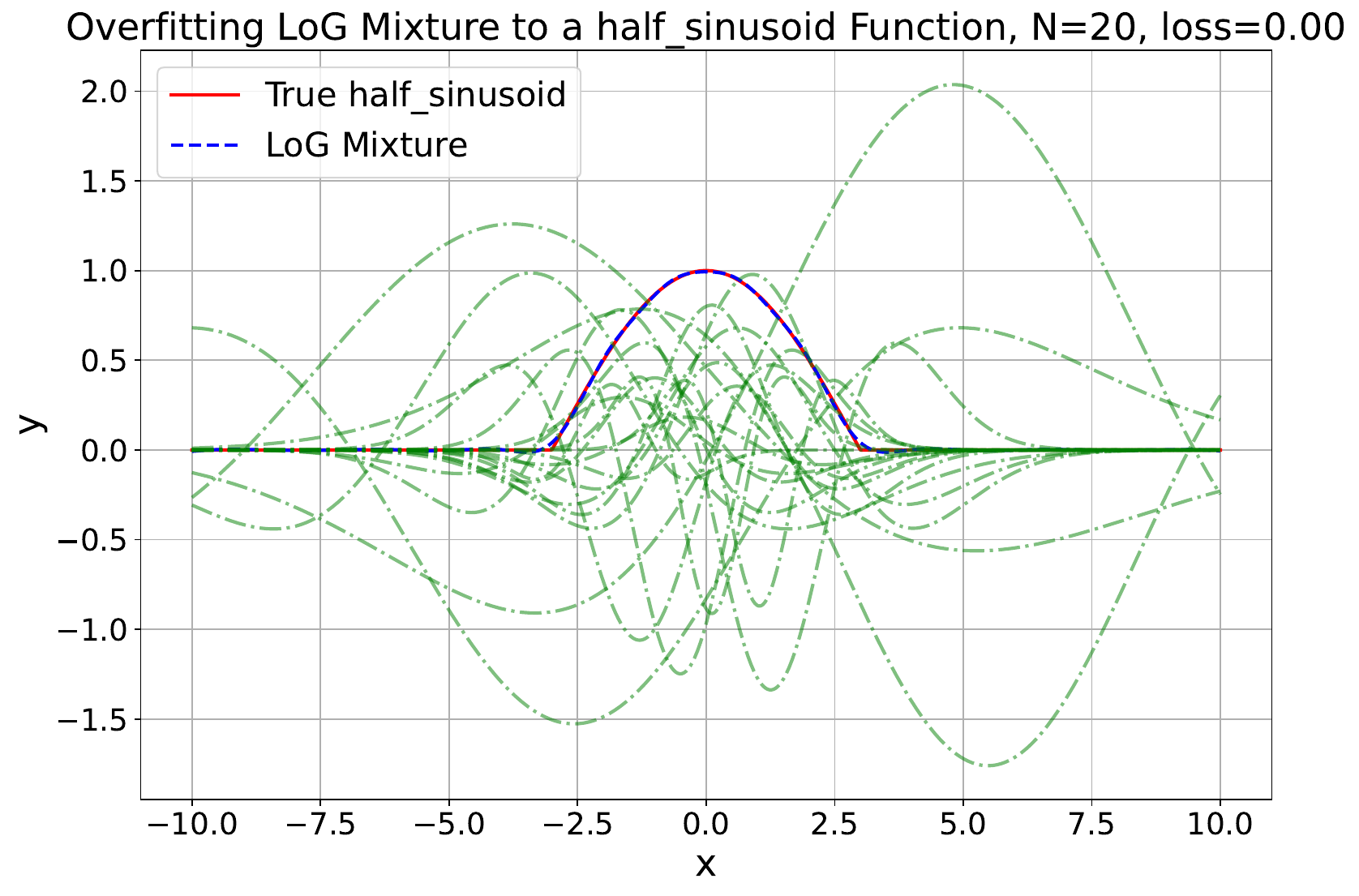} & 
    \includegraphics[width=0.24\linewidth]{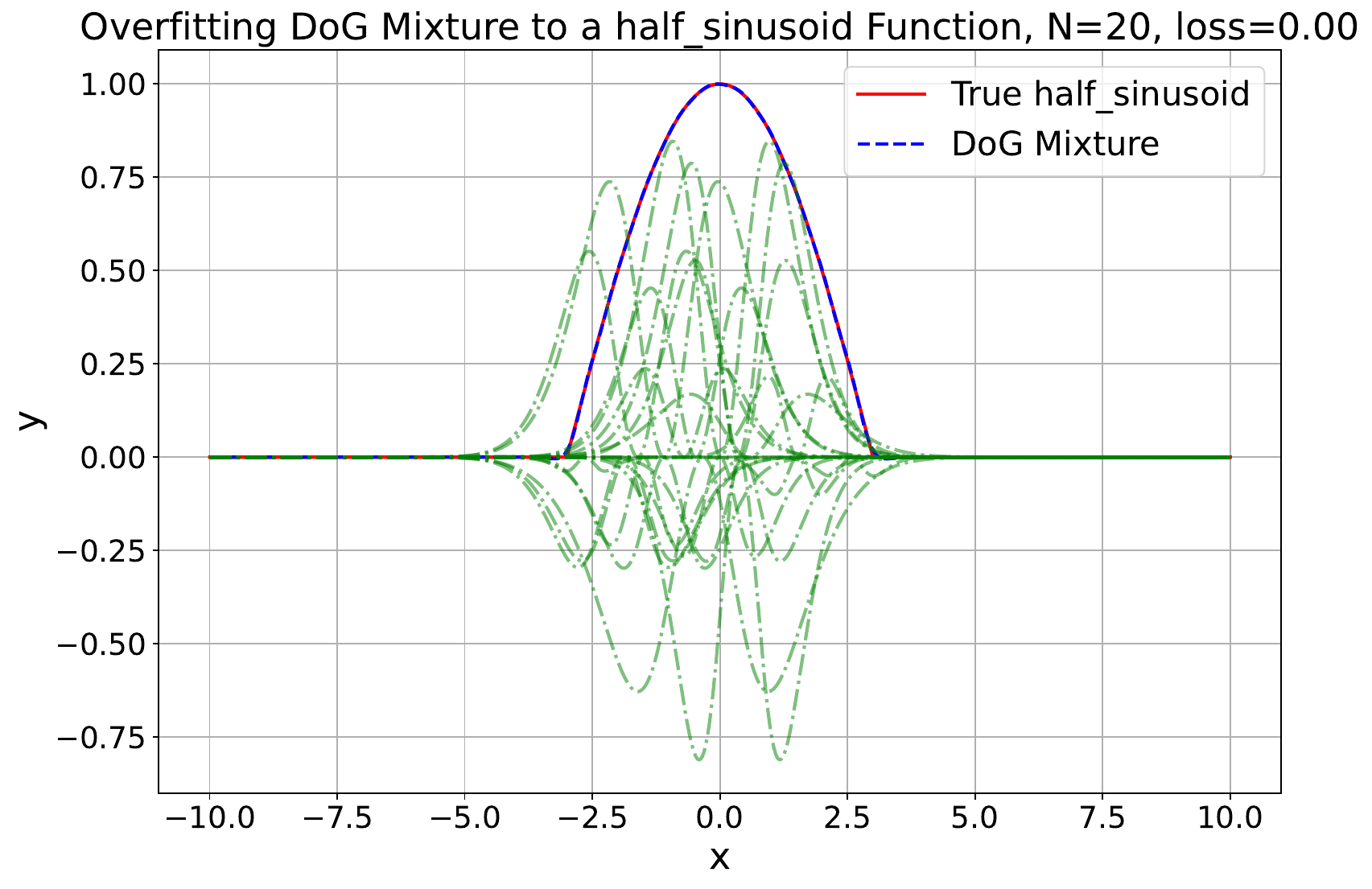} & 
    \includegraphics[width=0.24\linewidth]{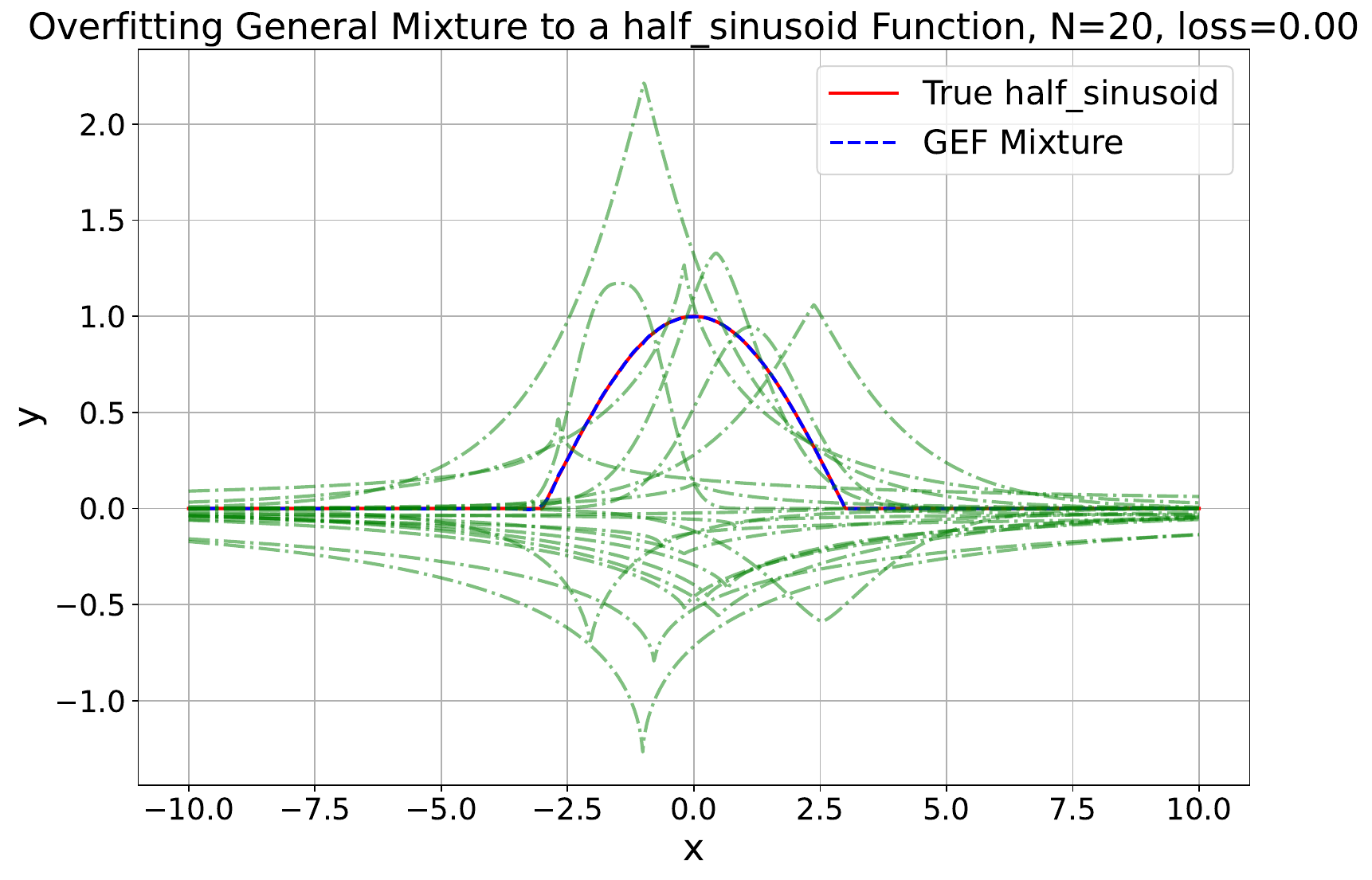}\\ 
    
    \end{tabular}
    }
    \caption{\textbf{Numerical Simulation Examples of Fitting Half sinusoids with Real Weights Mixtures ( N= 2, 5, 8, and 10 )}. We show some fitting examples for half sinusoid signals with Real weights mixtures (can be negative). The four mixtures used from left to right are Gaussians, LoG, DoG, and General mixtures. From top to bottom: N = 2, 8, and 10 components. The optimized individual components are shown in green. Some examples fail to optimize due to numerical instability in both Gaussians and GEF mixtures. Note that GEF is very efficient in fitting the half sinusoid with few components while LoG and DoG are more stable for a larger number of components. }
    \label{supfig:fitting_half_sinusoid_N}
    \end{figure*}

%% file: tables/image_to_3d.tex
\begin{table}[t!]
  \centering
  \resizebox{0.99\linewidth}{!}{%
  \begin{tabular}{@{}c|cccc}
  \toprule
  \textbf{Dataset} & \textbf{Metrics\textbackslash{}Methods} 
  & Point-E &DreamGaussian &\textbf{\methodname (Ours)}\\
  \midrule
  \multirow{2}{*}{\textbf{NeRF4}}        
  & CLIP-Similarity$\uparrow$  & \cellcolor{yellow!40} 0.48  & \cellcolor{orange!40}0.56 &0.58 \cellcolor{red!40}\\
  & PSNR$\uparrow$             & \cellcolor{yellow!40} 0.70  & \cellcolor{red!40} 13.48 & \cellcolor{orange!40}13.33\\ \hline
  \multirow{2}{*}{\textbf{RealFusion15}} 
  & CLIP-Similarity$\uparrow$ & \cellcolor{yellow!40} 0.53 & 0.70 \cellcolor{orange!40} &0.70  \cellcolor{orange!40}\\
  & PSNR$\uparrow$            & \cellcolor{yellow!40} 0.98   & \cellcolor{orange!40} 12.83 & \cellcolor{red!40}12.91  \\ \midrule
 - & Average Runtime $\downarrow$  & 78 secs \cellcolor{red!40}& \cellcolor{orange!40} 2 mins& \cellcolor{orange!40} 2 mins\\
  \bottomrule
  \end{tabular}
  }
  \vspace{-4pt}
  \caption{ \textbf{GES Application: Fast Image-to-3D Generation pipeline} We show quantitative results in terms of CLIP-Similarity$\uparrow$ / PSNR$\uparrow$ , and Runtime$\downarrow$, compared to fast methods: Point-E~\cite{PointE} and DreamGaussian~\cite{dreemgaussian} . \methodname offers a good option for a fast and effective image-to-3D solution.}
  \label{tab:img3d}
  \end{table}

%% file: figures/gen3d_vis_supp.tex
\begin{figure}[t!] 
\centering

\includegraphics[trim={2.6cm 0 0.3cm 1.0cm},clip,width=0.47\textwidth]{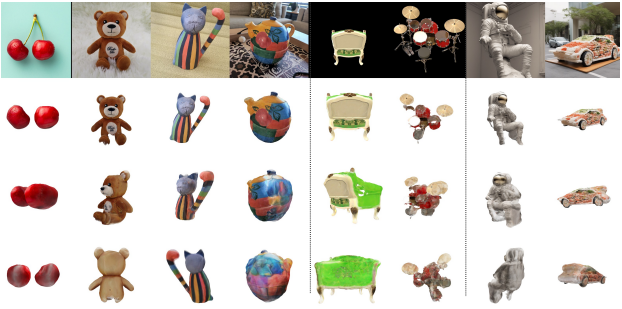}
\caption{
\textbf{Visulization for 3D generation}. We show selected generated examples by GES from Realfusion15 (\textit{left}) and NeRF4 datasets (\textit{middle}). Additionally, we pick two text prompts: \textit{"a car made out of sushi" and "Michelangelo style statue of an astronaut"}, and then use StableDiffusion-XL~\cite{sdxl} to generate the reference images before using \methodname on them(\textit{right}).
}
\label{figsup:gen3d_vis}
\vspace{-1mm}
\end{figure}

%% file: tables/ablate_density_grad_threshold.tex
\begin{figure}[t]
\centering
\footnotesize
\includegraphics[width=0.98\linewidth]{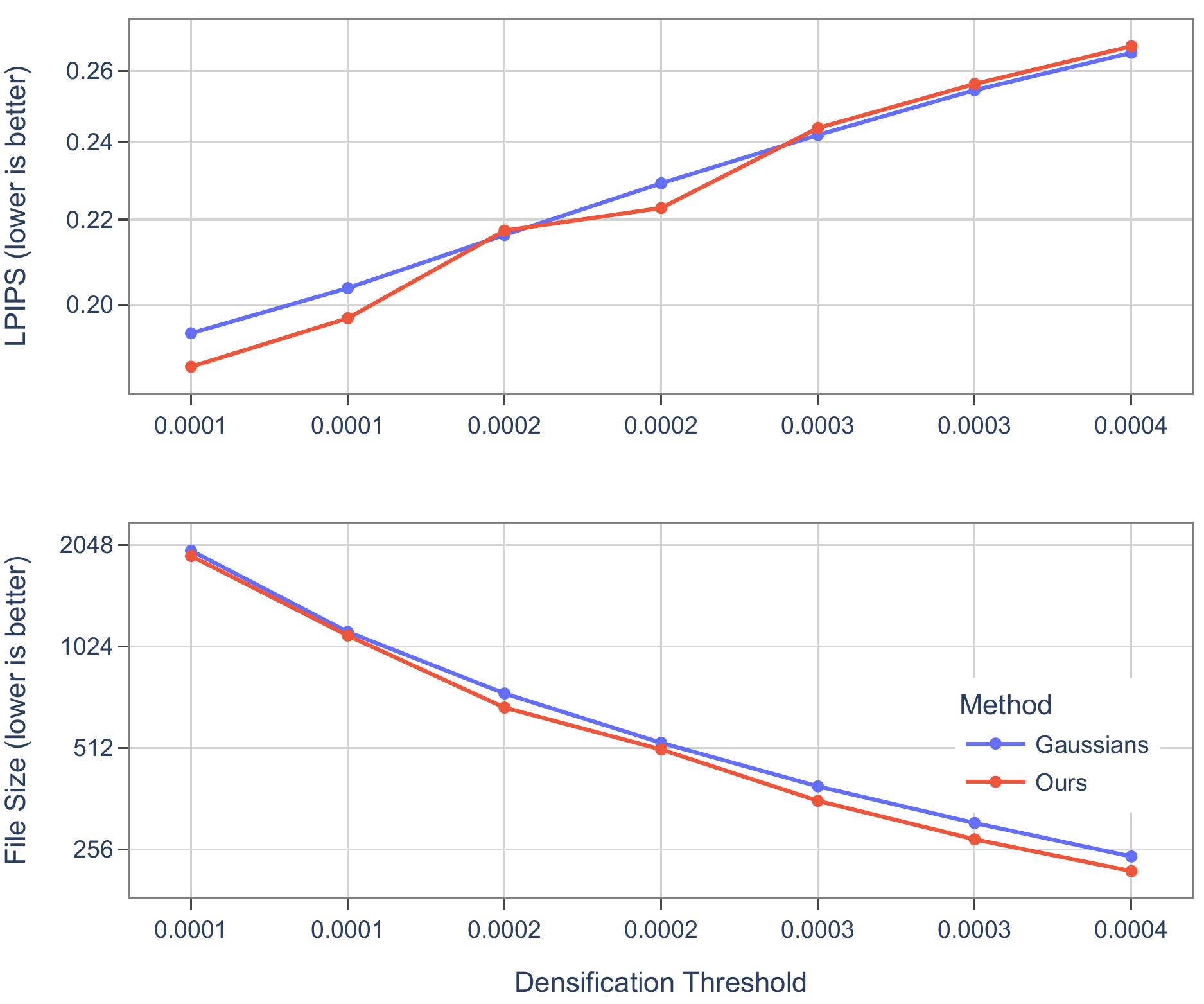}
\caption{\textbf{Ablation Study of Densification Threshold on Novel View Synthesis.} Impact of the densification threshold on reconstruction quality (LPIPS) and file size (MB) for our method and Gaussian Splatting~\cite{gaussiansplatter}, averaged across all scenes in the Mip-NeRF dataset. We see that the densification threshold has a significant impact on both file size and quality. Across the board, our method produces smaller scenes than Gaussian Splatting with similar or even slightly improved performance.}
\label{fig:ablatedenity}
\vspace{-2mm}
\end{figure}

%% file: figures/supplementary_ablations.tex
\begin{table}[htbp]
    \centering
    
    \begin{tabular}{llrrrr}
    \toprule
    $\lambda_{\text{freq}}$ &  Method   &  PSNR           &  LPIPS          &  SSIM           &  Size \\ \midrule
    \multicolumn{6}{c}{\textit{Deep Blending}} \\ \midrule
    \multirow{2}{*}{0.05}  &  GES     &  29.58 &  0.252 &  0.900 &  431      \\
                               &  GES (fixed $\beta=2$) &  29.53          &  0.251          &  0.901          &  433       \\ \midrule
    \multirow{2}{*}{0.10}  &  GES     &  29.54          &  0.252 &  0.901 &  428      \\
                               &  GES (fixed $\beta=2$) &  29.61 &  0.252 &  0.901 &  435       \\ \midrule
    \multirow{2}{*}{0.50}  &  GES     &  29.66 &  0.251          &  0.901 &  397      \\
                               &  GES (fixed $\beta=2$) &  29.61          &  0.252 &  0.901 &  437       \\ \midrule
    \multirow{2}{*}{0.90}  &  GES     &  27.21          &  0.259          &  0.899 &  366      \\
                               &  GES (fixed $\beta=2$) &  29.62 &  0.252 &  0.901          &  434       \\ \midrule
    \multicolumn{6}{c}{\textit{MipNeRF}} \\ \midrule
    \multirow{2}{*}{0.05}  &  GES     &  27.08 &  0.250 &  0.796          &  405      \\
                               &  GES (fixed $\beta=2$) &  27.05          &  0.250 &  0.795 &  411       \\ \midrule
    \multirow{2}{*}{0.10}  &  GES     &  27.05 &  0.250 &  0.795          &  403      \\
                               &  GES (fixed $\beta=2$) &  27.05 &  0.250 &  0.796          &  412       \\ \midrule
    \multirow{2}{*}{0.50}  &  GES     &  26.97          &  0.252 &  0.794 &  376      \\
                               &  GES (fixed $\beta=2$) &  27.09 &  0.250          &  0.796          &  415       \\ \midrule
    \multirow{2}{*}{0.90}  &  GES     &  25.82          &  0.255 &  0.792 &  364      \\
                               &  GES (fixed $\beta=2$) &  27.08 &  0.250          &  0.795          &  413       \\ \midrule
    \multicolumn{6}{c}{\textit{Tanks and Temples}} \\ \midrule
    \multirow{2}{*}{0.05}  &  GES     &  23.49          &  0.196 &  0.837          &  251      \\
                               &  GES (fixed $\beta=2$) &  23.55 &  0.196 &  0.836 &  255       \\ \midrule
    \multirow{2}{*}{0.10}  &  GES     &  23.54 &  0.196 &  0.837 &  247      \\
                               &  GES (fixed $\beta=2$) &  23.53          &  0.196 &  0.837 &  255       \\ \midrule
    \multirow{2}{*}{0.50}  &  GES     &  23.35          &  0.197 &  0.836 &  221      \\
                               &  GES (fixed $\beta=2$) &  23.65 &  0.196          &  0.837          &  256       \\ \midrule
    \multirow{2}{*}{0.90}  &  GES     &  22.65          &  0.200 &  0.834 &  210      \\
                               &  GES (fixed $\beta=2$) &  23.50 &  0.197          &  0.836          &  256       \\ \midrule
    \bottomrule
    \end{tabular}
    
    \caption{\textbf{Ablation of $\lambda_{\text{freq}}$.} We show a comparison of performance (PSNR, LPIPS, SSIM) for various values of $\lambda_{\text{freq}}$. Note that increasing $\lambda_{\text{freq}}$ in \methodname indeed reduces the size of the file, but can affect the performance. We chose $\lambda_{\text{freq}} =0.5$ as a middle ground between improved performance and reduced file size. }
    \label{figsup:lambda_laplace}
    \end{table}

%% file: figures/distribution.tex
\begin{figure}[h]
\centering
\includegraphics[trim= 0.0cm 0cm 0cm 0cm,clip, width=\linewidth]{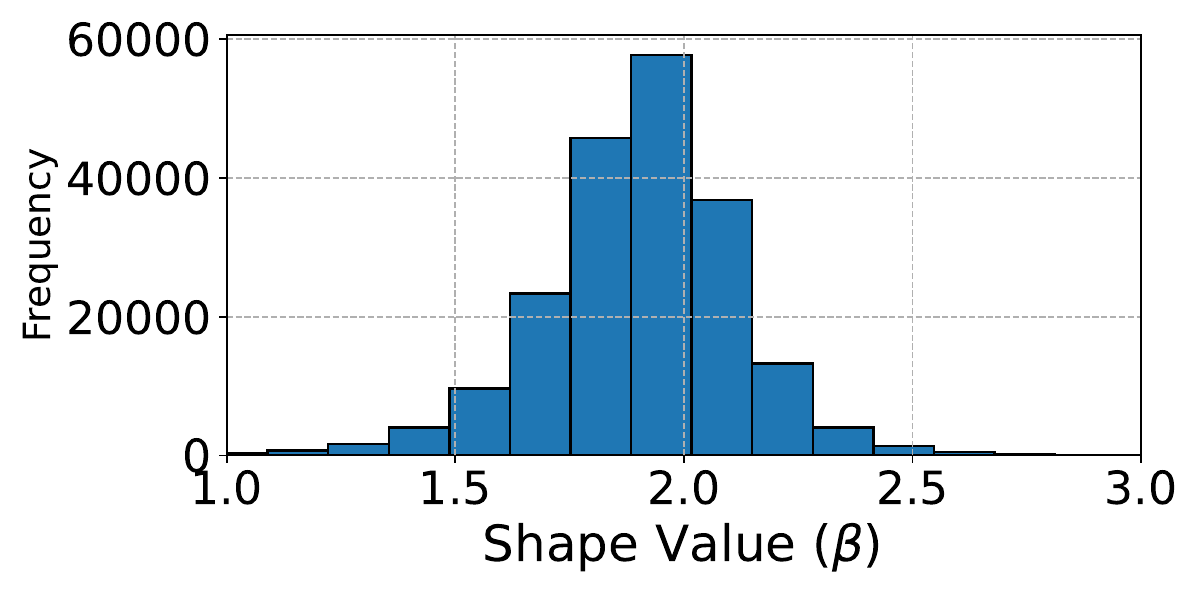}
\caption{\textbf{Distribution of Shape Values}. We show a distribution of $\beta$ values of a converged GES initialized with $\beta=2$. It shows a slight bias to $\beta$ smaller than 2.}
\label{fig:distribution}
\end{figure}

%% file: figures/sizes.tex
\begin{figure}[h]
\centering
\includegraphics[trim= 0.0cm 0cm 0cm 0cm,clip, width=\linewidth]{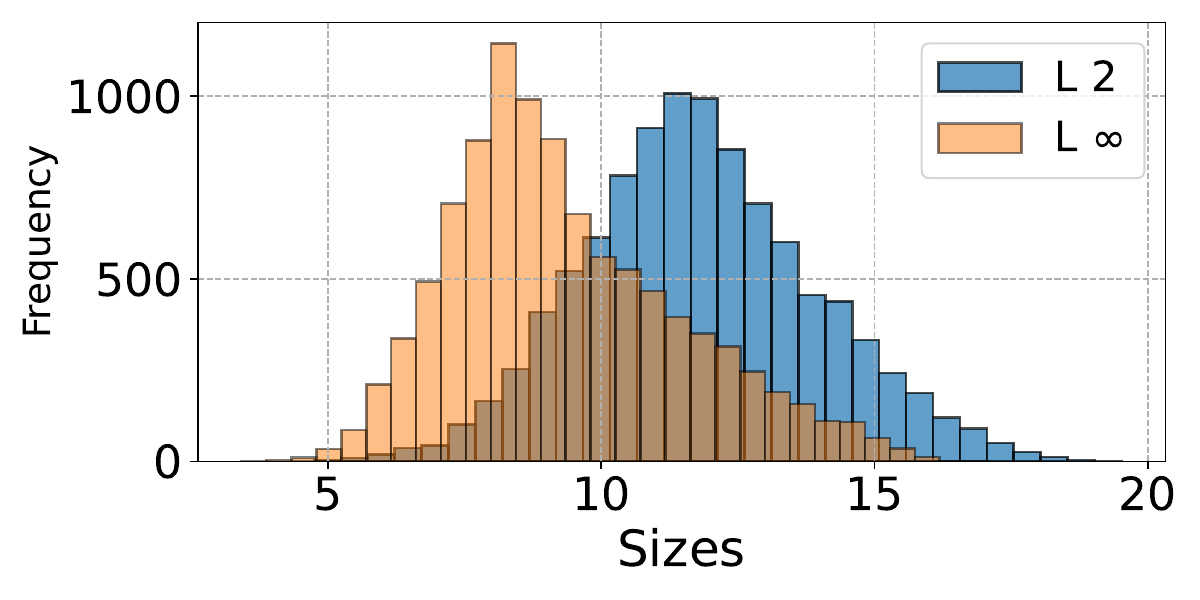}
\caption{\textbf{Distribution of Sizes}. We show a distribution of sizes ($L_2$ and $L_\infty$) of the GES components of a converged scene.}
\label{fig:sizes}
\end{figure}

%% file: figures/convergence.tex
\begin{figure*}[h]
    \centering
    \resizebox{1.0\linewidth}{!}{
    \begin{tabular}{c|c|c}
    \tabcolsep=0.01cm
    L1 Loss & Train Loss & Train PSNR \\ 
    \includegraphics[width=0.33\linewidth]{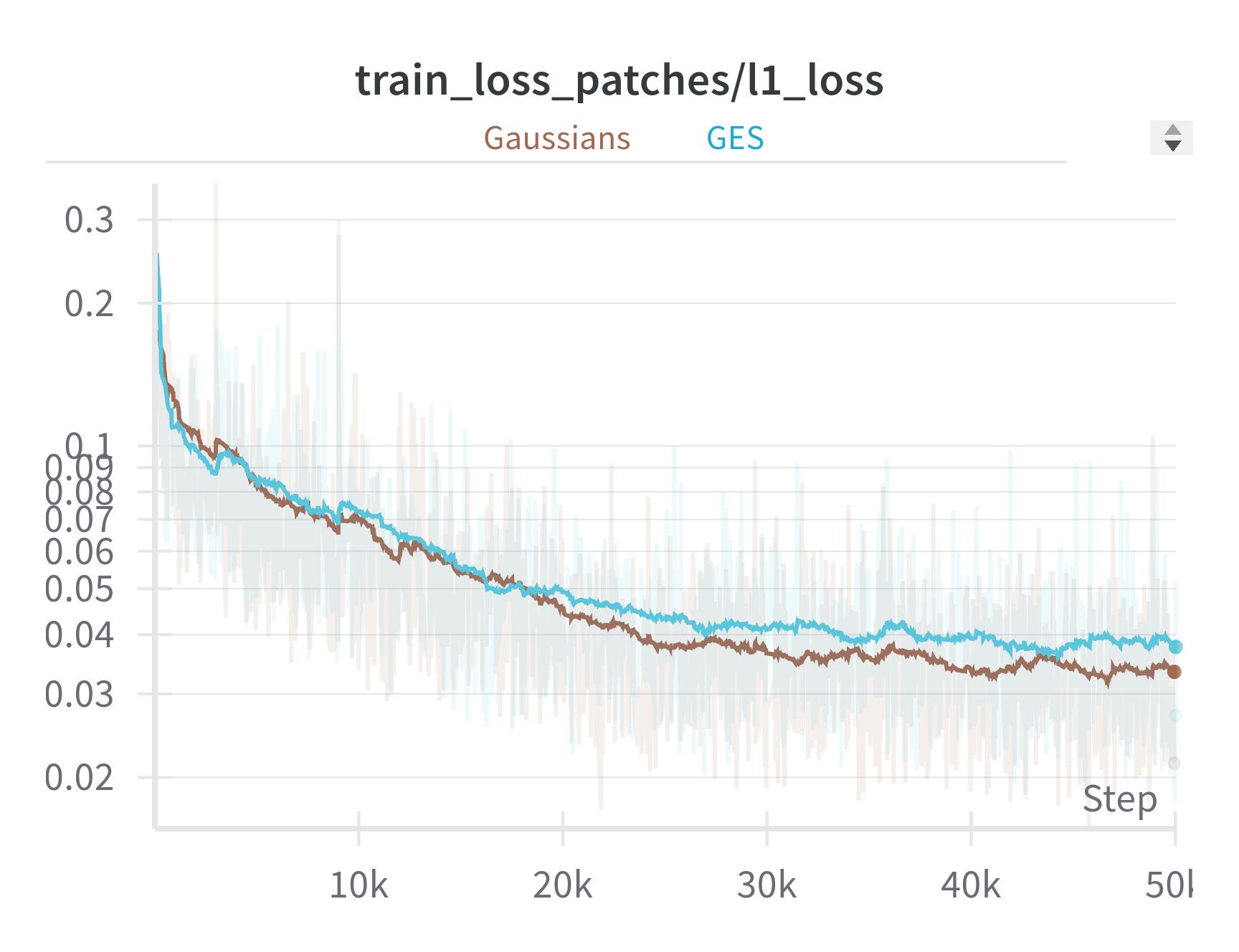} & 
    \includegraphics[width=0.33\linewidth]{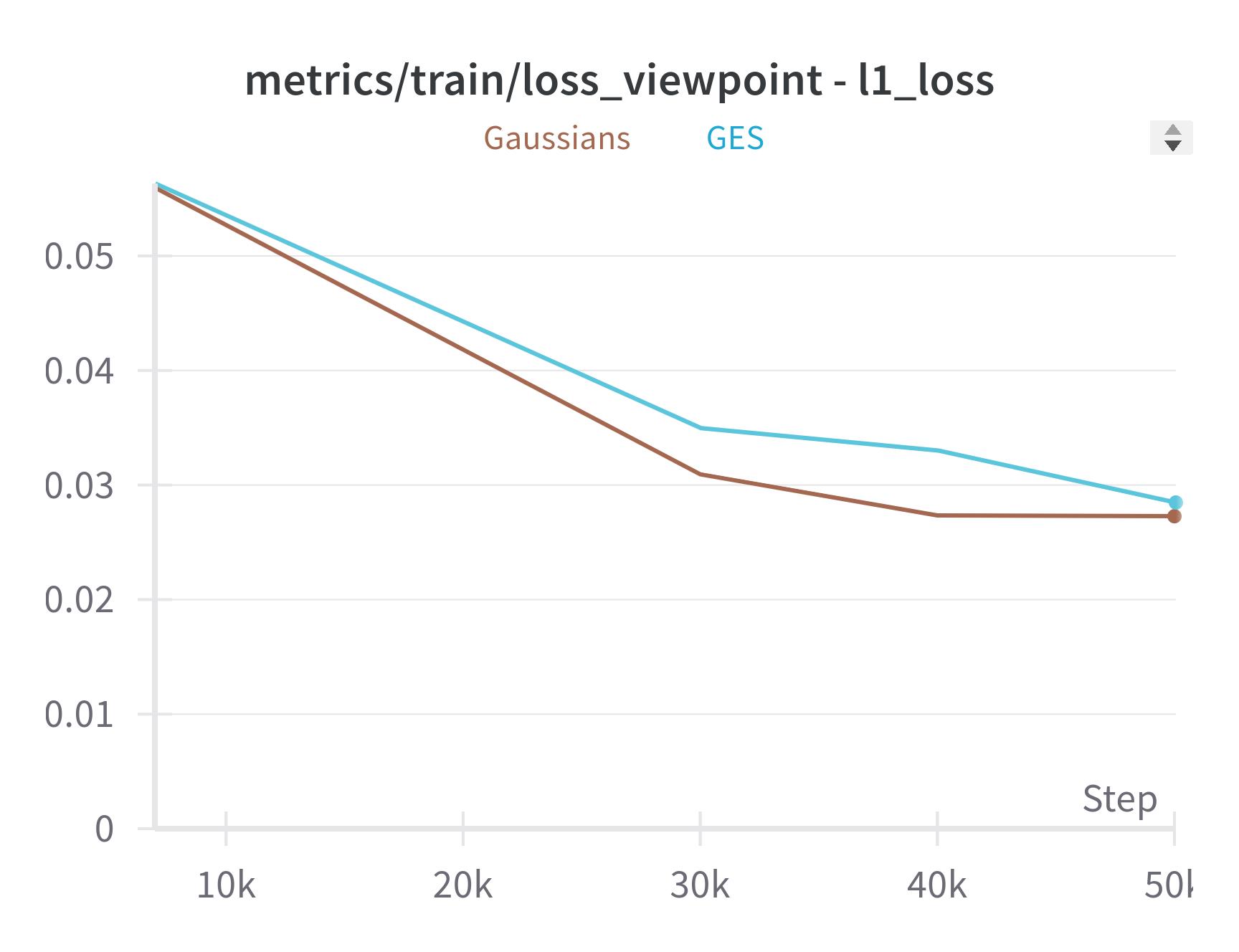} & 
    \includegraphics[width=0.33\linewidth]{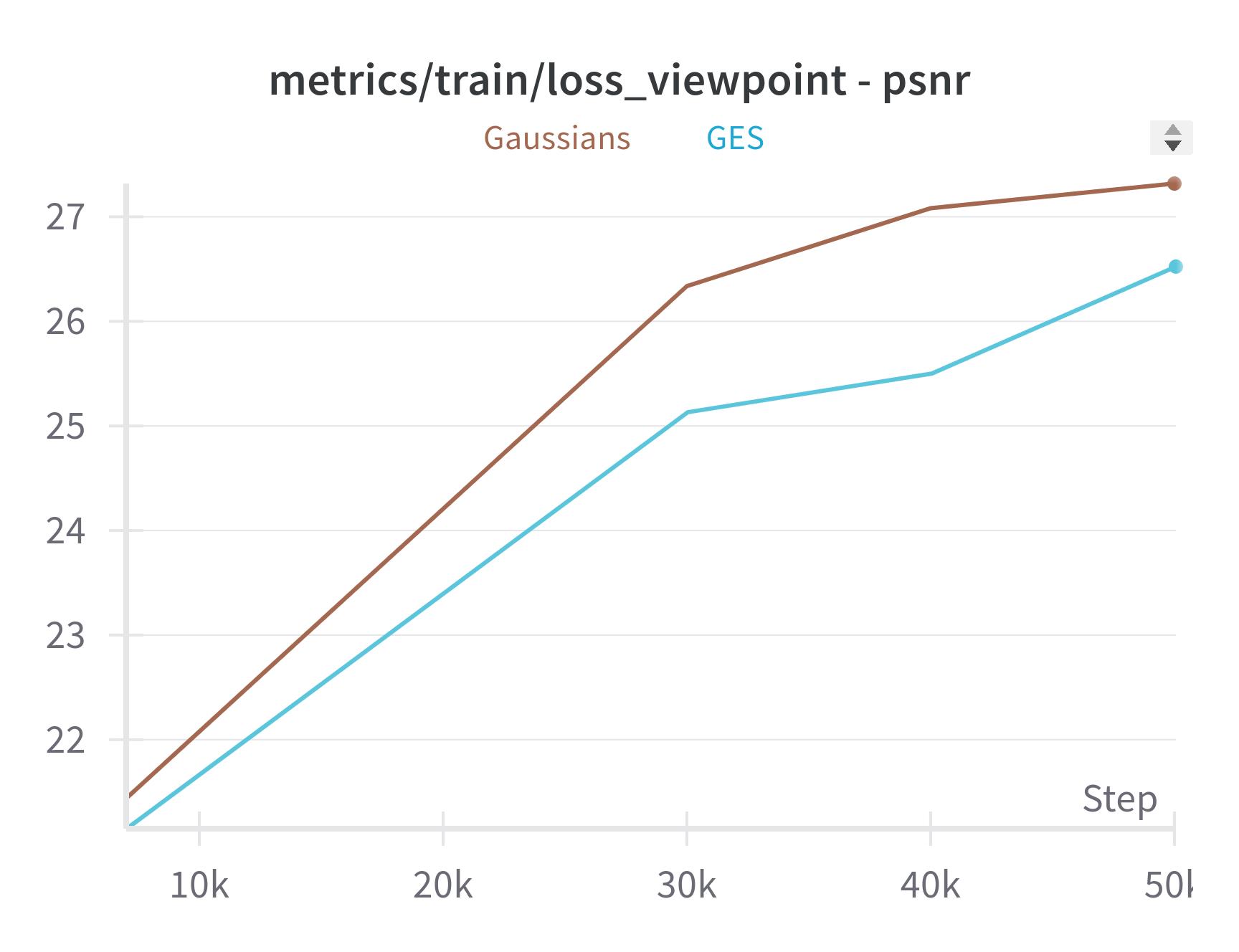} \\ \midrule
    Number of Components & Test Loss & Test PSNR \\ 
    \includegraphics[width=0.33\linewidth]{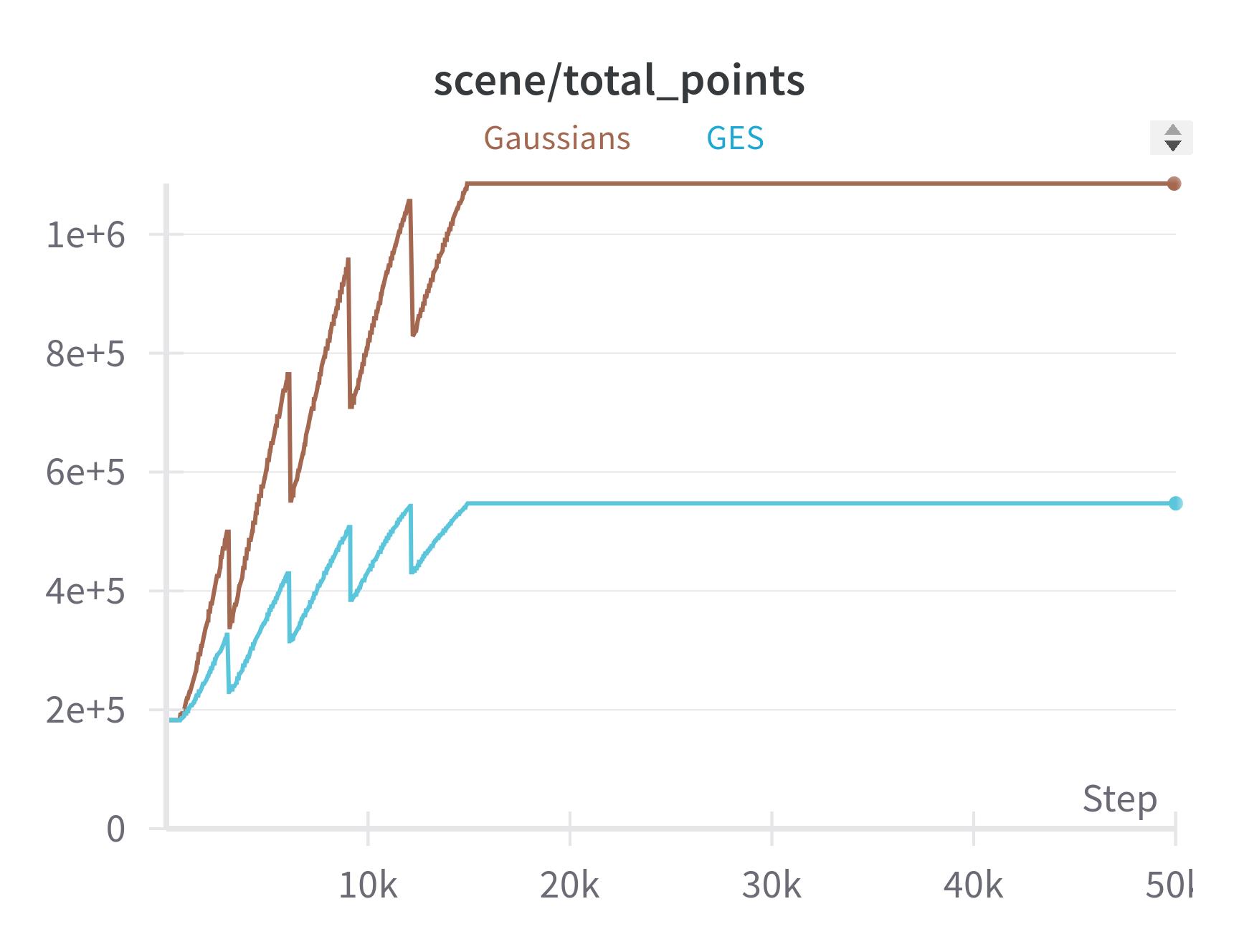} & 
    \includegraphics[width=0.33\linewidth]{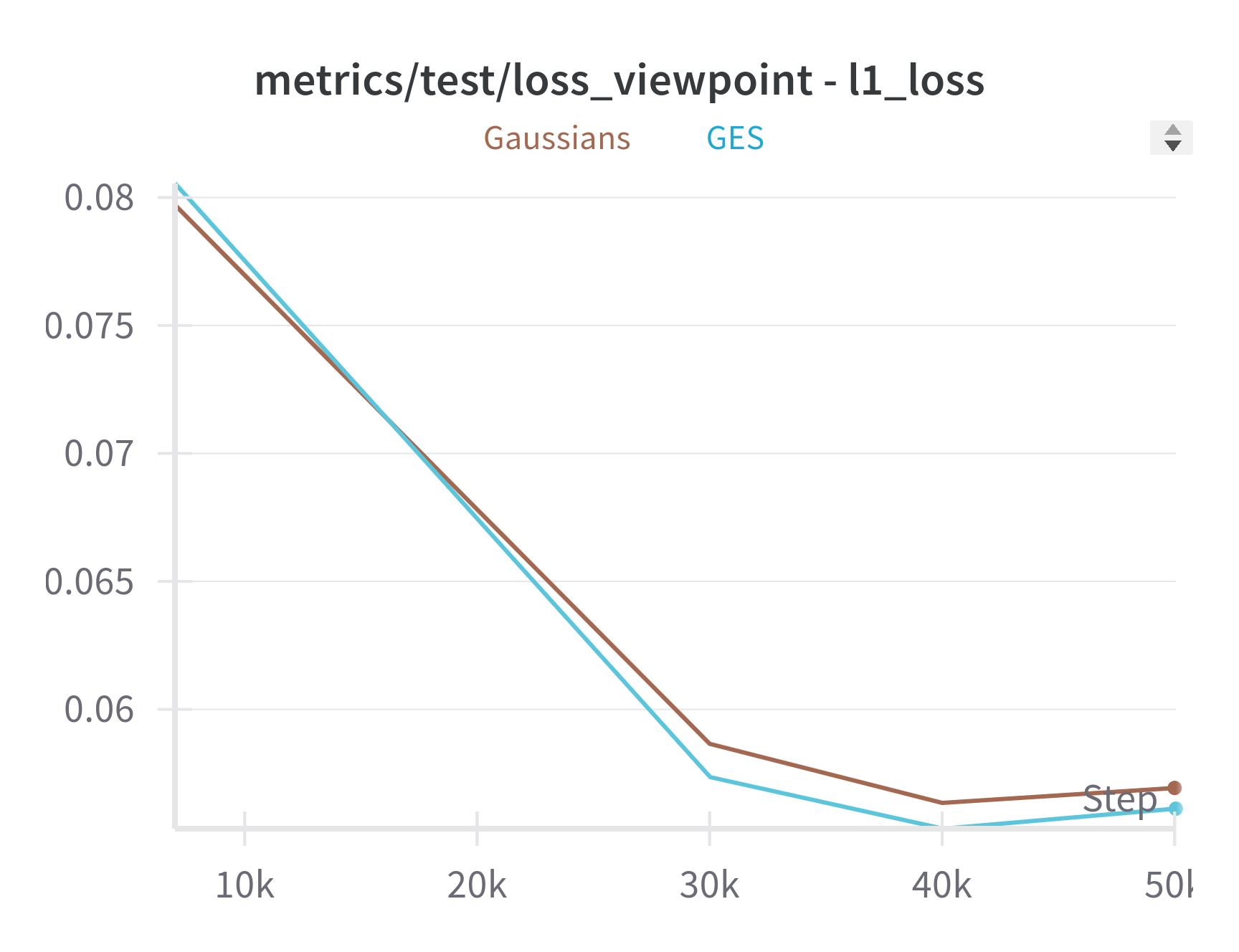} & 
    \includegraphics[width=0.33\linewidth]{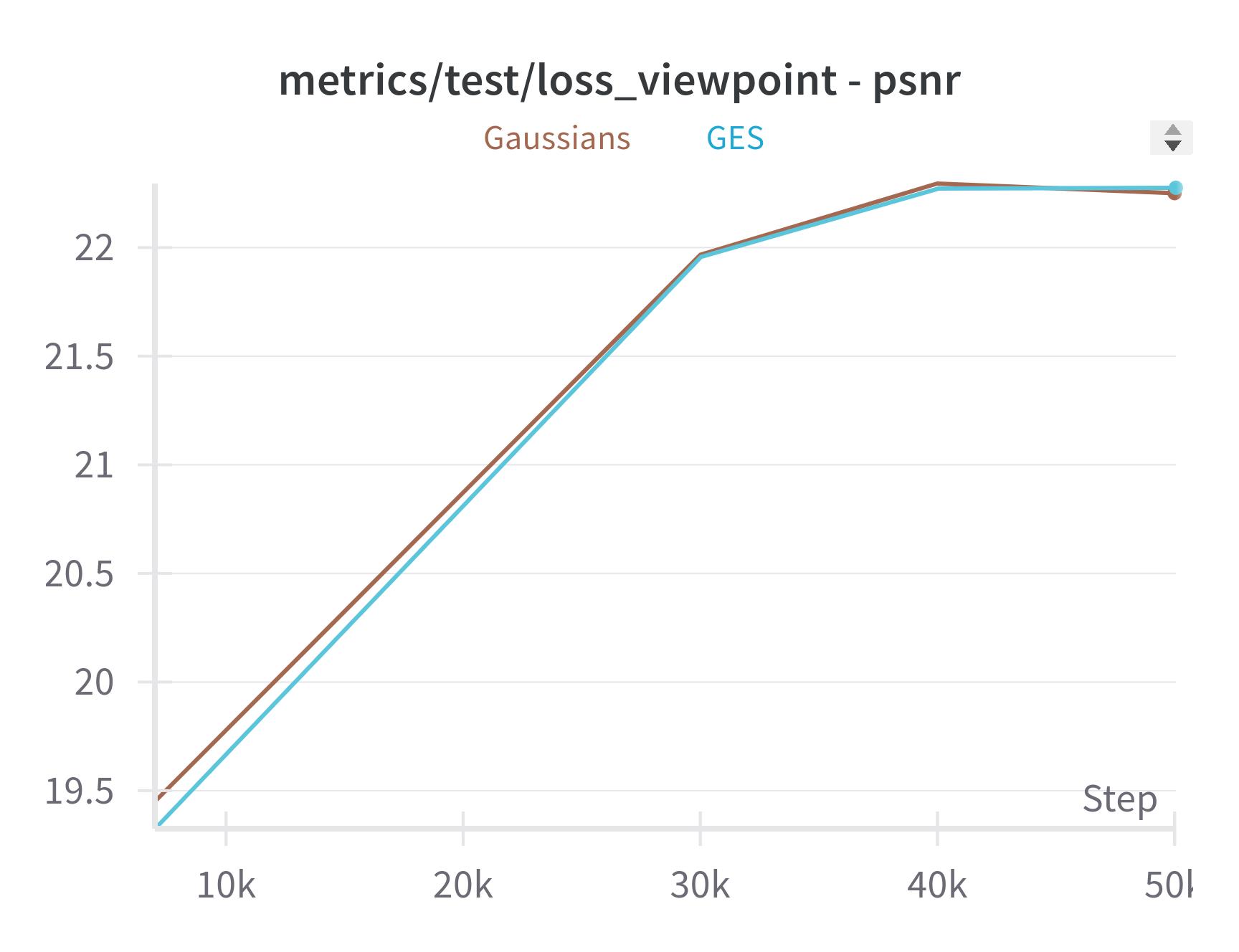} \\
    
    \end{tabular}
    }
    \caption{\textbf{Convergence Plots of Gaussians \vs \methodname }. We show an example of the convergence plots of both \methodname and Gaussians if the training continues up to 50K iterations to inspect the diminishing returns of more training. Despite requiring more iterations to converge, \methodname trains faster than Gaussians due to its smaller number of splatting components.}
    \label{fig:convergence}
    \end{figure*}

%% file: figures/masks_supp.tex
\begin{figure*}[t] 
    \centering
    \includegraphics[width=0.24\linewidth]{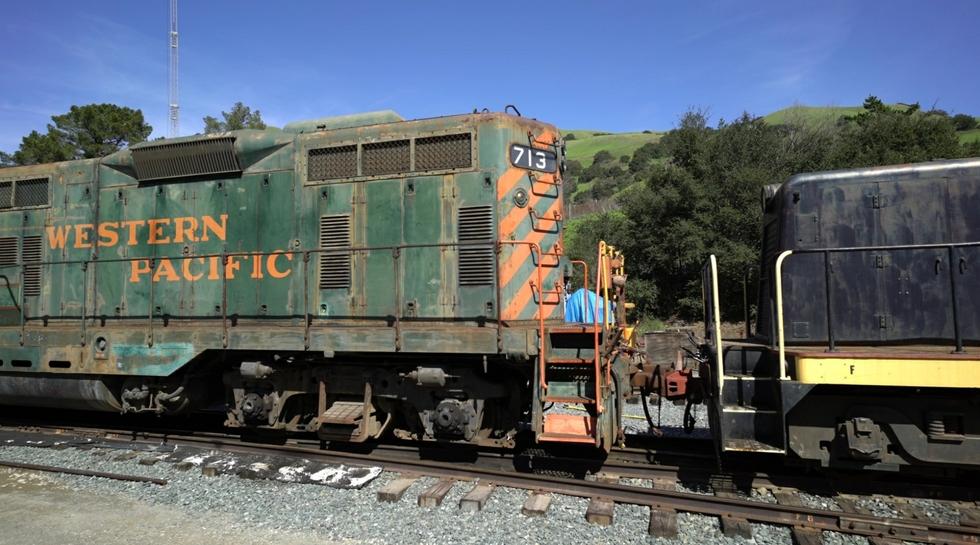}
    \includegraphics[trim={0 0 0 0},clip,width=0.24\linewidth]{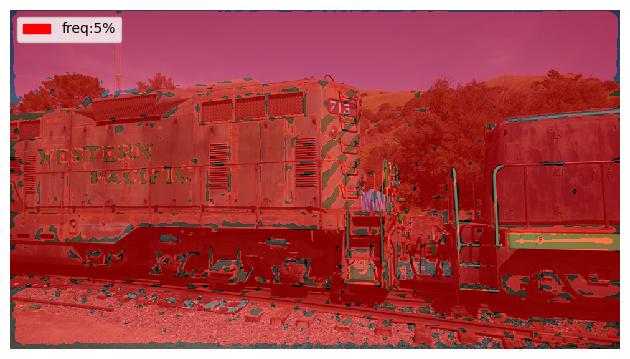}
    \includegraphics[trim={0 0 0 0},clip,width=0.24\linewidth]{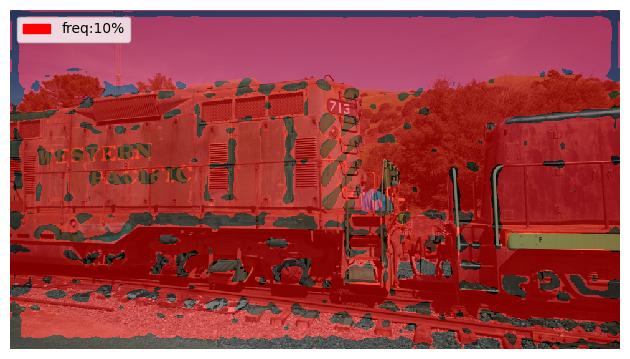}
    \includegraphics[trim={0 0 0 0},clip,width=0.24\linewidth]{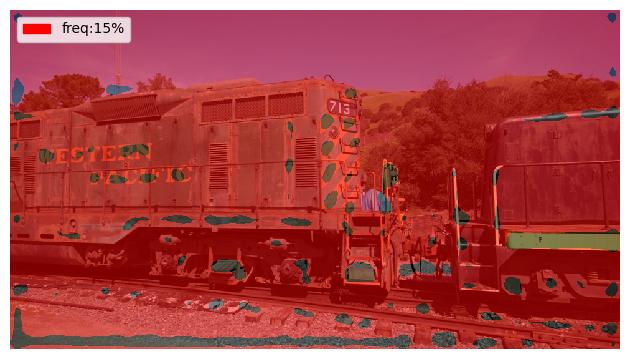}\\
    \includegraphics[trim={0 0 0 0},clip,width=0.24\linewidth]{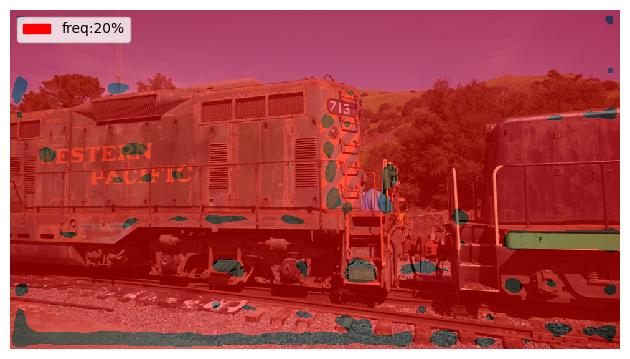}
    \includegraphics[trim={0 0 0 0},clip,width=0.24\linewidth]{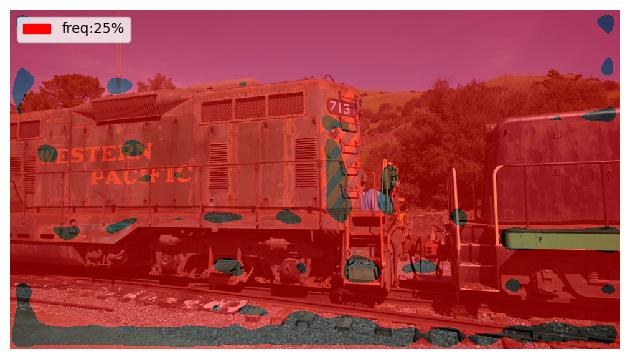}
    \includegraphics[trim={0 0 0 0},clip,width=0.24\linewidth]{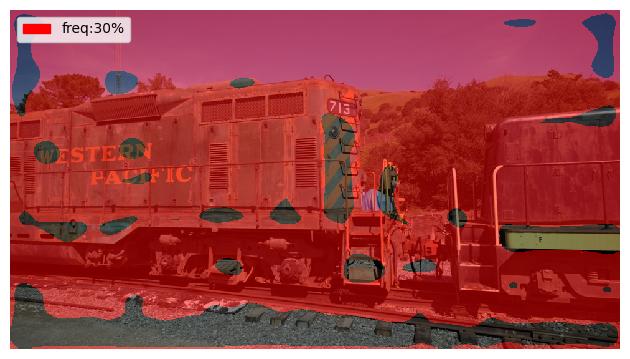}
    \includegraphics[trim={0 0 0 0},clip,width=0.24\linewidth]{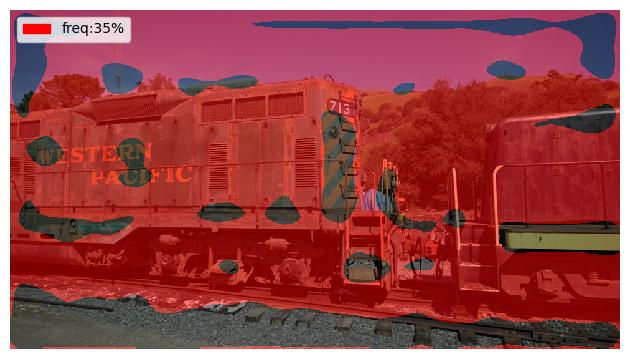}\\
    \includegraphics[trim={0 0 0 0},clip,width=0.24\linewidth]{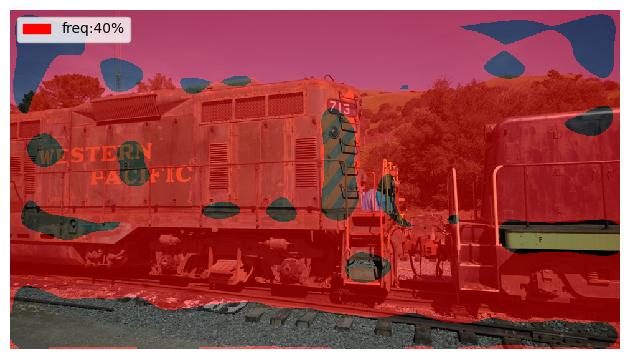}
    \includegraphics[trim={0 0 0 0},clip,width=0.24\linewidth]{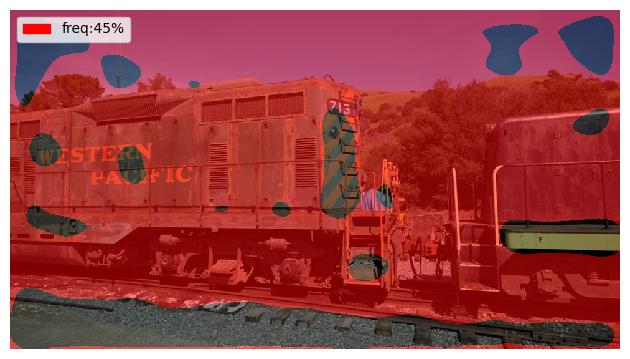}
    \includegraphics[trim={0 0 0 0},clip,width=0.24\linewidth]{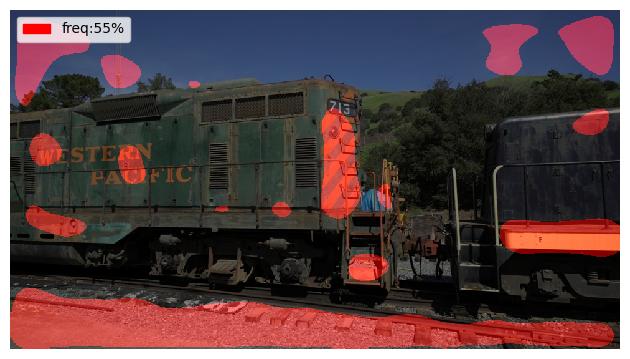}
    \includegraphics[trim={0 0 0 0},clip,width=0.24\linewidth]{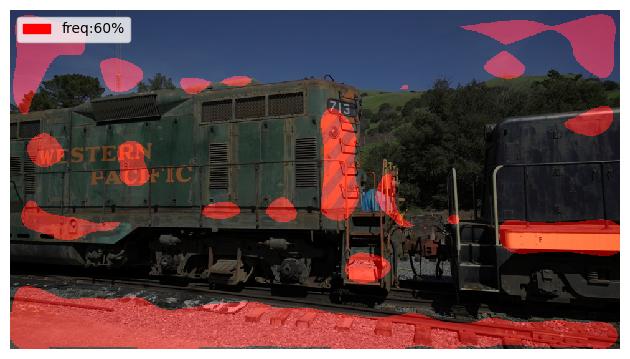}\\
    \includegraphics[trim={0 0 0 0},clip,width=0.24\linewidth]{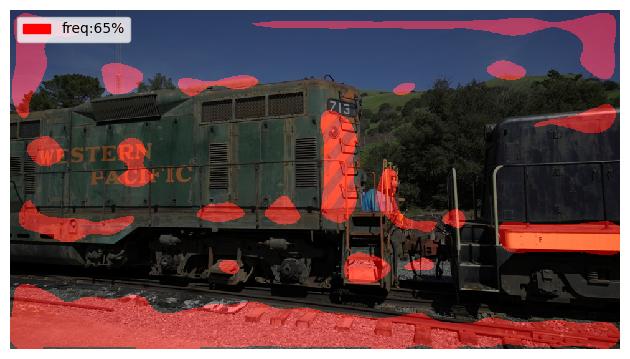}
    \includegraphics[trim={0 0 0 0},clip,width=0.24\linewidth]{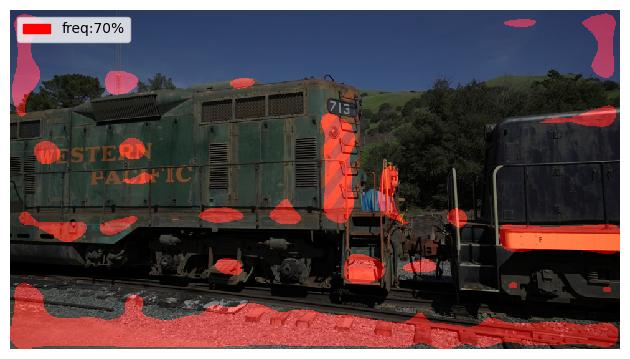}
    \includegraphics[trim={0 0 0 0},clip,width=0.24\linewidth]{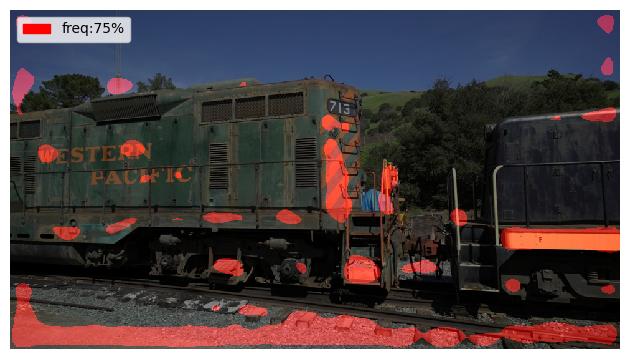}
    \includegraphics[trim={0 0 0 0},clip,width=0.24\linewidth]{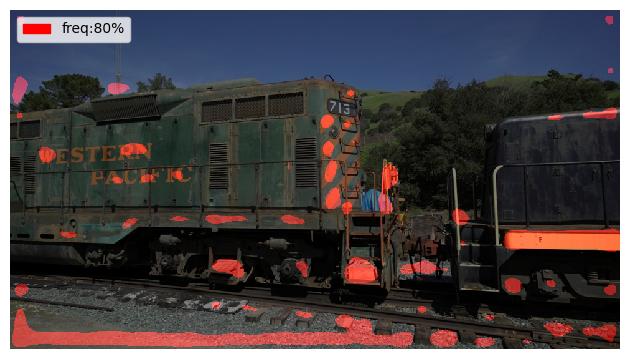}\\
    \includegraphics[trim={0 0 0 0},clip,width=0.24\linewidth]{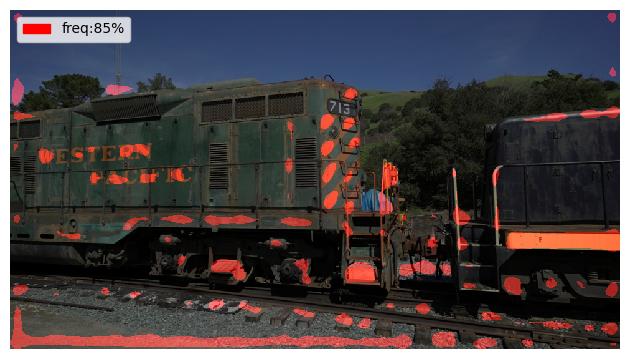}
    \includegraphics[trim={0 0 0 0},clip,width=0.24\linewidth]{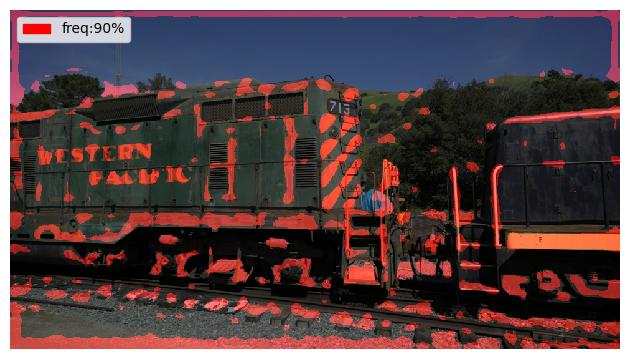}
    \includegraphics[trim={0 0 0 0},clip,width=0.24\linewidth]{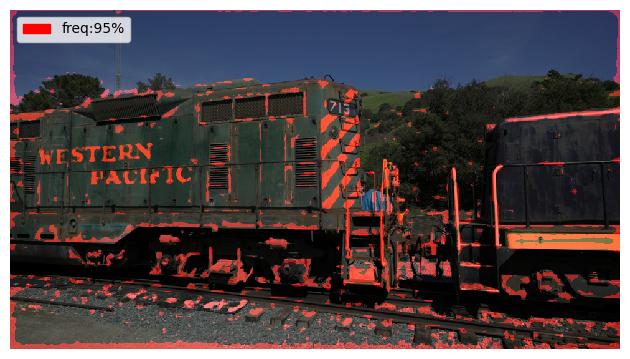}
    \includegraphics[trim={0 0 0 0},clip,width=0.24\linewidth]{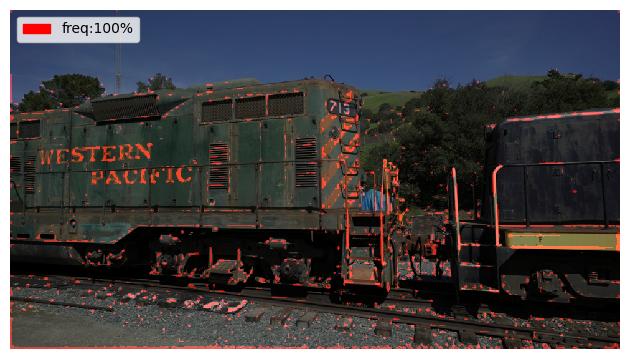}\\

\caption{
\textbf{Frequency-Modulated Image Masks.} For the input example image on the top left, We show examples of the frequency loss masks $M_\omega$ used in \seclabel{\ref{sec:image-laplacian-guidance}} for different numbers of target normalized frequencies $\omega$ ( $\omega= 0\%$ for low frequencies to $\omega= 100\%$ for high frequencies). This masked loss helps our \methodname learn specific bands of frequencies. Note that due to Laplacian filter sensitivity for high-frequencies, the mask for $ 0< \omega \leq 50 \%$ is defined as $1 - M_\omega $ for $ 50<\omega \leq 100 \%$. This ensures that all parts of the image will be covered by one of the masks $M_\omega$, while focusing on the details more as the optimization progresses.
}
    \label{figsup:mask}
    \end{figure*}

%% file: tables/ablate_shape.tex
\begin{table}[h]
\centering
\begin{center}
  \resizebox{0.93\linewidth}{!}{%

\begin{tabular}{lcccc}
\toprule
 \multirowcell{2}{Shape \\ Learning Rate}  & \multirowcell{2}{PSNR} & \multirowcell{2}{SSIM} & \multirowcell{2}{LPIPS} & \multirowcell{2}{File Size (MB)} \\
 & & & & \\
\midrule
0.0005 & 26.83 & 0.845 & \textbf{0.141} & 659 \\
0.0010 & 26.85 & 0.845 & \textbf{0.141} & 658 \\
0.0015 & \textbf{26.89} & \textbf{0.846} & \textbf{0.141} & \textbf{651} \\
0.0020 & 26.82 & 0.844 & 0.142 & 658 \\
\bottomrule

\toprule
 \multirowcell{2}{Shape \\ Reset Interval}  & \multirowcell{2}{PSNR} & \multirowcell{2}{SSIM} & \multirowcell{2}{LPIPS} & \multirowcell{2}{File Size (MB)} \\
 & & & & \\
\midrule
200 & 26.87 & 0.845 & \textbf{0.141} & 656 \\
500 & 26.86 & 0.845 & \textbf{0.141} & 658 \\
1000 & \textbf{26.89} & \textbf{0.846} & \textbf{0.141} & \textbf{651} \\
2000 & 26.84 & 0.845 & \textbf{0.141} & 657 \\
5000 & 26.84 & 0.845 & \textbf{0.141} & 661 \\
\bottomrule

\toprule
\multirowcell{2}{Shape \\ Strength}  & \multirowcell{2}{PSNR} & \multirowcell{2}{SSIM} & \multirowcell{2}{LPIPS} & \multirowcell{2}{File Size (MB)}  \\
 & & & & \\
\midrule
0.010 & 26.87 & 0.845 & \textbf{0.141} & 661 \\
0.050 & 26.84 & 0.845 & \textbf{0.141} & 653 \\
0.100 & \textbf{26.89} & \textbf{0.846} & \textbf{0.141} & \textbf{651} \\
0.150 & 26.83 & 0.844 & 0.142 & 656 \\
\bottomrule
\end{tabular}
}
\end{center}
\vspace{-3mm}
\footnotesize
\caption{\textbf{Ablation Study on Novel View Synthesis.} Impact of the shape parameter's learning rate, reset interval, and strength on reconstruction quality and file size for the garden scene from the Mip-NeRF dataset.}
\label{tab:ablateshape}
\end{table}

%% file: figures/viscomparison_supp.tex
\begin{figure*}[!h]
	\setlength\mytmplen{.193\linewidth}
	\setlength{\tabcolsep}{1pt}
	\renewcommand{\arraystretch}{0.5}
	\centering
	\begin{tabular}{cccccc}
		Ground Truth & \methodname (Ours) & Gaussians & Mip-NeRF360 & InstantNGP \\
		\includegraphics[width=\mytmplen,trim={0 1.5cm 0 2cm},clip]{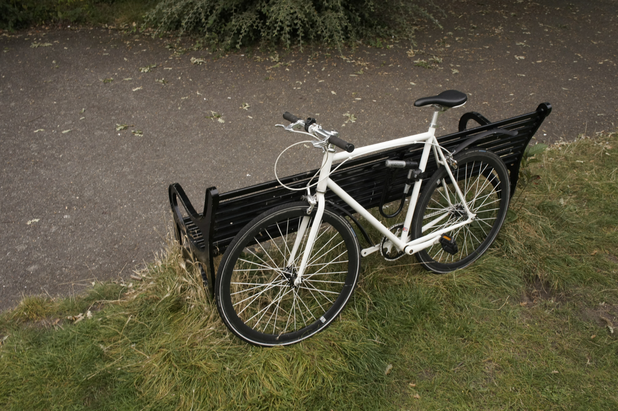} &
		\includegraphics[width=\mytmplen,trim={0 1.5cm 0 2cm},clip]{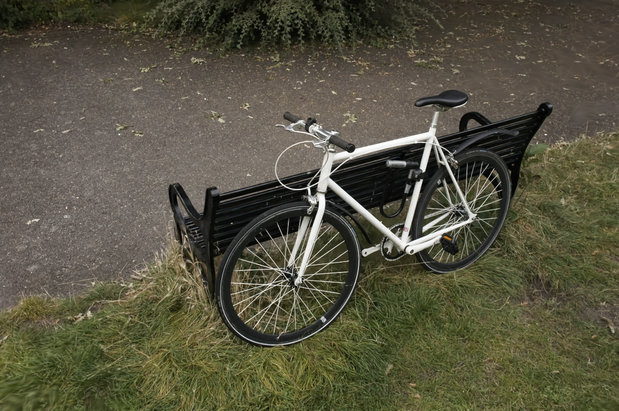} &
		\begin{tikzpicture}
			\node[anchor=south west,inner sep=0] (image) at (0,0) {\includegraphics[width=\mytmplen,trim={0 1.5cm 0 2cm},clip]{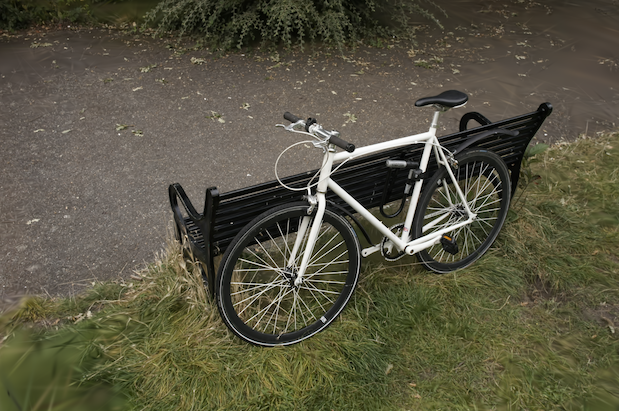}};
		\end{tikzpicture} &
	    \begin{tikzpicture}
			\node[anchor=south west,inner sep=0] (image) at (0,0) {\includegraphics[width=\mytmplen,trim={0 1.5cm 0 2cm},clip]{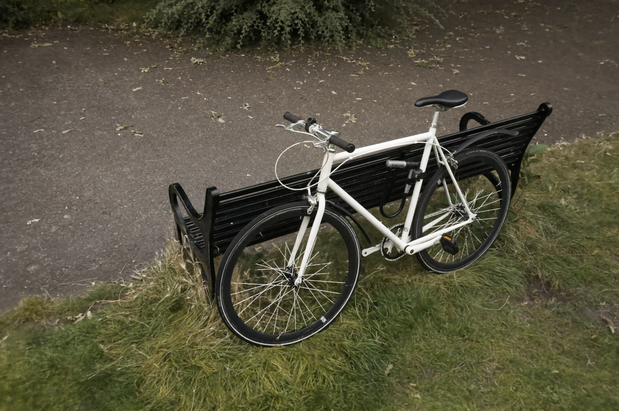}};
		\end{tikzpicture} &
		\includegraphics[width=\mytmplen,trim={0 1.5cm 0 2cm},clip]{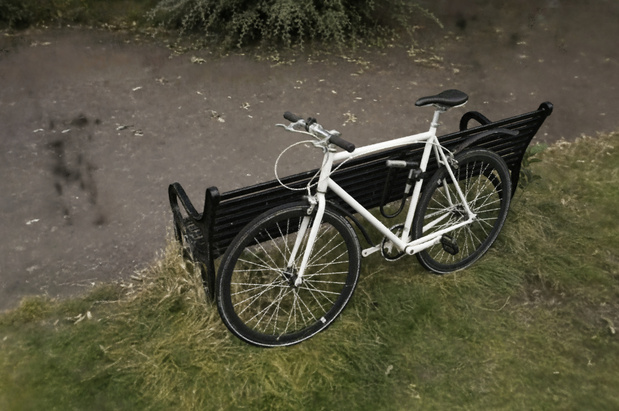} 
 \\
		\zoomin{figures/src/results2/garden/gt_gt_001.png}{0.6}{1.95}{0.55cm}{0.55cm}{1cm}{\mytmplen}{3}{red}&
		\zoomin{figures/src/results2/garden/laplacians_00001.png}{0.6}{1.95}{0.55cm}{0.55cm}{1cm}{\mytmplen}{3}{red} &
		\zoomin{figures/src/results2/garden/gaussians_DSC07964.png}{0.6}{1.95}{0.55cm}{0.55cm}{1cm}{\mytmplen}{3}{red}&
		\zoomin{figures/src/results2/garden/mipnerf360_color_001.png}{0.6}{1.95}{0.55cm}{0.55cm}{1cm}{\mytmplen}{3}{red}&
		\zoomin{figures/src/results2/garden/igp_out_DSC07964.jpg}{0.6}{1.95}{0.55cm}{0.55cm}{1cm}{\mytmplen}{3}{red}
\\
		\includegraphics[width=\mytmplen]{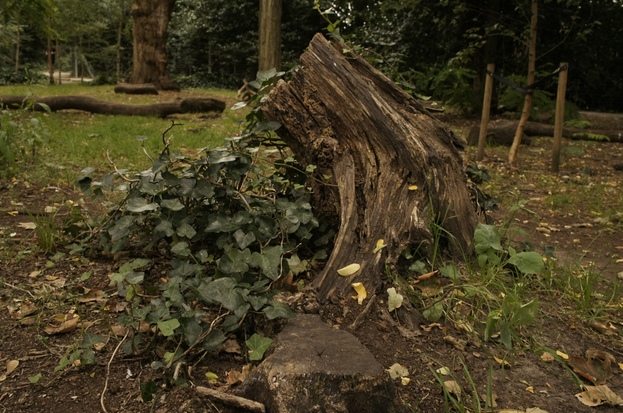} &
		\includegraphics[width=\mytmplen]{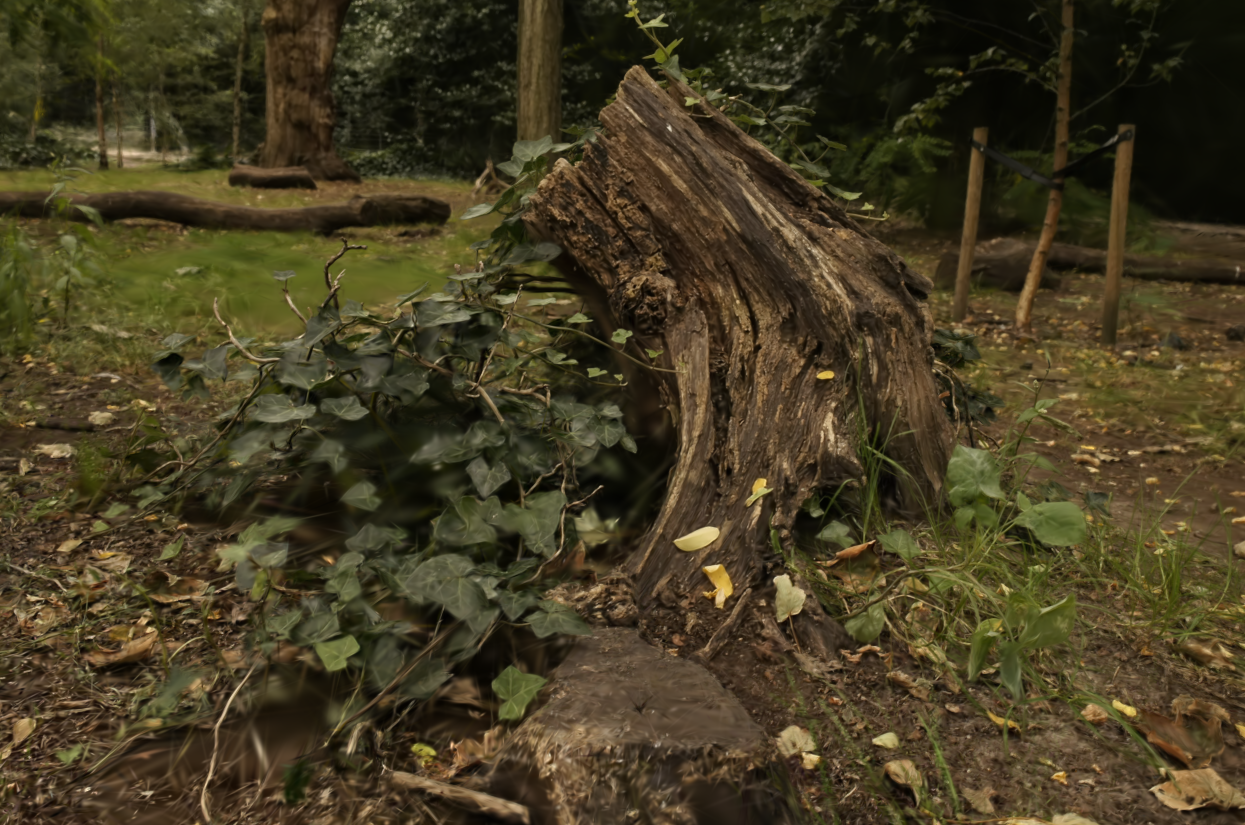} &
		\includegraphics[width=\mytmplen]{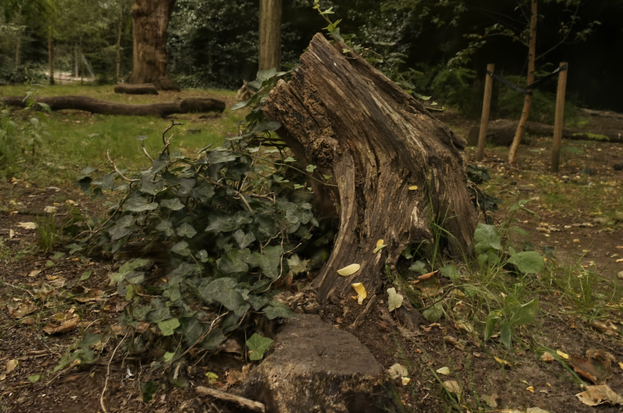} &
		\includegraphics[width=\mytmplen]{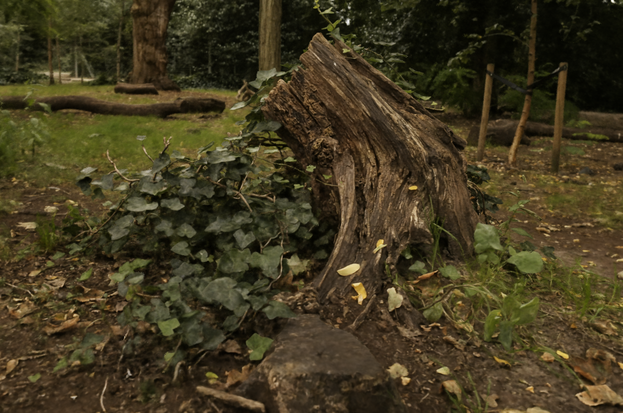} &
		\includegraphics[width=\mytmplen]{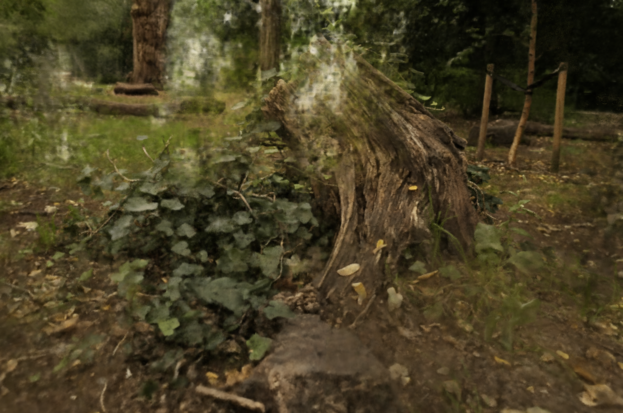} 
		\\
		\includegraphics[width=\mytmplen]{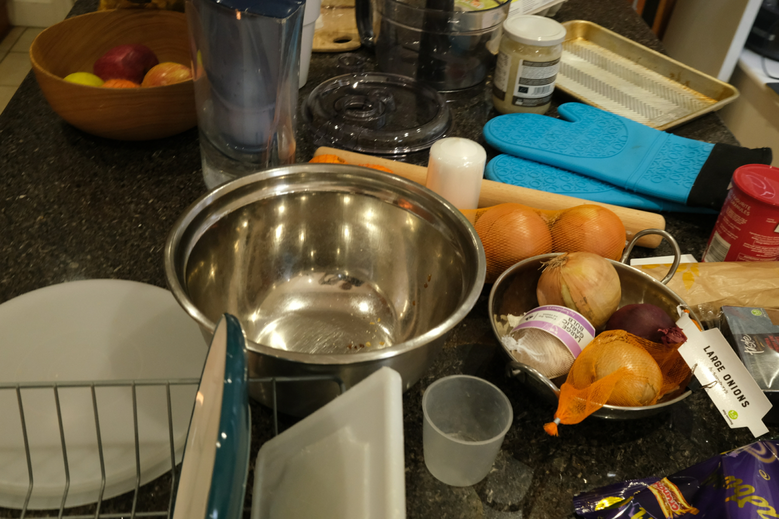} &
		\includegraphics[width=\mytmplen]{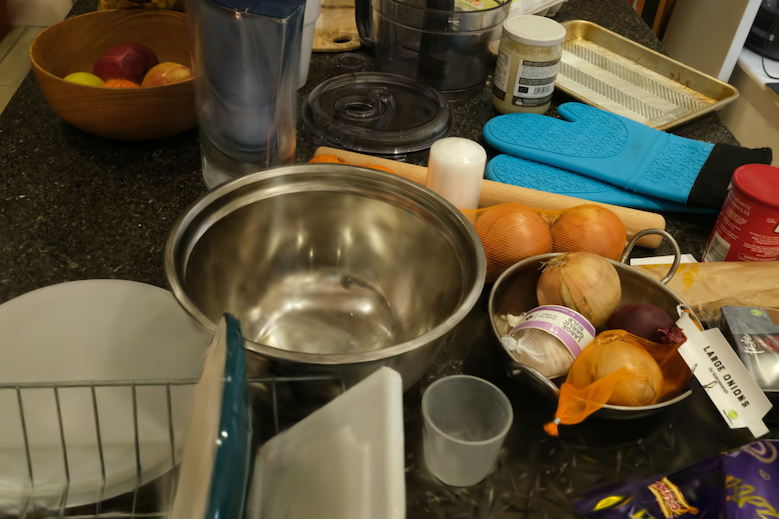} &
		\includegraphics[width=\mytmplen]{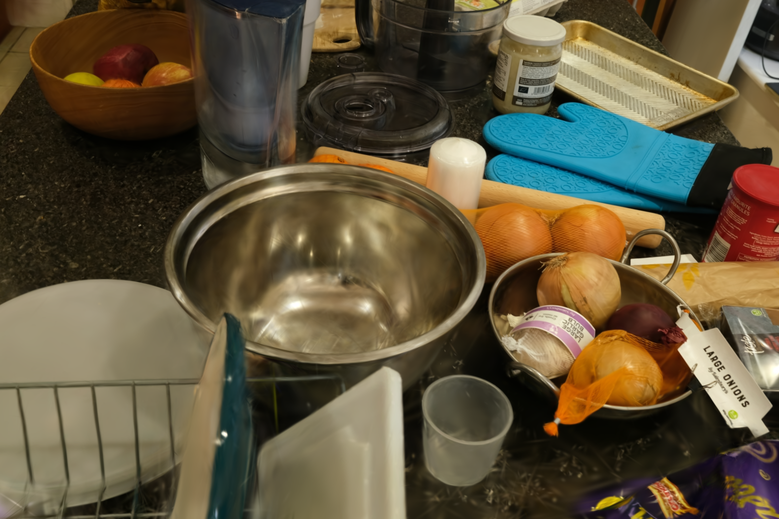} &
 		\begin{tikzpicture}
			\node[anchor=south west,inner sep=0] (image) at (0,0) {\includegraphics[width=\mytmplen]{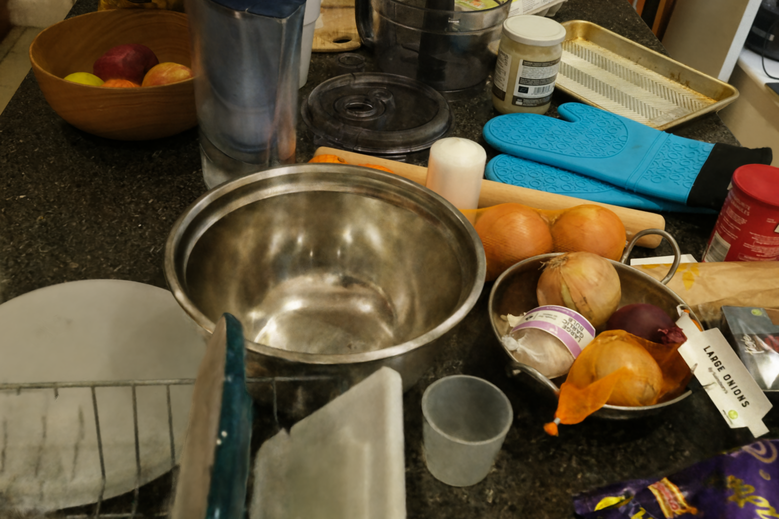}};
		\end{tikzpicture} &
		\includegraphics[width=\mytmplen]{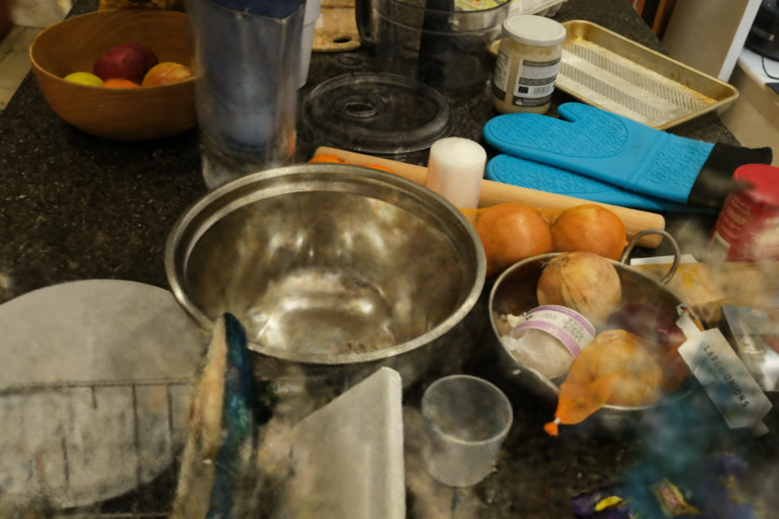} 
		\\
		\includegraphics[width=\mytmplen]{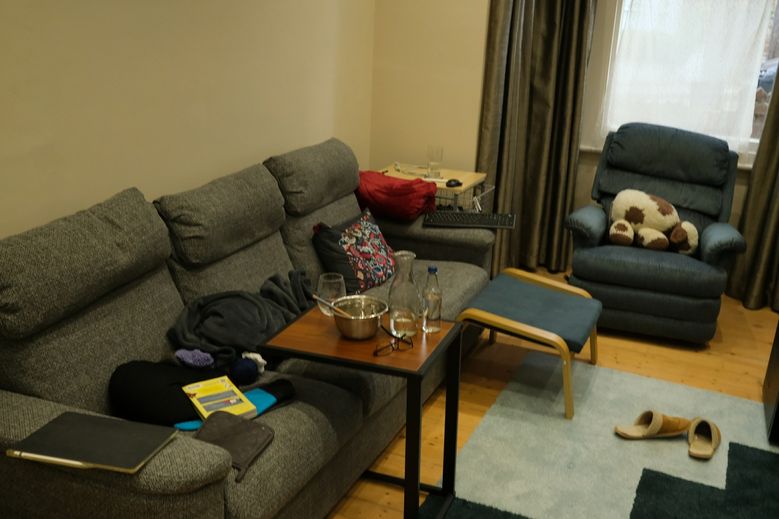} &
		\includegraphics[width=\mytmplen]{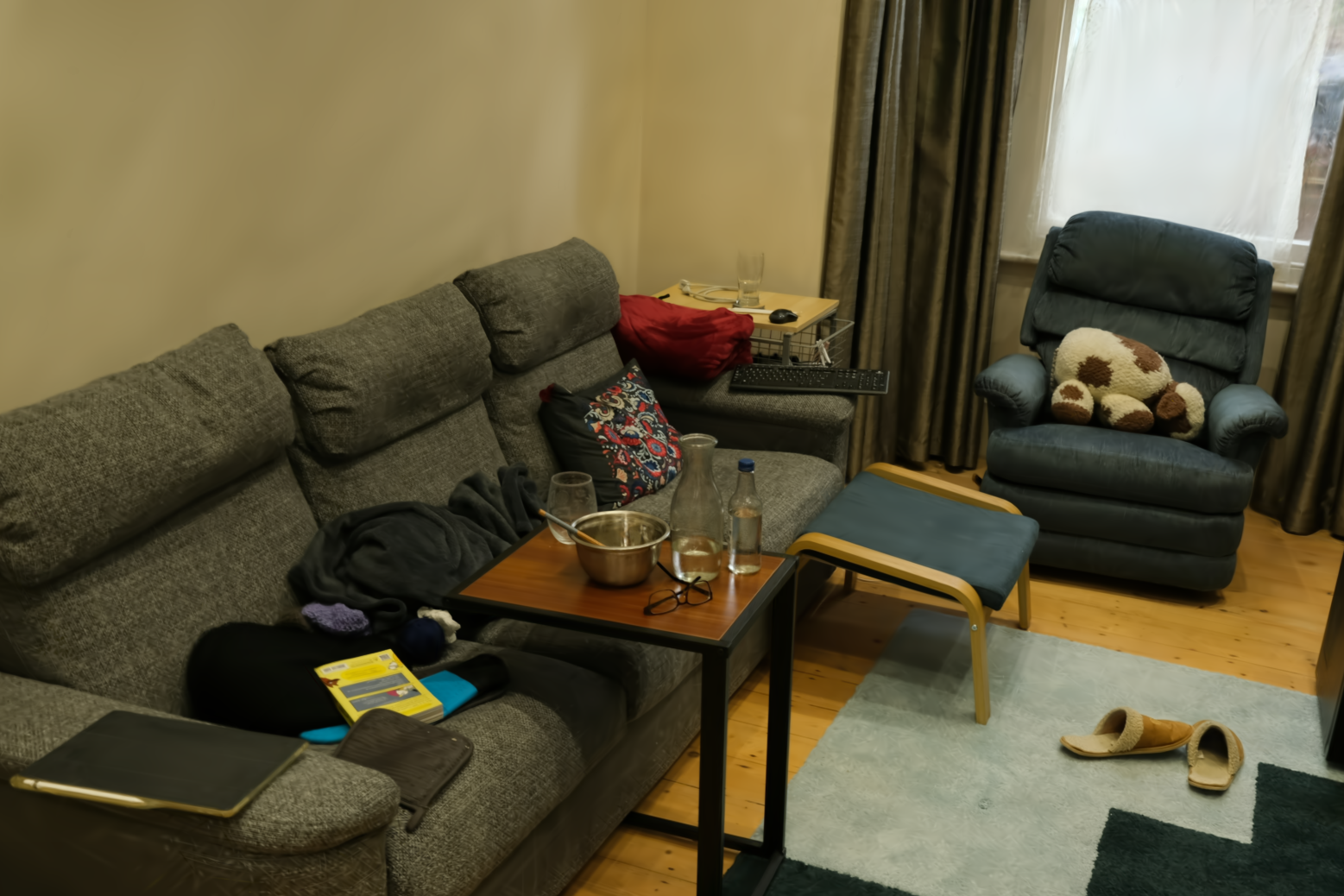} &
		\includegraphics[width=\mytmplen]{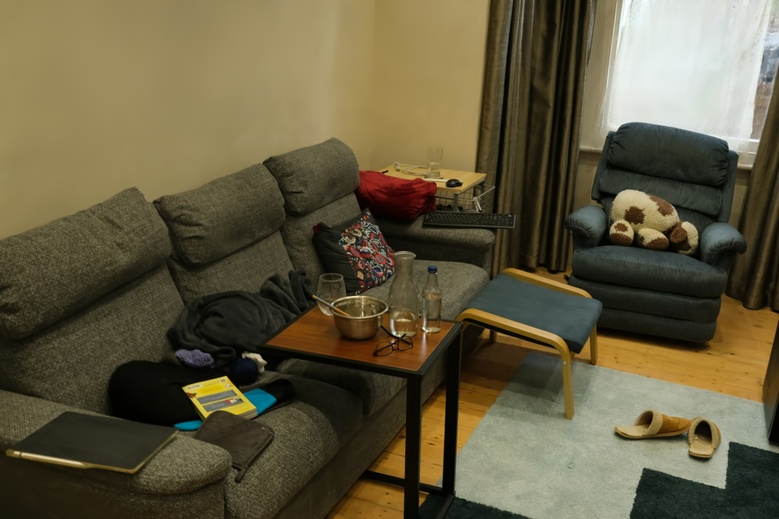} &
		\includegraphics[width=\mytmplen]{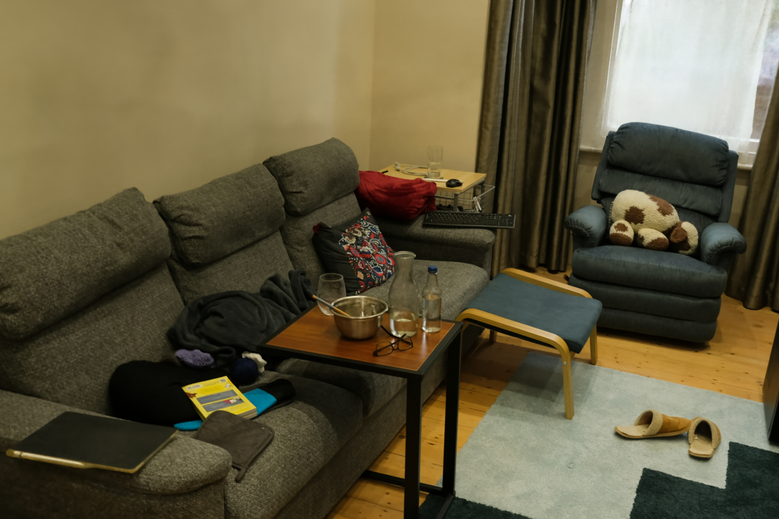} &
		\includegraphics[width=\mytmplen]{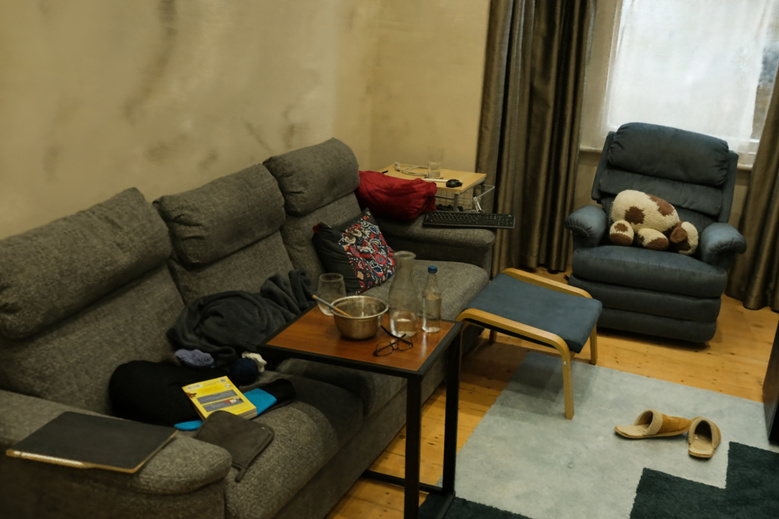} 
		\\
		\zoomin{figures/src/results2/playroom/gt_gt_014_s.png}{1.2}{1.4}{0.55cm}{0.55cm}{1cm}{\mytmplen}{3}{green}&
		\zoomin{figures/src/results2/playroom/laplacians_00014.png}{1.2}{1.4}{0.55cm}{0.55cm}{1cm}{\mytmplen}{3}{green} &
		\zoomin{figures/src/results2/playroom/gaussians_DSC05686_s.png}{1.2}{1.4}{0.55cm}{0.55cm}{1cm}{\mytmplen}{3}{green}&
		\zoomin{figures/src/results2/playroom/mipnerf360_color_014_s.png}{1.2}{1.4}{0.55cm}{0.55cm}{1cm}{\mytmplen}{3}{green}&
		\zoomin{figures/src/results2/playroom/igp_out_0014_s.png}{1.2}{1.4}{0.55cm}{0.55cm}{1cm}{\mytmplen}{3}{green}
\\
		\includegraphics[width=\mytmplen]{figures/src/results2/drjohnson/gt_IMG_6366_s.png} &
		\includegraphics[width=\mytmplen]{figures/src/results2/drjohnson/laplacians_00008.png} &
		\includegraphics[width=\mytmplen]{figures/src/results2/drjohnson/gaussians_IMG_6366_s.png} &
		\includegraphics[width=\mytmplen]{figures/src/results2/drjohnson/mipnerf360_color_008_s.png} &
		\includegraphics[width=\mytmplen]{figures/src/results2/drjohnson/igp_out_0008_s.png} 
\\
		\includegraphics[width=\mytmplen]{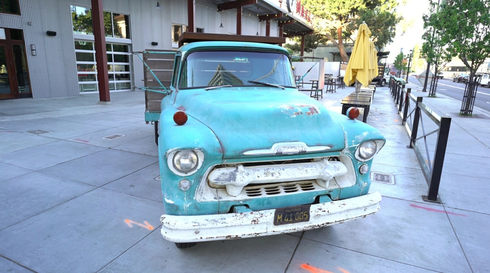} &
		\includegraphics[width=\mytmplen]{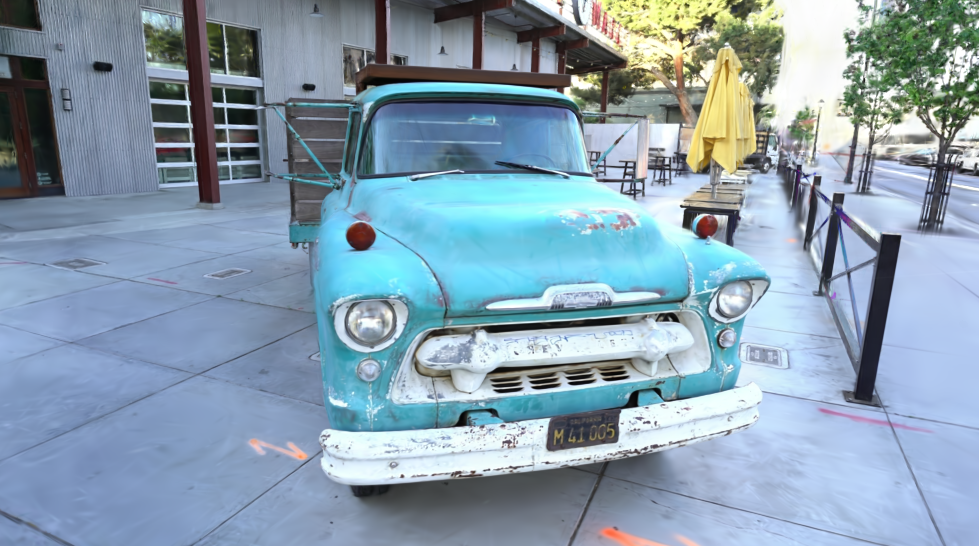} &
		\includegraphics[width=\mytmplen]{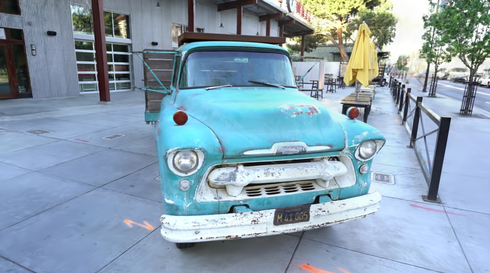} &
		 		\begin{tikzpicture}
			\node[anchor=south west,inner sep=0] (image) at (0,0) {\includegraphics[width=\mytmplen]{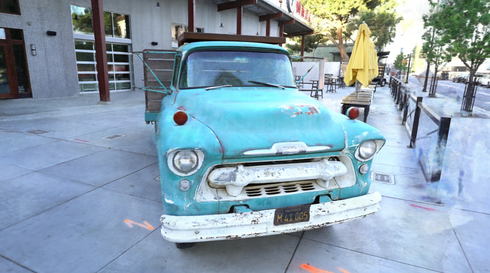}};
		\end{tikzpicture} &
		\includegraphics[width=\mytmplen]{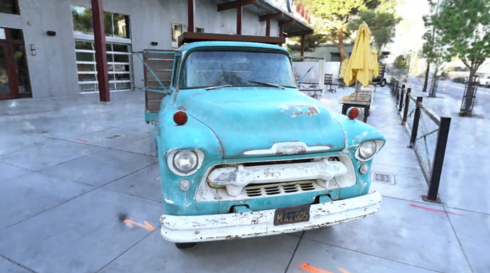}
\\
		\includegraphics[width=\mytmplen]{figures/src/results2/train/gt_gt_0008_s.png} &
		\includegraphics[width=\mytmplen]{figures/src/results2/train/laplacians_00008.png} &
		\includegraphics[width=\mytmplen]{figures/src/results2/train/gaussians_00065_s.png} &
		\includegraphics[width=\mytmplen]{figures/src/results2/train/mipnerf360_color_008_s.png} &
		\includegraphics[width=\mytmplen]{figures/src/results2/train/igp_out_0008_s.png}
		\\
	\end{tabular}
    \vspace{-.3cm} %
\caption{
    \label{fig:comparisons-supp}
    \textbf{Comparative Visualization Across Methods.} Displayed are side-by-side comparisons between our proposed method and established techniques alongside their respective ground truth imagery. The depicted scenes are ordered as follows: \textsc{Bicycle}, \textsc{Garden}, \textsc{Stump}, \textsc{Counter}, and \textsc{Room} from the Mip-NeRF360 dataset; \textsc{Playroom} and \textsc{DrJohnson} from the Deep Blending dataset, and \textsc{Truck} and \textsc{Train} from Tanks\&Temples. Subtle variances in rendering quality are accentuated through zoomed-in details. It might be difficult to see differences between \methodname and Gaussians because they have almost the same PSNR (despite \methodname requiring 50\% less memory).}
	
\end{figure*}

%% file: figures/detailed.tex
\begin{figure*}[t] 
    \centering
    \includegraphics[width=1.0\linewidth]{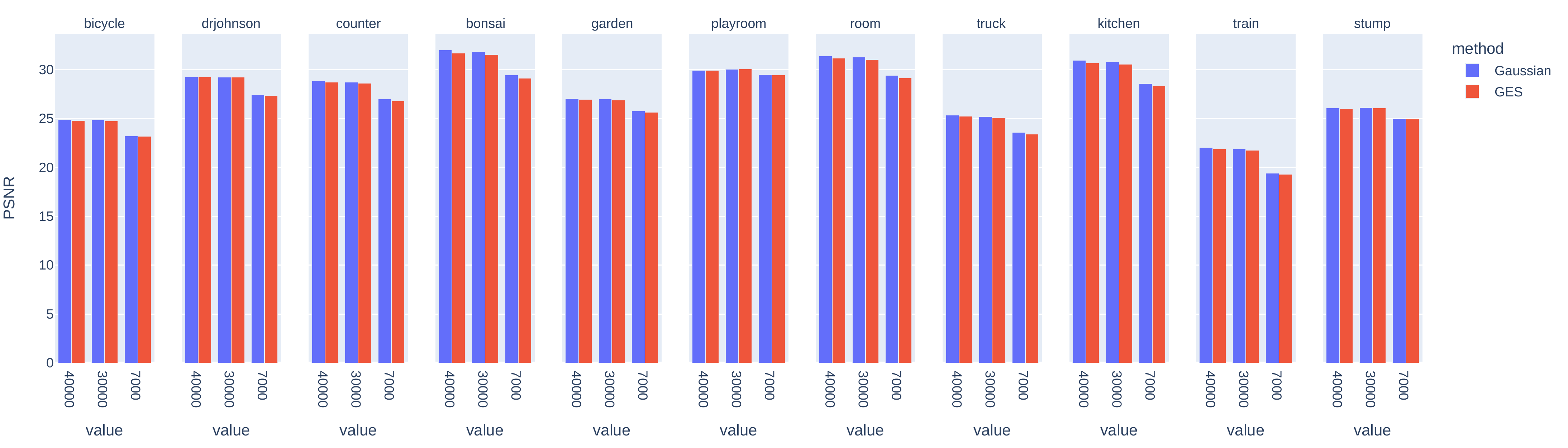}
    \includegraphics[width=1.0\linewidth]{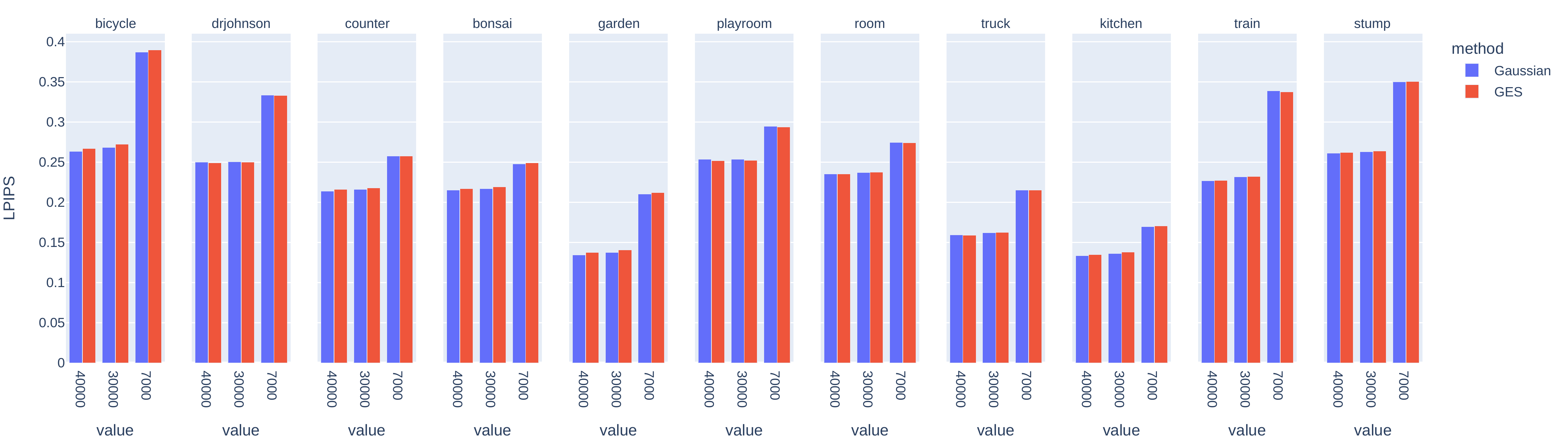}
    \includegraphics[width=1.0\linewidth]{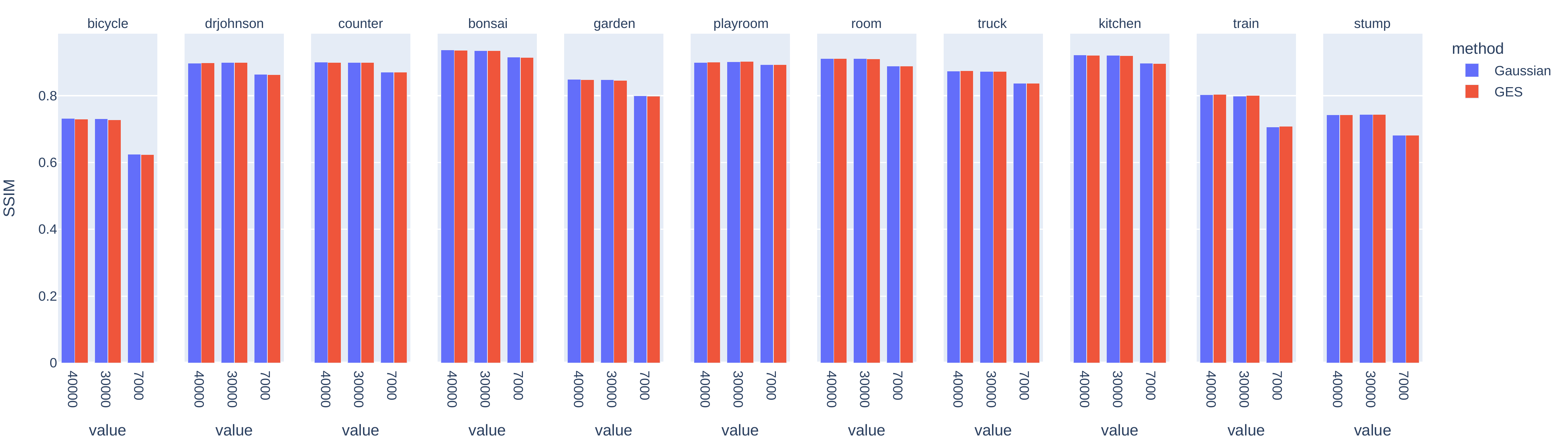}
    \includegraphics[width=1.0\linewidth]{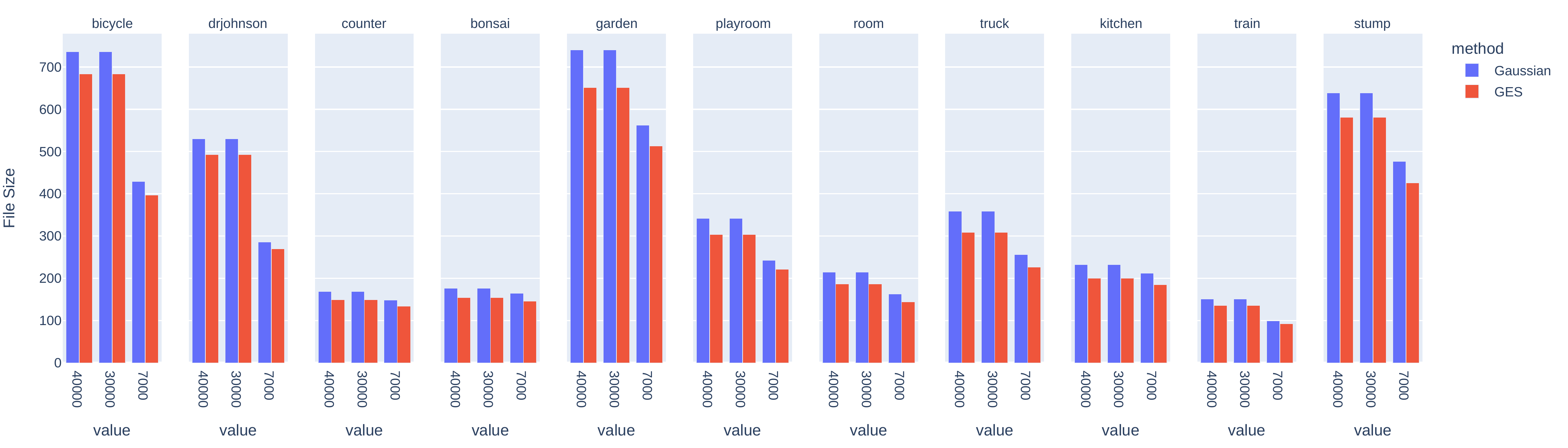}
    \caption{
    \textbf{Detailed Per Scene Results On MipNeRF 360 for Different Iteration Numbers.} We show PSNR, LPIPS, SSIM, and file size results for every single scene in MIPNeRF 360 dataset \cite{MipNeRF-360} of our \methodname and re-running the Gaussian Splatting \cite{gaussiansplatter} baseline with the \textit{exact same} hyperparameters of our  \methodname and on different number of iterations. 
    }
    \label{supfig:detailed}
    \end{figure*}

%% file: figures/laplacian_supp.tex
\begin{figure*}[t] 
    \centering
        \setlength\mytmplen{.243\linewidth}
    \resizebox{1.0\linewidth}{!}{
    \begin{tabular}{cccc}
    \tabcolsep=0.00cm
    Ground Truth& \textbf{\methodname (full)} & \methodname (w/o $\mathcal{L}_{\omega}$ )  & Gaussian Splatting \cite{gaussiansplatter} \\ 
\zoomin{figures/src/laplace/2/gt.png}{3.0}{2.1}{3.6cm}{0.55cm}{1cm}{\mytmplen}{3}{red} &
\zoomin{figures/src/laplace/2/full.png}{3.0}{2.1}{3.6cm}{0.55cm}{1cm}{\mytmplen}{3}{red} &
\zoomin{figures/src/laplace/2/no_mask.png}{3.0}{2.1}{3.6cm}{0.55cm}{1cm}{\mytmplen}{3}{red} &
\zoomin{figures/src/laplace/2/gaussian.png}{3.0}{2.1}{3.6cm}{0.55cm}{1cm}{\mytmplen}{3}{red} \\
    \end{tabular} }
    \caption{
    \textbf{Frequency-Modulated Loss Effect.} We show the effect of the frequency-modulated image loss $\mathcal{L}_{\omega}$ on the performance on novel views synthesis. Note how adding this $\mathcal{L}_{\omega}$ improves the optimization in areas where large contrast exists or where a smooth background is rendered.   
    }
    \label{figsup:laplace-supp}
    \end{figure*}

%% file: tables/ablation_full_tanks.tex
\begin{table}[t]
    \centering
      \resizebox{0.99\linewidth}{!}{%
          \tabcolsep=0.15cm
    \begin{tabular}{lcccc}
    \toprule
     \multirowcell{1}{Ablation Setup}  & \multirowcell{1}{$PSNR^\uparrow$} & \multirowcell{1}{$SSIM^\uparrow$} & \multirowcell{1}{$LPIPS^\downarrow$} & \multirowcell{1}{\textbf{Size (MB)}$^\downarrow$} \\
    \midrule
    Gaussians   &  23.14 & 0.841 & 0.183   &  411 \\
    GES w/o $\mathcal{L}_{\omega}$  & 23.54 & 0.836 & 0.197 & 254 \\
        GES w/ random $\beta$ init. & 23.37 & 0.836 & 0.198 & 223\\
             GES w/ $\beta=2$ init.  &  23.35 & 0.836 &  0.198 & \textbf{222} \\

    \bottomrule
    \end{tabular}
    \vspace{-4pt} 
    }   
    \footnotesize
    \caption{\textbf{Ablation Study on Novel View Synthesis.} We study the impact of several components in \methodname on the reconstruction quality and file size in the Tanks \& Temples dataset.}
    \label{tab:ablation-general-tanks}
    \end{table}